\documentclass[letterpaper,twoside,10pt]{report}

\usepackage[numbers,square,sort&compress]{natbib}
\usepackage[nottoc]{tocbibind} 

\usepackage[euler-digits, T1]{eulervm}
\usepackage[usenames,dvipsnames]{color}
\usepackage[margin=0.85in]{geometry}
\usepackage[
   colorlinks%
   ,plainpages=false
   ,hypertexnames=true
   ,naturalnames
   ,hyperindex
   ,citecolor=OliveGreen
   ,urlcolor=RoyalBlue
   ,pdfauthor={Dimitris Diochnos}
   ,pdftitle={Commonsense Reasoning and Large Network Analysis: A Computational Study of ConceptNet 4}
   ,pdfsubject={ConceptNet 4}
]{hyperref}
\usepackage{datetime}

\usepackage{bold-extra} 
\usepackage{upgreek}
\usepackage{enumitem}
\usepackage{graphicx}
\usepackage{caption}
\usepackage{subcaption}
\usepackage{wrapfig}
\usepackage{url}
\usepackage{multicol,multirow}
\usepackage{xspace}
\usepackage{upgreek}
\usepackage{ifthen}
\usepackage{rotating}
\usepackage{nicefrac}

\usepackage{amsmath}
\usepackage{amsthm,amssymb} 

\newtheorem{thm}{Theorem}

\newtheorem{prop}[thm]{Proposition}
\newtheorem{definition}{Definition}
\newtheorem{remark}{Remark}

\newtheorem{convention}[thm]{Convention}

\usepackage{pifont}
\newcommand{\cmark}{\ding{51}\xspace}%
\newcommand{\xmark}{\ding{55}\xspace}%

\newcommand{\conceptnet}{\ensuremath{\mbox{\texttt{ConceptNet 4}}}\xspace}
\newcommand{\igraph}{\ensuremath{\mbox{\texttt{igraph}}}\xspace}

\newcommand{\cfinder}{\ensuremath{\mbox{\texttt{CFinder}}}\xspace}
\newcommand{\plplot}{\ensuremath{\mbox{\texttt{plplot}}}\xspace}

\newcommand{\inlinecycle}{\ensuremath{\kern 0.2em %
\substack{\longrightarrow\\\longleftarrow} \kern 0.2em}}

\newcommand{\dbtext}[1]{\textcolor{Green}{\texttt{#1}}\xspace}

\providecommand{\abs}[1]{\ensuremath{\lvert#1\rvert}\xspace}

\newcommand{\OO}[1]{\ensuremath{\mathcal O\left(#1\right)}\xspace}

\newcommand{\surj}{\mbox{\kern .2em %
$\rightarrow$\hspace*{-.8em}$\rightarrow$\kern .2em}}
\newcommand{\bij}{\rightarrowtail\kern -.8em\rightarrow}

\newenvironment{code}
{\begin{quote}
\ttfamily 
\rule{\linewidth}{2pt}}
{\newline\rule{\linewidth}{2pt}
\end{quote}}

\newcommand{\NN}{\ensuremath{\mathbb N}\xspace}
\newcommand{\ZZ}{\ensuremath{\mathbb Z}\xspace}

\newcounter{eg}
\newtheorem{example}{Example}

\renewcommand{\thefootnote}{\fnsymbol{footnote}}
\title{Commonsense Reasoning and Large Network Analysis:\\ 
A Computational Study of ConceptNet 4 \protect\footnote{ Research supported by NSF Grants CCF-0916708 and IIS-0747369.}}
\author{Dimitrios I.~Diochnos\\
Department of Mathematics, Statistics, and Computer Science\\
University of Illinois at Chicago, Chicago IL 60607, USA}
\date{
}

\begin{document}

\pagenumbering{roman}

\maketitle

\bibliographystyle{plain}

\renewcommand{\thefootnote}{\arabic{footnote}}


\begin{abstract}
In this report a computational study of \conceptnet is performed 
using tools from the field of network analysis.
Part I describes the process of extracting the data from the SQL database that is available online,
as well as how the closure of the input among the assertions in the English language is computed.
This part also performs a validation of the input as well as checks for the consistency of the entire database.
Part II investigates the structural properties of \conceptnet. 
Different graphs are induced from the knowledge base by fixing different parameters.
The degrees and the degree distributions are examined, the number and sizes of connected components,
the transitivity and clustering coefficient, the cores, information related to shortest paths in the
graphs, and cliques.
%
%
Part III investigates non-overlapping, as well as overlapping communities that are found in \conceptnet.
%
%
Finally, Part IV describes an investigation on rules.
\end{abstract}

\newpage
\renewcommand{\abstractname}{Acknowledgements}
\begin{abstract}
The author would like to thank Robert H.~Sloan and Gy{\"o}rgy Tur{\'a}n for suggesting him this subject. 
The author would also like to thank Tanya Y.~Berger-Wolf for giving him access to a computer
with sufficiently enough memory to perform the computations required in Chapter \ref{ch:communities:overlapping}.
\end{abstract}

\pagenumbering{roman}
\setcounter{page}{1}

\tableofcontents


\newpage
\pagenumbering{arabic}
\setcounter{page}{1}
\part{Closure of the Input, Validity, and Consistency}

\chapter{Validity and Closure of the Database}\label{chapter:closure}
Our aim is to compute the minimal data-set implied by the assertions of the English language,
extract it from the database, and store it in files of our own format.
Towards this direction we read the table of assertions (conceptnet\_assertion) and keep
the entries that have their \texttt{language\_id} set to \dbtext{en}.
According to Table \ref{tbl:csv:assertion} in Appendix \ref{chapter:appendix:conceptnet-db}, every assertion
is associated with entries from the database tables conceptnet\_concept (Table \ref{tbl:csv:concept}), 
conceptnet\_relation (Table \ref{tbl:csv:relation}), nl\_frequency (Table \ref{tbl:csv:frequency}),
conceptnet\_frame (Table \ref{tbl:csv:frame}), conceptnet\_surfaceform (Table \ref{tbl:csv:surfaceform}), 
and conceptnet\_rawassertion (Table \ref{tbl:csv:rawassertion}).
Through conceptnet\_rawassertion the assertions are also associated with the actual sentences
which are located in the table corpus\_sentence (Table \ref{tbl:csv:surfaceform}).
Moreover, we do not need any other table from the database, as the important entries from all the above
tables are contained in among these tables.

\medskip

It turns out that reading once the assertions and then all the entries referenced from the assertions
in the English language is not enough to produce a minimal consistent data-set. 
Section \ref{sec:closure:computation} explains why, and gives a high-level overview of the process 
that we follow in order to compute the closure of the data-set implied by the assertions of the English language.
However, before we describe these reasons we mention which fields we are going to keep from each table
of the original \conceptnet database.

\paragraph{Fields Retained from Each Database Table.}
A description of the information of the fields that we retain in every case follows.
\begin{description}[noitemsep,leftmargin=8mm,topsep=0.5mm]
 \item [conceptnet\_assertion:] Everything but the \texttt{language\_id} field.
 \item [conceptnet\_concept:] \texttt{id}, \texttt{text}.
 \item [conceptnet\_relation:] Everything.
 \item [nl\_frequency:] Everything but the \texttt{language\_id} field.
 \item [conceptnet\_frame:] The fields \texttt{question\_yn}, \texttt{question1}, and \texttt{question2}
are null in the entire database, hence we can safely ignore them. We are also dropping the
\texttt{language\_id} field as well as the \texttt{goodness} field. 
 \item [conceptnet\_surfaceform:] We retain the information of the fields 
\texttt{id}, \texttt{concept\_id}, and \texttt{text}.
 \item [conceptnet\_rawassertion:] We retain the information of the fields \texttt{id},
\texttt{sentence\_id}, \texttt{assertion\_id}, \texttt{surface1\_id}, \texttt{surface2\_id},
\texttt{frame\_id}, and \texttt{score}.
 \item [corpus\_sentence:] We retain the information of the fields \texttt{id}, \texttt{text},
and \texttt{score}.
\end{description}

\section{High Level Description for the Computation of the Closure}\label{sec:closure:computation}
In this section we give a high-level description of the process that we follow in order to compute
the closure of the data-set implied by the assertions of the English language.
Ultimately we also want to use \emph{igraph} \citep{igraph} to aid our analysis of the networks
induced by \conceptnet.

\subsection{First Pass}
During the first pass of the tables in the database we read the IDs of all the objects
and store in matrices which IDs actually appear in the database. This is an important step
since some references in the database are inconsistent. For example, some \texttt{best\_raw\_id}'s
found in the table conceptnet\_assertion point essentially to nowhere, since these particular
IDs do not appear anywhere in the conceptnet\_rawassertion table.

\subsection{Second Pass} 
During the second pass of the tables in the database we extract the entries of
assertions in the English language and all the entries from the other tables that are referenced
from assertions. Moreover, we also extract all the sentences that are indirectly referenced by the 
assertions through the raw assertions. Ideally, one would expect that this process is enough in order
to compute the minimal closure implied by the assertions in the English language. However, this is 
not the case. Below we describe the issues that arise after the first pass.

\paragraph{Null Entries.} 
Some fields in the tables of the database do not have
data associated to them. In the case of assertions in the English language
these entries can appear in the fields \texttt{best\_frame\_id},
\texttt{best\_surface1\_id}, \texttt{best\_surface2\_id}, and \texttt{best\_raw\_id}.
The assertion with the minimum ID that has \texttt{best\_frame\_id} equal to null is $344873$.
The assertion with the minimum ID that has \texttt{best\_surface1\_id} and \texttt{best\_surface2\_id} 
equal to null is $885221$. The assertion with the minimum ID that has \texttt{best\_raw\_id}
equal to null is $320114$.

\paragraph{Undefined Raw Assertion IDs.} 
There is an inconsistency problem regarding the IDs 
for the raw assertions that are mentioned in some entries of the assertions table. It turns out
that $39312$ different \texttt{best\_raw\_id}'s are not defined in the table of raw assertions;
i.e.~the IDs do not appear in the conceptnet\_rawassertion table. The assertion instance with the
minimum ID is $962$ which points to \texttt{best\_raw\_id} $965$.

\paragraph{Duplicate Raw Assertion IDs.} 
Multiple assertions may point to the same raw assertion.
Hence, not only we have assertions that have their \texttt{best\_raw\_id} equal to null or
undefined, but the map between assertions and raw assertions is actually a \emph{surjection}.

\paragraph{Discrepancies due to Frames.} 
When we are able to read the frames in the field
\texttt{best\_frame\_id} for an assertion (see Table \ref{tbl:csv:assertion})
we would expect that the \texttt{relation\_id} and \texttt{frequency\_id}
mentioned in the relevant entry of the conceptnet\_frame table (Table \ref{tbl:csv:frame}) 
agree with the entries found in the assertion. 
However, it turns out that this is not necessarily the case for both of these values.
For more information see Section \ref{sec:closure:first-pass:frames}.

\paragraph{Discrepancies due to Surface Forms.} 
When we are able to read the surface forms in the
fields \texttt{best\_surface1\_id} and \texttt{best\_surface2\_id} for an assertion 
(see Table \ref{tbl:csv:assertion}) we would again expect that the \texttt{concept\_id} 
mentioned in the relevant entry of the conceptnet\_surfaceform table (Table \ref{tbl:csv:surfaceform})
agree with the respective \texttt{concept1\_id} and \texttt{concept2\_id} entries in the 
conceptnet\_assertion table (Table \ref{tbl:csv:assertion}). However, it turns out that this is
not necessarily the case. In fact, this time the nature of disagreement can be dual:
\begin{itemize}[noitemsep,leftmargin=8mm,topsep=0.5mm]
 \item the IDs for the concepts disagree and both IDs are mentioned in the assertions, or
 \item the IDs for the concepts disagree but the IDs coming from the conceptnet\_surfaceform table
 are not mentioned among the assertions in the English language. 
\end{itemize}
The second case of disagreement above forces us to perform a second pass through the data so that
we can collect all the data for the $388$ concept IDs that did not appear during our first pass from the assertions 
in the English language. For more information about the quantitative properties of the discrepancies
see Section \ref{sec:closure:first-pass:surfaceforms}.

\paragraph{Discrepancies due to Raw Assertions.} 
When we are able to read the raw assertion ID in the
fields \texttt{best\_raw\_id} for an assertion 
(see Table \ref{tbl:csv:assertion}) we would again expect that the 
entries \texttt{surface1\_id}, \texttt{surface2\_id}, and \texttt{frame\_id}
mentioned in the relevant entries of the conceptnet\_rawassertion table (Table \ref{tbl:csv:rawassertion})
agree with the \texttt{best\_surface1\_id}, \texttt{best\_surface2\_id},
and \texttt{best\_frame\_id} entries in the 
conceptnet\_assertion table (Table \ref{tbl:csv:assertion}). However, it turns out that this is
not necessarily the case either.
Similarly to the case above where we find $388$ concepts not mentioned among the assertions in the English language, 
this process uncovers $540$ surface form IDs that were not
best surfaces for any assertion in the English language.
Moreover, note that from our earlier remark the map of assertions
to raw assertions is a surjection, and hence it is guaranteed that there is a discrepancy
in the \texttt{assertion\_id} entry of the conceptnet\_rawassertion table.
For more information see Section \ref{sec:closure:first-pass:rawassertions}.

\paragraph{Discrepancies on the Score Entries.} 
Most of the assertions have a valid raw assertion
associated to them. Moreover, every raw assertion is associated with an actual sentence.
Inspecting the Tables \ref{tbl:csv:assertion}, \ref{tbl:csv:rawassertion}, and \ref{tbl:csv:sentence}
we see that each of the tables conceptnet\_assertion, conceptnet\_rawassertion, and corpus\_sentence
has an entry with the score associated score. One would expect that all these three scores actually
agree with each other whenever we have a valid chain of the form: 
assertion $\longrightarrow$ raw assertion $\longrightarrow$ sentence. However, it turns out that this
is not true either. For more information see Section \ref{sec:closure:first-pass:scores}.

\paragraph{End of Second Pass}
At the end of second pass we can observe that $459662$ assertions have
all their indicators equal to zero, out of which $450205$ have positive score
and the rest $9457$ have non-positive score.
Given the fact that we have the mapping $\mbox{assertions} \surj \mbox{raw assertions}$
we could allow the indicator for the raw assertions to achieve, apart from zero, the value $18$ as well
(see Table \ref{tbl:rawassertions:indicator-distribution}). However, 
the numbers mentioned above \emph{do not change at all}. 

\subsection{Third Pass}
In the third pass we parse the data in the tables conceptnet\_concept and conceptnet\_surfaceform.
This allows us to load the concepts and the surface forms that were raised from the previous pass.
In theory, it could be the case that these new additional surface forms were referring to concepts
that have not been raised yet from the previous passes, and hence we would require one more pass 
on the conceptnet\_concept table to add these last concepts. However, this is not the case.
In other words, these newly introduced surface forms from the last pass do not refer to concepts
that we have not encountered earlier. 
Hence, this third pass is the last pass that we perform on the tables of the database.

\section{First Pass: Validating IDs}
There is not much to be said about the first pass. 
We parse all $8$ tables of the database and record which IDs of these
objects are actually valid IDs in the sense that references from other
tables to these objects are guaranteed to return a result.
The issue that forces us to follow this direction is the fact that some
\texttt{best\_raw\_id}'s found in the conceptnet\_assertion table actually
point to nowhere, since we can not find raw assertions with these specific IDs
in the table conceptnet\_rawassertion.

\section{Second Pass: Discrepancies due to Frames}\label{sec:closure:first-pass:frames}
Looking at Tables \ref{tbl:csv:assertion} and \ref{tbl:csv:frame} we would expect
the relations and the frequencies mentioned in the associated entries to agree.
However, this is not always true in both cases. Moreover, when we have discrepancies
among the relations or the frequencies, these other values appear in some other assertion
in the English language. We mention this because in Section \ref{sec:closure:first-pass:surfaceforms}
similar discrepancies will occur only among values that can be observed based on the
input from the assertions in the English language.

\medskip 

Regarding the relations, there are $816$ assertions that have 
\texttt{best\_frame\_id} equal to null. Among the not null
entries, in $564445$ assertions the relation ID from the conceptnet\_assertion
table agrees with the respective relation ID from the relevant entry in the 
conceptnet\_frame table. The rest $833$ assertions have relation ID different
from the relevant entry mentioned in the table conceptnet\_frame.

Regarding the frequencies, there are again $816$ 
assertions that have \texttt{best\_frame\_id} equal to null. Among the not null
entries, in $562798$ assertions the frequency ID from the conceptnet\_assertion
table agrees with the respective frequency ID from the relevant entry in the 
conceptnet\_frame table. The rest $2480$ assertions have a frequency ID different
from the relevant entry mentioned in the table conceptnet\_frame.

\begin{remark}[Interesting Phenomenon]
These two fields never disagree simultaneously with the entries
found in the conceptnet\_assertion table. This implies that if the relation
changes, then the frequency, which expresses the extent to which the relation holds,
does not change. Moreover, if the frequency changes, i.e.~the extent to which the
relation holds, then the relation does not change.
\end{remark}

The above information is captured in Table \ref{tbl:frames:indicator}.

\begin{table}[ht]
\caption{Distribution of the frame indicator.}\label{tbl:frames:indicator}
\begin{center}
\begin{tabular}{|c|l|r|}\hline
indicator & \multicolumn{1}{|c|}{description} & entries \\\hline\hline
0 & \texttt{best\_frame\_id} is not null, relations agree, frequencies agree       & $564445$ \\\hline
1 & \texttt{best\_frame\_id} is not null, relations agree, frequencies disagree    &   $2480$ \\\hline
2 & \texttt{best\_frame\_id} is not null, relations disagree, frequencies agree    &    $833$ \\\hline
3 & \texttt{best\_frame\_id} is not null, relations disagree, frequencies disagree &      $0$ \\\hline
4 & \texttt{best\_frame\_id} is null                                               &    $816$ \\\hline
\multicolumn{2}{r|}{sum} & $566094$ \\\cline{3-3}
\end{tabular}
\end{center}
\end{table}

\paragraph{Examples.}
We give one example for each case presented in Table \ref{tbl:frames:indicator}.
\begin{description}
 \item [Indicator = 0.] Assertion $2$ associates the concepts $5$ (\dbtext{something}) and $6$ (\dbtext{to})
with relation $6$ (\dbtext{AtLocation}) and frequency $1$ (which has value $5 (> 0)$ and the empty string for description). 
The best raw assertion for this assertion has ID $3$
which is associated with the sentence \dbtext{Somewhere something can be is next to}.
The best frame for the assertion is $3$ which is \dbtext{Somewhere \{1\} can be is next \{2\}}
and the relation associated to that frame is again $6$ (\dbtext{AtLocation}),
as well as the frequency is $1$.

 \item [Indicator = 1.] Assertion $36294$ associates the concepts $481$ (\dbtext{milk})
and $1503$ (\dbtext{refrigerator}) with relation $6$ (\dbtext{AtLocation}) and 
frequency $1$ (which has value $5$ and the empty string for description). 
The best raw assertion for this assertion has ID $368705$ which is associated with the
sentence \dbtext{Something you find the refrigerator is milk.}.
The best frame for the assertion is $2761$ which is \dbtext{Something you find \{2\} is \{1\}.}.
However, this frame is associated with relation $6$ (\dbtext{AtLocation}) and frequency $25$
(which has value $-5 < 0$ and the string description is \dbtext{not})!

 \item [Indicator = 2.] Assertion $17691$ associates the concepts $18845$ (\dbtext{bread knife})
and $13506$ (\dbtext{cut bread}) with relation $7$ (\dbtext{UsedFor}) and 
frequency $1$ (which has value $5$ and the empty string for description). 
The best raw assertion for this assertion has ID $18168$ which is associated with the
sentence \dbtext{bread knives are for cutting bread}.
The best frame for the assertion is $40$ which is \dbtext{\{1\} are \{2\}}.
However, this frame is associated with frequency $1$ and relation $5$ (\dbtext{IsA})!

 \item [Indicator = 3.] No instances.

 \item [Indicator = 4.] Assertion $344873$ associates the concepts $217239$ (\dbtext{don't})
and $217240$ (\dbtext{manipulate gene}) with relation $8$ (\dbtext{CapableOf}).
However, the \texttt{best\_frame\_id} field for this assertion is null. The \texttt{best\_raw\_id}
field is also null.
\end{description}

\section{Second Pass: Discrepancies due to Surface Forms}\label{sec:closure:first-pass:surfaceforms}
Every assertion has two best surface forms; one for each concept.
Hence, \texttt{best\_surface1\_id} describes the concept in the entry
\texttt{concept1\_id}, and \texttt{best\_surface2\_id} describes the
concept \texttt{concept2\_id}. On the other hand, every entry in the
table conceptnet\_surfaceform associates each surface form with a concept
and also provides a string representation of that concept in that particular surface form.

Ideally, we would expect \texttt{concept1\_id} or \texttt{concept2\_id} from the table
conceptnet\_assertion to match respectively with the entry of \texttt{concept\_id}
of the respective surface ID (i.e.~respectively \texttt{best\_surface1\_id} and 
\texttt{best\_surface2\_id}). However, it turns out that this is not the case.

\medskip

First of all, the table conceptnet\_assertion has $810$ entries in the English language
where the \texttt{best\_surface1\_id} and \texttt{best\_surface2\_id} are simultaneously null
(all other entries in the English language do not have null entries for either of these two parameters).
Moreover, even when we do have valid surface IDs in the respective entries of the table
conceptnet\_assertion, the \texttt{concept\_id} on the relevant entry of the table 
conceptnet\_surfaceform may point to a concept with ID not matching the respective one
obtained through conceptnet\_assertion. In fact, it may not match the relevant concept ID
of the conceptnet\_assertion table in two different ways:
\begin{itemize}[noitemsep]
 \item the concept IDs differ and both are part of the input, or
 \item the concept IDs differ but the concept with ID \texttt{concept\_id} is \emph{missing} (i.e.~does not appear)
among all the assertions of the English language.
\end{itemize}
The last case may happen either because that concept ID did not appear in any assertion but does
appear in the conceptnet\_concept table, or (I have not checked and I find quite unlikely the following)
because no such concept ID appears in the conceptnet\_concept table (similar phenomenon to that
observed on the IDs of raw assertions).

As a consequence of the above we can distinguish $16$ cases which are shown in Table \ref{tbl:surfaceforms:indicator}.

\begin{table}[h]
\caption{Distribution of the indicator for surface forms.}\label{tbl:surfaceforms:indicator}
\begin{center}
\begin{tabular}{|c|l|r|}\hline
\multicolumn{1}{|c|}{indicator} & \multicolumn{1}{c|}{description} & \multicolumn{1}{c|}{entries} \\\hline\hline
 0 & surface 1 not null and concept IDs agree; surface 2 not null and concept IDs agree       & 561530 \\\hline
 1 & surface 1 not null and concept IDs agree; surface 2 not null and concept IDs disagree    &   2513 \\\hline
 2 & surface 1 not null and concept IDs agree; surface 2 not null and concept ID missing      &    383 \\\hline
 3 & surface 1 not null and concept IDs agree; surface 2 is null                              &      0 \\\hline
 4 & surface 1 not null and concept IDs disagree; surface 2 not null and concept IDs agree    &    814 \\\hline
 5 & surface 1 not null and concept IDs disagree; surface 2 not null and concept IDs disagree &     28 \\\hline
 6 & surface 1 not null and concept IDs disagree; surface 2 not null and concept ID missing   &      3 \\\hline
 7 & surface 1 not null and concept IDs disagree; surface 2 is null                           &      0 \\\hline
 8 & surface 1 not null and concept ID missing; surface 2 not null and concept IDs agree      &     13 \\\hline
 9 & surface 1 not null and concept ID missing; surface 2 not null and concept IDs disagree   &      0 \\\hline
10 & surface 1 not null and concept ID missing; surface 2 not null and concept ID missing     &      0 \\\hline
11 & surface 1 not null and concept ID missing; surface 2 is null                             &      0 \\\hline
12 & surface 1 is null; surface 2 not null and concept IDs agree                              &      0 \\\hline
13 & surface 1 is null; surface 2 not null and concept IDs disagree                           &      0 \\\hline
14 & surface 1 is null; surface 2 not null and concept ID missing                             &      0 \\\hline
15 & surface 1 is null; surface 2 is null                                                     &    810 \\\hline
\multicolumn{2}{r|}{sum} & 566094 \\\cline{3-3} 
\end{tabular}
\end{center}
\end{table}


\paragraph{Examples.}
Below we give the first example (as we parse the assertions in order) for each case of the indicator
variable for surface forms.
\begin{description}
 \item [Indicator = 0.] Assertion $2$ relates concepts $5$ (\dbtext{something}) and $6$ (\dbtext{to}).
The best surface forms have IDs respectively $5$ and $6$ (same numbers as the concept IDs; it just happened).
These surface forms in turn point to the same concepts that we have in the assertion; i.e.~$5$ and $6$ respectively
and the text representation of the concepts obtained from the conceptnet\_concept table is the same as the text
representation obtained from the conceptnet\_surfaceform table.

Note that the text representation of the concepts need not be the same in general, even in this class.
An example in this direction is assertion $7$ which relates concepts $13$ (\dbtext{strike match}) and 
$14$ (\dbtext{burn down church}). The best surface form IDs respectively are $14$ and $15$.
The concept IDs obtained from these surface forms agree respectively with the concept IDs that we have in the assertion.
However, the strings that we get through the surface forms are respectively
\dbtext{striking a match} and \dbtext{burning down churches}.

 \item [Indicator = 1.] Assertion $335$ relates concepts $538$ (\dbtext{toothpaste}) and $327340$ (\dbtext{clean one tooth}).
The best surface forms have IDs respectively $565$ and $22753$.
The second surface form disagrees, since it points to the concept $311600$.
The string from the second surface form is \dbtext{cleaning teeth}.

 \item [Indicator = 2.] Assertion $29378$ relates concepts $5$ (\dbtext{something}) and $312273$ (\dbtext{one with all that be}).
The best surface forms have IDs respectively $5$ and $186249$.
The concept obtained from the second surface form has ID $322557$ and is not part of the input 
since it does not appear in any assertion in the English language.
The string from the second surface form is \dbtext{where it should be}.

 \item [Indicator = 3.] No instances.

 \item [Indicator = 4.] Assertion $1464$ relates concepts $1906$ (\dbtext{most people}) and $1121$ (\dbtext{read book}).
The best surface forms have IDs respectively $2173$ and $2174$.
However, the concept obtained from the first surface form has ID $9$. 
The string obtained for that concept from conceptnet\_concept is \dbtext{person}
while the string obtained from the surface form is \dbtext{most people}.

 \item [Indicator = 5.] Assertion $6233$ relates concepts $980$ (\dbtext{movies}) and $7356$ (\dbtext{show theater}).
The best surface forms have IDs respectively $8520$ and $8521$.
This time both concept IDs obtained from the surface forms disagree with the IDs
of the concepts found in the assertion. The surface forms give respectively
concept IDs $213$ and $316392$.
The strings for these concepts from the conceptnet\_concept table are respectively
\dbtext{movie} and \dbtext{shown theater}.
The strings for these concepts from the surface form entries are respectively
\dbtext{Movies} and \dbtext{shown in theaters}.

 \item [Indicator = 6.] Assertion $280329$ relates concepts $17626$ (\dbtext{entertain people})
and $186703$ (\dbtext{make keep friends}).
The best surface forms have IDs respectively $36638$ and $250423$.
The concept IDs obtained from the surface forms are respectively $427797$ and $326698$.
The strings for these concepts from the conceptnet\_concept table are
\dbtext{entertain person} and \dbtext{make keep friend}.
The strings for these concepts from the surface form entries are respectively
\dbtext{entertaining people} and \dbtext{making and keeping friends}.

 \item [Indicator = 7.] No instances.

 \item [Indicator = 8.] Assertion $60579$ relates concepts $49223$ (\dbtext{wah one's hair}) 
and $2697$ (\dbtext{good idea}).
The best surface forms have IDs respectively $63422$ and $63423$.
The concept IDs obtained from the surface forms are respectively $314140$ and $2697$.
The strings for these concepts from the conceptnet\_concept table are
\dbtext{wah hair} and \dbtext{good idea}.
The strings for these concepts from the surface form entries are respectively
\dbtext{Wahing one's hair} and \dbtext{is a good idea}.

 \item [Indicator = 9, \ldots, 14.] No instances.

 \item [Indicator = 15.] Assertion $885221$ relates concepts 
$25036$ (\dbtext{see particular program}) and $643$ (\dbtext{enjoyment}).
The best surface form IDs are null in both cases.
\end{description}

\subsection{Concepts Raised}
This verification process raises $388$ concept IDs, all of which are valid, but were not mentioned
among the assertions in the English language.

\section{Second Pass: Discrepancies due to Raw Assertions}\label{sec:closure:first-pass:rawassertions}
Table \ref{tbl:rawassertions:indicator-distribution} shows the distribution for the indicator.

\begin{table}[p]
\caption{The distribution of the indicator for the discrepancies due to the raw assertions.}\label{tbl:rawassertions:indicator-distribution}
\begin{center}
\begin{tabular}{|c|llll|r|}\hline
indicator & \multicolumn{4}{|c|}{description} & \multicolumn{1}{|c|}{entries} \\\hline\hline
 0 & assertion agrees,    & frame agrees,    & surface 1 agrees,    & surface 2 agrees    & 523306 \\\hline
 1 & assertion agrees,    & frame agrees,    & surface 1 agrees,    & surface 2 disagrees &      0 \\\hline
 2 & assertion agrees,    & frame agrees,    & surface 1 agrees,    & surface 2 missing   &      0 \\\hline
 3 & assertion agrees,    & frame agrees,    & surface 1 disagrees, & surface 2 agrees    &      0 \\\hline
 4 & assertion agrees,    & frame agrees,    & surface 1 disagrees, & surface 2 disagrees &      0 \\\hline
 5 & assertion agrees,    & frame agrees,    & surface 1 disagrees, & surface 2 missing   &      0 \\\hline
 6 & assertion agrees,    & frame agrees,    & surface 1 missing,   & surface 2 agrees    &      0 \\\hline
 7 & assertion agrees,    & frame agrees,    & surface 1 missing,   & surface 2 disagrees &      0 \\\hline
 8 & assertion agrees,    & frame agrees,    & surface 1 missing,   & surface 2 missing   &      0 \\\hline
 9 & assertion agrees,    & frame disagrees, & surface 1 agrees,    & surface 2 agrees    &      0 \\\hline
10 & assertion agrees,    & frame disagrees, & surface 1 agrees,    & surface 2 disagrees &      0 \\\hline
11 & assertion agrees,    & frame disagrees, & surface 1 agrees,    & surface 2 missing   &      0 \\\hline
12 & assertion agrees,    & frame disagrees, & surface 1 disagrees, & surface 2 agrees    &      0 \\\hline
13 & assertion agrees,    & frame disagrees, & surface 1 disagrees, & surface 2 disagrees &      0 \\\hline
14 & assertion agrees,    & frame disagrees, & surface 1 disagrees, & surface 2 missing   &      0 \\\hline
15 & assertion agrees,    & frame disagrees, & surface 1 missing,   & surface 2 agrees    &      0 \\\hline
16 & assertion agrees,    & frame disagrees, & surface 1 missing,   & surface 2 disagrees &      0 \\\hline
17 & assertion agrees,    & frame disagrees, & surface 1 missing,   & surface 2 missing   &      0 \\\hline\hline
18 & assertion disagrees, & frame agrees,    & surface 1 agrees,    & surface 2 agrees    &   1848 \\\hline
19 & assertion disagrees, & frame agrees,    & surface 1 agrees,    & surface 2 disagrees &      0 \\\hline
20 & assertion disagrees, & frame agrees,    & surface 1 agrees,    & surface 2 missing   &      0 \\\hline
21 & assertion disagrees, & frame agrees,    & surface 1 disagrees, & surface 2 agrees    &      0 \\\hline
22 & assertion disagrees, & frame agrees,    & surface 1 disagrees, & surface 2 disagrees &      0 \\\hline
23 & assertion disagrees, & frame agrees,    & surface 1 disagrees, & surface 2 missing   &      0 \\\hline
24 & assertion disagrees, & frame agrees,    & surface 1 missing,   & surface 2 agrees    &      0 \\\hline
25 & assertion disagrees, & frame agrees,    & surface 1 missing,   & surface 2 disagrees &      0 \\\hline
26 & assertion disagrees, & frame agrees,    & surface 1 missing,   & surface 2 missing   &      0 \\\hline
27 & assertion disagrees, & frame disagrees, & surface 1 agrees,    & surface 2 agrees    &      0 \\\hline
28 & assertion disagrees, & frame disagrees, & surface 1 agrees,    & surface 2 disagrees &    189 \\\hline
29 & assertion disagrees, & frame disagrees, & surface 1 agrees,    & surface 2 missing   &    607 \\\hline
30 & assertion disagrees, & frame disagrees, & surface 1 disagrees, & surface 2 agrees    &      0 \\\hline
31 & assertion disagrees, & frame disagrees, & surface 1 disagrees, & surface 2 disagrees &      0 \\\hline
32 & assertion disagrees, & frame disagrees, & surface 1 disagrees, & surface 2 missing   &      0 \\\hline
33 & assertion disagrees, & frame disagrees, & surface 1 missing,   & surface 2 agrees    &      0 \\\hline
34 & assertion disagrees, & frame disagrees, & surface 1 missing,   & surface 2 disagrees &      0 \\\hline
35 & assertion disagrees, & frame disagrees, & surface 1 missing,   & surface 2 missing   &      0 \\\hline
\multicolumn{5}{r|}{partial sum}                                                          & 525950 \\\cline{6-6}
\multicolumn{6}{c}{} \\\hline 
36 & \multicolumn{4}{|c|}{raw assertion is null}                                          &    832 \\\hline
37 & \multicolumn{4}{|c|}{raw assertion is undefined}                                     &  39312 \\\hline
\multicolumn{5}{r|}{sum}                                                                  & 566094 \\\cline{6-6}
\end{tabular}
\end{center}
\end{table}

\suppressfloats[b]

\paragraph{Examples.}
Below we give the first example (as we parse the assertions in order) 
for each case of the indicator variable for the raw assertions.

\begin{description}
\item [Indicator = 0.] Assertion $2$ 
has best raw assertion equal to $3$ which is associated with 
the sentence \dbtext{Somewhere something can be is next to} (ID $715991$).
The best frame for the assertion is $3$ which is \dbtext{Somewhere \{1\} can be is next \{2\}}.
The best surface forms are respectively $5$ (\dbtext{something}) and $6$ (\dbtext{to}).
The raw assertion points to the assertion $2$ and has the same surface forms and frame.

\item [Indicator = 1-17.] No instances.

\item [Indicator = 18.] Assertion $674$ has best raw assertion equal to $675$
which is associated with the sentence \dbtext{something can be at the movies} (ID $716856$).
Frame $43$ (\dbtext{\{1\} can be at \{2\}}) is the best frame for this assertion, 
and the two surface forms are respectively $5$ (\dbtext{something}) and $1047$ (\dbtext{the movies}).
The raw assertion has the same frame and same surfaces respectively, but points to the assertion
with ID $40199$. 
Interestingly enough, the assertion $40199$ does not point back to this
raw assertion but rather to the raw assertion with ID $43017$ which is associated with the sentence
\dbtext{Somewhere something can be is a movie}.

\item [Indicator = 19-27.] No instances.

\item [Indicator = 28.] Assertion $7270$ has best raw assertion equal to $7375$
which is associated with the sentence \dbtext{speakers are for making sound} (ID $728720$).
The best frame for this assertion is $40$ (\dbtext{\{1\} are \{2\}})
and the two surface forms are respectively $9819$ (\dbtext{speakers}) and $9820$ (\dbtext{for making sound}).
The raw assertion has frame $7$ (\dbtext{\{1\} is for \{2\}})
and the surface forms are respectively $9819$ (\dbtext{speakers}) and $143185$ (\dbtext{making sound})
The raw assertion points to the assertion $429487$ which has its \texttt{best\_raw\_id} equal to
$7375$. Both of these assertions, i.e.~the one with ID $7270$ and the one with ID $429487$,
relate the concepts $8419$ (\dbtext{speaker}) and $8420$ (\dbtext{make sound}).

\item [Indicator = 29.] Assertion $29506$ has best raw assertion equal to $31257$
which is associated with the sentence \dbtext{hands are for touching things} (ID $768019$).
The best frame for this assertion is $40$ (\dbtext{\{1\} are \{2\}})
and the two surface forms are respectively $2624$ (\dbtext{hands}) 
and $34422$ (\dbtext{for touching things}).
The raw assertion has frame $7$ (\dbtext{\{1\} is for \{2\}})
and the surface forms are respectively $2624$ (\dbtext{hands}) 
and $287669$ (\dbtext{touching things}).
The raw assertion points to the assertion $393267$ which has its \texttt{best\_raw\_id}
equal to $977445$ which is associated with the sentence \dbtext{hand is used for touch}
(ID $2308541$).

\item [Indicator = 30-35.] No instances.

\item [Indicator = 36.] The assertion with ID $320114$ has \texttt{best\_raw\_id} equal to null.

\item [Indicator = 37.] The assertion with ID $962$ points to an undefined \texttt{best\_raw\_id}
(ID $965$).
\end{description}

\subsection{Surface Forms Raised}
This verification process raises $540$ surface form IDs, all of which are valid, but were not mentioned
among the assertions in the English language.

\section{Second Pass: Discrepancies on the Score Entries}\label{sec:closure:first-pass:scores}
Table \ref{tbl:scores:discrepancies-distribution} gives a 
distribution of discrepancies according to the metric $h$ given by (\ref{eq:score:half-discrepancy})
in Section \ref{sec:scores:discrepancies-magnitude}.
Table \ref{tbl:scores:discrepancy-classes} gives a distribution of discrepancies observed
among the three tables that refer to scores. 

\subsection{Signs on Scores}
\begin{remark}[Two Signs for Scores]
We distinguish only two signs for the scores. Strictly positive ($> 0$) and 
non-positive ($\leq 0$). We do so, since every assertion when first entered into
\conceptnet has score equal to $1$. Hence, a non-positive score implies that the
assertion is not so good. This approach was also followed in \citep{analogyspace:2008}.
\end{remark}

Table \ref{tbl:scores:signs} presents the number of entries that have positive and
non-positive scores in the three tables.

\begin{table}[ht]
\caption{Positive and non-positive scores on the entries of the three tables.
In the cases of the entries of raw assertions and sentences, the values are obtained
by observing the entries in the chain assertion 
$\rightarrow$ \texttt{best\_raw\_id} $\rightarrow$ \texttt{sentence\_id}.}\label{tbl:scores:signs}
\begin{center}
\begin{tabular}{|r||c|c||c|}
\cline{2-4}
\multicolumn{1}{c|}{} & entries with positive score & entries with non-positive score & total entries \\\hline
assertions     &    492389 &         73705 & 566094 \\\hline
raw assertions &    493108 &         32842 & 525950 \\\hline
sentences      &    516324 &          9626 & 525950 \\\hline
\end{tabular}
\end{center}
\end{table}

\suppressfloats[t]

\subsection{Bounds on Scores}\label{sec:scores:extremes}
Tables \ref{tbl:scores:extremes:all} and \ref{tbl:scores:extremes:english} present the
extreme values that scores can obtain in \conceptnet. Table \ref{tbl:scores:extremes:all}
refers to the entire tables, while Table \ref{tbl:scores:extremes:english} refers
to the entries that have their \texttt{language\_id} equal to \texttt{en}.

\begin{table}[ht]
\caption{Bounds on scores from different tables; \emph{any language}}\label{tbl:scores:extremes:all}
\begin{center}
\begin{tabular}{|r||c|c|c|c|}\cline{2-5}
\multicolumn{1}{c|}{}& minimum score & maximum score & id for minimum score & id for maximum score \\\cline{2-5}\hline
assertions     & $-10$         & $311$         & $330369$  & $741038$ \\\hline
raw assertions & $-10$         & $265$         & $377317$  & $566768$ \\\hline
sentences      & $-10$         & $265$         & $1690862$ & $1509374$ \\\hline
\end{tabular}
\end{center}
\end{table}

\begin{table}[ht]
\caption{Bounds on scores from different tables when the language is restricted to \emph{English}.}\label{tbl:scores:extremes:english}
\begin{center}
\begin{tabular}{|r||c|c|c|c|}\cline{2-5}
\multicolumn{1}{c|}{}& minimum score & maximum score & id for minimum score & id for maximum score \\\cline{2-5}\hline
assertions     & $-10$         & $147$         & $330369$  & $1664$ \\\hline
raw assertions & $-10$         & $124$         & $377317$  & $19218$ \\\hline
sentences      & $-10$         & $124$/$48$    & $1690862$ & $1241798$/$1318471$ \\\hline
\end{tabular}
\end{center}
\end{table}


\paragraph{Minimum Score - Both Tables.} 
The assertion with ID $330369$ has \texttt{best\_raw\_id} equal to $377317$, which in turn
has \texttt{sentence\_id} equal to $1690862$. Hence, all the minimum values are obtained through
the same sentence of the English language:
\begin{center}
\dbtext{college is a kind of musical instrument.}
\end{center}
The frame and the surfaces also agree with each other between the assertion and the raw assertion.

\paragraph{Maximum Score - Any language.}
The assertion with ID $741038$ has score $311$ and it is an assertion in Chinese (\texttt{language\_id} is \texttt{zh-Hant}).
It refers to the raw assertion with ID $981853$ (score equal to $1$), which in turn refers to the sentence with ID $2312949$
(score equal to $1$). Google Translate gives:
\begin{center}
\dbtext{You eat because you're hungry.}
\end{center}

The raw assertion with ID $566768$ has score $265$ and it is a raw assertion in Portuguese (\texttt{language\_id} is \texttt{pt}).
It refers to the sentence with ID $1897890$ (score equal to $1$). Google Translate gives:
\begin{center}
\dbtext{People sleep when they are sleepy.}
\end{center}

The sentence with ID $1509374$ has score $265$ and it is also a sentence in Portuguese (\texttt{language\_id} is \texttt{pt}).
According to Google Translate the sentence is:
\begin{center}
\dbtext{People sleep when they are sleepy}
\end{center}

\begin{remark}[Slight Variations $\Rightarrow$ Big Score Discrepancies]
The last two examples in Portuguese differ only by a full stop!
However, the difference in scores is very large.
\end{remark}

\paragraph{Maximum Score - English Language.}
The assertion with ID $1664$ has score $147$ and has \texttt{best\_raw\_id} $19218$.
It relates the concepts \dbtext{baseball} (\texttt{concept1\_id} is $1890$) and \dbtext{sport}
(\texttt{concept2\_id} is $2130$) through the relation \texttt{IsA} (ID is $5$).

However, the score for the raw assertion $19218$ is $124$ (i.e.~different from $147$).
Note here, that this raw assertion points to the sentence with ID $748040$ which is:
\begin{center}
\dbtext{Baseball is a sport played in the U.S.}
\end{center}

\medskip

Regarding the maximum score obtained through the \texttt{corpus\_sentence} table we have
a very strange phenomenon. Just by looking at the table on those sentences that have a
tag for the English language, the maximum score is $124$ and is obtained through the sentence
with ID $1241798$. That 
sentence is:
\begin{center}
\dbtext{Baseball is a sport}
\end{center}

\begin{remark}[Baseball inconsistency on Sentences]
The sentence \dbtext{Baseball is a sport} 
with ID $1241798$ is \emph{not} referred by \emph{any} raw assertion!
This is very strange, especially because the score of this sentence is $124$, just like the score
of the raw assertion with ID $19218$, which refers to the sentence 
\dbtext{Baseball is a sport played in the U.S.} 
which in turn has score only $1$
(see above for the maximum score obtained among the raw assertions in the English language).
\end{remark}

On the other hand, if we look on all those sentences, that are associated with a raw assertion
(i.e.~we can find their IDs in some row of the \texttt{conceptnet\_rawassertion} table),
such that the raw assertion appears as \texttt{best\_raw\_id} in some assertion of the 
\texttt{conceptnet\_assertion} table,
then, the maximum score obtained is $48$ through the sentence with ID $1318471$ which is:
\begin{center}
\dbtext{bottles are often made of plastic}
\end{center}
We note here, that we have the same result even if we do the simpler search of finding the 
maximum score among the sentences in the English language that appear in some raw assertion
of the English language. In other words, the following two sets of SQL queries return the same values:
\begin{itemize}
\item SQL Query Set 1:
\begin{code}
sqlite> select max(score) from corpus\_sentence where id in (\\
\phantom{sql}...> select sentence\_id from conceptnet\_rawassertion where id in (\\
\phantom{sql}...> select best\_raw\_id from conceptnet\_assertion where language\_id = 'en'));\\
48\\
sqlite> select id from corpus\_sentence where score = 48 and id in (\\
\phantom{sql}...> select sentence\_id from conceptnet\_rawassertion where id in (\\
\phantom{sql}...> select best\_raw\_id from conceptnet\_assertion where language\_id = 'en'));\\
1318471
\end{code}

\item SQL Query Set 2:

\begin{code}
sqlite> select max(score) from corpus\_sentence where id in (\\
\phantom{sqlite> }select sentence\_id from conceptnet\_rawassertion where language\_id = 'en');\\
48\\
sqlite> select id from corpus\_sentence where score = 48 and id in (\\
\phantom{sqlite> }select sentence\_id from conceptnet\_rawassertion where language\_id = 'en');\\
1318471
\end{code}
\end{itemize}

\subsection{Magnitude of Score Inconsistencies: Discrepancy and Half-Discrepancy}\label{sec:scores:discrepancies-magnitude}
This section gives a brief description of the magnitude of the inconsistencies
that can be observed as we restrict the assertions in the English language.

\begin{definition}[Discrepancy]\label{def:score:discrepancy}
We define the discrepancy $d$ to be
\begin{equation}\label{eq:score:discrepancy}
d = \abs{s_1 - s_2} + \abs{s_2 - s_3} + \abs{s_3 - s_1} \; ,
\end{equation}
where $s_1, s_2$, and $s_3$ are the 
scores appearing
respectively
in the tables \texttt{conceptnet\_assertion},
\texttt{conceptnet\_rawassertion},
and 
\texttt{corpus\_sentence}. 
\end{definition}

\begin{definition}[Half-Discrepancy]\label{def:score:half-discrepancy}
We define half-discrepancy to be
\begin{equation}\label{eq:score:half-discrepancy}
h = \frac{d}{2} \; .
\end{equation}
\end{definition}

\medskip

The following theorem guarantees that $h\in\NN$, and hence $d$ is an even natural number.

\begin{prop}[Integer Half-Discrepancies]
Quantity $h$ is a natural number.
\end{prop}
\begin{proof}
Let us look at the quantity $d = 2\cdot h$
and note that $s_1, s_2, s_3 \in \ZZ$.
We will prove that $d$ can only be even.
Towards contradiction, assume that $d$ is odd. Then, $d$ is either
the sum of three odd values or one odd and two even values.

If $d$ is the sum of one odd and two even values, then, without loss of generality we can assume
that $\abs{s_1 - s_2} = 2k+1$, while $\abs{s_2-s_3}=2m$ and $\abs{s_3-s_1}=2n$,
where $k, m, n\in\NN$.
Since $\abs{s_2-s_3} = 2m$ and $\abs{s_3-s_1} = 2n$, it follows that $s_1$ and $s_2$
have the same parity since they differ an even amount of integer values from $s_3$.
However, this is a contradiction to $\abs{s_1-s_2} = 2k+1$ which implies that the
parity of $s_1$ and $s_2$ is different.

On the other hand, if $d$ is the sum of three odd values, then without loss of generality we can assume
that $\abs{s_1 - s_2} = 2k+1$, $\abs{s_2-s_3}=2m+1$, and $\abs{s_3-s_1}=2n+1$,
where $k, m, n\in\NN$.
Similarly to the case above, 
since $\abs{s_1 - s_2} = 2k+1$ and $\abs{s_2-s_3}=2m+1$ it follows that $s_1$ and $s_3$ have the
same parity because they differ an odd number of integer values from $s_2$. But this is
a contradiction to the assumption that $\abs{s_3-s_1}$ is odd.\qedhere
\end{proof}

Table \ref{tbl:scores:discrepancies-distribution} presents the distribution
of the magnitude of the discrepancies that we can observe among the three tables
that have score entries.

\begin{table}[h]
\caption{Distribution of half-discrepancies $h$. Half-discrepancies are given by (\ref{eq:score:half-discrepancy}).
In the cases where \texttt{best\_raw\_id} is null or undefined/missing, we set $d, h = 0$.}\label{tbl:scores:discrepancies-distribution}
\begin{center}
\hspace{\fill}
\begin{tabular}{|c|r|}\hline
$h$ & entries \\\hline\hline
  0 &  504889 \\\hline
  1 &   50972 \\\hline
  2 &    5990 \\\hline
  3 &    1976 \\\hline 
  4 &     931 \\\hline
  5 &     499 \\\hline
  6 &     224 \\\hline
  7 &     122 \\\hline
  8 &      51 \\\hline
  9 &     143 \\\hline
 10 &      86 \\\hline
\end{tabular}
\hspace{\fill}
\begin{tabular}{|c|r|}\hline
 $h$ & entries \\\hline\hline
  11 &      68 \\\hline
  12 &      42 \\\hline
  13 &      19 \\\hline
  14 &      20 \\\hline 
  15 &       8 \\\hline
  16 &       8 \\\hline
  17 &       2 \\\hline
  18 &       7 \\\hline
  19 &       5 \\\hline
  20 &       6 \\\hline
  21 &       1 \\\hline
\end{tabular}
\hspace{\fill}
\begin{tabular}{|c|r|}\hline
 $h$ & entries \\\hline\hline
  22 &       3 \\\hline
  23 &       2 \\\hline
  24 &       1 \\\hline
  25 &       2 \\\hline 
  26 &       1 \\\hline
  29 &       1 \\\hline
  32 &       1 \\\hline
  33 &       2 \\\hline
  35 &       1 \\\hline
  36 &       1 \\\hline
  39 &       1 \\\hline
\end{tabular}
\hspace{\fill}
\begin{tabular}{|c|r|}\hline
 $h$ & entries \\\hline\hline
  40 &      1 \\\hline
  41 &      1 \\\hline
  48 &      1 \\\hline
  62 &      1 \\\hline 
  64 &      1 \\\hline
  73 &      1 \\\hline
  77 &      1 \\\hline
 108 &      1 \\\hline
 146 &      1 \\\hline
 \multicolumn{2}{c}{}\\
 \multicolumn{2}{c}{}\\
\end{tabular}
\hspace{\fill}
\hspace{\fill}
\end{center}
\end{table}

Regarding the instance that achieves the maximum discrepancy ($146$)
please have a look in the discussion in Section \ref{sec:scores:enumarating};
in particular when the indicator is equal to $8$.

\subsection{Enumerating Score Inconsistencies between Tables}\label{sec:scores:enumarating}
Table \ref{tbl:scores:discrepancy-classes} presents the inconsistencies among the score entries
found in the three different tables as we read the assertions in the English language.

\begin{table}[ht]
\caption{Distribution of inconsistencies among the score entries found in the three
different tables of \texttt{conceptnet\_assertion}, \texttt{conceptnet\_rawassertion}, and
\texttt{corpus\_sentence}. The last column presents the maximum half-discrepancy obtained 
in each group.}\label{tbl:scores:discrepancy-classes}
\begin{center}
\begin{tabular}{|c|l||r||c|}\cline{4-4}
\multicolumn{3}{c|}{} & maximum \\\cline{1-3}
indicator & scores where \ldots & entries & $h$ \\\hline\hline
0 & all three agree                                                    & 464745 &  0 \\\hline
1 & \texttt{best\_raw\_id} is null or undefined/missing                &  40144 &  0 \\\hline
2 & assertions and raw assertions agree; sentences have same sign      &   7614 &  15 \\\hline
3 & assertions and raw assertions agree; sentences have different sign &  22933 &   3 \\\hline
4 & assertions and sentences agree; raw assertions have same sign      &    152 &   9 \\\hline
5 & assertions and sentences agree; raw assertions have different sign &    129 &   4 \\\hline
6 & raw assertions and sentences agree; assertions have same sign      &  22915 &  73 \\\hline
7 & raw assertions and sentences agree; assertions have different sign &   1616 &   8 \\\hline
8 & all three disagree; same sign ($> 0$, or $\leq 0$)                 &   5670 & 146 \\\hline
9 & all three disagree; different signs ($> 0$, or $\leq 0$)           &    176 &  15 \\\hline
\multicolumn{2}{r|}{sum}                                               & \multicolumn{1}{|c|}{566094} & \multicolumn{1}{c}{}\\\cline{3-3}
\end{tabular}
\end{center}
\end{table}

\paragraph{Examples.} We give one example for each case of the score discrepancy indicator.
\begin{description}
 \item [Indicator = 0.] 
 Assertion ID $12279$, Raw Assertion ID $351620$, Sentence ID $1432008$.
 The assertion relates concepts \dbtext{goldfish} (ID $14183$) and \dbtext{carp} (ID $14184$) with the relation \texttt{IsA}
(ID $5$). The sentence is \dbtext{a goldfish is a carp.}. The score in each case is $16$. This is 
the maximum score among all the cases in this class, and no other instance in this class achieves this score.

 \item [Indicator = 1.] Assertion with ID $320114$ has \texttt{best\_raw\_id} equal to null.
 The assertion relates the concept \dbtext{drink} (ID $120$) with itself with the relation \texttt{UsedFor} (ID $7$).
 There are $832$ assertions with null \texttt{best\_raw\_id}.

 Assertion with ID $962$ refers to raw assertion with ID $965$, but there is no raw assertion with such an ID.
 The assertion relates the concepts \dbtext{fight war} (ID $437$) and \dbtext{hate} (ID $1342$)
 with the relation \texttt{HasPrerequisite} (ID $3$), and has score equal to zero.
 There are $39312$ different \texttt{best\_raw\_id}'s such that there is no raw assertion with such an ID.
 Apparently, all these IDs are mentioned only once in the table \texttt{conceptnet\_assertion}.

 \item [Indicator = 2.]
 Assertion ID $39773$, Raw Assertion ID $42548$, Sentence ID $787525$.
 The assertion relates the concepts \dbtext{baseball} (ID $1890$) and \dbtext{game} (ID $732$) with the relation
 \texttt{IsA} (ID $5$). The sentence is \dbtext{Baseball is a game}.
 Assertion and raw assertion give a score of $16$, while the sentence gives a score of $1$.
 The half-discrepancy is $15$.
 This is the maximum half-discrepancy that can be observed in this class, and no other triple can achieve this value.

 \item [Indicator = 3.]
 Assertion ID $115013$, Raw Assertion ID $127019$, Sentence ID $942706$.
 The assertion relates the concepts \dbtext{jason} (ID $82025$) and \dbtext{late} (ID $1520$) with the relation
 \texttt{HasProperty} (ID $20$). The sentence is \dbtext{jason is not late}.
 Assertion and raw assertion give a score of $-2$, while the sentence gives a score of $1$.
 The half-discrepancy is $3$.
 This is the maximum half-discrepancy that can be observed in this class.
 A similar half-discrepancy of $3$ can be observed for the assertion with ID $544123$.

 \item [Indicator = 4.]
 Assertion ID $181807$, Raw Assertion ID $72566$, Sentence ID $840388$.
 The assertion relates the concepts \dbtext{jack} (ID $14299$) and \dbtext{child game} (ID $127337$) with the relation
 \texttt{IsA} (ID $5$). The sentence is \dbtext{Jacks is a children's game that requires agility.}.
 Assertion and sentence give a score of $1$, while the raw assertion gives a score of $10$.
 The half-discrepancy is $9$.
 This is the maximum half-discrepancy that can be observed in this class, and no other triple can achieve this value.

 \item [Indicator = 5.]
 Assertion ID $197813$, Raw Assertion ID $224508$, Sentence ID $1177474$.
 The assertion relates the concepts \dbtext{marujuana} (ID $137113$) and \dbtext{cannabis} (ID $37883$) with the relation
 \texttt{IsA} (ID $5$). The sentence is \dbtext{Marujuana is Cannabis}.
 Assertion and sentence give a score of $2$, while the raw assertion gives a score of $-2$.
 The half-discrepancy is $4$.
 This is the maximum half-discrepancy that can be observed in this class, and no other triple can achieve this value.

 \item [Indicator = 6.]
 Assertion ID $56287$, Raw Assertion ID $83533$, Sentence ID $861172$.
 The assertion relates the concepts \dbtext{pen} (ID $1205$) and \dbtext{write} (ID $1893$) with the relation
 \texttt{IsA} (ID $5$). The sentence is \dbtext{a pen is something you write with}.
 Raw assertion and sentence give a score of $1$, while the assertion gives a score of $74$.
 The half-discrepancy is $73$.
 This is the maximum half-discrepancy that can be observed in this class, and no other triple can achieve this value.

 \item [Indicator = 7.]
 Assertion ID $67530$, Raw Assertion ID $176468$, Sentence ID $1052796$.
 The assertion relates the concepts \dbtext{snake} (ID $369$) and \dbtext{leg} (ID $1252$) with the relation
 \texttt{HasA} (ID $16$). The sentence is \dbtext{A snake does not have legs.}.
 Raw assertion and sentence give a score of $1$, while the assertion gives a score of $-7$.
 The half-discrepancy is $8$.
 This is the maximum half-discrepancy that can be observed in this class, and no other triple can achieve this value.

 \item [Indicator = 8.]
 Assertion ID $1664$, Raw Assertion ID $19218$, Sentence ID $748040$.
 The assertion relates the concepts \dbtext{baseball} (ID $1890$) and \dbtext{sport} (ID $2130$) with the relation
 \texttt{IsA} (ID $5$). The sentence is \dbtext{Baseball is a sport played in the U.S.}.
 Assertion gives a score of $147$, raw assertion a score of $124$, and sentence a score of $1$.
 The half-discrepancy is $146$.
 This is the maximum half-discrepancy that can be observed in this class, and no other triple can achieve this value.

 \item [Indicator = 9.]
 Assertion ID $196090$, Raw Assertion ID $222398$, Sentence ID $1173220$.
 The assertion relates the concepts \dbtext{person} (ID $9$) and \dbtext{headache} (ID $2062$) with the relation
 \texttt{Desires} (ID $10$). The sentence is \dbtext{a person wants a headache}.
 Assertion gives a score of $7$, raw assertion a score of $-2$, and sentence a score of $13$.
 The half-discrepancy is $15$.
 This is the maximum half-discrepancy that can be observed in this class, and no other triple can achieve this value.
\end{description}

\section{Third and Final Pass}
In the third pass we parse the data in the tables conceptnet\_concept and conceptnet\_surfaceform.
This allows us to load the concepts and the surface forms that were raised from the previous pass.
In theory, it could be the case that these new additional surface forms were referring to concepts
that have not been raised yet from the previous passes, and hence we would require one more pass 
on the conceptnet\_concept table to add these last concepts. However, this is not the case.
In other words, these newly introduced surface forms from the last pass do not refer to concepts
that we have not encountered earlier. 
Hence, this third pass is the last pass that we perform on the tables of the database.


\chapter{Consistency of the Database}
The database is inconsistent. 
We have different assertions between the same concepts using the \emph{same} relation but
\emph{different} frequency. Not only that, but the \emph{value of the frequency can have opposite signs}, 
\emph{implying} essentially \emph{controversial statements}. 
Moreover, both statements can be characterized as correct since the
score (measure of the validity of the statement) is positive in both cases!

\begin{example}
We have the following instance. 
\begin{center}
\begin{tabular}{ll}
 concept 1: & man (id 7)\\
 concept 2: & animal (id 902)\\
 relation:  & IsA (id 5)
\end{tabular}
\end{center}

\begin{description}
\item [Assertion ID:] 103395
\begin{itemize}
 \item frequency: 1 (value: 5, string description: empty string)
 \item score: 3
 \item best raw assertion: 368795 (points to sentence 1672478)
 \item sentence: \textquotedblleft man is a kind of animal.\textquotedblright
\end{itemize}

\item [Assertion ID:] 616165
\begin{itemize}
 \item frequency: 25 (value: -5, string description: \textquotedblleft not\textquotedblright)
 \item score: 1
 \item best raw assertion: 827499 (points to sentence 2158613)
 \item sentence: \textquotedblleft man is not animal\textquotedblright
\end{itemize}
\end{description}
\end{example}

\part{Structural Properties of ConceptNet 4}

\chapter{High Level Overview and Conventions}\label{chapter:introduction}
In this part we will examine basic structural properties of \conceptnet.
All the results are based on the \texttt{ConceptNet.db} file located in
the~\texttt{.conceptnet} directory under our home directory. 
Regarding the specifics of the~.db file we have:
{\small
\begin{verbatim}
$ ls -l ~/.conceptnet/ConceptNet.db 
-rw-r--r-- 1 user user 959354880 Feb 11  2010 /home/user/.conceptnet/ConceptNet.db
$ file ~/.conceptnet/ConceptNet.db 
/home/user/.conceptnet/ConceptNet.db: SQLite 3.x database
$
\end{verbatim}
}

\section{Assertions}
The \conceptnet database has $828,252$ assertions; $566,094$ are in English.
These assertions define the input for the edges of the induced graphs.

\begin{convention}[Input Definition]
The input is defined by the assertions of the English language only.
\end{convention}

\begin{remark}
The preliminary analysis will consider edges that have both negative and positive score.
However, as the analysis progresses we will focus on edges that have strictly positive score,
since the rest of the assertions have received at least one negative vote, and the number of
negative votes is at least as the number of positive votes.
\end{remark}

\section{Concepts}
The \conceptnet database has $460,306$ concept IDs; $321,993$ are in English.
\begin{itemize}[noitemsep,leftmargin=8mm,topsep=0.5mm]
\item The minimum concept ID found among the assertions of the English language is: 
$5$ for \dbtext{something}.
\item The maximum concept ID found among the assertions of the English language is: 
$482,783$ for \dbtext{understand human mind brain}.
\item Number of different concepts appearing in assertions: $279,497$.
\item Number of different concepts appearing in the closure of the input: $279,885$.
\item Allowing self-loops there are $262,577$ different concepts with non-zero total degree on the induced
subgraphs formed by edges with positive score.
\item Disallowing self-loops there are $262,575$ different concepts with non-zero total degree on the induced
subgraphs formed by edges with positive score.
\end{itemize}

\begin{convention}[Number of Concepts in \conceptnet]\label{conv:number-of-concepts}
In what follows, when we refer to the total number of concepts found in \conceptnet, we mean $279,497$,
which is the number of concepts appearing in the assertions of the English language, which in turn define
our input.
\end{convention}

\section{Relations}
\conceptnet has $30$ relations; $27$ appear among the assertions in the English language.
Table \ref{tbl:relations} in Appendix \ref{sec:appendix:db-related} gives an overview of all the relations found in \conceptnet.

\section{Frequencies}
Table \ref{tbl:different-frequencies} in Appendix \ref{sec:appendix:db-related} presents the different frequencies that we
can encounter in \conceptnet in the assertions of the English language.

\section{Edges and Isolated Vertices in the Induced (Multi-)Graph Variants}
Table \ref{tbl:number-of-edges-isolated-vertices:overall} presents the number of edges as well as the isolated
vertices that we encounter in 12 different cases in \conceptnet. The cases are 12 since we can distinguish
cases based on the following:
\begin{itemize}[noitemsep,leftmargin=8mm,topsep=0.5mm]
 \item whether we allow edges with all scores or not,
 \item whether we allow self-loops or not,
 \item whether we allow edges with negative polarity, positive polarity, or finally both.
\end{itemize}

\begin{table}[ht]
\caption{Number of edges and isolated vertices on different variants of the induced subgraphs 
that can be obtained in \conceptnet by looking at the assertions of the English language.
The marks \cmark and \xmark indicate respectively whether we allow self-loops in the induced (multi-)graphs
or not. The enumeration allows all possible relations and frequencies on the edges.}\label{tbl:number-of-edges-isolated-vertices:overall}
\begin{center}
\begin{tabular}{|c|c|c||r|r|r||r|}\hline
score    & self-loops & polarity & multigraph & directed graph & undirected graph & isolated vertices \\\hline\hline
all      & \xmark     & negative &     15,327 &         15,168 &           14,707 & 267,187 \\\hline
all      & \xmark     & positive &    550,277 &        465,866 &          452,445 &   5,764 \\\hline
all      & \xmark     & both     &    565,604 &        478,624 &          464,767 &       2 \\\hline\hline
all      & \cmark     & negative &     15,342 &         15,182 &           14,721 & 267,187 \\\hline
all      & \cmark     & positive &    550,752 &        466,166 &          452,745 &   5,762 \\\hline
all      & \cmark     & both     &    566,094 &        478,929 &          465,072 &       0 \\\hline\hline
positive & \xmark     & negative &     13,497 &         13,387 &           12,989 & 267,790 \\\hline
positive & \xmark     & positive &    478,499 &        412,956 &          401,367 &  22,651 \\\hline
positive & \xmark     & both     &    491,996 &        424,525 &          412,569 &  16,922 \\\hline\hline
positive & \cmark     & negative &     13,510 &         13,399 &           13,001 & 267,790 \\\hline
positive & \cmark     & positive &    478,879 &        413,216 &          401,627 &  22,649 \\\hline
positive & \cmark     & both     &    492,389 &        424,790 &          412,834 &  16,920 \\\hline
\end{tabular}
\end{center}
\end{table}

\section{Non-Zero Degrees and Self-Loops in the Induced (Multi-)Graph Variants}
Again we distinguish four cases based on whether we include edges with all possible scores
or not and on whether we allow self-loops or not.
There are two nodes that have self-loops only among their edges.
These are the nodes with IDs $56,959$ 
(\dbtext{hansome} \footnote{This is the actual spelling of the concept.}) and 
$201,444$ (\dbtext{needless death}).
Table \ref{tbl:conceptnet:degrees:overview:directed} gives an overview of the directed case,
while
Table \ref{tbl:conceptnet:degrees:overview:undirected} gives an overview of the undirected case.
The entries for the number of vertices with non-zero degree in the undirected case are obtained 
by subtracting the number of isolated vertices found in 
Table \ref{tbl:number-of-edges-isolated-vertices:overall} from $279,497$.
Regarding the number of nodes that that have self-loops, these numbers are identical to the directed
case which is presented in Table \ref{tbl:conceptnet:degrees:overview:directed}.
However, we write down these numbers for clarity. Note that the numbers found in these column refer to
vertices and are not counting distinct self-loops. Counting distinct self-loops in different cases 
will be examined in Section \ref{sec:decomposition}.

\begin{table}[ht]
\caption{Overview on the degrees of the induced directed multigraphs and 
graphs for \conceptnet in the English language.}\label{tbl:conceptnet:degrees:overview:directed}
\begin{center}
\begin{tabular}{|c|c|c||r|r|r|r|r|}\cline{4-8}
\multicolumn{3}{c}{}  & \multicolumn{5}{|c|}{number of nodes with} \\\hline
%
%
%
score    & self-loops & polarity & $\neq 0$ in-deg & $\neq 0$ out-deg & $\neq 0$ in-/out-deg & self-loops & self-loops only \\\hline\hline
all      & \xmark     & negative &         $9,291$ &          $4,412$ &              $1,393$ &         -- & -- \\\hline
all      & \xmark     & positive &       $233,456$ &         $60,628$ &             $20,351$ &         -- & -- \\\hline
all      & \xmark     &     both &       $238,389$ &         $61,839$ &             $20,733$ &         -- & -- \\\hline\hline
all      & \cmark     & negative &         $9,293$ &          $4,412$ &              $1,395$ &       $14$ & $0$ \\\hline
all      & \cmark     & positive &       $233,462$ &         $60,634$ &             $20,361$ &      $300$ & $2$ \\\hline
all      & \cmark     &     both &       $238,395$ &         $61,845$ &             $20,743$ &      $305$ & $2$ \\\hline\hline
positive & \xmark     & negative &         $8,884$ &          $4,041$ &              $1,218$ &         -- & -- \\\hline
positive & \xmark     & positive &       $216,198$ &         $60,052$ &             $19,404$ &         -- & -- \\\hline
positive & \xmark     &     both &       $221,114$ &         $61,241$ &             $19,780$ &         -- & -- \\\hline\hline
positive & \cmark     & negative &         $8,886$ &          $4,041$ &              $1,220$ &       $12$ & $0$ \\\hline
positive & \cmark     & positive &       $216,204$ &         $60,057$ &             $19,413$ &      $260$ & $2$ \\\hline
positive & \cmark     &     both &       $221,120$ &         $61,246$ &             $19,789$ &      $265$ & $2$ \\\hline
\end{tabular}
\end{center}
\end{table}

\begin{table}[ht]
\caption{Overview of the degrees for the induced undirected multigraphs and 
graphs for \conceptnet in the English language. The columns about self-loops refer to the
multigraph only.}\label{tbl:conceptnet:degrees:overview:undirected}
\begin{center}
\begin{tabular}{|c|c|c||c|c|c|c|c|}\cline{4-6}
\multicolumn{3}{c}{}  & \multicolumn{3}{|c|}{number of nodes with} \\\hline
score    & self-loops & polarity & $\neq 0$ degree & self-loops & self-loops only \\\hline\hline
all      & \xmark     & negative &        $12,310$ &         -- & --  \\\hline
all      & \xmark     & positive &       $273,733$ &         -- & --  \\\hline
all      & \xmark     &     both &       $279,495$ &         -- & --  \\\hline\hline
all      & \cmark     & negative &        $12,310$ &       $14$ & $0$  \\\hline
all      & \cmark     & positive &       $273,735$ &      $300$ & $2$  \\\hline
all      & \cmark     &     both &       $279,497$ &      $305$ & $2$  \\\hline\hline
positive & \xmark     & negative &        $11,707$ &         -- & --  \\\hline
positive & \xmark     & positive &       $256,846$ &         -- & --  \\\hline
positive & \xmark     &     both &       $262,575$ &         -- & --  \\\hline\hline
positive & \cmark     & negative &        $11,707$ &       $12$ & $0$  \\\hline
positive & \cmark     & positive &       $256,848$ &      $260$ & $2$  \\\hline
positive & \cmark     &     both &       $262,577$ &      $265$ & $2$  \\\hline
\end{tabular}
\end{center}
\end{table}


\suppressfloats[t]

\section{Decomposition of Assertions and Edges}\label{sec:decomposition}
Table \ref{tbl:assertionDecomposition} gives the decomposition of the assertions
in the English language.

\begin{table}[ht]
\caption{Decomposition of assertions in the English language found in \conceptnet. We consider
all assertions regardless of their score and all assertions with positive score. 
Next to the number of edges or self-loops 
for each relation we see, in that order, how many have negative polarity and how many have positive polarity.}\label{tbl:assertionDecomposition}
\begin{center}
\resizebox{\textwidth}{!}{
\begin{tabular}{|r||r|l||rl|rl||rl|rl||}\cline{4-11}
\multicolumn{3}{c}{} & \multicolumn{8}{|c||}{induced directed multigraph based on assertions with}\\\cline{4-11}\hline
\multicolumn{1}{|c||}{\multirow{2}{*}{\rotatebox{90}{index}}} & \multicolumn{2}{c||}{relation} & \multicolumn{4}{c||}{all scores} & \multicolumn{4}{c||}{positive score} \\\cline{2-3}\cline{4-7}\cline{8-11}
   & id & \multicolumn{1}{c||}{name} & \multicolumn{2}{c|}{edges} & \multicolumn{2}{c||}{self-loops} & \multicolumn{2}{c|}{edges} & \multicolumn{2}{c||}{self-loops} \\\hline
 0 &  1 & HasFirstSubevent      &   4192 &      (4/4188) &      1 &   (0/1) &   4121 &     (3/4118) &      1 & (0/1)   \\\hline
 1 &  2 & HasLastSubevent       &   3066 &      (8/3058) &      2 &   (0/2) &   2971 &     (8/2963) &      2 & (0/2)   \\\hline
 2 &  3 & HasPrerequisite       &  23801 &    (68/23733) &     56 &  (0/56) &  23404 &   (55/23349) &     56 & (0/56)  \\\hline
 3 &  4 & MadeOf                &   1662 &     (29/1633) &      5 &   (1/4) &   1545 &    (25/1520) &      4 & (1/3)   \\\hline
 4 &  5 & IsA                   & 111547 & (4797/106750) &     89 & (11/78) &  94726 & (3884/90842) &     73 & (10/63) \\\hline
 5 &  6 & AtLocation            &  49508 &   (973/48535) &     43 &  (0/43) &  45192 &  (764/44428) &     26 & (0/26)  \\\hline
 6 &  7 & UsedFor               &  52135 &   (276/51859) &     31 &  (1/30) &  50451 &  (194/50257) &     30 & (1/29)  \\\hline
 7 &  8 & CapableOf             &  40141 &  (2994/37147) &     13 &  (0/13) &  39391 & (2924/36467) &     11 & (0/11)  \\\hline
 8 &  9 & MotivatedByGoal       &  15312 &    (33/15279) &     36 &  (0/36) &  15116 &   (27/15089) &     35 & (0/35)  \\\hline
 9 & 10 & Desires               &   9295 &   (4083/5212) &      4 &   (1/3) &   9059 &  (4048/5011) &      3 & (1/2)   \\\hline
10 & 12 & ConceptuallyRelatedTo &  23097 &     (0/23097) &     21 &  (0/21) &  23010 &    (0/23010) &     21 & (0/21)  \\\hline
11 & 13 & DefinedAs             &   6500 &      (3/6497) &      7 &   (0/7) &   6428 &     (0/6428) &      7 & (0/7)   \\\hline
12 & 14 & InstanceOf            &     70 &        (0/70) &      0 &   (0/0) &     69 &       (0/69) &      0 & (0/0)   \\\hline
13 & 15 & SymbolOf              &    167 &       (0/167) &      0 &   (0/0) &    166 &      (0/166) &      0 & (0/0)   \\\hline
14 & 16 & HasA                  &  55311 &   (415/54896) &     41 &  (0/41) &  22786 &  (399/22387) &     11 & (0/11)  \\\hline
15 & 17 & CausesDesire          &   5179 &     (20/5159) &      3 &   (0/3) &   4989 &    (15/4974) &      2 & (0/2)   \\\hline
16 & 18 & Causes                &  18624 &    (53/18571) &     25 &  (0/25) &  18257 &   (34/18223) &     24 & (0/24)  \\\hline
17 & 19 & HasSubevent           &  26206 &   (119/26087) &     19 &  (0/19) &  25444 &   (93/25351) &     18 & (0/18)  \\\hline
18 & 20 & HasProperty           &  93384 &  (1447/91937) &     63 &  (0/63) &  82458 & (1027/81431) &     53 & (0/53)  \\\hline
19 & 21 & PartOf                &   4935 &     (13/4922) &     13 &  (0/13) &   4676 &     (9/4667) &      8 & (0/8)   \\\hline
20 & 22 & ReceivesAction        &  10907 &     (1/10906) &      5 &   (1/4) &  10848 &    (0/10848) &      3 & (0/3)   \\\hline
21 & 24 & InheritsFrom          &    185 &       (0/185) &      2 &   (0/2) &     64 &       (0/64) &      2 & (0/2)   \\\hline
22 & 25 & CreatedBy             &    586 &       (6/580) &      2 &   (0/2) &    557 &      (1/556) &      1 & (0/1)   \\\hline
23 & 28 & HasPainCharacter      &     34 &        (0/34) &      0 &   (0/0) &     34 &       (0/34) &      0 & (0/0)   \\\hline
24 & 29 & HasPainIntensity      &     74 &        (0/74) &      0 &   (0/0) &     73 &       (0/73) &      0 & (0/0)   \\\hline
25 & 30 & LocatedNear           &   5053 &      (0/5053) &      1 &   (0/1) &   5044 &     (0/5044) &      1 & (0/1)   \\\hline
26 & 31 & SimilarSize           &   5123 &      (0/5123) &      8 &   (0/8) &   1510 &     (0/1510) &      1 & (0/1)   \\\hline
\multicolumn{11}{c}{} \\\cline{4-11}
\multicolumn{3}{r|}{total}      & 566094 & (15342/550752) &   490 & (15/475) & 492389 & (13510/478879) &   393 & (13/380) \\\cline{4-11}
\end{tabular}
}
\end{center}
\end{table}

\subsection{Partitioning Edges with Positive Score with respect to Frequencies}
Here we examine the number of edges of the induced subgraphs according to
different frequency value ranges. In every case we retain the edges with strictly positive score.
According to Convention \ref{conv:number-of-concepts} the number of nodes is $279,497$ in every case. 
Moreover, note that the number of edges of the induced multigraph with frequency values in the range
$\{-10, \ldots, 0\}$ plus the number of edges of the induced multigraph with frequency values in the range
$\{0, \ldots, 10\}$ is equal to $13,510 + 478,879 = 492,389$ which agrees with the total number of edges
with positive score mentioned in Table \ref{tbl:number-of-edges-isolated-vertices:overall}.

Table \ref{tbl:frequency-ranges:number-of-edges} gives a detailed overview in every case.
Note that from Table \ref{tbl:frequency-ranges:number-of-edges} it follows
that there are no edges with values for frequencies from the set $\{-9, -8, -7, -6, -4, -3, -1, 0, 1, 3, 6\}$,
which is, as it should be, in complete agreement with Table \ref{tbl:different-frequencies}.

\begin{table}[ht]
\caption{Number of edges in the induced subgraphs for various frequency ranges.
The columns in the cases of multigraph, directed graph, and undirected graph present the number
of edges with and without self-loops (in that order) in every case.
All relations are allowed between the concepts but the scores of the assertions have to be positive.}\label{tbl:frequency-ranges:number-of-edges}
\begin{center}
\begin{tabular}{|c|c||r|r||r|r||r|r||}\hline
\multirow{2}{*}{polarity} &
range for             & \multicolumn{6}{|c||}{number of edges with and without self-loops} \\\cline{3-8}
                      &
frequency values      & \multicolumn{2}{|c||}{multigraph} & \multicolumn{2}{c||}{directed graph} & \multicolumn{2}{c||}{undirected graph} \\\hline\hline
\multirow{11}{*}{\rotatebox{90}{negative}}
& \{-10\}             &     187 &     187 &     187 &     187 &     187 &     187 \\\cline{2-8}
& \{-10, -9\}         &     187 &     187 &     187 &     187 &     187 &     187 \\\cline{2-8}
& \{-10, -9, -8\}     &     187 &     187 &     187 &     187 &     187 &     187 \\\cline{2-8}
& \{-10, \ldots, -7\} &     187 &     187 &     187 &     187 &     187 &     187 \\\cline{2-8}
& \{-10, \ldots, -6\} &     187 &     187 &     187 &     187 &     187 &     187 \\\cline{2-8}
& \{-10, \ldots, -5\} &  13,395 &  13,382 &  13,287 &  13,275 &  12,889 &  12,877 \\\cline{2-8}
& \{-10, \ldots, -4\} &  13,395 &  13,382 &  13,287 &  13,275 &  12,889 &  12,877 \\\cline{2-8}
& \{-10, \ldots, -3\} &  13,395 &  13,382 &  13,287 &  13,275 &  12,889 &  12,877 \\\cline{2-8}
& \{-10, \ldots, -2\} &  13,510 &  13,497 &  13,399 &  13,387 &  13,001 &  12,989 \\\cline{2-8}
& \{-10, \ldots, -1\} &  13,510 &  13,497 &  13,399 &  13,387 &  13,001 &  12,989 \\\cline{2-8}
& \{-10, \ldots, 0\}  &  13,510 &  13,497 &  13,399 &  13,387 &  13,001 &  12,989 \\\hline\hline
\multirow{11}{*}{\rotatebox{90}{positive}}
& \{0, \ldots, 10\}   & 478,879 & 478,499 & 413,216 & 412,956 & 401,627 & 401,367 \\\cline{2-8}
& \{1, \ldots, 10\}   & 478,879 & 478,499 & 413,216 & 412,956 & 401,627 & 401,367 \\\cline{2-8}
& \{2, \ldots, 10\}   & 478,879 & 478,499 & 413,216 & 412,956 & 401,627 & 401,367 \\\cline{2-8}
& \{3, \ldots, 10\}   & 478,872 & 478,492 & 413,209 & 412,949 & 401,620 & 401,360 \\\cline{2-8}
& \{4, \ldots, 10\}   & 478,872 & 478,492 & 413,209 & 412,949 & 401,620 & 401,360 \\\cline{2-8}
& \{5, \ldots, 10\}   & 471,543 & 471,170 & 407,244 & 406,987 & 395,726 & 395,469 \\\cline{2-8}
& \{6, \ldots, 10\}   &   4,930 &   4,930 &   4,860 &   4,860 &   4,859 &   4,859 \\\cline{2-8}
& \{7, \ldots, 10\}   &   4,930 &   4,930 &   4,860 &   4,860 &   4,859 &   4,859 \\\cline{2-8}
& \{8, 9, 10\}        &   2,217 &   2,217 &   2,206 &   2,206 &   2,205 &   2,205 \\\cline{2-8}
& \{9, 10\}           &     445 &     445 &     444 &     444 &     443 &     443 \\\cline{2-8}
& \{10\}              &     386 &     386 &     385 &     385 &     384 &     384 \\\hline
\end{tabular}
\end{center}
\end{table}

\chapter{Degrees and Distributions}\label{chapter:degree-distributions}
Here we examine the degrees and the degree distributions on the various induced graphs. 
Figure \ref{fig:degreeDistribution:wordle} gives a snapshot of the total degree distribution
of the induced directed multigraph of \conceptnet 
as this was generated by \href{http://www.wordle.net/}{Wordle} \footnote{ Homepage: \url{http://www.wordle.net}}.
Table \ref{tbl:degreeDistribution:concepts:top40} presents the $100$ concepts with highest total degree
in the same graph (that is, directed multigraph induced by the assertions of the English language with positive score).

\begin{figure}[ht]
\begin{center}
\includegraphics[width=\textwidth]{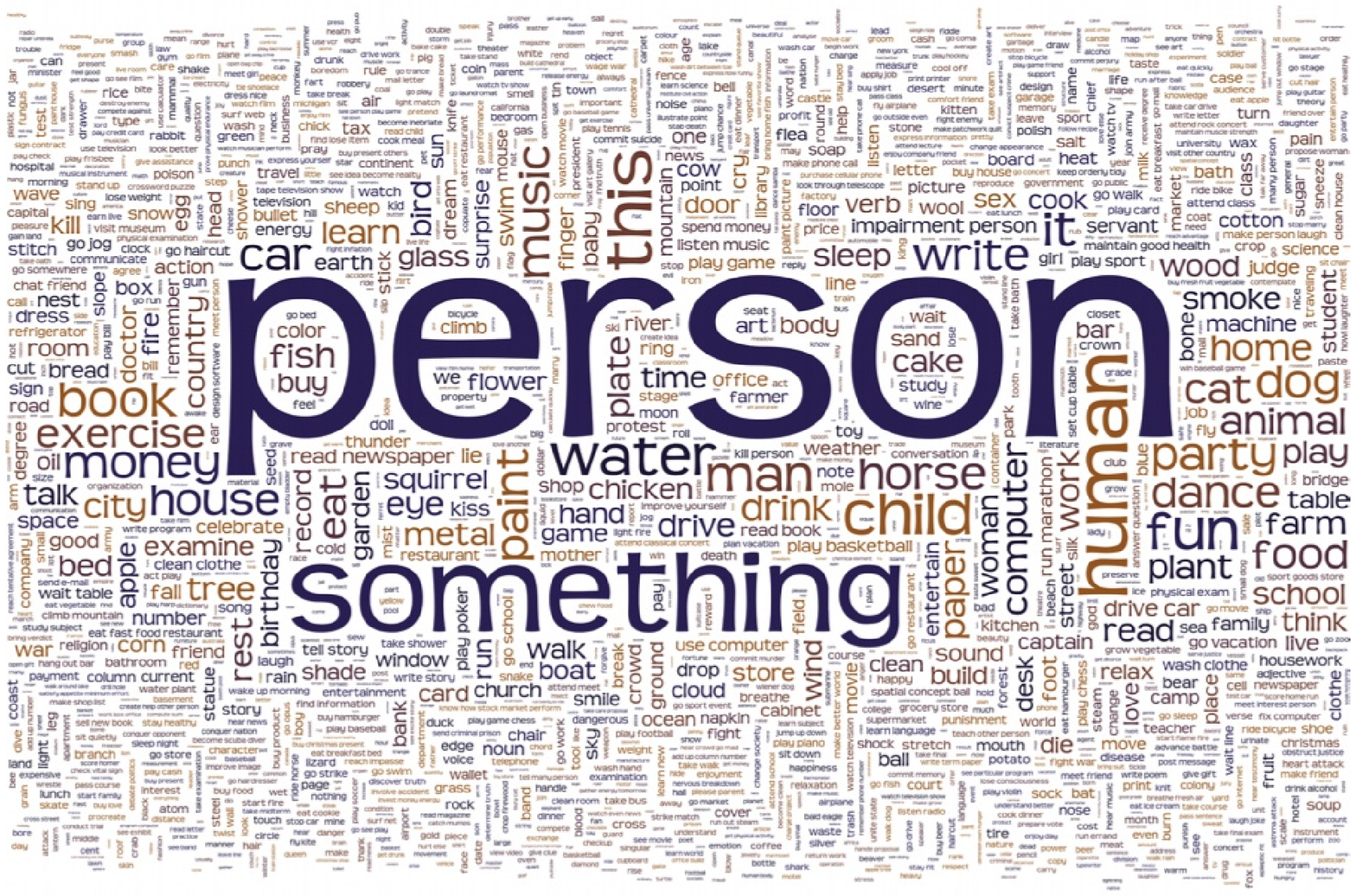}
\end{center}
\caption{The 2013 concepts with highest total degree in the directed multigraph induced by
assertions of positive score in the English language.
For clarity the total degree for the concept \dbtext{person} was scaled down to $\nicefrac{1}{3}$ 
of its actual value (see Table \ref{tbl:degreeDistribution:concepts:top40}) for better visual output.}\label{fig:degreeDistribution:wordle}
\end{figure}

\begin{table}
\caption{The $100$ concepts with the highest total degree in the directed multigraph induced
by the assertions with positive score in the English language.}\label{tbl:degreeDistribution:concepts:top40}
\begin{center}
\begin{tabular}{|l|r|}\hline
\multicolumn{1}{|c|}{concept} & \multicolumn{1}{|c|}{degree} \\\hline\hline
   person & $19,172$ \\\hline
something &  $2,893$ \\\hline
    human &  $1,794$ \\\hline
     this &  $1,637$ \\\hline
    child &  $1,500$ \\\hline
      fun &  $1,378$ \\\hline
    water &  $1,366$ \\\hline
     book &  $1,241$ \\\hline
       it &  $1,208$ \\\hline
      man &  $1,204$ \\\hline
      dog &  $1,152$ \\\hline
    money &  $1,133$ \\\hline
    party &  $1,128$ \\\hline
    paint &  $1,124$ \\\hline
    music &  $1,123$ \\\hline
    horse &  $1,122$ \\\hline
      car &  $1,114$ \\\hline
    write &  $1,095$ \\\hline
    house &  $1,089$ \\\hline
    dance &  $1,076$ \\\hline
\end{tabular}
\hspace{\fill}
\begin{tabular}{|l|r|}\hline
\multicolumn{1}{|c|}{concept} & \multicolumn{1}{|c|}{degree} \\\hline\hline
     food & $1,042$ \\\hline
      cat & $1,010$ \\\hline
 exercise &   $986$ \\\hline
   animal &   $971$ \\\hline
      eat &   $960$ \\\hline
    drink &   $927$ \\\hline
     home &   $906$ \\\hline
     fish &   $881$ \\\hline
 computer &   $876$ \\\hline
    paper &   $865$ \\\hline
    plant &   $846$ \\\hline
     city &   $832$ \\\hline
    plate &   $825$ \\\hline
     play &   $818$ \\\hline
     work &   $807$ \\\hline
     tree &   $801$ \\\hline
      eye &   $798$ \\\hline
    drive &   $796$ \\\hline
    learn &   $793$ \\\hline
     farm &   $793$ \\\hline
\end{tabular}
\hspace{\fill}
\begin{tabular}{|l|r|}\hline
\multicolumn{1}{|c|}{concept} & \multicolumn{1}{|c|}{degree} \\\hline\hline
   metal & $784$ \\\hline
    read & $781$ \\\hline
    cake & $760$ \\\hline
    rest & $756$ \\\hline
   sleep & $752$ \\\hline
    talk & $750$ \\\hline
     bed & $749$ \\\hline
    bird & $744$ \\\hline
   smoke & $732$ \\\hline
    wood & $732$ \\\hline
  school & $730$ \\\hline
    time & $714$ \\\hline
 country & $708$ \\\hline
 chicken & $704$ \\\hline
squirrel & $700$ \\\hline
   glass & $695$ \\\hline
     buy & $689$ \\\hline
   woman & $684$ \\\hline
    hand & $683$ \\\hline
   think & $672$ \\\hline
\end{tabular}
\hspace{\fill}
\begin{tabular}{|l|r|}\hline
\multicolumn{1}{|c|}{concept} & \multicolumn{1}{|c|}{degree} \\\hline\hline
    walk & $661$ \\\hline
    wind & $655$ \\\hline
birthday & $648$ \\\hline
    kill & $643$ \\\hline
  garden & $643$ \\\hline
   build & $642$ \\\hline
   apple & $638$ \\\hline
 examine & $638$ \\\hline
  record & $633$ \\\hline
    cook & $625$ \\\hline
   table & $620$ \\\hline
    verb & $620$ \\\hline
    boat & $618$ \\\hline
    fire & $615$ \\\hline
  flower & $615$ \\\hline
    door & $610$ \\\hline
    body & $610$ \\\hline
     run & $604$ \\\hline
    desk & $595$ \\\hline
     sex & $589$ \\\hline
\end{tabular}
\hspace{\fill}
\begin{tabular}{|l|r|}\hline
\multicolumn{1}{|c|}{concept} & \multicolumn{1}{|c|}{degree} \\\hline\hline
     game & $588$ \\\hline
   doctor & $583$ \\\hline
      die & $581$ \\\hline
      bar & $578$ \\\hline
      oil & $573$ \\\hline
    store & $568$ \\\hline
     room & $567$ \\\hline
    sound & $564$ \\\hline
     swim & $563$ \\\hline
     card & $562$ \\\hline
     baby & $561$ \\\hline
drive car & $558$ \\\hline
   finger & $544$ \\\hline
     live & $541$ \\\hline
     love & $541$ \\\hline
 surprise & $540$ \\\hline
  machine & $540$ \\\hline
    shade & $537$ \\\hline
     corn & $529$ \\\hline
    earth & $528$ \\\hline
\end{tabular}
\end{center}
\end{table}

\section{Average Degrees}
The average degree in every case is given by $2\abs{E}/\abs{V}$. 
Regarding the number of edges we use the entries found in Table \ref{tbl:number-of-edges-isolated-vertices:overall}.
As of the number of vertices, we use both $279,497$ which is the amount of concepts
appearing among all the assertions in the English language regardless of the score of the
assertions (Convention \ref{conv:number-of-concepts}), 
as well as the smaller values that are obtained when we subtract from that number
the number of the isolated vertices that is given in Table \ref{tbl:number-of-edges-isolated-vertices:overall}.
The multigraph has an average degree of roughly $3.6$, the directed graph of roughly $3.1$,
and the undirected graph of roughly $3.0$.
Table \ref{tbl:conceptnet:average-degree} gives the details in every case.
\begin{table}[ht]
\caption{The average degree of the induced multigraphs and graphs of \conceptnet.
All values are rounded in the third decimal point. The number of vertices in the induced graphs
is considered to be equal to $279,497$. Inside the parentheses we see the values
that are obtained when we subtract from those vertices the number of isolated vertices
as these are described in Table \ref{tbl:number-of-edges-isolated-vertices:overall}.}\label{tbl:conceptnet:average-degree}
\begin{center}
\begin{tabular}{|c|c|c||c|c|c|}\hline
score    & self-loops & polarity & directed multigraph & directed graph    & undirected graph \\\hline\hline

all      & \xmark     & negative & $0.110 \ (2.490)$ & $0.109 \ (2.464)$ & $0.105 \ (2.389)$ \\\hline
all      & \xmark     & positive & $3.938 \ (4.021)$ & $3.334 \ (3.404)$ & $3.238 \ (3.306)$ \\\hline
all      & \xmark     &     both & $4.047 \ (4.047)$ & $3.425 \ (3.425)$ & $3.326 \ (3.326)$ \\\hline\hline

all      & \cmark     & negative & $0.110 \ (2.493)$ & $0.109 \ (2.467)$ & $0.105 \ (2.392)$ \\\hline
all      & \cmark     & positive & $3.941 \ (4.024)$ & $3.336 \ (3.406)$ & $3.240 \ (3.308)$ \\\hline
all      & \cmark     &     both & $4.051$ & $3.427$ & $3.328$ \\\hline\hline

positive & \xmark     & negative & $0.097 \ (2.306)$ & $0.096 \ (2.287)$ & $0.093 \ (2.219)$ \\\hline
positive & \xmark     & positive & $3.424 \ (3.726)$ & $2.955 \ (3.216)$ & $2.872 \ (3.125)$ \\\hline
positive & \xmark     &     both & $3.521 \ (3.747)$ & $3.038 \ (3.234)$ & $2.952 \ (3.142)$ \\\hline\hline

positive & \cmark     & negative & $0.097 \ (2.308)$ & $0.096 \ (2.289)$ & $0.093 \ (2.221)$ \\\hline
positive & \cmark     & positive & $3.427 \ (3.729)$ & $2.957 \ (3.218)$ & $2.874 \ (3.127)$ \\\hline
positive & \cmark     &     both & $3.523 \ (3.750)$ & $3.040 \ (3.236)$ & $2.954 \ (3.144)$ \\\hline
\end{tabular}
\end{center}
\end{table}

\suppressfloats[b]

\section{Degree Distribution}
Figure \ref{fig:degreeDistribution:loglog} gives the degree distribution 
for the directed multigraph induced by assertions with positive score in three cases.
Recall that the polarity of the assertions can be both positive and negative. 
Hence the three cases that we
distinguish in the plots in Figure \ref{fig:degreeDistribution:loglog} are for the cases where:
\begin{itemize}[noitemsep,leftmargin=8mm,topsep=0.5mm]
 \item arbitrary polarity is allowed; that is both positive and negative,
 \item negative only polarity is allowed, and
 \item positive only polarity is allowed.
\end{itemize}


\begin{figure}[ht]
\begin{center}
\begin{subfigure}[b]{0.48\textwidth}
\includegraphics[width=\textwidth]{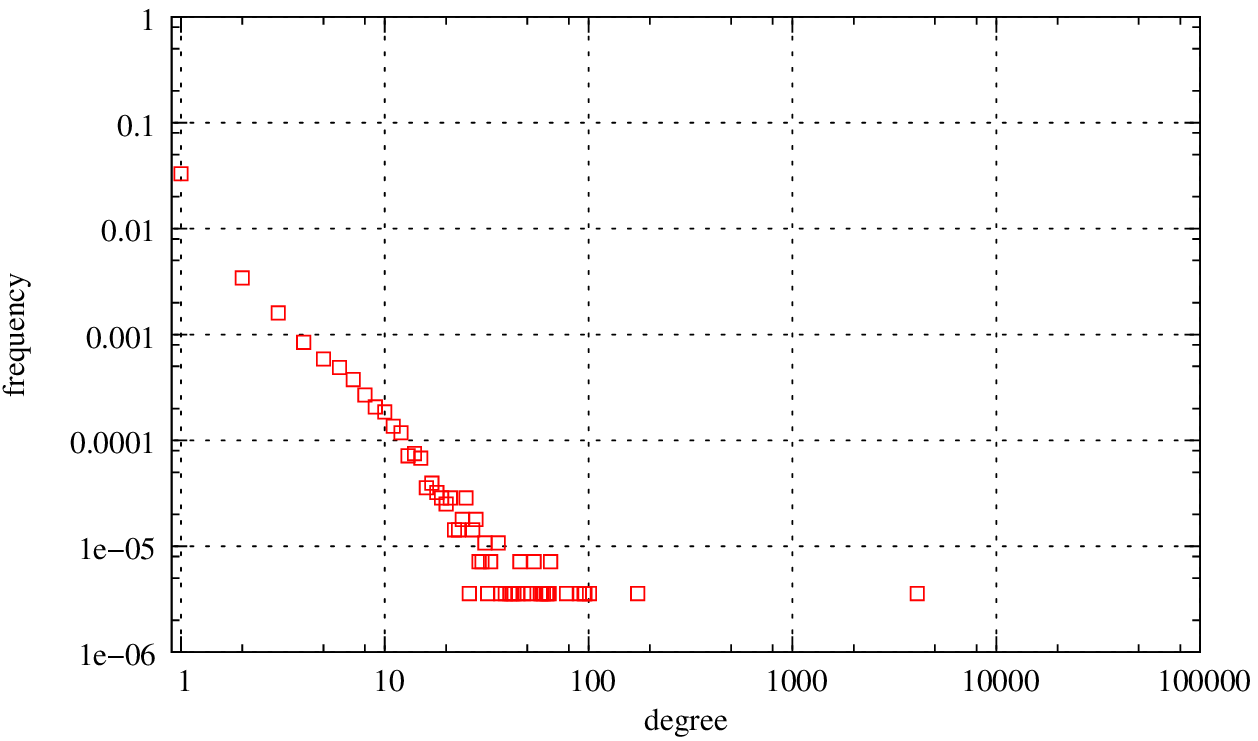}
\caption{Degree distribution when only negative polarity is taken into account.}\label{fig:degreeDistribution:loglog:negative}
\end{subfigure}
\hspace{0.02\textwidth}
\begin{subfigure}[b]{0.48\textwidth}
\includegraphics[width=\textwidth]{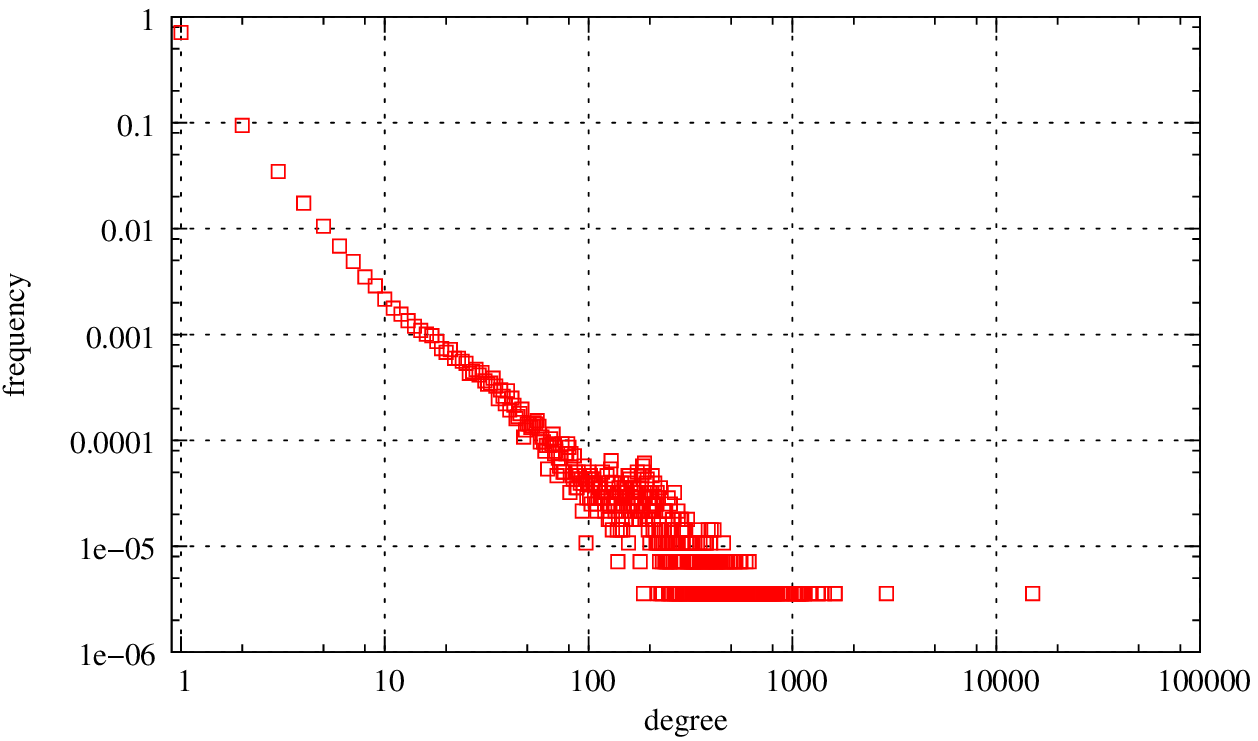}
\caption{Degree distribution when only positive polarity is taken into account.}\label{fig:degreeDistribution:loglog:positive}
\end{subfigure}

\vspace{0.05\columnwidth}

\begin{subfigure}[b]{0.9\textwidth}
\includegraphics[width=\textwidth]{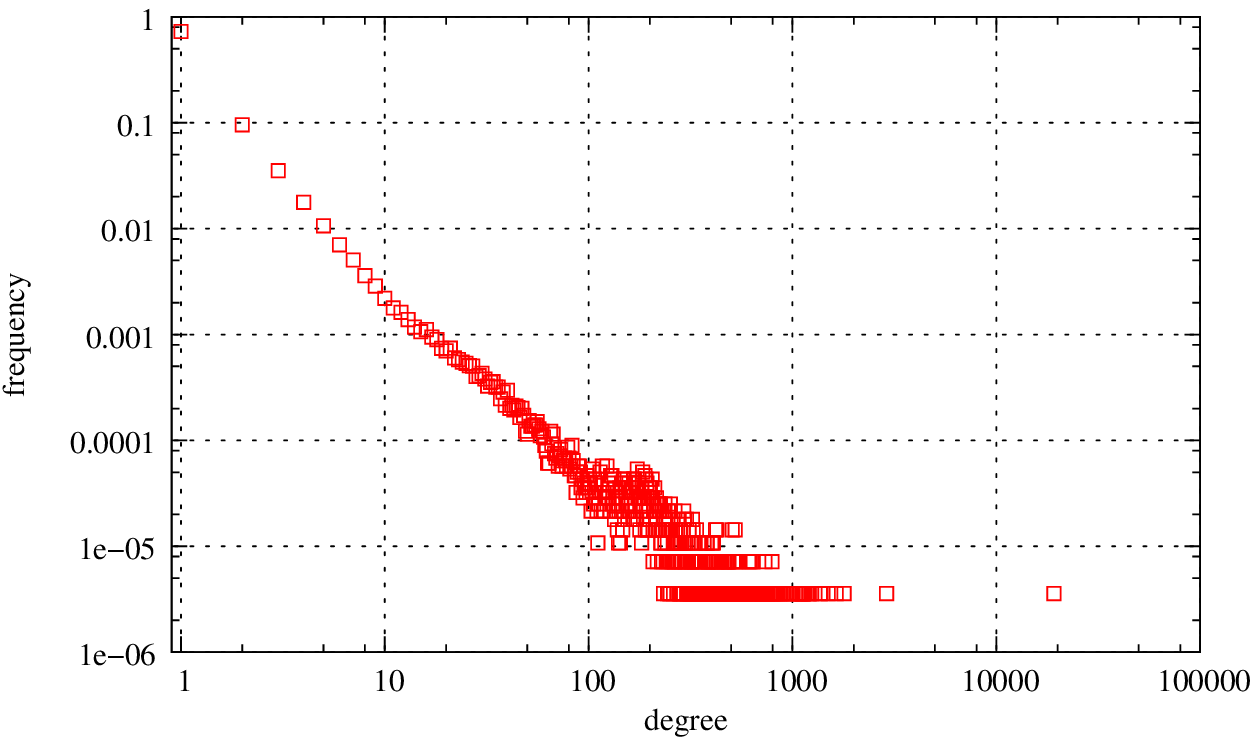}
\caption{Degree distribution when all polarities are taken into account.}\label{fig:degreeDistribution:loglog:all}
\end{subfigure}
\end{center}
\caption{Degree distributions in three different cases for the induced directed multigraph. 
In every case we take into account only the assertions of the English language with positive score. 
The different cases arise if we further want to differentiate and take into account assertions with 
negative polarity only, positive polarity only, or finally arbitrary polarity.}\label{fig:degreeDistribution:loglog}
\end{figure}

The initial segment of the total-degree distribution 
(edges with both positive and negative polarity are allowed) 
is given in Table \ref{tbl:degreeDistribution:total:initial-segment}.

\begin{table}
\caption{The initial segment of the total-degree distribution in the directed multigraph induced
by the assertions of the English language with positive score. The values shown for the frequencies
in the third row are merely the numerical values obtained from the quotient $\frac{\mbox{number of concepts with degree $d$}}{279,497}$\;,
where $279,497$ is the number of nodes for the entire network according to Convention \ref{conv:number-of-concepts}.}\label{tbl:degreeDistribution:total:initial-segment}
\begin{center}
\resizebox{\textwidth}{!}{
\begin{tabular}{|r||r|r|r|r|r|r|r|r|r|r|r|}\hline
   degree &        $0$ &        $1$ &        $2$ &        $3$ &        $4$ &        $5$ &        $6$ &         $7$ &        $8$ &        $9$ & $\ldots$ \\\hline
 concepts &   $16,920$ &  $203,556$ &   $26,775$ &    $9,880$ &    $4,959$ &    $2,968$ &    $1,962$ &     $1,415$ &    $1,007$ &      $802$ & $\ldots$ \\\hline
frequency & $0.060537$ & $0.728294$ & $0.095797$ & $0.035349$ & $0.017743$ & $0.010619$ & $0.007020$ & $0.0050627$ & $0.003603$ & $0.002869$ & $\ldots$ \\\hline
\end{tabular}
}
\end{center}
\end{table}

\subsection{Fitting}
We investigate the three networks presented in Figure \ref{fig:degreeDistribution:loglog}
using the \emph{method of maximum likelihood} suggested in \citep{powerlaw} 
and the relevant tools 
\footnote{ \url{http://tuvalu.santafe.edu/~aaronc/powerlaws/}}
that are available online
\footnote{ We urge the reader to go through \url{http://vserver1.cscs.lsa.umich.edu/~crshalizi/weblog/491.html} as well.}.
The script that we use for power law fitting, 
is the implementation of Tam\'as Nepusz \footnote{ \url{https://github.com/ntamas/plfit}},
version $0.7$. A typical execution of the script for the results presented below.
\begin{code}
\$ plfit -M -p approximate inputFile
\end{code}
Hence we also get the first four central moments of the degree distribution, 
as well as calculate an approximate $p$-value.
Note that in order for the input to make sense all the concepts that are part of the input should
have degree at least $1$. In other words, we have to omit from the input all the isolated vertices.
Detailed results for every case are presented in Table \ref{tbl:powerlaw}.
Using \plplot by Joel Ornstein we obtain the figures shown in Figure \ref{fig:powerlaw:plplot}.

\begin{table}[ht]
\caption{Fitting power law in the degree distributions of \conceptnet on the multigraphs induced by the assertions with
negative polarity only, positive polarity only, or both. The exponent (scaling) is denoted by $\upalpha$, 
$x_{\min}$ is the lower bound to the power law behavior, $\mathcal{L}$ is the maximum log-likelihood, 
$D$ is the Kolmogorov-Smirnov (or KS) statistic, and $p$ is for the $p$-value.}\label{tbl:powerlaw}
\centering
\resizebox{\columnwidth}{!}{
\begin{tabular}{|l||c|r||r|c|c||c|c|r|c|r|}\hline
\multicolumn{1}{|c||}{polarity} & $\upalpha$ & \multicolumn{1}{c||}{$x_{\min}$} & \multicolumn{1}{c|}{$\mathcal{L}$} & $D$ & $p$ & mean & variance & \multicolumn{1}{c|}{std.~dev.} & skewness & \multicolumn{1}{c|}{kurtosis} \\\hline\hline 
negative & $2.77868$ & $10$ & $   -994.91$ & $0.01532$ & $0.0082$ & $2.308$ & $1,450.692$ & $38.088$ & $106.245$ & $ 11,423.245$ \\\hline 
positive & $1.82643$ & $ 5$ & $-66,869.11$ & $0.02699$ & $0.0000$ & $3.729$ & $1,488.787$ & $38.585$ & $239.212$ & $ 90,850.012$ \\\hline 
both     & $1.82572$ & $ 5$ & $-68,098.45$ & $0.02646$ & $0.0000$ & $3.750$ & $2,021.043$ & $44.956$ & $300.041$ & $126,012.584$ \\\hline 
\end{tabular}
}
\end{table}

\begin{figure}[ht]
\begin{center}
\begin{center}
\begin{subfigure}[b]{0.48\textwidth}
\includegraphics[width=\textwidth]{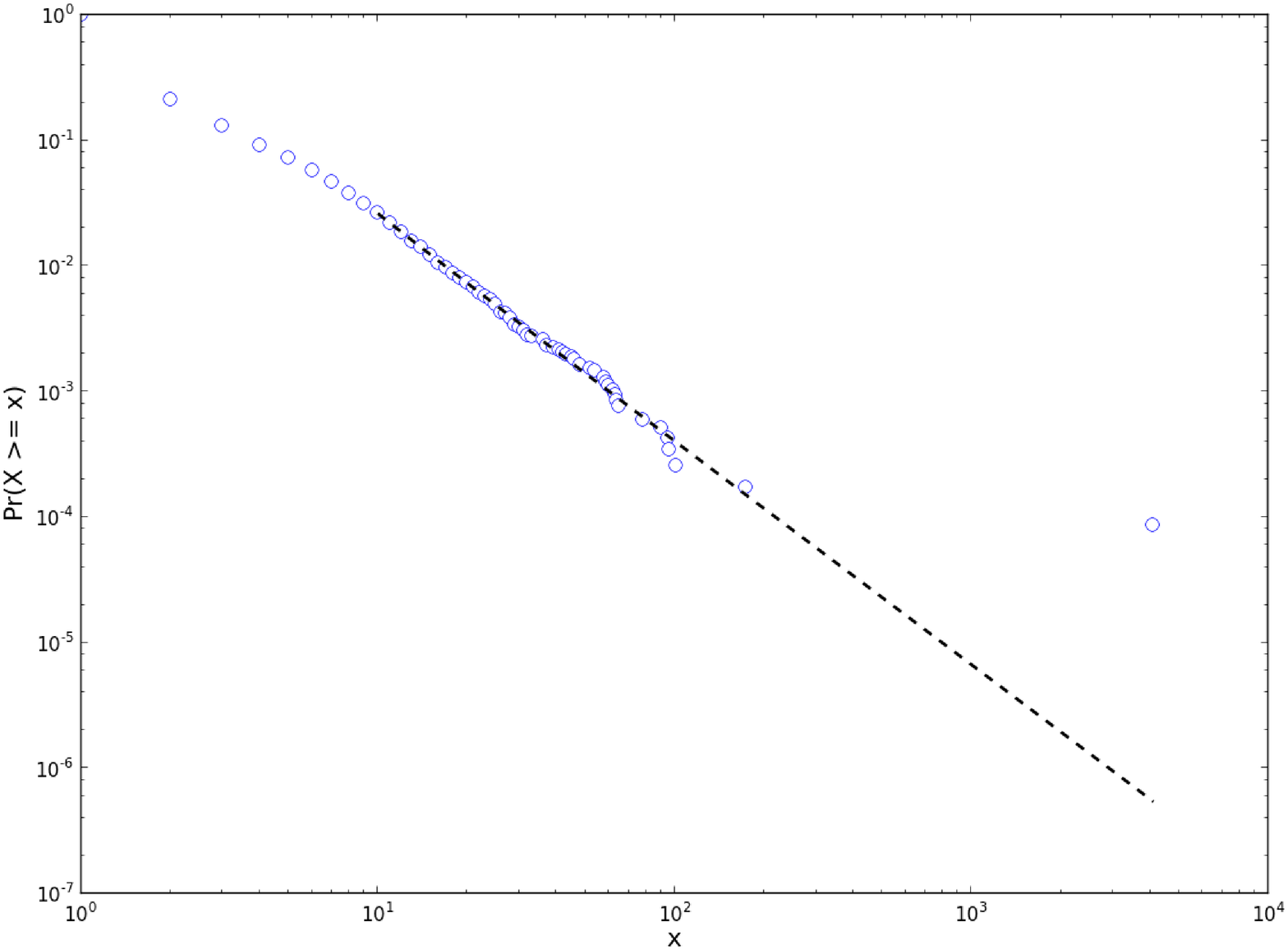}
\caption{Power law fitting when only negative polarity is taken into account.}\label{fig:MLE:negative}
\end{subfigure}
\hspace{0.02\textwidth}
\begin{subfigure}[b]{0.48\textwidth}
\includegraphics[width=\textwidth]{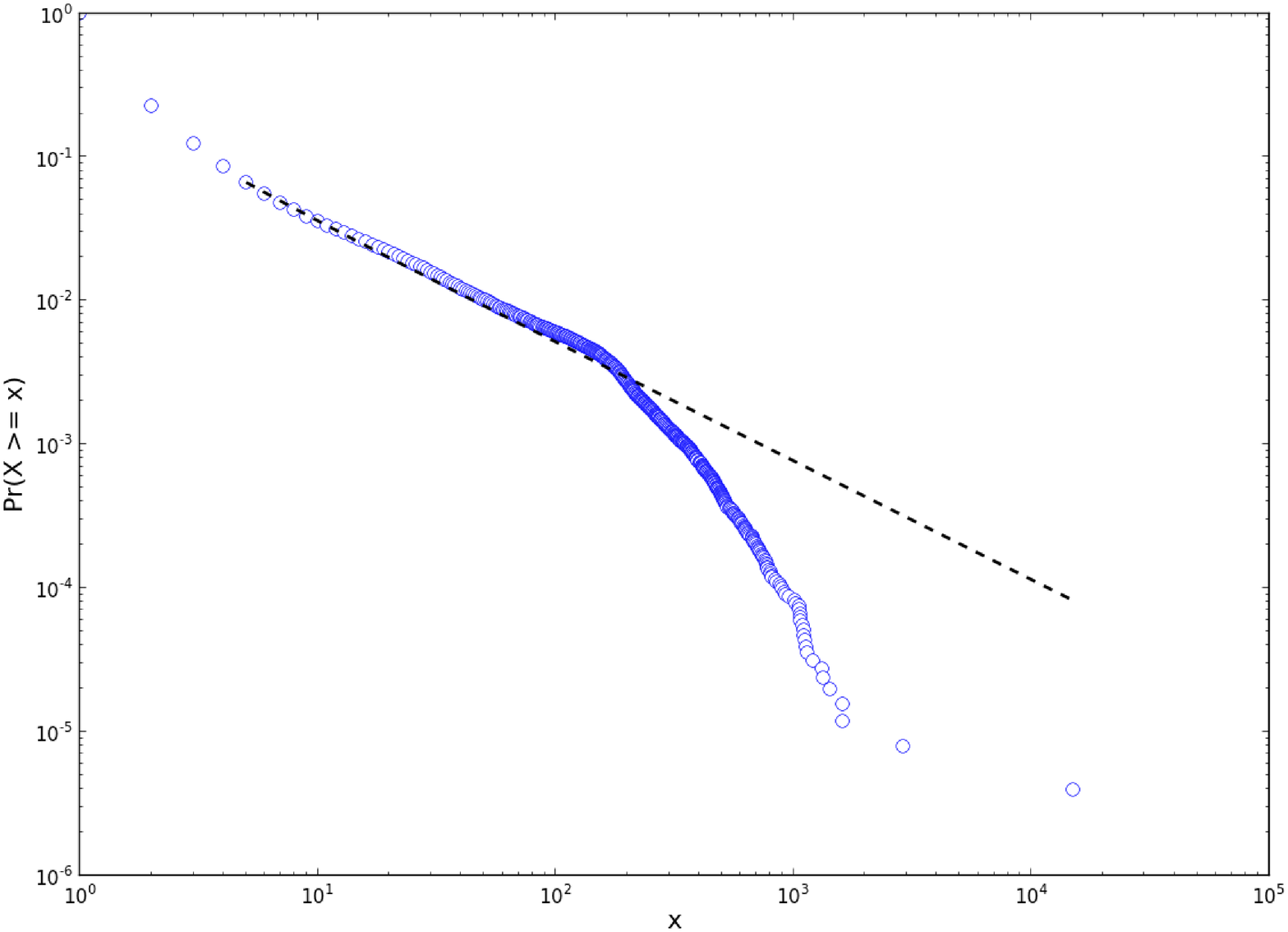}
\caption{Power law fitting when only positive polarity is taken into account.}\label{fig:MLE:positive}
\end{subfigure}

\vspace{0.05\columnwidth}

\begin{subfigure}[b]{0.9\textwidth}
\includegraphics[width=\textwidth]{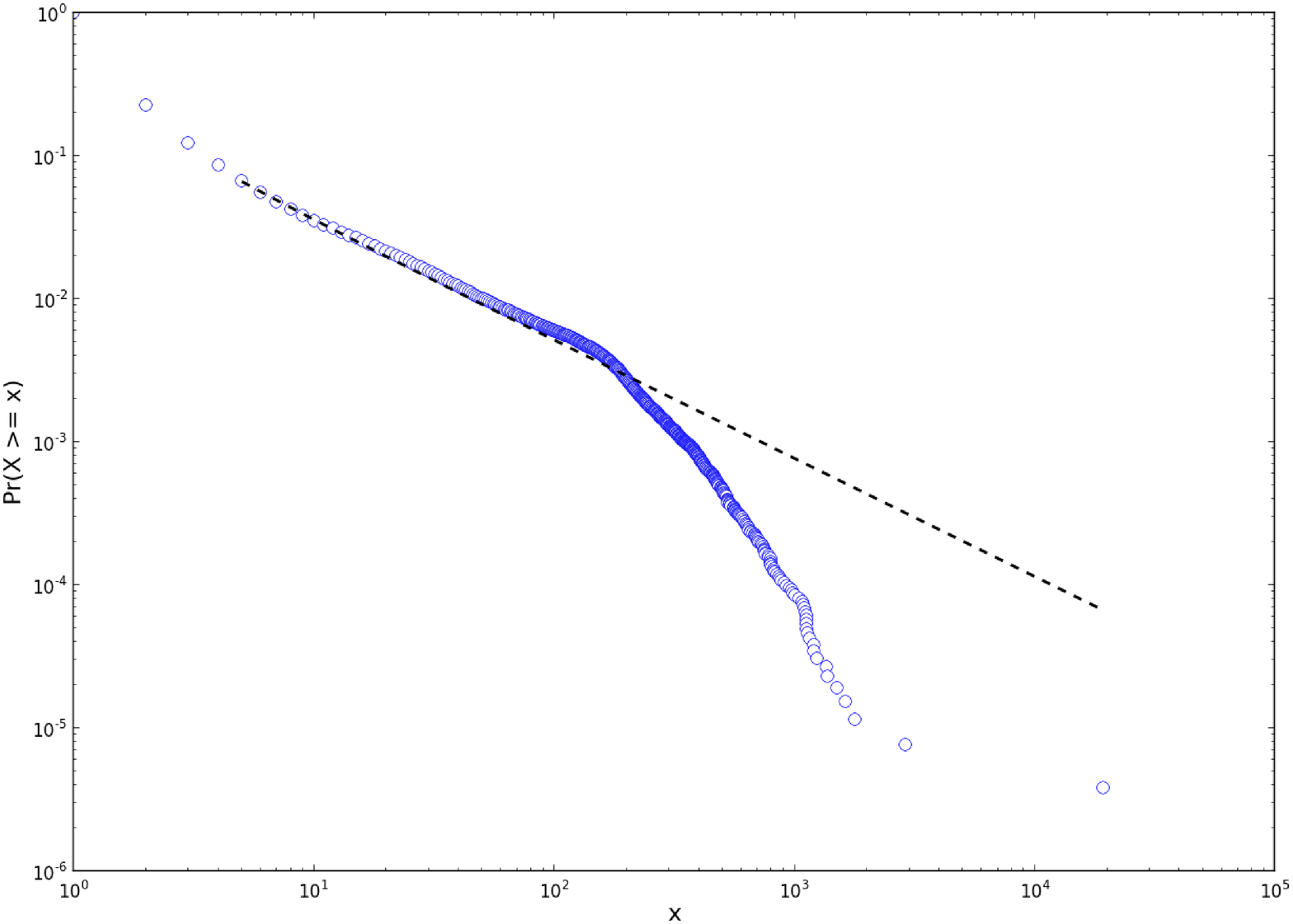}
\caption{Power law fitting when all polarities are taken into account.}\label{fig:MLE:all}
\end{subfigure}
\end{center}
\caption{}\label{fig:MLE}
\end{center}
\caption{Power law fitting in the three major degree distributions of \conceptnet using the
method of maximum likelihood presented in \citep{powerlaw}.}\label{fig:powerlaw:plplot}
\end{figure}

\chapter{Connected Components, Transitivity, and Clustering Coefficient}\label{chapter:components}
For the computations found in this chapter we are going to neglect self-loops in the
directed or undirected graphs induced by the assertions of the English language with positive score.
The reason is that self-loops do not affect the connectivity of the components.
We use the function \texttt{igraph\_clusters} of \igraph \citep{igraph} 
to compute the connected components of the graphs.

\begin{definition}[Global Transitivity \citep{transitivity}]
Transitivity measures the probability that two neighbors of a vertex are connected. 
More precisely, it is the ratio of the triangles and connected triples in the graph.
\end{definition}

\begin{definition}[Average Local Transitivity or Clustering Coefficient \citep{clusteringCoefficient}]
The average local transitivity also measures the probability that two neighbors of a vertex are connected. 
However, in case of the average local transitivity, this probability is calculated for each vertex and then the average is taken. 
Vertices with less than two neighbors require special treatment; 
they will either be left out from the calculation, or they will be considered as having zero transitivity.
Note that this measure is different from the global transitivity measure mentioned above 
as it simply takes the average local transitivity across the whole network. 
See \citep{clusteringCoefficient} for more details.
\end{definition}


Clustering coefficient is an alternative name for transitivity \citep{igraph}. In this document we will
imply the average local transitivity whenever we refer to the clustering coefficient.

\section{Transitivity and Clustering Coefficient}
Table \ref{tbl:transitivity:entire} presents
the transitivity and the clustering coefficient for the undirected graph induced by the assertions of the
English language with positive score neglecting self-loops.

\begin{table}[ht]
\caption{Transitivity and clustering coefficient for the entire graph of \conceptnet
induced by assertions with negative only polarity, positive only polarity, and both polarities.
The first value (\textsc{nan}) for the clustering coefficient gives the result of the calculation
when vertices with less than two neighbors are left out from the calculation,
while the second value (\textsc{zero}) gives the result of the calculation when 
vertices with less than two neighbors are considered as having zero transitivity.
Note that all values are the same both for directed as well as undirected graphs.}\label{tbl:transitivity:entire}
\begin{center}
\begin{tabular}{|l|r|r|r|}\cline{2-4}
                 \multicolumn{1}{r|}{} & \multicolumn{3}{c|}{polarity} \\\cline{2-4}
                 \multicolumn{1}{r|}{} & \multicolumn{1}{c|}{negative} & \multicolumn{1}{c|}{positive} & \multicolumn{1}{c|}{both} \\\hline
Transitivity                           & $0.000351298700593188$ & $0.004964054809387655$ & $0.003881697564836174$ \\\hline
Clustering Coefficient (\textsc{nan})  & $0.098300551193575281$ & $0.190154012754549323$ & $0.196101493828584605$ \\\hline
Clustering Coefficient (\textsc{zero}) & $0.000851478314399245$ & $0.032692554034191273$ & $0.034448280818741697$ \\\hline
\end{tabular}
\end{center}
\end{table}

\section{Negative Polarity: Connected Components}
First we examine the case of the directed and undirected graph induced by the assertions with negative polarity.

\subsection{Weakly Connected Components}
We get $269,167$ weakly connected components, out of which
$267,790$ are isolated vertices. Note that $267,790$ is in complete agreement with
Table \ref{tbl:number-of-edges-isolated-vertices:overall}. 
Among the rest $1,377$ components we can find
components with cardinalities between $2$ and $8,596$.

\paragraph{Distribution of Component Sizes.}
The distribution of the sizes for the various components is shown in 
Table \ref{tbl:distribution:component:negative:weak}.
This distribution presents the cardinalities of the weakly connected components
of the induced directed graph, as well as the cardinalities of the connected
components of the induced undirected graph.

\begin{table}[ht]
\caption{Distribution of sizes for weakly connected components for the induced directed 
graph. This is also the distribution of sizes for the connected 
components of the induced undirected graph. 
}\label{tbl:distribution:component:negative:weak}
\begin{center}
\begin{tabular}{|r||c|c|c|c|c|c|c|c|c|c|c|c|c|c|}\hline
\# of nodes      & $8,596$ & $13$ & $9$ & $8$ & $7$ & $6$ &  $5$ &  $4$ &   $3$ &     $2$ &       $1$ \\\hline
\# of components &     $1$ &  $2$ & $2$ & $2$ & $4$ & $8$ & $22$ & $28$ & $137$ & $1,171$ & $267,790$ \\\hline
\end{tabular}
\end{center}
\end{table}

Figure \ref{fig:weakly-connected:negative:maximal} presents the maximal weakly connected component of size $8,596$. 
Figure \ref{fig:weakly-connected:negative} presents the weakly connected components with sizes $8$, $9$ and $13$.

\begin{figure}[ht]
\begin{center}
\includegraphics[width=0.7\textwidth]{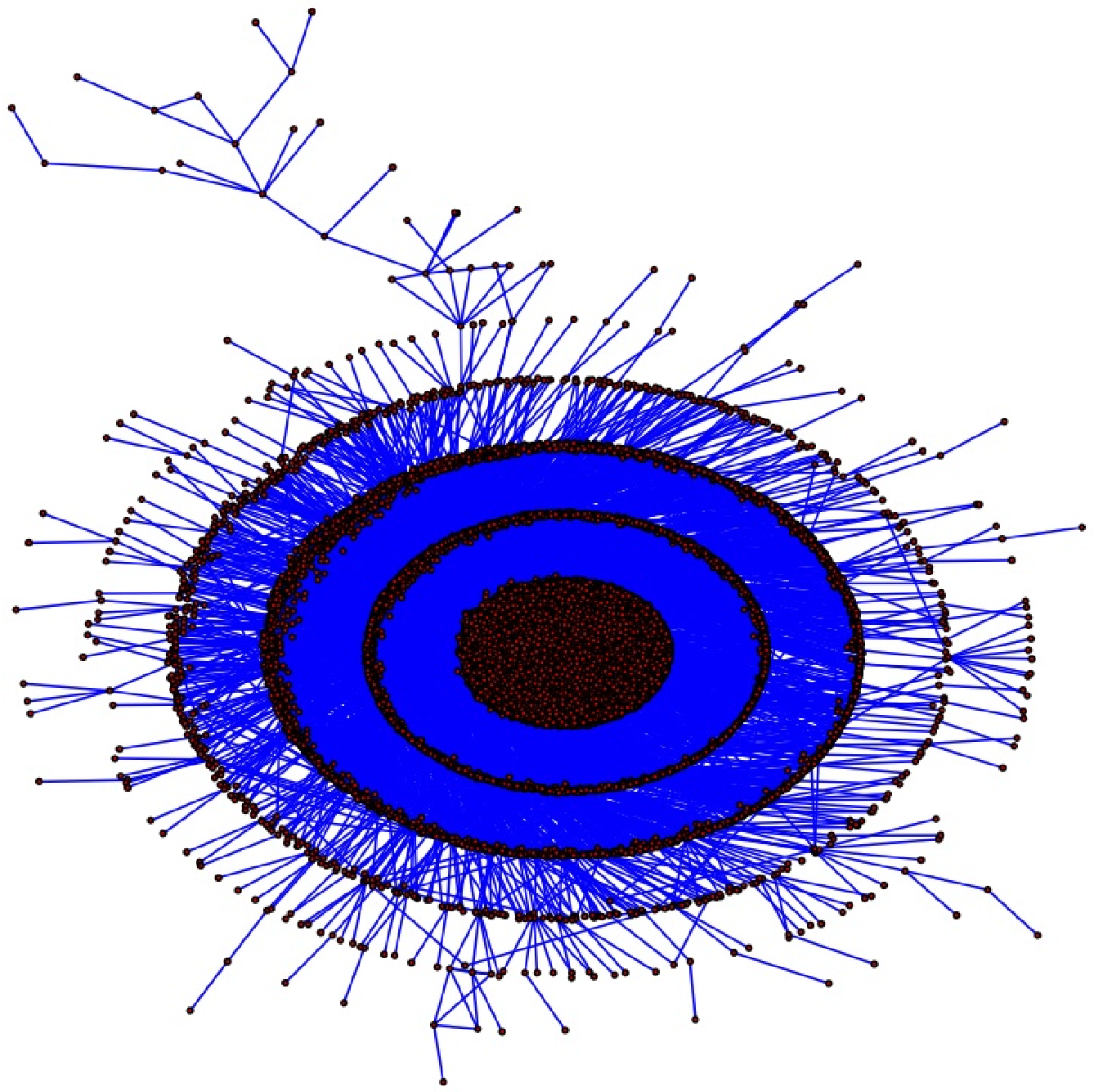}
\end{center}
\caption{The maximal weakly connected component of the graph induced by the assertions with 
negative polarity; see Table \ref{tbl:distribution:component:negative:weak}.
For simplicity we plot the induced undirected graph of that component.}\label{fig:weakly-connected:negative:maximal}
\end{figure}

\begin{figure}[ht]
\begin{center}
\begin{subfigure}[b]{0.3\textwidth}
\begin{center}
\includegraphics[width=\textwidth]{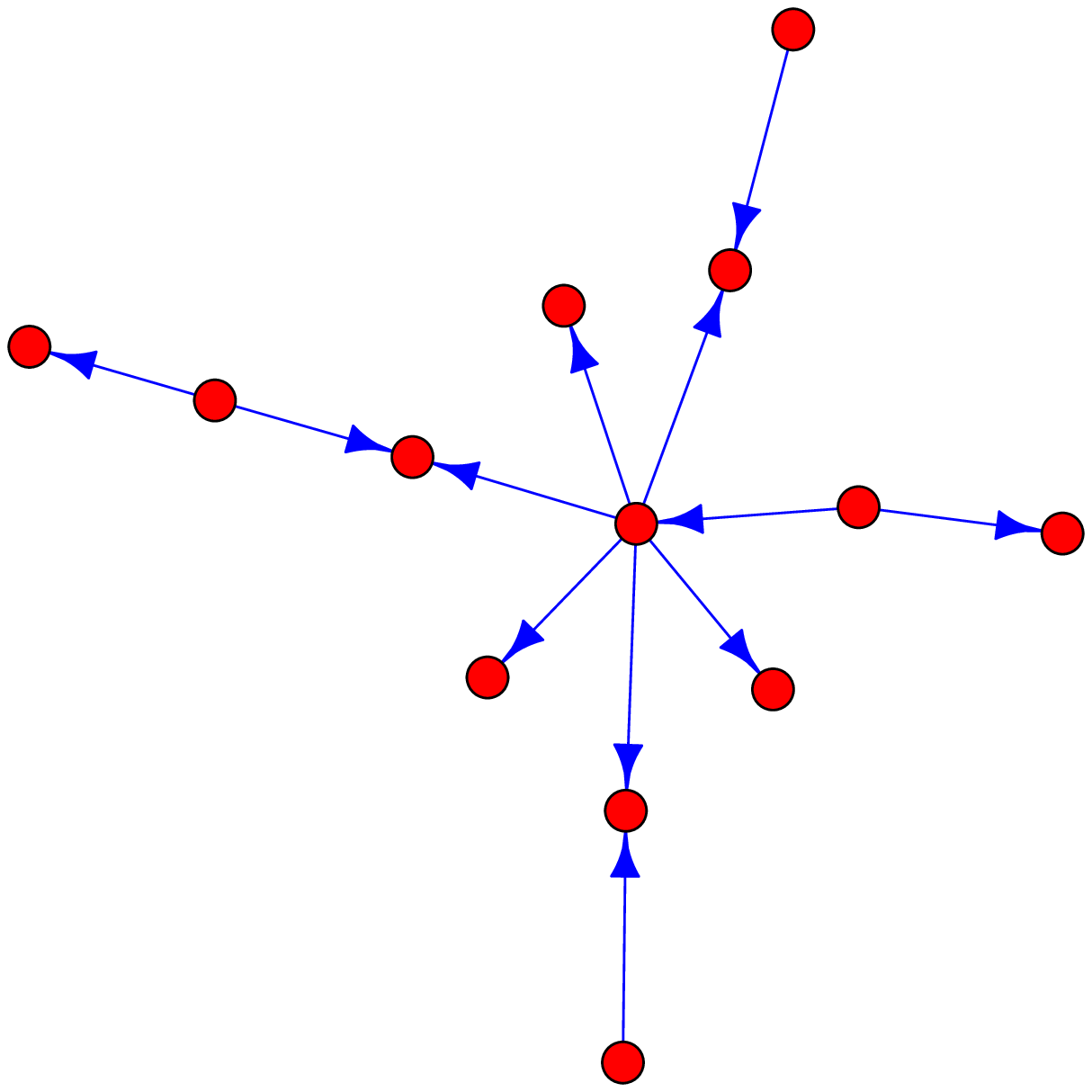}
\end{center}
\caption{may, April, make right, weak, march, will\\
$13$ nodes, $12$ edges.}\label{fig:weakly-connected:negative:13:a}
\end{subfigure}
\hspace{0.01\textwidth}
\begin{subfigure}[b]{0.3\textwidth}
\begin{center}
\includegraphics[width=\textwidth]{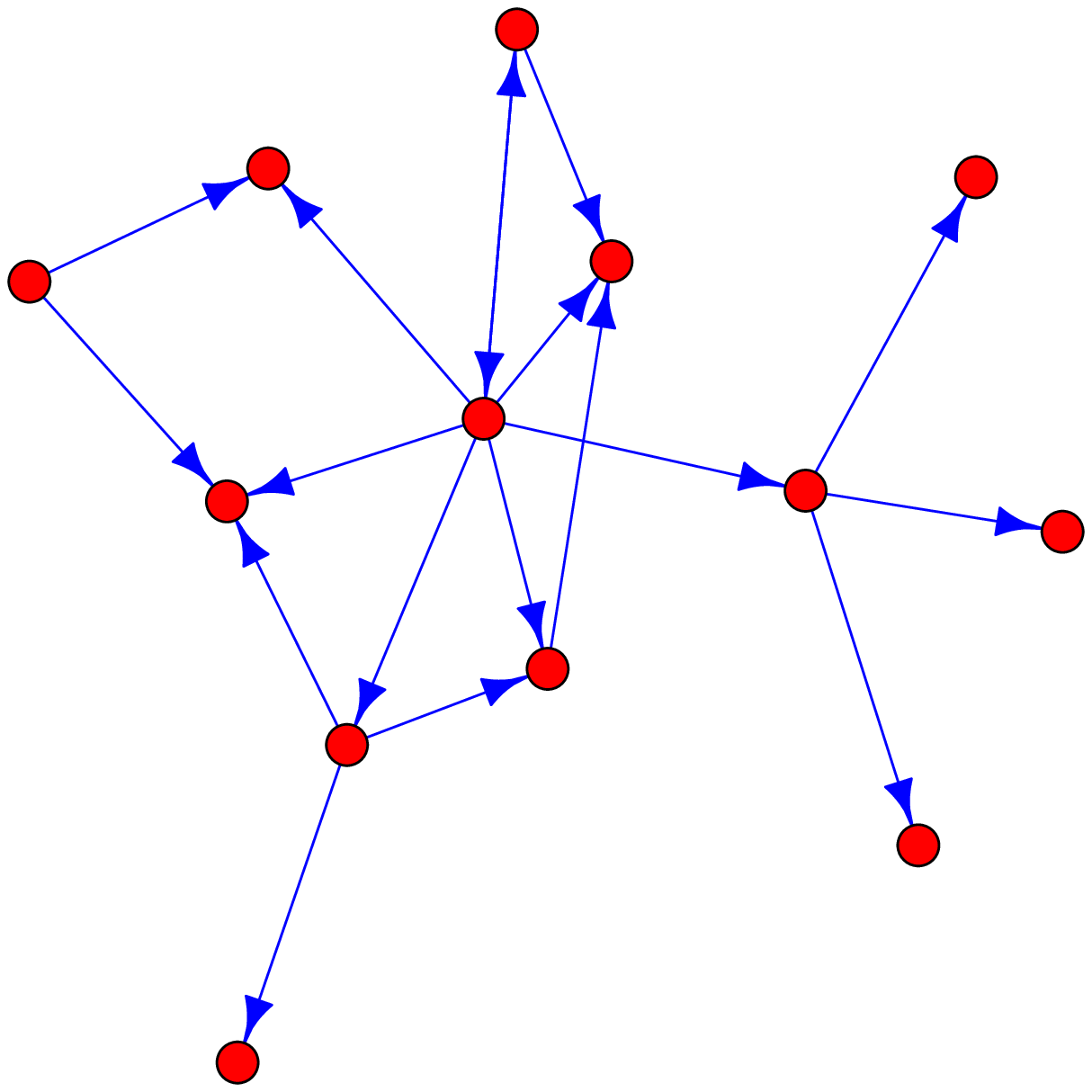}
\end{center}
\caption{division, union, add, addition, subtract, subtraction, multiplication, multiply, divide\\
$13$ nodes, $18$ edges.}\label{fig:weakly-connected:negative:13:b}
\end{subfigure}
\hspace{0.01\textwidth}
\begin{subfigure}[b]{0.3\textwidth}
\begin{center}
\includegraphics[width=\textwidth]{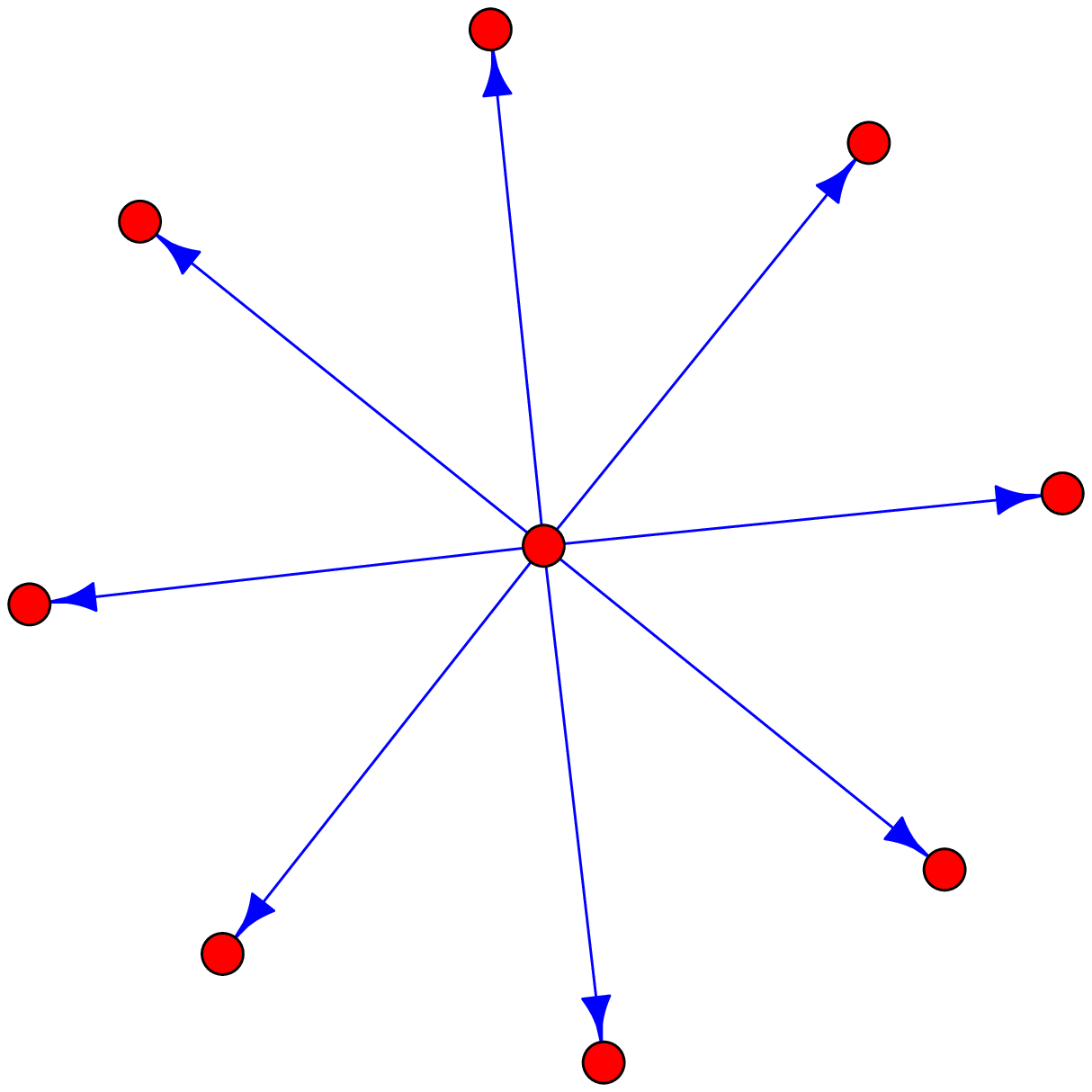}
\end{center}
\caption{if person\\
$9$ nodes, $8$ edges.}\label{fig:weakly-connected:negative:9:a}
\end{subfigure}

\begin{subfigure}[b]{0.3\textwidth}
\begin{center}
\includegraphics[width=\textwidth]{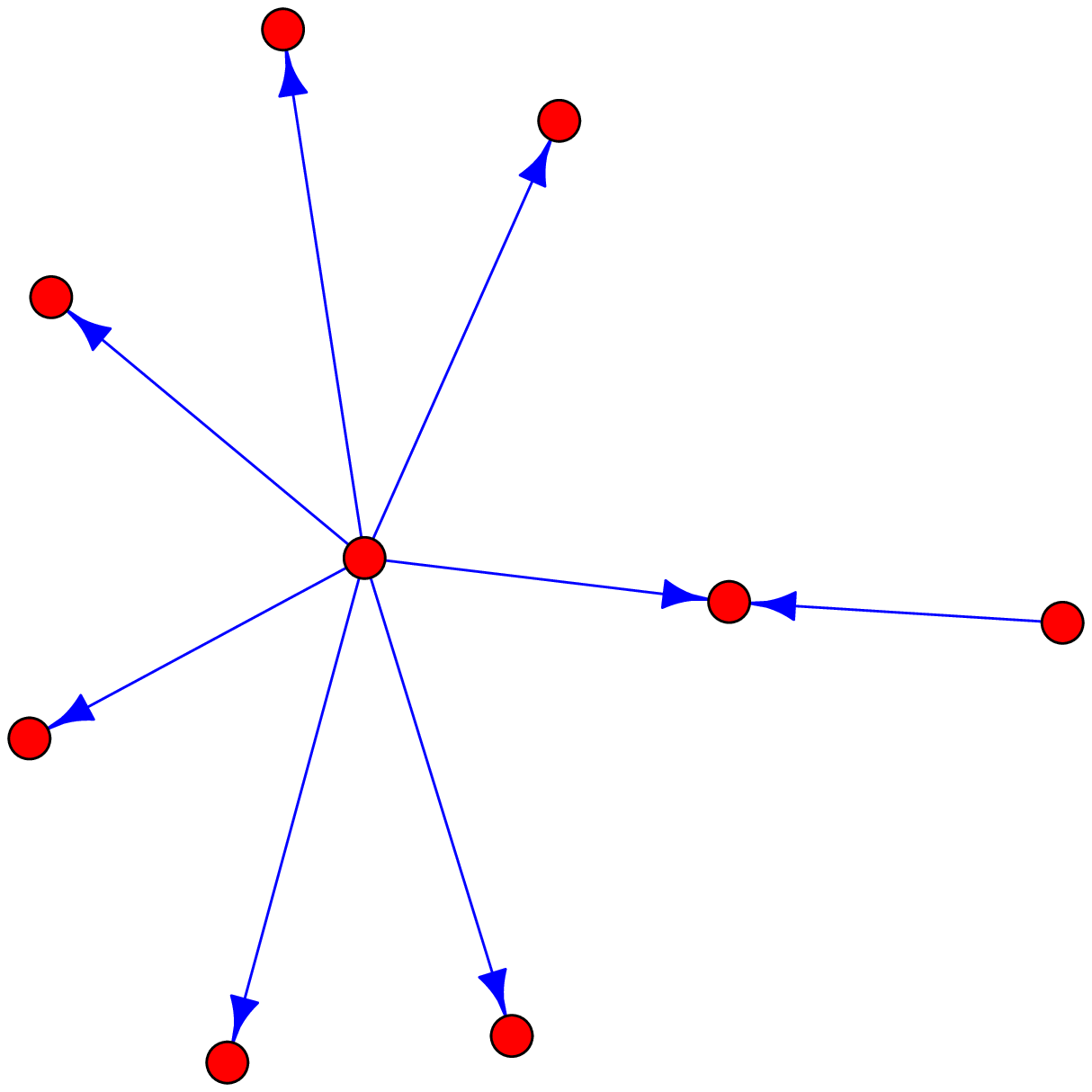}
\end{center}
\caption{topic 'sky, word drop\\
$9$ nodes, $8$ edges.}\label{fig:weakly-connected:negative:9:b}
\end{subfigure}
\hspace{0.01\textwidth}
\begin{subfigure}[b]{0.3\textwidth}
\begin{center}
\includegraphics[width=\textwidth]{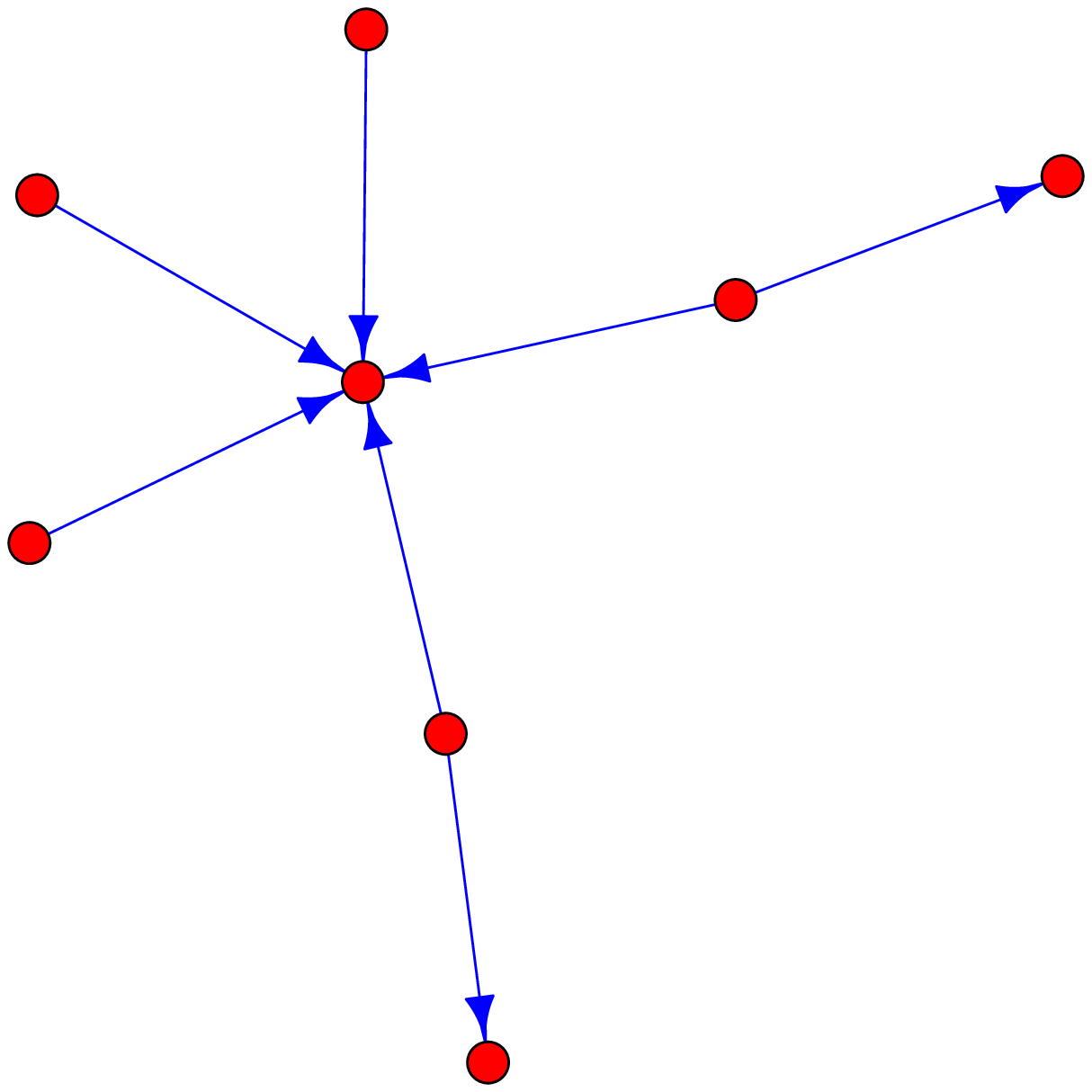}
\end{center}
\caption{comfortable, classroom chair, wicker\\
$8$ nodes, $7$ edges.}\label{fig:weakly-connected:negative:8:a}
\end{subfigure}
\hspace{0.01\textwidth}
\begin{subfigure}[b]{0.3\textwidth}
\begin{center}
\includegraphics[width=\textwidth]{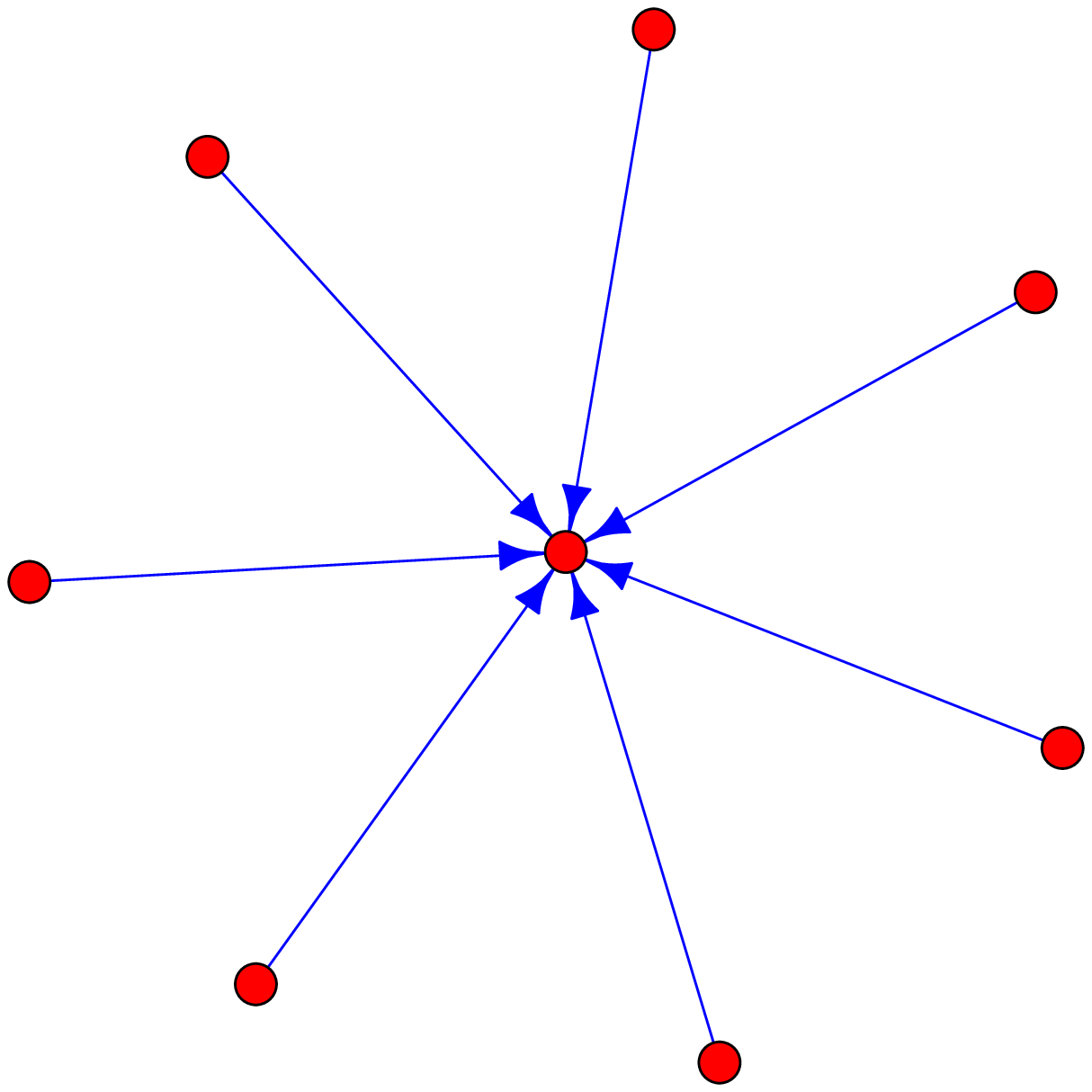}
\end{center}
\caption{good eat\\
$8$ nodes, $7$ edges.}\label{fig:weakly-connected:negative:8:b}
\end{subfigure}
\end{center}

\caption{The induced \emph{directed} subgraphs for some weakly connected components; 
see Table \ref{tbl:distribution:component:negative:weak}. 
The names of the subgraphs are given by the nodes with total degrees different from $1$.
All such nodes are listed in decreasing order of total degree.
In case of ties precedence takes the name of the node
that has larger in-degree.}\label{fig:weakly-connected:negative}
\end{figure}

\subsubsection{Big Weakly Connected Component}
The undirected graph induced by the concepts that appear in the big undirected component 
is composed of $8,596$ nodes and $11,247$ edges.
For information about shortest paths in this component please see Chapter \ref{chapter:shortest-paths}.

\subsubsection{Components of Size $13$}
In the first component of size $13$ 
concept \dbtext{may} (2606) has out-degree $6$ and in-degree $1$.
Concepts \dbtext{April} (2721), \dbtext{make right} (2766), and \dbtext{weak} (21769)
have out-degree $0$ and in-degree $2$.
Concepts \dbtext{march} (2719) and \dbtext{will} (20015)
have out-degree $2$ and in-degree $0$.
Concepts 
\dbtext{February} (2716),
\dbtext{June} (2725),
\dbtext{definite} (37022),
\dbtext{know definition word 'hemisphere} (328106), and
\dbtext{wont} (333527)
have out-degree $0$ and in-degree $1$.
Concepts 
\dbtext{two wrong} (2765) and
\dbtext{feminine woman} (109816)
have out-degree $1$ and in-degree $0$.

In the second component of size $13$
concept \dbtext{division} (14946) has out-degree $7$ and in-degree $1$.
Concepts \dbtext{union} (4832) and \dbtext{add} (54627)
have out-degree $3$ and in-degree $1$.
Concepts \dbtext{addition} (26573) and \dbtext{subtract} (108338)
have out-degree $0$ and in-degree $3$.
Concept \dbtext{subtraction} (161354) has out-degree $1$ and in-degree $2$.
Concept \dbtext{multiplication} (14387) has out-degree $2$ and in-degree $1$
Concept \dbtext{multiply} (25479) has out-degree $0$ and in-degree $2$.
Concept \dbtext{divide} (19901) has out-degree $2$ and in-degree $0$.
The rest four concepts 
\dbtext{intersection} (5593),
\dbtext{minus} (332948),
\dbtext{confederacy} (351157), and
\dbtext{confederate} (369189)
all have out-degree $0$ and in-degree $1$.

\subsubsection{Components of Size $9$}
In the first component of size $9$ concept \dbtext{if person} (48339)
has out-degree $8$ and in-degree $0$.
All the other concepts have out-degree $0$ and in-degree $1$.
These $8$ concepts are
\dbtext{water plant die} (20613),
\dbtext{green card i.n.} (58241),
\dbtext{read never succeed} (74287),
\dbtext{two telephon} (120923),
\dbtext{pay bill go bankrupt} (128794),
\dbtext{go out stay home} (189195),
\dbtext{beat join} (201293), and
\dbtext{license not drive} (428525).

In the second component of size $9$ concept \dbtext{topic 'sky} (311764)
has out-degree $7$ and in-degree $0$.
Concept \dbtext{word drop} (20977) has out-degree $0$ and in-degree $2$.
Concepts 
\dbtext{word metal-frame} (21208),
\dbtext{word aurora} (21870),
\dbtext{word pressure} (22718),
\dbtext{word sagittarius} (22726),
\dbtext{word high} (23450), and
\dbtext{word helium} (23546)
have out-degree $0$ and in-degree $1$.
Finally, concept \dbtext{topic 'liquid} (311723)
has out-degree $1$ and in-degree $0$.

\subsubsection{Components of Size $8$}
In the first component of size $8$ concept \dbtext{comfortable} (371)
has out-degree $0$ and in-degree $5$.
Concepts \dbtext{classroom chair} (26777) and \dbtext{wicker} (34790)
have out-degree $2$ and in-degree $0$.
Concepts \dbtext{sturdy oak} (51878) and \dbtext{build comfort} (175882)
have out-degree $0$ and in-degree $1$.
Finally the concepts \dbtext{sofabed} (4007), \dbtext{chair make outdoor use} (59113),
and \dbtext{sleep couch} (138393)
have out-degree $1$ and in-degree $0$.

In the second component of size $8$ concept \dbtext{good eat} (2543)
has out-degree $0$ and in-degree $7$. All the other concepts have
out-degree $1$ and in-degree $0$.
These $7$ concepts are
\dbtext{hair gel} (3104),
\dbtext{yellow snow} (24319),
\dbtext{cosmetic} (47806),
\dbtext{orange peel} (63084),
\dbtext{crabapple} (103589),
\dbtext{unripe orange} (117785), and
\dbtext{peel orange} (117790).

\subsection{Strongly Connected Components}
We get $278,783$ strongly connected components, out of which
$278,708$ are isolated vertices. 
Among the rest $75$ components we can find
components with cardinalities between $2$ and $592$.

\paragraph{Distribution of Component Sizes.}
The distribution of the sizes for the various components is shown in 
Table \ref{tbl:distribution:component:negative:strong}.
This distribution presents the cardinalities of the strongly connected components
of the induced directed graph.

\begin{table}[ht]
\caption{Distribution of sizes for strongly connected components for the induced directed 
graph.
}\label{tbl:distribution:component:negative:strong}
\begin{center}
\begin{tabular}{|r||c|c|c|c|c|c|c|c|c|c|c|c|}\hline
\# of nodes      & $592$ & $12$ & $10$ & $8$ & $5$ & $4$ & $3$ &  $2$ &       $1$ \\\hline
\# of components &   $1$ &  $1$ &  $2$ & $1$ & $1$ & $4$ & $6$ & $59$ & $278,708$ \\\hline
\end{tabular}
\end{center}
\end{table}

Figure \ref{fig:strongly-connected:negative:maximal} presents the
maximal strongly connected component.

\begin{figure}[ht]
\begin{center}
\includegraphics[width=0.7\textwidth]{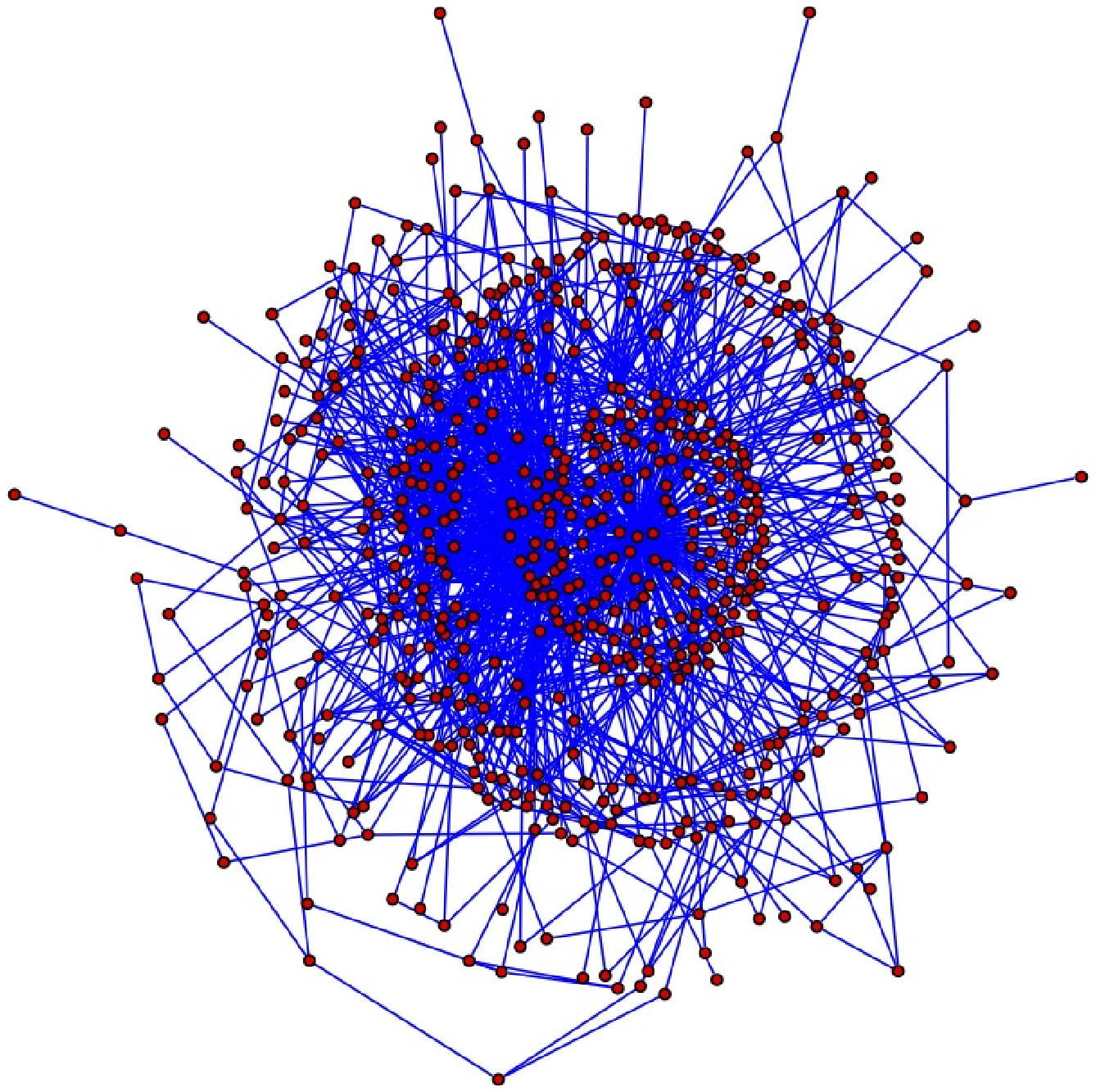}
\end{center}
\caption{The maximal strongly connected component of the graph induced by the assertions with 
negative polarity; see Table \ref{tbl:distribution:component:negative:strong}.
For simplicity we plot the induced undirected graph of that component.}\label{fig:strongly-connected:negative:maximal}
\end{figure}

\subsubsection{Big Strongly Connected Component}
The $592$ concepts found in the big directed component are
\dbtext{man} (7),
\dbtext{person} (9),
\dbtext{rock} (23),
\dbtext{beach} (24),
\dbtext{tree} (33),
\dbtext{work} (35),
\dbtext{actor} (47),
\dbtext{exercise} (61),
\dbtext{pant} (63),
\dbtext{love} (67),
\dbtext{library} (68),
\dbtext{bath} (70),
\dbtext{listen} (75),
\dbtext{wife} (76),
\dbtext{arm} (79),
\dbtext{human} (80),
\dbtext{run marathon} (101),
\dbtext{drink} (120),
\dbtext{examination} (121),
\dbtext{fun} (134),
\dbtext{it} (137),
\dbtext{paper} (149),
\dbtext{destroy} (150),
\dbtext{bed} (156),
\dbtext{dirty} (170),
\dbtext{dream} (172),
\dbtext{shower} (173),
\dbtext{child} (178),
\dbtext{smoke} (188),
\dbtext{chicken} (191),
\dbtext{blind} (233),
\dbtext{ball} (263),
\dbtext{mother} (301),
\dbtext{party} (307),
\dbtext{rest} (310),
\dbtext{remember} (325),
\dbtext{forget} (326),
\dbtext{housework} (343),
\dbtext{clean} (344),
\dbtext{street} (350),
\dbtext{watch tv} (351),
\dbtext{park} (365),
\dbtext{trouble} (366),
\dbtext{wood} (370),
\dbtext{play} (372),
\dbtext{bus} (377),
\dbtext{talk} (394),
\dbtext{go bed} (406),
\dbtext{sleep} (425),
\dbtext{eat} (432),
\dbtext{nothing} (466),
\dbtext{computer} (467),
\dbtext{rich} (469),
\dbtext{lover} (472),
\dbtext{buy} (475),
\dbtext{hunger} (478),
\dbtext{milk} (481),
\dbtext{sometimes} (526),
\dbtext{car} (529),
\dbtext{dog} (537),
\dbtext{music} (542),
\dbtext{film} (544),
\dbtext{zoo} (547),
\dbtext{dress} (562),
\dbtext{checkbook holder} (563),
\dbtext{bottle} (565),
\dbtext{better} (570),
\dbtext{live} (580),
\dbtext{one} (581),
\dbtext{aluminum} (590),
\dbtext{chair} (596),
\dbtext{skirt} (601),
\dbtext{drug} (610),
\dbtext{cat} (616),
\dbtext{gun} (635),
\dbtext{country} (640),
\dbtext{sell} (649),
\dbtext{house} (652),
\dbtext{fish} (655),
\dbtext{lake} (660),
\dbtext{baby} (678),
\dbtext{beauty} (702),
\dbtext{plant} (716),
\dbtext{silence} (730),
\dbtext{hide} (869),
\dbtext{girl} (876),
\dbtext{muscle} (891),
\dbtext{woman} (895),
\dbtext{animal} (902),
\dbtext{family} (915),
\dbtext{moon} (924),
\dbtext{robot} (935),
\dbtext{bird} (962),
\dbtext{death} (977),
\dbtext{play sport} (983),
\dbtext{sick} (991),
\dbtext{drive car} (1005),
\dbtext{bathroom} (1007),
\dbtext{city} (1013),
\dbtext{water} (1016),
\dbtext{truck} (1028),
\dbtext{desk} (1043),
\dbtext{office} (1044),
\dbtext{home} (1045),
\dbtext{bat} (1057),
\dbtext{penny} (1071),
\dbtext{couch} (1072),
\dbtext{build} (1104),
\dbtext{spoon} (1116),
\dbtext{travel} (1143),
\dbtext{eye} (1160),
\dbtext{see} (1161),
\dbtext{fail} (1167),
\dbtext{nose} (1171),
\dbtext{smell} (1172),
\dbtext{well} (1201),
\dbtext{pen} (1205),
\dbtext{mine} (1210),
\dbtext{die} (1227),
\dbtext{money} (1240),
\dbtext{bill} (1245),
\dbtext{snow} (1247),
\dbtext{leg} (1252),
\dbtext{triangle} (1257),
\dbtext{dead} (1279),
\dbtext{lady} (1281),
\dbtext{mouse} (1284),
\dbtext{cry} (1291),
\dbtext{television} (1298),
\dbtext{hate} (1342),
\dbtext{sea} (1347),
\dbtext{ocean} (1349),
\dbtext{sun} (1353),
\dbtext{sky} (1354),
\dbtext{food} (1359),
\dbtext{lie} (1395),
\dbtext{horse} (1412),
\dbtext{mug} (1422),
\dbtext{friend} (1429),
\dbtext{grocery store} (1447),
\dbtext{read} (1456),
\dbtext{hold} (1464),
\dbtext{kill} (1466),
\dbtext{break} (1476),
\dbtext{guy} (1479),
\dbtext{foot} (1485),
\dbtext{newspaper} (1506),
\dbtext{late} (1520),
\dbtext{hungry} (1533),
\dbtext{drive} (1545),
\dbtext{liquid} (1551),
\dbtext{oil} (1587),
\dbtext{plate} (1604),
\dbtext{smile} (1606),
\dbtext{cow} (1613),
\dbtext{earth} (1633),
\dbtext{dance} (1667),
\dbtext{potato} (1674),
\dbtext{fight} (1675),
\dbtext{curtain} (1694),
\dbtext{glass} (1776),
\dbtext{telephone} (1790),
\dbtext{pain} (1813),
\dbtext{audience} (1816),
\dbtext{soul} (1835),
\dbtext{drop} (1846),
\dbtext{bone} (1852),
\dbtext{meat} (1853),
\dbtext{rain} (1856),
\dbtext{body} (1861),
\dbtext{write} (1893),
\dbtext{pencil} (1953),
\dbtext{book} (2033),
\dbtext{black} (2063),
\dbtext{fart} (2079),
\dbtext{honest} (2087),
\dbtext{profit} (2167),
\dbtext{complete} (2201),
\dbtext{close} (2222),
\dbtext{bad} (2226),
\dbtext{heaven} (2241),
\dbtext{show} (2243),
\dbtext{trash} (2260),
\dbtext{can} (2261),
\dbtext{gold} (2266),
\dbtext{wind} (2284),
\dbtext{hand} (2300),
\dbtext{debt} (2306),
\dbtext{stop} (2358),
\dbtext{road} (2368),
\dbtext{brother} (2383),
\dbtext{boat} (2389),
\dbtext{lose} (2426),
\dbtext{war} (2438),
\dbtext{flower} (2459),
\dbtext{wallet} (2466),
\dbtext{suitcase} (2479),
\dbtext{time} (2494),
\dbtext{problem} (2500),
\dbtext{hell} (2510),
\dbtext{small} (2536),
\dbtext{bicycle} (2554),
\dbtext{need} (2557),
\dbtext{enemy} (2558),
\dbtext{continent} (2580),
\dbtext{iron} (2587),
\dbtext{cookie} (2595),
\dbtext{color} (2611),
\dbtext{white} (2612),
\dbtext{red} (2614),
\dbtext{yellow} (2616),
\dbtext{colour} (2626),
\dbtext{stone} (2631),
\dbtext{vegetable} (2636),
\dbtext{green} (2637),
\dbtext{life} (2638),
\dbtext{murder} (2663),
\dbtext{wrong} (2664),
\dbtext{good} (2666),
\dbtext{evil} (2692),
\dbtext{large} (2771),
\dbtext{shoe} (2790),
\dbtext{go} (2801),
\dbtext{sex} (2825),
\dbtext{wait} (2858),
\dbtext{steak} (2878),
\dbtext{fire} (2895),
\dbtext{exist} (2907),
\dbtext{government} (2932),
\dbtext{beer} (3052),
\dbtext{none} (3387),
\dbtext{carpet} (3450),
\dbtext{bowl} (3463),
\dbtext{freedom} (3492),
\dbtext{born} (3501),
\dbtext{leave} (3571),
\dbtext{coin} (3573),
\dbtext{fruit} (3590),
\dbtext{laugh} (3635),
\dbtext{sister} (3656),
\dbtext{laundry} (3665),
\dbtext{fork} (3671),
\dbtext{planet} (3683),
\dbtext{shirt} (3686),
\dbtext{begin} (3695),
\dbtext{steel} (3907),
\dbtext{sidewalk} (3962),
\dbtext{avenue} (4000),
\dbtext{theatre} (4095),
\dbtext{cup} (4116),
\dbtext{square} (4138),
\dbtext{busy} (4163),
\dbtext{full} (4189),
\dbtext{pleasure} (4231),
\dbtext{god} (4277),
\dbtext{care} (4323),
\dbtext{star} (4324),
\dbtext{watch} (4406),
\dbtext{mind} (4432),
\dbtext{space} (4435),
\dbtext{wealth} (4521),
\dbtext{this} (4539),
\dbtext{place} (4570),
\dbtext{apple} (4596),
\dbtext{pear} (4624),
\dbtext{win} (4676),
\dbtext{mail} (4691),
\dbtext{direct} (4753),
\dbtext{doctor} (4760),
\dbtext{theater} (4770),
\dbtext{river} (4784),
\dbtext{blue} (4808),
\dbtext{charge} (4811),
\dbtext{cheese} (4844),
\dbtext{whale} (4849),
\dbtext{mammal} (4850),
\dbtext{question} (4898),
\dbtext{crap} (4899),
\dbtext{lot} (4905),
\dbtext{coal} (5090),
\dbtext{touch} (5106),
\dbtext{noise} (5363),
\dbtext{husband} (5415),
\dbtext{plastic} (5505),
\dbtext{bug} (5563),
\dbtext{above} (5572),
\dbtext{unknown} (5613),
\dbtext{matter} (5619),
\dbtext{disease} (5645),
\dbtext{table} (5665),
\dbtext{peace} (5670),
\dbtext{key case} (5678),
\dbtext{often} (5700),
\dbtext{sing} (5711),
\dbtext{sand} (5768),
\dbtext{billfold} (5827),
\dbtext{bottom} (5887),
\dbtext{religion} (5915),
\dbtext{long hair} (5916),
\dbtext{closet} (5967),
\dbtext{boy} (5976),
\dbtext{like} (5989),
\dbtext{record} (6029),
\dbtext{find} (6040),
\dbtext{floor} (6062),
\dbtext{right} (6079),
\dbtext{old} (6092),
\dbtext{safety} (6244),
\dbtext{cut} (6250),
\dbtext{honesty} (6288),
\dbtext{slow} (6291),
\dbtext{frustrate} (6309),
\dbtext{adult} (6329),
\dbtext{conflict} (6331),
\dbtext{here} (6352),
\dbtext{bite} (6368),
\dbtext{science} (6395),
\dbtext{air} (6408),
\dbtext{lime} (6416),
\dbtext{banana} (6422),
\dbtext{metal} (6491),
\dbtext{do} (6503),
\dbtext{open} (6539),
\dbtext{quiet} (6583),
\dbtext{big} (6604),
\dbtext{present} (6681),
\dbtext{roll} (6734),
\dbtext{mess} (6818),
\dbtext{mineral} (6835),
\dbtext{clock} (6860),
\dbtext{black hole} (6876),
\dbtext{chaos} (6892),
\dbtext{distance} (6929),
\dbtext{many} (6989),
\dbtext{competitive activity} (7019),
\dbtext{safe} (7045),
\dbtext{still} (7048),
\dbtext{violence} (7055),
\dbtext{round} (7057),
\dbtext{computer language} (7112),
\dbtext{mercury} (7120),
\dbtext{art} (7424),
\dbtext{rust} (7512),
\dbtext{top} (7514),
\dbtext{wine} (7522),
\dbtext{jar} (7524),
\dbtext{crowd} (7763),
\dbtext{draw} (7764),
\dbtext{much} (7917),
\dbtext{thing} (7936),
\dbtext{energy} (7982),
\dbtext{land} (8060),
\dbtext{few} (8145),
\dbtext{musician} (8244),
\dbtext{little} (8268),
\dbtext{change} (8313),
\dbtext{ear} (8314),
\dbtext{bread} (8404),
\dbtext{dna} (8405),
\dbtext{pick} (8494),
\dbtext{gasoline} (8502),
\dbtext{petrol} (8691),
\dbtext{move} (8737),
\dbtext{try} (8794),
\dbtext{decide} (8824),
\dbtext{tin} (8891),
\dbtext{finish} (8996),
\dbtext{fear} (9006),
\dbtext{poverty} (9116),
\dbtext{island} (9131),
\dbtext{shade} (9151),
\dbtext{fly} (9215),
\dbtext{hear} (9269),
\dbtext{egg} (9339),
\dbtext{penis} (9458),
\dbtext{vagina} (9464),
\dbtext{two} (9549),
\dbtext{dad} (9672),
\dbtext{health} (9745),
\dbtext{pass} (9934),
\dbtext{wash} (10170),
\dbtext{sock} (10193),
\dbtext{head} (10228),
\dbtext{work hard} (10313),
\dbtext{his} (10419),
\dbtext{fill} (10468),
\dbtext{great} (10478),
\dbtext{end} (10507),
\dbtext{know} (13183),
\dbtext{program language} (13345),
\dbtext{daughter} (13446),
\dbtext{supermarket} (13550),
\dbtext{danger} (13607),
\dbtext{servant} (13683),
\dbtext{silver} (13722),
\dbtext{pie} (13747),
\dbtext{machine} (13790),
\dbtext{gas} (13908),
\dbtext{taste} (14093),
\dbtext{ant} (14190),
\dbtext{fix} (14209),
\dbtext{lemon} (14212),
\dbtext{gerbil} (14223),
\dbtext{dollar} (14251),
\dbtext{want} (14319),
\dbtext{galaxy} (14379),
\dbtext{circle} (14472),
\dbtext{cake} (14522),
\dbtext{blood} (14713),
\dbtext{law} (14805),
\dbtext{copy} (14847),
\dbtext{chick} (14872),
\dbtext{illusion} (14991),
\dbtext{cent} (14994),
\dbtext{orange} (15004),
\dbtext{large bird} (15149),
\dbtext{dirt} (15359),
\dbtext{son} (15379),
\dbtext{fast} (15507),
\dbtext{point} (15518),
\dbtext{choose} (15533),
\dbtext{fact} (15578),
\dbtext{be} (16974),
\dbtext{software} (17383),
\dbtext{common} (17473),
\dbtext{brain} (17555),
\dbtext{liar} (17830),
\dbtext{stay} (18183),
\dbtext{sickness} (18244),
\dbtext{performance} (18289),
\dbtext{motion} (18365),
\dbtext{whole} (18374),
\dbtext{opinion} (18525),
\dbtext{out} (18546),
\dbtext{in} (18553),
\dbtext{return} (18569),
\dbtext{ill} (18575),
\dbtext{victory} (18635),
\dbtext{fantasy} (18637),
\dbtext{propose woman} (18678),
\dbtext{ignorance} (18746),
\dbtext{ride} (18753),
\dbtext{free} (19126),
\dbtext{past} (19235),
\dbtext{new} (19512),
\dbtext{part} (19708),
\dbtext{course} (19871),
\dbtext{rat} (19911),
\dbtext{own} (19972),
\dbtext{dog die} (20317),
\dbtext{vegetarian} (20339),
\dbtext{owner} (20525),
\dbtext{over} (20622),
\dbtext{real} (20645),
\dbtext{same} (20650),
\dbtext{best} (20709),
\dbtext{necessary} (20908),
\dbtext{box office} (20927),
\dbtext{poor} (20993),
\dbtext{car key} (21233),
\dbtext{angel} (21240),
\dbtext{stage} (21403),
\dbtext{order} (21418),
\dbtext{print} (21683),
\dbtext{microsoft} (21796),
\dbtext{rush} (21894),
\dbtext{empty} (22345),
\dbtext{freeway} (22365),
\dbtext{go break} (22388),
\dbtext{artifact} (22487),
\dbtext{chore} (22621),
\dbtext{clear} (22671),
\dbtext{trip} (22700),
\dbtext{all} (22948),
\dbtext{cash register} (23016),
\dbtext{worse} (23274),
\dbtext{deaf} (23417),
\dbtext{truth} (23426),
\dbtext{conscious} (23506),
\dbtext{compassion} (23996),
\dbtext{food can} (24253),
\dbtext{elevator} (24427),
\dbtext{reality} (24722),
\dbtext{future} (24821),
\dbtext{loss} (24845),
\dbtext{orchestra pit} (24852),
\dbtext{sunshine} (25192),
\dbtext{answer} (25710),
\dbtext{solution} (25749),
\dbtext{rend} (26018),
\dbtext{master} (26090),
\dbtext{lord} (26283),
\dbtext{gentleman} (26487),
\dbtext{below} (27409),
\dbtext{slave} (27415),
\dbtext{brass} (27632),
\dbtext{come} (28590),
\dbtext{imaginary} (28877),
\dbtext{urban} (29003),
\dbtext{nurse} (29051),
\dbtext{away} (29340),
\dbtext{flat tire} (29840),
\dbtext{pest} (29938),
\dbtext{reply} (30251),
\dbtext{neglect} (30447),
\dbtext{lift} (30476),
\dbtext{run treadmill} (30496),
\dbtext{paste} (30733),
\dbtext{inch} (31249),
\dbtext{seek} (31416),
\dbtext{ask} (31437),
\dbtext{urine} (31765),
\dbtext{hatred} (32372),
\dbtext{itch} (33681),
\dbtext{some} (34413),
\dbtext{half} (34484),
\dbtext{checkbook cover} (34526),
\dbtext{ally} (34528),
\dbtext{bob} (34599),
\dbtext{bronze} (34633),
\dbtext{defeat} (34651),
\dbtext{computer virus} (34745),
\dbtext{enter} (36183),
\dbtext{shout} (36617),
\dbtext{park bench} (37143),
\dbtext{fine} (37229),
\dbtext{far} (37745),
\dbtext{miss} (38484),
\dbtext{vague} (38678),
\dbtext{balcony seat} (38825),
\dbtext{youth} (39134),
\dbtext{table cloth} (39246),
\dbtext{lint} (39879),
\dbtext{continue} (43163),
\dbtext{always} (43553),
\dbtext{early} (44789),
\dbtext{start} (44963),
\dbtext{kleenex} (47422),
\dbtext{okay} (47521),
\dbtext{now} (47894),
\dbtext{movie screen} (48200),
\dbtext{real duck} (48788),
\dbtext{monkey wrench} (51349),
\dbtext{recent} (52116),
\dbtext{ancient} (53725),
\dbtext{fondue} (55652),
\dbtext{snub} (56060),
\dbtext{myth} (57756),
\dbtext{agent} (58122),
\dbtext{enough} (60838),
\dbtext{apathy} (65195),
\dbtext{frown} (66280),
\dbtext{sandal} (70731),
\dbtext{bowie} (71734),
\dbtext{insure} (73501),
\dbtext{gain} (73685),
\dbtext{plasma} (75809),
\dbtext{any} (76623),
\dbtext{zero} (79459),
\dbtext{beast} (79780),
\dbtext{cost} (81860),
\dbtext{arrive} (88674),
\dbtext{gray} (93729),
\dbtext{celibate} (96998),
\dbtext{integer} (100946),
\dbtext{graze} (102383),
\dbtext{lease} (102411),
\dbtext{never} (126958),
\dbtext{modern} (131226),
\dbtext{idle} (138412),
\dbtext{illiteracy} (140633),
\dbtext{high octane} (144079),
\dbtext{occasional} (155305),
\dbtext{mistress} (174612),
\dbtext{entire} (177172),
\dbtext{ground} (184976),
\dbtext{rural} (185019),
\dbtext{under} (193674),
\dbtext{transportation device} (200905),
\dbtext{barack obama} (201863),
\dbtext{speedo} (203600),
\dbtext{fidelity} (203658),
\dbtext{norway rat} (203664),
\dbtext{complete thesis} (203671),
\dbtext{pour hot coffee mug} (203692),
\dbtext{low blood pressure} (203696),
\dbtext{cost little money} (203700),
\dbtext{broken} (311852),
\dbtext{written programmer} (328244),
\dbtext{foe} (332239),
\dbtext{mister} (332244), and
\dbtext{obscure} (344371).

Regarding the big strongly connected component with the $592$ nodes, it has $1,849$ edges
(self-loops were omitted from the enumeration).
Hence the average degree is about $6.24662$ after self-loops have been discarded.
Regarding the induced undirected graph that occurs after restricting ourselves
in these $592$ nodes (again, self-loops are omitted), the number of edges
is $1,566$. In other words, the average degree in this case is about
$5.29054$. The transitivity and the clustering coefficient of the big component are
presented in Table \ref{tbl:transitivity:negative:BigDC}.

\begin{table}[ht]
\caption{Transitivity and clustering coefficient for the big directed component of \conceptnet.
The first value (\textsc{nan}) for the clustering coefficient gives the result of the calculation
when vertices with less than two neighbors are left out from the calculation,
while the second value (\textsc{zero}) gives the result of the calculation when 
vertices with less than two neighbors are considered as having zero transitivity.
Note that all values are the same both for directed as well as undirected graphs.}\label{tbl:transitivity:negative:BigDC}
\begin{center}
\begin{tabular}{|r|l|}\hline
Transitivity                           & $0.000351298700593188$ \\\hline
Clustering Coefficient (\textsc{nan})  & $0.098300551193575281$ \\\hline
Clustering Coefficient (\textsc{zero}) & $0.000851478314399245$ \\\hline
\end{tabular}
\end{center}
\end{table}

For information about shortest paths in this component please see Chapter \ref{chapter:shortest-paths}.

\bigskip

Figures \ref{fig:strongly-connected:negative:1}, \ref{fig:strongly-connected:negative:2},
\ref{fig:strongly-connected:negative:3}, and \ref{fig:strongly-connected:negative:4} present
the strongly connected components of sizes $3$-$12$.

\begin{figure}[p]
\begin{center}
\begin{subfigure}[b]{0.48\textwidth}
\begin{center}
\includegraphics[width=\textwidth]{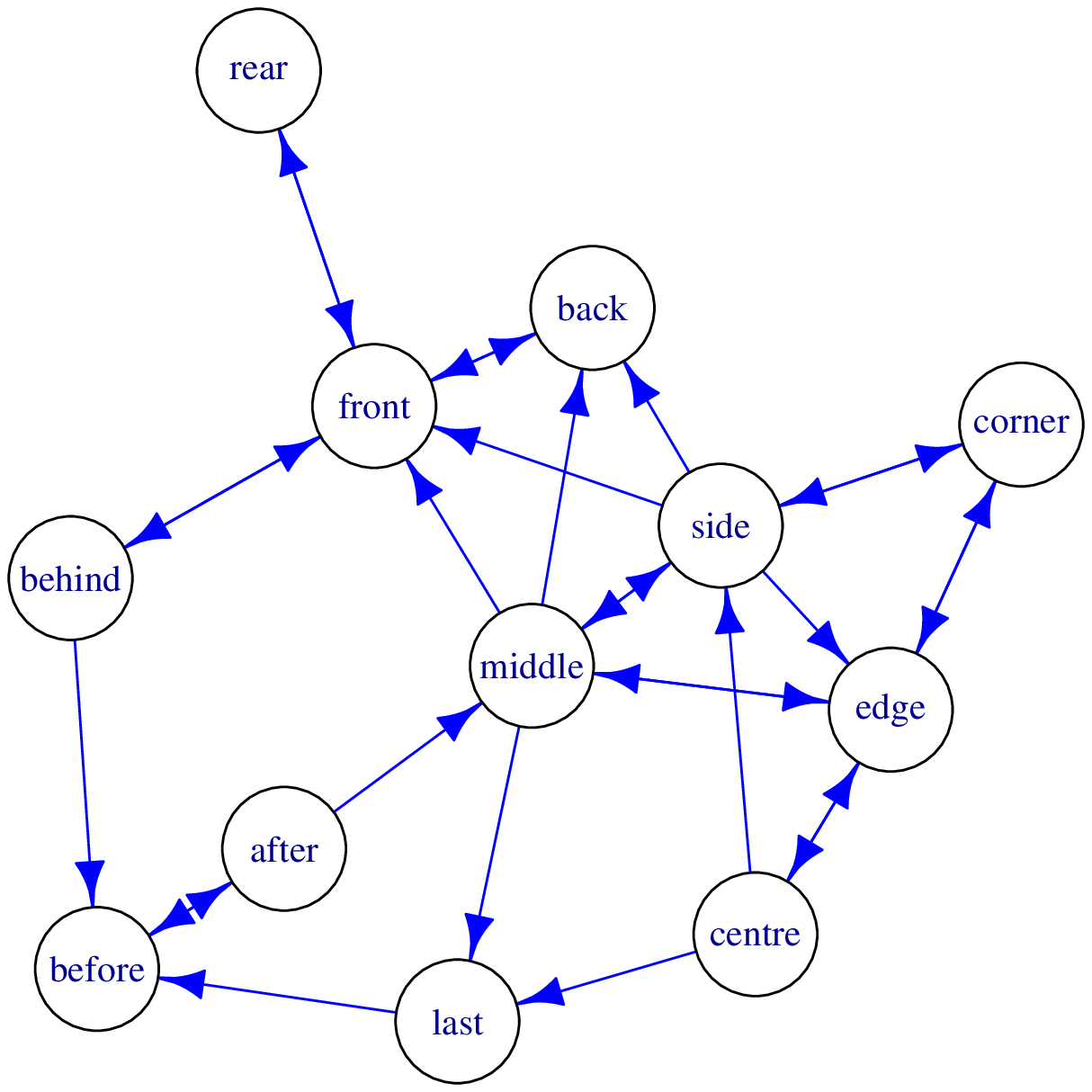}
\end{center}
\caption{$12$ nodes, $29$ edges.}\label{fig:strongly-connected:negative:12}
\end{subfigure}
\hspace{0.01\textwidth}
\begin{subfigure}[b]{0.48\textwidth}
\begin{center}
\includegraphics[width=\textwidth]{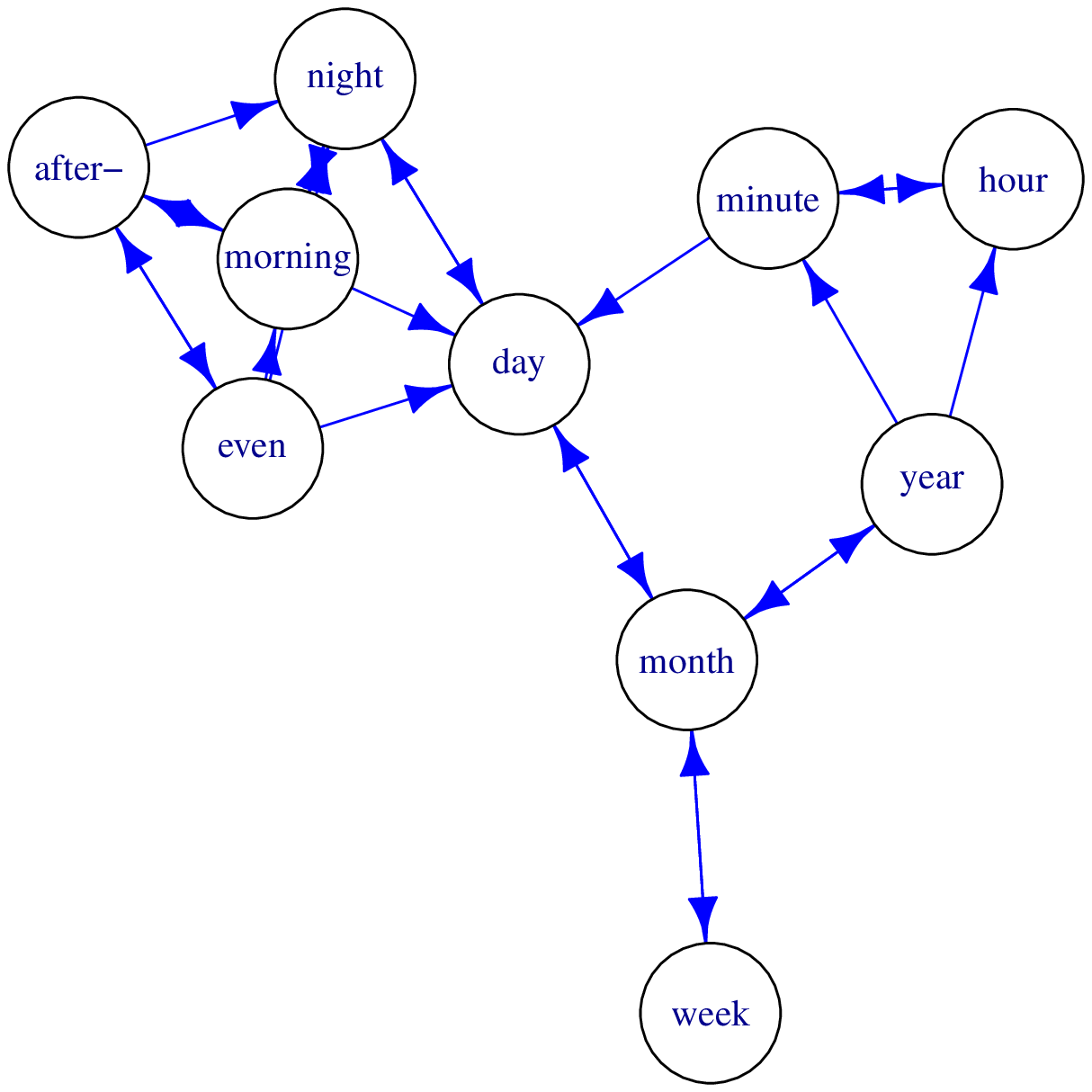}
\end{center}
\caption{$10$ nodes, $24$ edges.}\label{fig:strongly-connected:negative:10:a}
\end{subfigure}

%
%
\begin{subfigure}[b]{0.48\textwidth}
\begin{center}
\includegraphics[width=\textwidth]{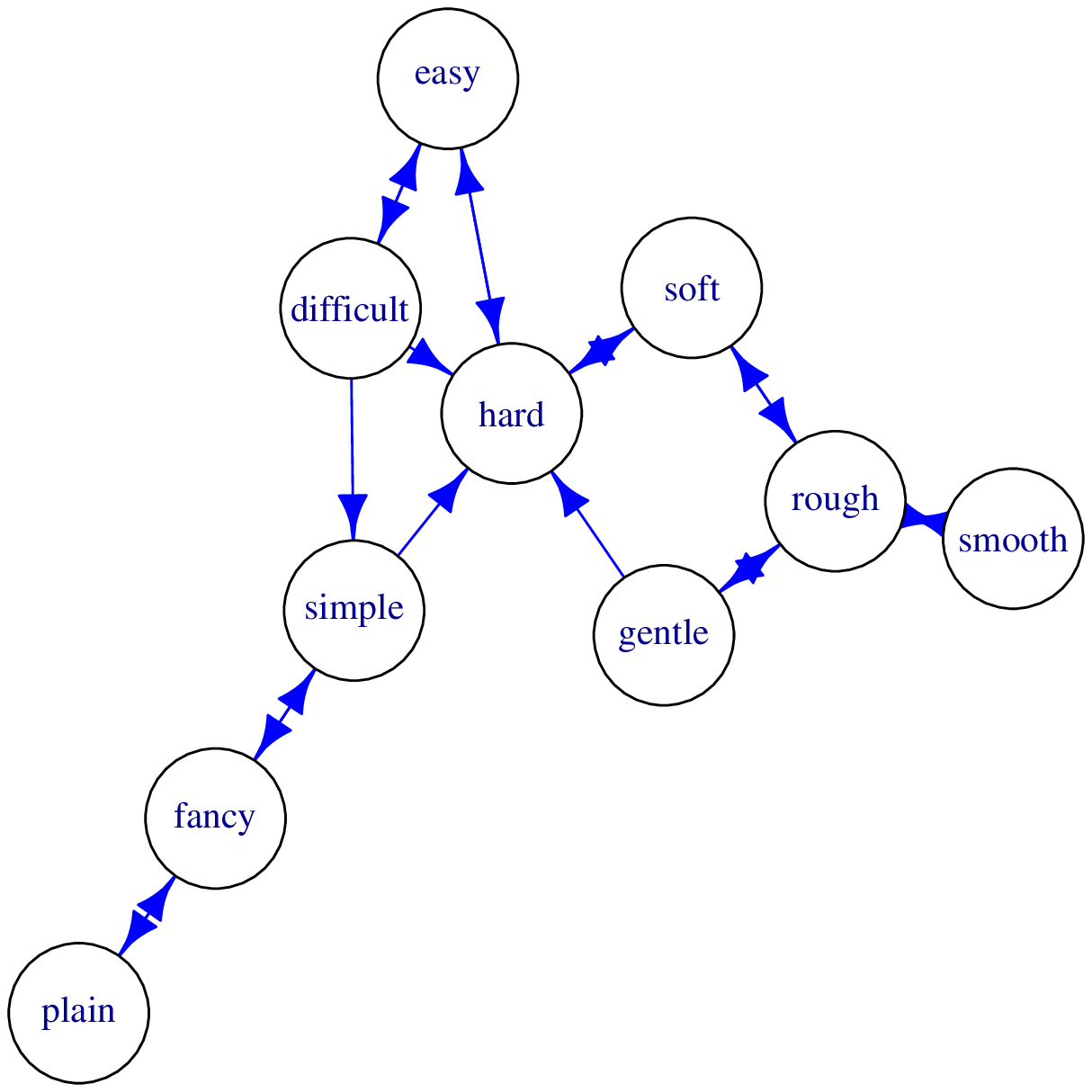}
\end{center}
\caption{$10$ nodes, $20$ edges.}\label{fig:strongly-connected:negative:10:b}
\end{subfigure}
\hspace{0.01\textwidth}
\begin{subfigure}[b]{0.48\textwidth}
\begin{center}
\includegraphics[width=\textwidth]{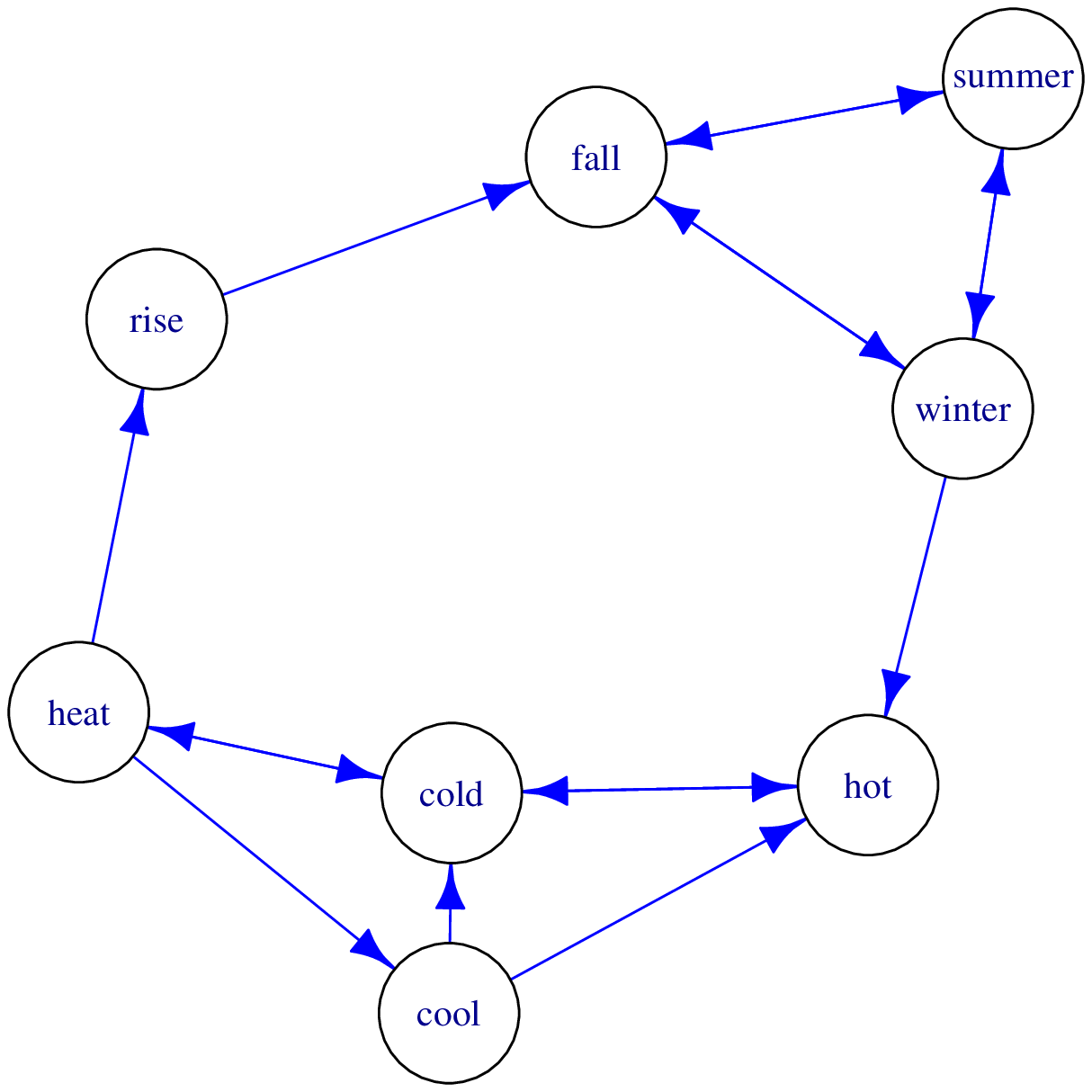}
\end{center}
\caption{$8$ nodes, $16$ edges.}\label{fig:strongly-connected:negative:8}
\end{subfigure}
\end{center}
\caption{The strongly connected components of size $8$-$12$ induced by assertions with negative polarity; 
see Table \ref{tbl:distribution:component:negative:strong}.}\label{fig:strongly-connected:negative:1}
\end{figure}

\begin{figure}[p]
\begin{center}
\begin{subfigure}[b]{0.48\textwidth}
\begin{center}
\includegraphics[width=\textwidth]{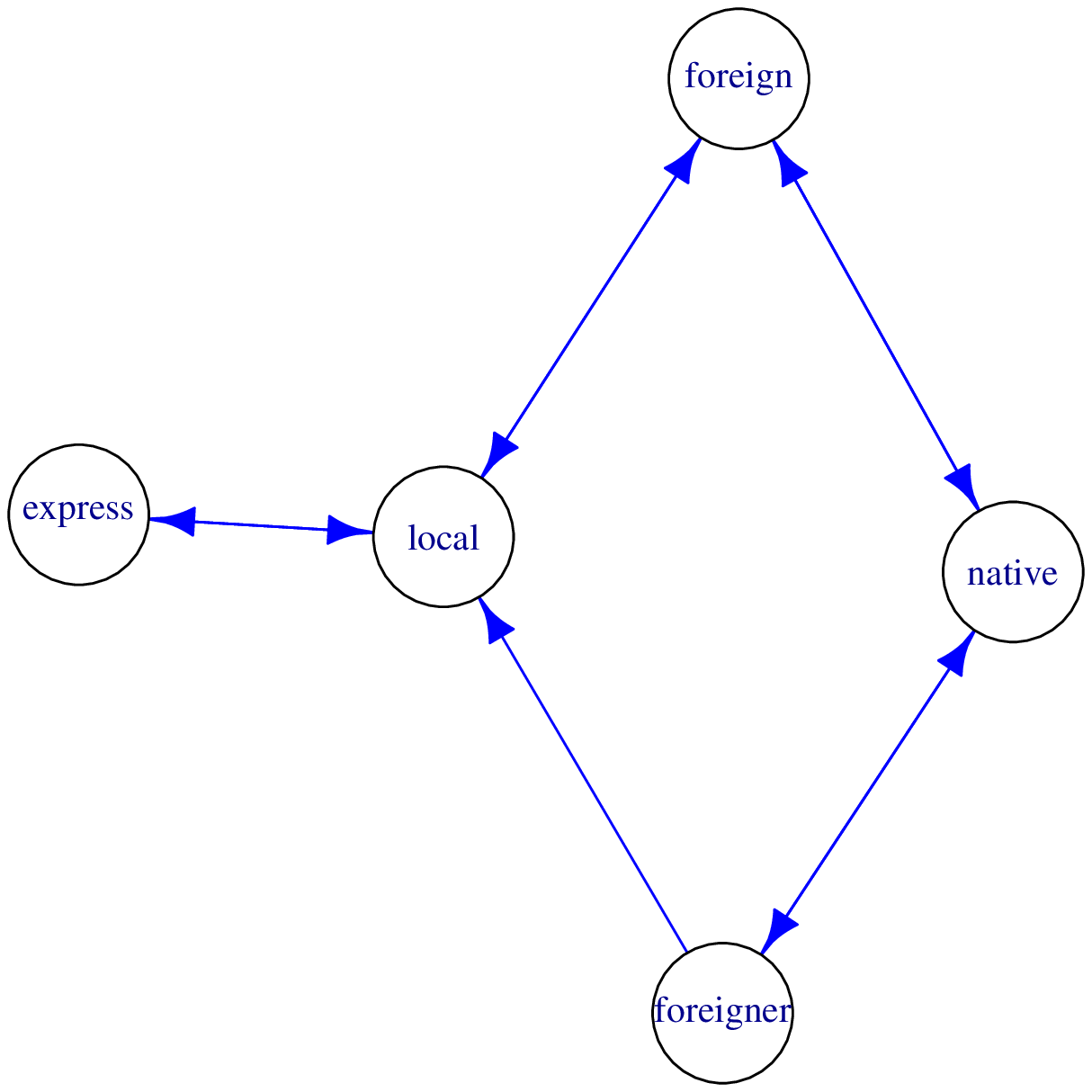}
\end{center}
\caption{$5$ nodes, $9$ edges.}\label{fig:strongly-connected:negative:5}
\end{subfigure}
\hspace{0.01\textwidth}
\begin{subfigure}[b]{0.48\textwidth}
\begin{center}
\includegraphics[width=\textwidth]{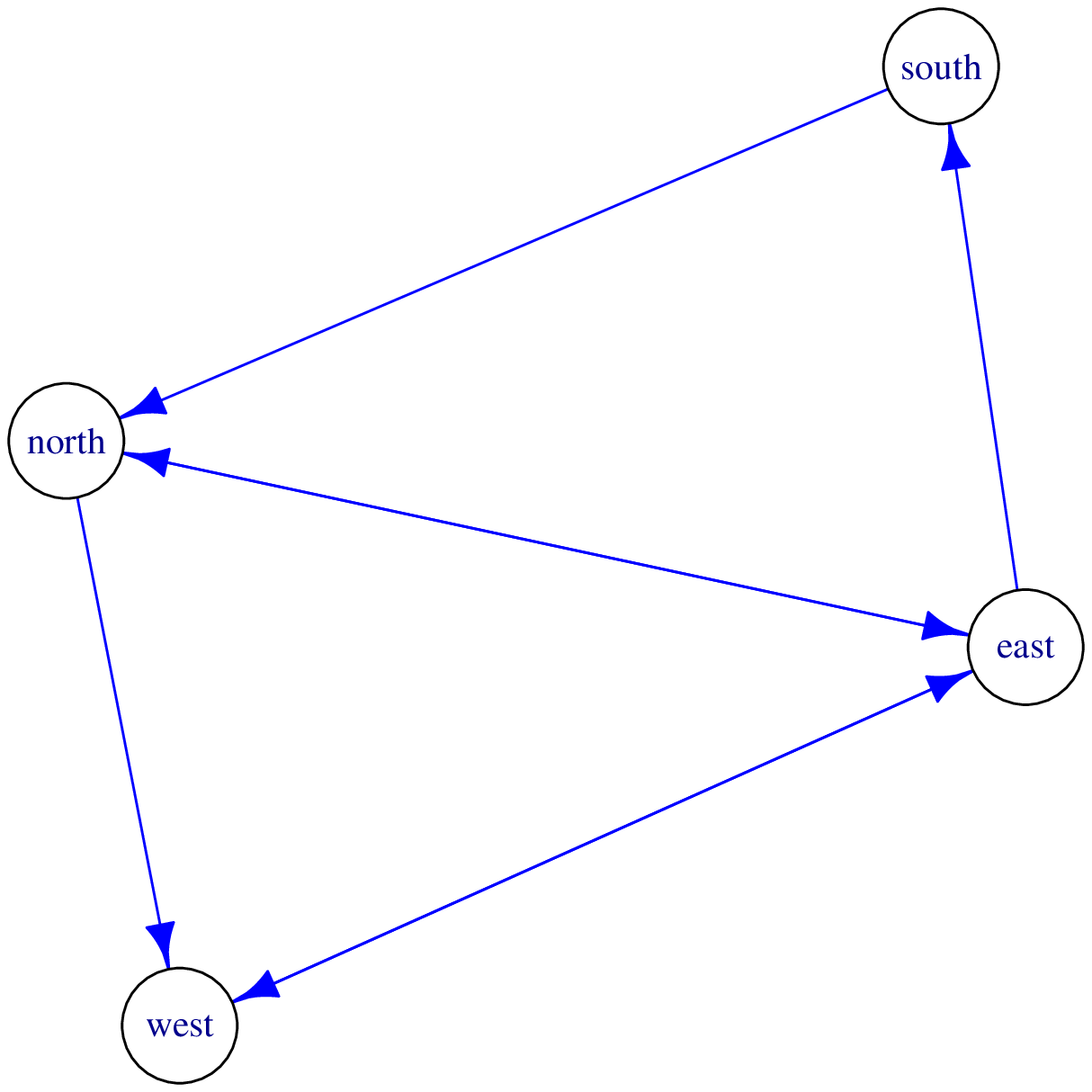}
\end{center}
\caption{$4$ nodes, $7$ edges.}\label{fig:strongly-connected:negative:4:a}
\end{subfigure}

\begin{subfigure}[b]{0.48\textwidth}
\begin{center}
\includegraphics[width=\textwidth]{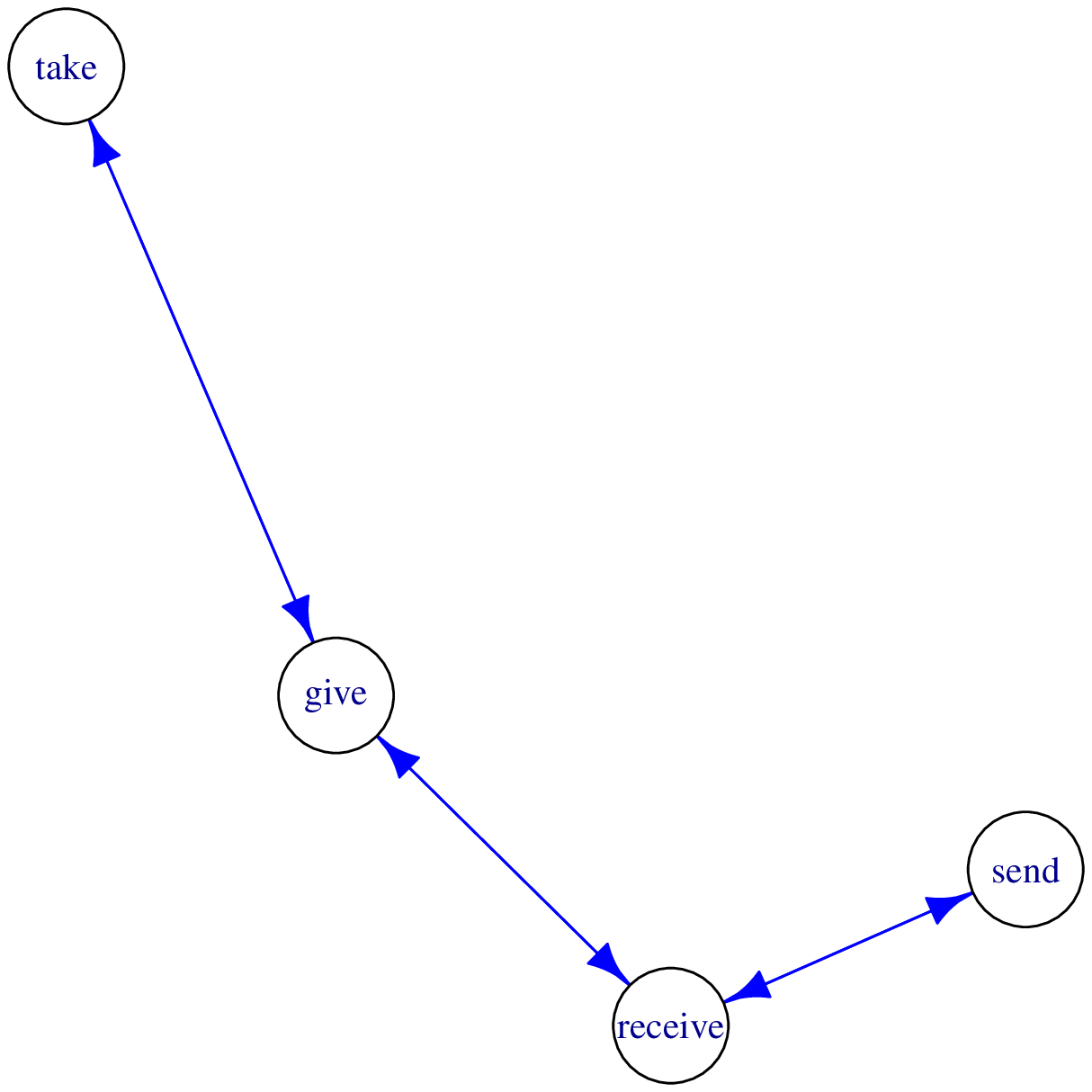}
\end{center}
\caption{$4$ nodes, $6$ edges.}\label{fig:strongly-connected:negative:4:b}
\end{subfigure}
\hspace{0.01\textwidth}
\begin{subfigure}[b]{0.48\textwidth}
\begin{center}
\includegraphics[width=\textwidth]{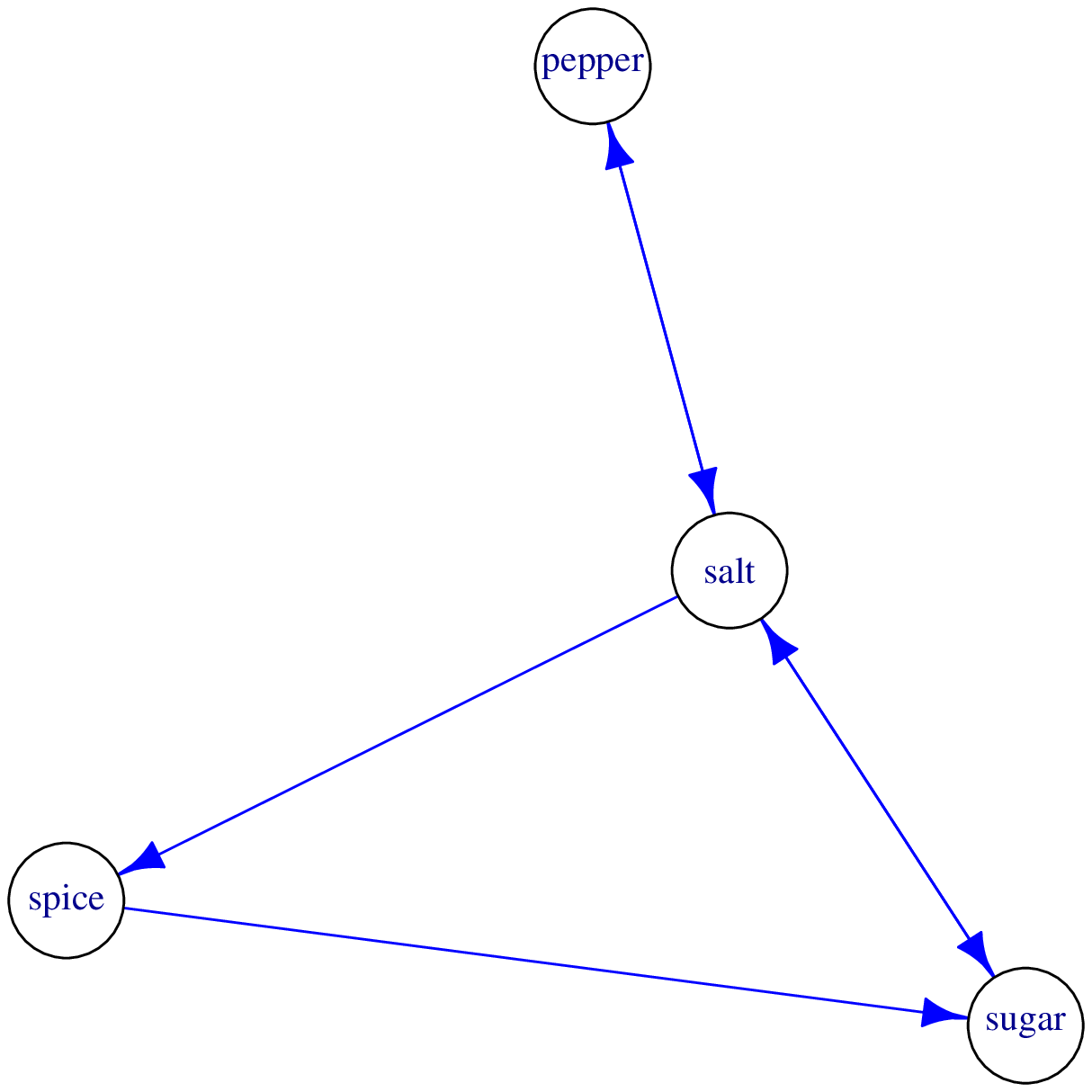}
\end{center}
\caption{$4$ nodes, $6$ edges.}\label{fig:strongly-connected:negative:4:c}
\end{subfigure}
\end{center}
\caption{Strongly connected components with sizes $4$-$5$ induced by assertions with negative polarity; 
see Table \ref{tbl:distribution:component:negative:strong}.}\label{fig:strongly-connected:negative:2}
\end{figure}

\begin{figure}[p]
\begin{center}
\begin{subfigure}[b]{0.48\textwidth}
\begin{center}
\includegraphics[width=\textwidth]{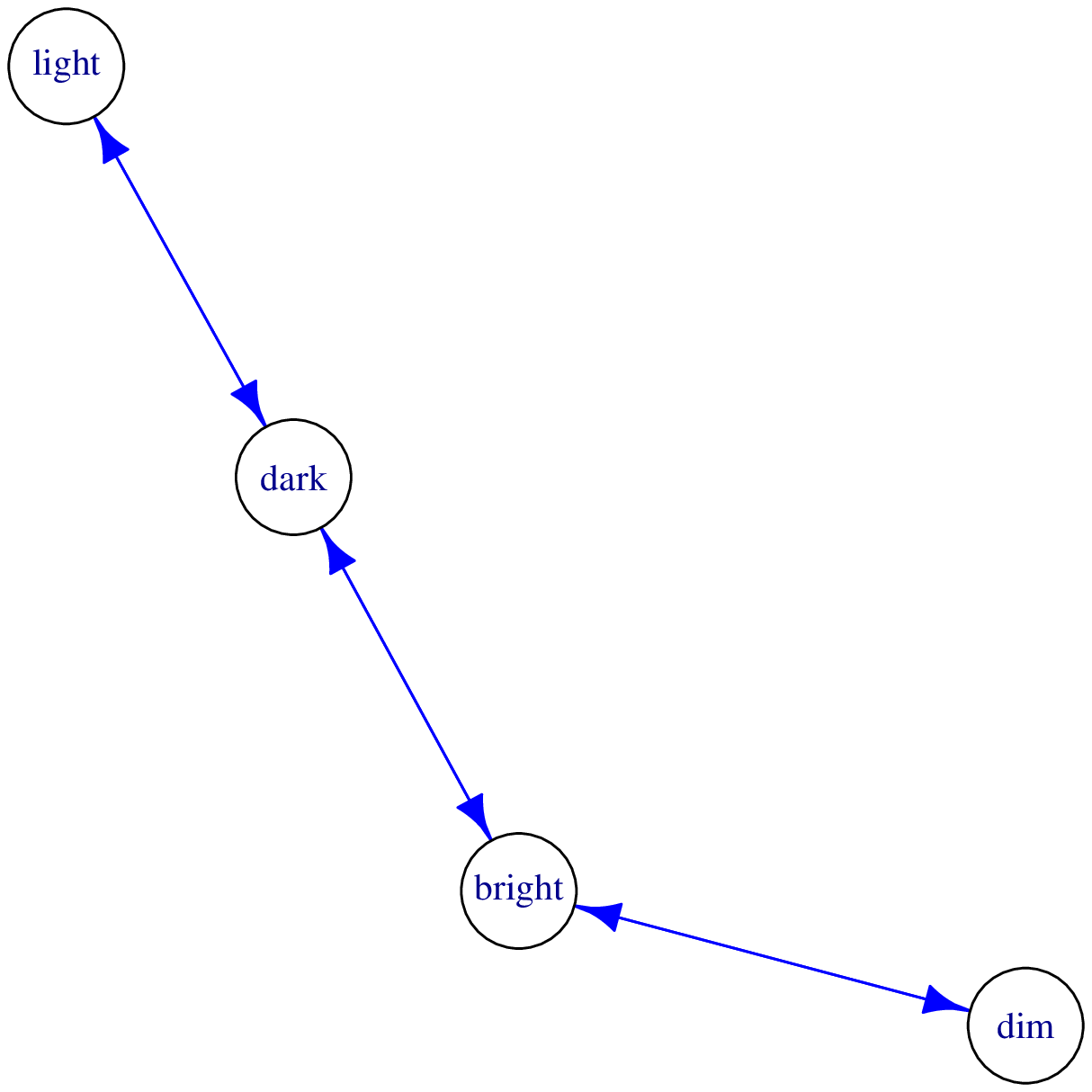}
\end{center}
\caption{$4$ nodes, $6$ edges.}\label{fig:strongly-connected:negative:4:d}
\end{subfigure}
\hspace{0.01\textwidth}
\begin{subfigure}[b]{0.48\textwidth}
\begin{center}
\includegraphics[width=\textwidth]{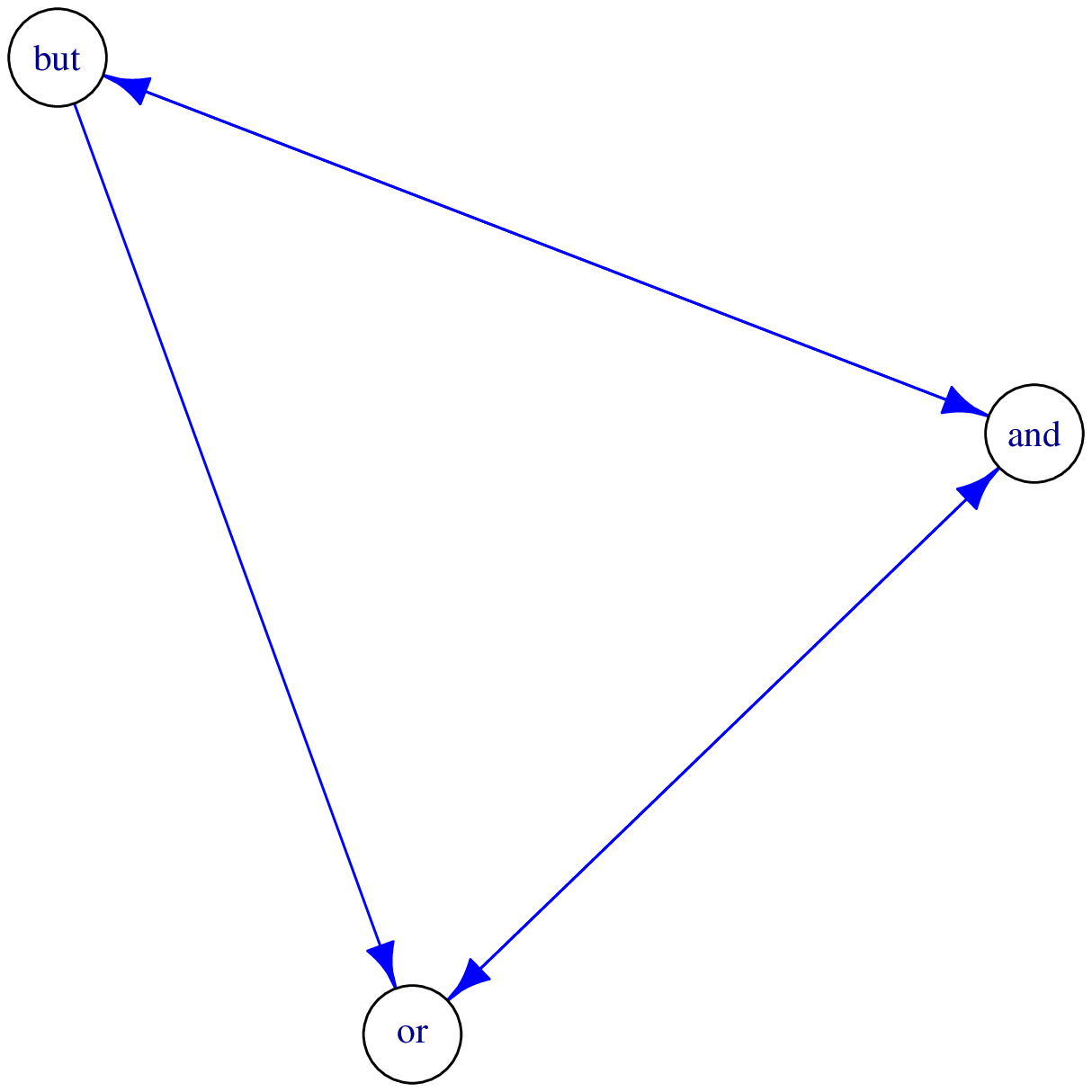}
\end{center}
\caption{$3$ nodes, $5$ edges.}\label{fig:strongly-connected:negative:3:a}
\end{subfigure}

\begin{subfigure}[b]{0.48\textwidth}
\begin{center}
\includegraphics[width=\textwidth]{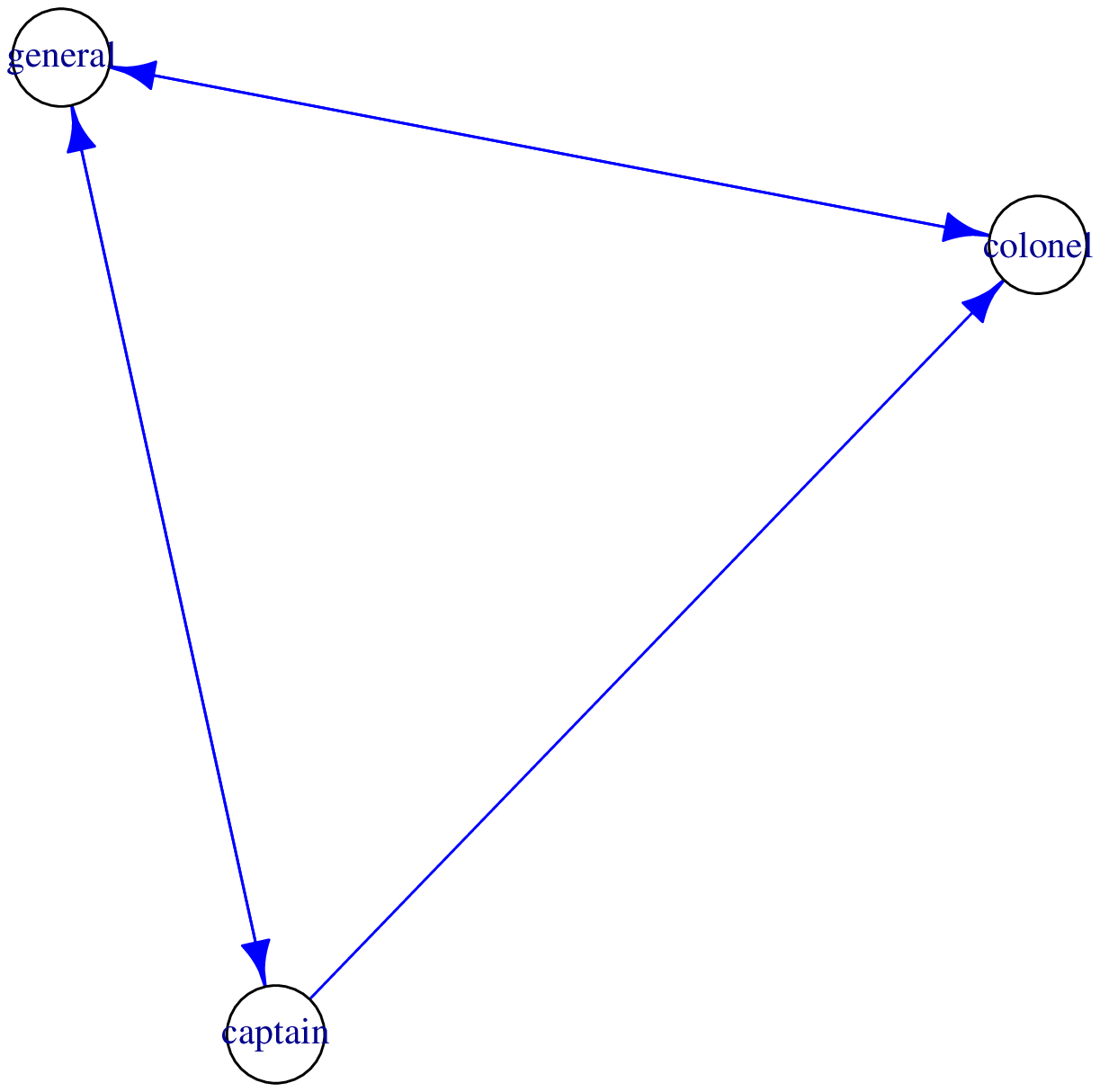}
\end{center}
\caption{$3$ nodes, $5$ edges.}\label{fig:strongly-connected:negative:3:b}
\end{subfigure}
\hspace{0.01\textwidth}
\begin{subfigure}[b]{0.48\textwidth}
\begin{center}
\includegraphics[width=\textwidth]{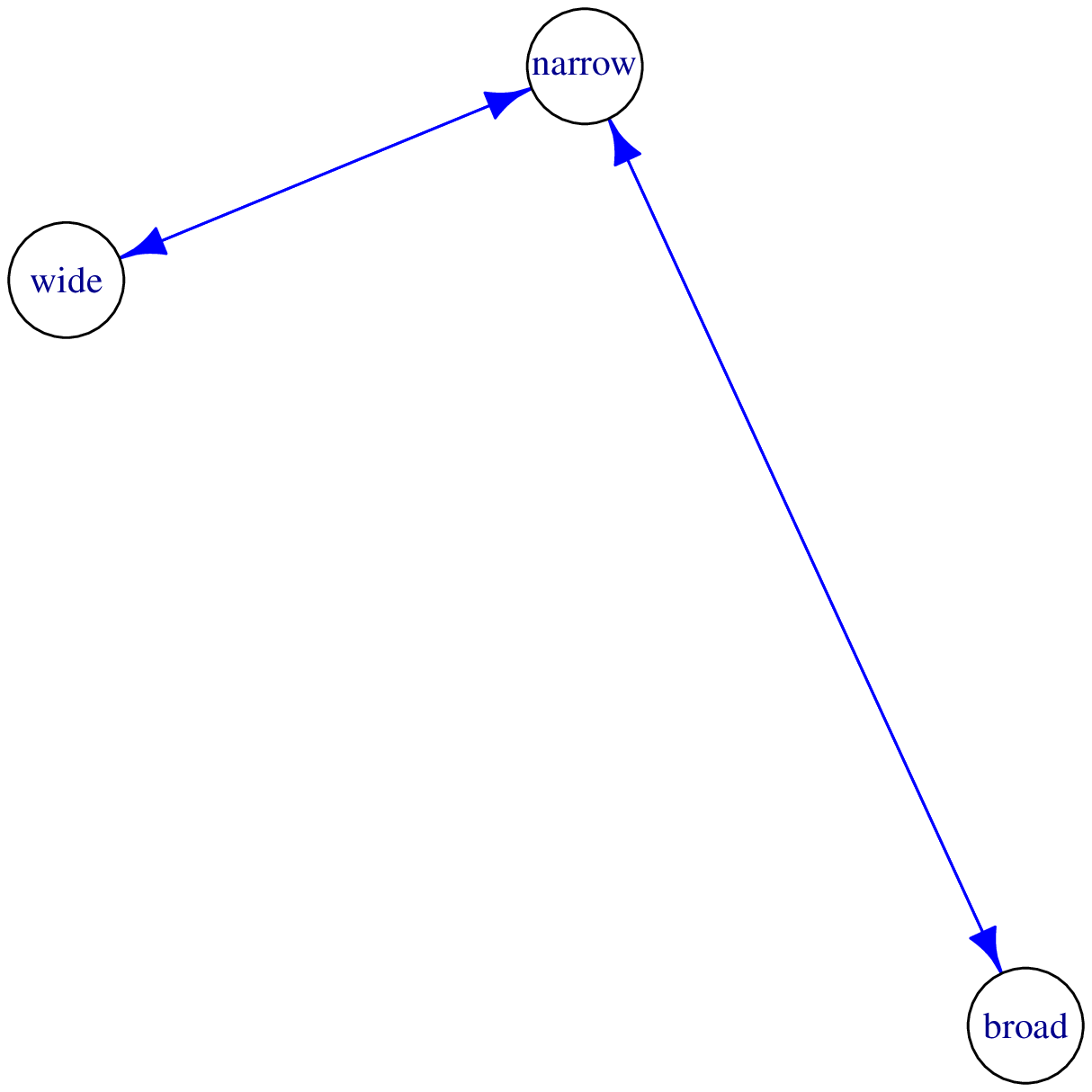}
\end{center}
\caption{$3$ nodes, $4$ edges.}\label{fig:strongly-connected:negative:3:c}
\end{subfigure}
\end{center}
\caption{The strongly connected components with sizes $3$-$4$ induced by assertions with negative polarity; 
see Table \ref{tbl:distribution:component:negative:strong}.}\label{fig:strongly-connected:negative:3}
\end{figure}

\begin{figure}[p]
\begin{center}
\begin{subfigure}[b]{0.48\textwidth}
\begin{center}
\includegraphics[width=\textwidth]{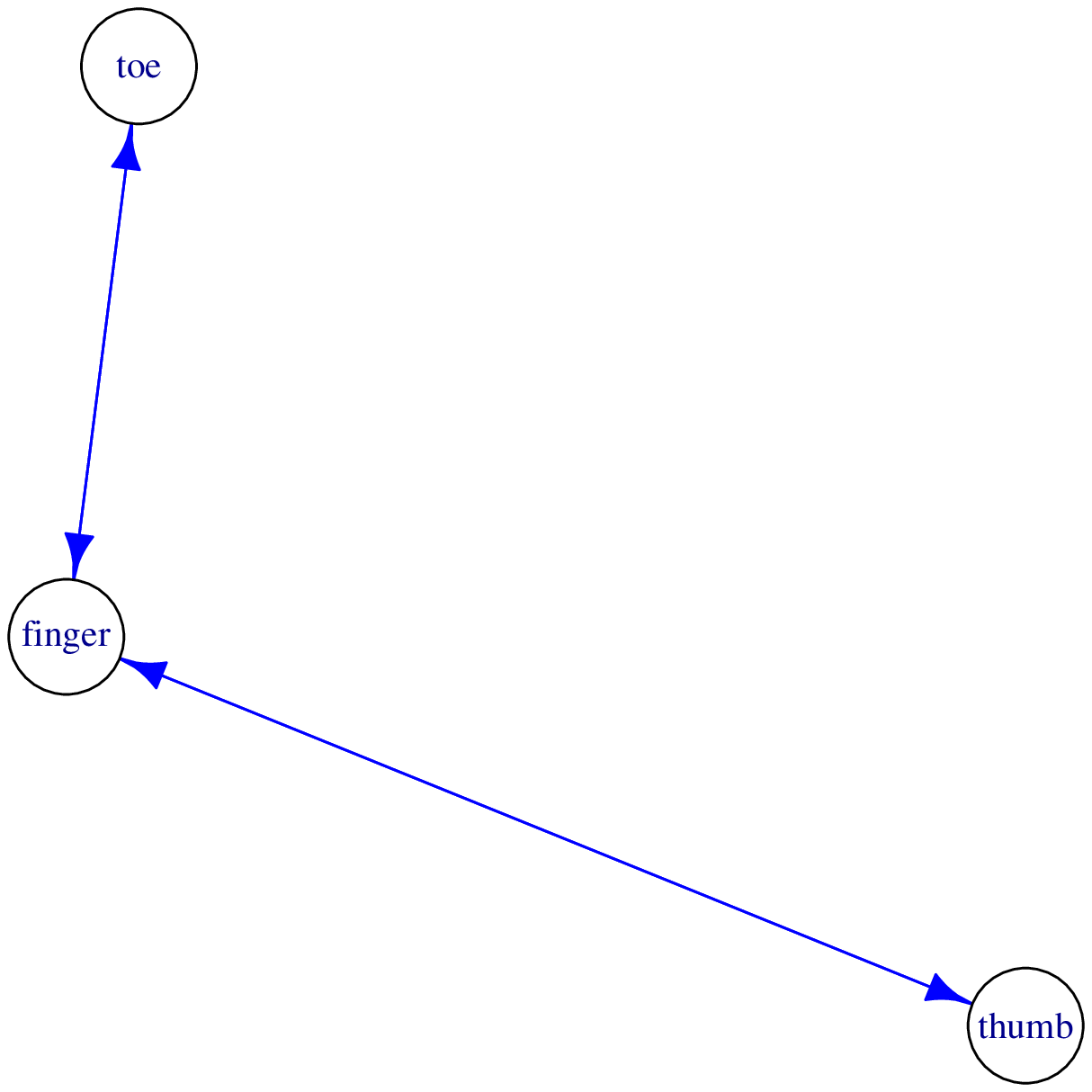}
\end{center}
\caption{$4$ nodes, $4$ edges.}\label{fig:strongly-connected:negative:3:d}
\end{subfigure}
\hspace{0.01\textwidth}
\begin{subfigure}[b]{0.48\textwidth}
\begin{center}
\includegraphics[width=\textwidth]{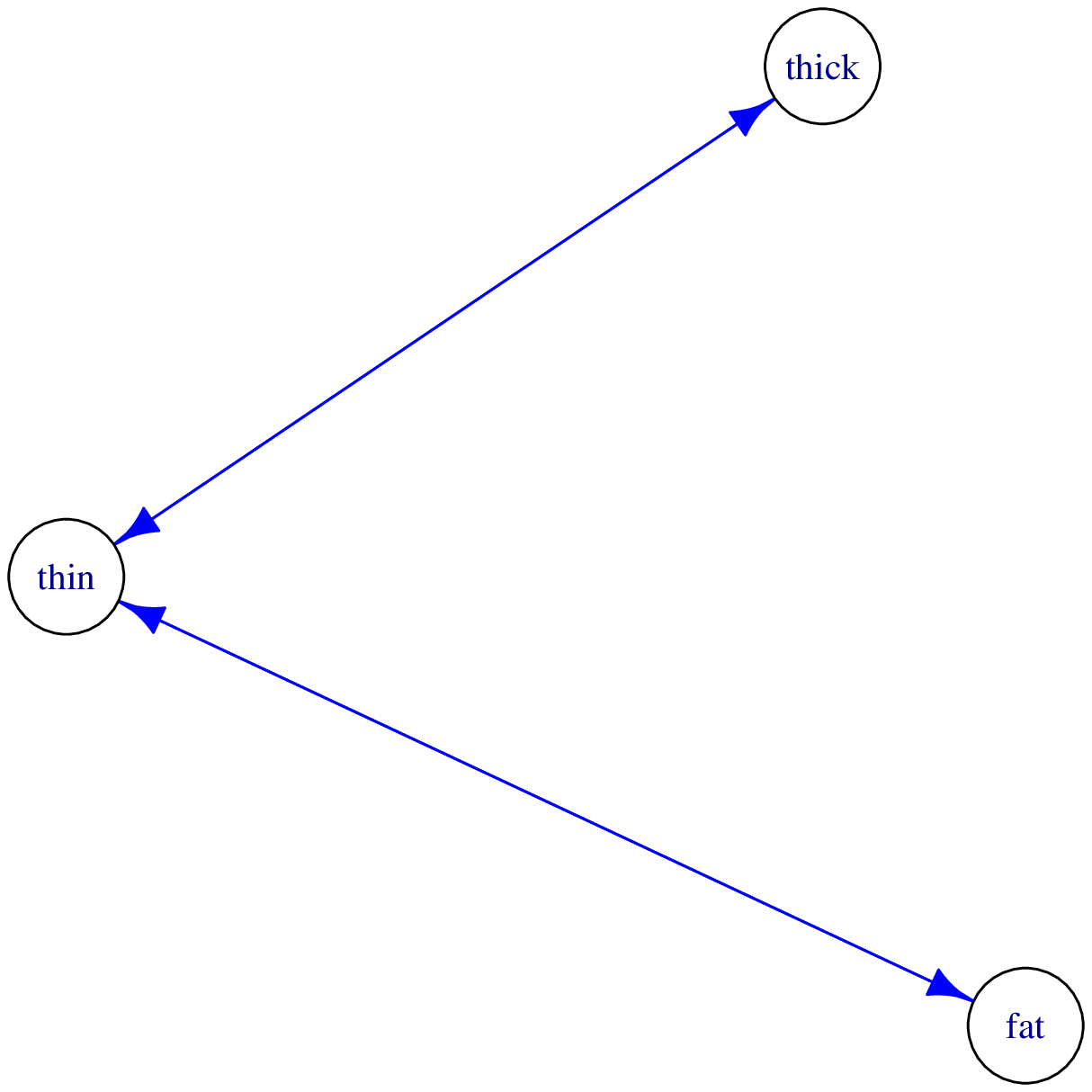}
\end{center}
\caption{$3$ nodes, $4$ edges.}\label{fig:strongly-connected:negative:3:e}
\end{subfigure}

\begin{subfigure}[b]{0.48\textwidth}
\begin{center}
\includegraphics[width=\textwidth]{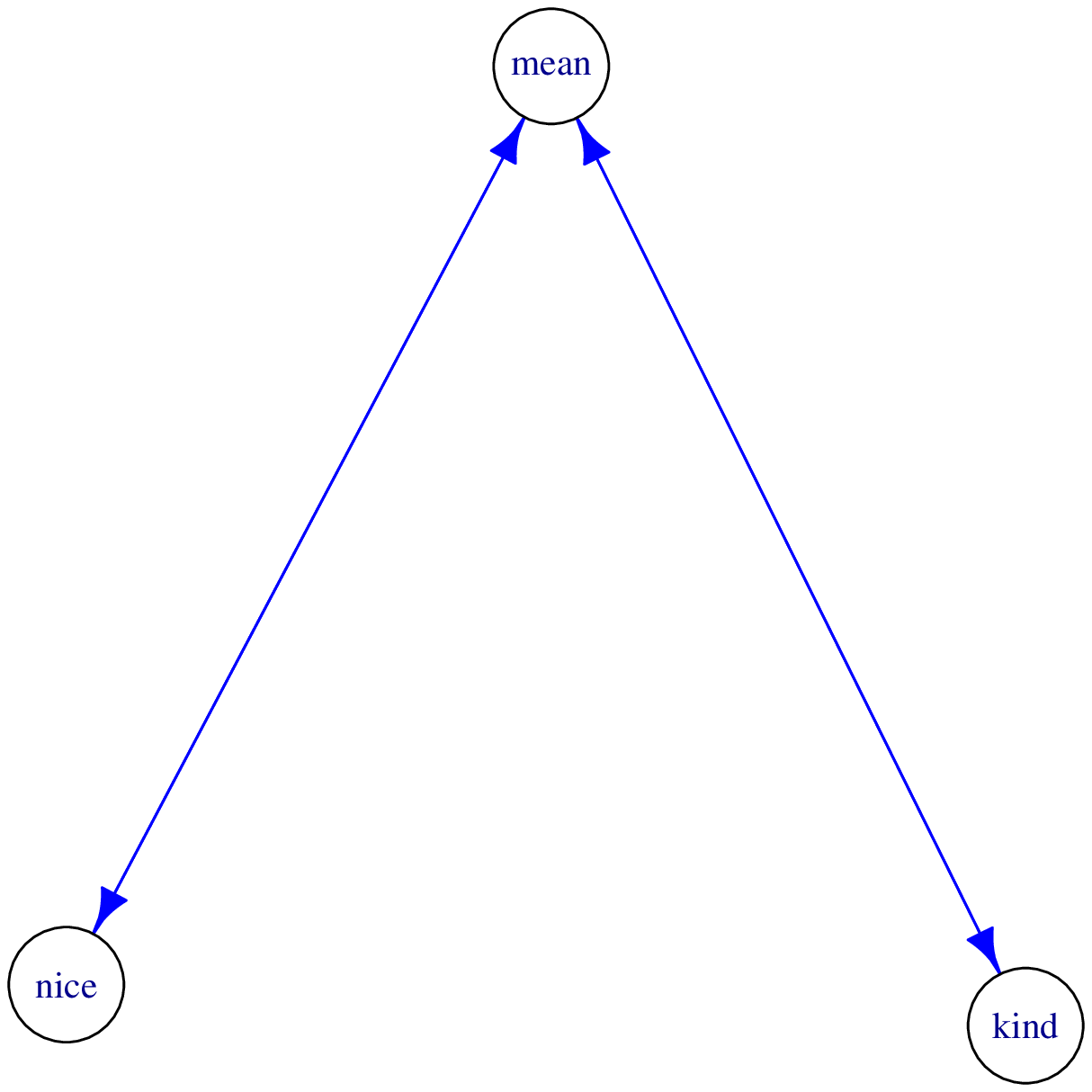}
\end{center}
\caption{$3$ nodes, $4$ edges.}\label{fig:strongly-connected:negative:3:f}
\end{subfigure}
\end{center}

\caption{Strongly connected components of size $3$ induced by assertions with negative polarity; 
see Table \ref{tbl:distribution:component:negative:strong}.}\label{fig:strongly-connected:negative:4}
\end{figure}

\subsubsection{Component of Size $12$}
In the strongly connected component of size $12$
we can find the concepts
\dbtext{front} (2423),
\dbtext{back} (15583),
\dbtext{side} (17836),
\dbtext{last} (23202),
\dbtext{edge} (24347),
\dbtext{corner} (29067),
\dbtext{after} (31656),
\dbtext{behind} (46824),
\dbtext{middle} (52077),
\dbtext{before} (108544),
\dbtext{rear} (141086), and
\dbtext{centre} (202139).
Figure \ref{fig:strongly-connected:negative:12} presents
the induced directed graph of that component.

\subsubsection{Component of Size $10$}
In the first strongly connected component of size $10$
we can find the concepts
\dbtext{year} (2709),
\dbtext{week} (2757),
\dbtext{day} (2759),
\dbtext{hour} (2762),
\dbtext{minute} (2764),
\dbtext{night} (8677),
\dbtext{morning} (15749),
\dbtext{afternoon} (15914),
\dbtext{even} (15946), and
\dbtext{month} (25290).
Figure \ref{fig:strongly-connected:negative:10:a} presents
the induced directed graph of that component.

In the second strongly connected component of size $10$
we can find the concepts
\dbtext{difficult} (195),
\dbtext{plain} (1155),
\dbtext{soft} (2842),
\dbtext{hard} (7545),
\dbtext{simple} (15368),
\dbtext{easy} (19144),
\dbtext{smooth} (24330),
\dbtext{fancy} (24730),
\dbtext{rough} (34315), and
\dbtext{gentle} (55184).
Figure \ref{fig:strongly-connected:negative:10:b} presents
the induced directed graph of that component.

\subsubsection{Component of size $8$}
In the strongly connected component of size $8$
we can find the concepts
\dbtext{cold} (912),
\dbtext{winter} (1431),
\dbtext{summer} (1437),
\dbtext{hot} (1438),
\dbtext{rise} (5930),
\dbtext{heat} (7301),
\dbtext{cool} (7306), and
\dbtext{fall} (9975).
Figure \ref{fig:strongly-connected:negative:8} presents
the induced directed graph of that component.

\subsubsection{Component of Size $5$}
In the strongly connected component of size $5$
we can find the concepts
\dbtext{local} (60886),
\dbtext{foreigner} (62358),
\dbtext{native} (94333),
\dbtext{express} (141657), and
\dbtext{foreign} (333670).
Figure \ref{fig:strongly-connected:negative:5} presents
the induced directed graph of that component.

\subsubsection{Components of Size $4$}
In the first strongly connected component of size $4$
we can find the concepts
\dbtext{south} (6265),
\dbtext{west} (9659),
\dbtext{north} (22569), and
\dbtext{east} (42579).
Figure \ref{fig:strongly-connected:negative:4:a} presents
the induced directed graph of that component.

In the second strongly connected component of size $4$
we can find the concepts
\dbtext{receive} (15790),
\dbtext{take} (17431),
\dbtext{give} (43731), and
\dbtext{send} (162951).
Figure \ref{fig:strongly-connected:negative:4:b} presents
the induced directed graph of that component.

In the third strongly connected component of size $4$
we can find the concepts
\dbtext{sugar} (1446),
\dbtext{salt} (1817),
\dbtext{pepper} (4326), and
\dbtext{spice} (8644).
Figure \ref{fig:strongly-connected:negative:4:c} presents
the induced directed graph of that component.

In the fourth strongly connected component of size $4$
we can find the concepts
\dbtext{light} (1716),
\dbtext{bright} (1717),
\dbtext{dark} (6376), and
\dbtext{dim} (101382).
Figure \ref{fig:strongly-connected:negative:4:d} presents
the induced directed graph of that component.

\subsubsection{Components of Size $3$}
In the first strongly connected component of size $3$
we can find the concepts
\dbtext{but} (35882),
\dbtext{and} (40224), and
\dbtext{or} (40776).
Figure \ref{fig:strongly-connected:negative:3:a} presents
the induced directed graph of that component.

In the second strongly connected component of size $3$
we can find the concepts
\dbtext{general} (6836),
\dbtext{captain} (23817), and
\dbtext{colonel} (332231).
Figure \ref{fig:strongly-connected:negative:3:b} presents
the induced directed graph of that component.

In the third strongly connected component of size $3$
we can find the concepts
\dbtext{narrow} (17316),
\dbtext{wide} (27291), and
\dbtext{broad} (48158).
Figure \ref{fig:strongly-connected:negative:3:c} presents
the induced directed graph of that component.

In the fourth strongly connected component of size $3$
we can find the concepts
\dbtext{finger} (3399),
\dbtext{toe} (5571), and
\dbtext{thumb} (15862).
Figure \ref{fig:strongly-connected:negative:3:d} presents
the induced directed graph of that component.

In the fifth strongly connected component of size $3$
we can find the concepts
\dbtext{fat} (1763),
\dbtext{thin} (9272), and
\dbtext{thick} (56754).
Figure \ref{fig:strongly-connected:negative:3:e} presents
the induced directed graph of that component.

In the sixth strongly connected component of size $3$
we can find the concepts
\dbtext{nice} (2028),
\dbtext{mean} (6744), and
\dbtext{kind} (31540).
Figure \ref{fig:strongly-connected:negative:3:f} presents
the induced directed graph of that component.

\section{Positive Polarity: Connected Components}
In this section we examine the weakly and strongly connected components
of the graphs induced by assertions with positive polarity only.

\subsection{Weakly Connected Components}
We get $38,153$ weakly connected components, out of which
$22,651$ are isolated vertices. Note that $22,651$ is in complete agreement with
Table \ref{tbl:number-of-edges-isolated-vertices:overall}. 
Among the rest $15,502$ components we can find
components with cardinalities between $2$ and $223,679$.

\paragraph{Distribution of Component Sizes.}
The distribution of the sizes for the various components is shown in 
Table \ref{tbl:distribution:component:positive:weak}.
This distribution presents the cardinalities of the weakly connected components
of the induced directed graph, as well as the cardinalities of the connected
components of the induced undirected graph.
For the induced graphs we consider assertions with positive score in the English language
and we allow all frequencies in the edges.

\begin{table}[ht]
\caption{Distribution of sizes for weakly connected components for the directed 
graph induced by the assertions with positive polarity only. 
This is also the distribution of sizes for the connected 
components of the induced undirected graph. 
}\label{tbl:distribution:component:positive:weak}
\begin{center}
\resizebox{\textwidth}{!}{
\begin{tabular}{|r||c|c|c|c|c|c|c|c|c|c|c|c|c|c|c|c|c|c|c|c|c|}\hline
\# of nodes      & $223,679$ & $55$ & $32$ & $31$ & $30$ & $22$ & $18$ & $16$ & $14$ & $12$ & $11$ & $10$ & $9$ &  $8$ &  $7$ &  $6$ &  $5$ &   $4$ &   $3$ &      $2$ &      $1$ \\\hline
\# of components &       $1$ &  $1$ &  $1$ &  $1$ &  $1$ &  $2$ &  $1$ &  $1$ &  $4$ &  $1$ &  $3$ &  $3$ & $4$ & $11$ & $14$ & $26$ & $81$ & $196$ & $943$ & $14,207$ & $22,651$ \\\hline
\end{tabular}
}
\end{center}
\end{table}

Figure \ref{fig:weakly-connected:positive} presents the weakly connected components with sizes
$11$-$55$. Note that in Chapter \ref{chapter:shortest-paths} we will explore the longest geodesic paths 
of the induced directed and undirected graphs and we will see that that in every case
such a path is at least $15$. Hence, Figure \ref{fig:weakly-connected:positive} apart from 
giving an overview of some of the weakly connected components, it also shows that the longest
geodesic paths do not come from any of these components.

\begin{figure}[p]
\begin{center}
\begin{subfigure}[b]{0.23\textwidth}
\begin{center}
\includegraphics[width=\textwidth]{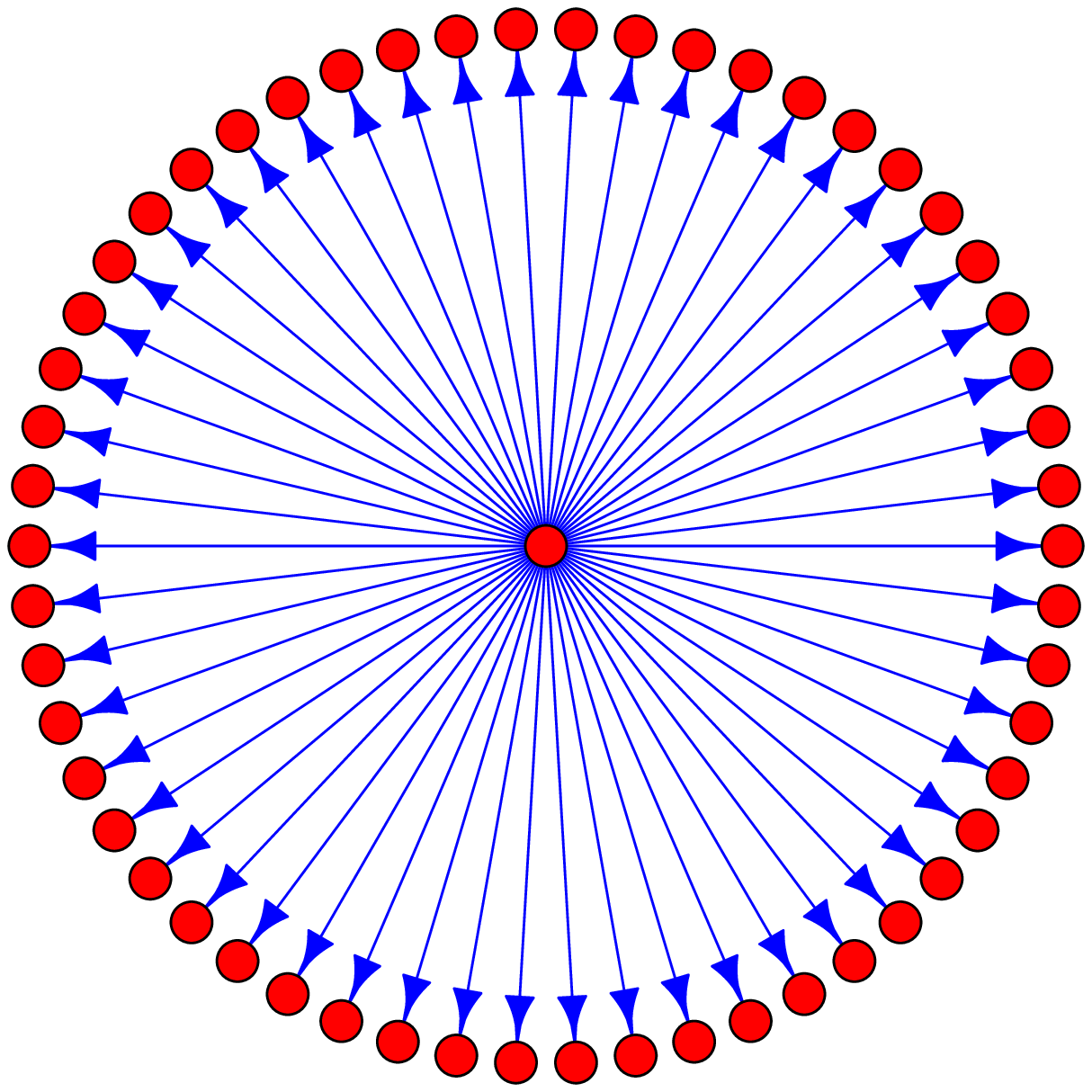}
\end{center}
\caption{medical specialty\\$55$ nodes, $54$ edges.}\label{fig:weakly-connected:positive:55}
\end{subfigure}
\hspace{0.01\textwidth}
\begin{subfigure}[b]{0.23\textwidth}
\begin{center}
\includegraphics[width=\textwidth]{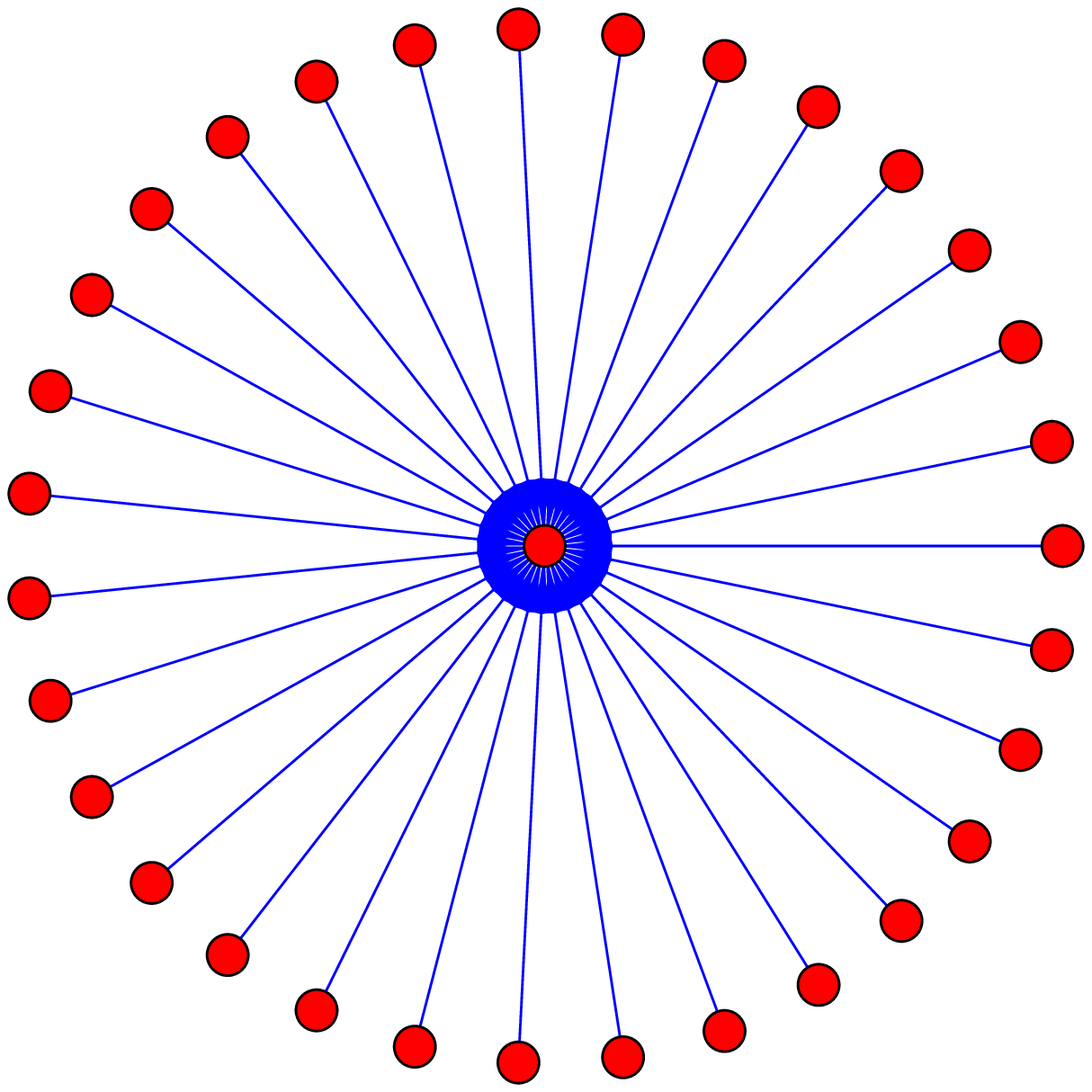}
\end{center}
\caption{pacific ocean 0 m\\$32$ nodes, $31$ edges.}\label{fig:weakly-connected:positive:32}
\end{subfigure}
\hspace{0.01\textwidth}
\begin{subfigure}[b]{0.23\textwidth}
\begin{center}
\includegraphics[width=\textwidth]{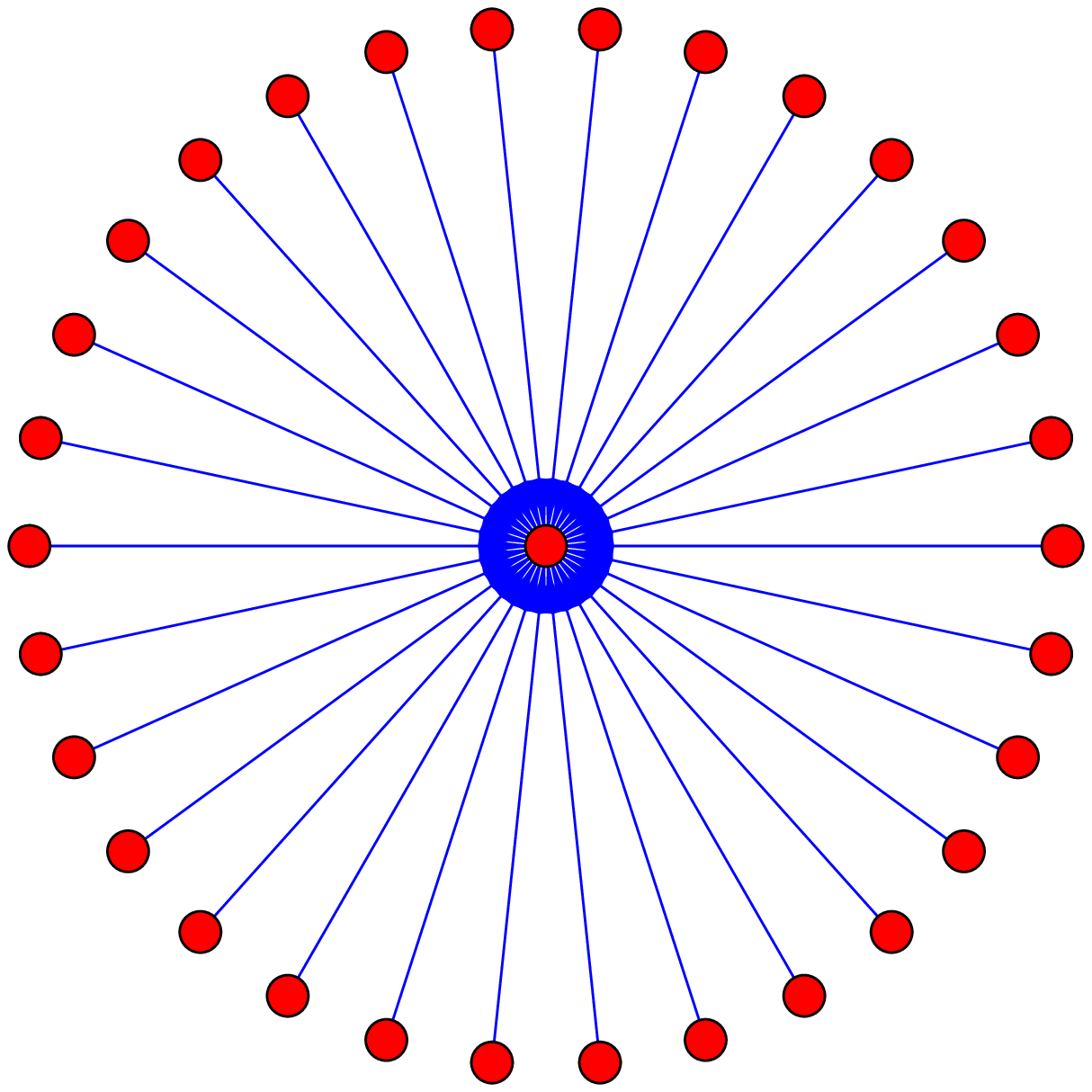}
\end{center}
\caption{atlantic ocean 0 m\\$31$ nodes, $30$ edges.}\label{fig:weakly-connected:positive:31}
\end{subfigure}
\hspace{0.01\textwidth}
\begin{subfigure}[b]{0.23\textwidth}
\begin{center}
\includegraphics[width=\textwidth]{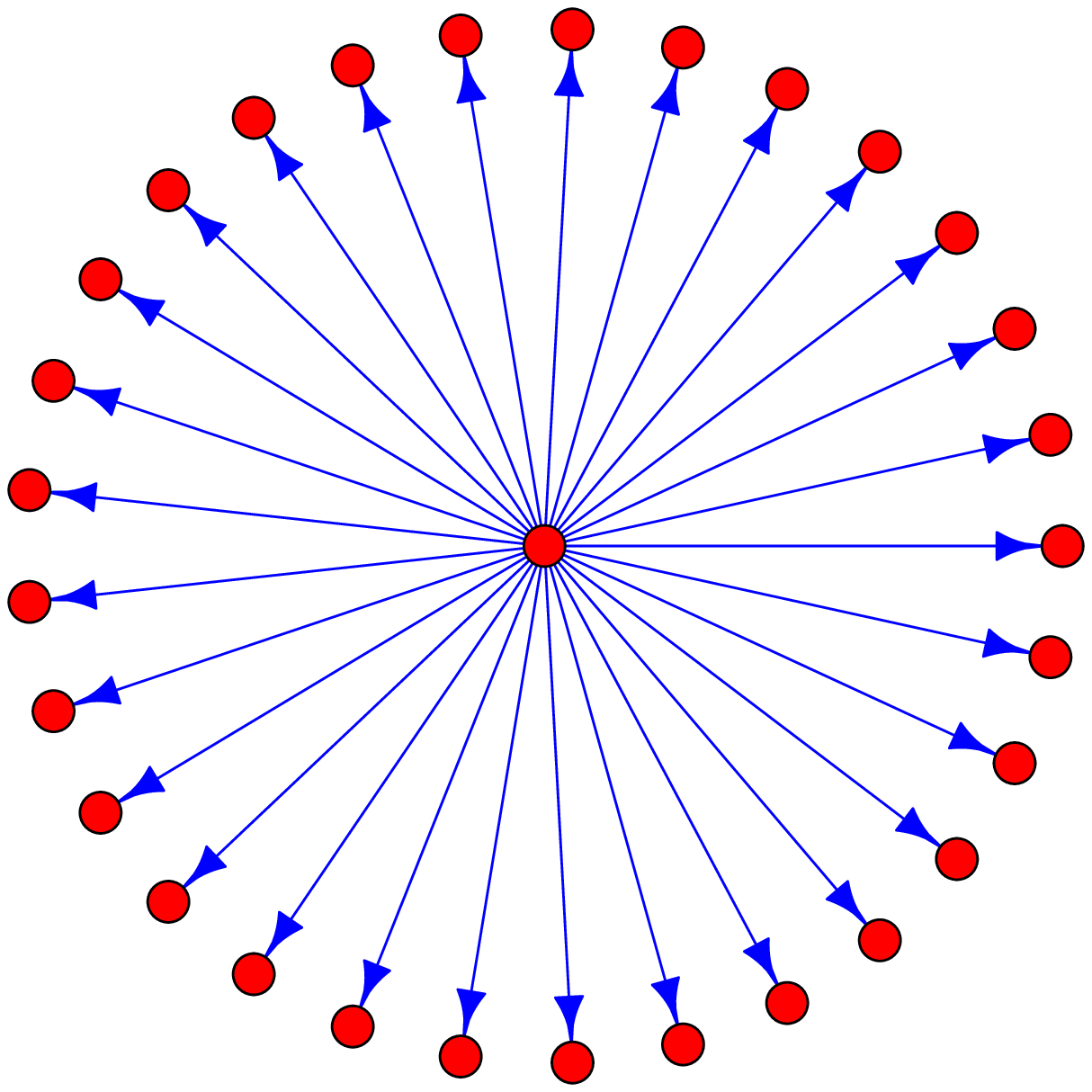}
\end{center}
\caption{haha\\$30$ nodes, $29$ edges.}\label{fig:weakly-connected:positive:30}
\end{subfigure}

\begin{subfigure}[b]{0.23\textwidth}
\begin{center}
\includegraphics[width=\textwidth]{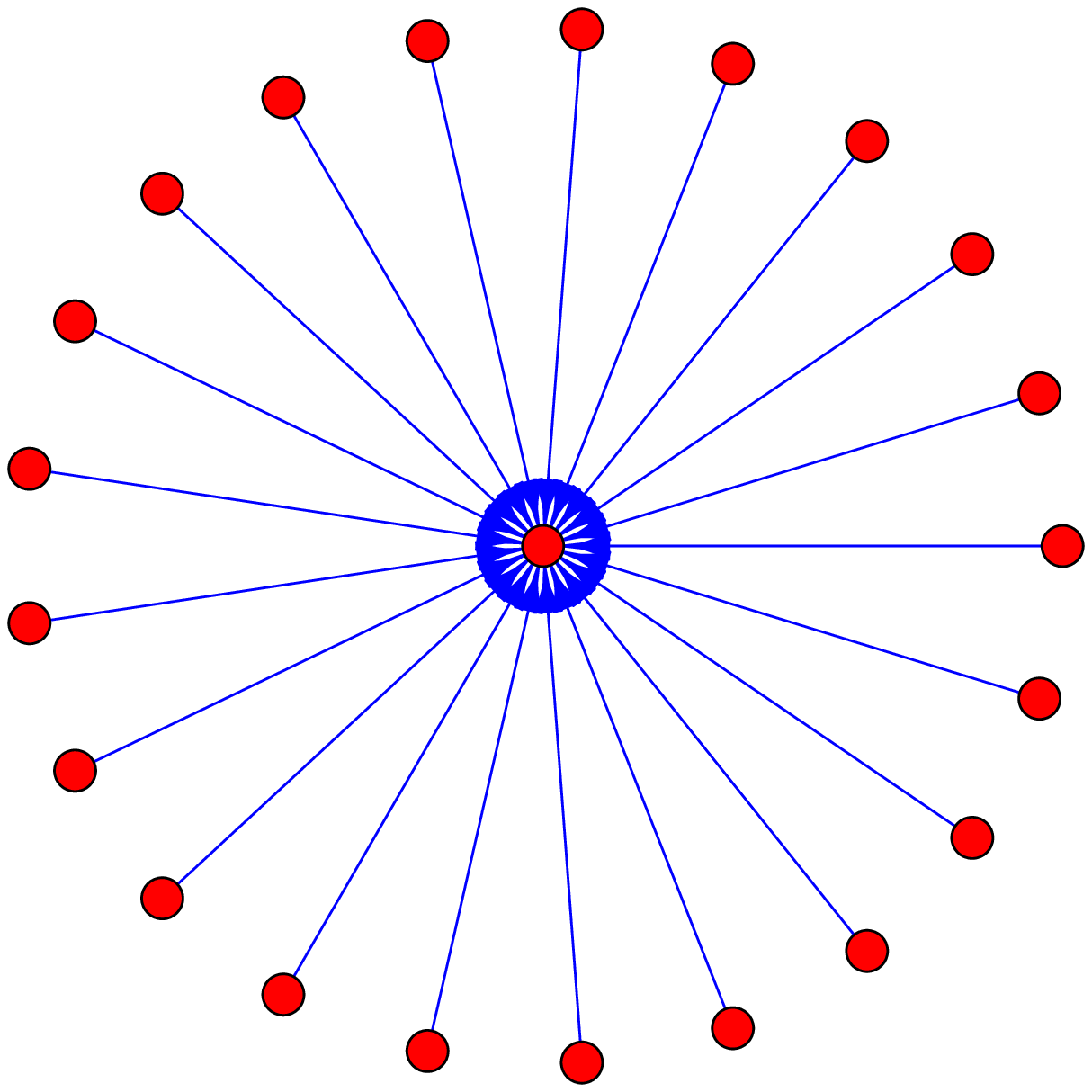}
\end{center}
\caption{indian ocean 0 m\\$22$ nodes, $21$ edges.}\label{fig:weakly-connected:positive:22a}
\end{subfigure}
\hspace{0.01\textwidth}
\begin{subfigure}[b]{0.23\textwidth}
\begin{center}
\includegraphics[width=\textwidth]{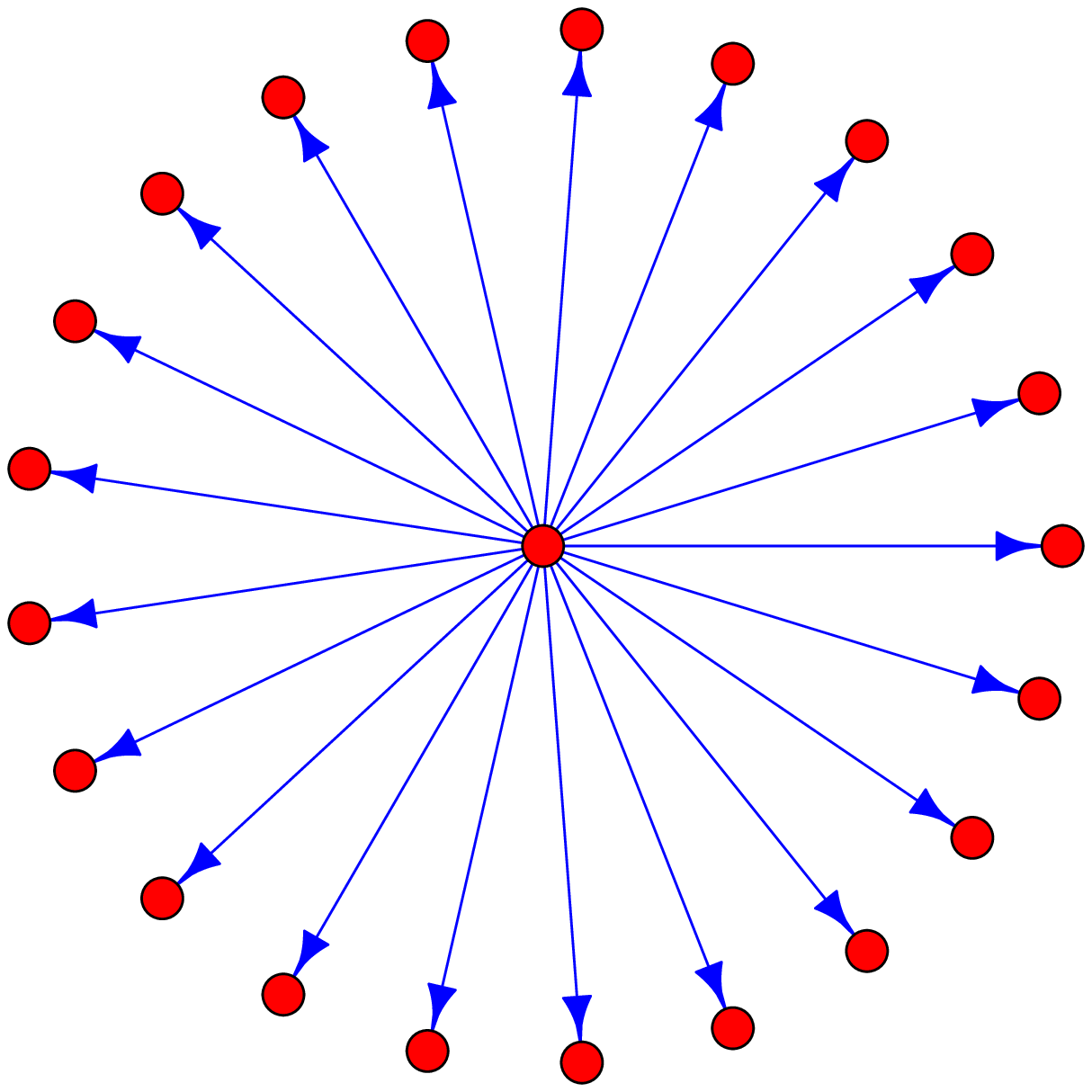}
\end{center}
\caption{space shuttle acronym\\$22$ nodes, $21$ edges.}\label{fig:weakly-connected:positive:22b}
\end{subfigure}
\hspace{0.01\textwidth}
\begin{subfigure}[b]{0.23\textwidth}
\begin{center}
\includegraphics[width=\textwidth]{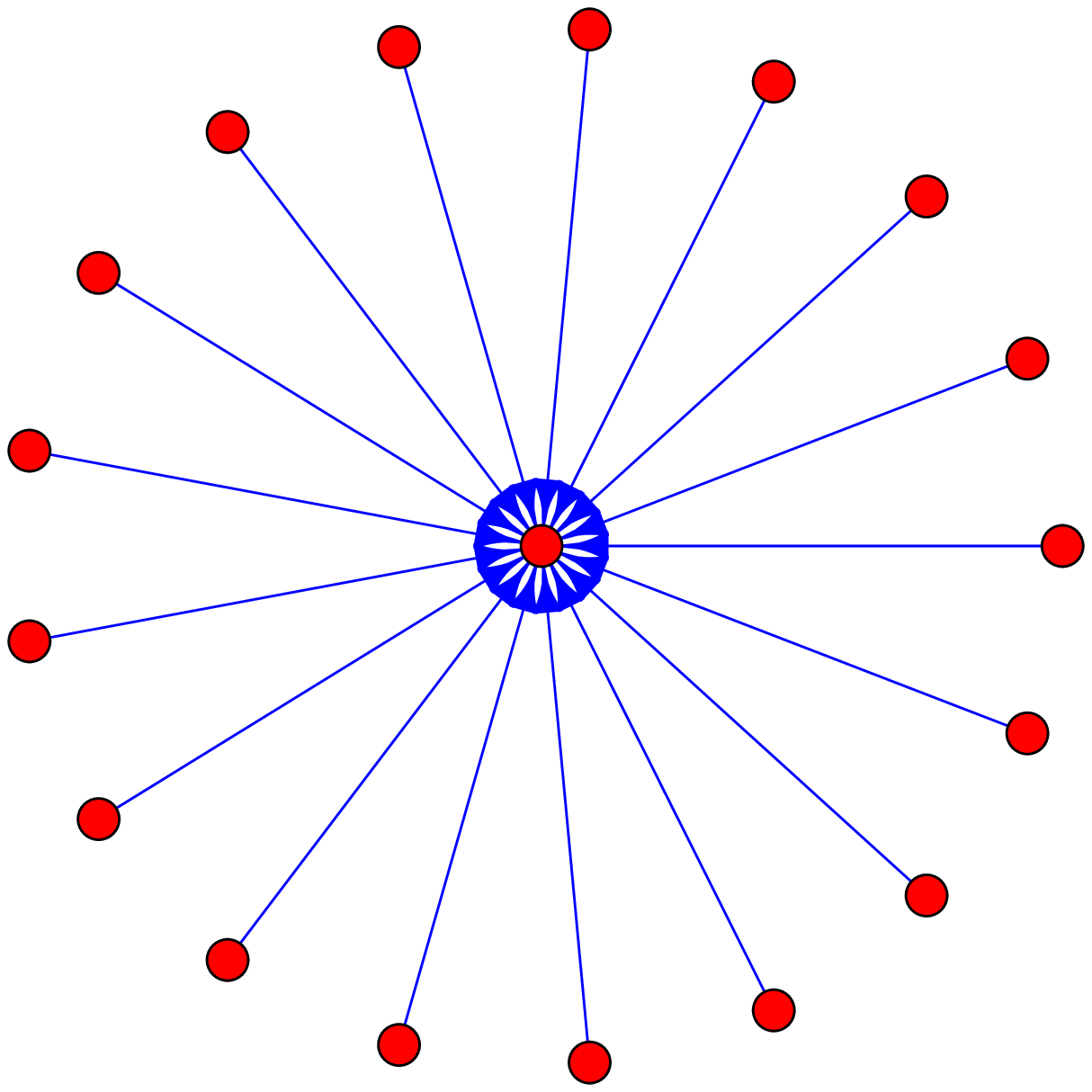}
\end{center}
\caption{caribbean sea 0 m\\$18$ nodes, $17$ edges.}\label{fig:weakly-connected:positive:18}
\end{subfigure}
\hspace{0.01\textwidth}
\begin{subfigure}[b]{0.23\textwidth}
\begin{center}
\includegraphics[width=\textwidth]{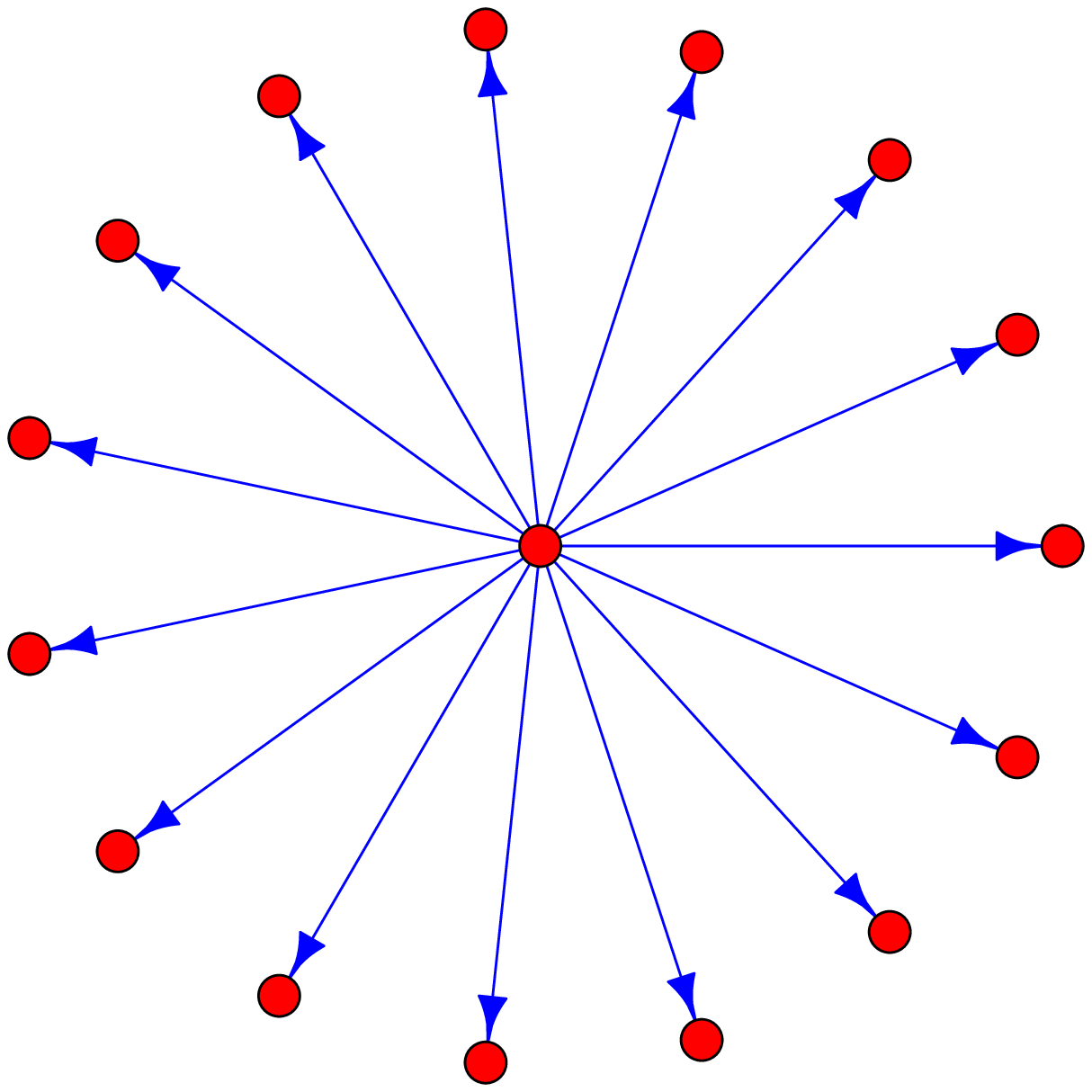}
\end{center}
\caption{another say safe\\$16$ nodes, $15$ edges.}\label{fig:weakly-connected:positive:16}
\end{subfigure}

\begin{subfigure}[b]{0.23\textwidth}
\begin{center}
\includegraphics[width=\textwidth]{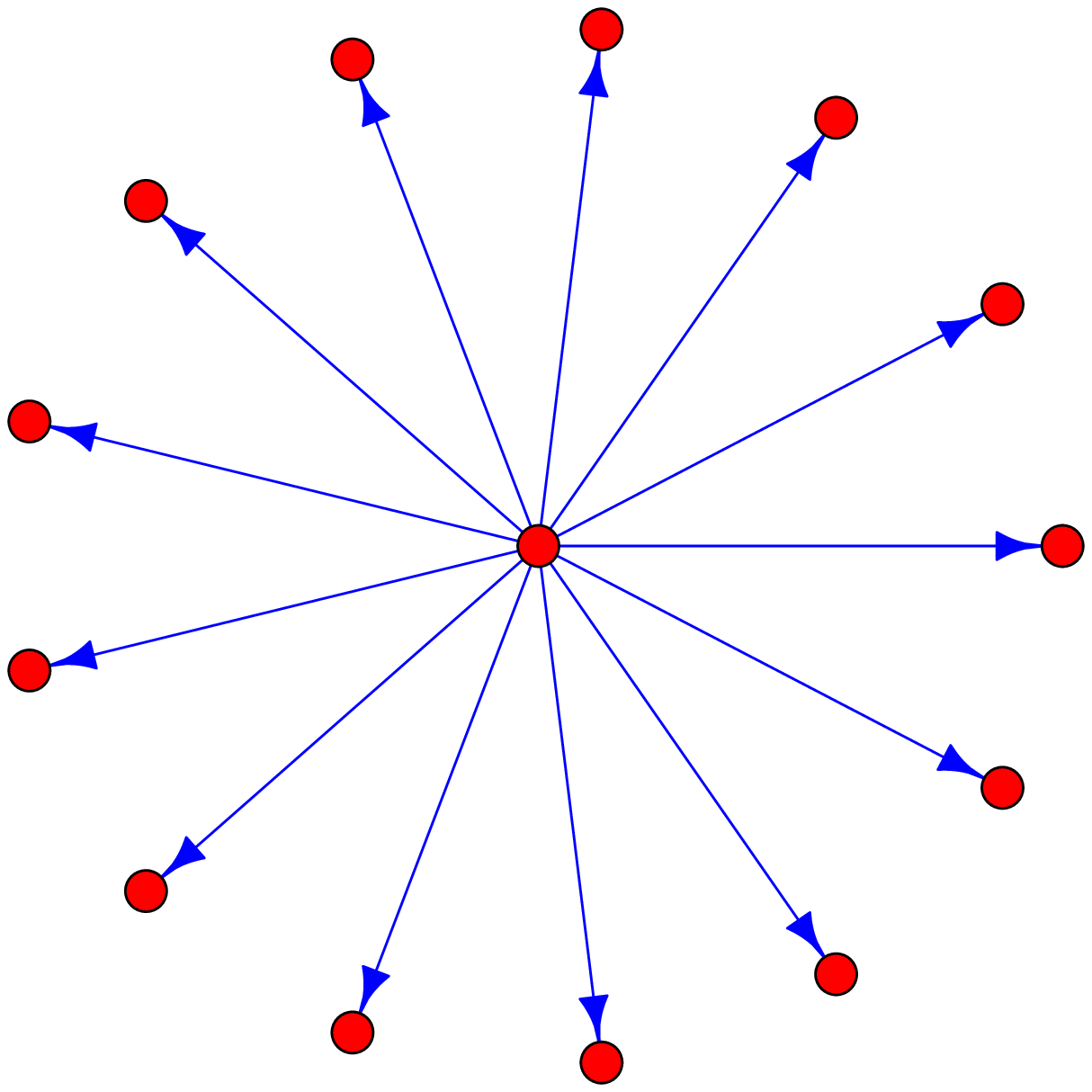}
\end{center}
\caption{alani\\$14$ nodes, $13$ edges.}\label{fig:weakly-connected:positive:14a}
\end{subfigure}
\hspace{0.01\textwidth}
\begin{subfigure}[b]{0.23\textwidth}
\begin{center}
\includegraphics[width=\textwidth]{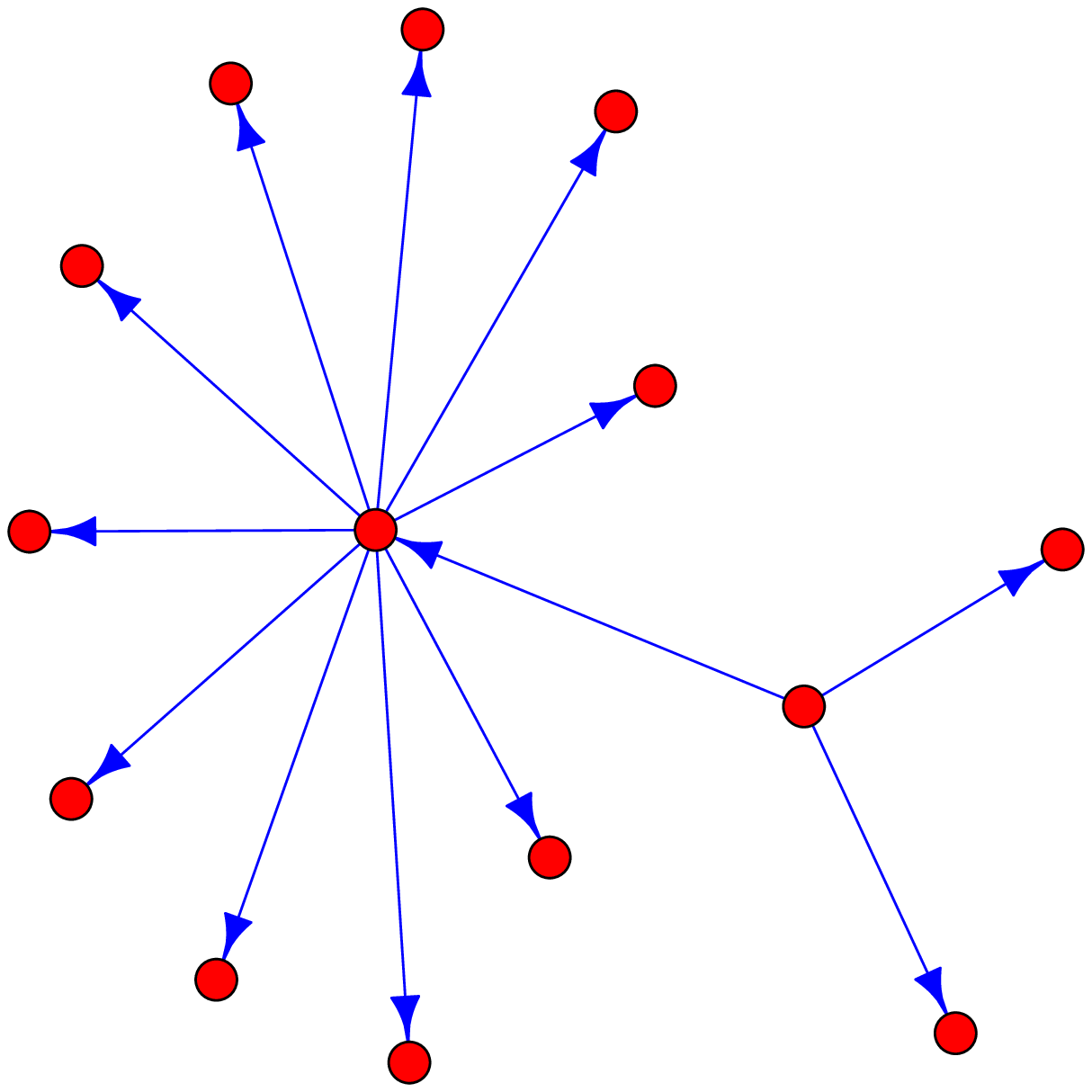}
\end{center}
\caption{different culture, different country\\$14$ nodes, $13$ edges.}\label{fig:weakly-connected:positive:14b}
\end{subfigure}
\hspace{0.01\textwidth}
\begin{subfigure}[b]{0.23\textwidth}
\begin{center}
\includegraphics[width=\textwidth]{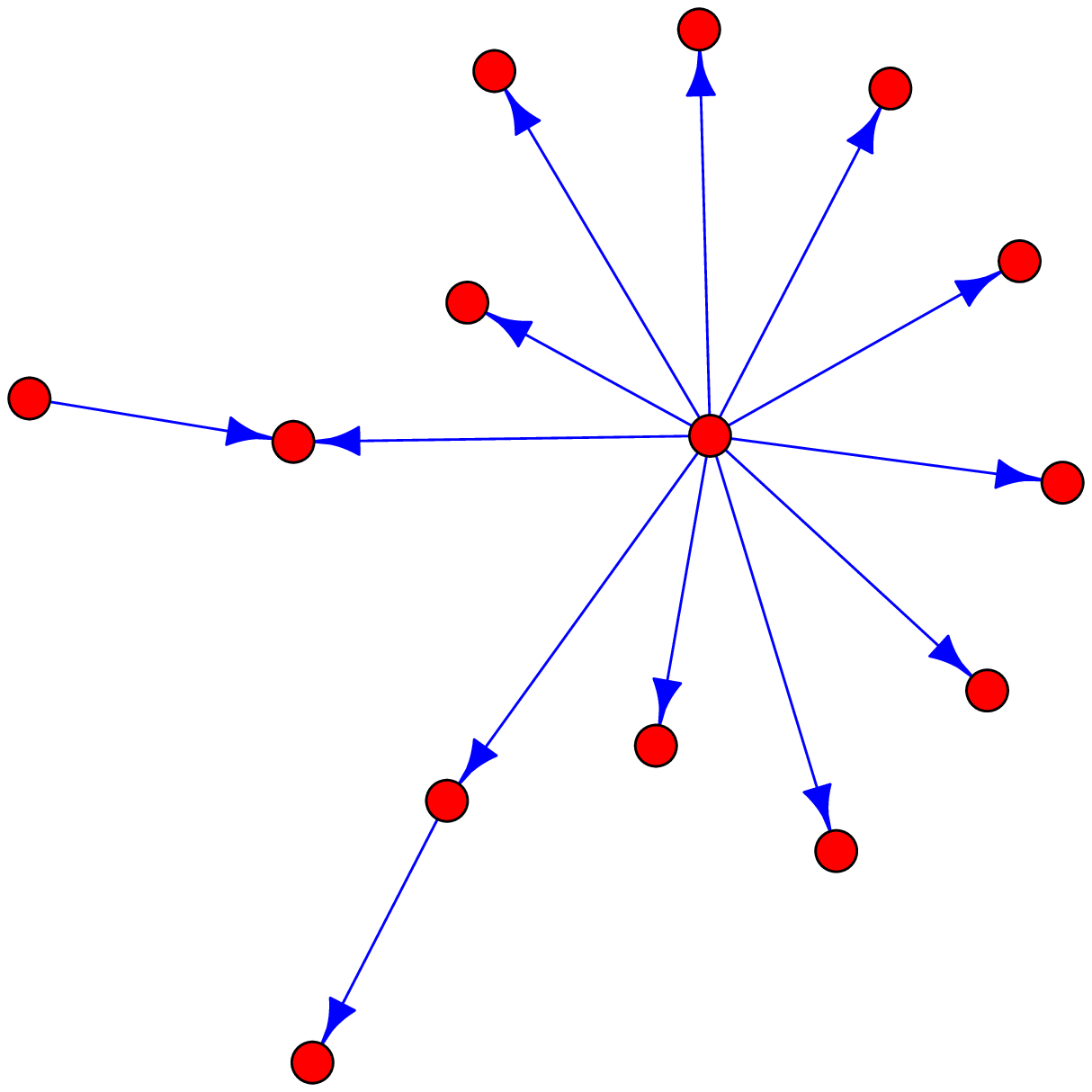}
\end{center}
\caption{type catheter, two channel, female catheter\\$14$ nodes, $13$ edges.}\label{fig:weakly-connected:positive:14c}
\end{subfigure}
\hspace{0.01\textwidth}
\begin{subfigure}[b]{0.23\textwidth}
\begin{center}
\includegraphics[width=\textwidth]{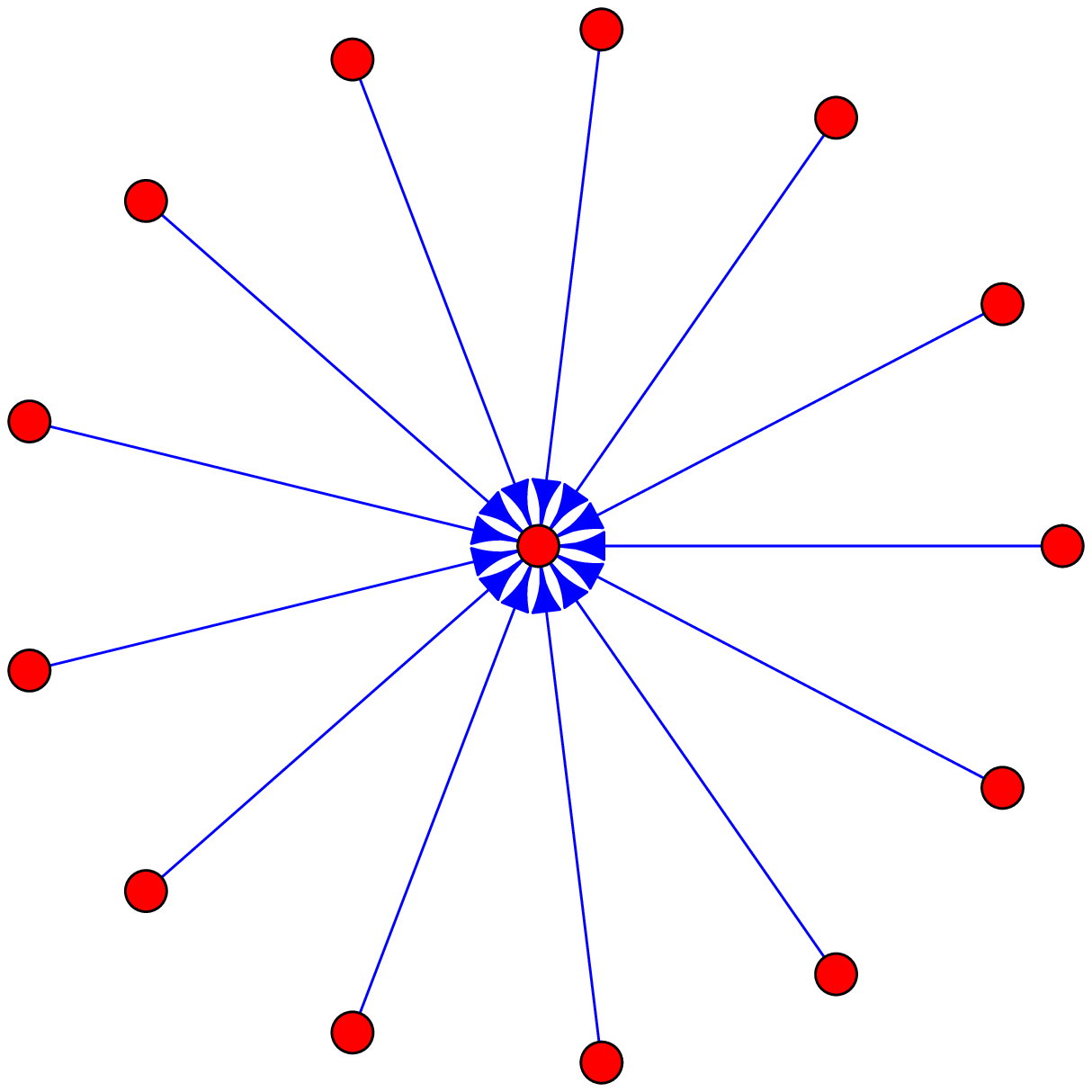}
\end{center}
\caption{rnum virus\\$14$ nodes, $13$ edges.}\label{fig:weakly-connected:positive:14d}
\end{subfigure}

\begin{subfigure}[b]{0.23\textwidth}
\begin{center}
\includegraphics[width=\textwidth]{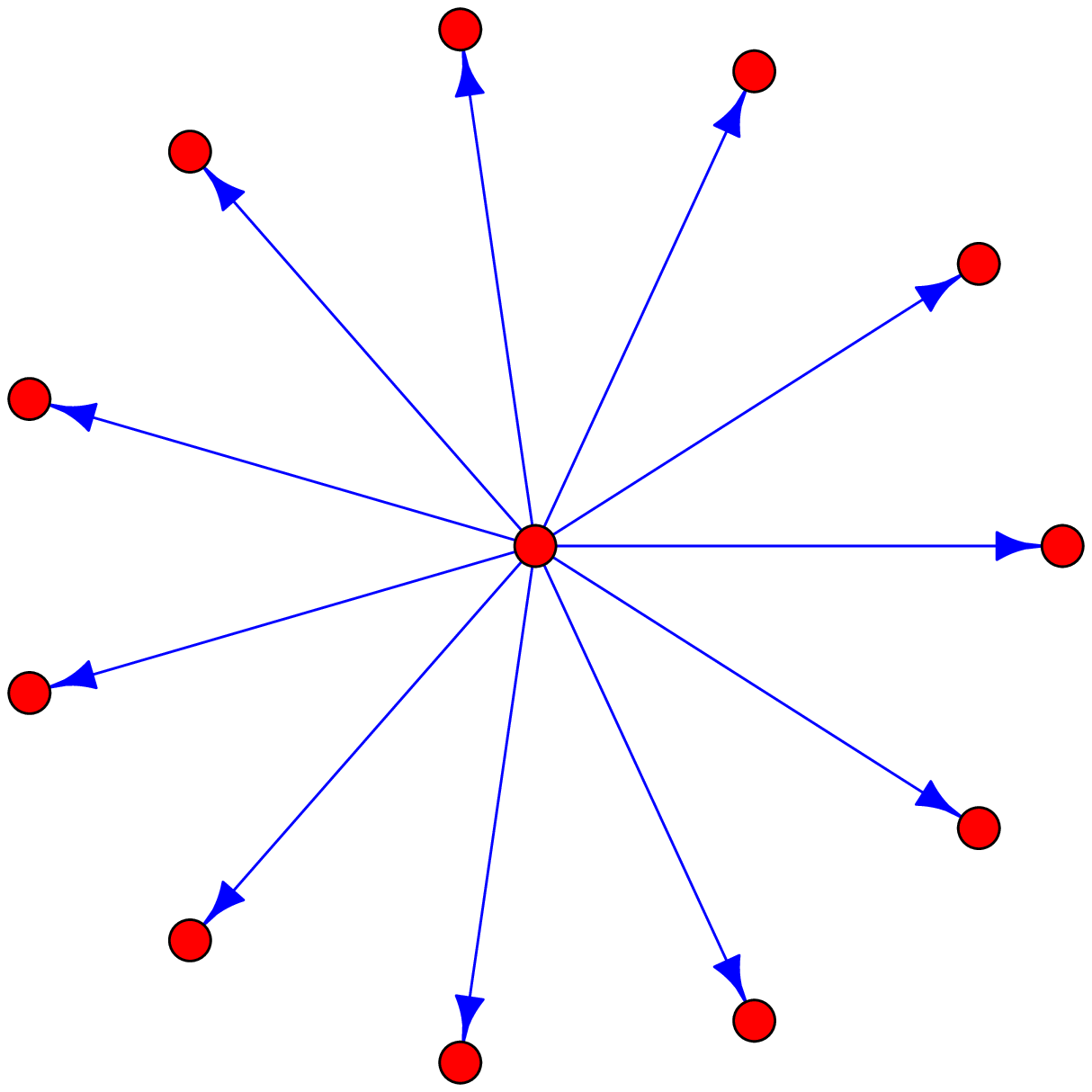}
\end{center}
\caption{dirge\\$12$ nodes, $11$ edges.}\label{fig:weakly-connected:positive:12}
\end{subfigure}
\hspace{0.01\textwidth}
\begin{subfigure}[b]{0.23\textwidth}
\begin{center}
\includegraphics[width=\textwidth]{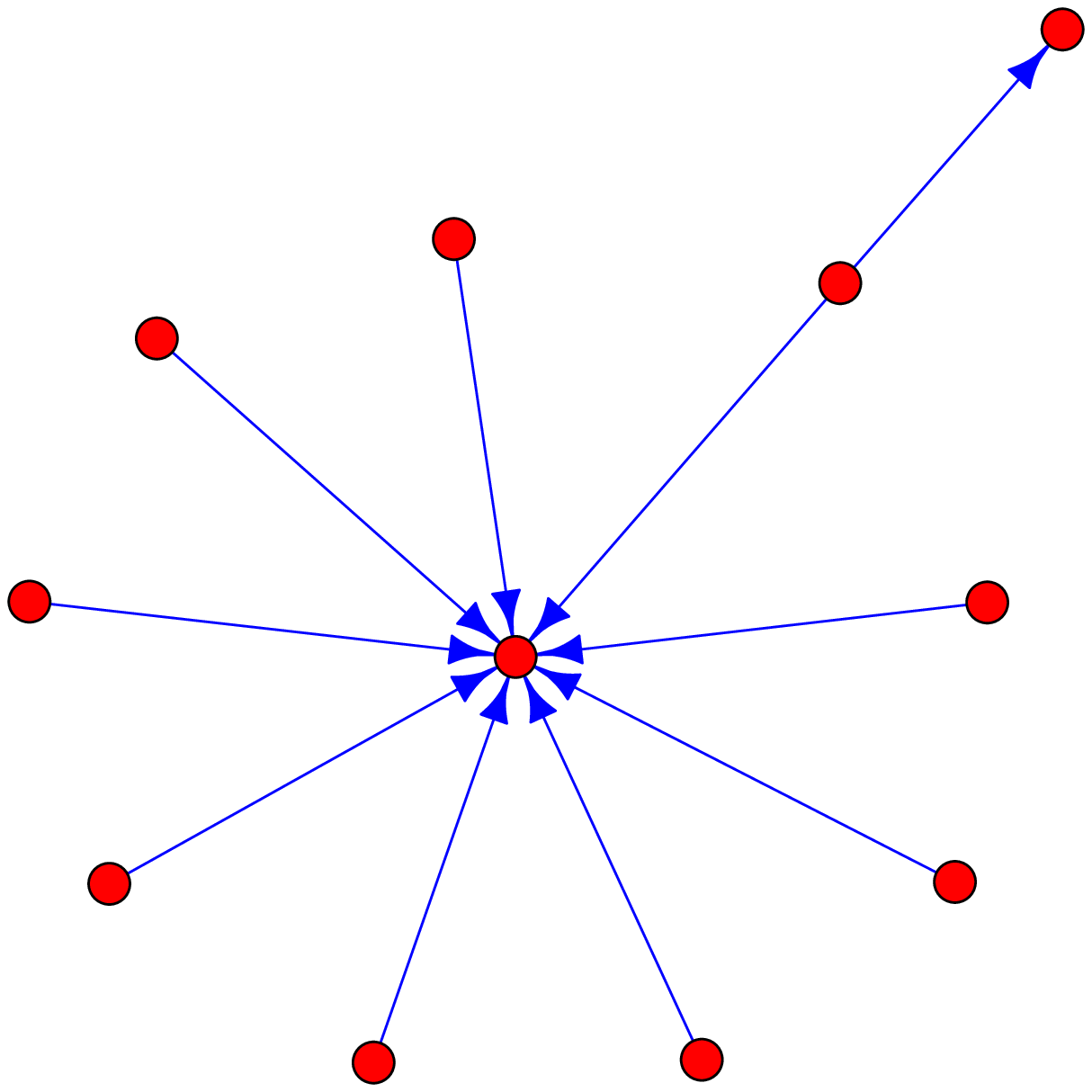}
\end{center}
\caption{darkish region mar, margaritifer sinus\\$11$ nodes, $10$ edges.}\label{fig:weakly-connected:positive:11a}
\end{subfigure}
\hspace{0.01\textwidth}
\begin{subfigure}[b]{0.23\textwidth}
\begin{center}
\includegraphics[width=\textwidth]{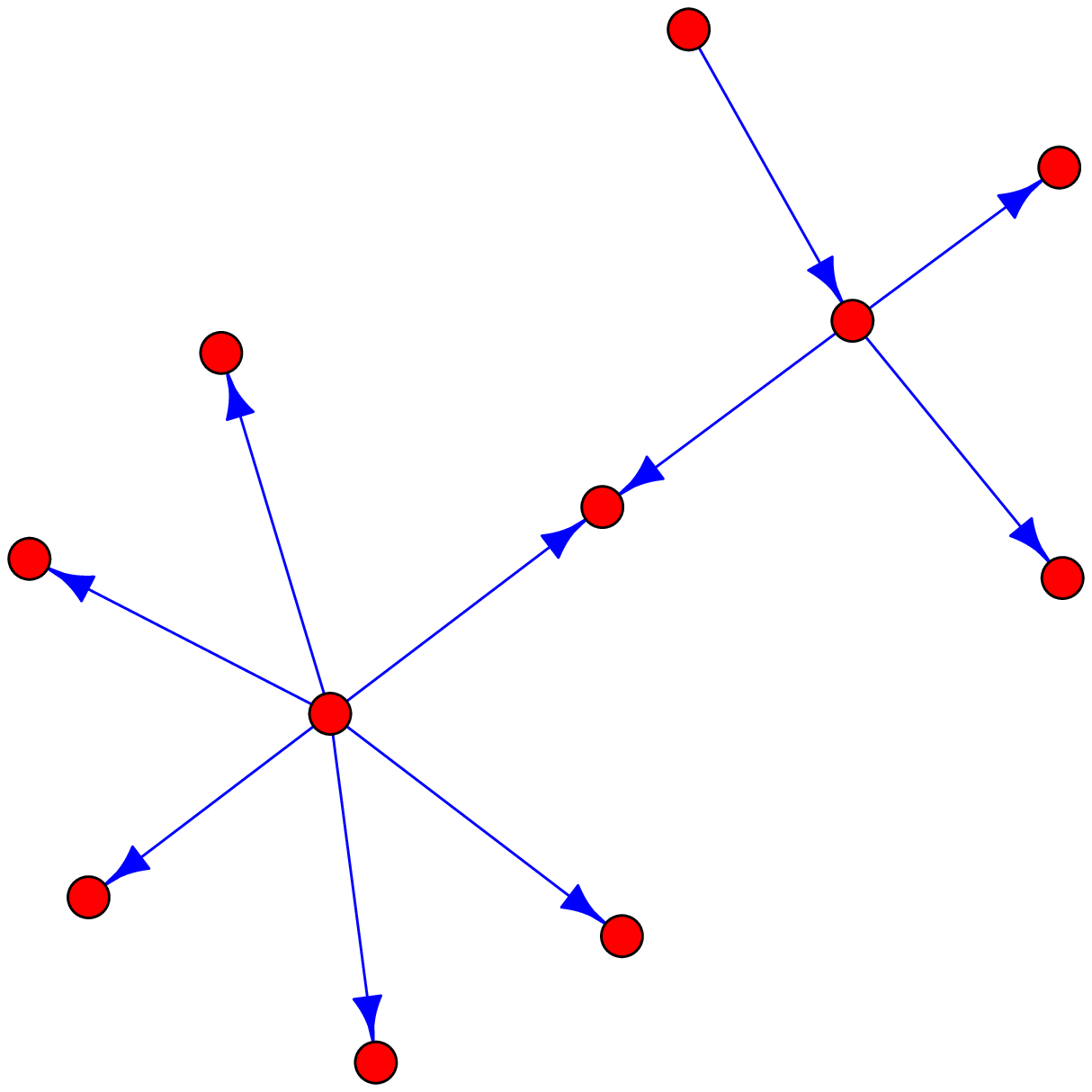}
\end{center}
\caption{hydrogen peroxide, h2o2, powerful oxidizer\\$11$ nodes, $10$ edges.}\label{fig:weakly-connected:positive:11b}
\end{subfigure}
\hspace{0.01\textwidth}
\begin{subfigure}[b]{0.23\textwidth}
\begin{center}
\includegraphics[width=\textwidth]{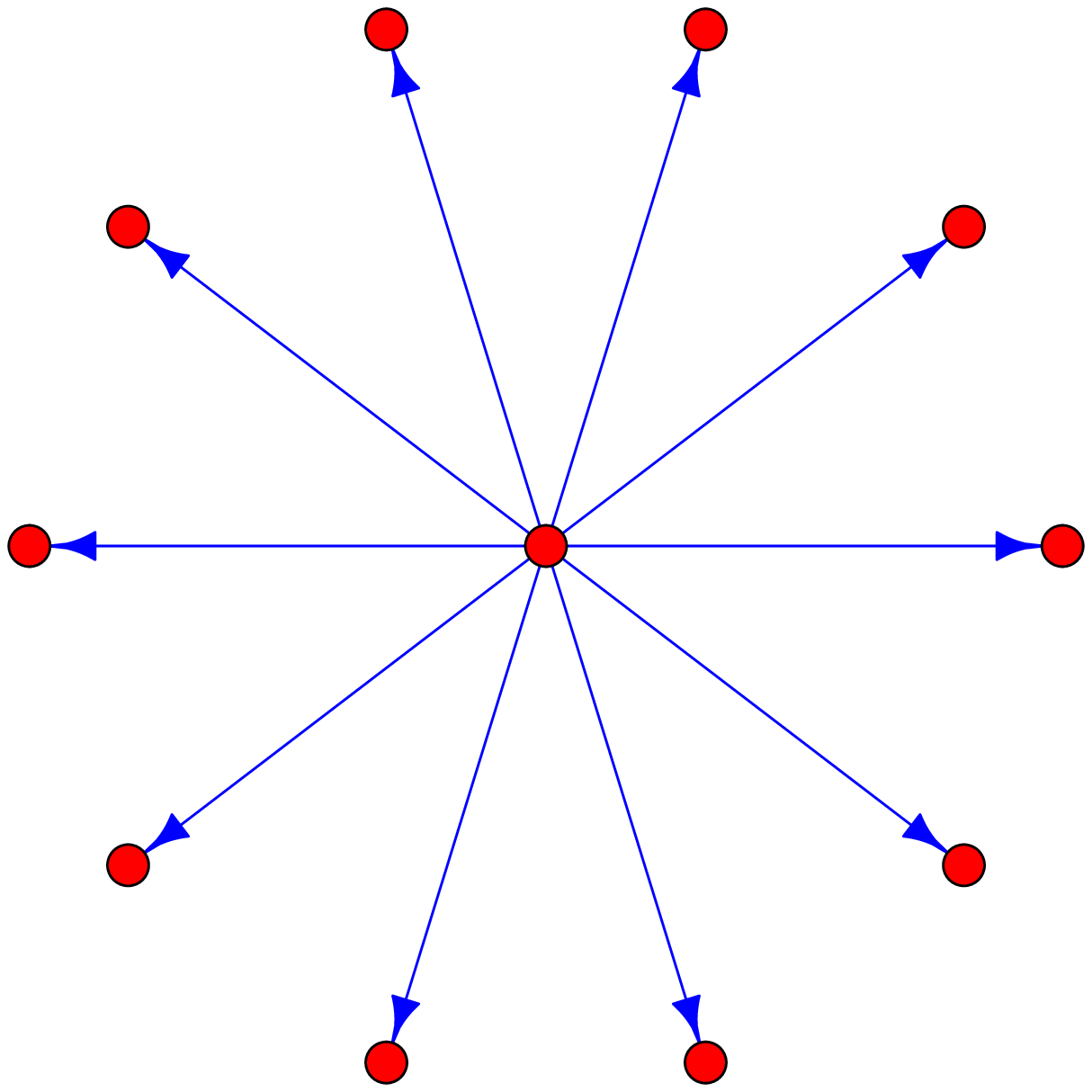}
\end{center}
\caption{sulfa drug\\$11$ nodes, $10$ edges.}\label{fig:weakly-connected:positive:11c}
\end{subfigure}
\end{center}

\caption{Weakly connected components that arise in the \emph{directed} graph 
induced by the assertions with positive polarity;
see Table \ref{tbl:distribution:component:positive:weak}. 
The names of the subgraphs are given by the nodes with total degrees different from $1$.
All such nodes are listed in decreasing order of total degree.
In case of a tie (Figure \ref{fig:weakly-connected:positive:14c}) precedence takes the name of the node
that has larger in-degree (two channel vs.~female catheter).}\label{fig:weakly-connected:positive}
\end{figure}

\subsubsection{Big Weakly Connected Component}
The undirected graph induced by the concepts that appear in the big undirected component 
is composed of $223,679$ nodes and $383,698$ edges.
For information about shortest paths in this component please see Chapter \ref{chapter:shortest-paths}.

\subsubsection{Component of Size $55$}
The component of size $55$ is a star about \emph{medical specialties}.
Concept \dbtext{medical specialty} (171593) has out-degree $54$ and in-degree $0$.
All the other concepts have out-degree equal to $0$ and in-degree equal to $1$.
These $54$ concepts are
\dbtext{concern anesthesia anesthesiology} (171594),
\dbtext{concern bactereia bacteriology} (171595),
\dbtext{concern birth obstetrics} (171596),
\dbtext{concern body function physiology} (171597),
\dbtext{concern body move\-ment kinesiology} (171598),
\dbtext{concern cell cytology} (171599),
\dbtext{concern child pediatrics} (171600),
\dbtext{con\-cern digestive system gastroenterology} (171601),
\dbtext{concern disease cause etiology} (171602),
\dbtext{concern disease classification nosology} (171603),
\dbtext{concern disease identification diagnostic} (171604),
\dbtext{con\-cern ear otology} (171605),
\dbtext{concern epidemic epidemiology} (171606),
\dbtext{concern contagious disease ep\-i\-de\-mi\-ol\-o\-gy} (171607),
\dbtext{concern eye ophthalmology} (171608),
\dbtext{concern gland adenology} (171609),
\dbtext{concern gum periodontics} (171610),
\dbtext{concern hear audiology} (171611),
\dbtext{concern heart cardiology} (171612),
\dbtext{con\-cern hernia herniology} (171613),
\dbtext{concern intestine entrology} (171614),
\dbtext{concern joint arthrology} (171615),
\dbtext{concern joint rheumatology} (171616),
\dbtext{concern kidney nephrology} (171617),
\dbtext{concern liver hepetology} (171618),
\dbtext{concern liver hepatology} (171619),
\dbtext{concern mental disorder psychiatry} (171620),
\dbtext{concern mouth stomatology} (171621),
\dbtext{concern mouth oralogy} (171622),
\dbtext{concern muscle myology} (171623),
\dbtext{concern muscle orthopedic} (171624),
\dbtext{concern nervous system neurology} (171625),
\dbtext{concern nervous sys\-tem neu\-ro\-pha\-thol\-o\-gy} (171626),
\dbtext{concern newborn neonatology} (171627),
\dbtext{concern nose rhinology} (171628),
\dbtext{concern parasite parasitology} (171629),
\dbtext{concern poison toxicology} (171630),
\dbtext{concern toxin tox\-i\-col\-o\-gy} (171631),
\dbtext{concern rheumatic disease rheumatology} (171632),
\dbtext{concern serum serology} (171633),
\dbtext{concern skin der\-ma\-tol\-o\-gy} (171634),
\dbtext{concern skull craniology} (171635),
\dbtext{concern stomach gastrology} (171636),
\dbtext{concern symptom symptomology} (171638),
\dbtext{concern tissue histology} (171642),
\dbtext{concern tumor oncology} (171643),
\dbtext{concern ulcer helcology} (171644),
\dbtext{concern vein phlebology} (171645),
\dbtext{concern virus virology} (171646),
\dbtext{concern x-ray radiology} (171647),
\dbtext{concern radiation therapy radiology} (171648),
\dbtext{concern dentistry tooth} (325385),
\dbtext{concern tooth straighten orthodontics} (325386),
and
\dbtext{concern tooth dentistry} (325387).

\subsubsection{Component of Size $32$}
The component of size $32$ is about the \emph{sea level of the pacific ocean}.
Concept \dbtext{pacific ocean 0 m} (5019) has out-degree $0$ and in-degree $31$.
All the other concepts have out-degree equal to $1$ and in-degree equal to $0$.
These $31$ concepts are
\dbtext{low point american samoa} (5018),
\dbtext{low point baker island} (5044),
\dbtext{low point chile} (5082),
\dbtext{low point colombia} (5086),
\dbtext{low point cook island} (5088),
\dbtext{low point costa rica} (5091),
\dbtext{low point ecuador} (5102),
\dbtext{low point el salvador} (5107),
\dbtext{low point fiji} (5117),
\dbtext{low point guam} (5132),
\dbtext{low point guatemala} (5133),
\dbtext{low point jarvis island} (5157),
\dbtext{low point kingman reef} (5163),
\dbtext{low point kiribati} (5164),
\dbtext{low point marshall island} (5187),
\dbtext{low point midway island} (5197),
\dbtext{low point nauru} (5220),
\dbtext{low point new zealand} (5225),
\dbtext{low point nicaragua} (5226),
\dbtext{low point niue} (5231),
\dbtext{low point norfolk island} (5232),
\dbtext{low point palau} (5241),
\dbtext{low point panama} (5243),
\dbtext{low point peru} (5247),
\dbtext{low point samoa} (5266),
\dbtext{low point solomon island} (5281),
\dbtext{low point tokelau} (5306),
\dbtext{low point tonga} (5308),
\dbtext{low point tuvalu} (5314),
\dbtext{low point vanuatu} (5330),
and
\dbtext{low point wake island} (5334).

\subsubsection{Component of Size $31$}
The component of size $31$ is about the \emph{sea level of the atlantic ocean}.
Concept \dbtext{atlantic ocean 0 m} (5022) has out-degree $0$ and in-degree $30$.
All the other concepts have out-degree equal to $1$ and in-degree equal to $0$.
These $30$ concepts are
\dbtext{low point angola} (5021),
\dbtext{low point barbados} (5046),
\dbtext{low point benin} (5055),
\dbtext{low point bermuda} (5056),
\dbtext{low point brazil} (5064),
\dbtext{low point cameroon} (5075),
\dbtext{low point canada} (5076),
\dbtext{low point cape verde} (5078),
\dbtext{low point french guiana} (5119),
\dbtext{low point gabon} (5120),
\dbtext{low point gha\-na} (5125),
\dbtext{low point greenland} (5129),
\dbtext{low point guernsey} (5134),
\dbtext{low point guinea} (5136),
\dbtext{low point guinea-bissau} (5137),
\dbtext{low point guyana} (5138),
\dbtext{low point iceland} (5143),
\dbtext{low point ireland} (5150),
\dbtext{low point jersey} (5158),
\dbtext{low point liberia} (5173),
\dbtext{low point namibia} (5218),
\dbtext{low point nigeria} (5230),
\dbtext{low point portugal} (5253),
\dbtext{low point saint helena} (5263),
\dbtext{low point senegal} (5270),
\dbtext{low point si\-er\-ra leone} (5274),
\dbtext{low point south africa} (5283),
\dbtext{low point spain} (5287),
\dbtext{low point togo} (5305),
and
\dbtext{low point uruguay} (5327).

\subsubsection{Component of Size $30$}
The component of size $30$ is about the endangered \emph{haha plant species}.
Concept \dbtext{haha} (13162) has out-degree $29$ and in-degree $0$.
All the other concepts have out-degree equal to $0$ and in-degree equal to $1$.
These $29$ concepts are
\dbtext{cyanea acuminata} (13163),
\dbtext{cyanea asarifolia} (13164),
\dbtext{cyanea copelandius copelandius} (13165),
\dbtext{cyanea copelandius haleakalaensis} (13166),
\dbtext{cyanea crispa} (13167),
\dbtext{cyanea dunbarius} (13168),
\dbtext{cyanea grimesiana grimesiana} (13172),
\dbtext{cyanea grimesiana obata} (13174),
\dbtext{cyanea hamatiflora hamatiflo\-ra} (13176),
\dbtext{cyanea humboldtiana} (13177),
\dbtext{cyanea koolauensis} (13178),
\dbtext{cyanea lobata} (13179),
\dbtext{cyanea longiflora} (13180),
\dbtext{cyanea mceldowneyi} (13184),
\dbtext{cyanea pinnatifida} (13185),
\dbtext{cyanea platyphylla} (13186),
\dbtext{cyanea procera} (13187),
\dbtext{cyanea recta} (13188),
\dbtext{cyanea remyi} (13189),
\dbtext{cyanea stictophylla} (13192),
\dbtext{cyanea superba} (13193),
\dbtext{cyanea truncata} (13194),
\dbtext{cyanea undulata} (13195),
\dbtext{cyanea glabrum} (311277),
\dbtext{cyanea hamatiflora carlsonie} (311278),
\dbtext{cyanea macrostegia gibsonie} (311279),
\dbtext{cyanea mannie} (311280),
\dbtext{cyanea st-johnie} (311281),
and
\dbtext{cyanea shipmannie} (311282).

\subsubsection{Components of Size $22$}
The first component of size $22$ is about the \emph{sea level of the indian ocean}.
Concept \dbtext{indian ocean 0 m} (5027) has out-degree $0$ and in-degree $21$.
All the other concepts have out-degree equal to $1$ and in-degree equal to $0$.
These $21$ concepts are
\dbtext{low point antarctica} (5026),
\dbtext{low point bangladesh} (5045),
\dbtext{low point christmas island} (5085),
\dbtext{low point comoro} (5087),
\dbtext{low point europa island} (5115),
\dbtext{low point glorioso island} (5127),
\dbtext{low point india} (5144),
\dbtext{low point indonesia} (5147),
\dbtext{low point kenya} (5162),
\dbtext{low point mad\-a\-gas\-car} (5181),
\dbtext{low point malaysia} (5182),
\dbtext{low point maldive} (5183),
\dbtext{low point mauritius} (5191),
\dbtext{low point mayotte} (5193),
\dbtext{low point mozambique} (5216),
\dbtext{low point pakistan} (5240),
\dbtext{low point reunion} (5256),
\dbtext{low point seychelles} (5273),
\dbtext{low point somalia} (5282),
\dbtext{low point sri lanka} (5288),
and
\dbtext{low point tanzania} (5303).

The second component of size $22$ is about \emph{space shuttle acronyms}.
Concept \dbtext{space shuttle acronym} (172559) has out-degree $21$ and in-degree $0$.
All the other concepts have out-degree equal to $0$ and in-degree equal to $1$.
These $21$ concepts are
\dbtext{adi attitude direction indicator} (172560),
\dbtext{apu auxiliary powewr unit} (172561),
\dbtext{css control stick steer} (172562),
\dbtext{dcm display control module} (172563),
\dbtext{eva extravehicular activity} (172564),
\dbtext{hsus horizontal situation indicator} (172565),
\dbtext{iva intravehicu\-lar activity} (172566),
\dbtext{lcc launch control center} (172567),
\dbtext{lo loss signal} (172568),
\dbtext{mcc mission con\-trol center} (172569),
\dbtext{meet mission elapse time} (172570),
\dbtext{mlp mobile launch platform} (172571),
\dbtext{mmu man maneuver unit} (172572),
\dbtext{om orbital maneuver system} (172573),
\dbtext{pam payload assist module} (172574),
\dbtext{plss portable life support system} (172575),
\dbtext{rc reaction control system} (172576),
\dbtext{rm remote ma\-nip\-u\-lator system} (172577),
\dbtext{srb sol\-id rocket booster} (172578),
\dbtext{tp thermal protection system} (172579),
and
\dbtext{wc waste collection system} (172580).

\subsubsection{Component of Size $18$}
The component of size $18$ is about the \emph{sea level of the caribbean sea}.
Concept \dbtext{caribbean sea 0 m} (5024) has out-degree $0$ and in-degree $17$.
All the other concepts have out-degree equal to $1$ and in-degree equal to $0$.
These $17$ concepts are
\dbtext{low point anguilla} (5023),
\dbtext{low point aruba} (5033),
\dbtext{low point belize} (5052),
\dbtext{low point cayman island} (5079),
\dbtext{low point cuba} (5093),
\dbtext{low point dominica} (5101),
\dbtext{low point grenada} (5130),
\dbtext{low point guadeloupe} (5131),
\dbtext{low point haiti} (5139),
\dbtext{low point honduras} (5140),
\dbtext{low point jamaica} (5154),
\dbtext{low point martinique} (5190),
\dbtext{low point montserrat} (5213),
\dbtext{low point puerto rico} (5254),
\dbtext{low point saint lucia} (5264),
\dbtext{low point venezuela} (5331),
and
\dbtext{low point virgin island} (5333).

\subsubsection{Component of Size $16$}
The component of size $16$ is about \emph{saying things in a safe way}.
Concept \dbtext{another say safe} (163403) has out-degree $15$ and in-degree $0$.
All the other concepts have out-degree equal to $0$ and in-degree equal to $1$.
These $15$ concepts are
\dbtext{say perfectly safe} (324626),
\dbtext{say absolutely safe} (324627),
\dbtext{say really safe} (324628),
\dbtext{say truly safe} (324629),
\dbtext{say obviously safe} (324630),
\dbtext{say undeniably safe} (324631),
\dbtext{say veritably safe} (324632),
\dbtext{say remarkably safe} (324633),
\dbtext{say notably safe} (324634),
\dbtext{say strikingly safe} (324635),
\dbtext{say markedly safe} (324636),
\dbtext{say eminently safe} (324638),
\dbtext{say greatly safe} (324639),
\dbtext{say vastly safe} (324640),
\dbtext{say hugely safe} (324641).

\subsubsection{Components of Size $14$}
The first component of size $14$ is about the plant species which is known in Hawaii as \emph{alani}.
Concept \dbtext{alani} (12772) has out-degree $13$ and in-degree $0$.
All the other concepts have out-degree equal to $0$ and in-degree equal to $1$.
These $13$ concepts are
\dbtext{melicope adscenden} (12773),
\dbtext{melicope balloui} (12774),
\dbtext{melicope haupuensis} (12775),
\dbtext{melicope lydgatei} (12777),
\dbtext{melicope mucronulata} (12778),
\dbtext{melicope munroi} (12779),
\dbtext{melicope ovali} (12780),
\dbtext{melicope pallida} (12781),
\dbtext{melicope quadrangularis} (12783),
\dbtext{melicope reflexa} (12784),
\dbtext{melicope zahlbruckneri} (12787),
\dbtext{melicope saint-johnie} (311223),
and
\dbtext{melicope knudsenie} (311225).

The second component of size $14$ revolves around differences that \emph{different cultures} have.
Concept \dbtext{different culture} (17023) has out-degree $10$ and in-degree $1$.
Concept \dbtext{different country} (76553) has out-degree $3$ and in-degree $0$.
All the other concepts have out-degree equal to $0$ and in-degree equal to $1$.
These $12$ concepts are
\dbtext{different tradition} (46475),
\dbtext{different idea taste beauty} (72184),
\dbtext{different law} (76554),
\dbtext{different form art} (89948),
\dbtext{different tonal scale} (90582),
\dbtext{different type jewelry} (91407),
\dbtext{different value system} (100103),
\dbtext{different currency money} (117726),
\dbtext{different tradition celebrate birth\-day} (131995),
\dbtext{different concept fairness} (175218),
\dbtext{different custom} (311469),
and
\dbtext{different custom talk} (316209).

The third component of size $14$ has as central notion different \emph{types of catheters}.
Concept \dbtext{type catheter} (169453) has out-degree $11$ and in-degree $0$.
Concept \dbtext{two channel} (169456) has out-degree $0$ and in-degree $2$.
Concept \dbtext{female catheter} (169462) has out-degree $1$ and in-degree $1$.
Concept \dbtext{double-current catheter} (169457) has out-degree $1$ and in-degree $0$.
All the other $10$ concepts have out-degree equal to $0$ and in-degree equal to $1$.
These $10$ concepts are 
\dbtext{uterine catheter} (169454),
\dbtext{cardiac catheter} (169455),
\dbtext{elbow catheter} (169458),
\dbtext{insert through female urethra} (169463),
\dbtext{dilate laryngeal stricture} (169465),
\dbtext{effect blad\-der drainage} (169466),
\dbtext{foley catheter} (169470),
\dbtext{itard catheter} (169471),
\dbtext{bozeman catheter} (325202),
and
\dbtext{mercy catheter} (325203).

The fourth component of size $14$ is about \emph{rnum virus}.
Concept \dbtext{rnum virus} (226834) has out-degree $0$ and in-degree $13$.
All the other $13$ concepts have out-degree equal to $1$ and in-degree equal to $0$.
These $13$ concepts are 
\dbtext{retrovirid} (226833),
\dbtext{arenavirid} (226835),
\dbtext{picornavirid} (226836),
\dbtext{calicivirid} (226837),
\dbtext{bunyavirid} (226838),
\dbtext{orthomyxovirid} (226839),
\dbtext{paramyxovirid} (226840),
\dbtext{rhabdovirid} (226841),
\dbtext{pilovirid} (226842),
\dbtext{togavirid} (226843),
\dbtext{flavivirid} (226844),
\dbtext{coronavirid} (226845),
and
\dbtext{reovirid} (226887).

\subsubsection{Component of Size $12$}
The component of size $12$ is about the notion of \emph{dirge}.
Concept \dbtext{dirge} (173532) has out-degree $11$ and in-degree $0$.
All the other concepts have out-degree equal to $0$ and in-degree equal to $1$.
These $11$ concepts are
\dbtext{slow mournful piece music} (173533),
\dbtext{hymn lamentation grief} (173534),
\dbtext{accompany funeral} (173535),
\dbtext{accompany memorial rite} (173536),
\dbtext{any slow solemn piece music} (173537),
\dbtext{death melody} (357753),
\dbtext{fu\-ne\-ral march} (361199),
\dbtext{funeral music} (361200),
\dbtext{funeral song} (361202),
\dbtext{mournful song} (368168),
\dbtext{death song} (384749).

\subsubsection{Components of Size $11$}
The first component of size $11$
is about \emph{dark regions on Mars}
\footnote{ See also \url{http://conceptnet5.media.mit.edu/web/c/en/darkish_region_on_mar}.}.
Concept \dbtext{darkish region mar} (106353) has out-degree $0$ and in-degree $9$.
Concept \dbtext{margaritifer sinus} (106363) has out-degree $2$ and in-degree $0$.
Concept \dbtext{darkish area mar} (106364) has out-degree $0$ and in-degree $1$.
All the other concepts have out-degree equal to $1$ and in-degree equal to $0$.
These $8$ concepts are
\dbtext{nilokera} (106352),
\dbtext{iapygia} (106354),
\dbtext{mare hadriaticum} (106355),
\dbtext{hellespontu} (106356),
\dbtext{propoutis} (106371),
\dbtext{noctis lacus} (106374),
\dbtext{tithonius lacus} (106379),
and
\dbtext{chrysokera} (106380).

In the second component of size $11$ the concept with the highest degree
is \emph{hydrogen peroxide}.
Concept \dbtext{hydrogen peroxide} (122047) has out-degree $6$ and in-degree $0$.
Concept \dbtext{h2o2} (122044) has out-degree $3$ and in-degree $1$.
Concept \dbtext{powerful oxidizer} (122048) has out-degree $0$ and in-degree $2$.
Concept \dbtext{chemical formula hydrogen peroxide} (122043)
has out-degree $1$ and in-degree $0$.
All the other concepts have out-degree equal to $0$ and in-degree equal to $1$.
These $7$ concepts are
\dbtext{natural metabolite many organism} (122052),
\dbtext{miscible water} (122053),
\dbtext{deodorize bleach agent} (122058),
\dbtext{clear colorless} (122059),
\dbtext{characteristic pungent odor} (122060),
\dbtext{sell water solution} (122061),
and
\dbtext{mild disinfectant} (122063).

The third component of size $11$ is about \emph{sulfa drug}.
Concept \dbtext{sulfa drug} (171559) has out-degree $10$ and in-degree $0$.
All the other concepts have out-degree equal to $0$ and in-degree equal to $1$.
These $10$ concepts are
\dbtext{derive sulfanilamide} (171560),
\dbtext{treat infection} (171561),
\dbtext{treat conjunctivitis} (171562),
\dbtext{treat bronchitis} (171563),
\dbtext{treat leprosy} (171564),
\dbtext{treat malaria} (171565),
\dbtext{treat dysentery} (171566),
\dbtext{treat gastroenteritis} (171567),
\dbtext{treat urinary infection} (171568),
and
\dbtext{prevent growth bacterium} (171569).

\suppressfloats[b]

\subsection{Strongly Connected Components}
We have $265,696$ strongly connected components, out of which
$265,596$ are isolated vertices. Among the rest $100$ components we can find
one component of size $13,700$, three components of size $3$,
and ninety six components of size $2$.

\medskip

Note that the numbers presented here for strongly connected components
refer to the case of the directed graph only since in the undirected case
we have the notion of connected components which is the same as the weakly
connected components of the directed graph.
Those were presented earlier.

\paragraph{Distribution of Component Sizes.} 
The distribution of the sizes for the various components is shown in 
Table \ref{tbl:distribution:component:positive:strong}.

\begin{table}[h]
\caption{Distribution of sizes for strongly connected components
for the induced directed graph.}\label{tbl:distribution:component:positive:strong}
\begin{center}
\begin{tabular}{|r||c|c|c|c|}\hline
\# of nodes per component & $13,700$ & $3$ &  $2$ &       $1$ \\\hline
\# of components          &      $1$ & $3$ & $96$ & $265,596$ \\\hline
\end{tabular}
\end{center}
\end{table}

\subsubsection{Big Strongly Connected Component}
Regarding the big strongly connected component with the $13,700$ nodes, it has $120,865$ edges
(self-loops were omitted from the enumeration).
Hence the average degree is about $17.64453$ after self-loops have been discarded.
Regarding the induced undirected graph that occurs after restricting ourselves
in these $13,700$ nodes (again, self-loops are omitted), the number of edges
is $109,378$. In other words, the average degree in this case is about
$15.96759$. The transitivity and the clustering coefficient of the big component are
presented in Table \ref{tbl:transitivity:positive:BigDC}.

\begin{table}[ht]
\caption{Transitivity and clustering coefficient for the big directed component of \conceptnet.
The first value (\textsc{nan}) for the clustering coefficient gives the result of the calculation
when vertices with less than two neighbors are left out from the calculation,
while the second value (\textsc{zero}) gives the result of the calculation when 
vertices with less than two neighbors are considered as having zero transitivity.
Note that all values are the same both for directed as well as undirected graphs.}\label{tbl:transitivity:positive:BigDC}
\begin{center}
\begin{tabular}{|r|l|}\hline
Transitivity                           & $0.045365818173714129$ \\\hline
Clustering Coefficient (\textsc{nan})  & $0.219425693644797526$ \\\hline
Clustering Coefficient (\textsc{zero}) & $0.195080653182017061$ \\\hline
\end{tabular}
\end{center}
\end{table}

For information about shortest paths in this component please see Chapter \ref{chapter:shortest-paths}.

\begin{figure}[ht]
\begin{center}
\includegraphics[width=0.7\textwidth]{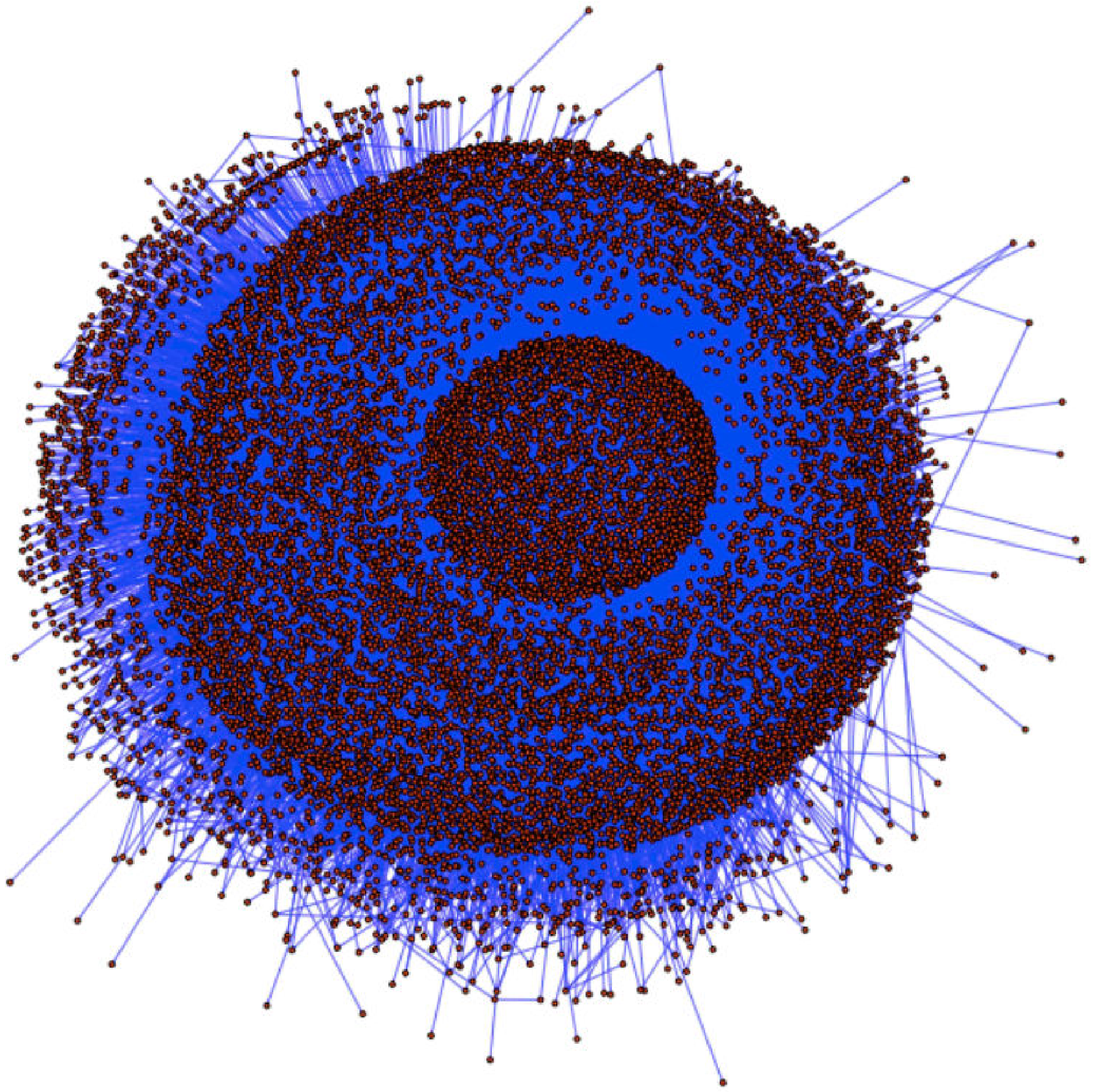}
\end{center}
\caption{The maximal strongly connected component; see Table \ref{tbl:distribution:component:positive:strong}.
For simplicity we plot the induced undirected graph of that component (in low resolution).}
\end{figure}

\subsubsection{Components of Size 3}
The first strongly connected component of size $3$ is composed of the concepts
\dbtext{first floor} (1598), \dbtext{second floor} (9162), and \dbtext{third floor} (141542).

The second strongly connected component of size $3$ is composed of the concepts
\dbtext{primary color} (9707), 
\dbtext{red yellow blue} (15197), and
\dbtext{three primary color} (32853).

The third strongly connected component of size $3$ is composed of the concepts
\dbtext{capital unite state} (3370),
\dbtext{washington dc} (3371), and
\dbtext{washington d.c} (5028).

\begin{figure}[ht]
\begin{center}
\includegraphics[width=0.23\textwidth]{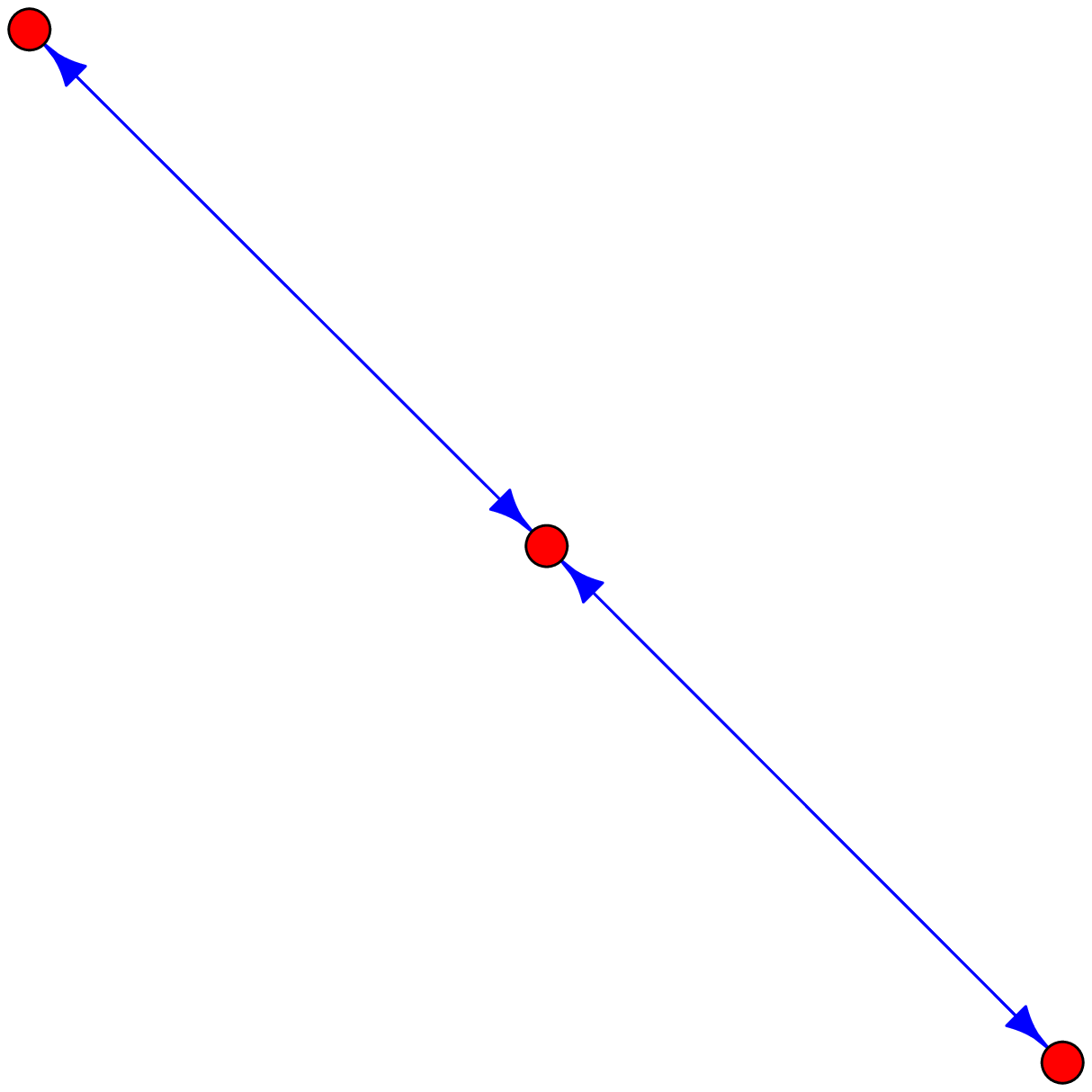}
\end{center}
\caption{The three strongly connected components of size 3 look identical 
(self-loops have been neglected).}\label{fig:strongly-connected:positive:3}
\end{figure}

\section{Both Polarities}
In this section we examine the weakly and strongly connected components of the
directed graph induced by the assertions with both polarities; that is, both negative and positive.

\subsection{Weakly Connected Components}
We get $32,702$ weakly connected components, out of which
$16,922$ are isolated vertices. Note that $16,922$ is in complete agreement with
Table \ref{tbl:number-of-edges-isolated-vertices:overall}. 
Among the rest $15,780$ components we can find
components with cardinalities between $2$ and $228,784$.

\paragraph{Distribution of Component Sizes.}
The distribution of the sizes for the various components is shown in 
Table \ref{tbl:distribution:component:weak}.
This distribution presents the cardinalities of the weakly connected components
of the induced directed graph, as well as the cardinalities of the connected
components of the induced undirected graph.
For the induced graphs we consider assertions with positive score in the English language
and we allow all frequencies in the edges.

\begin{table}[ht]
\caption{Distribution of sizes for weakly connected components for the induced directed 
graph. This is also the distribution of sizes for the connected 
components of the induced undirected graph. 
}\label{tbl:distribution:component:weak}
\begin{center}
\resizebox{\textwidth}{!}{
\begin{tabular}{|r||c|c|c|c|c|c|c|c|c|c|c|c|c|c|c|c|c|c|c|c|c|}\hline
\# of nodes      & $228,784$ & $55$ & $32$ & $31$ & $30$ & $22$ & $18$ & $16$ & $14$ & $12$ & $11$ & $10$ & $9$ & $8$  & $7$  & $6$  & $5$  & $4$   & $3$   & $2$      & $1$     \\\hline
\# of components & $1$       & $1$  & $1$  & $1$  & $1$  & $2$  & $1$  & $1$  & $4$  & $1$  & $3$  & $2$  & $5$ & $11$ & $16$ & $27$ & $85$ & $204$ & $970$ & $14,443$ & $16,922$ \\\hline
\end{tabular}
}
\end{center}
\end{table}

\subsubsection{Big Weakly Connected Component}
The undirected graph induced by the concepts that appear in the big undirected component 
is composed of $228,784$ nodes and $394,554$ edges.
For information about shortest paths in this component please see Chapter \ref{chapter:shortest-paths}.

\subsubsection{Components of Sizes $11$-$55$}
The weakly connected components of sizes $11$-$55$
are precisely the same as those mentioned as weakly connected components
that arise in the directed graph induced by the assertions with positive polarity only.

\suppressfloats[b]

\subsection{Strongly Connected Components}
We have $265,374$ strongly connected components, out of which
$265,276$ are isolated vertices. Among the rest $98$ components we can find
one component of size $14,025$, two components of size $3$,
and ninety five components of size $2$.

\medskip

Note that the numbers presented here for strongly connected components
refer to the case of the directed graph only since in the undirected case
we have the notion of connected components which is the same as the weakly
connected components of the directed graph and which were presented earlier.

\paragraph{Distribution of Component Sizes.} 
The distribution of the sizes for the various components is shown in 
Table \ref{tbl:distribution:component:strong}.

\begin{table}[ht]
\caption{Distribution of sizes for strongly connected components
for the induced directed graph.}\label{tbl:distribution:component:strong}
\begin{center}
\begin{tabular}{|r||c|c|c|c|}\hline
\# of nodes per component & $14,025$ & $3$ &  $2$ &       $1$ \\\hline
\# of components          &      $1$ & $2$ & $95$ & $265,276$ \\\hline
\end{tabular}
\end{center}
\end{table}

\subsubsection{Big Strongly Connected Component}
Regarding the big strongly connected component with the $14,025$ nodes, it has $126,151$ edges
(self-loops were omitted from the enumeration).
Hence the average degree is about $17.98945$ after self-loops have been discarded.
Regarding the induced undirected graph that occurs after restricting ourselves
in these $14,025$ nodes (again, self-loops are omitted), the number of edges
is $114,294$. In other words, the average degree in this case is about
$16.29861$. The transitivity and the clustering coefficient of the big component are
presented in Table \ref{tbl:transitivity:BigDC}.

\begin{table}[ht]
\caption{Transitivity and clustering coefficient for the big directed component of \conceptnet.
The first value (\textsc{nan}) for the clustering coefficient gives the result of the calculation
when vertices with less than two neighbors are left out from the calculation,
while the second value (\textsc{zero}) gives the result of the calculation when 
vertices with less than two neighbors are considered as having zero transitivity.
Note that all values are the same both for directed as well as undirected graphs.}\label{tbl:transitivity:BigDC}
\begin{center}
\begin{tabular}{|r|l|}\hline
Transitivity                           & $0.042730645545158707$ \\\hline
Clustering Coefficient (\textsc{nan})  & $0.228343346540729242$ \\\hline
Clustering Coefficient (\textsc{zero}) & $0.203807630088901875$ \\\hline
\end{tabular}
\end{center}
\end{table}

For information about shortest paths in this component please see Chapter \ref{chapter:shortest-paths}.

\begin{figure}[ht]
\begin{center}
\includegraphics[width=0.7\textwidth]{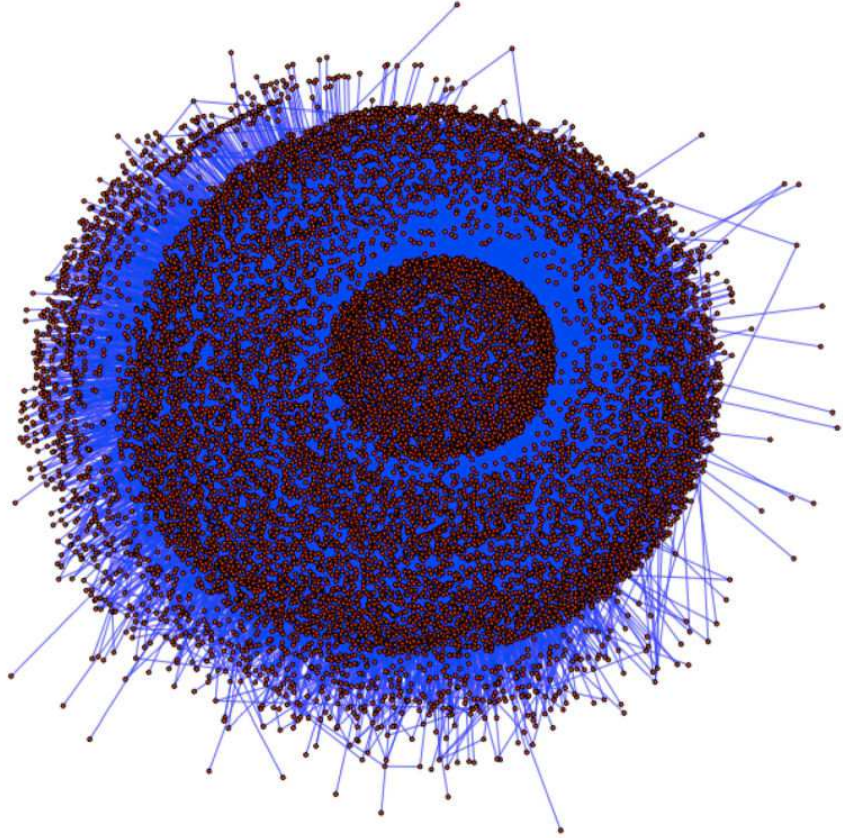}
\end{center}
\caption{The maximal strongly connected component; see Table \ref{tbl:distribution:component:strong}.
For simplicity we plot the induced undirected graph of that component (in low resolution).}
\end{figure}

\subsubsection{Components of Size 3}
The first strongly connected component of size $3$ is composed of the concepts
\dbtext{first floor} (1598), \dbtext{second floor} (9162), \dbtext{third floor} (141542).

The second strongly connected component of size $3$ is composed of the concepts
\dbtext{primary color} (9707), 
\dbtext{red yellow blue} (15197), 
\dbtext{three primary color} (32853).

The figures of these two components of size $3$ of course have not changed from the case
where they appeared as strongly connected components induced by assertions with positive
polarity only. As a reminder, Figure \ref{fig:strongly-connected:positive:3} presents the components.

\chapter{Cores}\label{chapter:cores}
We restrict on edges with positive score and allow all frequencies (that is both positive and negative polarity).
We distinguish three main cases on whether we allow edges with negative only polarity, positive only polarity,
or finally both polarities. 

\section{Negative Polarity}
We distinguish cases based on whether we allow self-loops or not.

\subsection{Loops are Neglected}
Table \ref{tbl:coreness:distribution:negative-polarity:no-loops}
presents the distribution of the vertices with specific coreness in the case where self-loops have been neglected.
Table \ref{tbl:core-decomposition:negative:no-loops} presents the number of vertices with coreness above a certain
threshold, as well as the number of edges and the average degree in every induced graph; 
whether that is a multigraph, a directed graph, or an undirected graph.

\begin{table}[ht]
\caption{Distribution of vertices with specific coreness. We only consider assertions with positive score in the English language.
The polarity is negative. Self-loops are neglected.}\label{tbl:coreness:distribution:negative-polarity:no-loops}
\begin{center}
\begin{tabular}{|r||r|r|r|r|r|r|r|}\hline
coreness &       $0$ &     $1$ &   $2$ &   $3$ &   $4$ &   $5$ &  $6$ \\\hline
vertices & $267,790$ & $9,952$ & $935$ & $473$ & $172$ & $107$ & $68$ \\\hline
\end{tabular}
\end{center}
\end{table}

\begin{table}[ht]
\caption{Number of vertices, edges, and the average degree of the induced subgraphs
in the case where we allow edges with negative polarity only. Self-loops are neglected.}\label{tbl:core-decomposition:negative:no-loops}
\begin{center}
\begin{tabular}{|c||r||r|r||r|r||r|r||}\hline
\multirow{2}{*}{coreness} & \multirow{2}{*}{vertices} & \multicolumn{2}{c||}{directed multigraph} & \multicolumn{2}{c||}{directed graph} & \multicolumn{2}{c||}{undirected graph} \\\cline{3-8}
          &          & \multicolumn{1}{c|}{edges} & avg.~degree & \multicolumn{1}{c|}{edges} & avg.~degree & \multicolumn{1}{c|}{edges} & avg.~degree \\\hline\hline
$\geq  0$ & $279497$ & $13497$ & $ 0.096581$ & $13387$ & $0.095794$ & $12989$ & $0.092946$ \\\hline
$\geq  1$ &  $11707$ & $13497$ & $ 2.305800$ & $13387$ & $2.287008$ & $12989$ & $2.219014$ \\\hline
$\geq  2$ &   $1755$ & $ 4839$ & $ 5.514530$ & $ 4747$ & $5.409687$ & $ 4411$ & $5.026781$ \\\hline
$\geq  3$ &    $820$ & $ 3006$ & $ 7.331707$ & $ 2930$ & $7.146341$ & $ 2710$ & $6.609756$ \\\hline
$\geq  4$ &    $347$ & $ 1593$ & $ 9.181556$ & $ 1540$ & $8.876081$ & $ 1447$ & $8.340058$ \\\hline
$\geq  5$ &    $175$ & $  911$ & $10.411429$ & $  867$ & $9.908571$ & $  819$ & $9.360000$ \\\hline
$\geq  6$ &     $68$ & $  348$ & $10.235294$ & $  331$ & $9.735294$ & $  308$ & $9.058824$ \\\hline
\end{tabular}
\end{center}
\end{table}

The $68$ concepts that we find in the innermost core are
\dbtext{person} (9),
\dbtext{tree} (33),
\dbtext{exercise} (61),
\dbtext{library} (68),
\dbtext{bath} (70),
\dbtext{human} (80),
\dbtext{walk} (97),
\dbtext{drink} (120),
\dbtext{examination} (121),
\dbtext{fun} (134),
\dbtext{bed} (156),
\dbtext{park} (365),
\dbtext{talk} (394),
\dbtext{eat} (432),
\dbtext{computer} (467),
\dbtext{car} (529),
\dbtext{dog} (537),
\dbtext{music} (542),
\dbtext{cat} (616),
\dbtext{house} (652),
\dbtext{fish} (655),
\dbtext{plant} (716),
\dbtext{animal} (902),
\dbtext{bird} (962),
\dbtext{drive car} (1005),
\dbtext{desk} (1043),
\dbtext{office} (1044),
\dbtext{home} (1045),
\dbtext{kitchen} (1078),
\dbtext{eye} (1160),
\dbtext{die} (1227),
\dbtext{money} (1240),
\dbtext{mouse} (1284),
\dbtext{television} (1298),
\dbtext{food} (1359),
\dbtext{horse} (1412),
\dbtext{hot} (1438),
\dbtext{read} (1456),
\dbtext{drive} (1545),
\dbtext{potato} (1674),
\dbtext{telephone} (1790),
\dbtext{audience} (1816),
\dbtext{rain} (1856),
\dbtext{book} (2033),
\dbtext{boat} (2389),
\dbtext{time} (2494),
\dbtext{fire} (2895),
\dbtext{god} (4277),
\dbtext{space} (4435),
\dbtext{cabinet} (5663),
\dbtext{table} (5665),
\dbtext{long hair} (5916),
\dbtext{metal} (6491),
\dbtext{way} (6679),
\dbtext{competitive activity} (7019),
\dbtext{ear} (8314),
\dbtext{gasoline} (8502),
\dbtext{fly} (9215),
\dbtext{program language} (13345),
\dbtext{gerbil} (14223),
\dbtext{software} (17383),
\dbtext{brain} (17555),
\dbtext{cash register} (23016),
\dbtext{conscious} (23506),
\dbtext{singular} (33174),
\dbtext{transportation device} (200905),
\dbtext{speedo} (203600),
and
\dbtext{fidelity} (203658).

\subsection{Loops are Retained}
Table \ref{tbl:coreness:distribution:negative-polarity:loops}
presents the distribution of the vertices with specific coreness in the case where self-loops are retained.
Table \ref{tbl:core-decomposition:negative:loops} presents the number of vertices with coreness above a certain
threshold, as well as the number of edges and the average degree in every induced graph; 
whether that is a multigraph, a directed graph, or an undirected graph.

\begin{table}[ht]
\caption{Distribution of vertices with specific coreness. We only consider assertions with positive score in the English language.
The polarity is negative. Self-loops are retained.}\label{tbl:coreness:distribution:negative-polarity:loops}
\begin{center}
\begin{tabular}{|r||r|r|r|r|r|r|r|}\hline
coreness &       $0$ &     $1$ &   $2$ &   $3$ &   $4$ &   $5$ &  $6$ \\\hline
vertices & $267,790$ & $9,949$ & $934$ & $477$ & $170$ &  $91$ & $86$ \\\hline
\end{tabular}
\end{center}
\end{table}

\begin{table}[ht]
\caption{Number of vertices, edges, and the average degree of the induced subgraphs
in the case where we allow edges with negative polarity only. Self-loops are retained.}\label{tbl:core-decomposition:negative:loops}
\begin{center}
\begin{tabular}{|c||r||r|r||r|r||r|r||}\hline
\multirow{2}{*}{coreness} & \multirow{2}{*}{vertices} & \multicolumn{2}{c||}{directed multigraph} & \multicolumn{2}{c||}{directed graph} & \multicolumn{2}{c||}{undirected graph} \\\cline{3-8}
          &          & \multicolumn{1}{c|}{edges} & avg.~degree & \multicolumn{1}{c|}{edges} & avg.~degree & \multicolumn{1}{c|}{edges} & avg.~degree \\\hline\hline
$\geq  0$ & $279497$ & $13510$ & $ 0.096674$ & $13399$ & $0.095879$ & $13001$ & $0.093031$ \\\hline
$\geq  1$ & $ 11707$ & $13510$ & $ 2.308021$ & $13399$ & $2.289058$ & $13001$ & $2.221064$ \\\hline
$\geq  2$ & $  1758$ & $ 4853$ & $ 5.521047$ & $ 4760$ & $5.415245$ & $ 4424$ & $5.032992$ \\\hline
$\geq  3$ & $   824$ & $ 3025$ & $ 7.342233$ & $ 2948$ & $7.155340$ & $ 2727$ & $6.618932$ \\\hline
$\geq  4$ & $   347$ & $ 1601$ & $ 9.227666$ & $ 1547$ & $8.916427$ & $ 1454$ & $8.380403$ \\\hline
$\geq  5$ & $   177$ & $  926$ & $10.463277$ & $  881$ & $9.954802$ & $  833$ & $9.412429$ \\\hline
$\geq  6$ & $    86$ & $  457$ & $10.627907$ & $  428$ & $9.953488$ & $  401$ & $9.325581$ \\\hline
\end{tabular}
\end{center}
\end{table}

In both cases the maximum coreness is equal to $6$. The core in this case contains all the concepts mentioned
earlier (case where self-loops were neglected), as well as the concepts
\dbtext{man} (7),
\dbtext{work} (35),
\dbtext{it} (137),
\dbtext{child} (178),
\dbtext{rest} (310),
\dbtext{housework} (343),
\dbtext{sleep} (425),
\dbtext{drawer} (495),
\dbtext{baby} (678),
\dbtext{water} (1016),
\dbtext{see} (1161),
\dbtext{speak} (1305),
\dbtext{lie} (1395),
\dbtext{write} (1893),
\dbtext{wet} (2456),
\dbtext{sex} (2825),
\dbtext{wait} (2858),
and
\dbtext{eye up down} (32844).

\section{Positive Polarity}
We distinguish cases based on whether we allow self-loops or not.

\subsection{Loops are Neglected}
Table \ref{tbl:coreness:distribution:positive-polarity:no-loops}
presents the distribution of the vertices with specific coreness in the case where self-loops have been neglected.
Table \ref{tbl:core-decomposition:positive:no-loops} presents the number of vertices with coreness above a certain
threshold, as well as the number of edges and the average degree in every induced graph; 
whether that is a multigraph, a directed graph, or an undirected graph.

\begin{table}[ht]
\caption{Distribution of vertices with specific coreness. We only consider assertions with positive score in the English language.
The polarity is positive. Self-loops are neglected.}\label{tbl:coreness:distribution:positive-polarity:no-loops}
\begin{center}
\resizebox{\textwidth}{!}{
\begin{tabular}{|r||r|r|r|r|r|r|r|r|r|r|r|r|r|r|}\hline
coreness &     0 &      1 &     2 &    3 &    4 &    5 &    6 &    7 &   8 &   9 &  10 &  11 &  12 &  13 \\\hline
vertices & 22651 & 215187 & 19847 & 6948 & 3381 & 2091 & 1488 & 1154 & 867 & 701 & 548 & 474 & 414 & 339 \\\hline
\end{tabular}
}
\end{center}

\begin{center}
\begin{tabular}{|r||r|r|r|r|r|r|r|r|r|r|r|r|r|}\hline
coreness &  14 &  15 &  16 &  17 &  18 &  19 &  20 &  21 &  22 &  23 &  24 &  25 &  26 \\\hline
vertices & 302 & 258 & 230 & 233 & 211 & 166 & 142 & 156 & 195 & 156 & 191 & 298 & 869 \\\hline
\end{tabular}
\end{center}
\end{table}

\begin{table}[ht]
\caption{Number of vertices, edges, and the average degree of the induced subgraphs
in the case where we allow edges with positive polarity only. Self-loops are neglected.}\label{tbl:core-decomposition:positive:no-loops}
\begin{center}
\begin{tabular}{|c||r||r|r||r|r||r|r||}\hline
\multirow{2}{*}{coreness} & \multirow{2}{*}{vertices} & \multicolumn{2}{c||}{directed multigraph} & \multicolumn{2}{c||}{directed graph} & \multicolumn{2}{c||}{undirected graph} \\\cline{3-8}
          &          & \multicolumn{1}{c|}{edges} & avg.~degree & \multicolumn{1}{c|}{edges} & avg.~degree & \multicolumn{1}{c|}{edges} & avg.~degree \\\hline\hline
$\geq  0$ & $279497$ & $478499$ & $ 3.424001$ &  $412956$ & $ 2.954994$ &  $401367$ & $ 2.872067$ \\\hline
$\geq  1$ & $256846$ & $478499$ & $ 3.725960$ &  $412956$ & $ 3.215592$ &  $401367$ & $ 3.125351$ \\\hline
$\geq  2$ & $ 41659$ & $265649$ & $12.753499$ &  $211716$ & $10.164238$ &  $201678$ & $ 9.682326$ \\\hline
$\geq  3$ & $ 21812$ & $220926$ & $20.257290$ &  $172260$ & $15.794975$ &  $162691$ & $14.917568$ \\\hline
$\geq  4$ & $ 14864$ & $196715$ & $26.468649$ &  $151300$ & $20.357912$ &  $142112$ & $19.121636$ \\\hline
$\geq  5$ & $ 11483$ & $180655$ & $31.464774$ &  $137587$ & $23.963598$ &  $128731$ & $22.421144$ \\\hline
$\geq  6$ & $  9392$ & $168135$ & $35.803876$ &  $126962$ & $27.036201$ &  $118389$ & $25.210605$ \\\hline
$\geq  7$ & $  7904$ & $157301$ & $39.802885$ &  $117870$ & $29.825405$ &  $109561$ & $27.722925$ \\\hline
$\geq  8$ & $  6750$ & $147343$ & $43.657185$ &  $109592$ & $32.471704$ &  $101564$ & $30.093037$ \\\hline
$\geq  9$ & $  5883$ & $138622$ & $47.126296$ &  $102469$ & $34.835628$ &  $ 94709$ & $32.197518$ \\\hline
$\geq 10$ & $  5182$ & $130577$ & $50.396372$ &  $ 95907$ & $37.015438$ &  $ 88462$ & $34.142030$ \\\hline
$\geq 11$ & $  4634$ & $123595$ & $53.342685$ &  $ 90225$ & $38.940440$ &  $ 83039$ & $35.839016$ \\\hline
$\geq 12$ & $  4160$ & $116859$ & $56.182212$ &  $ 84803$ & $40.770673$ &  $ 77882$ & $37.443269$ \\\hline
$\geq 13$ & $  3746$ & $110365$ & $58.924186$ &  $ 79617$ & $42.507742$ &  $ 72978$ & $38.963161$ \\\hline
$\geq 14$ & $  3407$ & $104665$ & $61.441151$ &  $ 75015$ & $44.035809$ &  $ 68613$ & $40.277664$ \\\hline
$\geq 15$ & $  3105$ & $ 98979$ & $63.754589$ &  $ 70526$ & $45.427375$ &  $ 64412$ & $41.489211$ \\\hline
$\geq 16$ & $  2847$ & $ 93671$ & $65.803302$ &  $ 66458$ & $46.686336$ &  $ 60583$ & $42.559185$ \\\hline
$\geq 17$ & $  2617$ & $ 88669$ & $67.763852$ &  $ 62580$ & $47.825755$ &  $ 56939$ & $43.514712$ \\\hline
$\geq 18$ & $  2384$ & $ 83213$ & $69.809564$ &  $ 58343$ & $48.945470$ &  $ 53011$ & $44.472315$ \\\hline
$\geq 19$ & $  2173$ & $ 77733$ & $71.544409$ &  $ 54297$ & $49.974229$ &  $ 49265$ & $45.342844$ \\\hline
$\geq 20$ & $  2007$ & $ 73342$ & $73.086198$ &  $ 50929$ & $50.751370$ &  $ 46145$ & $45.984056$ \\\hline
$\geq 21$ & $  1865$ & $ 69363$ & $74.383914$ &  $ 47915$ & $51.383378$ &  $ 43330$ & $46.466488$ \\\hline
$\geq 22$ & $  1709$ & $ 64691$ & $75.706261$ &  $ 44442$ & $52.009362$ &  $ 40099$ & $46.926858$ \\\hline
$\geq 23$ & $  1514$ & $ 58327$ & $77.050198$ &  $ 39828$ & $52.612946$ &  $ 35870$ & $47.384412$ \\\hline
$\geq 24$ & $  1358$ & $ 52859$ & $77.848306$ &  $ 35945$ & $52.938144$ &  $ 32314$ & $47.590574$ \\\hline
$\geq 25$ & $  1167$ & $ 45989$ & $78.815767$ &  $ 30980$ & $53.093402$ &  $ 27810$ & $47.660668$ \\\hline
$\geq 26$ & $   869$ & $ 34394$ & $79.157652$ &  $ 22898$ & $52.699655$ &  $ 20526$ & $47.240506$ \\\hline
\end{tabular}
\end{center}
\end{table}

The $869$ concepts that we find in the innermost core are
\dbtext{something} (5),
\dbtext{man} (7),
\dbtext{person} (9),
\dbtext{type} (11),
\dbtext{train} (19),
\dbtext{town} (21),
\dbtext{rock} (23),
\dbtext{beach} (24),
\dbtext{tree} (33),
\dbtext{work} (35),
\dbtext{write program} (38),
\dbtext{monkey} (42),
\dbtext{soup} (43),
\dbtext{go concert} (44),
\dbtext{hear music} (45),
\dbtext{weasel} (48),
\dbtext{word} (51),
\dbtext{exercise} (61),
\dbtext{pant} (63),
\dbtext{love} (67),
\dbtext{library} (68),
\dbtext{bath} (70),
\dbtext{school} (73),
\dbtext{listen} (75),
\dbtext{kitten} (78),
\dbtext{arm} (79),
\dbtext{human} (80),
\dbtext{go performance} (86),
\dbtext{plane} (89),
\dbtext{class} (93),
\dbtext{take walk} (96),
\dbtext{walk} (97),
\dbtext{entertain} (100),
\dbtext{run marathon} (101),
\dbtext{beaver} (103),
\dbtext{wait line} (106),
\dbtext{attend lecture} (108),
\dbtext{drink} (120),
\dbtext{study} (122),
\dbtext{go walk} (128),
\dbtext{play basketball} (133),
\dbtext{fun} (134),
\dbtext{it} (137),
\dbtext{paper} (149),
\dbtext{bore} (152),
\dbtext{bed} (156),
\dbtext{wait table} (157),
\dbtext{go see film} (159),
\dbtext{go work} (161),
\dbtext{watch tv show} (163),
\dbtext{dirty} (170),
\dbtext{wake up morning} (171),
\dbtext{dream} (172),
\dbtext{shower} (173),
\dbtext{child} (178),
\dbtext{smoke} (188),
\dbtext{chicken} (191),
\dbtext{go fish} (193),
\dbtext{state} (196),
\dbtext{tell story} (199),
\dbtext{surf web} (203),
\dbtext{gym} (206),
\dbtext{play football} (209),
\dbtext{office build} (210),
\dbtext{movie} (213),
\dbtext{wiener dog} (220),
\dbtext{go restaurant} (225),
\dbtext{visit museum} (228),
\dbtext{study subject} (234),
\dbtext{live life} (236),
\dbtext{go sport event} (241),
\dbtext{go play} (242),
\dbtext{sit} (243),
\dbtext{play soccer} (252),
\dbtext{go jog} (260),
\dbtext{take shower} (261),
\dbtext{play ball} (262),
\dbtext{ball} (263),
\dbtext{eat food} (264),
\dbtext{watch movie} (265),
\dbtext{watch film} (269),
\dbtext{stretch} (271),
\dbtext{play frisbee} (274),
\dbtext{go school} (276),
\dbtext{box} (279),
\dbtext{object} (280),
\dbtext{surprise} (289),
\dbtext{paint picture} (291),
\dbtext{mother} (301),
\dbtext{go film} (305),
\dbtext{party} (307),
\dbtext{rest} (310),
\dbtext{listen radio} (311),
\dbtext{coffee} (314),
\dbtext{kiss} (316),
\dbtext{remember} (325),
\dbtext{candle} (327),
\dbtext{housework} (343),
\dbtext{clean} (344),
\dbtext{lunch} (345),
\dbtext{street} (350),
\dbtext{watch tv} (351),
\dbtext{fungus} (354),
\dbtext{attend school} (355),
\dbtext{play tennis} (357),
\dbtext{park} (365),
\dbtext{trouble} (366),
\dbtext{snake} (369),
\dbtext{wood} (370),
\dbtext{comfortable} (371),
\dbtext{play} (372),
\dbtext{take bus} (376),
\dbtext{bus} (377),
\dbtext{conversation} (390),
\dbtext{talk} (394),
\dbtext{take course} (400),
\dbtext{learn} (401),
\dbtext{plan} (408),
\dbtext{think} (412),
\dbtext{go run} (423),
\dbtext{sleep} (425),
\dbtext{hang out bar} (427),
\dbtext{plan vacation} (429),
\dbtext{go see play} (431),
\dbtext{eat} (432),
\dbtext{attend class} (433),
\dbtext{go swim} (442),
\dbtext{bridge} (444),
\dbtext{cloud} (446),
\dbtext{ride bike} (460),
\dbtext{nothing} (466),
\dbtext{computer} (467),
\dbtext{line} (474),
\dbtext{buy} (475),
\dbtext{eat restaurant} (479),
\dbtext{milk} (481),
\dbtext{tv} (483),
\dbtext{stress} (486),
\dbtext{drawer} (495),
\dbtext{storage} (496),
\dbtext{boredom} (519),
\dbtext{ticket} (522),
\dbtext{car} (529),
\dbtext{vehicle} (530),
\dbtext{dog} (537),
\dbtext{music} (542),
\dbtext{zoo} (547),
\dbtext{use television} (560),
\dbtext{dress} (562),
\dbtext{bottle} (565),
\dbtext{live} (580),
\dbtext{one} (581),
\dbtext{turn} (583),
\dbtext{material} (591),
\dbtext{chair} (596),
\dbtext{entertainment} (607),
\dbtext{cat} (616),
\dbtext{hat} (629),
\dbtext{country} (640),
\dbtext{listen music} (642),
\dbtext{enjoyment} (643),
\dbtext{market} (648),
\dbtext{house} (652),
\dbtext{fish} (655),
\dbtext{lake} (660),
\dbtext{baby} (678),
\dbtext{hurt} (686),
\dbtext{hotel} (688),
\dbtext{plant} (716),
\dbtext{game} (732),
\dbtext{hospital} (865),
\dbtext{bank} (867),
\dbtext{hide} (869),
\dbtext{girl} (876),
\dbtext{student} (886),
\dbtext{muscle} (891),
\dbtext{woman} (895),
\dbtext{animal} (902),
\dbtext{church} (904),
\dbtext{cold} (912),
\dbtext{family} (915),
\dbtext{go movie} (920),
\dbtext{moon} (924),
\dbtext{enlightenment} (926),
\dbtext{pet} (933),
\dbtext{cook} (946),
\dbtext{shop} (948),
\dbtext{stand line} (958),
\dbtext{letter} (960),
\dbtext{bird} (962),
\dbtext{attend classical concert} (972),
\dbtext{death} (977),
\dbtext{play sport} (983),
\dbtext{eat dinner} (984),
\dbtext{effort} (1000),
\dbtext{concert} (1001),
\dbtext{drive car} (1005),
\dbtext{bathroom} (1007),
\dbtext{city} (1013),
\dbtext{traveling} (1014),
\dbtext{shark} (1015),
\dbtext{water} (1016),
\dbtext{rosebush} (1031),
\dbtext{yard} (1032),
\dbtext{knowledge} (1040),
\dbtext{desk} (1043),
\dbtext{office} (1044),
\dbtext{home} (1045),
\dbtext{sloth} (1047),
\dbtext{teach} (1052),
\dbtext{bat} (1057),
\dbtext{call} (1061),
\dbtext{couch} (1072),
\dbtext{kitchen} (1078),
\dbtext{lizard} (1084),
\dbtext{laugh joke} (1095),
\dbtext{run} (1102),
\dbtext{build} (1104),
\dbtext{restaurant} (1111),
\dbtext{spoon} (1116),
\dbtext{butter} (1118),
\dbtext{read book} (1121),
\dbtext{education} (1122),
\dbtext{beautiful} (1124),
\dbtext{take note} (1136),
\dbtext{travel} (1143),
\dbtext{key} (1151),
\dbtext{electricity} (1153),
\dbtext{go store} (1157),
\dbtext{eye} (1160),
\dbtext{see} (1161),
\dbtext{story} (1164),
\dbtext{nose} (1171),
\dbtext{smell} (1172),
\dbtext{stand} (1183),
\dbtext{well} (1201),
\dbtext{pen} (1205),
\dbtext{go sleep} (1207),
\dbtext{tire} (1221),
\dbtext{attention} (1224),
\dbtext{die} (1227),
\dbtext{fall asleep} (1234),
\dbtext{money} (1240),
\dbtext{bill} (1245),
\dbtext{snow} (1247),
\dbtext{weather} (1248),
\dbtext{leg} (1252),
\dbtext{everything} (1262),
\dbtext{run errand} (1274),
\dbtext{patience} (1275),
\dbtext{mouse} (1284),
\dbtext{spend money} (1286),
\dbtext{cry} (1291),
\dbtext{pay bill} (1292),
\dbtext{earn money} (1293),
\dbtext{television} (1298),
\dbtext{speak} (1305),
\dbtext{magazine} (1310),
\dbtext{take bath} (1316),
\dbtext{hole} (1318),
\dbtext{nature} (1324),
\dbtext{band} (1330),
\dbtext{bald eagle} (1331),
\dbtext{nest} (1332),
\dbtext{drink water} (1333),
\dbtext{crab} (1334),
\dbtext{paint} (1338),
\dbtext{ficus} (1339),
\dbtext{sea} (1347),
\dbtext{anemone} (1348),
\dbtext{ocean} (1349),
\dbtext{sun} (1353),
\dbtext{sky} (1354),
\dbtext{fatigue} (1357),
\dbtext{food} (1359),
\dbtext{grape} (1366),
\dbtext{take break} (1368),
\dbtext{bedroom} (1372),
\dbtext{hike} (1383),
\dbtext{drink alcohol} (1386),
\dbtext{lie} (1395),
\dbtext{play chess} (1398),
\dbtext{horse} (1412),
\dbtext{store} (1414),
\dbtext{friend} (1429),
\dbtext{hot} (1438),
\dbtext{airport} (1439),
\dbtext{anger} (1441),
\dbtext{sugar} (1446),
\dbtext{grocery store} (1447),
\dbtext{read} (1456),
\dbtext{curiosity} (1460),
\dbtext{basket} (1463),
\dbtext{hold} (1464),
\dbtext{kill} (1466),
\dbtext{pay} (1473),
\dbtext{swim} (1475),
\dbtext{break} (1476),
\dbtext{foot} (1485),
\dbtext{verb} (1490),
\dbtext{refrigerator} (1503),
\dbtext{newspaper} (1506),
\dbtext{rice} (1510),
\dbtext{drive} (1545),
\dbtext{surface} (1550),
\dbtext{liquid} (1551),
\dbtext{meadow} (1558),
\dbtext{camp} (1566),
\dbtext{use computer} (1576),
\dbtext{window} (1577),
\dbtext{oil} (1587),
\dbtext{cover} (1592),
\dbtext{take film} (1595),
\dbtext{plate} (1604),
\dbtext{dinner} (1605),
\dbtext{smile} (1606),
\dbtext{den} (1610),
\dbtext{cow} (1613),
\dbtext{earth} (1633),
\dbtext{garage} (1647),
\dbtext{fiddle} (1652),
\dbtext{we} (1653),
\dbtext{garden} (1660),
\dbtext{wrestle} (1665),
\dbtext{see new} (1666),
\dbtext{dance} (1667),
\dbtext{poop} (1672),
\dbtext{potato} (1674),
\dbtext{fight} (1675),
\dbtext{outside} (1676),
\dbtext{job} (1677),
\dbtext{smart} (1678),
\dbtext{play baseball} (1687),
\dbtext{frog} (1692),
\dbtext{napkin} (1698),
\dbtext{excite} (1704),
\dbtext{light} (1716),
\dbtext{salad} (1720),
\dbtext{fox} (1746),
\dbtext{forest} (1747),
\dbtext{attend rock concert} (1754),
\dbtext{hear news} (1758),
\dbtext{glass} (1776),
\dbtext{cupboard} (1777),
\dbtext{contemplate} (1784),
\dbtext{telephone} (1790),
\dbtext{marmot} (1796),
\dbtext{mountain} (1797),
\dbtext{pain} (1813),
\dbtext{audience} (1816),
\dbtext{salt} (1817),
\dbtext{motel} (1827),
\dbtext{drop} (1846),
\dbtext{bone} (1852),
\dbtext{meat} (1853),
\dbtext{bookstore} (1854),
\dbtext{rain} (1856),
\dbtext{understand} (1858),
\dbtext{body} (1861),
\dbtext{use} (1867),
\dbtext{ferret} (1880),
\dbtext{small dog} (1882),
\dbtext{write} (1893),
\dbtext{cloth} (1903),
\dbtext{factory} (1917),
\dbtext{bottle wine} (1918),
\dbtext{doll} (1931),
\dbtext{stay healthy} (1932),
\dbtext{pencil} (1953),
\dbtext{research} (1978),
\dbtext{learn new} (1983),
\dbtext{wheel} (1995),
\dbtext{lemur} (1998),
\dbtext{sweat} (2002),
\dbtext{name} (2003),
\dbtext{nice} (2028),
\dbtext{book} (2033),
\dbtext{museum} (2036),
\dbtext{pool} (2049),
\dbtext{headache} (2062),
\dbtext{black} (2063),
\dbtext{canada} (2076),
\dbtext{fart} (2079),
\dbtext{instrument} (2086),
\dbtext{read newspaper} (2102),
\dbtext{sport} (2130),
\dbtext{understand better} (2163),
\dbtext{bad} (2226),
\dbtext{show} (2243),
\dbtext{trash} (2260),
\dbtext{can} (2261),
\dbtext{a} (2263),
\dbtext{wind} (2284),
\dbtext{hand} (2300),
\dbtext{write story} (2335),
\dbtext{pee} (2354),
\dbtext{stop} (2358),
\dbtext{picture} (2360),
\dbtext{transportation} (2364),
\dbtext{road} (2368),
\dbtext{fall down} (2369),
\dbtext{seat} (2374),
\dbtext{boat} (2389),
\dbtext{wild} (2391),
\dbtext{practice} (2399),
\dbtext{help} (2410),
\dbtext{clothe} (2415),
\dbtext{dish} (2419),
\dbtext{train station} (2424),
\dbtext{lose} (2426),
\dbtext{war} (2438),
\dbtext{mall} (2447),
\dbtext{close eye} (2449),
\dbtext{wet} (2456),
\dbtext{flower} (2459),
\dbtext{wallet} (2466),
\dbtext{room} (2480),
\dbtext{satisfaction} (2483),
\dbtext{time} (2494),
\dbtext{answer question} (2512),
\dbtext{perform} (2523),
\dbtext{cell} (2535),
\dbtext{small} (2536),
\dbtext{bicycle} (2554),
\dbtext{new york} (2556),
\dbtext{need} (2557),
\dbtext{farm} (2562),
\dbtext{sink} (2563),
\dbtext{pocket} (2566),
\dbtext{everyone} (2589),
\dbtext{go somewhere} (2592),
\dbtext{color} (2611),
\dbtext{white} (2612),
\dbtext{red} (2614),
\dbtext{stone} (2631),
\dbtext{vegetable} (2636),
\dbtext{green} (2637),
\dbtext{life} (2638),
\dbtext{burn} (2644),
\dbtext{sound} (2660),
\dbtext{good} (2666),
\dbtext{play card} (2667),
\dbtext{large} (2771),
\dbtext{shoe} (2790),
\dbtext{go} (2801),
\dbtext{scale} (2817),
\dbtext{sex} (2825),
\dbtext{soft} (2842),
\dbtext{wait} (2858),
\dbtext{buy ticket} (2866),
\dbtext{steak} (2878),
\dbtext{gain knowledge} (2890),
\dbtext{fire} (2895),
\dbtext{news} (2905),
\dbtext{beer} (3052),
\dbtext{interest} (3086),
\dbtext{finger} (3399),
\dbtext{feel} (3404),
\dbtext{knife} (3405),
\dbtext{dangerous} (3439),
\dbtext{sit down} (3442),
\dbtext{marmoset} (3443),
\dbtext{carpet} (3450),
\dbtext{bowl} (3463),
\dbtext{australia} (3494),
\dbtext{ski} (3524),
\dbtext{surf} (3525),
\dbtext{corn} (3531),
\dbtext{fridge} (3535),
\dbtext{soap} (3536),
\dbtext{expensive} (3546),
\dbtext{teacher} (3556),
\dbtext{leave} (3571),
\dbtext{coin} (3573),
\dbtext{number} (3576),
\dbtext{fruit} (3590),
\dbtext{happiness} (3603),
\dbtext{exhaustion} (3605),
\dbtext{sit chair} (3608),
\dbtext{laugh} (3635),
\dbtext{heavy} (3663),
\dbtext{map} (3668),
\dbtext{fork} (3671),
\dbtext{cuba} (3797),
\dbtext{france} (3826),
\dbtext{italy} (3881),
\dbtext{steel} (3907),
\dbtext{piano} (4010),
\dbtext{wall} (4030),
\dbtext{club} (4076),
\dbtext{theatre} (4095),
\dbtext{unite state} (4102),
\dbtext{cup} (4116),
\dbtext{hill} (4124),
\dbtext{square} (4138),
\dbtext{relax} (4187),
\dbtext{apple tree} (4194),
\dbtext{shelf} (4203),
\dbtext{waste time} (4217),
\dbtext{pleasure} (4231),
\dbtext{relaxation} (4254),
\dbtext{god} (4277),
\dbtext{care} (4323),
\dbtext{friend house} (4329),
\dbtext{procreate} (4344),
\dbtext{airplane} (4359),
\dbtext{watch} (4406),
\dbtext{space} (4435),
\dbtext{phone} (4517),
\dbtext{this} (4539),
\dbtext{place} (4570),
\dbtext{radio} (4587),
\dbtext{tool} (4595),
\dbtext{apple} (4596),
\dbtext{mouth} (4628),
\dbtext{funny} (4647),
\dbtext{win} (4676),
\dbtext{go mall} (4699),
\dbtext{bag} (4743),
\dbtext{doctor} (4760),
\dbtext{theater} (4770),
\dbtext{river} (4784),
\dbtext{blue} (4808),
\dbtext{grass} (4815),
\dbtext{cheese} (4844),
\dbtext{mammal} (4850),
\dbtext{bean} (4896),
\dbtext{lot} (4905),
\dbtext{hair} (4957),
\dbtext{flirt} (4969),
\dbtext{pass time} (5077),
\dbtext{make} (5239),
\dbtext{noise} (5363),
\dbtext{measure} (5370),
\dbtext{shape} (5400),
\dbtext{flat} (5450),
\dbtext{utah} (5454),
\dbtext{plastic} (5505),
\dbtext{container} (5516),
\dbtext{climb} (5526),
\dbtext{wash hand} (5539),
\dbtext{go home} (5555),
\dbtext{bar} (5558),
\dbtext{bug} (5563),
\dbtext{view} (5574),
\dbtext{live room} (5581),
\dbtext{toilet} (5616),
\dbtext{love else} (5621),
\dbtext{tooth} (5622),
\dbtext{drunk} (5628),
\dbtext{cabinet} (5663),
\dbtext{table} (5665),
\dbtext{furniture} (5668),
\dbtext{peace} (5670),
\dbtext{lamp} (5671),
\dbtext{pizza} (5708),
\dbtext{sing} (5711),
\dbtext{buy beer} (5734),
\dbtext{dust} (5736),
\dbtext{sand} (5768),
\dbtext{internet} (5811),
\dbtext{kid} (5854),
\dbtext{hall} (5865),
\dbtext{dictionary} (5905),
\dbtext{rise} (5930),
\dbtext{closet} (5967),
\dbtext{boy} (5976),
\dbtext{like} (5989),
\dbtext{date} (5999),
\dbtext{door} (6022),
\dbtext{record} (6029),
\dbtext{find} (6040),
\dbtext{floor} (6062),
\dbtext{song} (6068),
\dbtext{play game} (6081),
\dbtext{meet} (6085),
\dbtext{not} (6150),
\dbtext{activity} (6207),
\dbtext{basement} (6220),
\dbtext{sofa} (6231),
\dbtext{cut} (6250),
\dbtext{page} (6264),
\dbtext{company} (6274),
\dbtext{bite} (6368),
\dbtext{dark} (6376),
\dbtext{science} (6395),
\dbtext{college} (6396),
\dbtext{world} (6404),
\dbtext{air} (6408),
\dbtext{sheep} (6424),
\dbtext{statue} (6436),
\dbtext{metal} (6491),
\dbtext{jog} (6511),
\dbtext{open} (6539),
\dbtext{warm} (6561),
\dbtext{quiet} (6583),
\dbtext{big} (6604),
\dbtext{high} (6606),
\dbtext{squirrel} (6609),
\dbtext{alcohol} (6616),
\dbtext{skill} (6644),
\dbtext{hobby} (6671),
\dbtext{birthday} (6705),
\dbtext{university} (6708),
\dbtext{roll} (6734),
\dbtext{tiredness} (6738),
\dbtext{mean} (6744),
\dbtext{communication} (6769),
\dbtext{drink coffee} (6817),
\dbtext{general} (6836),
\dbtext{clock} (6860),
\dbtext{read magazine} (7049),
\dbtext{round} (7057),
\dbtext{good time} (7209),
\dbtext{good health} (7268),
\dbtext{act} (7272),
\dbtext{play hockey} (7283),
\dbtext{heat} (7301),
\dbtext{cool} (7306),
\dbtext{eat ice cream} (7359),
\dbtext{learn language} (7364),
\dbtext{dive} (7367),
\dbtext{skin} (7399),
\dbtext{go zoo} (7405),
\dbtext{go internet} (7420),
\dbtext{art} (7424),
\dbtext{noun} (7478),
\dbtext{top} (7514),
\dbtext{wine} (7522),
\dbtext{jar} (7524),
\dbtext{hard} (7545),
\dbtext{cash} (7584),
\dbtext{put} (7625),
\dbtext{important} (7681),
\dbtext{duck} (7686),
\dbtext{toy} (7701),
\dbtext{ring} (7720),
\dbtext{read child} (7755),
\dbtext{crowd} (7763),
\dbtext{draw} (7764),
\dbtext{edible} (7792),
\dbtext{enjoy yourself} (7798),
\dbtext{wyom} (7836),
\dbtext{see movie} (7891),
\dbtext{thing} (7936),
\dbtext{energy} (7982),
\dbtext{land} (8060),
\dbtext{rug} (8135),
\dbtext{pot} (8213),
\dbtext{kill person} (8251),
\dbtext{emotion} (8261),
\dbtext{little} (8268),
\dbtext{clean house} (8295),
\dbtext{change} (8313),
\dbtext{ear} (8314),
\dbtext{alive} (8379),
\dbtext{bread} (8404),
\dbtext{fit} (8548),
\dbtext{view video} (8571),
\dbtext{play poker} (8588),
\dbtext{excitement} (8614),
\dbtext{field} (8720),
\dbtext{move} (8737),
\dbtext{fly airplane} (8753),
\dbtext{ride horse} (8755),
\dbtext{wave} (8813),
\dbtext{stay bed} (8815),
\dbtext{look} (8821),
\dbtext{voice} (8828),
\dbtext{face} (8835),
\dbtext{lawn} (8860),
\dbtext{event} (8862),
\dbtext{tin} (8891),
\dbtext{happy} (8925),
\dbtext{find information} (8931),
\dbtext{fear} (9006),
\dbtext{oven} (9066),
\dbtext{long} (9087),
\dbtext{go vacation} (9089),
\dbtext{breathe} (9104),
\dbtext{shade} (9151),
\dbtext{carry} (9178),
\dbtext{recreation} (9180),
\dbtext{fly} (9215),
\dbtext{test} (9242),
\dbtext{enjoy} (9244),
\dbtext{hear} (9269),
\dbtext{organization} (9275),
\dbtext{jump} (9278),
\dbtext{ride bicycle} (9319),
\dbtext{egg} (9339),
\dbtext{building} (9384),
\dbtext{bee} (9700),
\dbtext{health} (9745),
\dbtext{communicate} (9747),
\dbtext{business} (9787),
\dbtext{make money} (9788),
\dbtext{become tire} (9805),
\dbtext{action} (9908),
\dbtext{pass} (9934),
\dbtext{fall} (9975),
\dbtext{resturant} (10012),
\dbtext{wash} (10170),
\dbtext{sock} (10193),
\dbtext{bear} (10208),
\dbtext{bell} (10210),
\dbtext{head} (10228),
\dbtext{lose weight} (10298),
\dbtext{jump up down} (10301),
\dbtext{watch television} (10343),
\dbtext{sign} (10388),
\dbtext{count} (10461),
\dbtext{healthy} (10482),
\dbtext{end} (10507),
\dbtext{group} (12400),
\dbtext{know} (13183),
\dbtext{pantry} (13248),
\dbtext{learn subject} (13303),
\dbtext{bullet} (13342),
\dbtext{degree} (13403),
\dbtext{note} (13429),
\dbtext{card} (13442),
\dbtext{supermarket} (13550),
\dbtext{joy} (13641),
\dbtext{stand up} (13725),
\dbtext{machine} (13790),
\dbtext{information} (13861),
\dbtext{read letter} (13879),
\dbtext{lay} (13886),
\dbtext{jump rope} (13894),
\dbtext{gas} (13908),
\dbtext{celebrate} (13996),
\dbtext{roof} (14069),
\dbtext{brown} (14263),
\dbtext{circle} (14472),
\dbtext{cake} (14522),
\dbtext{solid} (15343),
\dbtext{dirt} (15359),
\dbtext{point} (15518),
\dbtext{useful} (15524),
\dbtext{handle} (15706),
\dbtext{adjective} (15912),
\dbtext{alaska} (15970),
\dbtext{michigan} (15975),
\dbtext{maryland} (15980),
\dbtext{maine} (15996),
\dbtext{delaware} (16177),
\dbtext{kansa} (16333),
\dbtext{department} (16725),
\dbtext{be} (16974),
\dbtext{steam} (17055),
\dbtext{pretty} (17204),
\dbtext{sadness} (17314),
\dbtext{bike} (17583),
\dbtext{side} (17836),
\dbtext{decoration} (18070),
\dbtext{watch musician perform} (18250),
\dbtext{stapler} (18341),
\dbtext{motion} (18365),
\dbtext{feel better} (18399),
\dbtext{classroom} (18421),
\dbtext{compete} (18538),
\dbtext{out} (18546),
\dbtext{feel good} (18562),
\dbtext{accident} (18579),
\dbtext{transport} (18619),
\dbtext{stay fit} (18712),
\dbtext{injury} (18717),
\dbtext{ride} (18753),
\dbtext{play piano} (19011),
\dbtext{step} (19524),
\dbtext{apartment} (19557),
\dbtext{part} (19708),
\dbtext{bush} (19864),
\dbtext{course} (19871),
\dbtext{learn world} (19935),
\dbtext{countryside} (19993),
\dbtext{see exhibit} (20008),
\dbtext{power} (20085),
\dbtext{same} (20650),
\dbtext{release energy} (20692),
\dbtext{see art} (20765),
\dbtext{see excite story} (20985),
\dbtext{stage} (21403),
\dbtext{any large city} (21865),
\dbtext{comfort} (22238),
\dbtext{orgasm} (22445),
\dbtext{trip} (22700),
\dbtext{laughter} (22777),
\dbtext{express yourself} (23577),
\dbtext{discover truth} (24279),
\dbtext{edge} (24347),
\dbtext{see favorite show} (24507),
\dbtext{case} (24649),
\dbtext{go party} (24657),
\dbtext{grow} (24688),
\dbtext{competition} (24712),
\dbtext{express information} (24906),
\dbtext{board} (24939),
\dbtext{climb mountain} (24954),
\dbtext{attend meet} (25060),
\dbtext{sunshine} (25192),
\dbtext{fly kite} (25205),
\dbtext{examine} (25210),
\dbtext{race} (25233),
\dbtext{meet friend} (25238),
\dbtext{read news} (25239),
\dbtext{shock} (25396),
\dbtext{flea} (25677),
\dbtext{return work} (25747),
\dbtext{see band} (25769),
\dbtext{visit art gallery} (26118),
\dbtext{earn live} (26632),
\dbtext{punch} (26708),
\dbtext{cool off} (26965),
\dbtext{watch television show} (27279),
\dbtext{socialize} (27285),
\dbtext{skate} (27495),
\dbtext{movement} (27707),
\dbtext{create art} (27886),
\dbtext{crossword puzzle} (28017),
\dbtext{enjoy film} (28066),
\dbtext{go pub} (28343),
\dbtext{feel happy} (28593),
\dbtext{play lacrosse} (28752),
\dbtext{corner} (29067),
\dbtext{socialis} (29314),
\dbtext{away} (29340),
\dbtext{physical activity} (29359),
\dbtext{get} (29712),
\dbtext{short} (30110),
\dbtext{many person} (30864),
\dbtext{outdoors} (30992),
\dbtext{stick} (31425),
\dbtext{singular} (33174),
\dbtext{find house} (33328),
\dbtext{find outside} (34925),
\dbtext{winery} (36809),
\dbtext{branch} (37065),
\dbtext{polish} (38832),
\dbtext{wax} (39314),
\dbtext{make person laugh} (69984),
\dbtext{make friend} (71547),
\dbtext{chat friend} (81516),
\dbtext{meet person} (119411),
\dbtext{meet interest person} (123750),
\dbtext{general term} (172489),
\dbtext{generic} (179027),
\dbtext{ground} (184976),
\dbtext{get drunk} (310177),
\dbtext{eaten} (310995),
\dbtext{friend over} (311108),
\dbtext{get exercise} (311524),
\dbtext{get tire} (311724),
\dbtext{enjoy company friend} (311972),
\dbtext{neighbor house} (312175),
\dbtext{play game friend} (312284),
\dbtext{get physical activity} (312389),
\dbtext{go opus} (312412),
\dbtext{get shape} (312438),
\dbtext{sit quietly} (312805),
\dbtext{do it} (313139),
\dbtext{get fit} (323709),
\dbtext{usually} (328606),
\dbtext{unit} (332537),
\dbtext{generic term} (332695),
\dbtext{teach other person} (427795),
\dbtext{entertain person} (427797),
and
\dbtext{see person play game} (427799).

\subsection{Loops are Retained}
Table \ref{tbl:coreness:distribution:positive-polarity:loops}
presents the distribution of the vertices with specific coreness in the case where self-loops are retained.
Table \ref{tbl:core-decomposition:positive:loops} presents the number of vertices with coreness above a certain
threshold, as well as the number of edges and the average degree in every induced graph; 
whether that is a multigraph, a directed graph, or an undirected graph.

\begin{table}[ht]
\caption{Distribution of vertices with specific coreness. We only consider assertions with positive score in the English language.
The polarity is positive. Self-loops are retained.}\label{tbl:coreness:distribution:positive-polarity:loops}
\begin{center}
\resizebox{\textwidth}{!}{
\begin{tabular}{|r||r|r|r|r|r|r|r|r|r|r|r|r|r|r|}\hline
coreness &     0 &      1 &     2 &    3 &    4 &    5 &    6 &    7 &   8 &   9 &  10 &  11 &  12 &  13 \\\hline
vertices & 22649 & 215183 & 19841 & 6955 & 3383 & 2091 & 1486 & 1156 & 868 & 701 & 548 & 475 & 410 & 339 \\\hline
\end{tabular}
}
\end{center}

\begin{center}
\begin{tabular}{|r||r|r|r|r|r|r|r|r|r|r|r|r|r|}\hline
coreness &  14 &  15 &  16 &  17 &  18 &  19 &  20 &  21 &  22 &  23 &  24 &  25 &  26 \\\hline
vertices & 302 & 261 & 227 & 234 & 213 & 163 & 142 & 149 & 182 & 170 & 172 & 295 & 902 \\\hline
\end{tabular}
\end{center}
\end{table}

\begin{table}[ht]
\caption{Number of vertices, edges, and the average degree of the induced subgraphs
in the case where we allow edges with positive polarity only. Self-loops are retained.}\label{tbl:core-decomposition:positive:loops}
\begin{center}
\begin{tabular}{|c||r||r|r||r|r||r|r||}\hline
\multirow{2}{*}{coreness} & \multirow{2}{*}{vertices} & \multicolumn{2}{c||}{directed multigraph} & \multicolumn{2}{c||}{directed graph} & \multicolumn{2}{c||}{undirected graph} \\\cline{3-8}
          &          & \multicolumn{1}{c|}{edges} & avg.~degree & \multicolumn{1}{c|}{edges} & avg.~degree & \multicolumn{1}{c|}{edges} & avg.~degree \\\hline\hline
$\geq  0$ & $279497$ & $478879$ & $ 3.426720$ &  $413216$ & $ 2.956855$ &  $401627$ & $ 2.873927$ \\\hline
$\geq  1$ & $256848$ & $478879$ & $ 3.728890$ &  $413216$ & $ 3.217592$ &  $401627$ & $ 3.127352$ \\\hline
$\geq  2$ & $ 41665$ & $266033$ & $12.770095$ &  $211980$ & $10.175447$ &  $201942$ & $ 9.693604$ \\\hline
$\geq  3$ & $ 21824$ & $221325$ & $20.282716$ &  $172537$ & $15.811675$ &  $162968$ & $14.934751$ \\\hline
$\geq  4$ & $ 14869$ & $197105$ & $26.512207$ &  $151564$ & $20.386576$ &  $142375$ & $19.150582$ \\\hline
$\geq  5$ & $ 11486$ & $181040$ & $31.523594$ &  $137846$ & $24.002438$ &  $128989$ & $22.460212$ \\\hline
$\geq  6$ & $  9395$ & $168517$ & $35.873763$ &  $127222$ & $27.082916$ &  $118649$ & $25.257903$ \\\hline
$\geq  7$ & $  7909$ & $157701$ & $39.878872$ &  $118143$ & $29.875585$ &  $109834$ & $27.774434$ \\\hline
$\geq  8$ & $  6753$ & $147724$ & $43.750629$ &  $109853$ & $32.534577$ &  $101825$ & $30.156967$ \\\hline
$\geq  9$ & $  5885$ & $138997$ & $47.237723$ &  $102724$ & $34.910450$ &  $ 94964$ & $32.273237$ \\\hline
$\geq 10$ & $  5184$ & $130954$ & $50.522377$ &  $ 96163$ & $37.099923$ &  $ 88718$ & $34.227623$ \\\hline
$\geq 11$ & $  4636$ & $123983$ & $53.487058$ &  $ 90485$ & $39.035807$ &  $ 83297$ & $35.934858$ \\\hline
$\geq 12$ & $  4161$ & $117232$ & $56.347993$ &  $ 85052$ & $40.880558$ &  $ 78130$ & $37.553473$ \\\hline
$\geq 13$ & $  3751$ & $110801$ & $59.078113$ &  $ 79917$ & $42.611037$ &  $ 73273$ & $39.068515$ \\\hline
$\geq 14$ & $  3412$ & $105105$ & $61.609027$ &  $ 75313$ & $44.145955$ &  $ 68909$ & $40.392145$ \\\hline
$\geq 15$ & $  3110$ & $ 99411$ & $63.929904$ &  $ 70829$ & $45.549196$ &  $ 64711$ & $41.614791$ \\\hline
$\geq 16$ & $  2849$ & $ 94086$ & $66.048438$ &  $ 66719$ & $46.836785$ &  $ 60840$ & $42.709723$ \\\hline
$\geq 17$ & $  2622$ & $ 89162$ & $68.010679$ &  $ 62895$ & $47.974828$ &  $ 57244$ & $43.664378$ \\\hline
$\geq 18$ & $  2388$ & $ 83643$ & $70.052764$ &  $ 58639$ & $49.111390$ &  $ 53301$ & $44.640704$ \\\hline
$\geq 19$ & $  2175$ & $ 78126$ & $71.840000$ &  $ 54557$ & $50.167356$ &  $ 49521$ & $45.536552$ \\\hline
$\geq 20$ & $  2012$ & $ 73836$ & $73.395626$ &  $ 51253$ & $50.947316$ &  $ 46459$ & $46.181909$ \\\hline
$\geq 21$ & $  1870$ & $ 69849$ & $74.704813$ &  $ 48240$ & $51.593583$ &  $ 43646$ & $46.680214$ \\\hline
$\geq 22$ & $  1721$ & $ 65389$ & $75.989541$ &  $ 44928$ & $52.211505$ &  $ 40560$ & $47.135386$ \\\hline
$\geq 23$ & $  1539$ & $ 59479$ & $77.295647$ &  $ 40643$ & $52.817414$ &  $ 36627$ & $47.598441$ \\\hline
$\geq 24$ & $  1369$ & $ 53549$ & $78.230825$ &  $ 36419$ & $53.205259$ &  $ 32764$ & $47.865595$ \\\hline
$\geq 25$ & $  1197$ & $ 47392$ & $79.184628$ &  $ 31966$ & $53.410192$ &  $ 28715$ & $47.978279$ \\\hline
$\geq 26$ & $   902$ & $ 35985$ & $79.789357$ &  $ 23980$ & $53.170732$ &  $ 21509$ & $47.691796$ \\\hline
\end{tabular}
\end{center}
\end{table}

In both cases the maximum coreness is equal to $26$. The core in this case contains all the concepts mentioned
earlier (case where self-loops were neglected), as well as the concepts
\dbtext{eat lunch} (969),
\dbtext{buy food} (1068),
\dbtext{eat fast food restaurant} (1407),
\dbtext{football} (1448),
\dbtext{water plant} (1470),
\dbtext{hungry} (1533),
\dbtext{eat breakfast} (1540),
\dbtext{clean room} (1981),
\dbtext{wash clothe} (2121),
\dbtext{suitcase} (2479),
\dbtext{iron} (2587),
\dbtext{idea} (2837),
\dbtext{coat} (4020),
\dbtext{order food} (4424),
\dbtext{eat vegetable} (4895),
\dbtext{touch} (5106),
\dbtext{pray} (5292),
\dbtext{look better} (6191),
\dbtext{wool} (6425),
\dbtext{rabbit} (7815),
\dbtext{clean clothe} (8216),
\dbtext{sneeze} (8538),
\dbtext{analyse} (10415),
\dbtext{taste} (14093),
\dbtext{knit} (14683),
\dbtext{son} (15379),
\dbtext{sense} (18386),
\dbtext{memory} (18563),
\dbtext{inspiration} (18885),
\dbtext{awake} (26369),
\dbtext{butt} (27369),
\dbtext{find truth} (29101),
and
\dbtext{stitch} (50513).

\section{Both Polarities}
We distinguish cases based on whether we allow self-loops or not.

\subsection{Loops are Neglected}
Table \ref{tbl:coreness:distribution:both-polarities:no-loops}
presents the distribution of the vertices with specific coreness in the case where self-loops have been neglected.
Table \ref{tbl:core-decomposition:both:no-loops} presents the number of vertices with coreness above a certain
threshold, as well as the number of edges and the average degree in every induced graph; 
whether that is a multigraph, a directed graph, or an undirected graph.

\begin{table}[ht]
\caption{Distribution of vertices with specific coreness. We only consider assertions with positive score in the English language.
The polarity can be anything. Self-loops are neglected.}\label{tbl:coreness:distribution:both-polarities:no-loops}
\begin{center}
\resizebox{\textwidth}{!}{
\begin{tabular}{|r||r|r|r|r|r|r|r|r|r|r|r|r|r|r|}\hline
coreness &     0 &      1 &     2 &    3 &    4 &    5 &    6 &    7 &   8 &   9 &  10 &  11 &  12 &  13 \\\hline
vertices & 16922 & 219999 & 20265 & 7122 & 3429 & 2151 & 1520 & 1138 & 893 & 713 & 545 & 492 & 416 & 358 \\\hline
\end{tabular}
}
\end{center}

\begin{center}
\begin{tabular}{|r||r|r|r|r|r|r|r|r|r|r|r|r|r|r|}\hline
coreness &  14 &  15 &  16 &  17 &  18 &  19 &  20 &  21 &  22 &  23 &  24 &  25 &  26 &  27 \\\hline
vertices & 293 & 265 & 258 & 219 & 233 & 172 & 140 & 150 & 196 & 133 & 180 & 166 & 424 & 705 \\\hline
\end{tabular}
\end{center}
\end{table}

\begin{table}[ht]
\caption{Number of vertices, edges, and the average degree of the induced subgraphs
in the case where we allow edges with any polarity. Self-loops are neglected.}\label{tbl:core-decomposition:both:no-loops}
\begin{center}
\begin{tabular}{|c||r||r|r||r|r||r|r||}\hline
\multirow{2}{*}{coreness} & \multirow{2}{*}{vertices} & \multicolumn{2}{c||}{directed multigraph} & \multicolumn{2}{c||}{directed graph} & \multicolumn{2}{c||}{undirected graph} \\\cline{3-8}
          &          & \multicolumn{1}{c|}{edges} & avg.~degree & \multicolumn{1}{c|}{edges} & avg.~degree & \multicolumn{1}{c|}{edges} & avg.~degree \\\hline\hline
$\geq  0$ &   279497 & 491996 &    3.520582 & 424525 &    3.037779 & 412569 &    2.952225 \\\hline
$\geq  1$ &   262575 & 491996 &    3.747470 & 424525 &    3.233552 & 412569 &    3.142485 \\\hline
$\geq  2$ &    42576 & 274660 &   12.902104 & 218759 &   10.276165 & 208346 &    9.787016 \\\hline
$\geq  3$ &    22311 & 229111 &   20.537941 & 178510 &   16.001972 & 168552 &   15.109318 \\\hline
$\geq  4$ &    15189 & 204384 &   26.912107 & 157027 &   20.676411 & 147455 &   19.416025 \\\hline
$\geq  5$ &    11760 & 188168 &   32.001361 & 143145 &   24.344388 & 133885 &   22.769558 \\\hline
$\geq  6$ &     9609 & 175298 &   36.486211 & 132209 &   27.517744 & 123250 &   25.653034 \\\hline
$\geq  7$ &     8089 & 164224 &   40.604277 & 122932 &   30.394857 & 114234 &   28.244282 \\\hline
$\geq  8$ &     6951 & 154411 &   44.428428 & 114757 &   33.018846 & 106341 &   30.597324 \\\hline
$\geq  9$ &     6058 & 145511 &   48.039287 & 107443 &   35.471443 &  99289 &   32.779465 \\\hline
$\geq 10$ &     5345 & 137380 &   51.405051 & 100778 &   37.709261 &  92940 &   34.776427 \\\hline
$\geq 11$ &     4800 & 130499 &   54.374583 &  95130 &   39.637500 &  87548 &   36.478333 \\\hline
$\geq 12$ &     4308 & 123564 &   57.364903 &  89514 &   41.557103 &  82188 &   38.155989 \\\hline
$\geq 13$ &     3892 & 116967 &   60.106372 &  84290 &   43.314491 &  77254 &   39.698869 \\\hline
$\geq 14$ &     3534 & 111010 &   62.823995 &  79440 &   44.957555 &  72638 &   41.108093 \\\hline
$\geq 15$ &     3241 & 105484 &   65.093490 &  75108 &   46.348658 &  68570 &   42.314101 \\\hline
$\geq 16$ &     2976 & 100050 &   67.237903 &  70901 &   47.648522 &  64627 &   43.432124 \\\hline
$\geq 17$ &     2718 &  94354 &   69.428992 &  66510 &   48.940397 &  60533 &   44.542311 \\\hline
$\geq 18$ &     2499 &  89258 &   71.434974 &  62555 &   50.064026 &  56844 &   45.493397 \\\hline
$\geq 19$ &     2266 &  83297 &   73.518976 &  58094 &   51.274492 &  52699 &   46.512798 \\\hline
$\geq 20$ &     2094 &  78680 &   75.148042 &  54598 &   52.147087 &  49464 &   47.243553 \\\hline
$\geq 21$ &     1954 &  74761 &   76.520983 &  51624 &   52.839304 &  46693 &   47.792221 \\\hline
$\geq 22$ &     1804 &  70232 &   77.862528 &  48259 &   53.502217 &  43580 &   48.314856 \\\hline
$\geq 23$ &     1608 &  63809 &   79.364428 &  43626 &   54.261194 &  39333 &   48.921642 \\\hline
$\geq 24$ &     1475 &  59301 &   80.408136 &  40348 &   54.709153 &  36302 &   49.223051 \\\hline
$\geq 25$ &     1295 &  52717 &   81.416216 &  35686 &   55.113514 &  32054 &   49.504247 \\\hline
$\geq 26$ &     1129 &  46745 &   82.807795 &  31213 &   55.293180 &  27972 &   49.551816 \\\hline
$\geq 27$ &      705 &  29273 &   83.043972 &  19212 &   54.502128 &  17144 &   48.635461 \\\hline
\end{tabular}
\end{center}
\end{table}

The $705$ concepts that we find in the innermost core are
\dbtext{something} (5),
\dbtext{man} (7),
\dbtext{$\langle$censored f-word$\rangle$} (8),
\dbtext{person} (9),
\dbtext{type} (11),
\dbtext{train} (19),
\dbtext{town} (21),
\dbtext{rock} (23),
\dbtext{beach} (24),
\dbtext{tree} (33),
\dbtext{work} (35),
\dbtext{monkey} (42),
\dbtext{soup} (43),
\dbtext{go concert} (44),
\dbtext{hear music} (45),
\dbtext{weasel} (48),
\dbtext{word} (51),
\dbtext{exercise} (61),
\dbtext{love} (67),
\dbtext{library} (68),
\dbtext{bath} (70),
\dbtext{school} (73),
\dbtext{listen} (75),
\dbtext{kitten} (78),
\dbtext{arm} (79),
\dbtext{human} (80),
\dbtext{go performance} (86),
\dbtext{plane} (89),
\dbtext{class} (93),
\dbtext{take walk} (96),
\dbtext{walk} (97),
\dbtext{entertain} (100),
\dbtext{run marathon} (101),
\dbtext{beaver} (103),
\dbtext{wait line} (106),
\dbtext{attend lecture} (108),
\dbtext{drink} (120),
\dbtext{study} (122),
\dbtext{go walk} (128),
\dbtext{play basketball} (133),
\dbtext{fun} (134),
\dbtext{it} (137),
\dbtext{paper} (149),
\dbtext{bore} (152),
\dbtext{bed} (156),
\dbtext{wait table} (157),
\dbtext{go see film} (159),
\dbtext{go work} (161),
\dbtext{watch tv show} (163),
\dbtext{dirty} (170),
\dbtext{wake up morning} (171),
\dbtext{dream} (172),
\dbtext{shower} (173),
\dbtext{child} (178),
\dbtext{smoke} (188),
\dbtext{chicken} (191),
\dbtext{go fish} (193),
\dbtext{state} (196),
\dbtext{tell story} (199),
\dbtext{surf web} (203),
\dbtext{gym} (206),
\dbtext{play football} (209),
\dbtext{office build} (210),
\dbtext{movie} (213),
\dbtext{wiener dog} (220),
\dbtext{visit museum} (228),
\dbtext{live life} (236),
\dbtext{go play} (242),
\dbtext{sit} (243),
\dbtext{play soccer} (252),
\dbtext{go jog} (260),
\dbtext{take shower} (261),
\dbtext{ball} (263),
\dbtext{watch movie} (265),
\dbtext{watch film} (269),
\dbtext{stretch} (271),
\dbtext{play frisbee} (274),
\dbtext{go school} (276),
\dbtext{box} (279),
\dbtext{object} (280),
\dbtext{surprise} (289),
\dbtext{mother} (301),
\dbtext{go film} (305),
\dbtext{party} (307),
\dbtext{rest} (310),
\dbtext{listen radio} (311),
\dbtext{coffee} (314),
\dbtext{kiss} (316),
\dbtext{remember} (325),
\dbtext{housework} (343),
\dbtext{clean} (344),
\dbtext{lunch} (345),
\dbtext{street} (350),
\dbtext{watch tv} (351),
\dbtext{fungus} (354),
\dbtext{attend school} (355),
\dbtext{play tennis} (357),
\dbtext{park} (365),
\dbtext{trouble} (366),
\dbtext{snake} (369),
\dbtext{wood} (370),
\dbtext{play} (372),
\dbtext{take bus} (376),
\dbtext{bus} (377),
\dbtext{conversation} (390),
\dbtext{talk} (394),
\dbtext{learn} (401),
\dbtext{plan} (408),
\dbtext{think} (412),
\dbtext{go run} (423),
\dbtext{sleep} (425),
\dbtext{hang out bar} (427),
\dbtext{go see play} (431),
\dbtext{eat} (432),
\dbtext{attend class} (433),
\dbtext{bridge} (444),
\dbtext{cloud} (446),
\dbtext{ride bike} (460),
\dbtext{nothing} (466),
\dbtext{computer} (467),
\dbtext{line} (474),
\dbtext{buy} (475),
\dbtext{milk} (481),
\dbtext{tv} (483),
\dbtext{stress} (486),
\dbtext{drawer} (495),
\dbtext{boredom} (519),
\dbtext{ticket} (522),
\dbtext{car} (529),
\dbtext{vehicle} (530),
\dbtext{dog} (537),
\dbtext{music} (542),
\dbtext{zoo} (547),
\dbtext{use television} (560),
\dbtext{dress} (562),
\dbtext{bottle} (565),
\dbtext{live} (580),
\dbtext{one} (581),
\dbtext{turn} (583),
\dbtext{material} (591),
\dbtext{chair} (596),
\dbtext{entertainment} (607),
\dbtext{cat} (616),
\dbtext{hat} (629),
\dbtext{country} (640),
\dbtext{listen music} (642),
\dbtext{enjoyment} (643),
\dbtext{market} (648),
\dbtext{house} (652),
\dbtext{fish} (655),
\dbtext{lake} (660),
\dbtext{baby} (678),
\dbtext{hurt} (686),
\dbtext{hotel} (688),
\dbtext{plant} (716),
\dbtext{game} (732),
\dbtext{hospital} (865),
\dbtext{bank} (867),
\dbtext{girl} (876),
\dbtext{student} (886),
\dbtext{muscle} (891),
\dbtext{woman} (895),
\dbtext{animal} (902),
\dbtext{church} (904),
\dbtext{cold} (912),
\dbtext{family} (915),
\dbtext{go movie} (920),
\dbtext{moon} (924),
\dbtext{pet} (933),
\dbtext{cook} (946),
\dbtext{shop} (948),
\dbtext{stand line} (958),
\dbtext{letter} (960),
\dbtext{bird} (962),
\dbtext{attend classical concert} (972),
\dbtext{death} (977),
\dbtext{play sport} (983),
\dbtext{concert} (1001),
\dbtext{drive car} (1005),
\dbtext{bathroom} (1007),
\dbtext{city} (1013),
\dbtext{traveling} (1014),
\dbtext{water} (1016),
\dbtext{yard} (1032),
\dbtext{knowledge} (1040),
\dbtext{desk} (1043),
\dbtext{office} (1044),
\dbtext{home} (1045),
\dbtext{sloth} (1047),
\dbtext{teach} (1052),
\dbtext{bat} (1057),
\dbtext{call} (1061),
\dbtext{couch} (1072),
\dbtext{kitchen} (1078),
\dbtext{lizard} (1084),
\dbtext{run} (1102),
\dbtext{build} (1104),
\dbtext{restaurant} (1111),
\dbtext{butter} (1118),
\dbtext{read book} (1121),
\dbtext{education} (1122),
\dbtext{beautiful} (1124),
\dbtext{take note} (1136),
\dbtext{travel} (1143),
\dbtext{key} (1151),
\dbtext{electricity} (1153),
\dbtext{go store} (1157),
\dbtext{eye} (1160),
\dbtext{see} (1161),
\dbtext{story} (1164),
\dbtext{nose} (1171),
\dbtext{smell} (1172),
\dbtext{stand} (1183),
\dbtext{pen} (1205),
\dbtext{go sleep} (1207),
\dbtext{tire} (1221),
\dbtext{die} (1227),
\dbtext{fall asleep} (1234),
\dbtext{money} (1240),
\dbtext{bill} (1245),
\dbtext{snow} (1247),
\dbtext{leg} (1252),
\dbtext{everything} (1262),
\dbtext{patience} (1275),
\dbtext{mouse} (1284),
\dbtext{spend money} (1286),
\dbtext{cry} (1291),
\dbtext{television} (1298),
\dbtext{speak} (1305),
\dbtext{magazine} (1310),
\dbtext{hole} (1318),
\dbtext{nature} (1324),
\dbtext{bald eagle} (1331),
\dbtext{nest} (1332),
\dbtext{drink water} (1333),
\dbtext{crab} (1334),
\dbtext{paint} (1338),
\dbtext{ficus} (1339),
\dbtext{sea} (1347),
\dbtext{ocean} (1349),
\dbtext{sun} (1353),
\dbtext{sky} (1354),
\dbtext{fatigue} (1357),
\dbtext{food} (1359),
\dbtext{grape} (1366),
\dbtext{take break} (1368),
\dbtext{bedroom} (1372),
\dbtext{hike} (1383),
\dbtext{lie} (1395),
\dbtext{play chess} (1398),
\dbtext{horse} (1412),
\dbtext{store} (1414),
\dbtext{friend} (1429),
\dbtext{hot} (1438),
\dbtext{airport} (1439),
\dbtext{anger} (1441),
\dbtext{sugar} (1446),
\dbtext{grocery store} (1447),
\dbtext{read} (1456),
\dbtext{curiosity} (1460),
\dbtext{basket} (1463),
\dbtext{hold} (1464),
\dbtext{kill} (1466),
\dbtext{pay} (1473),
\dbtext{swim} (1475),
\dbtext{break} (1476),
\dbtext{foot} (1485),
\dbtext{verb} (1490),
\dbtext{refrigerator} (1503),
\dbtext{newspaper} (1506),
\dbtext{rice} (1510),
\dbtext{drive} (1545),
\dbtext{surface} (1550),
\dbtext{liquid} (1551),
\dbtext{meadow} (1558),
\dbtext{camp} (1566),
\dbtext{use computer} (1576),
\dbtext{window} (1577),
\dbtext{oil} (1587),
\dbtext{cover} (1592),
\dbtext{take film} (1595),
\dbtext{plate} (1604),
\dbtext{dinner} (1605),
\dbtext{smile} (1606),
\dbtext{den} (1610),
\dbtext{cow} (1613),
\dbtext{earth} (1633),
\dbtext{garage} (1647),
\dbtext{we} (1653),
\dbtext{garden} (1660),
\dbtext{see new} (1666),
\dbtext{dance} (1667),
\dbtext{potato} (1674),
\dbtext{fight} (1675),
\dbtext{outside} (1676),
\dbtext{job} (1677),
\dbtext{play baseball} (1687),
\dbtext{napkin} (1698),
\dbtext{light} (1716),
\dbtext{salad} (1720),
\dbtext{fox} (1746),
\dbtext{forest} (1747),
\dbtext{hear news} (1758),
\dbtext{glass} (1776),
\dbtext{cupboard} (1777),
\dbtext{telephone} (1790),
\dbtext{marmot} (1796),
\dbtext{mountain} (1797),
\dbtext{pain} (1813),
\dbtext{audience} (1816),
\dbtext{salt} (1817),
\dbtext{motel} (1827),
\dbtext{drop} (1846),
\dbtext{bone} (1852),
\dbtext{meat} (1853),
\dbtext{bookstore} (1854),
\dbtext{rain} (1856),
\dbtext{understand} (1858),
\dbtext{body} (1861),
\dbtext{use} (1867),
\dbtext{ferret} (1880),
\dbtext{small dog} (1882),
\dbtext{write} (1893),
\dbtext{cloth} (1903),
\dbtext{bottle wine} (1918),
\dbtext{doll} (1931),
\dbtext{pencil} (1953),
\dbtext{research} (1978),
\dbtext{learn new} (1983),
\dbtext{wheel} (1995),
\dbtext{sweat} (2002),
\dbtext{nice} (2028),
\dbtext{book} (2033),
\dbtext{museum} (2036),
\dbtext{headache} (2062),
\dbtext{black} (2063),
\dbtext{canada} (2076),
\dbtext{fart} (2079),
\dbtext{read newspaper} (2102),
\dbtext{sport} (2130),
\dbtext{bad} (2226),
\dbtext{show} (2243),
\dbtext{trash} (2260),
\dbtext{wind} (2284),
\dbtext{hand} (2300),
\dbtext{write story} (2335),
\dbtext{stop} (2358),
\dbtext{picture} (2360),
\dbtext{transportation} (2364),
\dbtext{road} (2368),
\dbtext{fall down} (2369),
\dbtext{seat} (2374),
\dbtext{boat} (2389),
\dbtext{practice} (2399),
\dbtext{help} (2410),
\dbtext{clothe} (2415),
\dbtext{dish} (2419),
\dbtext{train station} (2424),
\dbtext{lose} (2426),
\dbtext{war} (2438),
\dbtext{mall} (2447),
\dbtext{wet} (2456),
\dbtext{flower} (2459),
\dbtext{wallet} (2466),
\dbtext{room} (2480),
\dbtext{time} (2494),
\dbtext{answer question} (2512),
\dbtext{perform} (2523),
\dbtext{cell} (2535),
\dbtext{small} (2536),
\dbtext{bicycle} (2554),
\dbtext{new york} (2556),
\dbtext{need} (2557),
\dbtext{farm} (2562),
\dbtext{pocket} (2566),
\dbtext{everyone} (2589),
\dbtext{go somewhere} (2592),
\dbtext{color} (2611),
\dbtext{white} (2612),
\dbtext{red} (2614),
\dbtext{stone} (2631),
\dbtext{vegetable} (2636),
\dbtext{green} (2637),
\dbtext{life} (2638),
\dbtext{burn} (2644),
\dbtext{sound} (2660),
\dbtext{good} (2666),
\dbtext{play card} (2667),
\dbtext{large} (2771),
\dbtext{shoe} (2790),
\dbtext{go} (2801),
\dbtext{scale} (2817),
\dbtext{sex} (2825),
\dbtext{wait} (2858),
\dbtext{buy ticket} (2866),
\dbtext{steak} (2878),
\dbtext{gain knowledge} (2890),
\dbtext{fire} (2895),
\dbtext{beer} (3052),
\dbtext{interest} (3086),
\dbtext{finger} (3399),
\dbtext{feel} (3404),
\dbtext{knife} (3405),
\dbtext{dangerous} (3439),
\dbtext{sit down} (3442),
\dbtext{carpet} (3450),
\dbtext{bowl} (3463),
\dbtext{australia} (3494),
\dbtext{ski} (3524),
\dbtext{corn} (3531),
\dbtext{fridge} (3535),
\dbtext{soap} (3536),
\dbtext{expensive} (3546),
\dbtext{leave} (3571),
\dbtext{coin} (3573),
\dbtext{number} (3576),
\dbtext{fruit} (3590),
\dbtext{happiness} (3603),
\dbtext{sit chair} (3608),
\dbtext{laugh} (3635),
\dbtext{heavy} (3663),
\dbtext{map} (3668),
\dbtext{philippine} (3998),
\dbtext{wall} (4030),
\dbtext{theatre} (4095),
\dbtext{unite state} (4102),
\dbtext{cup} (4116),
\dbtext{hill} (4124),
\dbtext{square} (4138),
\dbtext{relax} (4187),
\dbtext{apple tree} (4194),
\dbtext{shelf} (4203),
\dbtext{pleasure} (4231),
\dbtext{relaxation} (4254),
\dbtext{god} (4277),
\dbtext{care} (4323),
\dbtext{friend house} (4329),
\dbtext{procreate} (4344),
\dbtext{airplane} (4359),
\dbtext{watch} (4406),
\dbtext{space} (4435),
\dbtext{phone} (4517),
\dbtext{this} (4539),
\dbtext{place} (4570),
\dbtext{radio} (4587),
\dbtext{tool} (4595),
\dbtext{apple} (4596),
\dbtext{mouth} (4628),
\dbtext{win} (4676),
\dbtext{go mall} (4699),
\dbtext{bag} (4743),
\dbtext{doctor} (4760),
\dbtext{theater} (4770),
\dbtext{river} (4784),
\dbtext{blue} (4808),
\dbtext{grass} (4815),
\dbtext{cheese} (4844),
\dbtext{mammal} (4850),
\dbtext{lot} (4905),
\dbtext{hair} (4957),
\dbtext{flirt} (4969),
\dbtext{pass time} (5077),
\dbtext{make} (5239),
\dbtext{noise} (5363),
\dbtext{shape} (5400),
\dbtext{flat} (5450),
\dbtext{utah} (5454),
\dbtext{plastic} (5505),
\dbtext{container} (5516),
\dbtext{climb} (5526),
\dbtext{bar} (5558),
\dbtext{bug} (5563),
\dbtext{live room} (5581),
\dbtext{drunk} (5628),
\dbtext{cabinet} (5663),
\dbtext{table} (5665),
\dbtext{furniture} (5668),
\dbtext{pizza} (5708),
\dbtext{sing} (5711),
\dbtext{dust} (5736),
\dbtext{sand} (5768),
\dbtext{kid} (5854),
\dbtext{hall} (5865),
\dbtext{closet} (5967),
\dbtext{boy} (5976),
\dbtext{like} (5989),
\dbtext{date} (5999),
\dbtext{door} (6022),
\dbtext{record} (6029),
\dbtext{find} (6040),
\dbtext{floor} (6062),
\dbtext{song} (6068),
\dbtext{play game} (6081),
\dbtext{meet} (6085),
\dbtext{not} (6150),
\dbtext{ice cream} (6157),
\dbtext{activity} (6207),
\dbtext{basement} (6220),
\dbtext{storm} (6222),
\dbtext{sofa} (6231),
\dbtext{cut} (6250),
\dbtext{page} (6264),
\dbtext{company} (6274),
\dbtext{dark} (6376),
\dbtext{science} (6395),
\dbtext{college} (6396),
\dbtext{world} (6404),
\dbtext{air} (6408),
\dbtext{statue} (6436),
\dbtext{metal} (6491),
\dbtext{jog} (6511),
\dbtext{open} (6539),
\dbtext{warm} (6561),
\dbtext{big} (6604),
\dbtext{squirrel} (6609),
\dbtext{alcohol} (6616),
\dbtext{skill} (6644),
\dbtext{hobby} (6671),
\dbtext{university} (6708),
\dbtext{roll} (6734),
\dbtext{communication} (6769),
\dbtext{general} (6836),
\dbtext{clock} (6860),
\dbtext{competitive activity} (7019),
\dbtext{read magazine} (7049),
\dbtext{round} (7057),
\dbtext{good time} (7209),
\dbtext{act} (7272),
\dbtext{play hockey} (7283),
\dbtext{heat} (7301),
\dbtext{cool} (7306),
\dbtext{dive} (7367),
\dbtext{go zoo} (7405),
\dbtext{art} (7424),
\dbtext{noun} (7478),
\dbtext{wine} (7522),
\dbtext{jar} (7524),
\dbtext{hard} (7545),
\dbtext{put} (7625),
\dbtext{important} (7681),
\dbtext{duck} (7686),
\dbtext{toy} (7701),
\dbtext{ring} (7720),
\dbtext{crowd} (7763),
\dbtext{draw} (7764),
\dbtext{edible} (7792),
\dbtext{see movie} (7891),
\dbtext{thing} (7936),
\dbtext{energy} (7982),
\dbtext{land} (8060),
\dbtext{rug} (8135),
\dbtext{kill person} (8251),
\dbtext{emotion} (8261),
\dbtext{change} (8313),
\dbtext{ear} (8314),
\dbtext{alive} (8379),
\dbtext{bread} (8404),
\dbtext{fit} (8548),
\dbtext{view video} (8571),
\dbtext{play poker} (8588),
\dbtext{excitement} (8614),
\dbtext{field} (8720),
\dbtext{move} (8737),
\dbtext{fly airplane} (8753),
\dbtext{ride horse} (8755),
\dbtext{wave} (8813),
\dbtext{look} (8821),
\dbtext{voice} (8828),
\dbtext{face} (8835),
\dbtext{happy} (8925),
\dbtext{find information} (8931),
\dbtext{fear} (9006),
\dbtext{oven} (9066),
\dbtext{long} (9087),
\dbtext{go vacation} (9089),
\dbtext{breathe} (9104),
\dbtext{shade} (9151),
\dbtext{carry} (9178),
\dbtext{recreation} (9180),
\dbtext{fly} (9215),
\dbtext{enjoy} (9244),
\dbtext{hear} (9269),
\dbtext{jump} (9278),
\dbtext{ride bicycle} (9319),
\dbtext{egg} (9339),
\dbtext{building} (9384),
\dbtext{bee} (9700),
\dbtext{health} (9745),
\dbtext{communicate} (9747),
\dbtext{business} (9787),
\dbtext{make money} (9788),
\dbtext{action} (9908),
\dbtext{pass} (9934),
\dbtext{fall} (9975),
\dbtext{resturant} (10012),
\dbtext{wash} (10170),
\dbtext{sock} (10193),
\dbtext{bear} (10208),
\dbtext{head} (10228),
\dbtext{jump up down} (10301),
\dbtext{watch television} (10343),
\dbtext{sign} (10388),
\dbtext{count} (10461),
\dbtext{know} (13183),
\dbtext{pantry} (13248),
\dbtext{learn subject} (13303),
\dbtext{degree} (13403),
\dbtext{note} (13429),
\dbtext{card} (13442),
\dbtext{supermarket} (13550),
\dbtext{joy} (13641),
\dbtext{stand up} (13725),
\dbtext{machine} (13790),
\dbtext{information} (13861),
\dbtext{lay} (13886),
\dbtext{jump rope} (13894),
\dbtext{gas} (13908),
\dbtext{celebrate} (13996),
\dbtext{gerbil} (14223),
\dbtext{brown} (14263),
\dbtext{circle} (14472),
\dbtext{cake} (14522),
\dbtext{dirt} (15359),
\dbtext{son} (15379),
\dbtext{adjective} (15912),
\dbtext{michigan} (15975),
\dbtext{maine} (15996),
\dbtext{kansa} (16333),
\dbtext{be} (16974),
\dbtext{steam} (17055),
\dbtext{pretty} (17204),
\dbtext{sadness} (17314),
\dbtext{software} (17383),
\dbtext{decoration} (18070),
\dbtext{watch musician perform} (18250),
\dbtext{stapler} (18341),
\dbtext{motion} (18365),
\dbtext{classroom} (18421),
\dbtext{out} (18546),
\dbtext{feel good} (18562),
\dbtext{accident} (18579),
\dbtext{transport} (18619),
\dbtext{injury} (18717),
\dbtext{ride} (18753),
\dbtext{play piano} (19011),
\dbtext{step} (19524),
\dbtext{apartment} (19557),
\dbtext{part} (19708),
\dbtext{learn world} (19935),
\dbtext{countryside} (19993),
\dbtext{see exhibit} (20008),
\dbtext{same} (20650),
\dbtext{release energy} (20692),
\dbtext{see art} (20765),
\dbtext{stage} (21403),
\dbtext{any large city} (21865),
\dbtext{comfort} (22238),
\dbtext{orgasm} (22445),
\dbtext{trip} (22700),
\dbtext{laughter} (22777),
\dbtext{see favorite show} (24507),
\dbtext{case} (24649),
\dbtext{go party} (24657),
\dbtext{grow} (24688),
\dbtext{competition} (24712),
\dbtext{board} (24939),
\dbtext{climb mountain} (24954),
\dbtext{fly kite} (25205),
\dbtext{examine} (25210),
\dbtext{meet friend} (25238),
\dbtext{visit art gallery} (26118),
\dbtext{cool off} (26965),
\dbtext{watch television show} (27279),
\dbtext{socialize} (27285),
\dbtext{skate} (27495),
\dbtext{movement} (27707),
\dbtext{crossword puzzle} (28017),
\dbtext{enjoy film} (28066),
\dbtext{play lacrosse} (28752),
\dbtext{corner} (29067),
\dbtext{away} (29340),
\dbtext{physical activity} (29359),
\dbtext{get} (29712),
\dbtext{short} (30110),
\dbtext{outdoors} (30992),
\dbtext{stick} (31425),
\dbtext{singular} (33174),
\dbtext{make friend} (71547),
\dbtext{chat friend} (81516),
\dbtext{meet person} (119411),
\dbtext{general term} (172489),
\dbtext{ground} (184976),
\dbtext{eaten} (310995),
\dbtext{friend over} (311108),
\dbtext{get exercise} (311524),
\dbtext{get tire} (311724),
\dbtext{enjoy company friend} (311972),
\dbtext{opus} (311995),
\dbtext{neighbor house} (312175),
\dbtext{play game friend} (312284),
\dbtext{go opus} (312412),
\dbtext{sit quietly} (312805),
\dbtext{usually} (328606),
\dbtext{entertain person} (427797),
and
\dbtext{see person play game} (427799).

\subsection{Loops are Retained}
Table \ref{tbl:coreness:distribution:both-polarities:loops}
presents the distribution of the vertices with specific coreness in the case where self-loops are retained.
Table \ref{tbl:core-decomposition:both:loops} presents the number of vertices with coreness above a certain
threshold, as well as the number of edges and the average degree in every induced graph; 
whether that is a multigraph, a directed graph, or an undirected graph.

\begin{table}[ht]
\caption{Distribution of vertices with specific coreness. We only consider assertions with positive score in the English language.
The polarity can be anything. Self-loops are retained.}\label{tbl:coreness:distribution:both-polarities:loops}
\begin{center}
\resizebox{\textwidth}{!}{
\begin{tabular}{|r||r|r|r|r|r|r|r|r|r|r|r|r|r|r|}\hline
coreness &     0 &      1 &     2 &    3 &    4 &    5 &    6 &    7 &   8 &   9 &  10 &  11 &  12 &  13 \\\hline
vertices & 16920 & 219994 & 20259 & 7130 & 3431 & 2152 & 1517 & 1140 & 895 & 711 & 545 & 495 & 413 & 356 \\\hline
\end{tabular}
}
\end{center}

\begin{center}
\begin{tabular}{|r||r|r|r|r|r|r|r|r|r|r|r|r|r|r|}\hline
coreness &  14 &  15 &  16 &  17 &  18 &  19 &  20 &  21 &  22 &  23 &  24 &  25 &  26 &  27 \\\hline
vertices & 292 & 269 & 256 & 219 & 234 & 173 & 136 & 139 & 201 & 139 & 174 & 166 & 258 & 883 \\\hline
\end{tabular}
\end{center}
\end{table}

\begin{table}[ht]
\caption{Number of vertices, edges, and the average degree of the induced subgraphs
in the case where we allow edges with any polarity. Self-loops are retained.}\label{tbl:core-decomposition:both:loops}
\begin{center}
\begin{tabular}{|c||r||r|r||r|r||r|r||}\hline
\multirow{2}{*}{coreness} & \multirow{2}{*}{vertices} & \multicolumn{2}{c||}{directed multigraph} & \multicolumn{2}{c||}{directed graph} & \multicolumn{2}{c||}{undirected graph} \\\cline{3-8}
          &          & \multicolumn{1}{c|}{edges} & avg.~degree & \multicolumn{1}{c|}{edges} & avg.~degree & \multicolumn{1}{c|}{edges} & avg.~degree \\\hline\hline
$\geq  0$ &   279497 & 492389 &  3.523394 & 424790 &  3.039675 & 412834 &  2.954121 \\\hline
$\geq  1$ &   262577 & 492389 &  3.750435 & 424790 &  3.235546 & 412834 &  3.144480 \\\hline
$\geq  2$ &    42583 & 275058 & 12.918676 & 219029 & 10.287157 & 208616 &  9.798088 \\\hline
$\geq  3$ &    22324 & 229524 & 20.562982 & 178793 & 16.018008 & 168835 & 15.125873 \\\hline
$\geq  4$ &    15194 & 204786 & 26.956167 & 157295 & 20.704884 & 147722 & 19.444781 \\\hline
$\geq  5$ &    11763 & 188565 & 32.060699 & 143408 & 24.382896 & 134147 & 22.808297 \\\hline
$\geq  6$ &     9611 & 175686 & 36.559359 & 132468 & 27.565914 & 123509 & 25.701592 \\\hline
$\geq  7$ &     8094 & 164638 & 40.681492 & 123210 & 30.444774 & 114512 & 28.295528 \\\hline
$\geq  8$ &     6954 & 154806 & 44.522865 & 115022 & 33.080817 & 106606 & 30.660339 \\\hline
$\geq  9$ &     6059 & 145891 & 48.156792 & 107695 & 35.548770 &  99541 & 32.857237 \\\hline
$\geq 10$ &     5348 & 137781 & 51.526178 & 101049 & 37.789454 &  93210 & 34.857891 \\\hline
$\geq 11$ &     4803 & 130911 & 54.512180 &  95405 & 39.727254 &  87820 & 36.568811 \\\hline
$\geq 12$ &     4308 & 123932 & 57.535747 &  89755 & 41.668988 &  82429 & 38.267874 \\\hline
$\geq 13$ &     3895 & 117384 & 60.274198 &  84570 & 43.424904 &  77531 & 39.810526 \\\hline
$\geq 14$ &     3539 & 111471 & 62.995762 &  79745 & 45.066403 &  72940 & 41.220684 \\\hline
$\geq 15$ &     3247 & 105947 & 65.258392 &  75430 & 46.461349 &  68888 & 42.431783 \\\hline
$\geq 16$ &     2978 & 100447 & 67.459369 &  71165 & 47.793821 &  64889 & 43.578912 \\\hline
$\geq 17$ &     2722 &  94828 & 69.675239 &  66810 & 49.088905 &  60827 & 44.692873 \\\hline
$\geq 18$ &     2503 &  89701 & 71.674790 &  62857 & 50.225330 &  57139 & 45.656412 \\\hline
$\geq 19$ &     2269 &  83728 & 73.801675 &  58377 & 51.456148 &  52977 & 46.696342 \\\hline
$\geq 20$ &     2096 &  79106 & 75.482824 &  54865 & 52.352099 &  49726 & 47.448473 \\\hline
$\geq 21$ &     1960 &  75291 & 76.827551 &  51981 & 53.041837 &  47035 & 47.994898 \\\hline
$\geq 22$ &     1821 &  71103 & 78.092257 &  48864 & 53.667216 &  44154 & 48.494234 \\\hline
$\geq 23$ &     1620 &  64538 & 79.676543 &  44126 & 54.476543 &  39808 & 49.145679 \\\hline
$\geq 24$ &     1481 &  59875 & 80.857529 &  40707 & 54.972316 &  36644 & 49.485483 \\\hline
$\geq 25$ &     1307 &  53462 & 81.808722 &  36189 & 55.377200 &  32534 & 49.784239 \\\hline
$\geq 26$ &     1141 &  47526 & 83.305872 &  31729 & 55.616126 &  28456 & 49.879053 \\\hline
$\geq 27$ &      883 &  37043 & 83.902605 &  24470 & 55.424689 &  21918 & 49.644394 \\\hline
\end{tabular}
\end{center}
\end{table}

In both cases the maximum coreness is equal to $27$. The core in this case contains all the concepts mentioned
earlier (case where self-loops were neglected), as well as the concepts
\dbtext{write program} (38),
\dbtext{pant} (63),
\dbtext{examination} (121),
\dbtext{study subject} (234),
\dbtext{go sport event} (241),
\dbtext{eat food} (264),
\dbtext{paint picture} (291),
\dbtext{candle} (327),
\dbtext{take course} (400),
\dbtext{storage} (496),
\dbtext{sometimes} (526),
\dbtext{gun} (635),
\dbtext{hide} (869),
\dbtext{enlightenment} (926),
\dbtext{effort} (1000),
\dbtext{shark} (1015),
\dbtext{rosebush} (1031),
\dbtext{laugh joke} (1095),
\dbtext{spoon} (1116),
\dbtext{well} (1201),
\dbtext{weather} (1248),
\dbtext{dead} (1279),
\dbtext{take bath} (1316),
\dbtext{purse} (1322),
\dbtext{anemone} (1348),
\dbtext{drink alcohol} (1386),
\dbtext{football} (1448),
\dbtext{water plant} (1470),
\dbtext{ice} (1634),
\dbtext{fiddle} (1652),
\dbtext{wrestle} (1665),
\dbtext{poop} (1672),
\dbtext{smart} (1678),
\dbtext{frog} (1692),
\dbtext{excite} (1704),
\dbtext{contemplate} (1784),
\dbtext{subway} (1804),
\dbtext{factory} (1917),
\dbtext{stay healthy} (1932),
\dbtext{lemur} (1998),
\dbtext{name} (2003),
\dbtext{pool} (2049),
\dbtext{instrument} (2086),
\dbtext{understand better} (2163),
\dbtext{can} (2261),
\dbtext{a} (2263),
\dbtext{sweet} (2330),
\dbtext{pee} (2354),
\dbtext{candy} (2386),
\dbtext{close eye} (2449),
\dbtext{suitcase} (2479),
\dbtext{satisfaction} (2483),
\dbtext{problem} (2500),
\dbtext{math} (2506),
\dbtext{sink} (2563),
\dbtext{iron} (2587),
\dbtext{cookie} (2595),
\dbtext{idea} (2837),
\dbtext{soft} (2842),
\dbtext{news} (2905),
\dbtext{surf} (3525),
\dbtext{teacher} (3556),
\dbtext{exhaustion} (3605),
\dbtext{fork} (3671),
\dbtext{planet} (3683),
\dbtext{france} (3826),
\dbtext{italy} (3881),
\dbtext{steel} (3907),
\dbtext{piano} (4010),
\dbtext{coat} (4020),
\dbtext{waste time} (4217),
\dbtext{mind} (4432),
\dbtext{funny} (4647),
\dbtext{eat vegetable} (4895),
\dbtext{bean} (4896),
\dbtext{touch} (5106),
\dbtext{pray} (5292),
\dbtext{measure} (5370),
\dbtext{view} (5574),
\dbtext{toilet} (5616),
\dbtext{program} (5620),
\dbtext{love else} (5621),
\dbtext{tooth} (5622),
\dbtext{disease} (5645),
\dbtext{peace} (5670),
\dbtext{lamp} (5671),
\dbtext{often} (5700),
\dbtext{buy beer} (5734),
\dbtext{internet} (5811),
\dbtext{dictionary} (5905),
\dbtext{rise} (5930),
\dbtext{bite} (6368),
\dbtext{sheep} (6424),
\dbtext{wool} (6425),
\dbtext{quiet} (6583),
\dbtext{high} (6606),
\dbtext{birthday} (6705),
\dbtext{reproduce} (6721),
\dbtext{mean} (6744),
\dbtext{drink coffee} (6817),
\dbtext{freezer} (6822),
\dbtext{good health} (7268),
\dbtext{eat ice cream} (7359),
\dbtext{learn language} (7364),
\dbtext{skin} (7399),
\dbtext{top} (7514),
\dbtext{cash} (7584),
\dbtext{leather} (7629),
\dbtext{read child} (7755),
\dbtext{rabbit} (7815),
\dbtext{pot} (8213),
\dbtext{clean clothe} (8216),
\dbtext{little} (8268),
\dbtext{clean house} (8295),
\dbtext{stay bed} (8815),
\dbtext{lawn} (8860),
\dbtext{event} (8862),
\dbtext{tin} (8891),
\dbtext{test} (9242),
\dbtext{seed} (9375),
\dbtext{cotton} (9729),
\dbtext{become tire} (9805),
\dbtext{lip} (9870),
\dbtext{lose weight} (10298),
\dbtext{healthy} (10482),
\dbtext{end} (10507),
\dbtext{group} (12400),
\dbtext{bullet} (13342),
\dbtext{melt} (13459),
\dbtext{roof} (14069),
\dbtext{taste} (14093),
\dbtext{organ} (14628),
\dbtext{solid} (15343),
\dbtext{point} (15518),
\dbtext{useful} (15524),
\dbtext{handle} (15706),
\dbtext{alaska} (15970),
\dbtext{department} (16725),
\dbtext{brain} (17555),
\dbtext{side} (17836),
\dbtext{chocolate} (18107),
\dbtext{sense} (18386),
\dbtext{feel better} (18399),
\dbtext{compete} (18538),
\dbtext{memory} (18563),
\dbtext{stay fit} (18712),
\dbtext{bush} (19864),
\dbtext{course} (19871),
\dbtext{power} (20085),
\dbtext{edge} (24347),
\dbtext{express information} (24906),
\dbtext{sunshine} (25192),
\dbtext{race} (25233),
\dbtext{flea} (25677),
\dbtext{return work} (25747),
\dbtext{earn live} (26632),
\dbtext{punch} (26708),
\dbtext{butt} (27369),
\dbtext{create art} (27886),
\dbtext{many person} (30864),
\dbtext{find house} (33328),
\dbtext{find outside} (34925),
\dbtext{winery} (36809),
\dbtext{branch} (37065),
\dbtext{polish} (38832),
\dbtext{wax} (39314),
\dbtext{slip} (47533),
\dbtext{agent} (58122),
\dbtext{slope} (64669),
\dbtext{make person laugh} (69984),
\dbtext{generic} (179027),
\dbtext{speedo} (203600),
\dbtext{get physical activity} (312389),
\dbtext{get shape} (312438),
\dbtext{do it} (313139),
\dbtext{get fit} (323709),
\dbtext{generic term} (332695),
and
\dbtext{teach other person} (427795).

\chapter{Shortest Paths}\label{chapter:shortest-paths}
In this chapter we examine properties of the shortest paths found in \conceptnet.

\section{Average Shortest Path Lengths}
In this section we examine the average path lengths in \conceptnet, both for the entire graphs
as well as for the big connected components that arise.
The number of vertices in every case is $279,497$.
Recall from Chapter \ref{chapter:components} that the entire graphs either by allowing assertions
with negative only polarity, or positive only polarity, or both, are disconnected.
Hence, in the calculation of the average path lengths we compute the average of the shortest paths
within the components; that is, 
the average of all the pairs of vertices that can be reached by at least one path.

\subsection{Negative Polarity}
Regarding the graph induced by the assertions of the English language with positive score and 
negative polarity we can observe the following.
The average path length for the directed graph is about $6.737$.
The average path length for the undirected graph is about $3.863$.
As a reminder, the number of edges of the directed graph (self-loops are omitted) is $13,387$,
while the number of edges of the undirected graph (again omitting self-loops) is $12,989$.

\subsubsection{Big Weakly Connected Component}
The average path length of the big weakly connected component found in the graph induced by the
assertions of the English language with positive score and negative polarity is about $3.864$.

\subsubsection{Big Strongly Connected Component}
The average path length of the big strongly connected component found in the graph induced by the
assertions of the English language with positive score and negative polarity is about $6.428$.
If we consider the same component as an undirected graph, then the average path length is about $3.537$.

\subsection{Positive Polarity}
Regarding the graph induced by the assertions of the English language with positive score and positive
polarity we can observe the following.
The average path length for the directed graph is about $4.811$.
The average path length for the undirected graph is about $4.330$.
As a reminder, the number of edges of the directed graph (self-loops are omitted) is $412,956$,
while the number of edges of the undirected graph (again omitting self-loops) is $401,367$.

\subsubsection{Big Weakly Connected Component}
The average path length of the big weakly connected component found in the graph induced by the
assertions of the English language with positive score and positive polarity is about $4.330$.

\subsubsection{Big Strongly Connected Component}
The average path length of the big strongly connected component found in the graph induced by the
assertions of the English language with positive score and positive polarity is about $4.205$.
If we consider the same component as an undirected graph, then the average path length is about $3.337$.

\subsection{Both Polarities}
Regarding the graph induced by the assertions of the English language with positive score and any
polarity we can observe the following.
The average path length for the directed graph is about $4.772$.
The average path length for the undirected graph is about $4.280$.
As a reminder, the number of edges of the directed graph (self-loops are omitted) is $424,525$,
while the number of edges of the undirected graph (again omitting self-loops) is $412,569$.

\subsubsection{Big Weakly Connected Component}
The average path length of the big weakly connected component found in the graph induced by the
assertions of the English language with positive score and any polarity is about $4.280$.

\subsubsection{Big Strongly Connected Component}
The average path length of the big strongly connected component found in the graph induced by the
assertions of the English language with positive score and any polarity is about $4.167$.
If we consider the same component as an undirected graph, then the average path length is about $3.291$.

\section{Path Length Distributions}
In this section we examine the distributions of the shortest path lengths in \conceptnet, 
both for the entire graph, as well as the big connected components that arise in every case.
Again we distinguish cases based on the polarity that we allow on the edges.

\subsection{Negative Polarity}
Table \ref{tbl:shortest-paths:histogram:negative} gives the distribution of the shortest paths
in the directed and undirected graph induced by the assertions of the English language with positive
score and negative polarity. It also presents the number of pairs for which the second vertex is unreachable
from the first one.

\begin{table}[ht]
\caption{Distribution of shortest paths in the graph induced by the assertions with positive score
and negative polarity in \conceptnet. The table on the left presents the case of the directed
graph, while the table on the right presents the case of the undirected graph.
The length is equal to $\infty$ for a pair of vertices when the second vertex is unreachable from the
first vertex.}\label{tbl:shortest-paths:histogram:negative}
\hspace{\fill}
\begin{tabular}{|r|r|}\hline
\multicolumn{2}{|c|}{directed graph} \\\hline
\multicolumn{1}{|c|}{path}   & \multicolumn{1}{c|}{number of} \\
\multicolumn{1}{|c|}{length} & \multicolumn{1}{c|}{shortest paths} \\\hline
$ 1$ & $   13,387$ \\\hline
$ 2$ & $  124,135$ \\\hline
$ 3$ & $  482,551$ \\\hline
$ 4$ & $  977,349$ \\\hline
$ 5$ & $1,539,103$ \\\hline
$ 6$ & $1,461,467$ \\\hline
$ 7$ & $1,400,197$ \\\hline
$ 8$ & $  936,127$ \\\hline
$ 9$ & $  856,899$ \\\hline
$10$ & $  510,899$ \\\hline
$11$ & $  271,808$ \\\hline
$12$ & $  171,242$ \\\hline
$13$ & $   98,542$ \\\hline
$14$ & $   71,825$ \\\hline
$15$ & $   36,628$ \\\hline
$16$ & $   15,213$ \\\hline
$17$ & $    4,973$ \\\hline
$18$ & $    1,953$ \\\hline
$19$ & $      841$ \\\hline
$20$ & $      424$ \\\hline
$21$ & $      165$ \\\hline
$22$ & $       51$ \\\hline
$23$ & $        9$ \\\hline
$24$ & $        1$ \\\hline
\multicolumn{2}{c}{} \\\hline
$\infty$ & $78,109,317,723$ \\\hline
\multicolumn{2}{c}{} \\\cline{2-2}
\multicolumn{1}{r|}{sum} & $78,118,293,512$ \\\cline{2-2}
\end{tabular}
\hspace{\fill}
\begin{tabular}{|r|r|}\hline
\multicolumn{2}{|c|}{undirected graph} \\\hline
\multicolumn{1}{|c|}{path}   & \multicolumn{1}{c|}{number of} \\
\multicolumn{1}{|c|}{length} & \multicolumn{1}{c|}{shortest paths} \\\hline
$ 1$ & $    12,989$ \\\hline
$ 2$ & $ 8,271,128$ \\\hline
$ 3$ & $ 7,529,595$ \\\hline
$ 4$ & $10,133,416$ \\\hline
$ 5$ & $ 6,074,004$ \\\hline
$ 6$ & $ 3,057,701$ \\\hline
$ 7$ & $ 1,191,562$ \\\hline
$ 8$ & $   400,130$ \\\hline
$ 9$ & $   121,610$ \\\hline
$10$ & $    57,323$ \\\hline
$11$ & $    37,909$ \\\hline
$12$ & $    34,148$ \\\hline
$13$ & $    14,184$ \\\hline
$14$ & $     6,137$ \\\hline
$15$ & $     1,510$ \\\hline
$16$ & $       366$ \\\hline
$17$ & $        48$ \\\hline
$18$ & $         8$ \\\hline
\multicolumn{2}{c}{} \\\hline
$\infty$ & $39,022,202,988$ \\\hline
\multicolumn{2}{c}{} \\\cline{2-2}
\multicolumn{1}{r|}{sum} & $39,059,146,756$ \\\cline{2-2}
\multicolumn{2}{c}{} \\
\multicolumn{2}{c}{} \\
\multicolumn{2}{c}{} \\
\multicolumn{2}{c}{} \\
\multicolumn{2}{c}{} \\
\multicolumn{2}{c}{} \\
\end{tabular}
\hspace{\fill}
\end{table}

\subsubsection{Negative Polarity: Big Weakly Connected Component}
First we examine the big weakly connected component that arises in the graph induced
by the assertions of the English language with positive score and negative polarity.
The component has $8,596$ nodes and $11,247$ undirected edges.
Table \ref{tbl:pathlength:negative:BigUC:histogram} gives the distribution of
shortest path lengths in this big undirected component.

\begin{table}[ht]
\caption{Distribution of undirected shortest path lengths in the big weakly connected component 
induced by the assertions with negative polarity of \conceptnet.}\label{tbl:pathlength:negative:BigUC:histogram}
\hspace{\fill}
\begin{tabular}{|r|r|}\hline
\multicolumn{1}{|c|}{path}   & \multicolumn{1}{c|}{number of} \\
\multicolumn{1}{|c|}{length} & \multicolumn{1}{c|}{shortest paths} \\\hline
$ 1$ & $    11,247$ \\\hline
$ 2$ & $ 8,270,557$ \\\hline
$ 3$ & $ 7,529,480$ \\\hline
$ 4$ & $10,133,389$ \\\hline
$ 5$ & $ 6,074,001$ \\\hline
$ 6$ & $ 3,057,701$ \\\hline
$ 7$ & $ 1,191,562$ \\\hline
$ 8$ & $   400,130$ \\\hline
$ 9$ & $   121,610$ \\\hline
$10$ & $    57,323$ \\\hline
$11$ & $    37,909$ \\\hline
$12$ & $    34,148$ \\\hline
$13$ & $    14,184$ \\\hline
$14$ & $     6,137$ \\\hline
$15$ & $     1,510$ \\\hline
$16$ & $       366$ \\\hline
$17$ & $        48$ \\\hline
$18$ & $         8$ \\\hline
\multicolumn{2}{c}{}\\\cline{2-2}
\multicolumn{1}{r|}{sum} & $36,941,310$ \\\cline{2-2}
\end{tabular}
\hspace{\fill}
\end{table}

\subsubsection{Negative Polarity: Big Strongly Connected Component}
Next we examine the big strongly connected component that arises in the graph induced
by the assertions of the English language with positive score and negative polarity.
The component has $592$ nodes and $1,849$ directed edges
(self-loops were omitted from the enumeration).
The number of edges in the induced undirected graph that occurs after restricting ourselves
in these $592$ nodes (again, self-loops are omitted) is $1,566$. 
Table \ref{tbl:pathlength:negative:BigDC:histogram} gives the distribution of
directed shortest path lengths in this directed component as well as the distribution 
of the undirected shortest path lengths in the undirected graph induced by the concepts 
that appear in the big directed component induced by the assertions with 
negative polarity of \conceptnet.

\begin{table}[ht]
\caption{Distribution of directed shortest path lengths in the big directed component 
induced by the assertions with negative polarity of \conceptnet
as well as the distribution of the undirected shortest path lengths in the undirected
graph induced by the concepts that appear in the big directed component with negative polarity 
of \conceptnet.}\label{tbl:pathlength:negative:BigDC:histogram}
\hspace{\fill}
\begin{tabular}{|r|r|}\hline
\multicolumn{2}{|c|}{directed graph} \\\hline
\multicolumn{1}{|c|}{path}   & \multicolumn{1}{c|}{number of} \\
\multicolumn{1}{|c|}{length} & \multicolumn{1}{c|}{shortest paths} \\\hline
$ 1$ & $ 1,849$ \\\hline
$ 2$ & $ 8,458$ \\\hline
$ 3$ & $24,779$ \\\hline
$ 4$ & $44,834$ \\\hline
$ 5$ & $59,644$ \\\hline
$ 6$ & $58,813$ \\\hline
$ 7$ & $49,665$ \\\hline
$ 8$ & $34,593$ \\\hline
$ 9$ & $24,926$ \\\hline
$10$ & $17,389$ \\\hline
$11$ & $10,916$ \\\hline
$12$ & $ 6,382$ \\\hline
$13$ & $ 3,589$ \\\hline
$14$ & $ 2,260$ \\\hline
$15$ & $ 1,010$ \\\hline
$16$ & $   452$ \\\hline
$17$ & $   162$ \\\hline
$18$ & $    74$ \\\hline
$19$ & $    40$ \\\hline
$20$ & $    20$ \\\hline
$21$ & $    12$ \\\hline
$22$ & $     5$ \\\hline
\multicolumn{2}{c}{}\\\cline{2-2}
\multicolumn{1}{r|}{sum} & $349,872$ \\\cline{2-2}
\end{tabular}
\hspace{\fill}
\begin{tabular}{|r|r|}\hline
\multicolumn{2}{|c|}{undirected graph}\\\hline
\multicolumn{1}{|c|}{path}   & \multicolumn{1}{c|}{number of} \\
\multicolumn{1}{|c|}{length} & \multicolumn{1}{c|}{shortest paths} \\\hline
$1$ & $ 1,566$ \\\hline
$2$ & $24,978$ \\\hline
$3$ & $62,562$ \\\hline
$4$ & $56,424$ \\\hline
$5$ & $23,425$ \\\hline
$6$ & $ 5,274$ \\\hline
$7$ & $   655$ \\\hline
$8$ & $    50$ \\\hline
$9$ & $     2$ \\\hline
\multicolumn{2}{c}{}\\\cline{2-2}
\multicolumn{1}{r|}{sum} & $174,936$ \\\cline{2-2}
\multicolumn{2}{c}{}\\
\multicolumn{2}{c}{}\\
\multicolumn{2}{c}{}\\
\multicolumn{2}{c}{}\\
\multicolumn{2}{c}{}\\
\multicolumn{2}{c}{}\\
\multicolumn{2}{c}{}\\
\multicolumn{2}{c}{}\\
\multicolumn{2}{c}{}\\
\multicolumn{2}{c}{}\\
\multicolumn{2}{c}{}\\
\multicolumn{2}{c}{}\\
\multicolumn{2}{c}{}
\end{tabular}
\hspace{\fill}
\end{table}

\subsection{Positive Polarity}
Table \ref{tbl:shortest-paths:histogram:positive} gives the distribution of the shortest paths
in the directed and undirected graph induced by the assertions of the English language with positive
score and positive polarity. It also presents the number of pairs for which the second vertex is unreachable
from the first one.

\begin{table}[ht]
\caption{Distribution of shortest paths in the graph induced by the assertions with positive score
and positive polarity in \conceptnet. The table on the left presents the case of the directed
graph, while the table on the right presents the case of the undirected graph.
The length is equal to $\infty$ for a pair of vertices when the second vertex is unreachable from the
first vertex.}\label{tbl:shortest-paths:histogram:positive}
\hspace{\fill}
\begin{tabular}{|r|r|}\hline
\multicolumn{2}{|c|}{directed graph} \\\hline
\multicolumn{1}{|c|}{path}   & \multicolumn{1}{c|}{number of} \\
\multicolumn{1}{|c|}{length} & \multicolumn{1}{c|}{shortest paths} \\\hline
$ 1$ & $      412,956$ \\\hline
$ 2$ & $   20,909,748$ \\\hline
$ 3$ & $  354,226,806$ \\\hline
$ 4$ & $1,867,492,200$ \\\hline
$ 5$ & $2,569,306,798$ \\\hline
$ 6$ & $  988,364,189$ \\\hline
$ 7$ & $  197,669,166$ \\\hline
$ 8$ & $   31,493,222$ \\\hline
$ 9$ & $    4,522,183$ \\\hline
$10$ & $      804,884$ \\\hline
$11$ & $      169,392$ \\\hline
$12$ & $       21,064$ \\\hline
$13$ & $        2,175$ \\\hline
$14$ & $          113$ \\\hline
$15$ & $            5$ \\\hline
\multicolumn{2}{c}{} \\\hline
$\infty$ & $72,082,898,611$ \\\hline
\multicolumn{2}{c}{} \\\cline{2-2}
\multicolumn{1}{r|}{sum} & $78,118,293,512$ \\\cline{2-2}
\multicolumn{2}{c}{} \\
\end{tabular}
\hspace{\fill}
\begin{tabular}{|r|r|}\hline
\multicolumn{2}{|c|}{undirected graph} \\\hline
\multicolumn{1}{|c|}{path}   & \multicolumn{1}{c|}{number of} \\
\multicolumn{1}{|c|}{length} & \multicolumn{1}{c|}{shortest paths} \\\hline
$ 1$ & $       401,367$ \\\hline
$ 2$ & $   136,176,653$ \\\hline
$ 3$ & $ 2,601,936,809$ \\\hline
$ 4$ & $12,781,641,328$ \\\hline
$ 5$ & $ 8,094,579,408$ \\\hline
$ 6$ & $ 1,203,650,632$ \\\hline
$ 7$ & $   165,209,595$ \\\hline
$ 8$ & $    26,997,091$ \\\hline
$ 9$ & $     4,242,831$ \\\hline
$10$ & $       765,445$ \\\hline
$11$ & $       390,830$ \\\hline
$12$ & $        62,288$ \\\hline
$13$ & $         5,142$ \\\hline
$14$ & $           658$ \\\hline
$15$ & $           104$ \\\hline
$16$ & $             4$ \\\hline
\multicolumn{2}{c}{} \\\hline
$\infty$ & $14,043,086,571$ \\\hline
\multicolumn{2}{c}{} \\\cline{2-2}
\multicolumn{1}{r|}{sum} & $39,059,146,756$ \\\cline{2-2}
\end{tabular}
\hspace{\fill}
\end{table}

\subsubsection{Positive Polarity: Big Weakly Connected Component}
Here we examine the big weakly connected component that arises in the graph induced
by the assertions of the English language with positive score and positive polarity.
The component has $223,679$ nodes and $383,698$ undirected edges.
Table \ref{tbl:pathlength:positive:BigUC:histogram} gives the distribution of
undirected shortest path lengths in the big undirected component induced by the assertions
with positive polarity of \conceptnet.

\begin{table}[ht]
\caption{Distribution of undirected shortest path lengths in the big weakly connected component 
induced by the assertions with positive polarity of \conceptnet.}\label{tbl:pathlength:positive:BigUC:histogram}
\hspace{\fill}
\begin{tabular}{|r|r|}\hline
\multicolumn{1}{|c|}{path}   & \multicolumn{1}{c|}{number of} \\
\multicolumn{1}{|c|}{length} & \multicolumn{1}{c|}{shortest paths} \\\hline
$ 1$ & $       383,698$ \\\hline
$ 2$ & $   136,170,064$ \\\hline
$ 3$ & $ 2,601,936,595$ \\\hline
$ 4$ & $12,781,641,299$ \\\hline
$ 5$ & $ 8,094,579,405$ \\\hline
$ 6$ & $ 1,203,650,632$ \\\hline
$ 7$ & $   165,209,595$ \\\hline
$ 8$ & $    26,997,091$ \\\hline
$ 9$ & $     4,242,831$ \\\hline
$10$ & $       765,445$ \\\hline
$11$ & $       390,830$ \\\hline
$12$ & $        62,288$ \\\hline
$13$ & $         5,142$ \\\hline
$14$ & $           658$ \\\hline
$15$ & $           104$ \\\hline
$16$ & $             4$ \\\hline
\multicolumn{2}{c}{}\\\cline{2-2}
\multicolumn{1}{r|}{sum} & $25,016,035,681$ \\\cline{2-2}
\end{tabular}
\hspace{\fill}
\end{table}

\subsubsection{Positive Polarity: Big Strongly Connected Component}
Now we examine the big weakly connected component that arises in the graph induced
by the assertions of the English language with positive score and positive polarity.
The component has $13,700$ nodes and $120,865$ edges.
Table \ref{tbl:pathlength:positive:BigDC:histogram} gives the distribution of
directed shortest path lengths in the big directed component induced by the assertions
with positive polarity of \conceptnet as well as the distribution of the undirected
shortest path lengths in the undirected graph induced by the concepts that appear in the
big directed component induced by the assertions with positive polarity of \conceptnet.

\begin{table}[ht]
\caption{Distribution of directed shortest path lengths in the big directed component 
induced by the assertions with positive polarity of \conceptnet
as well as the distribution of the undirected shortest path lengths in the undirected
graph induced by the concepts that appear in the big directed component with positive polarity 
of \conceptnet.}\label{tbl:pathlength:positive:BigDC:histogram}
\hspace{\fill}
\begin{tabular}{|r|r|}\hline
\multicolumn{2}{|c|}{directed graph} \\\hline
\multicolumn{1}{|c|}{path}   & \multicolumn{1}{c|}{number of} \\
\multicolumn{1}{|c|}{length} & \multicolumn{1}{c|}{shortest paths} \\\hline
$ 1$ & $   120,865$ \\\hline
$ 2$ & $ 4,006,764$ \\\hline
$ 3$ & $35,415,728$ \\\hline
$ 4$ & $82,100,213$ \\\hline
$ 5$ & $52,346,292$ \\\hline
$ 6$ & $11,632,181$ \\\hline
$ 7$ & $ 1,709,055$ \\\hline
$ 8$ & $   283,094$ \\\hline
$ 9$ & $    52,735$ \\\hline
$10$ & $     8,503$ \\\hline
$11$ & $       800$ \\\hline
$12$ & $        70$ \\\hline
\multicolumn{2}{c}{}\\\cline{2-2}
\multicolumn{1}{r|}{sum} & $187,676,300$ \\\cline{2-2}
\end{tabular}
\hspace{\fill}
\begin{tabular}{|r|r|}\hline
\multicolumn{2}{|c|}{undirected graph}\\\hline
\multicolumn{1}{|c|}{path}   & \multicolumn{1}{c|}{number of} \\
\multicolumn{1}{|c|}{length} & \multicolumn{1}{c|}{shortest paths} \\\hline
$1$ & $   109,378$ \\\hline
$2$ & $ 8,211,734$ \\\hline
$3$ & $47,622,303$ \\\hline
$4$ & $35,784,328$ \\\hline
$5$ & $ 2,084,925$ \\\hline
$6$ & $    25,409$ \\\hline
$7$ & $        73$ \\\hline
\multicolumn{2}{c}{}\\\cline{2-2}
\multicolumn{1}{r|}{sum} & $93,838,150$ \\\cline{2-2}
\multicolumn{2}{c}{}\\
\multicolumn{2}{c}{}\\
\multicolumn{2}{c}{}\\
\multicolumn{2}{c}{}\\
\multicolumn{2}{c}{}
\end{tabular}
\hspace{\fill}
\end{table}

\subsection{Both Polarities}
Table \ref{tbl:shortest-paths:histogram:both} gives the distribution of the shortest paths
in the directed and undirected graph induced by the assertions of the English language with positive
score and any polarity. It also presents the number of pairs for which the second vertex is unreachable
from the first one.

\begin{table}[ht]
\caption{Distribution of shortest paths in the graph induced by the assertions with positive score
and any polarity in \conceptnet. The table on the left presents the case of the directed
graph, while the table on the right presents the case of the undirected graph.
The length is equal to $\infty$ for a pair of vertices when the second vertex is unreachable from the
first vertex.}\label{tbl:shortest-paths:histogram:both}
\hspace{\fill}
\begin{tabular}{|r|r|}\hline
\multicolumn{2}{|c|}{directed graph} \\\hline
\multicolumn{1}{|c|}{path}   & \multicolumn{1}{c|}{number of} \\
\multicolumn{1}{|c|}{length} & \multicolumn{1}{c|}{shortest paths} \\\hline
$ 1$ & $      424,525$ \\\hline
$ 2$ & $   23,978,858$ \\\hline
$ 3$ & $  406,012,505$ \\\hline
$ 4$ & $2,045,811,557$ \\\hline
$ 5$ & $2,652,013,614$ \\\hline
$ 6$ & $  979,044,479$ \\\hline
$ 7$ & $  192,535,914$ \\\hline
$ 8$ & $   32,167,500$ \\\hline
$ 9$ & $    5,342,778$ \\\hline
$10$ & $    1,023,297$ \\\hline
$11$ & $      200,328$ \\\hline
$12$ & $       24,916$ \\\hline
$13$ & $        2,471$ \\\hline
$14$ & $          132$ \\\hline
$15$ & $            5$ \\\hline
\multicolumn{2}{c}{} \\\hline
$\infty$ & $71,779,710,633$ \\\hline
\multicolumn{2}{c}{} \\\cline{2-2}
\multicolumn{1}{r|}{sum} & $78,118,293,512$ \\\cline{2-2}
\multicolumn{2}{c}{} \\
\end{tabular}
\hspace{\fill}
\begin{tabular}{|r|r|}\hline
\multicolumn{2}{|c|}{undirected graph} \\\hline
\multicolumn{1}{|c|}{path}   & \multicolumn{1}{c|}{number of} \\
\multicolumn{1}{|c|}{length} & \multicolumn{1}{c|}{shortest paths} \\\hline
$ 1$ & $       412,569$ \\\hline
$ 2$ & $   194,269,357$ \\\hline
$ 3$ & $ 3,140,569,387$ \\\hline
$ 4$ & $13,521,818,553$ \\\hline
$ 5$ & $ 7,986,933,052$ \\\hline
$ 6$ & $ 1,141,783,884$ \\\hline
$ 7$ & $   154,810,521$ \\\hline
$ 8$ & $    25,254,401$ \\\hline
$ 9$ & $     3,923,373$ \\\hline
$10$ & $       740,256$ \\\hline
$11$ & $       389,913$ \\\hline
$12$ & $        59,126$ \\\hline
$13$ & $         4,761$ \\\hline
$14$ & $           632$ \\\hline
$15$ & $           104$ \\\hline
$16$ & $             4$ \\\hline
\multicolumn{2}{c}{} \\\hline
$\infty$ & $12,888,176,863$ \\\hline
\multicolumn{2}{c}{} \\\cline{2-2}
\multicolumn{1}{r|}{sum} & $39,059,146,756$ \\\cline{2-2}
\end{tabular}
\hspace{\fill}
\end{table}

\subsubsection{Both Polarities: Big Weakly Connected Component}
The big weakly connected component that arises in the graph induced by the assertions of the English language
with positive score and no restrictions to polarity has $228,784$ nodes and $394,554$ undirected edges.
Table \ref{tbl:pathlength:both:BigUC:histogram} gives the distribution of
undirected shortest path lengths in the big undirected component induced by the assertions
with positive polarity of \conceptnet.

\begin{table}[ht]
\caption{Distribution of undirected shortest path lengths in the big weakly connected component 
induced by the assertions with any polarity of \conceptnet.}\label{tbl:pathlength:both:BigUC:histogram}
\hspace{\fill}
\begin{tabular}{|r|r|}\hline
\multicolumn{1}{|c|}{path}   & \multicolumn{1}{c|}{number of} \\
\multicolumn{1}{|c|}{length} & \multicolumn{1}{c|}{shortest paths} \\\hline
$ 1$ & $       394,554$ \\\hline
$ 2$ & $   194,262,673$ \\\hline
$ 3$ & $ 3,140,569,163$ \\\hline
$ 4$ & $13,521,818,522$ \\\hline
$ 5$ & $ 7,986,933,049$ \\\hline
$ 6$ & $ 1,141,783,884$ \\\hline
$ 7$ & $   154,810,521$ \\\hline
$ 8$ & $    25,254,401$ \\\hline
$ 9$ & $     3,923,373$ \\\hline
$10$ & $       740,256$ \\\hline
$11$ & $       389,913$ \\\hline
$12$ & $        59,126$ \\\hline
$13$ & $         4,761$ \\\hline
$14$ & $           632$ \\\hline
$15$ & $           104$ \\\hline
$16$ & $             4$ \\\hline
\multicolumn{2}{c}{}\\\cline{2-2}
\multicolumn{1}{r|}{sum} & $25,016,035,681$ \\\cline{2-2}
\end{tabular}
\hspace{\fill}
\end{table}

\subsubsection{Both Polarities: Big Strongly Connected Component}
The big weakly connected component that arises in the graph induced by the assertions of the English language
with positive score and no restrictions to polarity has $14,025$ nodes and $126,151$ edges.
Table \ref{tbl:pathlength:both:BigDC:histogram} gives the distribution of
directed shortest path lengths in the big directed component induced by the assertions
with any polarity of \conceptnet as well as the distribution of the undirected
shortest path lengths in the undirected graph induced by the concepts that appear in the
big directed component induced by the assertions with any polarity of \conceptnet.

\begin{table}[ht]
\caption{Distribution of directed shortest path lengths in the big directed component 
induced by the assertions with any polarity of \conceptnet
as well as the distribution of the undirected shortest path lengths in the undirected
graph induced by the concepts that appear in the big directed component with any polarity 
of \conceptnet.}\label{tbl:pathlength:both:BigDC:histogram}
\hspace{\fill}
\begin{tabular}{|r|r|}\hline
\multicolumn{2}{|c|}{directed graph}\\\hline
\multicolumn{1}{|c|}{path}   & \multicolumn{1}{c|}{number of} \\
\multicolumn{1}{|c|}{length} & \multicolumn{1}{c|}{shortest paths} \\\hline
$ 1$ & $   126,151$ \\\hline
$ 2$ & $ 4,499,601$ \\\hline
$ 3$ & $39,414,540$ \\\hline
$ 4$ & $86,974,943$ \\\hline
$ 5$ & $52,332,076$ \\\hline
$ 6$ & $11,279,229$ \\\hline
$ 7$ & $ 1,667,767$ \\\hline
$ 8$ & $   311,440$ \\\hline
$ 9$ & $    68,835$ \\\hline
$10$ & $    10,879$ \\\hline
$11$ & $     1,047$ \\\hline
$12$ & $        92$ \\\hline
\multicolumn{2}{c}{}\\\cline{2-2}
\multicolumn{1}{r|}{sum} & $196,686,600$ \\\cline{2-2}
\end{tabular}
\hspace{\fill}
\begin{tabular}{|r|r|}\hline
\multicolumn{2}{|c|}{undirected graph}\\\hline
\multicolumn{1}{|c|}{path}   & \multicolumn{1}{c|}{number of} \\
\multicolumn{1}{|c|}{length} & \multicolumn{1}{c|}{shortest paths} \\\hline
$1$ & $   114,294$ \\\hline
$2$ & $ 9,938,647$ \\\hline
$3$ & $51,498,460$ \\\hline
$4$ & $34,859,851$ \\\hline
$5$ & $ 1,908,614$ \\\hline
$6$ & $    23,366$ \\\hline
$7$ & $        68$ \\\hline
\multicolumn{2}{c}{}\\\cline{2-2}
\multicolumn{1}{r|}{sum} & $98,343,300$ \\\cline{2-2}
\multicolumn{2}{c}{}\\
\multicolumn{2}{c}{}\\
\multicolumn{2}{c}{}\\
\multicolumn{2}{c}{}\\
\multicolumn{2}{c}{}
\end{tabular}
\hspace{\fill}
\end{table}

\section{Longest Geodesic Paths}\label{section:longest-geodesic}
Chapter \ref{chapter:components} showed that the entire graph is
disconnected. Hence, instead of examining the diameter which 
is formally infinite, we will examine the longest geodesic paths.
For the computations we consider subgraphs with positive score
on the assertions of the English language.

\subsection{Negative Polarity}
In this section we consider the directed and undirected graph
induced by assertions with negative polarity only.

%
%
\subsubsection{Directed Graph}
The longest geodesic path has length $24$ and connects
the concepts \dbtext{farmer} (908) and \dbtext{brass} (27632).
The full sequence of the longest geodesic path is given by
\dbtext{farmer} (908) $\rightarrow$
\dbtext{farm} (2562) $\rightarrow$
\dbtext{zoo} (547) $\rightarrow$
\dbtext{country} (640) $\rightarrow$
\dbtext{urban} (29003) $\rightarrow$
\dbtext{rural} (185019) $\rightarrow$
\dbtext{common} (17473) $\rightarrow$
\dbtext{occasional} (155305) $\rightarrow$
\dbtext{often} (5700) $\rightarrow$
\dbtext{never} (126958) $\rightarrow$
\dbtext{exist} (2907) $\rightarrow$
\dbtext{touch} (5106) $\rightarrow$
\dbtext{see} (1161) $\rightarrow$
\dbtext{computer} (467) $\rightarrow$
\dbtext{human} (80) $\rightarrow$
\dbtext{animal} (902) $\rightarrow$
\dbtext{man} (7) $\rightarrow$
\dbtext{chick} (14872) $\rightarrow$
\dbtext{egg} (9339) $\rightarrow$
\dbtext{chicken} (191) $\rightarrow$
\dbtext{cow} (1613) $\rightarrow$
\dbtext{horse} (1412) $\rightarrow$
\dbtext{gold} (2266) $\rightarrow$
\dbtext{silver} (13722) $\rightarrow$
\dbtext{brass} (27632).
The justification is given by the following sentences.
\begin{enumerate}[noitemsep,leftmargin=8mm,topsep=0.5mm]
 \item farmer is not farm
 \item farm is not zoo
 \item a Zoo is not a kind of country.
 \item country is not urban
 \item urban is not rural
 \item rural is not common
 \item common is not occasional
 \item occasional is not often
 \item often is not never
 \item never is not existing
 \item Some things that exist you can't touch.
 \item touch is not seeing
 \item a saw is not a kind of computer.
 \item A computer should not want to be a human
 \item human is not animal
 \item animal is not man
 \item men is not chicks
 \item chick is not egg
 \item egg is not chicken
 \item chicken is not cow
 \item cow is not horse
 \item horses is generally not gold.
 \item gold is not silver
 \item silver is not brass
\end{enumerate}

\paragraph{Big Strongly Connected Component.}
The diameter of the big directed component is equal to $22$.
The full sequence of the diameter is given by
\dbtext{zoo} (547) $\rightarrow$
\dbtext{country} (640) $\rightarrow$
\dbtext{urban} (29003) $\rightarrow$
\dbtext{rural} (185019) $\rightarrow$
\dbtext{common} (17473) $\rightarrow$
\dbtext{occasional} (155305) $\rightarrow$
\dbtext{often} (5700) $\rightarrow$
\dbtext{never} (126958) $\rightarrow$
\dbtext{exist} (2907) $\rightarrow$
\dbtext{touch} (5106) $\rightarrow$
\dbtext{see} (1161) $\rightarrow$
\dbtext{computer} (467) $\rightarrow$
\dbtext{human} (80) $\rightarrow$
\dbtext{animal} (902) $\rightarrow$
\dbtext{man} (7) $\rightarrow$
\dbtext{chick} (14872) $\rightarrow$
\dbtext{egg} (9339) $\rightarrow$
\dbtext{chicken} (191) $\rightarrow$
\dbtext{cow} (1613) $\rightarrow$
\dbtext{horse} (1412) $\rightarrow$
\dbtext{gold} (2266) $\rightarrow$
\dbtext{silver} (13722) $\rightarrow$
\dbtext{brass} (27632).
The justification is given by the following sentences.
\begin{enumerate}[noitemsep,leftmargin=8mm,topsep=0.5mm]
 \item a Zoo is not a kind of country.
 \item country is not urban
 \item urban is not rural
 \item rural is not common
 \item common is not occasional
 \item occasional is not often
 \item often is not never
 \item never is not existing
 \item Some things that exist you can't touch.
 \item touch is not seeing
 \item a saw is not a kind of computer.
 \item A computer should not want to be a human
 \item human is not animal
 \item animal is not man
 \item men is not chicks
 \item chick is not egg
 \item egg is not chicken
 \item chicken is not cow
 \item cow is not horse
 \item horses is generally not gold.
 \item gold is not silver
 \item silver is not brass
\end{enumerate}

\smallskip

The equivalent undirected graph of this component has diameter equal to $9$.
The full sequence of the diameter in this case is given by
\dbtext{lime} (6416) $\rightarrow$
\dbtext{lemon} (14212) $\rightarrow$
\dbtext{orange} (15004) $\rightarrow$
\dbtext{apple} (4596) $\rightarrow$
\dbtext{computer} (467) $\rightarrow$
\dbtext{person} (9) $\rightarrow$
\dbtext{listen} (75) $\rightarrow$
\dbtext{sometimes} (526) $\rightarrow$
\dbtext{always} (43553) $\rightarrow$
\dbtext{occasional} (155305).
The justification is given by the following sentences.
\begin{enumerate}[noitemsep,leftmargin=8mm,topsep=0.5mm]
 \item lime is not lemon
 \item lemon is not orange
 \item orange is not apple
 \item my computer is not a apple
 \item person does not want to be a computer
 \item a person doesn't want to listen. / the people don't listen, usually
 \item Sometimes we don't listen.
 \item always is not sometimes
 \item occasional is not always
\end{enumerate}

\paragraph{Big Weakly Connected Component.}
In the following section we will see that the longest geodesic path in the
undirected graph induced by the assertions of the English language with 
positive score and 
negative polarity is $18$.
This fact, together with the decomposition of the weakly connected components
of the graph that is induced by the assertions of the English language with negative
polarity and which is presented in Chapter \ref{chapter:components} 
(Table \ref{tbl:distribution:component:negative:weak}),
it follows that the diameter of this component is equal to $18$.
One detailed instance admitting this diameter is given in the following section
which describes the longest geodesic path in the graph induced by the assertions
with negative polarity.

%
%
\subsubsection{Undirected Graph}
The longest geodesic path has length $18$ and connects
the concepts \dbtext{twin} (13665) and \dbtext{height} (96373).
The full sequence of the longest geodesic path is given by
\dbtext{twin} (13665) $\rightarrow$
\dbtext{look alike} (58776) $\rightarrow$
\dbtext{bell} (10210) $\rightarrow$
\dbtext{verb} (1490) $\rightarrow$
\dbtext{subject} (6754) $\rightarrow$
\dbtext{king} (1443) $\rightarrow$
\dbtext{queen} (9693) $\rightarrow$
\dbtext{america} (2852) $\rightarrow$
\dbtext{monarchy} (18801) $\rightarrow$
\dbtext{republic} (46056) $\rightarrow$
\dbtext{dictatorship} (22962) $\rightarrow$
\dbtext{person} (9) $\rightarrow$
\dbtext{late} (1520) $\rightarrow$
\dbtext{recent} (52116) $\rightarrow$
\dbtext{long} (9087) $\rightarrow$
\dbtext{wide} (27291) $\rightarrow$
\dbtext{narrowness} (345590) $\rightarrow$
\dbtext{width} (130163) $\rightarrow$
\dbtext{height} (96373).
The justification is given by the following sentences.
\begin{enumerate}[noitemsep,leftmargin=8mm,topsep=0.5mm]
 \item Twins don't necessarily look alike
 \item all bells do not look alike
 \item \textquotedblleft Bell\textquotedblright is not a verb.
 \item subject is not verb
 \item The king is not a subject.
 \item king is not queen
 \item America does not have a queen.
 \item America is not a monarchy.
 \item republic is not monarchy
 \item republic is not dictatorship
 \item a person doesn't want dictatorship.
 \item person does not want to be late
 \item recent is not late
 \item recent is not long
 \item wide is not long
 \item narrowness is not wide
 \item narrowness is not width
 \item height is not width
\end{enumerate}

\subsection{Positive Polarity}
In this section we consider the directed and undirected graph
induced by assertions with positive polarity only.

%
%
\subsubsection{Directed Graph}
The longest geodesic path has length $15$ and connects
the concepts \dbtext{american alphabet} (40903) and \dbtext{mosque} (177603).
The full sequence for the path is given by
\dbtext{american alphabet} (40903) $\rightarrow$
\dbtext{twenty six letter} (40904) $\rightarrow$
\dbtext{english alphabet} (8492) $\rightarrow$
\dbtext{26 letter} (2622) $\rightarrow$
\dbtext{english language} (2623) $\rightarrow$
\dbtext{confuse} (1871) $\rightarrow$
\dbtext{ask question} (8559) $\rightarrow$
\dbtext{find information} (8931) $\rightarrow$
\dbtext{discover new} (87726) $\rightarrow$
\dbtext{tell many person} (427796) $\rightarrow$
\dbtext{evangelist} (98420) $\rightarrow$
\dbtext{fundamentalist} (176617) $\rightarrow$
\dbtext{taliban} (119866) $\rightarrow$
\dbtext{islamist} (119867) $\rightarrow$
\dbtext{muslim} (8663) $\rightarrow$
\dbtext{mosque} (177603).
The justification is given by the following sentences.
\begin{enumerate}[noitemsep,leftmargin=8mm,topsep=0.5mm]
 \item The American alphabet contains twenty six letters.
 \item There are twenty six letters in the english alphabet
 \item The English alphabet has 26 letters.
 \item There are 26 letters in the english language.
 \item The English language is sometimes confusing.
 \item When you are confused about something you should ask questions.
 \item If you want to find information then you should ask questions
 \item Something that might happen while finding information is that you discover new things
 \item discovering something new would make you want to tell many people about something
 \item telling many people about something is for evangelists.
 \item evangelist is a type of fundamentalist
 \item You are likely to find fundamentalists in the Taliban.
 \item The Taliban are Islamists.
 \item an Islamist is a kind of Muslim.
 \item You are likely to find Muslims in the mosque.
\end{enumerate}

\begin{remark}[Polarity Misclassification]\label{rem:diameter:misclassification}
We note that the longest geodesic path that was originally returned had the concept
\dbtext{eat pork} (20781) as the final node for the path.
However, this was purely a result of misclassification in the database, since the
sentence associated with the edge admitting the connection was 
\dbtext{Muslims can eat anything but pork.}.
The assertion that justifies the edge has ID $177981$, with best frame ID equal to $30$
which implied the form \dbtext{\{1\} can \{2\}}, 
which in turn implies positive polarity as expected during our search.
However, the actual sentence uses the frame for the opposite polarity.

After the above observation we searched in the database manually to see if we could
replace that particular edge with another one that actually has positive polarity.
There are indeed five more sentences with positive polarity and the one with the 
highest score ($2$) was chosen and presented above. 
We further note that among the other four sentences that have positive polarity we 
encounter \dbtext{a Muslim can fast during Ramadan} and \dbtext{muslims can fast for ramadan} 
which connect the concept \dbtext{muslim} (8663) with \dbtext{fast during ramadan} (53518) 
and \dbtext{fast ramadan} (65620) respectively.
\end{remark}

\paragraph{Big Strongly Connected Component.}
The diameter of this big directed component is equal to $12$.
The full sequence of the diameter is given by
\dbtext{sixth day week} (2754) $\rightarrow$
\dbtext{Friday} (2755) $\rightarrow$
\dbtext{day week} (203694) $\rightarrow$
\dbtext{calendar} (1228) $\rightarrow$
\dbtext{house} (652) $\rightarrow$
\dbtext{person} (9) $\rightarrow$
\dbtext{discover new} (87726) $\rightarrow$
\dbtext{tell many person} (427796) $\rightarrow$
\dbtext{evangelist} (98420) $\rightarrow$
\dbtext{fundamentalist} (176617) $\rightarrow$
\dbtext{taliban} (119866) $\rightarrow$
\dbtext{islamist} (119867) $\rightarrow$
\dbtext{muslim} (8663).
The justification is given by the following sentences.
\begin{enumerate}[noitemsep,leftmargin=8mm,topsep=0.5mm]
 \item The sixth day of the week is Friday.
 \item Friday is a kind of day of the week.
 \item calendar has days of the week
 \item You are likely to find Calendar in a house.
 \item a house is created by people / house has people
 \item a person wants to discover new things
 \item discovering something new would make you want to tell many people about something
 \item telling many people about something is for evangelists.
 \item evangelist is a type of fundamentalist
 \item You are likely to find fundamentalists in the Taliban.
 \item The Taliban are Islamists.
 \item an Islamist is a kind of Muslim.
\end{enumerate}

\smallskip

The equivalent undirected graph of this component has diameter equal to $7$.
The full sequence of the diameter in this case is given by
\dbtext{tell punishment} (978) $\rightarrow$
\dbtext{pass sentence} (297) $\rightarrow$
\dbtext{word} (51) $\rightarrow$
\dbtext{person} (9) $\rightarrow$
\dbtext{office build} (210) $\rightarrow$
\dbtext{television studio} (15853) $\rightarrow$
\dbtext{helsinki} (3075) $\rightarrow$
\dbtext{capital finland} (3074).
The justification is given by the following sentences.
\begin{enumerate}[noitemsep,leftmargin=8mm,topsep=0.5mm]
 \item If you want to pass sentence then you should tell somebody their punishment
 \item passing sentence requires words
 \item I  can word this
 \item Somewhere someone can be is the office building
 \item the television studio is part of the office building
 \item You are likely to find a television studio in Helsinki.
 \item helsinki is the capital of finland
\end{enumerate}

\paragraph{Big Weakly Connected Component.}
In the following section we will see that the longest geodesic path in the
undirected graph induced by the assertions of the English language with 
positive score and positive polarity is $16$.
This fact, together with the decomposition of the weakly connected components
of the graph that is induced by the assertions of the English language 
with positive score and positive polarity and which is presented in Chapter \ref{chapter:components} 
(Table \ref{tbl:distribution:component:positive:weak} and Figure \ref{fig:weakly-connected:positive}),
it follows that the diameter of this component is equal to $16$.
One detailed instance admitting this diameter is given in the following section
which describes the longest geodesic path in the graph induced by the assertions
with positive score and polarity.

%
%
\subsubsection{Undirected Graph}
The longest geodesic path in this case is $16$ and connects
the concepts \dbtext{anti-charm quark} (15922) and \dbtext{double-breasted de fursac jacket} (328674).
The full sequence for the path is given by
\dbtext{anti-charm quark} (15922) $\rightarrow$
\dbtext{c c-bar meson} (15620) $\rightarrow$
\dbtext{charm quark} (15616) $\rightarrow$
\dbtext{charm lambda-plus} (15621) $\rightarrow$
\dbtext{down quark} (15659) $\rightarrow$
\dbtext{neutron} (15664) $\rightarrow$
\dbtext{universe} (1639) $\genfrac{}{}{0pt}{}{\nearrow}{\searrow}$
\begin{tabular}{lcl}
\dbtext{something} (5) & $\rightarrow$ & \dbtext{dress} (562)\\
\dbtext{god} (4277)    & $\rightarrow$ & \dbtext{armor} (30372)
\end{tabular}  
$\genfrac{}{}{0pt}{}{\searrow}{\nearrow}$
\dbtext{vest} (15219) $\rightarrow$
\dbtext{waistcoat} (15873) $\rightarrow$
\dbtext{single-breasted three-piece suit} (311438) $\rightarrow$
\dbtext{single-breasted jacket} (311437) $\rightarrow$
\dbtext{single-breasted two-piece suit} (311439) $\rightarrow$
\dbtext{man suit pant} (311447) $\rightarrow$
\dbtext{double-breasted two-piece de fursac suit} (328675) $\rightarrow$
\dbtext{double-breasted de fursac jacket} (328674).
The justification is given by the following sentences.
\begin{enumerate}[noitemsep,leftmargin=8mm,topsep=0.5mm]
 \item an anti-charm quark is part of a c c-bar meson
 \item a charm quark is part of a c c-bar meson
 \item a charm quark is part of a charmed lambda-plus
 \item a down quark is part of a charmed lambda-plus
 \item a down quark is part of a neutron
 \item Something you find in the universe is neutrons
 \item Somewhere something can be is the universe / The universe is created by God.
 \item A dress is something worn on the body / armor is a type of god
 \item vest is to dress / armor is related to vest
 \item vest is a type of waistcoat
 \item a waistcoat is part of a single-breasted three-piece suit
 \item a single-breasted jacket is part of a single-breasted three-piece suit
 \item a single-breasted jacket is part of a single-breasted two-piece suit
 \item some men's suit pants is part of a single-breasted two-piece suit
 \item some men's suit pants is part of a double-breasted two-piece de fursac suit
 \item a double-breasted de fursac jacket is part of a double-breasted two-piece de fursac suit
\end{enumerate}

\subsection{Both Polarities}
Finally, in the case where both polarities are allowed on the induced graphs,
the longest geodesic paths are the same as in the case of the graphs induced
by assertions of positive polarity only.

\subsubsection{Directed Graph}
In the case of the directed graph the path that was returned was different only
in the final edge, where we encountered the connection 
\dbtext{muslim} (8663) $\rightarrow$
\dbtext{believe jesus god} (51958), admitted by the sentence
\dbtext{Muslims do not believe that Jesus is god.}.

\paragraph{Big Strongly Connected Component.}
The diameter of this big directed component is equal to $12$.
The full sequence of the diameter returned by \igraph is given by
\dbtext{sixth day week} (2754) $\rightarrow$
\dbtext{Friday} (2755) $\rightarrow$
\dbtext{day week} (203694) $\rightarrow$
\dbtext{bathroom} (1007) $\rightarrow$
\dbtext{library} (68) $\rightarrow$
\dbtext{learn} (401) $\rightarrow$
\dbtext{pride} (14745) $\rightarrow$
\dbtext{tell many person} (427796) $\rightarrow$
\dbtext{evangelist} (98420) $\rightarrow$
\dbtext{fundamentalist} (176617) $\rightarrow$
\dbtext{taliban} (119866) $\rightarrow$
\dbtext{islamist} (119867) $\rightarrow$
\dbtext{muslim} (8663).
The justification is given by the following sentences.
\begin{enumerate}[noitemsep,leftmargin=8mm,topsep=0.5mm]
 \item The sixth day of the week is Friday.
 \item Friday is a kind of day of the week.
 \item You are not likely to find a day of the week in the bathroom.
 \item You are likely to find a bathroom in a library
 \item library is for learning.
 \item Sometimes learning causes pride.
 \item pride would make you want to tell many people about something
 \item telling many people about something is for evangelists.
 \item evangelist is a type of fundamentalist
 \item You are likely to find fundamentalists in the Taliban.
 \item The Taliban are Islamists.
 \item an Islamist is a kind of Muslim.
\end{enumerate}

\smallskip

In the case of the equivalent undirected graph, the diameter is $7$
and is identical to the one found in the big strongly connected component
that was found in the graph induced by the assertions with positive polarity only.
Please refer to that case for the complete description.

\paragraph{Big Weakly Connected Component.}
In the following section we will see that the longest geodesic path in the
undirected graph induced by the assertions of the English language with 
positive score and any polarity is $16$.
This fact, together with the decomposition of the weakly connected components
of the graph that is induced by the assertions of the English language 
with positive score and positive polarity and which is presented in Chapter \ref{chapter:components} 
(Table \ref{tbl:distribution:component:weak} and 
Figure \ref{fig:weakly-connected:positive} \footnote{Even though 
Figure \ref{fig:weakly-connected:positive} refers to the weakly connected components with positive
polarity only, it still suffices for our purposes, since the connected components that could be candidates
for giving a possible different longest geodesic path, are the same.}),
it follows that the diameter of this component is equal to $16$.
One detailed instance admitting this diameter 
has already been given earlier in the examination of the undirected graph induced by the 
vertices that appear in the big strongly connected component of \conceptnet induced
by the assertions of the English language with positive score and polarity.

\subsubsection{Undirected Graph}
In the case of the undirected graph the path that was returned was entirely
identical to the case of the undirected graph induced by assertions of
positive polarity only.

\section{Summary}
Tables \ref{tbl:shortest-path:summary:entire-graphs} and
\ref{tbl:shortest-path:summary:big-components} give a brief summary of the results
related to shortest paths that were presented earlier.

\begin{table}[ht]
\caption{The average shortest path length of the graphs induced by the assertions of the English
language with positive score and various polarities, together with the length of the longest geodesic
path in each graph. Recall that the graphs are disconnected, and hence the diameter is infinite 
in every case. Moreover, the last column indicates whether the length of the longest geodesic 
path is unique in the respective graph or not.}\label{tbl:shortest-path:summary:entire-graphs}
\centering
\begin{tabular}{|l|c|c|c|c|}\hline
\multicolumn{1}{|c|}{\multirow{2}{*}{polarity}} & directed & average & longest & \multicolumn{1}{c|}{\multirow{2}{*}{unique}} \\
         &  graph & shortest path & geodesic path &  \\\hline\hline
negative & \xmark & $3.863$ & $18$ & \xmark \\\hline
negative & \cmark & $6.737$ & $24$ & \cmark \\\hline\hline
positive & \xmark & $4.330$ & $16$ & \xmark \\\hline
positive & \cmark & $4.811$ & $15$ & \xmark \\\hline\hline
both     & \xmark & $4.280$ & $16$ & \xmark \\\hline
both     & \cmark & $4.772$ & $15$ & \xmark \\\hline
\end{tabular}
\end{table}

\begin{table}[ht]
\caption{The average shortest path length of the big components that arise in the graphs
induced by the assertions of the English language with positive score and various polarities, 
together with the length of the diameter in every case. 
The last column indicates whether the diameter is unique in the respective component 
or not.}\label{tbl:shortest-path:summary:big-components}
\centering
\begin{tabular}{|l|c|c|c|c|c|}\hline
\multicolumn{1}{|c|}{\multirow{2}{*}{polarity}} & connected & oriented & average & \multirow{2}{*}{diameter} & \multicolumn{1}{c|}{\multirow{2}{*}{unique}} \\
         & component &  edges & shortest path &  & \\\hline\hline
negative &    weakly & \xmark & $3.864$ & $18$ & \xmark \\\hline
negative &  strongly & \cmark & $6.428$ & $22$ & \xmark \\\hline
negative &  strongly & \xmark & $3.537$ & $ 9$ & \xmark \\\hline\hline
positive &    weakly & \xmark & $4.330$ & $16$ & \xmark \\\hline
positive &  strongly & \cmark & $4.205$ & $12$ & \xmark \\\hline
positive &  strongly & \xmark & $3.337$ & $ 7$ & \xmark \\\hline\hline
both     &    weakly & \xmark & $4.280$ & $16$ & \xmark \\\hline
both     &  strongly & \cmark & $4.167$ & $12$ & \xmark \\\hline
both     &  strongly & \xmark & $3.291$ & $ 7$ & \xmark \\\hline
\end{tabular}
\end{table}



\chapter{Cliques}\label{chapter:cliques}
In this section we give an overview of the maximal cliques that we encounter in \conceptnet.
The edges that are retained are those that come from assertions with positive score.
In every case we examine the induced undirected graph with no self-loops in order 
to determine the cliques. 
Table \ref{tbl:cliques:distribution} presents the number of cliques found in every case.

\begin{table}[ht]
\caption{The number of maximal cliques as well as the distribution of the maximal cliques for 
various frequency ranges and both polarities.
All relations are allowed but the scores of the assertions have to be positive. 
As usual the assertions are those in the English language.}\label{tbl:cliques:distribution}
\begin{center}
\resizebox{\textwidth}{!}{
\begin{tabular}{|c|c||c||r|r|r|r|r|r|r|r|r|r|}\hline
\multirow{2}{*}{polarity}
& range for        & \multicolumn{1}{c||}{number of} & \multicolumn{10}{c|}{number of maximal cliques of size}\\\cline{4-13}
& frequency values & \multicolumn{1}{c||}{maximal cliques} & \multicolumn{1}{c|}{3} & \multicolumn{1}{c|}{4} & \multicolumn{1}{c|}{5} & \multicolumn{1}{c|}{6} & \multicolumn{1}{c|}{7} & \multicolumn{1}{c|}{8} & \multicolumn{1}{c|}{9} & \multicolumn{1}{c|}{10} & \multicolumn{1}{c|}{11} & \multicolumn{1}{c|}{12} \\\hline\hline
\multirow{11}{*}{\rotatebox{90}{negative}}
& \{-10\}             &       0 &      0 &      0 &      0 &     0 &     0 &   0 &   0 &   0 &  0 &  0 \\\cline{2-13}
& \{-10, -9\}         &       0 &      0 &      0 &      0 &     0 &     0 &   0 &   0 &   0 &  0 &  0 \\\cline{2-13}
& \{-10, -9, -8\}     &       0 &      0 &      0 &      0 &     0 &     0 &   0 &   0 &   0 &  0 &  0 \\\cline{2-13}
& \{-10, \ldots, -7\} &       0 &      0 &      0 &      0 &     0 &     0 &   0 &   0 &   0 &  0 &  0 \\\cline{2-13}
& \{-10, \ldots, -6\} &       0 &      0 &      0 &      0 &     0 &     0 &   0 &   0 &   0 &  0 &  0 \\\cline{2-13}
& \{-10, \ldots, -5\} &     835 &    779 &     56 &      0 &     0 &     0 &   0 &   0 &   0 &  0 &  0 \\\cline{2-13}
& \{-10, \ldots, -4\} &     835 &    779 &     56 &      0 &     0 &     0 &   0 &   0 &   0 &  0 &  0 \\\cline{2-13}
& \{-10, \ldots, -3\} &     835 &    779 &     56 &      0 &     0 &     0 &   0 &   0 &   0 &  0 &  0 \\\cline{2-13}
& \{-10, \ldots, -2\} &     836 &    780 &     56 &      0 &     0 &     0 &   0 &   0 &   0 &  0 &  0 \\\cline{2-13}
& \{-10, \ldots, -1\} &     836 &    780 &     56 &      0 &     0 &     0 &   0 &   0 &   0 &  0 &  0 \\\cline{2-13}
& \{-10, \ldots, 0\}  &     836 &    780 &     56 &      0 &     0 &     0 &   0 &   0 &   0 &  0 &  0 \\\hline
\hline
\multirow{11}{*}{\rotatebox{90}{positive}}
& \{0, \ldots, 10\}   & 107,100 & 47,026 & 28,655 & 17,884 & 9,046 & 3,083 & 955 & 314 & 113 & 23 &  1 \\\cline{2-13}
& \{1, \ldots, 10\}   & 107,100 & 47,026 & 28,655 & 17,884 & 9,046 & 3,083 & 955 & 314 & 113 & 23 &  1 \\\cline{2-13}
& \{2, \ldots, 10\}   & 107,100 & 47,026 & 28,655 & 17,884 & 9,046 & 3,083 & 955 & 314 & 113 & 23 &  1 \\\cline{2-13}
& \{3, \ldots, 10\}   & 107,097 & 47,024 & 28,655 & 17,883 & 9,046 & 3,083 & 955 & 314 & 113 & 23 &  1 \\\cline{2-13}
& \{4, \ldots, 10\}   & 107,097 & 47,024 & 28,655 & 17,883 & 9,046 & 3,083 & 955 & 314 & 113 & 23 &  1 \\\cline{2-13}
& \{5, \ldots, 10\}   & 103,946 & 45,997 & 27,805 & 17,181 & 8,620 & 2,956 & 948 & 305 & 112 & 21 &  1 \\\cline{2-13}
& \{6, \ldots, 10\}   &      15 &     15 &      0 &      0 &     0 &     0 &   0 &   0 &   0 &  0 &  0 \\\cline{2-13}
& \{7, \ldots, 10\}   &      15 &     15 &      0 &      0 &     0 &     0 &   0 &   0 &   0 &  0 &  0 \\\cline{2-13}
& \{8, 9, 10\}        &       8 &      8 &      0 &      0 &     0 &     0 &   0 &   0 &   0 &  0 &  0 \\\cline{2-13}
& \{9, 10\}           &       0 &      0 &      0 &      0 &     0 &     0 &   0 &   0 &   0 &  0 &  0 \\\cline{2-13}
& \{10\}              &       0 &      0 &      0 &      0 &     0 &     0 &   0 &   0 &   0 &  0 &  0 \\\hline
\end{tabular}
}
\end{center}
\end{table}

\section{Maximum Clique: All Relations, Positive Polarity}
There is a unique maximum clique when all relations are allowed in the induced
graph of the English assertions with positive score.
The largest maximal clique has size $12$ and relates the concepts 
\dbtext{person}, 
\dbtext{build}, 
\dbtext{house}, 
\dbtext{home}, 
\dbtext{apartment}, 
\dbtext{room}, 
\dbtext{live room}, 
\dbtext{couch}, 
\dbtext{table}, 
\dbtext{chair}, 
\dbtext{cat}, and
\dbtext{dog}. 
The interpretation (\emph{surface form}) of 
\dbtext{live room} should be \dbtext{living room}, or \dbtext{in a living room}, etc., 
\dbtext{build} should be interpreted as \dbtext{a building}, etc.

\section{On the Maximal Cliques with Negative Polarity}
The clique (triangle) that is introduced when the range for the frequency values is expanded from 
$\{-10, \ldots, -3\}$ to $\{-10, \ldots, -2\}$ (see Table \ref{tbl:cliques:distribution}) is composed of the concepts
\dbtext{person} (9), \dbtext{chore} (22621), and \dbtext{fun} (134); inside the parentheses
we can read the IDs of the specific concepts. The justification comes from the sentences
\dbtext{a person doesn't want to do chores.}, \dbtext{People do things that are not fun.},
and \dbtext{chores are rarely fun.}.
Regarding the maximal cliques of size $4$, below we give the list with all $56$ of them.
Again, inside the parentheses we can read the ID of each concept.
{\footnotesize
\begin{verbatim}
scarce (196339), lot (4905), many (6989), much (7917)
second (130981), year (2709), hour (2762), minute (2764)
average (76629), bad (2226), good (2666), best (20709)
where (54686), who (23034), why (5469), when (38265)
middle (52077), start (44963), begin (3695), end (10507)
middle (52077), side (17836), top (7514), bottom (5887)
middle (52077), side (17836), front (2423), back (15583)
far (37745), near (25285), here (6352), away (29340)
sight (18526), smell (1172), taste (14093), touch (5106)
sight (18526), smell (1172), taste (14093), sound (2660)
even (15946), night (8677), morning (15749), afternoon (15914)
even (15946), night (8677), morning (15749), day (2759)
taste (14093), smell (1172), hear (9269), touch (5106)
few (8145), lot (4905), many (6989), much (7917)
many (6989), much (7917), lot (4905), little (8268)
spring (5537), winter (1431), summer (1437), fall (9975)
touch (5106), see (1161), smell (1172), hear (9269)
blue (4808), red (2614), yellow (2616), green (2637)
woman (895), man (7), girl (876), boy (5976)
plant (716), human (80), animal (902), god (4277)
plant (716), human (80), animal (902), die (1227)
person (9), plural (28735), child (178), eye (1160)
person (9), slave (27415), pay (1473), free (19126)
person (9), deaf (23417), listen (75), hear (9269)
person (9), best (20709), bad (2226), good (2666)
person (9), female (15676), man (7), boy (5976)
person (9), program language (13345), computer (467), hot (1438)
person (9), know (13183), right (6079), wrong (2664)
person (9), know (13183), understand (1858), unknown (5613)
person (9), write paper (8025), computer (467), telephone (1790)
person (9), boy (5976), man (7), girl (876)
person (9), wait (2858), money (1240), long hair (5916)
person (9), wallet (2466), money (1240), long hair (5916)
person (9), crime (1803), sleep (425), lie (1395)
person (9), telephone (1790), computer (467), television (1298)
person (9), kill (1466), live (580), die (1227)
person (9), hot (1438), computer (467), television (1298)
person (9), lie (1395), talk (394), dog (537)
person (9), television (1298), computer (467), book (2033)
person (9), clean (344), dirty (170), gerbil (14223)
person (9), clean (344), dirty (170), time (2494)
person (9), bed (156), examination (121), conscious (23506)
person (9), bed (156), examination (121), money (1240)
person (9), examination (121), long hair (5916), money (1240)
person (9), human (80), like play (203698), animal (902)
person (9), human (80), face (8835), money (1240)
person (9), human (80), long hair (5916), money (1240)
person (9), human (80), god (4277), animal (902)
person (9), human (80), die (1227), animal (902)
person (9), human (80), animal (902), fly (9215)
person (9), human (80), computer (467), conscious (23506)
person (9), human (80), computer (467), fly (9215)
person (9), human (80), computer (467), book (2033)
person (9), human (80), computer (467), house (652)
person (9), man (7), animal (902), fly (9215)
person (9), man (7), animal (902), god (4277)
\end{verbatim}
}

\section{On the Maximal Cliques with Positive Polarity}
Table \ref{tbl:cliques:positive:high-frequency} presents the maximal cliques in the case of positive polarity 
and high frequency. In this table, the frequency values are in the set $\{7, 8, 9, 10\}$.
More importantly, the first $8$ cliques presented in Table \ref{tbl:cliques:positive:high-frequency}
are maximal cliques from assertions with very high frequency values; i.e.~the values for the frequencies
are in the set $\{8, 9, 10\}$.

\begin{table}[ht]
\caption{Concepts participating in maximal cliques with positive polarity and high frequency (the 
values of the frequency are in the range $\{7, \ldots, 10\}$).
The cliques are obtained from assertions in the English language with positive score.
Cliques $1$-$8$ are obtained when the frequency values range in $\{8, 9, 10\}$,
while cliques $9$-$15$ are obtained when the frequency values range in $\{7, \ldots, 10\}$.}\label{tbl:cliques:positive:high-frequency}
\begin{center}
\resizebox{\textwidth}{!}{
\begin{tabular}{|r||r|l||c|c|c||c||c||c||c|c||c||c||c|c|c|c|c||c|c|c|}\cline{1-18}\cline{20-20}
\multicolumn{3}{|c||}{concept} & \multicolumn{15}{c||}{clique} & & \multicolumn{1}{|c|}{\multirow{2}{*}{\cmark}} \\\cline{1-18}
\multicolumn{1}{|c||}{\#} & \multicolumn{1}{c|}{id} & \multicolumn{1}{c||}{description}
                              &      1 &      2 &      3 &      4 &      5 &      6 &      7 &      8 &      9 &     10 &     11 &     12 &     13 &     14 &     15 & &   \\\cline{1-18}\cline{20-20}\cline{1-18}\cline{20-20}
 1 &     9 &           person &        &        &        &        &        &        & \cmark & \cmark &        &        & \cmark & \cmark & \cmark & \cmark & \cmark & & 7 \\\cline{1-18}\cline{20-20}
 2 &    33 &             tree &        &        &        &        &        &        &        &        &        & \cmark &        &        &        &        &        & & 1 \\\cline{1-18}\cline{20-20}
 3 &    80 &            human &        &        &        &        &        &        & \cmark & \cmark &        &        & \cmark & \cmark & \cmark & \cmark & \cmark & & 7 \\\cline{1-18}\cline{20-20}
 4 &   137 &               it &        &        &        &        &        &        &        &        & \cmark &        &        &        &        &        &        & & 1 \\\cline{1-18}\cline{20-20}
 5 &   716 &            plant &        &        &        &        &        & \cmark &        &        &        &        &        &        &        &        &        & & 1 \\\cline{1-18}\cline{20-20}
 6 &  1114 & national highway & \cmark & \cmark & \cmark &        &        &        &        &        &        &        &        &        &        &        &        & & 3 \\\cline{1-18}\cline{20-20}
 7 &  1443 &             king &        &        &        & \cmark &        &        &        &        &        &        &        &        &        &        &        & & 1 \\\cline{1-18}\cline{20-20}
 8 &  1577 &           window &        &        &        &        & \cmark &        &        &        &        &        &        &        &        &        &        & & 1 \\\cline{1-18}\cline{20-20}
 9 &  1776 &            glass &        &        &        &        & \cmark &        &        &        &        &        &        &        &        &        &        & & 1 \\\cline{1-18}\cline{20-20}
10 &  2637 &            green &        &        &        &        &        & \cmark &        &        &        & \cmark &        &        &        &        &        & & 2 \\\cline{1-18}\cline{20-20}
11 &  3571 &            leave &        &        &        &        &        & \cmark &        &        &        & \cmark &        &        &        &        &        & & 2 \\\cline{1-18}\cline{20-20}
12 &  3663 &            heavy &        &        &        &        &        &        &        &        & \cmark &        &        &        &        &        &        & & 1 \\\cline{1-18}\cline{20-20}
13 &  6491 &            metal &        &        &        &        &        &        &        &        & \cmark &        &        &        &        &        &        & & 1 \\\cline{1-18}\cline{20-20}
14 &  8689 &      wear clothe &        &        &        &        &        &        &        & \cmark &        &        &        &        &        &        &        & & 1 \\\cline{1-18}\cline{20-20}
15 &  9693 &            queen &        &        &        & \cmark &        &        &        &        &        &        &        &        &        &        &        & & 1 \\\cline{1-18}\cline{20-20}
16 & 18322 & write right hand &        &        &        &        &        &        & \cmark &        &        &        &        &        &        &        &        & & 1 \\\cline{1-18}\cline{20-20}
17 & 18735 &     eat together &        &        &        &        &        &        &        &        &        &        &        &        &        &        & \cmark & & 1 \\\cline{1-18}\cline{20-20}
18 & 20788 &       avoid pain &        &        &        &        &        &        &        &        &        &        &        &        &        & \cmark &        & & 1 \\\cline{1-18}\cline{20-20}
19 & 21317 &        eat table &        &        &        &        &        &        &        &        &        &        &        &        & \cmark &        &        & & 1 \\\cline{1-18}\cline{20-20}
20 & 21364 &   live apartment &        &        &        &        &        &        &        &        &        &        &        & \cmark &        &        &        & & 1 \\\cline{1-18}\cline{20-20}
21 & 22671 &            clear &        &        &        &        & \cmark &        &        &        &        &        &        &        &        &        &        & & 1 \\\cline{1-18}\cline{20-20}
22 & 41958 &      live castle &        &        &        & \cmark &        &        &        &        &        &        &        &        &        &        &        & & 1 \\\cline{1-18}\cline{20-20}
23 & 69743 &  federal highway & \cmark & \cmark & \cmark &        &        &        &        &        &        &        &        &        &        &        &        & & 3 \\\cline{1-18}\cline{20-20}
24 & 69746 &    well maintain &        &        & \cmark &        &        &        &        &        &        &        &        &        &        &        &        & & 1 \\\cline{1-18}\cline{20-20}
25 & 69747 &      wide smooth &        & \cmark &        &        &        &        &        &        &        &        &        &        &        &        &        & & 1 \\\cline{1-18}\cline{20-20}
26 & 69750 &     top concrete & \cmark &        &        &        &        &        &        &        &        &        &        &        &        &        &        & & 1 \\\cline{1-18}\cline{20-20}
27 & 81916 &   disagree other &        &        &        &        &        &        &        &        &        &        & \cmark &        &        &        &        & & 1 \\\cline{1-18}\cline{20-20}
\multicolumn{20}{c}{} \\\cline{4-18}
\multicolumn{3}{r|}{clique size} & 3 & 3 & 3 & 3 & 3 & 3 & 3 & 3 & 3 & 3 & 3 & 3 & 3 & 3 & 3 \\\cline{4-18}
\multicolumn{3}{r|}{frequency range}   & \multicolumn{8}{c||}{8--10} & \multicolumn{7}{c||}{7--10} \\\cline{4-18}
\end{tabular}
}
\end{center}
\end{table}

Table \ref{tbl:cliques:positive:moderate-frequency} presents the largest and second largest 
maximal cliques in the case of positive polarity but with moderate frequency. 
In this table, the frequency values are in the set $\{4, \ldots, 10\}$.
Recall that the largest maximal clique is composed of the $12$ concepts
\dbtext{person}, 
\dbtext{apartment}, 
\dbtext{home}, 
\dbtext{house}, 
\dbtext{build}, 
\dbtext{room}, 
\dbtext{live room}, 
\dbtext{cat}, 
\dbtext{couch}, 
\dbtext{table}, 
\dbtext{dog}, and 
\dbtext{chair}.
This is the first clique presented in Table \ref{tbl:cliques:positive:moderate-frequency}.
Figure \ref{fig:cliques:positive:maxCliques} presents the graph induced by the concepts
that appear in Table \ref{tbl:cliques:positive:moderate-frequency}.

\begin{table}[t]
\caption{
Concepts participating in maximal cliques with positive polarity and moderate frequency (the 
values of the frequency are in the range $\{4, \ldots, 10\}$).
The cliques are obtained from assertions in the English language with positive score.
Moreover, cliques 23 and 24 are obtained when the frequency ranges in $\{4, \ldots, 10\}$,
while all the other cases can also be obtained when the frequency ranges in $\{5, \ldots, 10\}$
as well.}\label{tbl:cliques:positive:moderate-frequency}
\begin{center}
\resizebox{\textwidth}{!}{
\begin{tabular}{|r||r|l||c||c|c|c|c|c|c|c||c|c||c|c|c|c|c|c|c|c|c|c|c||c||c|c||c|r|}\cline{1-27}\cline{29-29}
\multicolumn{3}{|c||}{concept} & \multicolumn{24}{c||}{clique} & & \multicolumn{1}{|c|}{\multirow{2}{*}{\cmark}} \\\cline{1-27}
\multicolumn{1}{|c||}{\#} & \multicolumn{1}{c|}{id} & \multicolumn{1}{c||}{description} &
                                 1 &      2 &      3 &      4 &      5 &      6 &      7 &      8 &      9 &     10 &     11 &     12 &     13 &     14 &     15 &     16 &     17 &     18 &     19 &     20 &     21 &     22 &     23 &     24 &   &     \\\cline{1-27}\cline{29-29}\cline{1-27}\cline{29-29}
  1 &      5 & something  &        &        &        &        &        &        &        &        &        &        & \cmark & \cmark & \cmark & \cmark & \cmark & \cmark & \cmark & \cmark & \cmark & \cmark & \cmark &        & \cmark & \cmark &   &  13 \\\cline{1-27}\cline{29-29}
  2 &      7 & man        &        &        &        &        &        &        &        &        &        &        &        &        &        &        &        &        &        &        &        &        &        & \cmark &        &        &   &   1 \\\cline{1-27}\cline{29-29}
  3 &      9 & person     & \cmark & \cmark & \cmark & \cmark & \cmark & \cmark & \cmark & \cmark & \cmark & \cmark & \cmark & \cmark & \cmark & \cmark & \cmark & \cmark & \cmark & \cmark & \cmark & \cmark & \cmark & \cmark & \cmark & \cmark &   &  24 \\\cline{1-27}\cline{29-29}
  4 &     21 & town       &        &        &        & \cmark & \cmark & \cmark & \cmark & \cmark &        &        &        &        &        &        &        &        &        &        &        & \cmark & \cmark &        & \cmark & \cmark &   &   9 \\\cline{1-27}\cline{29-29}
  5 &     67 & love       &        &        &        &        &        &        &        &        &        &        &        &        &        &        &        &        &        &        &        &        &        & \cmark &        &        &   &   1 \\\cline{1-27}\cline{29-29}
  6 &     68 & library    &        &        &        &        &        &        &        &        &        &        &        & \cmark &        &        & \cmark &        & \cmark &        & \cmark & \cmark & \cmark &        &        &        &   &   6 \\\cline{1-27}\cline{29-29}
  7 &     73 & school     &        &        &        &        &        &        &        &        &        &        &        & \cmark &        &        & \cmark &        & \cmark &        & \cmark & \cmark & \cmark &        &        &        &   &   6 \\\cline{1-27}\cline{29-29}
  8 &     80 & human      &        &        &        &        & \cmark &        & \cmark & \cmark &        &        &        &        &        &        &        &        &        &        &        &        &        & \cmark &        &        &   &   4 \\\cline{1-27}\cline{29-29}
  9 &    156 & bed        &        & \cmark &        &        &        &        &        &        &        &        &        &        & \cmark &        &        &        &        &        &        &        &        & \cmark &        &        &   &   3 \\\cline{1-27}\cline{29-29}
 10 &    178 & child      &        &        &        &        &        &        &        &        &        &        &        &        &        &        &        &        &        &        &        &        &        & \cmark &        &        &   &   1 \\\cline{1-27}\cline{29-29}
 11 &    467 & computer   &        &        &        &        &        &        &        &        &        &        &        &        &        &        &        & \cmark & \cmark &        &        &        & \cmark &        &        &        &   &   3 \\\cline{1-27}\cline{29-29}
 12 &    537 & dog        & \cmark &        & \cmark & \cmark & \cmark & \cmark & \cmark & \cmark &        & \cmark &        &        &        &        &        &        &        &        &        &        &        &        &        &        &   &   8 \\\cline{1-27}\cline{29-29}
 13 &    580 & live       &        &        &        &        &        &        &        & \cmark &        &        &        &        &        &        &        &        &        &        &        &        &        &        &        &        &   &   1 \\\cline{1-27}\cline{29-29}
 14 &    596 & chair      & \cmark &        & \cmark & \cmark & \cmark & \cmark & \cmark &        & \cmark & \cmark &        &        &        & \cmark & \cmark &        &        & \cmark & \cmark & \cmark &        &        & \cmark & \cmark &   &  15 \\\cline{1-27}\cline{29-29}
 15 &    616 & cat        & \cmark & \cmark & \cmark & \cmark &        & \cmark &        &        & \cmark & \cmark &        &        &        &        &        &        &        &        &        &        &        &        &        &        &   &   7 \\\cline{1-27}\cline{29-29}
 16 &    652 & house      & \cmark & \cmark & \cmark & \cmark & \cmark & \cmark & \cmark & \cmark & \cmark & \cmark & \cmark & \cmark & \cmark & \cmark & \cmark & \cmark & \cmark & \cmark & \cmark & \cmark & \cmark &        &        &        &   &  21 \\\cline{1-27}\cline{29-29}
 17 &    688 & hotel      &        &        &        &        &        &        &        &        &        &        &        &        &        &        &        &        &        &        &        &        &        &        &        & \cmark &   &   1 \\\cline{1-27}\cline{29-29}
 18 &    876 & girl       &        &        &        &        &        &        &        &        &        &        &        &        &        &        &        &        &        &        &        &        &        & \cmark &        &        &   &   1 \\\cline{1-27}\cline{29-29}
 19 &    895 & woman      &        &        &        &        &        &        &        &        &        &        &        &        &        &        &        &        &        &        &        &        &        & \cmark &        &        &   &   1 \\\cline{1-27}\cline{29-29}
 20 &   1013 & city       &        &        &        &        &        & \cmark & \cmark & \cmark &        &        &        &        &        &        &        &        &        &        &        & \cmark & \cmark &        & \cmark & \cmark &   &   7 \\\cline{1-27}\cline{29-29}
 21 &   1043 & desk       &        &        &        &        &        &        &        &        & \cmark & \cmark & \cmark & \cmark &        & \cmark & \cmark & \cmark & \cmark & \cmark & \cmark &        &        &        &        &        &   &  10 \\\cline{1-27}\cline{29-29}
 22 &   1045 & home       & \cmark & \cmark & \cmark & \cmark & \cmark & \cmark & \cmark & \cmark & \cmark & \cmark & \cmark &        & \cmark & \cmark &        & \cmark &        & \cmark &        &        &        & \cmark &        &        &   &  16 \\\cline{1-27}\cline{29-29}
 23 &   1072 & couch      & \cmark & \cmark &        &        &        &        &        &        &        &        &        &        &        &        &        &        &        &        &        &        &        &        &        &        &   &   2 \\\cline{1-27}\cline{29-29}
 24 &   1104 & build      & \cmark & \cmark & \cmark & \cmark & \cmark & \cmark & \cmark & \cmark & \cmark & \cmark & \cmark & \cmark & \cmark & \cmark & \cmark & \cmark & \cmark & \cmark & \cmark & \cmark & \cmark &        & \cmark & \cmark &   &  23 \\\cline{1-27}\cline{29-29}
 25 &   1111 & restaurant &        &        &        &        &        &        &        &        &        &        &        &        &        &        &        &        &        &        &        &        &        &        & \cmark & \cmark &   &   2 \\\cline{1-27}\cline{29-29}
 26 &   1372 & bedroom    &        &        & \cmark & \cmark & \cmark &        &        &        & \cmark & \cmark & \cmark &        & \cmark & \cmark &        & \cmark &        & \cmark &        &        &        & \cmark &        &        &   &  11 \\\cline{1-27}\cline{29-29}
 27 &   1414 & store      &        &        &        &        &        &        &        &        &        &        &        &        &        &        &        &        &        &        &        &        &        &        & \cmark &        &   &   1 \\\cline{1-27}\cline{29-29}
 28 &   2033 & book       &        &        &        &        &        &        &        &        &        &        & \cmark & \cmark & \cmark &        &        & \cmark & \cmark &        &        &        & \cmark &        &        &        &   &   6 \\\cline{1-27}\cline{29-29}
 29 &   2480 & room       & \cmark & \cmark & \cmark & \cmark & \cmark & \cmark & \cmark & \cmark & \cmark & \cmark & \cmark & \cmark & \cmark & \cmark & \cmark & \cmark & \cmark & \cmark & \cmark & \cmark & \cmark &        & \cmark & \cmark &   &  23 \\\cline{1-27}\cline{29-29}
 30 &   2825 & sex        &        &        &        &        &        &        &        &        &        &        &        &        &        &        &        &        &        &        &        &        &        & \cmark &        &        &   &   1 \\\cline{1-27}\cline{29-29}
 31 &   4570 & place      &        &        &        &        &        &        &        &        &        &        &        &        &        &        &        &        &        & \cmark & \cmark & \cmark &        &        & \cmark & \cmark &   &   5 \\\cline{1-27}\cline{29-29}
 32 &   5558 & bar        &        &        &        &        &        &        &        &        &        &        &        &        &        &        &        &        &        &        &        &        &        &        & \cmark & \cmark &   &   2 \\\cline{1-27}\cline{29-29}
 33 &   5581 & live room  & \cmark & \cmark &        &        &        &        &        &        &        &        &        &        &        &        &        &        &        &        &        &        &        &        &        &        &   &   2 \\\cline{1-27}\cline{29-29}
 34 &   5665 & table      & \cmark & \cmark & \cmark &        &        &        &        &        & \cmark & \cmark & \cmark & \cmark & \cmark & \cmark & \cmark & \cmark & \cmark & \cmark & \cmark &        &        &        &        &        &   &  14 \\\cline{1-27}\cline{29-29}
 35 &   6062 & floor      &        &        &        &        &        &        &        &        & \cmark &        & \cmark & \cmark & \cmark & \cmark & \cmark &        &        &        &        &        &        &        &        &        &   &   6 \\\cline{1-27}\cline{29-29}
 36 &  19557 & apartment  & \cmark & \cmark & \cmark & \cmark & \cmark & \cmark & \cmark & \cmark &        &        &        &        &        &        &        &        &        &        &        &        &        &        &        &        &   &   8 \\\cline{1-27}\cline{29-29}
\multicolumn{29}{c}{} \\\cline{4-27}
\multicolumn{3}{r|}{clique size} & 12 & 11 & 11 & 11 & 11 & 11 & 11 & 11 & 11 & 11 & 11 & 11 & 11 & 11 & 11 & 11 & 11 & 11 & 11 & 11 & 11 & 11 & 11 & 11 \\\cline{4-27}
\multicolumn{3}{r|}{frequency range}   & \multicolumn{22}{c||}{5--10} & \multicolumn{2}{c||}{4--10} \\\cline{4-27}
\end{tabular}
}
\end{center}
\end{table}

\begin{figure}[ht]
\begin{center}
\includegraphics[width=0.9\columnwidth]{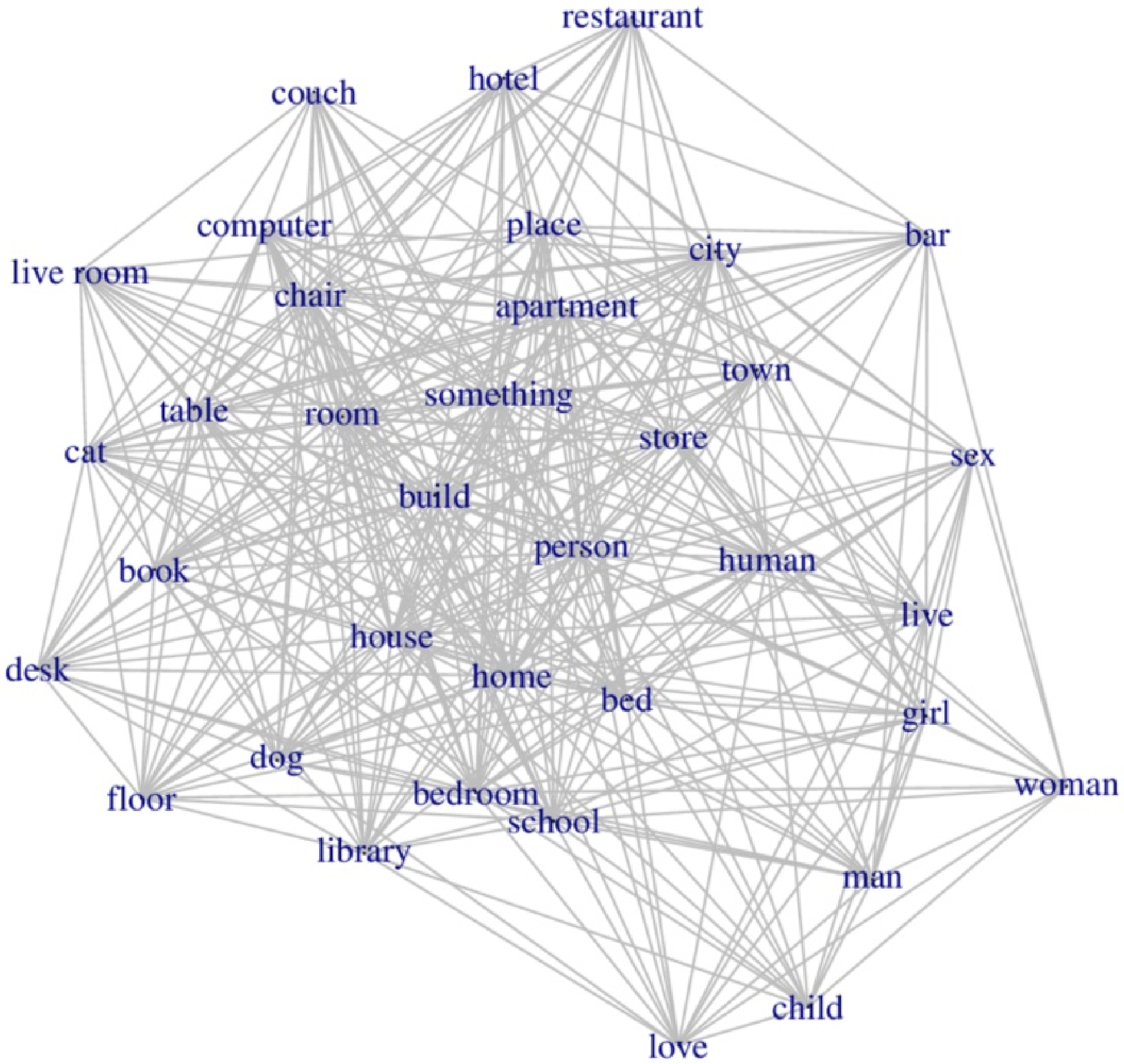}
\end{center}
\caption{The subgraph induced by the concepts that appear in 
Table \ref{tbl:cliques:positive:moderate-frequency}.}\label{fig:cliques:positive:maxCliques}
\end{figure}

\suppressfloats[b]

\section{Maximal Cliques: \texttt{ConceptuallyRelatedTo} Relation}
We restrict our focus on subgraphs composed of edges with positive score only.
In the entire graph, edges representing the relation \dbtext{ConceptuallyRelatedTo} have positive polarity only.
The number of multi-edges is $23010$ and this is the same as the number of
directed edges in the induced directed graph. The number of undirected edges is $21848$.
If we neglect the self-loops the number of multi-edges is $22989$ and this is the same as the number
of directed edges in the induced directed graph. The number of undirected edges is $21827$.

There are $2,199$ maximal cliques of size $3$, $364$ maximal cliques of size $4$, $61$ maximal cliques
of size $5$, and $3$ maximal (maximum) cliques of size $6$. The $3$ maximum cliques (that is of size $6$) are
among the concepts:
\begin{enumerate}[noitemsep,leftmargin=8mm,topsep=0.5mm]
 \item \dbtext{circle} (14472), \dbtext{round} (7057), \dbtext{ball} (263), \dbtext{sphere} (5508), \dbtext{eye} (1160), and \dbtext{head} (10228).
 \item \dbtext{circle} (14472), \dbtext{round} (7057), \dbtext{ball} (263), \dbtext{sphere} (5508), \dbtext{eye} (1160), and \dbtext{egg} (9339).
 \item \dbtext{person} (9), \dbtext{man} (7), \dbtext{female} (15676), \dbtext{girl} (876), \dbtext{woman} (895), and \dbtext{doll} (1931).
\end{enumerate}

\noindent Cliques of size $5$ which are related with the above are the following.
\begin{enumerate}[noitemsep,leftmargin=8mm,topsep=0.5mm]
 \item \dbtext{circle} (14472), \dbtext{round} (7057), \dbtext{ball} (263), \dbtext{sphere} (5508), and \dbtext{drop} (1846).
 \item \dbtext{person} (9), \dbtext{sister} (3656), \dbtext{mother} (301), \dbtext{girl} (876), \dbtext{female} (15676)
 \item \dbtext{person} (9), \dbtext{family} (915), \dbtext{mother} (301), \dbtext{dad} (9672), \dbtext{father} (13663)
 \item \dbtext{person} (9), \dbtext{mother} (301), \dbtext{girl} (876), \dbtext{woman} (895), \dbtext{female} (15676)
 \item \dbtext{person} (9), \dbtext{man} (7), \dbtext{father} (13663), \dbtext{male} (6169), \dbtext{dad} (9672)
 \item \dbtext{person} (9), \dbtext{man} (7), \dbtext{statue} (6436), \dbtext{woman} (895), \dbtext{doll} (1931)
 \item \dbtext{person} (9), \dbtext{man} (7), \dbtext{male} (6169), \dbtext{brother} (2383), \dbtext{boy} (5976)
\end{enumerate}

\medskip

The rest of the cliques of size $5$ are given in raw format below.
{\footnotesize
\begin{verbatim}
fluffy white (339045), cloud (446), sheep (6424), wool (6425), cotton (9729)
ground (184976), land (8060), soil (13912), earth (1633), dirt (15359)
cost (81860), price (14042), buy (475), purchase (18262), payment (25417)
cost (81860), price (14042), buy (475), purchase (18262), pay (1473)
cost (81860), price (14042), buy (475), money (1240), payment (25417)
cost (81860), bill (1245), buy (475), money (1240), payment (25417)
sibling (53730), family (915), brother (2383), sister (3656), daughter (13446)
rectangle (41018), square (4138), paper (149), book (2033), card (13442)
rectangle (41018), square (4138), paper (149), book (2033), page (6264)
fog (32237), smoke (188), mist (16981), cloud (446), steam (17055)
white fluffy (23851), cotton (9729), cloud (446), sheep (6424), wool (6425)
silk (22088), wool (6425), cotton (9729), material (591), fabric (1913)
silk (22088), wool (6425), cotton (9729), material (591), cloth (1903)
fluffy (22025), cotton (9729), wool (6425), cloud (446), sheep (6424)
print (21683), write (1893), paper (149), book (2033), text (4472)
stage (21403), play (372), theatre (4095), act (7272), scene (15813)
injury (18717), hurt (686), pain (1813), wind (2284), cut (6250)
sight (18526), vision (7204), see (1161), look (8821), view (5574)
sight (18526), vision (7204), see (1161), look (8821), eye (1160)
purchase (18262), buy (475), sell (649), sale (13614), price (14042)
purchase (18262), buy (475), sell (649), sale (13614), trade (10511)
steam (17055), smoke (188), cloud (446), white (2612), mist (16981)
grey (15391), bullet (13342), steel (3907), metal (6491), silver (13722)
son (15379), parent (696), mother (301), daughter (13446), father (13663)
son (15379), parent (696), mother (301), daughter (13446), child (178)
son (15379), parent (696), mother (301), dad (9672), father (13663)
silver (13722), metal (6491), tin (8891), steel (3907), iron (2587)
silver (13722), metal (6491), tin (8891), steel (3907), shiny (1382)
daughter (13446), mother (301), female (15676), girl (876), sister (3656)
daughter (13446), mother (301), father (13663), parent (696), family (915)
daughter (13446), mother (301), child (178), parent (696), family (915)
sign (10388), car (529), street (350), drive (1545), road (2368)
wash (10170), bath (70), shower (173), water (1016), soap (3536)
cotton (9729), wool (6425), sheep (6424), cloud (446), white (2612)
hear (9269), music (542), sound (2660), ear (8314), noise (5363)
hear (9269), music (542), sound (2660), ear (8314), listen (75)
globe (9265), sphere (5508), round (7057), ball (263), eye (1160)
ear (8314), head (10228), eye (1160), nose (1171), face (8835)
ship (8013), boat (2389), sail (385), sea (1347), captain (23817)
draw (7764), art (7424), paint (1338), picture (2360), color (2611)
page (6264), paper (149), book (2033), read (1456), write (1893)
door (6022), window (1577), room (2480), house (652), wall (4030)
furniture (5668), wood (370), chair (596), desk (1043), table (5665)
river (4784), water (1016), sea (1347), ocean (1349), blue (4808)
river (4784), water (1016), sea (1347), ocean (1349), lake (660)
text (4472), paper (149), read (1456), write (1893), book (2033)
boat (2389), sea (1347), ocean (1349), fish (655), water (1016)
hand (2300), body (1861), arm (79), leg (1252), foot (1485)
thunder (2274), sky (1354), cloud (446), weather (1248), rain (1856)
ocean (1349), water (1016), sea (1347), fish (655), lake (660)
lady (1281), woman (895), mother (301), girl (876), female (15676)
lady (1281), woman (895), man (7), girl (876), female (15676)
education (1122), school (73), class (93), learn (401), student (886)
parent (696), family (915), dad (9672), mother (301), father (13663)
\end{verbatim}
}

\section{Maximal Cliques: \texttt{IsA} Relation}
We restrict our focus on subgraphs composed of edges with positive score only.
We distinguish two cases; negative and positive polarity.

\paragraph{Negative Polarity.}
The induced directed multigraph and directed graph is composed of $3874$ edges,
while the number of undirected edges is $3498$. Note that self-loops are not taken into account
since these do not affect the number of cliques.
There are $263$ cliques with negative polarity. Out of those, $242$ are of size $3$,
while $21$ are of size $4$. In particular, the maximal cliques of size $4$ are given below.
{\footnotesize
\begin{verbatim}
scarce (196339), lot (4905), many (6989), much (7917)
second (130981), year (2709), hour (2762), minute (2764)
average (76629), bad (2226), good (2666), best (20709)
where (54686), who (23034), why (5469), when (38265)
middle (52077), start (44963), begin (3695), end (10507)
middle (52077), side (17836), top (7514), bottom (5887)
middle (52077), side (17836), front (2423), back (15583)
far (37745), near (25285), here (6352), away (29340)
sight (18526), smell (1172), taste (14093), touch (5106)
sight (18526), smell (1172), taste (14093), sound (2660)
even (15946), night (8677), morning (15749), afternoon (15914)
even (15946), night (8677), morning (15749), day (2759)
taste (14093), smell (1172), hear (9269), touch (5106)
fall (9975), winter (1431), summer (1437), spring (5537)
hear (9269), see (1161), smell (1172), touch (5106)
little (8268), lot (4905), many (6989), much (7917)
few (8145), lot (4905), many (6989), much (7917)
blue (4808), red (2614), yellow (2616), green (2637)
woman (895), man (7), girl (876), boy (5976)
person (9), man (7), boy (5976), female (15676)
person (9), man (7), boy (5976), girl (876)
\end{verbatim}
}

\paragraph{Positive Polarity.}
The induced directed multigraph is composed of $90779$ edges, 
the induced directed graph is composed of $90732$ edges,
while the number of undirected edges is $88654$. 
Note that self-loops are not taken into account
since these do not affect the number of cliques.
There are $10132$ maximal cliques with positive polarity. Out of those, $7698$ are of size $3$,
$2033$ are of size $4$, $359$ are of size $5$, $41$ are of size $6$, and $1$ clique is of size 7.
The maximum clique (that is of size $7$) is
\dbtext{relation} (201900), \dbtext{person} (9), \dbtext{relative} (8531), \dbtext{family} (915), 
\dbtext{brother} (2383), \dbtext{sister} (3656), \dbtext{daughter} (13446).
The maximal cliques of size $6$ related to the above maximum clique are given below.
\begin{enumerate}[noitemsep,leftmargin=8mm,topsep=0.5mm]
 \item \dbtext{relation} (201900), \dbtext{person} (9), \dbtext{relative} (8531), \dbtext{family} (915), \dbtext{father} (13663), \dbtext{dad} (9672)
 \item \dbtext{relation} (201900), \dbtext{person} (9), \dbtext{relative} (8531), \dbtext{family} (915), \dbtext{father} (13663), \dbtext{mother} (301)
\end{enumerate}

\medskip

The rest of the maximal cliques of size $6$ with positive polarity are listed below.
{\footnotesize
\begin{verbatim}
abide (222135), house (652), home (1045), nest (1332), dwell (45162), live place (169946)
ground (184976), land (8060), place (4570), farm (2562), garden (1660), field (8720)
chief (180633), person (9), leader (3561), ruler (4313), king (1443), president (7061)
occasion (126340), birthday (6705), christmas (4290), party (307), event (8862), celebration (29221)
cost (81860), fee (36815), bill (1245), charge (4811), tax (9547), payment (25417)
cost (81860), price (14042), payment (25417), bill (1245), charge (4811), expense (25151)
cost (81860), tax (9547), bill (1245), payment (25417), charge (4811), expense (25151)
twist (71331), shake (5439), dance (1667), move (8737), verb (1490), action (9908)
twist (71331), turn (583), roll (6734), move (8737), verb (1490), action (9908)
twist (71331), turn (583), dance (1667), move (8737), verb (1490), action (9908)
structure (54435), build (1104), house (652), nest (1332), home (1045), dwell (45162)
structure (54435), build (1104), house (652), castle (996), home (1045), dwell (45162)
dwell (45162), house (652), home (1045), build (1104), mansion (25687), castle (996)
possession (40599), own (19972), mine (1210), belong (4322), owner (20525), property (21705)
youth (39134), child (178), person (9), girl (876), boy (5976), young person (14434)
fog (32237), mist (16981), smoke (188), cloud (446), air (6408), steam (17055)
poison (24474), water (1016), drink (120), wine (7522), beverage (10164), liquid (1551)
poison (24474), water (1016), drink (120), wine (7522), beverage (10164), food (1359)
property (21705), noun (7478), place (4570), farm (2562), house (652), home (1045)
primate (20931), animal (902), mammal (4850), human (80), man (7), person (9)
gender (19976), person (9), female (15676), girl (876), woman (895), daughter (13446)
rat (19911), animal (902), mammal (4850), rodent (6841), mouse (1284), hamster (15121)
rat (19911), animal (902), mammal (4850), rodent (6841), mouse (1284), squirrel (6609)
female (15676), woman (895), person (9), sister (3656), girl (876), daughter (13446)
female (15676), woman (895), person (9), human (80), girl (876), chick (14872)
female (15676), woman (895), person (9), human (80), girl (876), daughter (13446)
female (15676), woman (895), person (9), human (80), girl (876), lady (1281)
female (15676), woman (895), person (9), human (80), girl (876), mother (301)
son (15379), person (9), brother (2383), family (915), relative (8531), daughter (13446)
son (15379), person (9), child (178), family (915), daughter (13446), relative (8531)
son (15379), person (9), child (178), family (915), daughter (13446), kid (5854)
young person (14434), person (9), child (178), kid (5854), girl (876), boy (5976)
action (9908), activity (6207), exercise (61), move (8737), walk (97), run (1102)
action (9908), verb (1490), move (8737), pass (9934), drive (1545), go (2801)
cotton (9729), fabric (1913), wool (6425), material (591), cloth (1903), textile (8844)
cotton (9729), fabric (1913), wool (6425), material (591), cloth (1903), linen (1137)
voice (8828), sound (2660), call (1061), talk (394), speak (1305), communication (6769)
field (8720), place (4570), farm (2562), garden (1660), area (1915), land (8060)
this (4539), it (137), live room (5581), house (652), home (1045), place (4570)
\end{verbatim}
}

\section{Maximal Cliques: \texttt{UsedFor} Relation}
We restrict our focus on subgraphs composed of edges with positive score only.
We distinguish two cases; negative and positive polarity.

\paragraph{Negative Polarity.}
The induced directed multigraph, directed graph, and undirected graph is composed of $193$ edges in each case.
Again, self-loops are not taken into account
since these do not affect the number of cliques.
There is only maximal clique in this case and it has size $3$.
It is composed of the concepts \dbtext{gerbil} (14223), \dbtext{exercise} (61), \dbtext{drive car} (1005).

\paragraph{Positive Polarity.}
The induced directed multigraph
as well as the induced directed graph is composed of $50228$ edges,
while the number of undirected edges is $50016$. 
Again, self-loops are not taken into account
since these do not affect the number of cliques.
There are $4427$ maximal cliques with positive polarity. Out of those, $3667$ are of size $3$,
$686$ are of size $4$, $73$ are of size $5$, and $1$ is of size $6$.
The maximum clique (that is of size $6$) is
\dbtext{get drunk} (310177), \dbtext{fun} (134), \dbtext{party} (307), 
\dbtext{drink alcohol} (1386), \dbtext{buy beer} (5734), and \dbtext{celebrate} (13996).
The maximal cliques of size $5$ related to the above maximum clique are given below.
\begin{enumerate}[noitemsep,leftmargin=8mm,topsep=0.5mm]
 \item \dbtext{get drunk} (310177), \dbtext{fun} (134), \dbtext{go party} (24657), \dbtext{buy beer} (5734), \dbtext{celebrate} (13996)
 \item \dbtext{get drunk} (310177), \dbtext{fun} (134), \dbtext{party} (307), \dbtext{socialis} (29314), \dbtext{nightclub} (10669)
\end{enumerate}

\medskip

The rest of the maximal cliques of size $5$ with positive polarity are listed below.
{\footnotesize
\begin{verbatim}
get clean (315026), bath (70), take bath (1316), soap (3536), hygiene (13991)
get shape (312438), improve health (10456), go jog (260), jog (6511), lose weight (10298)
get shape (312438), lose weight (10298), jog (6511), health (9745), go run (423)
get shape (312438), lose weight (10298), jog (6511), health (9745), go jog (260)
get physical activity (312389), get exercise (311524), fun (134), enjoyment (643), play lacrosse (28752)
meet person (119411), socialize (27285), party (307), dance (1667), dance club (21969)
meet person (119411), network (6746), socialis (29314), make friend (71547), hang out bar (427)
meet person (119411), network (6746), socialis (29314), make friend (71547), party (307)
meet person (119411), fun (134), socialis (29314), hang out bar (427), make friend (71547)
meet person (119411), fun (134), socialis (29314), party (307), make friend (71547)
meet person (119411), fun (134), socialis (29314), party (307), nightclub (10669)
meet person (119411), fun (134), dance club (21969), party (307), dance (1667)
meet person (119411), fun (134), dance (1667), party (307), nightclub (10669)
remove dirt (42600), soap (3536), clean (344), wash (10170), bathe (26690)
remove dirt (42600), soap (3536), clean (344), wash (10170), bath (70)
see band (25769), dance (1667), fun (134), enjoyment (643), listen music (642)
see band (25769), dance (1667), fun (134), enjoyment (643), music (542)
see band (25769), dance (1667), fun (134), enjoyment (643), hear music (45)
go party (24657), fun (134), celebrate (13996), buy beer (5734), good time (7209)
dance club (21969), dance (1667), listen music (642), fun (134), party (307)
give information (18633), talk (394), make phone call (402), call (1061), telephone (1790)
classroom (18421), learn (401), class (93), student (886), teach (1052)
purchase (18262), buy (475), store (1414), sale (13614), price (14042)
celebrate (13996), pub (13545), party (307), drink alcohol (1386), buy beer (5734)
celebrate (13996), fun (134), good time (7209), party (307), buy beer (5734)
nightclub (10669), fun (134), party (307), listen music (642), dance (1667)
watch television (10343), relax (4187), rest (310), sit down (3442), beanbag chair (9797)
watch television (10343), relax (4187), rest (310), sleep (425), beanbag chair (9797)
watch television (10343), relax (4187), rest (310), sleep (425), sofa (6231)
exchange information (10018), conversation (390), talk (394), telephone (1790), communicate (9747)
business (9787), telephone (1790), conversation (390), talk (394), make phone call (402)
communicate (9747), talk (394), make phone call (402), telephone (1790), call (1061)
communicate (9747), talk (394), make phone call (402), telephone (1790), conversation (390)
egg (9339), chicken (191), eat (432), cook (946), food (1359)
find information (8931), learn (401), research (1978), surf web (203), computer (467)
move (8737), transportation (2364), drive (1545), car (529), highway (2851)
move (8737), travel (1143), highway (2851), car (529), drive (1545)
see movie (7891), go film (305), fun (134), entertain (100), enjoyment (643)
enjoy yourself (7798), fun (134), sex (2825), procreate (4344), copulate (5623)
toy (7701), play (372), entertainment (607), ball (263), game (732)
toy (7701), play (372), fun (134), ball (263), game (732)
communication (6769), talk (394), make phone call (402), telephone (1790), call (1061)
communication (6769), talk (394), make phone call (402), telephone (1790), conversation (390)
play game (6081), entertainment (607), pass time (5077), surf web (203), computer (467)
play game (6081), fun (134), surf web (203), computer (467), pass time (5077)
play game (6081), fun (134), surf web (203), computer (467), learn (401)
song (6068), listen music (642), sing (5711), fun (134), enjoyment (643)
song (6068), music (542), sing (5711), fun (134), enjoyment (643)
song (6068), entertain (100), fun (134), enjoyment (643), sing (5711)
sing (5711), music (542), fun (134), enjoyment (643), guitar (989)
sing (5711), music (542), fun (134), enjoyment (643), tell story (199)
sing (5711), entertain (100), fun (134), tell story (199), enjoyment (643)
copulate (5623), fun (134), sex (2825), pleasure (4231), procreate (4344)
relax (4187), enjoyment (643), dance (1667), party (307), listen music (642)
literature (3901), learn (401), study (122), read (1456), research (1978)
sport (2130), fun (134), ball (263), play (372), game (732)
book (2033), show (2243), entertain (100), fun (134), enjoyment (643)
book (2033), read (1456), learn (401), study (122), text (4472)
book (2033), read (1456), learn (401), study (122), research (1978)
doll (1931), fun (134), child (178), play (372), learn (401)
dance (1667), fun (134), enjoyment (643), party (307), listen music (642)
dance (1667), fun (134), enjoyment (643), party (307), entertain (100)
go sleep (1207), bed (156), dream (172), go bed (406), relaxation (4254)
go movie (920), fun (134), enjoyment (643), entertain (100), watch movie (265)
student (886), learn (401), school (73), class (93), teach (1052)
computer (467), research (1978), learn (401), study (122), read (1456)
sleep (425), bed (156), dream (172), rest (310), relaxation (4254)
go bed (406), bed (156), relaxation (4254), dream (172), rest (310)
go film (305), fun (134), watch movie (265), entertain (100), enjoyment (643)
tell story (199), fun (134), enjoyment (643), entertain (100), show (2243)
library (68), read (1456), learn (401), study (122), research (1978)
\end{verbatim}
}

\section{Maximal Cliques: \texttt{LocatedNear} Relation}
We restrict our focus on subgraphs composed of edges with positive score only.
There are only edges with positive polarity.
The induced directed multigraph
as well as the induced directed graph is composed of $5043$ edges,
while the number of undirected edges is $4846$. 
Again, self-loops are not taken into account
since these do not affect the number of cliques.
There are $385$ maximal cliques (with positive polarity). Out of those, $358$ are of size $3$,
$23$ are of size $4$, and $4$ are of size $5$.
The maximum cliques are:
\begin{enumerate}[noitemsep,leftmargin=8mm,topsep=0.5mm]
 \item \dbtext{shore} (20212), \dbtext{sea} (1347), \dbtext{ocean} (1349), \dbtext{coast} (6350), \dbtext{wave} (8813)
 \item \dbtext{wave} (8813), \dbtext{beach} (24), \dbtext{sea} (1347), \dbtext{ocean} (1349), \dbtext{coast} (6350)
 \item \dbtext{water} (1016), \dbtext{sea} (1347), \dbtext{ocean} (1349), \dbtext{beach} (24), \dbtext{coast} (6350)
 \item \dbtext{water} (1016), \dbtext{sea} (1347), \dbtext{ocean} (1349), \dbtext{beach} (24), \dbtext{sand} (5768)
\end{enumerate}

\medskip

The $23$ maximal cliques of size $4$ with positive polarity are listed below.
{\footnotesize
\begin{verbatim}
ground (184976), floor (6062), foot (1485), bottom (5887)
ground (184976), plant (716), seed (9375), dirt (15359)
crop (33366), farmer (908), farm (2562), field (8720)
stick (31425), tree (33), wood (370), forest (1747)
cheek (15623), nose (1171), eye (1160), face (8835)
soil (13912), plant (716), garden (1660), seed (9375)
head (10228), ear (8314), eye (1160), face (8835)
bear (10208), forest (1747), tree (33), wood (370)
test (9242), school (73), student (886), teacher (3556)
face (8835), eye (1160), nose (1171), ear (8314)
field (8720), farm (2562), horse (1412), barn (4112)
squirrel (6609), tree (33), wood (370), forest (1747)
air (6408), cloud (446), bird (962), sky (1354)
door (6022), house (652), window (1577), room (2480)
table (5665), plate (1604), dinner (1605), napkin (1698)
soap (3536), bath (70), tub (1006), wash (10170)
soap (3536), bath (70), tub (1006), shower (173)
crab (1334), beach (24), sea (1347), ocean (1349)
water (1016), pier (25602), beach (24), ocean (1349)
water (1016), sea (1347), ocean (1349), mist (16981)
water (1016), sea (1347), ocean (1349), blue (4808)
water (1016), sea (1347), ocean (1349), boat (2389)
water (1016), sea (1347), ocean (1349), fish (655)
\end{verbatim}
}

\section{Maximal Cliques: \texttt{SimilarSize} Relation}
We restrict our focus on subgraphs composed of edges with positive score only.
There are only edges with positive polarity.
The induced directed multigraph
as well as the induced directed graph is composed of $1509$ edges,
while the number of undirected edges is $1459$. 
Again, self-loops are not taken into account
since these do not affect the number of cliques.
There are $55$ maximal cliques (with positive polarity) all of which are of size $3$ and are listed below.
{\footnotesize
\begin{verbatim}
pin head (333090), flea (25677), louse (40958)
footstep (332497), foot (1485), shoe (2790)
person person (311652), human (80), body (1861)
two person (151400), bed (156), mattress (35203)
fist (99732), hand (2300), apple (4596)
handkerchief (63112), napkin (1698), small towel (28420)
cathedral (58941), temple (15854), big build (32344)
dime (52812), penny (1071), cent (14994)
dime (52812), penny (1071), coin (3573)
crumb (47406), salt (1817), flea (25677)
infinity (40966), everything (1262), space (4435)
branch (37065), twig (14431), stick (31425)
tiny (35891), ant (14190), flea (25677)
tiny (35891), ant (14190), little (8268)
golf ball (35075), eye (1160), egg (9339)
singular (33174), eye (1160), egg (9339)
flea (25677), ant (14190), bug (5563)
flea (25677), sand (5768), grain (14893)
flea (25677), sand (5768), dust (5736)
flea (25677), salt (1817), grain (14893)
rat (19911), squirrel (6609), rodent (6841)
thumb (15862), finger (3399), bullet (13342)
quarter (15172), eye (1160), ring (7720)
olive (15042), eye (1160), grape (1366)
crown (15023), hat (629), head (10228)
skunk (14906), cat (616), squirrel (6609)
grain (14893), seed (9375), sand (5768)
grain (14893), seed (9375), rice (1510)
card (13442), paper (149), envelope (5487)
card (13442), paper (149), book (2033)
head (10228), plate (1604), face (8835)
sock (10193), foot (1485), shoe (2790)
cucumber (9642), corn (3531), banana (6422)
two (9549), one (581), eye (1160)
seed (9375), pebble (6018), pill (569)
egg (9339), ball (263), eye (1160)
face (8835), plate (1604), hand (2300)
atmosphere (8084), sky (1354), air (6408)
rabbit (7815), squirrel (6609), rodent (6841)
rabbit (7815), squirrel (6609), cat (616)
wolf (6387), dog (537), fox (1746)
record (6029), plate (1604), frisbee (2597)
envelope (5487), paper (149), letter (960)
cup (4116), drink (120), glass (1776)
bowl (3463), nest (1332), plate (1604)
hand (2300), nest (1332), plate (1604)
ocean (1349), water (1016), sea (1347)
person (9), slave (27415), servant (13683)
person (9), grave (14465), coffin (14874)
person (9), grave (14465), body (1861)
person (9), sister (3656), brother (2383)
person (9), clothe (2415), body (1861)
person (9), human (80), dummy (72290)
person (9), human (80), coffin (14874)
person (9), human (80), body (1861)
\end{verbatim}
}

\section{Maximal Cliques: \texttt{ReceivesAction} Relation}
We restrict our focus on subgraphs composed of edges with positive score only.
There are only edges with positive polarity.
The induced directed multigraph,
the induced directed graph,
as well as the induced undirected graph,
are composed of $10845$ edges.
Again, self-loops are not taken into account
since these do not affect the number of cliques.
There are $23$ maximal cliques (with positive polarity)
all of which are of size $3$.
These maximum cliques are listed below.
{\footnotesize
\begin{verbatim}
emac (405208), close (2222), open (6539)
spoken (312030), english (7102), language (9326)
eaten (310995), bagel (6956), toast (31357)
project onto screen (153271), movie (213), film (544)
fist (99732), close (2222), open (6539)
build brick (80792), house (652), build (1104)
find mall (75852), store (1414), clothe (2415)
find mailbox (69073), letter (960), mail (4691)
pothole (33862), repair (2289), broken (311852)
find house (33328), carpet (3450), floor (6062)
catch hook (27363), catfish (654), fish (655)
puncture (22394), repair (2289), broken (311852)
tune (19043), play (372), instrument (2086)
keep pet (18614), cat (616), animal (902)
feed (12213), cat (616), animal (902)
soda (9270), store (1414), open (6539)
open (6539), close (2222), door (6022)
open (6539), close (2222), window (1577)
open (6539), close (2222), door lock (1250)
open (6539), book (2033), store (1414)
person (9), worker (14094), fire (2895)
person (9), daughter (13446), born (3501)
person (9), kill (1466), murder (2663)
\end{verbatim}
}

\part{Communities}

\chapter{Non-Overlapping Communities}
In this chapter we will go through results based on non-overlapping community finding algorithms.
In every case we apply the community-finding algorithms to graphs (no multiple edges) 
that are undirected and without any self-loops.
So far all the algorithms used have been implemented in \igraph \citep[version $0.6.1$]{igraph}.

We will use $\uppi$ to indicate the coreness of the vertices. Modularity will be denoted by $\upmu$.
The number of communities found by an algorithm in a single run will be denoted by $\upkappa$.
Finally, we will denote with \abs{C} the number of connected components for every graph 
that we are going to examine. The graph will be clear from the context and it will be induced
by vertices that have coreness at least a minimum value.

\section{Negative Polarity}
First we will examine the graphs induced by assertions with negative polarity only.
Table \ref{tbl:communities:cores:negative:overall} gives an overview of the results achieved
by various community-finding algorithms that have been implemented in \igraph.

{
\setlength{\tabcolsep}{3pt}
\begin{table}[ht]
\caption{Overall comparison of the community finding algorithms implemented in \igraph. 
We use $\uppi$ to indicate coreness.
We can see the average number $\overline{\upkappa}$ of communities found per run,
as well as the average modularity $\overline{\upmu}$ 
achieved by each algorithm.
Bold entries in the columns with the modularities indicate the maximum value
achieved among all algorithms per row; that is per subgraph induced by vertices that have
coreness at least a certain lower bound value. The best values for the modularity
are achieved by the spinglass algorithm and the multilevel algorithm.
We do not see a result for spinglass in the last row (that is coreness $\geq 2$)
because the implementation of the algorithm expects a connected graph. Hence we need to run
the algorithm in each connected component. However, the difference in terms of actual computation time
is huge compared to the multilevel algorithm.}\label{tbl:communities:cores:negative:overall}
\begin{center}
\resizebox{\textwidth}{!}{
\begin{tabular}{|r|r|r|r||rl|rl|rl|rl|rl|rl|rl|rl|}\hline
\multicolumn{4}{|c||}{\multirow{2}{*}{subgraph properties}} & 
\multicolumn{2}{c|}{spin} & 
\multicolumn{2}{c|}{eigen-} &
\multicolumn{2}{c|}{\multirow{2}{*}{walktrap}} & \multicolumn{2}{c|}{between-} & 
\multicolumn{2}{c|}{fast} &
\multicolumn{2}{c|}{multi-} & \multicolumn{2}{c|}{label} & 
\multicolumn{2}{c|}{\multirow{2}{*}{infoMAP}} \\
 \multicolumn{4}{|c||}{} & 
 \multicolumn{2}{c|}{glass} & 
 \multicolumn{2}{c|}{vectors} & 
 \multicolumn{2}{c|}{} & 
 \multicolumn{2}{c|}{ness} & 
 \multicolumn{2}{c|}{greedy} & 
 \multicolumn{2}{c|}{level} & 
 \multicolumn{2}{c|}{propagation} &
 \multicolumn{2}{c|}{}\\\hline
 $\uppi\geq$ & \multicolumn{1}{c|}{$\abs{V}$} & \multicolumn{1}{c|}{$\abs{E}$} & \multicolumn{1}{c||}{$\abs{C}$} & 
\multicolumn{1}{c}{$\overline{\upkappa}$} & \multicolumn{1}{c|}{$\overline{\upmu}$} & 
\multicolumn{1}{c}{$\overline{\upkappa}$} & \multicolumn{1}{c|}{$\overline{\upmu}$} & 
\multicolumn{1}{c}{$\overline{\upkappa}$} & \multicolumn{1}{c|}{$\overline{\upmu}$} & 
\multicolumn{1}{c}{$\overline{\upkappa}$} & \multicolumn{1}{c|}{$\overline{\upmu}$} & 
\multicolumn{1}{c}{$\overline{\upkappa}$} & \multicolumn{1}{c|}{$\overline{\upmu}$} & 
\multicolumn{1}{c}{$\overline{\upkappa}$} & \multicolumn{1}{c|}{$\overline{\upmu}$} & 
\multicolumn{1}{c}{$\overline{\upkappa}$} & \multicolumn{1}{c|}{$\overline{\upmu}$} &
\multicolumn{1}{c}{$\overline{\upkappa}$} & \multicolumn{1}{c|}{$\overline{\upmu}$} \\\hline\hline
$ 6$ &    68 &   308 &    1 &  5.8 & 0.283 &    4 & 0.247 &    5 & 0.225 & 22 & 0.163 &    4 & 0.260 &    5 & 0.265 &    1.1 & 0.010 &    1.0 & 0.000 \\\hline
$ 5$ &   175 &   819 &    1 &  7.9 & 0.346 &    7 & 0.255 &   12 & 0.269 & 61 & 0.183 &    6 & 0.321 &    7 & 0.324 &    1.0 & 0.002 &    1.0 & 0.000 \\\hline
$ 4$ &   347 &  1447 &    1 & 10.0 & 0.411 &    8 & 0.296 &   25 & 0.332 & 55 & 0.309 &    9 & 0.384 &    9 & 0.387 &    2.2 & 0.027 &   17.7 & 0.308 \\\hline
$ 3$ &   820 &  2710 &    1 & 20.0 & 0.507 &   17 & 0.352 &   38 & 0.422 & 62 & 0.437 &   13 & 0.467 &   12 & 0.493 &   10.6 & 0.074 &   72.3 & 0.447 \\\hline
$ 2$ &  1755 &  4411 &    3 &   -- &    -- &   24 & 0.414 &  136 & 0.472 & 91 & 0.502 &   31 & 0.542 &   24 & 0.550 &   20.4 & 0.076 &  178.0 & 0.487 \\\hline
$ 1$ & 11707 & 12989 & 1377 &   -- &    -- & 1434 & 0.669 & 1947 & 0.672 & -- &    -- & 1506 & 0.747 & 1435 & 0.758 & 1772.0 & 0.621 & 1911.7 & 0.700 \\\hline
\multicolumn{18}{l}{} \\\cline{5-20}
\multicolumn{4}{r|}{duration/run (secs)} & 
\multicolumn{2}{c|}{$89.0$} & 
\multicolumn{2}{c|}{$98.2$} & 
\multicolumn{2}{c|}{$13.6$} &
\multicolumn{2}{c|}{$952.0$} &
\multicolumn{2}{c|}{$37.0$} & 
\multicolumn{2}{c|}{$0.1$} & 
\multicolumn{2}{c|}{$19.2$} & 
\multicolumn{2}{c|}{$21.8$} \\\cline{5-20}
\end{tabular}
}
\end{center}
\end{table}
}

\subsection{Spin Glass}
The algorithm used is the one implemented in \igraph 
which is based on \citep{communities:spinglass}.
Table \ref{tbl:communities:negative:spinglass:cores} presents the results when we apply the algorithm for 
subgraphs in which we include vertices with successively lower coreness and allow negative polarity on the edges only.

\begin{table}[ht]
\caption{Applying the Spin Glass algorithm for community finding implemented in \igraph by successively including
vertices with lower coreness on the undirected graph induced by the assertions with negative polarity (self-loops are removed).
In every row we have the number of vertices and the number of edges of each such subgraph together with
the number of components ($\abs{C}$) that we find in that subgraph. 
The next three columns present the number of communities found by the algorithm;
the average among all runs, the minimum, and the maximum.
The next three columns present the modularity achieved by the algorithm due to the cut induced by the communities;
the average among all runs, the minimum, and the maximum.
The entire computation lasted about $890.25$ for $10$ runs;
that is about $89.03$ seconds per run.}\label{tbl:communities:negative:spinglass:cores}
\begin{center}
\begin{tabular}{|c|r|r|c||r|r|r||r|r|r||}\hline
\multirow{2}{*}{coreness} & \multicolumn{1}{c|}{\multirow{2}{*}{$\abs{V}$}} & \multicolumn{1}{c|}{\multirow{2}{*}{$\abs{E}$}} & \multirow{2}{*}{$\abs{C}$} & \multicolumn{3}{c||}{communities found} & \multicolumn{3}{c||}{modularity} \\\cline{5-10}
          &        &        &   & \multicolumn{1}{c|}{avg} & \multicolumn{1}{c|}{min} & \multicolumn{1}{c||}{max} & \multicolumn{1}{c|}{avg} & \multicolumn{1}{c|}{min} & \multicolumn{1}{c||}{max}  \\\hline\hline
$\geq  6$ &     68 &    308 &    1 &  5.800 &    5 &    7 & 0.283301 & 0.281013 & 0.285530 \\\hline
$\geq  5$ &    175 &    819 &    1 &  7.900 &    6 &    9 & 0.345868 & 0.343487 & 0.347246 \\\hline
$\geq  4$ &    347 &   1447 &    1 & 10.000 &    9 &   11 & 0.410580 & 0.406916 & 0.412461 \\\hline
$\geq  3$ &    820 &   2710 &    1 & 20.000 &   17 &   23 & 0.507353 & 0.505082 & 0.510296 \\\hline
$\geq  2$ &   1755 &   4411 &    3 &     -- &   -- &   -- &       -- &       -- &       -- \\\hline
$\geq  1$ &  11707 &  12989 & 1377 &     -- &   -- &   -- &       -- &       -- &       -- \\\hline
\end{tabular}
\end{center}
\end{table}

\subsubsection{Spin Glass: Communities in the Inner-Most Core}
An instance generating $6$ communities with modularity equal to $0.283142$ is shown below.
\paragraph{Community 1 (size is 5).}
\dbtext{exercise}, \dbtext{fun}, \dbtext{drive car}, \dbtext{gasoline}, \dbtext{fidelity}

\paragraph{Community 2 (size is 12).} 
\dbtext{library}, \dbtext{bath}, \dbtext{park}, \dbtext{eat}, \dbtext{office}, \dbtext{home}, \dbtext{eye}, \dbtext{time}, 
\dbtext{space}, \dbtext{program language}, \dbtext{cash register}, \dbtext{singular}

\paragraph{Community 3 (size is 12).} 
\dbtext{talk}, \dbtext{dog}, \dbtext{cat}, \dbtext{fish}, \dbtext{animal}, \dbtext{bird}, 
\dbtext{die}, \dbtext{mouse}, \dbtext{horse}, \dbtext{read}, \dbtext{ear}, \dbtext{fly}

\paragraph{Community 4 (size is 5).} 
\dbtext{tree}, \dbtext{walk}, \dbtext{plant}, \dbtext{god}, \dbtext{transportation device}

\paragraph{Community 5 (size is 18).} 
\dbtext{drink}, \dbtext{car}, \dbtext{music}, \dbtext{desk}, \dbtext{kitchen}, \dbtext{television}, \dbtext{food}, 
\dbtext{drive}, \dbtext{telephone}, \dbtext{audience}, \dbtext{boat}, \dbtext{cabinet}, \dbtext{table}, \dbtext{way}, 
\dbtext{competitive activity}, \dbtext{gerbil}, \dbtext{software}, \dbtext{speedo}

\paragraph{Community 6 (size is 16).} 
\dbtext{person}, \dbtext{human}, \dbtext{examination}, \dbtext{bed}, \dbtext{computer}, \dbtext{house}, 
\dbtext{money}, \dbtext{hot}, \dbtext{potato}, \dbtext{rain}, \dbtext{book}, \dbtext{fire}, \dbtext{long hair}, 
\dbtext{metal}, \dbtext{brain}, \dbtext{conscious}

\subsection{Eigenvectors}
The algorithm used is the one implemented in \igraph 
which is based on \citep{communities:eigenvectors}.
Table \ref{tbl:communities:negative:eigenvector:cores} presents the results when we apply the algorithm for 
subgraphs in which we include vertices with successively lower coreness and allow negative polarity on the edges only.
The computation lasted $196.3$ seconds for $2$ runs.

\begin{table}[ht]
\caption{Applying the leading eigenvector algorithm for community finding implemented in \igraph by successively including
vertices with lower coreness on the undirected graph induced by the assertions with negative polarity (self-loops are removed).
In every row we have the number of vertices and the number of edges of each such subgraph together with
the number of components ($\abs{C}$) that we find in that subgraph. 
The next three columns present the number of communities found by the algorithm;
the average among all runs, the minimum, and the maximum.
The next three columns present the modularity achieved by the algorithm due to the cut induced by the communities;
the average among all runs, the minimum, and the maximum.
The entire computation lasted $196.3$ seconds for $2$ runs;
that is about $98.15$ seconds per run.}\label{tbl:communities:negative:eigenvector:cores}
\begin{center}
\begin{tabular}{|c|r|r|c||r|r|r||r|r|r||}\hline
\multirow{2}{*}{coreness} & \multicolumn{1}{c|}{\multirow{2}{*}{$\abs{V}$}} & \multicolumn{1}{c|}{\multirow{2}{*}{$\abs{E}$}} & \multirow{2}{*}{$\abs{C}$} & \multicolumn{3}{c||}{communities found} & \multicolumn{3}{c||}{modularity} \\\cline{5-10}
          &        &        &   & \multicolumn{1}{c|}{avg} & \multicolumn{1}{c|}{min} & \multicolumn{1}{c||}{max} & \multicolumn{1}{c|}{avg} & \multicolumn{1}{c|}{min} & \multicolumn{1}{c||}{max} \\\hline\hline
$\geq  6$ &     68 &    308 &    1 &    4.000 &    4 &    4 & 0.246674 & 0.246674 & 0.246674 \\\hline
$\geq  5$ &    175 &    819 &    1 &    7.000 &    7 &    7 & 0.254509 & 0.254509 & 0.254509 \\\hline
$\geq  4$ &    347 &   1447 &    1 &    8.000 &    8 &    8 & 0.295650 & 0.295650 & 0.295650 \\\hline
$\geq  3$ &    820 &   2710 &    1 &   17.000 &   17 &   17 & 0.352476 & 0.352476 & 0.352476 \\\hline
$\geq  2$ &   1755 &   4411 &    3 &   24.000 &   24 &   24 & 0.414328 & 0.414328 & 0.414328 \\\hline
$\geq  1$ &  11707 &  12989 & 1377 & 1434.000 & 1434 & 1434 & 0.668898 & 0.668898 & 0.668898 \\\hline
\end{tabular}
\end{center}
\end{table}

\subsubsection{Eigenvectors: Communities in the Inner-Most Core}
We have the following communities.
\paragraph{Community 1 (size is 18).}
\dbtext{person}, \dbtext{bath}, \dbtext{examination}, \dbtext{fun}, \dbtext{computer}, \dbtext{dog}, 
\dbtext{house}, \dbtext{home}, \dbtext{eye}, \dbtext{hot}, \dbtext{rain}, \dbtext{book}, \dbtext{time}, 
\dbtext{space}, \dbtext{long hair}, \dbtext{program language}, \dbtext{brain}, \dbtext{conscious}

\paragraph{Community 2 (size is 17).}
\dbtext{library}, \dbtext{park}, \dbtext{talk}, \dbtext{eat}, \dbtext{cat}, \dbtext{fish}, \dbtext{bird}, 
\dbtext{drive car}, \dbtext{office}, \dbtext{mouse}, \dbtext{drive}, \dbtext{competitive activity}, \dbtext{ear}, 
\dbtext{fly}, \dbtext{gerbil}, \dbtext{cash register}, \dbtext{singular}

\paragraph{Community 3 (size is 15).} 
\dbtext{tree}, \dbtext{exercise}, \dbtext{human}, \dbtext{walk}, \dbtext{bed}, \dbtext{plant}, \dbtext{animal}, 
\dbtext{die}, \dbtext{money}, \dbtext{horse}, \dbtext{fire}, \dbtext{god}, \dbtext{metal}, \dbtext{gasoline}, 
\dbtext{transportation device}

\paragraph{Community 4 (size is 18).} 
\dbtext{drink}, \dbtext{car}, \dbtext{music}, \dbtext{desk}, \dbtext{kitchen}, \dbtext{television}, \dbtext{food}, 
\dbtext{read}, \dbtext{potato}, \dbtext{telephone}, \dbtext{audience}, \dbtext{boat}, \dbtext{cabinet}, \dbtext{table}, 
\dbtext{way}, \dbtext{software}, \dbtext{speedo}, \dbtext{fidelity}

\subsection{Walktrap}
The algorithm used is the one implemented in \igraph 
which is based on \citep{communities:walktrap}.
Table \ref{tbl:communities:negative:walktrap:cores} presents the results when we apply the algorithm for 
subgraphs in which we include vertices with successively lower coreness and allow negative polarity on the edges only.
We use random walks of length $5$ throughout all the runs.

\begin{table}[ht]
\caption{Applying the Walktrap algorithm for community finding implemented in \igraph by successively including
vertices with lower coreness on the undirected graph induced by the assertions with negative polarity (self-loops are removed).
We use $5$ steps for every random walk generated throughout all our runs.
In every row we have the number of vertices and the number of edges of each such subgraph together with
the number of components ($\abs{C}$) that we find in that subgraph. 
The next three columns present the number of communities found by the algorithm;
the average among all runs, the minimum, and the maximum.
The next three columns present the modularity achieved by the algorithm due to the cut induced by the communities;
the average among all runs, the minimum, and the maximum.
The entire computation lasted $27.19$ seconds for 2 runs.}\label{tbl:communities:negative:walktrap:cores}
\begin{center}
\begin{tabular}{|c|r|r|c||r|r|r||r|r|r||}\hline
\multirow{2}{*}{coreness} & \multicolumn{1}{c|}{\multirow{2}{*}{$\abs{V}$}} & \multicolumn{1}{c|}{\multirow{2}{*}{$\abs{E}$}} & \multirow{2}{*}{$\abs{C}$} & \multicolumn{3}{c||}{communities found} & \multicolumn{3}{c||}{modularity} \\\cline{5-10}
          &        &        &   & \multicolumn{1}{c|}{avg} & \multicolumn{1}{c|}{min} & \multicolumn{1}{c||}{max} & \multicolumn{1}{c|}{avg} & \multicolumn{1}{c|}{min} & \multicolumn{1}{c||}{max} \\\hline\hline
$\geq  6$ &     68 &    308 &    1 &    5.000 &    5 &    5 & 0.225064 & 0.225064 & 0.225064 \\\hline
$\geq  5$ &    175 &    819 &    1 &   12.000 &   12 &   12 & 0.269032 & 0.269032 & 0.269032 \\\hline
$\geq  4$ &    347 &   1447 &    1 &   25.000 &   25 &   25 & 0.331810 & 0.331810 & 0.331810 \\\hline
$\geq  3$ &    820 &   2710 &    1 &   38.000 &   38 &   38 & 0.421854 & 0.421854 & 0.421854 \\\hline
$\geq  2$ &   1755 &   4411 &    3 &  136.000 &  136 &  136 & 0.471808 & 0.471808 & 0.471808 \\\hline
$\geq  1$ &  11707 &  12989 & 1377 & 1947.000 & 1947 & 1947 & 0.671694 & 0.671694 & 0.671694 \\\hline
\end{tabular}
\end{center}
\end{table}

\subsubsection{Walktrap: Communities in the Inner-Most Core}
We have the following communities.
\paragraph{Community 1 (size is 13).} 
\dbtext{tree}, \dbtext{walk}, \dbtext{examination}, \dbtext{bed}, \dbtext{plant}, \dbtext{die}, \dbtext{money}, \dbtext{mouse}, 
\dbtext{god}, \dbtext{long hair}, \dbtext{brain}, \dbtext{conscious}, \dbtext{transportation device}

\paragraph{Community 2 (size is 22).} \dbtext{person}, \dbtext{exercise}, \dbtext{bath}, \dbtext{human}, \dbtext{fun}, \dbtext{computer}, 
\dbtext{music}, \dbtext{house}, \dbtext{drive car}, \dbtext{eye}, \dbtext{hot}, \dbtext{potato}, \dbtext{telephone}, \dbtext{rain}, 
\dbtext{book}, \dbtext{time}, \dbtext{fire}, \dbtext{cabinet}, \dbtext{table}, \dbtext{metal}, \dbtext{way}, \dbtext{gasoline}

\paragraph{Community 3 (size is 12).} \dbtext{talk}, \dbtext{eat}, \dbtext{dog}, \dbtext{cat}, \dbtext{fish}, \dbtext{animal}, 
\dbtext{bird}, \dbtext{horse}, \dbtext{read}, \dbtext{ear}, \dbtext{fly}, \dbtext{fidelity}

\paragraph{Community 4 (size is 13).} \dbtext{drink}, \dbtext{car}, \dbtext{desk}, \dbtext{kitchen}, \dbtext{television}, \dbtext{food}, 
\dbtext{drive}, \dbtext{audience}, \dbtext{boat}, \dbtext{competitive activity}, \dbtext{gerbil}, \dbtext{software}, \dbtext{speedo}

\paragraph{Community 5 (size is 8).} \dbtext{library}, \dbtext{park}, \dbtext{office}, \dbtext{home}, \dbtext{space}, 
\dbtext{program language}, \dbtext{cash register}, \dbtext{singular}

\subsection{Betweenness}
The algorithm used is the one implemented in \igraph 
which is based on \citep{communities:betweenness}.
The idea of the algorithm is described below; it is taken from \igraph documentation online.
\begin{quote}
The idea is that the betweenness of the edges connecting two communities is typically high, 
as many of the shortest paths between nodes in separate communities go through them. 
So we gradually remove the edge with highest betweenness from the network, 
and recalculate edge betweenness after every removal. 
This way sooner or later the network falls off to two components, 
then after a while one of these components falls off to two smaller components, 
etc.~until all edges are removed. 
This is a divisive hierarchical approach, the result is a dendrogram. 
\end{quote}
The algorithm has complexity 
\OO{\abs{V}\abs{E}^2}, as the betweenness calculation 
requires \OO{\abs{V}\abs{E}} time and we do it $\abs{E}-1$ times.
Hence, we applied the algorithm only on the subgraph induced by the vertices with maximum
coreness \footnote{ Recall that the subgraph has no self-loops.} .

Table \ref{tbl:communities:negative:betweenness:cores} presents the results when we apply the algorithm for 
subgraphs in which we include vertices with successively lower coreness
and allow negative polarity on the edges only.

\begin{table}[ht]
\caption{Applying the Edge Betweenness algorithm for community finding implemented in \igraph by successively including
vertices with lower coreness on the undirected graph induced by the assertions with negative polarity (self-loops are removed).
In every row we have the number of vertices and the number of edges of each such subgraph together with
the number of components ($\abs{C}$) that we find in that subgraph. 
The next three columns present the number of communities found by the algorithm;
the average among all runs, the minimum, and the maximum.
The next three columns present the modularity achieved by the algorithm due to the cut induced by the communities;
the average among all runs, the minimum, and the maximum.
The entire computation lasted about $951.51$ seconds for a single run.}\label{tbl:communities:negative:betweenness:cores}
\begin{center}
\begin{tabular}{|c|r|r|c||r|r|r||r|r|r||}\hline
\multirow{2}{*}{coreness} & \multicolumn{1}{c|}{\multirow{2}{*}{$\abs{V}$}} & \multicolumn{1}{c|}{\multirow{2}{*}{$\abs{E}$}} & \multirow{2}{*}{$\abs{C}$} & \multicolumn{3}{c||}{communities found} & \multicolumn{3}{c||}{modularity} \\\cline{5-10}
          &        &        &   & \multicolumn{1}{c|}{avg} & \multicolumn{1}{c|}{min} & \multicolumn{1}{c||}{max} & \multicolumn{1}{c|}{avg} & \multicolumn{1}{c|}{min} & \multicolumn{1}{c||}{max} \\\hline\hline
$\geq  6$ &     68 &    308 &    1 &   22.000 &   22 &   22 & 0.163213 & 0.163213 & 0.163213 \\\hline
$\geq  5$ &    175 &    819 &    1 &   61.000 &   61 &   61 & 0.183308 & 0.183308 & 0.183308 \\\hline
$\geq  4$ &    347 &   1447 &    1 &   55.000 &   55 &   55 & 0.309215 & 0.309215 & 0.309215 \\\hline
$\geq  3$ &    820 &   2710 &    1 &   62.000 &   62 &   62 & 0.436608 & 0.436608 & 0.436608 \\\hline
$\geq  2$ &   1755 &   4411 &    3 &   91.000 &   91 &   91 & 0.502412 & 0.502412 & 0.502412 \\\hline
$\geq  1$ &  11707 &  12989 & 1377 &       -- &   -- &   -- &       -- &       -- &       -- \\\hline
\end{tabular}
\end{center}
\end{table}

\subsubsection{Betweenness: Communities in the Inner-Most Core}
We have the following communities.
\setlength{\columnseprule}{.2pt}
\begin{multicols}{2}
{ 
\paragraph{Community 1 (size is 9).} 
\dbtext{person}, \dbtext{human}, \dbtext{plant}, \dbtext{animal}, \dbtext{die}, \dbtext{money}, \dbtext{potato}, 
\dbtext{god}, \dbtext{long hair}
\paragraph{Community 2 (size is 2).} \dbtext{tree}, \dbtext{brain}
\paragraph{Community 3 (size is 1).} \dbtext{exercise}
\paragraph{Community 4 (size is 26).} \dbtext{library}, \dbtext{drink}, \dbtext{park}, \dbtext{talk}, \dbtext{car}, 
\dbtext{cat}, \dbtext{fish}, \dbtext{bird}, \dbtext{drive car}, \dbtext{desk}, \dbtext{kitchen}, \dbtext{eye}, 
\dbtext{mouse}, \dbtext{television}, \dbtext{food}, \dbtext{drive}, \dbtext{audience}, \dbtext{book}, \dbtext{boat}, 
\dbtext{table}, \dbtext{competitive activity}, \dbtext{ear}, \dbtext{fly}, \dbtext{gerbil}, \dbtext{software}, \dbtext{speedo}
\paragraph{Community 5 (size is 1).} \dbtext{bath}
\paragraph{Community 6 (size is 2).} \dbtext{walk}, \dbtext{fun}
\paragraph{Community 7 (size is 1).} \dbtext{examination}
\paragraph{Community 8 (size is 2).} \dbtext{bed}, \dbtext{conscious}
\paragraph{Community 9 (size is 1).} \dbtext{eat}
\paragraph{Community 10 (size is 9).} \dbtext{computer}, \dbtext{telephone}, \dbtext{time}, \dbtext{space}, 
\dbtext{cabinet}, \dbtext{metal}, \dbtext{way}, \dbtext{gasoline}, \dbtext{program language}
\paragraph{Community 11 (size is 1).} \dbtext{dog}
\paragraph{Community 12 (size is 1).} \dbtext{music}
\paragraph{Community 13 (size is 2).} \dbtext{house}, \dbtext{home}
\paragraph{Community 14 (size is 1).} \dbtext{office}
\paragraph{Community 15 (size is 1).} \dbtext{horse}
\paragraph{Community 16 (size is 2).} \dbtext{hot}, \dbtext{fire}
\paragraph{Community 17 (size is 1).} \dbtext{read}
\paragraph{Community 18 (size is 1).} \dbtext{rain}
\paragraph{Community 19 (size is 1).} \dbtext{cash register}
\paragraph{Community 20 (size is 1).} \dbtext{singular}
\paragraph{Community 21 (size is 1).} \dbtext{transportation device}
\paragraph{Community 22 (size is 1).} \dbtext{fidelity}
}
\end{multicols}

\subsection{Fast Greedy}
The algorithm used is the one implemented in \igraph 
which is based on \citep{communities:fast-greedy}.
According to igraph version $0.6.1$ which was used at the time of the writing,
some improvements mentioned in \citep{communities:fast-greedy:improvements}
have also been implemented. 
Table \ref{tbl:communities:negative:fast-greedy:cores} presents the results when we apply the algorithm for 
subgraphs in which we include vertices with successively lower coreness
and allow negative polarity on the edges only.

\begin{table}[ht]
\caption{Applying the Fast Greedy algorithm for community finding implemented in \igraph by successively including
vertices with lower coreness on the undirected graph induced by the assertions with negative polarity (self-loops are removed).
In every row we have the number of vertices and the number of edges of each such subgraph together with
the number of components ($\abs{C}$) that we find in that subgraph. 
The next three columns present the number of communities found by the algorithm;
the average among all runs, the minimum, and the maximum.
The next three columns present the modularity achieved by the algorithm due to the cut induced by the communities;
the average among all runs, the minimum, and the maximum.
The entire computation lasted about $370.2$ seconds for $10$ runs; 
that is about $37.02$ seconds per run.}\label{tbl:communities:negative:fast-greedy:cores}
\begin{center}
\begin{tabular}{|c|r|r|c||r|r|r||r|r|r||}\hline
\multirow{2}{*}{coreness} & \multicolumn{1}{c|}{\multirow{2}{*}{$\abs{V}$}} & \multicolumn{1}{c|}{\multirow{2}{*}{$\abs{E}$}} & \multirow{2}{*}{$\abs{C}$} & \multicolumn{3}{c||}{communities found} & \multicolumn{3}{c||}{modularity} \\\cline{5-10}
          &        &        &   & \multicolumn{1}{c|}{avg} & \multicolumn{1}{c|}{min} & \multicolumn{1}{c||}{max} & \multicolumn{1}{c|}{avg} & \multicolumn{1}{c|}{min} & \multicolumn{1}{c||}{max} \\\hline\hline
$\geq  6$ &     68 &    308 &    1 &    4.000 &    4 &    4 & 0.260410 & 0.260410 & 0.260410 \\\hline
$\geq  5$ &    175 &    819 &    1 &    6.000 &    6 &    6 & 0.320716 & 0.320716 & 0.320716 \\\hline
$\geq  4$ &    347 &   1447 &    1 &    9.000 &    9 &    9 & 0.383727 & 0.383727 & 0.383727 \\\hline
$\geq  3$ &    820 &   2710 &    1 &   13.000 &   13 &   13 & 0.467128 & 0.467128 & 0.467128 \\\hline
$\geq  2$ &   1755 &   4411 &    3 &   31.000 &   31 &   31 & 0.542274 & 0.542274 & 0.542274 \\\hline
$\geq  1$ &  11707 &  12989 & 1377 & 1506.000 & 1506 & 1506 & 0.747473 & 0.747473 & 0.747473 \\\hline
\end{tabular}
\end{center}
\end{table}

\subsubsection{Fast Greedy: Communities in the Inner-Most Core}
We have the following communities.
\paragraph{Community 1 (size is 19).}
\dbtext{person}, \dbtext{exercise}, \dbtext{library}, \dbtext{fun}, \dbtext{park}, \dbtext{house}, \dbtext{office}, 
\dbtext{home}, \dbtext{eye}, \dbtext{horse}, \dbtext{hot}, \dbtext{rain}, \dbtext{time}, \dbtext{fire}, \dbtext{space}, 
\dbtext{program language}, \dbtext{cash register}, \dbtext{singular}, \dbtext{fidelity}

\paragraph{Community 2 (size is 17).} \dbtext{drink}, \dbtext{eat}, \dbtext{music}, \dbtext{desk}, \dbtext{kitchen}, 
\dbtext{television}, \dbtext{food}, \dbtext{drive}, \dbtext{audience}, \dbtext{book}, \dbtext{cabinet}, \dbtext{metal}, 
\dbtext{way}, \dbtext{competitive activity}, \dbtext{gasoline}, \dbtext{software}, \dbtext{speedo}

\paragraph{Community 3 (size is 18).} \dbtext{tree}, \dbtext{human}, \dbtext{walk}, \dbtext{examination}, \dbtext{bed}, 
\dbtext{computer}, \dbtext{plant}, \dbtext{die}, \dbtext{money}, \dbtext{mouse}, \dbtext{potato}, \dbtext{telephone}, 
\dbtext{god}, \dbtext{table}, \dbtext{long hair}, \dbtext{brain}, \dbtext{conscious}, \dbtext{transportation device}

\paragraph{Community 4 (size is 14).} \dbtext{bath}, \dbtext{talk}, \dbtext{car}, \dbtext{dog}, \dbtext{cat}, \dbtext{fish}, 
\dbtext{animal}, \dbtext{bird}, \dbtext{drive car}, \dbtext{read}, \dbtext{boat}, \dbtext{ear}, \dbtext{fly}, \dbtext{gerbil}

\subsection{Multilevel}
The algorithm used is the one implemented in \igraph 
which is based on \citep{communities:multilevel}.
Table \ref{tbl:communities:negative:multilevel:cores} presents the results when we apply the algorithm for 
subgraphs in which we include vertices with successively lower coreness
and allow negative polarity on the edges only.

\begin{table}[ht]
\caption{Applying the Multilevel algorithm for community finding implemented in \igraph by successively including
vertices with lower coreness on the undirected graph induced by the assertions with negative polarity (self-loops are removed).
In every row we have the number of vertices and the number of edges of each such subgraph together with
the number of components ($\abs{C}$) that we find in that subgraph. 
The next three columns present the number of communities found by the algorithm;
the average among all runs, the minimum, and the maximum.
The next three columns present the modularity achieved by the algorithm due to the cut induced by the communities;
the average among all runs, the minimum, and the maximum.
The entire computation lasted $12.7$ seconds for $100$ runs; 
that is about $0.127$ seconds per run.}\label{tbl:communities:negative:multilevel:cores}
\begin{center}
\begin{tabular}{|c|r|r|c||r|r|r||r|r|r||}\hline
\multirow{2}{*}{coreness} & \multicolumn{1}{c|}{\multirow{2}{*}{$\abs{V}$}} & \multicolumn{1}{c|}{\multirow{2}{*}{$\abs{E}$}} & \multirow{2}{*}{$\abs{C}$} & \multicolumn{3}{c||}{communities found} & \multicolumn{3}{c||}{modularity} \\\cline{5-10}
          &        &        &   & \multicolumn{1}{c|}{avg} & \multicolumn{1}{c|}{min} & \multicolumn{1}{c||}{max} & \multicolumn{1}{c|}{avg} & \multicolumn{1}{c|}{min} & \multicolumn{1}{c||}{max} \\\hline\hline
$\geq  6$ &     68 &    308 &    1 &    5.000 &    5 &    5 & 0.264953 & 0.264953 & 0.264953 \\\hline
$\geq  5$ &    175 &    819 &    1 &    7.000 &    7 &    7 & 0.323902 & 0.323902 & 0.323902 \\\hline
$\geq  4$ &    347 &   1447 &    1 &    9.000 &    9 &    9 & 0.386802 & 0.386802 & 0.386802 \\\hline
$\geq  3$ &    820 &   2710 &    1 &   12.000 &   12 &   12 & 0.492571 & 0.492571 & 0.492571 \\\hline
$\geq  2$ &   1755 &   4411 &    3 &   24.000 &   24 &   24 & 0.550032 & 0.550032 & 0.550032 \\\hline
$\geq  1$ &  11707 &  12989 & 1377 & 1435.000 & 1435 & 1435 & 0.758024 & 0.758024 & 0.758024 \\\hline
\end{tabular}
\end{center}
\end{table}

\subsubsection{Multilevel: Communities in the Inner-Most Core}
We have the following communities.
\paragraph{Community 1 (size is 20).}
\dbtext{person}, \dbtext{tree}, \dbtext{human}, \dbtext{examination}, \dbtext{bed}, \dbtext{computer}, 
\dbtext{house}, \dbtext{home}, \dbtext{money}, \dbtext{hot}, \dbtext{rain}, \dbtext{book}, \dbtext{time}, 
\dbtext{fire}, \dbtext{space}, \dbtext{long hair}, \dbtext{metal}, \dbtext{program language}, \dbtext{brain}, 
\dbtext{conscious}

\paragraph{Community 2 (size is 15).}
\dbtext{bath}, \dbtext{drink}, \dbtext{talk}, \dbtext{eat}, \dbtext{dog}, \dbtext{cat}, \dbtext{fish}, \dbtext{bird}, 
\dbtext{mouse}, \dbtext{horse}, \dbtext{read}, \dbtext{competitive activity}, \dbtext{ear}, \dbtext{fly}, \dbtext{fidelity}

\paragraph{Community 3 (size is 19).}
\dbtext{library}, \dbtext{fun}, \dbtext{car}, \dbtext{music}, \dbtext{desk}, \dbtext{kitchen}, \dbtext{eye}, 
\dbtext{television}, \dbtext{food}, \dbtext{potato}, \dbtext{telephone}, \dbtext{audience}, \dbtext{boat}, 
\dbtext{cabinet}, \dbtext{table}, \dbtext{way}, \dbtext{software}, \dbtext{cash register}, \dbtext{speedo}

\paragraph{Community 4 (size is 8).} 
\dbtext{exercise}, \dbtext{plant}, \dbtext{animal}, \dbtext{drive car}, \dbtext{die}, \dbtext{god}, 
\dbtext{gasoline}, \dbtext{transportation device}

\paragraph{Community 5 (size is 6).}
\dbtext{walk}, \dbtext{park}, \dbtext{office}, \dbtext{drive}, \dbtext{gerbil}, \dbtext{singular}

\subsection{Label Propagation}
The algorithm used is the one implemented in \igraph 
which is based on \citep{communities:lp}.
Table \ref{tbl:communities:negative:lp:cores} presents the results when we apply the algorithm for 
subgraphs in which we include vertices with successively lower coreness
and allow negative polarity on the edges only.

\begin{table}[ht]
\caption{Applying the Label Propagation algorithm for community finding implemented in \igraph by successively including
vertices with lower coreness on the undirected graph induced by the assertions with negative polarity (self-loops are removed).
In every row we have the number of vertices and the number of edges of each such subgraph together with
the number of components ($\abs{C}$) that we find in that subgraph. 
The next three columns present the number of communities found by the algorithm;
the average among all runs, the minimum, and the maximum.
The next three columns present the modularity achieved by the algorithm due to the cut induced by the communities;
the average among all runs, the minimum, and the maximum.
Finally the last two columns
present in how many runs the algorithm computed as many communities as we had components in that subgraph.
The entire computation lasted $1917.0$ seconds for $100$ runs; 
that is about $19.17$ seconds per run.}\label{tbl:communities:negative:lp:cores}
\begin{center}
\begin{tabular}{|c|r|r|c||r|r|r||r|r|r||r|r|}\hline
\multirow{2}{*}{coreness} & \multicolumn{1}{c|}{\multirow{2}{*}{$\abs{V}$}} & \multicolumn{1}{c|}{\multirow{2}{*}{$\abs{E}$}} & \multirow{2}{*}{$\abs{C}$} & \multicolumn{3}{c||}{communities found} & \multicolumn{3}{c||}{modularity} & \multicolumn{2}{c|}{agreement} \\\cline{5-12}
          &        &        &   & \multicolumn{1}{c|}{avg} & \multicolumn{1}{c|}{min} & \multicolumn{1}{c||}{max} & \multicolumn{1}{c|}{avg} & \multicolumn{1}{c|}{min} & \multicolumn{1}{c||}{max} & \multicolumn{1}{c|}{Y} & \multicolumn{1}{c|}{N} \\\hline\hline
$\geq  6$ &     68 &    308 &    1 &    1.050 &    1 &    2 & 0.010145 & 0.000000 & 0.223894 &  95 &   5 \\\hline
$\geq  5$ &    175 &    819 &    1 &    1.010 &    1 &    2 & 0.002370 & 0.000000 & 0.236981 &  99 &   1 \\\hline
$\geq  4$ &    347 &   1447 &    1 &    2.210 &    2 &    5 & 0.027395 & 0.010952 & 0.284524 &   0 & 100 \\\hline
$\geq  3$ &    820 &   2710 &    1 &   10.590 &    7 &   16 & 0.074087 & 0.033087 & 0.371248 &   0 & 100 \\\hline
$\geq  2$ &   1755 &   4411 &    3 &   20.410 &   15 &   34 & 0.075914 & 0.050672 & 0.410376 &   0 & 100 \\\hline
$\geq  1$ &  11707 &  12989 & 1377 & 1772.040 & 1691 & 1861 & 0.620714 & 0.526809 & 0.702329 &   0 & 100 \\\hline
\end{tabular}
\end{center}
\end{table}

\subsubsection{Label Propagation: Communities in the Inner-Most Core}
An example where Label Propagation identifies two communities in the innermost core is shown below.
\paragraph{Community 1 (size is 42).}
\dbtext{person}, \dbtext{tree}, \dbtext{exercise}, \dbtext{library}, \dbtext{bath}, \dbtext{human}, \dbtext{walk}, 
\dbtext{examination}, \dbtext{fun}, \dbtext{bed}, \dbtext{park}, \dbtext{computer}, \dbtext{music}, \dbtext{house}, 
\dbtext{plant}, \dbtext{drive car}, \dbtext{office}, \dbtext{home}, \dbtext{eye}, \dbtext{money}, \dbtext{hot}, 
\dbtext{potato}, \dbtext{telephone}, \dbtext{rain}, \dbtext{book}, \dbtext{time}, \dbtext{fire}, \dbtext{god}, 
\dbtext{space}, \dbtext{cabinet}, \dbtext{table}, \dbtext{long hair}, \dbtext{metal}, \dbtext{way}, \dbtext{gasoline}, 
\dbtext{program language}, \dbtext{gerbil}, \dbtext{software}, \dbtext{brain}, \dbtext{conscious}, \dbtext{singular}, 
\dbtext{transportation device}

\paragraph{Community 2 (size is 26).} 
\dbtext{drink}, \dbtext{talk}, \dbtext{eat}, \dbtext{car}, \dbtext{dog}, \dbtext{cat}, \dbtext{fish}, \dbtext{animal}, 
\dbtext{bird}, \dbtext{desk}, \dbtext{kitchen}, \dbtext{die}, \dbtext{mouse}, \dbtext{television}, \dbtext{food}, 
\dbtext{horse}, \dbtext{read}, \dbtext{drive}, \dbtext{audience}, \dbtext{boat}, \dbtext{competitive activity}, 
\dbtext{ear}, \dbtext{fly}, \dbtext{cash register}, \dbtext{speedo}, \dbtext{fidelity}

\subsection{InfoMAP}
The algorithm used is the one implemented in \igraph 
which is based on \citep{communities:infoMAP};
see also \citep{communities:infoMAP:additional}.
Table \ref{tbl:communities:negative:infoMAP:cores} presents the results when we apply the algorithm for 
subgraphs in which we include vertices with successively lower coreness
and allow edges with negative polarity only.

\begin{table}[ht]
\caption{Applying the InfoMAP algorithm for community finding implemented in \igraph by successively including
vertices with lower coreness on the undirected graph induced by the assertions with negative polarity (self-loops are removed).
In every row we have the number of vertices and the number of edges of each such subgraph together with
the number of components ($\abs{C}$) that we find in that subgraph. 
The next three columns present the number of communities found by the algorithm;
the average among all runs, the minimum, and the maximum.
The next three columns present the codelength of the partitioning found by the algorithm;
the average among all runs, the minimum, and the maximum.
The next three columns present the modularity achieved by the algorithm due to the cut induced by the communities;
the average among all runs, the minimum, and the maximum.
Finally the last two columns
present in how many runs the algorithm computed as many communities as we had components in that subgraph.
The entire computation lasted $217.97$ seconds for $10$ runs; 
that is about $21.80$ seconds per run.}\label{tbl:communities:negative:infoMAP:cores}
\begin{center}
\resizebox{\textwidth}{!}{
\begin{tabular}{|r|r|r|c||r|r|r||r|r|r||r|r|r||r|r|}\cline{14-15}
\multicolumn{13}{c|}{} & \multicolumn{2}{c|}{agree-} \\\cline{1-13}
$\uppi$ & \multicolumn{1}{c|}{\multirow{2}{*}{$\abs{V}$}} & \multicolumn{1}{c|}{\multirow{2}{*}{$\abs{E}$}} & \multirow{2}{*}{$\abs{C}$} & \multicolumn{3}{c||}{communities found} & \multicolumn{3}{c||}{codelength} & \multicolumn{3}{c||}{modularity} & \multicolumn{2}{c|}{ment} \\\cline{5-15}
$\geq$ &      &        &   & \multicolumn{1}{c|}{avg} & \multicolumn{1}{c|}{min} & \multicolumn{1}{c||}{max} & \multicolumn{1}{c|}{avg} & \multicolumn{1}{c|}{min} & \multicolumn{1}{c||}{max} & \multicolumn{1}{c|}{avg} & \multicolumn{1}{c|}{min} & \multicolumn{1}{c||}{max} & \multicolumn{1}{c|}{Y} & \multicolumn{1}{c|}{N} \\\hline\hline
$6$ &     68 &    308 &    1 &    1.000 &    1 &    1 & 5.993573 & 5.993573 & 5.993573 & 0.000000 & 0.000000 & 0.000000 &  10 &   0 \\\hline
$5$ &    175 &    819 &    1 &    1.000 &    1 &    1 & 7.279672 & 7.279672 & 7.279672 & 0.000000 & 0.000000 & 0.000000 &  10 &   0 \\\hline
$4$ &    347 &   1447 &    1 &   17.700 &   10 &   26 & 7.970802 & 7.949532 & 7.981013 & 0.307898 & 0.193377 & 0.385923 &   0 &  10 \\\hline
$3$ &    820 &   2710 &    1 &   72.300 &   70 &   76 & 8.268610 & 8.263414 & 8.277101 & 0.447178 & 0.441671 & 0.454124 &   0 &  10 \\\hline
$2$ &   1755 &   4411 &    3 &  178.000 &  170 &  183 & 8.330739 & 8.328037 & 8.334383 & 0.487219 & 0.482917 & 0.489891 &   0 &  10 \\\hline
$1$ &  11707 &  12989 & 1377 & 1911.700 & 1909 & 1916 & 6.828170 & 6.827590 & 6.828669 & 0.699840 & 0.698775 & 0.700512 &   0 &  10 \\\hline
\end{tabular}
}
\end{center}
\end{table}

\subsubsection{InfoMAP: Communities in the Inner-Most Core}
The entire core is one big community. This is the result in all of our runs.

\section{Positive Polarity}
In this section we will examine the graphs induced by assertions with positive polarity only.
Table \ref{tbl:communities:cores:positive:overall} gives an overview of the results achieved
by various community-finding algorithms that have been implemented in \igraph.

{
\setlength{\tabcolsep}{3pt}
\begin{table}[ht]
\caption{Overall comparison of the community finding algorithms implemented in \igraph. 
We use $\uppi$ to indicate coreness.
We can see the average number $\overline{\upkappa}$ of communities found per run,
as well as the average modularity $\overline{\upmu}$ 
achieved by each algorithm.
Bold entries in the columns with the modularities indicate the maximum value
achieved among all algorithms per row; that is per subgraph induced by vertices that have
coreness at least a certain lower bound value. The best values for the modularity
are achieved by the spinglass algorithm and the multilevel algorithm.
We do not see a result for spinglass in the last row (that is coreness $\geq 2$)
because the implementation of the algorithm expects a connected graph. Hence we need to run
the algorithm in each connected component. However, the difference in terms of actual computation time
is huge compared to the multilevel algorithm.}\label{tbl:communities:cores:positive:overall}
\begin{center}
\resizebox{\textwidth}{!}{
\begin{tabular}{|r|r|r|r||rl|rl|rl|rl|rl|rl|rl|rl|}\hline
\multicolumn{4}{|c||}{\multirow{2}{*}{subgraph properties}} & 
\multicolumn{2}{c|}{spin} & 
\multicolumn{2}{c|}{eigen-} &
\multicolumn{2}{c|}{\multirow{2}{*}{walktrap}} & \multicolumn{2}{c|}{between-} & 
\multicolumn{2}{c|}{fast} &
\multicolumn{2}{c|}{multi-} & \multicolumn{2}{c|}{label} & 
\multicolumn{2}{c|}{\multirow{2}{*}{infoMAP}} \\
 \multicolumn{4}{|c||}{} & 
 \multicolumn{2}{c|}{glass} & 
 \multicolumn{2}{c|}{vectors} & 
 \multicolumn{2}{c|}{} & 
 \multicolumn{2}{c|}{ness} & 
 \multicolumn{2}{c|}{greedy} & 
 \multicolumn{2}{c|}{level} & 
 \multicolumn{2}{c|}{propagation} &
 \multicolumn{2}{c|}{}\\\hline
 $\uppi\geq$ & \multicolumn{1}{c|}{$\abs{V}$} & \multicolumn{1}{c|}{$\abs{E}$} & \multicolumn{1}{c||}{$\abs{C}$} & 
\multicolumn{1}{c}{$\overline{\upkappa}$} & \multicolumn{1}{c|}{$\overline{\upmu}$} & 
\multicolumn{1}{c}{$\overline{\upkappa}$} & \multicolumn{1}{c|}{$\overline{\upmu}$} & 
\multicolumn{1}{c}{$\overline{\upkappa}$} & \multicolumn{1}{c|}{$\overline{\upmu}$} & 
\multicolumn{1}{c}{$\overline{\upkappa}$} & \multicolumn{1}{c|}{$\overline{\upmu}$} & 
\multicolumn{1}{c}{$\overline{\upkappa}$} & \multicolumn{1}{c|}{$\overline{\upmu}$} & 
\multicolumn{1}{c}{$\overline{\upkappa}$} & \multicolumn{1}{c|}{$\overline{\upmu}$} & 
\multicolumn{1}{c}{$\overline{\upkappa}$} & \multicolumn{1}{c|}{$\overline{\upmu}$} &
\multicolumn{1}{c}{$\overline{\upkappa}$} & \multicolumn{1}{c|}{$\overline{\upmu}$} \\\hline\hline
$26$ &    869 &  20526 & 1 &  6 & 0.323 &  4 & 0.304 &   3 & 0.275 & 42  & 0.268508 &   4 & 0.287 &  5 & 0.322 &  1.00 & 0.000 &    1.0 & 0.000 \\\hline
$25$ &   1167 &  27810 & 1 &  7 & 0.309 &  5 & 0.300 &   3 & 0.282 & --  &       -- &   3 & 0.295 &  6 & 0.320 &  1.01 & 0.003 &    1.0 & 0.000 \\\hline
$24$ &   1358 &  32314 & 1 &  8 & 0.315 &  4 & 0.290 &   3 & 0.282 & --  &       -- &   4 & 0.285 &  7 & 0.321 &  1.06 & 0.017 &    1.0 & 0.000 \\\hline
$23$ &   1514 &  35870 & 1 &  8 & 0.316 &  3 & 0.283 &   4 & 0.275 & --  &       -- &   4 & 0.284 &  6 & 0.322 &  1.00 & 0.000 &    1.0 & 0.000 \\\hline
$22$ &   1709 &  40099 & 1 &  6 & 0.320 &  4 & 0.287 &   3 & 0.287 & --  &       -- &   5 & 0.283 &  7 & 0.323 &  1.01 & 0.003 &    1.0 & 0.000 \\\hline
$21$ &   1865 &  43330 & 1 & 10 & 0.320 &  4 & 0.286 &   3 & 0.283 & --  &       -- &   4 & 0.292 &  8 & 0.314 &  1.00 & 0.000 &    1.0 & 0.000 \\\hline
$20$ &   2007 &  46145 & 1 &  7 & 0.328 &  4 & 0.290 &   4 & 0.285 & --  &       -- &   3 & 0.291 &  7 & 0.321 &  1.03 & 0.008 &    1.0 & 0.000 \\\hline
$19$ &   2173 &  49265 & 1 & 12 & 0.323 &  5 & 0.293 &   3 & 0.274 & --  &       -- &   4 & 0.294 &  7 & 0.321 &  1.01 & 0.003 &    1.0 & 0.000 \\\hline
$18$ &   2384 &  53011 & 1 & 10 & 0.326 &  4 & 0.292 &   4 & 0.292 & --  &       -- &   6 & 0.287 &  8 & 0.329 &  1.01 & 0.003 &    1.0 & 0.000 \\\hline
$17$ &   2617 &  56939 & 1 & 12 & 0.328 &  5 & 0.293 &   5 & 0.279 & --  &       -- &   6 & 0.285 &  9 & 0.331 &  1.00 & 0.000 &    2.0 & 0.009 \\\hline
$16$ &   2847 &  60583 & 1 & 10 & 0.331 &  4 & 0.279 &   4 & 0.285 & --  &       -- &   6 & 0.295 &  6 & 0.331 &  1.02 & 0.006 &    2.0 & 0.012 \\\hline
$15$ &   3105 &  64412 & 1 & 11 & 0.334 &  4 & 0.280 &   4 & 0.284 & --  &       -- &   6 & 0.292 &  7 & 0.336 &  1.00 & 0.000 &    2.0 & 0.012 \\\hline
$14$ &   3407 &  68613 & 1 & 14 & 0.337 &  5 & 0.272 &   6 & 0.306 & --  &       -- &   5 & 0.303 &  9 & 0.348 &  1.01 & 0.003 &    3.2 & 0.039 \\\hline
$13$ &   3746 &  72978 & 1 & 12 & 0.341 &  5 & 0.275 &   5 & 0.291 & --  &       -- &   6 & 0.300 &  9 & 0.344 &  1.00 & 0.000 &    2.0 & 0.014 \\\hline
$12$ &   4160 &  77882 & 1 & 14 & 0.344 &  6 & 0.283 &   8 & 0.290 & --  &       -- &  10 & 0.292 & 10 & 0.349 &  1.00 & 0.000 &    3.3 & 0.042 \\\hline
$11$ &   4634 &  83039 & 1 &  7 & 0.351 &  5 & 0.303 &   8 & 0.320 & --  &       -- &  12 & 0.298 & 10 & 0.348 &  1.01 & 0.003 &   18.0 & 0.188 \\\hline
$10$ &   5182 &  88462 & 1 &  9 & 0.357 &  4 & 0.281 &  10 & 0.311 & --  &       -- &  10 & 0.307 & 10 & 0.362 &  1.00 & 0.000 &   34.5 & 0.269 \\\hline
$ 9$ &   5883 &  94709 & 1 & 11 & 0.363 &  4 & 0.275 &  13 & 0.319 & --  &       -- &  14 & 0.319 &  9 & 0.364 &  1.02 & 0.003 &   55.9 & 0.347 \\\hline
$ 8$ &   6750 & 101564 & 1 & 10 & 0.372 &  4 & 0.273 &  17 & 0.323 & --  &       -- &  13 & 0.319 & 10 & 0.368 &  1.00 & 0.000 &   81.1 & 0.345 \\\hline
$ 7$ &   7904 & 109561 & 1 &  7 & 0.384 &  2 & 0.267 &  15 & 0.328 & --  &       -- &  18 & 0.326 & 10 & 0.375 &  1.02 & 0.002 &  117.1 & 0.348 \\\hline
$ 6$ &   9392 & 118389 & 1 & 11 & 0.393 &  4 & 0.283 &  17 & 0.341 & --  &       -- &  17 & 0.334 & 10 & 0.387 &  1.00 & 0.000 &  166.8 & 0.350 \\\hline
$ 5$ &  11483 & 128731 & 1 &  8 & 0.389 &  3 & 0.290 &  31 & 0.339 & --  &       -- &  24 & 0.333 & 12 & 0.392 &  1.00 & 0.000 &  251.6 & 0.353 \\\hline
$ 4$ &  14864 & 142112 & 1 & 12 & 0.413 &  3 & 0.290 &  54 & 0.334 & --  &       -- &  37 & 0.339 & 11 & 0.402 &  1.12 & 0.000 &  407.9 & 0.350 \\\hline
$ 3$ &  21812 & 162691 & 1 & 15 & 0.423 &  4 & 0.288 & 223 & 0.343 & --  &       -- &  61 & 0.373 & 13 & 0.419 &  2.63 & 0.006 &  747.5 & 0.348 \\\hline
$ 2$ &  41659 & 201678 & 4 & -- &    -- & 12 & 0.248 &  -- &    -- & --  &       -- & 252 & 0.409 & 25 & 0.449 & 12.71 & 0.010 & 1753.5 & 0.351 \\\hline
\multicolumn{18}{l}{} \\\cline{5-20}
\multicolumn{4}{r|}{duration/run (secs)} & 
\multicolumn{2}{c|}{$22,443.0$} & 
\multicolumn{2}{c|}{$47.5$} & 
\multicolumn{2}{c|}{$572.4$} &
\multicolumn{2}{c|}{$8,671.2$} &
\multicolumn{2}{c|}{$625.6$} & 
\multicolumn{2}{c|}{$6.9$} & 
\multicolumn{2}{c|}{$20.7$} & 
\multicolumn{2}{c|}{$1,337.6$} \\\cline{5-20}
\end{tabular}
}
\end{center}
\end{table}
}

\subsection{Spin Glass}
The algorithm used is the one implemented in \igraph 
which is based on \citep{communities:spinglass}.
Table \ref{tbl:communities:positive:spinglass:cores} presents the results when we apply the algorithm for 
subgraphs in which we include vertices with successively lower coreness and allow positive polarity on the edges only.

\begin{table}[ht]
\caption{Applying the Spin Glass algorithm for community finding implemented in \igraph by successively including
vertices with lower coreness on the undirected graph induced by the assertions with positive polarity (self-loops are removed).
In every row we have the number of vertices and the number of edges of each such subgraph together with
the number of components ($\abs{C}$) that we find in that subgraph. 
The next three columns present the number of communities found by the algorithm;
the average among all runs, the minimum, and the maximum.
The next three columns present the modularity achieved by the algorithm due to the cut induced by the communities;
the average among all runs, the minimum, and the maximum.
The entire computation lasted $22,442.95$ seconds for a single run.}\label{tbl:communities:positive:spinglass:cores}
\begin{center}
\begin{tabular}{|c|r|r|c||r|r|r||r|r|r||}\hline
\multirow{2}{*}{coreness} & \multicolumn{1}{c|}{\multirow{2}{*}{$\abs{V}$}} & \multicolumn{1}{c|}{\multirow{2}{*}{$\abs{E}$}} & \multirow{2}{*}{$\abs{C}$} & \multicolumn{3}{c||}{communities found} & \multicolumn{3}{c||}{modularity} \\\cline{5-10}
          &        &        &   & \multicolumn{1}{c|}{avg} & \multicolumn{1}{c|}{min} & \multicolumn{1}{c||}{max} & \multicolumn{1}{c|}{avg} & \multicolumn{1}{c|}{min} & \multicolumn{1}{c||}{max}  \\\hline\hline
$\geq 26$ &    869 &  20526 &    1 &  6.000 &    6 &    6 & 0.323492 & 0.323492 & 0.323492 \\\hline
$\geq 25$ &   1167 &  27810 &    1 &  7.000 &    7 &    7 & 0.309399 & 0.309399 & 0.309399 \\\hline
$\geq 24$ &   1358 &  32314 &    1 &  8.000 &    8 &    8 & 0.315379 & 0.315379 & 0.315379 \\\hline
$\geq 23$ &   1514 &  35870 &    1 &  8.000 &    8 &    8 & 0.316294 & 0.316294 & 0.316294 \\\hline
$\geq 22$ &   1709 &  40099 &    1 &  6.000 &    6 &    6 & 0.320150 & 0.320150 & 0.320150 \\\hline
$\geq 21$ &   1865 &  43330 &    1 & 10.000 &   10 &   10 & 0.320400 & 0.320400 & 0.320400 \\\hline
$\geq 20$ &   2007 &  46145 &    1 &  7.000 &    7 &    7 & 0.328107 & 0.328107 & 0.328107 \\\hline
$\geq 19$ &   2173 &  49265 &    1 & 12.000 &   12 &   12 & 0.323378 & 0.323378 & 0.323378 \\\hline
$\geq 18$ &   2384 &  53011 &    1 & 10.000 &   10 &   10 & 0.326275 & 0.326275 & 0.326275 \\\hline
$\geq 17$ &   2617 &  56939 &    1 & 12.000 &   12 &   12 & 0.327978 & 0.327978 & 0.327978 \\\hline
$\geq 16$ &   2847 &  60583 &    1 & 10.000 &   10 &   10 & 0.331368 & 0.331368 & 0.331368 \\\hline
$\geq 15$ &   3105 &  64412 &    1 & 11.000 &   11 &   11 & 0.334370 & 0.334370 & 0.334370 \\\hline
$\geq 14$ &   3407 &  68613 &    1 & 14.000 &   14 &   14 & 0.337464 & 0.337464 & 0.337464 \\\hline
$\geq 13$ &   3746 &  72978 &    1 & 12.000 &   12 &   12 & 0.340777 & 0.340777 & 0.340777 \\\hline
$\geq 12$ &   4160 &  77882 &    1 & 14.000 &   14 &   14 & 0.343900 & 0.343900 & 0.343900 \\\hline
$\geq 11$ &   4634 &  83039 &    1 &  7.000 &    7 &    7 & 0.351369 & 0.351369 & 0.351369 \\\hline
$\geq 10$ &   5182 &  88462 &    1 &  9.000 &    9 &    9 & 0.357447 & 0.357447 & 0.357447 \\\hline
$\geq  9$ &   5883 &  94709 &    1 & 11.000 &   11 &   11 & 0.362553 & 0.362553 & 0.362553 \\\hline
$\geq  8$ &   6750 & 101564 &    1 & 10.000 &   10 &   10 & 0.371669 & 0.371669 & 0.371669 \\\hline
$\geq  7$ &   7904 & 109561 &    1 &  7.000 &    7 &    7 & 0.383777 & 0.383777 & 0.383777 \\\hline
$\geq  6$ &   9392 & 118389 &    1 & 11.000 &   11 &   11 & 0.393008 & 0.393008 & 0.393008 \\\hline
$\geq  5$ &  11483 & 128731 &    1 &  8.000 &    8 &    8 & 0.388912 & 0.388912 & 0.388912 \\\hline
$\geq  4$ &  14864 & 142112 &    1 & 12.000 &   12 &   12 & 0.413350 & 0.413350 & 0.413350 \\\hline
$\geq  3$ &  21812 & 162691 &    1 & 15.000 &   15 &   15 & 0.423025 & 0.423025 & 0.423025 \\\hline
$\geq  2$ &  41659 & 201678 &    4 &     -- &   -- &   -- &       -- &       -- &       -- \\\hline
\end{tabular}
\end{center}
\end{table}

\subsubsection{Spin Glass: Communities in the Inner-Most Core}
An instance of $7$ communities with modularity $0.313939$ is shown below.
\paragraph{Community 1 (size is 1).} 
\dbtext{power}

\paragraph{Community 2 (size is 224).} 
\dbtext{rock}, \dbtext{beach}, \dbtext{tree}, \dbtext{monkey}, \dbtext{weasel}, \dbtext{pant}, \dbtext{kitten}, \dbtext{arm}, \dbtext{human}, 
\dbtext{beaver}, \dbtext{it}, \dbtext{smoke}, \dbtext{chicken}, \dbtext{state}, \dbtext{ball}, \dbtext{fungus}, \dbtext{park}, \dbtext{trouble}, 
\dbtext{snake}, \dbtext{wood}, \dbtext{bridge}, \dbtext{cloud}, \dbtext{nothing}, \dbtext{dog}, \dbtext{zoo}, \dbtext{live}, \dbtext{one}, 
\dbtext{cat}, \dbtext{hat}, \dbtext{country}, \dbtext{fish}, \dbtext{lake}, \dbtext{baby}, \dbtext{plant}, \dbtext{hide}, \dbtext{animal}, 
\dbtext{cold}, \dbtext{moon}, \dbtext{pet}, \dbtext{bird}, \dbtext{shark}, \dbtext{water}, \dbtext{rosebush}, \dbtext{yard}, \dbtext{sloth}, 
\dbtext{bat}, \dbtext{lizard}, \dbtext{beautiful}, \dbtext{eye}, \dbtext{nose}, \dbtext{smell}, \dbtext{well}, \dbtext{bill}, \dbtext{snow}, 
\dbtext{weather}, \dbtext{leg}, \dbtext{everything}, \dbtext{mouse}, \dbtext{hole}, \dbtext{nature}, \dbtext{bald eagle}, \dbtext{nest}, 
\dbtext{crab}, \dbtext{ficus}, \dbtext{sea}, \dbtext{anemone}, \dbtext{ocean}, \dbtext{sun}, \dbtext{sky}, \dbtext{grape}, \dbtext{horse}, 
\dbtext{hot}, \dbtext{kill}, \dbtext{foot}, \dbtext{meadow}, \dbtext{camp}, \dbtext{den}, \dbtext{cow}, \dbtext{earth}, \dbtext{garden}, 
\dbtext{poop}, \dbtext{outside}, \dbtext{frog}, \dbtext{light}, \dbtext{fox}, \dbtext{forest}, \dbtext{marmot}, \dbtext{mountain}, \dbtext{drop}, 
\dbtext{bone}, \dbtext{rain}, \dbtext{body}, \dbtext{ferret}, \dbtext{small dog}, \dbtext{doll}, \dbtext{lemur}, \dbtext{name}, \dbtext{nice}, 
\dbtext{museum}, \dbtext{black}, \dbtext{canada}, \dbtext{bad}, \dbtext{wind}, \dbtext{hand}, \dbtext{pee}, \dbtext{road}, \dbtext{boat}, 
\dbtext{wild}, \dbtext{war}, \dbtext{wet}, \dbtext{flower}, \dbtext{small}, \dbtext{new york}, \dbtext{farm}, \dbtext{color}, \dbtext{red}, 
\dbtext{stone}, \dbtext{green}, \dbtext{life}, \dbtext{burn}, \dbtext{large}, \dbtext{soft}, \dbtext{fire}, \dbtext{finger}, \dbtext{dangerous}, 
\dbtext{marmoset}, \dbtext{australia}, \dbtext{leave}, \dbtext{heavy}, \dbtext{cuba}, \dbtext{france}, \dbtext{italy}, \dbtext{unite state}, 
\dbtext{hill}, \dbtext{apple tree}, \dbtext{god}, \dbtext{space}, \dbtext{mouth}, \dbtext{river}, \dbtext{blue}, \dbtext{grass}, \dbtext{mammal}, 
\dbtext{lot}, \dbtext{hair}, \dbtext{measure}, \dbtext{utah}, \dbtext{bug}, \dbtext{tooth}, \dbtext{sand}, \dbtext{dictionary}, \dbtext{rise}, 
\dbtext{not}, \dbtext{bite}, \dbtext{dark}, \dbtext{science}, \dbtext{world}, \dbtext{air}, \dbtext{sheep}, \dbtext{statue}, \dbtext{warm}, 
\dbtext{big}, \dbtext{high}, \dbtext{squirrel}, \dbtext{mean}, \dbtext{general}, \dbtext{heat}, \dbtext{cool}, \dbtext{skin}, \dbtext{art}, 
\dbtext{noun}, \dbtext{hard}, \dbtext{duck}, \dbtext{ring}, \dbtext{wyom}, \dbtext{thing}, \dbtext{land}, \dbtext{kill person}, \dbtext{little}, 
\dbtext{ear}, \dbtext{alive}, \dbtext{field}, \dbtext{wave}, \dbtext{face}, \dbtext{lawn}, \dbtext{long}, \dbtext{shade}, \dbtext{fly}, \dbtext{bee}, 
\dbtext{bear}, \dbtext{head}, \dbtext{bullet}, \dbtext{degree}, \dbtext{gas}, \dbtext{brown}, \dbtext{dirt}, \dbtext{adjective}, \dbtext{alaska}, 
\dbtext{michigan}, \dbtext{maryland}, \dbtext{maine}, \dbtext{delaware}, \dbtext{kansa}, \dbtext{be}, \dbtext{steam}, \dbtext{pretty}, 
\dbtext{decoration}, \dbtext{step}, \dbtext{part}, \dbtext{bush}, \dbtext{countryside}, \dbtext{same}, \dbtext{grow}, \dbtext{sunshine}, 
\dbtext{flea}, \dbtext{short}, \dbtext{outdoors}, \dbtext{stick}, \dbtext{singular}, \dbtext{find outside}, \dbtext{branch}, \dbtext{wax}, 
\dbtext{generic}, \dbtext{ground}, \dbtext{unit}

\paragraph{Community 3 (size is 1).} 
\dbtext{out}

\paragraph{Community 4 (size is 1).} 
\dbtext{course}

\paragraph{Community 5 (size is 40).} 
\dbtext{type}, \dbtext{word}, \dbtext{movie}, \dbtext{music}, \dbtext{game}, \dbtext{family}, \dbtext{concert}, \dbtext{call}, \dbtext{story}, 
\dbtext{speak}, \dbtext{band}, \dbtext{audience}, \dbtext{write}, \dbtext{pool}, \dbtext{show}, \dbtext{sound}, \dbtext{news}, \dbtext{club}, 
\dbtext{theatre}, \dbtext{theater}, \dbtext{noise}, \dbtext{view}, \dbtext{record}, \dbtext{song}, \dbtext{company}, \dbtext{communication}, 
\dbtext{act}, \dbtext{crowd}, \dbtext{voice}, \dbtext{event}, \dbtext{hear}, \dbtext{organization}, \dbtext{pass}, \dbtext{end}, \dbtext{group}, 
\dbtext{point}, \dbtext{stage}, \dbtext{race}, \dbtext{many person}, \dbtext{general term}

\paragraph{Community 6 (size is 234).} \dbtext{something}, \dbtext{town}, \dbtext{soup}, \dbtext{library}, \dbtext{school}, \dbtext{plane}, 
\dbtext{class}, \dbtext{drink}, \dbtext{paper}, \dbtext{bed}, \dbtext{dirty}, \dbtext{gym}, \dbtext{office build}, \dbtext{wiener dog}, \dbtext{box}, 
\dbtext{object}, \dbtext{mother}, \dbtext{coffee}, \dbtext{candle}, \dbtext{street}, \dbtext{bus}, \dbtext{eat}, \dbtext{computer}, \dbtext{line}, 
\dbtext{milk}, \dbtext{tv}, \dbtext{drawer}, \dbtext{storage}, \dbtext{car}, \dbtext{vehicle}, \dbtext{bottle}, \dbtext{turn}, \dbtext{material}, 
\dbtext{chair}, \dbtext{market}, \dbtext{house}, \dbtext{hotel}, \dbtext{hospital}, \dbtext{bank}, \dbtext{girl}, \dbtext{church}, \dbtext{cook}, 
\dbtext{shop}, \dbtext{letter}, \dbtext{bathroom}, \dbtext{city}, \dbtext{desk}, \dbtext{office}, \dbtext{home}, \dbtext{couch}, \dbtext{kitchen}, 
\dbtext{build}, \dbtext{restaurant}, \dbtext{spoon}, \dbtext{butter}, \dbtext{key}, \dbtext{electricity}, \dbtext{stand}, \dbtext{pen}, 
\dbtext{television}, \dbtext{magazine}, \dbtext{paint}, \dbtext{food}, \dbtext{bedroom}, \dbtext{store}, \dbtext{airport}, \dbtext{sugar}, 
\dbtext{grocery store}, \dbtext{basket}, \dbtext{hold}, \dbtext{refrigerator}, \dbtext{newspaper}, \dbtext{rice}, \dbtext{surface}, \dbtext{liquid}, 
\dbtext{window}, \dbtext{oil}, \dbtext{cover}, \dbtext{plate}, \dbtext{dinner}, \dbtext{garage}, \dbtext{potato}, \dbtext{napkin}, \dbtext{salad}, 
\dbtext{glass}, \dbtext{cupboard}, \dbtext{telephone}, \dbtext{salt}, \dbtext{motel}, \dbtext{meat}, \dbtext{bookstore}, \dbtext{use}, \dbtext{cloth}, 
\dbtext{factory}, \dbtext{bottle wine}, \dbtext{pencil}, \dbtext{wheel}, \dbtext{book}, \dbtext{instrument}, \dbtext{trash}, \dbtext{can}, \dbtext{a}, 
\dbtext{picture}, \dbtext{seat}, \dbtext{clothe}, \dbtext{dish}, \dbtext{train station}, \dbtext{mall}, \dbtext{wallet}, \dbtext{room}, \dbtext{cell}, 
\dbtext{bicycle}, \dbtext{sink}, \dbtext{pocket}, \dbtext{white}, \dbtext{vegetable}, \dbtext{shoe}, \dbtext{scale}, \dbtext{steak}, \dbtext{beer}, 
\dbtext{knife}, \dbtext{carpet}, \dbtext{bowl}, \dbtext{corn}, \dbtext{fridge}, \dbtext{soap}, \dbtext{expensive}, \dbtext{coin}, \dbtext{number}, 
\dbtext{fruit}, \dbtext{map}, \dbtext{fork}, \dbtext{steel}, \dbtext{piano}, \dbtext{wall}, \dbtext{cup}, \dbtext{square}, \dbtext{shelf}, 
\dbtext{friend house}, \dbtext{airplane}, \dbtext{phone}, \dbtext{this}, \dbtext{place}, \dbtext{radio}, \dbtext{tool}, \dbtext{apple}, 
\dbtext{bag}, \dbtext{doctor}, \dbtext{cheese}, \dbtext{bean}, \dbtext{make}, \dbtext{flat}, \dbtext{plastic}, \dbtext{container}, \dbtext{bar}, 
\dbtext{live room}, \dbtext{toilet}, \dbtext{cabinet}, \dbtext{table}, \dbtext{furniture}, \dbtext{lamp}, \dbtext{pizza}, \dbtext{dust}, 
\dbtext{hall}, \dbtext{closet}, \dbtext{boy}, \dbtext{door}, \dbtext{floor}, \dbtext{meet}, \dbtext{basement}, \dbtext{sofa}, \dbtext{cut}, 
\dbtext{page}, \dbtext{college}, \dbtext{metal}, \dbtext{open}, \dbtext{alcohol}, \dbtext{university}, \dbtext{roll}, \dbtext{clock}, 
\dbtext{round}, \dbtext{top}, \dbtext{wine}, \dbtext{jar}, \dbtext{put}, \dbtext{toy}, \dbtext{draw}, \dbtext{edible}, \dbtext{rug}, 
\dbtext{pot}, \dbtext{change}, \dbtext{bread}, \dbtext{tin}, \dbtext{oven}, \dbtext{carry}, \dbtext{test}, \dbtext{egg}, \dbtext{building}, 
\dbtext{business}, \dbtext{resturant}, \dbtext{wash}, \dbtext{sock}, \dbtext{bell}, \dbtext{sign}, \dbtext{pantry}, \dbtext{note}, 
\dbtext{card}, \dbtext{supermarket}, \dbtext{machine}, \dbtext{roof}, \dbtext{circle}, \dbtext{cake}, \dbtext{solid}, \dbtext{useful}, 
\dbtext{handle}, \dbtext{department}, \dbtext{side}, \dbtext{stapler}, \dbtext{classroom}, \dbtext{transport}, \dbtext{apartment}, 
\dbtext{any large city}, \dbtext{comfort}, \dbtext{edge}, \dbtext{case}, \dbtext{board}, \dbtext{corner}, \dbtext{find house}, \dbtext{winery}, 
\dbtext{polish}, \dbtext{eaten}, \dbtext{neighbor house}, \dbtext{usually}, \dbtext{generic term}

\paragraph{Community 7 (size is 368).} 
\dbtext{man}, \dbtext{person}, \dbtext{train}, \dbtext{work}, \dbtext{write program}, \dbtext{go concert}, \dbtext{hear music}, \dbtext{exercise}, 
\dbtext{love}, \dbtext{bath}, \dbtext{listen}, \dbtext{go performance}, \dbtext{take walk}, \dbtext{walk}, \dbtext{entertain}, \dbtext{run marathon}, 
\dbtext{wait line}, \dbtext{attend lecture}, \dbtext{study}, \dbtext{go walk}, \dbtext{play basketball}, \dbtext{fun}, \dbtext{bore}, 
\dbtext{wait table}, \dbtext{go see film}, \dbtext{go work}, \dbtext{watch tv show}, \dbtext{wake up morning}, \dbtext{dream}, \dbtext{shower}, 
\dbtext{child}, \dbtext{go fish}, \dbtext{tell story}, \dbtext{surf web}, \dbtext{play football}, \dbtext{go restaurant}, \dbtext{visit museum}, 
\dbtext{study subject}, \dbtext{live life}, \dbtext{go sport event}, \dbtext{go play}, \dbtext{sit}, \dbtext{play soccer}, \dbtext{go jog}, 
\dbtext{take shower}, \dbtext{play ball}, \dbtext{eat food}, \dbtext{watch movie}, \dbtext{watch film}, \dbtext{stretch}, \dbtext{play frisbee}, 
\dbtext{go school}, \dbtext{surprise}, \dbtext{paint picture}, \dbtext{go film}, \dbtext{party}, \dbtext{rest}, \dbtext{listen radio}, 
\dbtext{kiss}, \dbtext{remember}, \dbtext{housework}, \dbtext{clean}, \dbtext{lunch}, \dbtext{watch tv}, \dbtext{attend school}, 
\dbtext{play tennis}, \dbtext{comfortable}, \dbtext{play}, \dbtext{take bus}, \dbtext{conversation}, \dbtext{talk}, \dbtext{take course}, 
\dbtext{learn}, \dbtext{plan}, \dbtext{think}, \dbtext{go run}, \dbtext{sleep}, \dbtext{hang out bar}, \dbtext{plan vacation}, 
\dbtext{go see play}, \dbtext{attend class}, \dbtext{go swim}, \dbtext{ride bike}, \dbtext{buy}, \dbtext{eat restaurant}, \dbtext{stress}, 
\dbtext{boredom}, \dbtext{ticket}, \dbtext{use television}, \dbtext{dress}, \dbtext{entertainment}, \dbtext{listen music}, 
\dbtext{enjoyment}, \dbtext{hurt}, \dbtext{student}, \dbtext{muscle}, \dbtext{woman}, \dbtext{go movie}, \dbtext{enlightenment}, 
\dbtext{stand line}, \dbtext{attend classical concert}, \dbtext{death}, \dbtext{play sport}, \dbtext{eat dinner}, \dbtext{effort}, 
\dbtext{drive car}, \dbtext{traveling}, \dbtext{knowledge}, \dbtext{teach}, \dbtext{laugh joke}, \dbtext{run}, \dbtext{read book}, 
\dbtext{education}, \dbtext{take note}, \dbtext{travel}, \dbtext{go store}, \dbtext{see}, \dbtext{go sleep}, \dbtext{tire}, \dbtext{attention}, 
\dbtext{die}, \dbtext{fall asleep}, \dbtext{money}, \dbtext{run errand}, \dbtext{patience}, \dbtext{spend money}, \dbtext{cry}, 
\dbtext{pay bill}, \dbtext{earn money}, \dbtext{take bath}, \dbtext{drink water}, \dbtext{fatigue}, \dbtext{take break}, \dbtext{hike}, 
\dbtext{drink alcohol}, \dbtext{lie}, \dbtext{play chess}, \dbtext{friend}, \dbtext{anger}, \dbtext{read}, \dbtext{curiosity}, \dbtext{pay}, 
\dbtext{swim}, \dbtext{break}, \dbtext{verb}, \dbtext{drive}, \dbtext{use computer}, \dbtext{take film}, \dbtext{smile}, \dbtext{fiddle}, 
\dbtext{we}, \dbtext{wrestle}, \dbtext{see new}, \dbtext{dance}, \dbtext{fight}, \dbtext{job}, \dbtext{smart}, \dbtext{play baseball}, 
\dbtext{excite}, \dbtext{attend rock concert}, \dbtext{hear news}, \dbtext{contemplate}, \dbtext{pain}, \dbtext{understand}, 
\dbtext{stay healthy}, \dbtext{research}, \dbtext{learn new}, \dbtext{sweat}, \dbtext{headache}, \dbtext{fart}, \dbtext{read newspaper}, 
\dbtext{sport}, \dbtext{understand better}, \dbtext{write story}, \dbtext{stop}, \dbtext{transportation}, \dbtext{fall down}, \dbtext{practice}, 
\dbtext{help}, \dbtext{lose}, \dbtext{close eye}, \dbtext{satisfaction}, \dbtext{time}, \dbtext{answer question}, \dbtext{perform}, \dbtext{need}, 
\dbtext{everyone}, \dbtext{go somewhere}, \dbtext{good}, \dbtext{play card}, \dbtext{go}, \dbtext{sex}, \dbtext{wait}, \dbtext{buy ticket}, 
\dbtext{gain knowledge}, \dbtext{interest}, \dbtext{feel}, \dbtext{sit down}, \dbtext{ski}, \dbtext{surf}, \dbtext{teacher}, \dbtext{happiness}, 
\dbtext{exhaustion}, \dbtext{sit chair}, \dbtext{laugh}, \dbtext{relax}, \dbtext{waste time}, \dbtext{pleasure}, \dbtext{relaxation}, 
\dbtext{care}, \dbtext{procreate}, \dbtext{watch}, \dbtext{funny}, \dbtext{win}, \dbtext{go mall}, \dbtext{flirt}, \dbtext{pass time}, 
\dbtext{shape}, \dbtext{climb}, \dbtext{wash hand}, \dbtext{go home}, \dbtext{love else}, \dbtext{drunk}, \dbtext{peace}, \dbtext{sing}, 
\dbtext{buy beer}, \dbtext{internet}, \dbtext{kid}, \dbtext{like}, \dbtext{date}, \dbtext{find}, \dbtext{play game}, \dbtext{activity}, 
\dbtext{jog}, \dbtext{quiet}, \dbtext{skill}, \dbtext{hobby}, \dbtext{birthday}, \dbtext{tiredness}, \dbtext{drink coffee}, 
\dbtext{read magazine}, \dbtext{good time}, \dbtext{good health}, \dbtext{play hockey}, \dbtext{eat ice cream}, \dbtext{learn language}, 
\dbtext{dive}, \dbtext{go zoo}, \dbtext{go internet}, \dbtext{cash}, \dbtext{important}, \dbtext{read child}, \dbtext{enjoy yourself}, 
\dbtext{see movie}, \dbtext{energy}, \dbtext{emotion}, \dbtext{clean house}, \dbtext{fit}, \dbtext{view video}, \dbtext{play poker}, 
\dbtext{excitement}, \dbtext{move}, \dbtext{fly airplane}, \dbtext{ride horse}, \dbtext{stay bed}, \dbtext{look}, \dbtext{happy}, 
\dbtext{find information}, \dbtext{fear}, \dbtext{go vacation}, \dbtext{breathe}, \dbtext{recreation}, \dbtext{enjoy}, \dbtext{jump}, 
\dbtext{ride bicycle}, \dbtext{health}, \dbtext{communicate}, \dbtext{make money}, \dbtext{become tire}, \dbtext{action}, \dbtext{fall}, 
\dbtext{lose weight}, \dbtext{jump up down}, \dbtext{watch television}, \dbtext{count}, \dbtext{healthy}, \dbtext{know}, \dbtext{learn subject}, 
\dbtext{joy}, \dbtext{stand up}, \dbtext{information}, \dbtext{read letter}, \dbtext{lay}, \dbtext{jump rope}, \dbtext{celebrate}, 
\dbtext{sadness}, \dbtext{bike}, \dbtext{watch musician perform}, \dbtext{motion}, \dbtext{feel better}, \dbtext{compete}, \dbtext{feel good}, 
\dbtext{accident}, \dbtext{stay fit}, \dbtext{injury}, \dbtext{ride}, \dbtext{play piano}, \dbtext{learn world}, \dbtext{see exhibit}, 
\dbtext{release energy}, \dbtext{see art}, \dbtext{see excite story}, \dbtext{orgasm}, \dbtext{trip}, \dbtext{laughter}, \dbtext{express yourself}, 
\dbtext{discover truth}, \dbtext{see favorite show}, \dbtext{go party}, \dbtext{competition}, \dbtext{express information}, \dbtext{climb mountain}, 
\dbtext{attend meet}, \dbtext{fly kite}, \dbtext{examine}, \dbtext{meet friend}, \dbtext{read news}, \dbtext{shock}, \dbtext{return work}, 
\dbtext{see band}, \dbtext{visit art gallery}, \dbtext{earn live}, \dbtext{punch}, \dbtext{cool off}, \dbtext{watch television show}, 
\dbtext{socialize}, \dbtext{skate}, \dbtext{movement}, \dbtext{create art}, \dbtext{crossword puzzle}, \dbtext{enjoy film}, \dbtext{go pub}, 
\dbtext{feel happy}, \dbtext{play lacrosse}, \dbtext{socialis}, \dbtext{away}, \dbtext{physical activity}, \dbtext{get}, \dbtext{make person laugh}, 
\dbtext{make friend}, \dbtext{chat friend}, \dbtext{meet person}, \dbtext{meet interest person}, \dbtext{get drunk}, \dbtext{friend over}, 
\dbtext{get exercise}, \dbtext{get tire}, \dbtext{enjoy company friend}, \dbtext{play game friend}, \dbtext{get physical activity}, 
\dbtext{go opus}, \dbtext{get shape}, \dbtext{sit quietly}, \dbtext{do it}, \dbtext{get fit}, \dbtext{teach other person}, 
\dbtext{entertain person}, \dbtext{see person play game}

\subsection{Eigenvectors}
The algorithm used is the one implemented in \igraph 
which is based on \citep{communities:eigenvectors}.
Table \ref{tbl:communities:positive:eigenvector:cores} presents the results when we apply the algorithm for 
subgraphs in which we include vertices with successively lower coreness and allow positive polarity on the edges only.
The computation lasted $94.98$ seconds for $2$ runs.

\begin{table}[ht]
\caption{Applying the leading eigenvector algorithm for community finding implemented in \igraph by successively including
vertices with lower coreness on the undirected graph induced by the assertions with positive polarity (self-loops are removed).
In every row we have the number of vertices and the number of edges of each such subgraph together with
the number of components ($\abs{C}$) that we find in that subgraph. 
The next three columns present the number of communities found by the algorithm;
the average among all runs, the minimum, and the maximum.
The next three columns present the modularity achieved by the algorithm due to the cut induced by the communities;
the average among all runs, the minimum, and the maximum.
The entire computation lasted $94.98$ seconds for $2$ runs;
that is about $47.49$ seconds per run.}\label{tbl:communities:positive:eigenvector:cores}
\begin{center}
\begin{tabular}{|c|r|r|c||r|r|r||r|r|r||}\hline
\multirow{2}{*}{coreness} & \multicolumn{1}{c|}{\multirow{2}{*}{$\abs{V}$}} & \multicolumn{1}{c|}{\multirow{2}{*}{$\abs{E}$}} & \multirow{2}{*}{$\abs{C}$} & \multicolumn{3}{c||}{communities found} & \multicolumn{3}{c||}{modularity} \\\cline{5-10}
          &        &        &   & \multicolumn{1}{c|}{avg} & \multicolumn{1}{c|}{min} & \multicolumn{1}{c||}{max} & \multicolumn{1}{c|}{avg} & \multicolumn{1}{c|}{min} & \multicolumn{1}{c||}{max} \\\hline\hline
$\geq 26$ &    869 &  20526 &    1 &  4.000 &    4 &    4 & 0.304160 & 0.304160 & 0.304160 \\\hline
$\geq 25$ &   1167 &  27810 &    1 &  5.000 &    5 &    5 & 0.300161 & 0.300161 & 0.300161 \\\hline
$\geq 24$ &   1358 &  32314 &    1 &  4.000 &    4 &    4 & 0.289870 & 0.289870 & 0.289870 \\\hline
$\geq 23$ &   1514 &  35870 &    1 &  3.000 &    3 &    3 & 0.282888 & 0.282888 & 0.282888 \\\hline
$\geq 22$ &   1709 &  40099 &    1 &  4.000 &    4 &    4 & 0.287385 & 0.287385 & 0.287385 \\\hline
$\geq 21$ &   1865 &  43330 &    1 &  4.000 &    4 &    4 & 0.286488 & 0.286488 & 0.286488 \\\hline
$\geq 20$ &   2007 &  46145 &    1 &  4.000 &    4 &    4 & 0.289613 & 0.289613 & 0.289613 \\\hline
$\geq 19$ &   2173 &  49265 &    1 &  5.000 &    5 &    5 & 0.293102 & 0.293102 & 0.293102 \\\hline
$\geq 18$ &   2384 &  53011 &    1 &  4.000 &    4 &    4 & 0.292260 & 0.292260 & 0.292260 \\\hline
$\geq 17$ &   2617 &  56939 &    1 &  5.000 &    5 &    5 & 0.293102 & 0.293102 & 0.293102 \\\hline
$\geq 16$ &   2847 &  60583 &    1 &  4.000 &    4 &    4 & 0.278867 & 0.278867 & 0.278867 \\\hline
$\geq 15$ &   3105 &  64412 &    1 &  4.000 &    4 &    4 & 0.280229 & 0.280229 & 0.280229 \\\hline
$\geq 14$ &   3407 &  68613 &    1 &  5.000 &    5 &    5 & 0.271588 & 0.271588 & 0.271588 \\\hline
$\geq 13$ &   3746 &  72978 &    1 &  5.000 &    5 &    5 & 0.274841 & 0.274841 & 0.274841 \\\hline
$\geq 12$ &   4160 &  77882 &    1 &  6.000 &    6 &    6 & 0.282679 & 0.282679 & 0.282679 \\\hline
$\geq 11$ &   4634 &  83039 &    1 &  5.000 &    5 &    5 & 0.302635 & 0.302635 & 0.302635 \\\hline
$\geq 10$ &   5182 &  88462 &    1 &  4.000 &    4 &    4 & 0.281035 & 0.281035 & 0.281035 \\\hline
$\geq  9$ &   5883 &  94709 &    1 &  4.000 &    4 &    4 & 0.275202 & 0.275202 & 0.275202 \\\hline
$\geq  8$ &   6750 & 101564 &    1 &  4.000 &    4 &    4 & 0.273375 & 0.273375 & 0.273375 \\\hline
$\geq  7$ &   7904 & 109561 &    1 &  2.000 &    2 &    2 & 0.266965 & 0.266965 & 0.266965 \\\hline
$\geq  6$ &   9392 & 118389 &    1 &  4.000 &    4 &    4 & 0.282719 & 0.282719 & 0.282719 \\\hline
$\geq  5$ &  11483 & 128731 &    1 &  3.000 &    3 &    3 & 0.289544 & 0.289544 & 0.289544 \\\hline
$\geq  4$ &  14864 & 142112 &    1 &  3.000 &    3 &    3 & 0.290085 & 0.290085 & 0.290085 \\\hline
$\geq  3$ &  21812 & 162691 &    1 &  4.000 &    4 &    4 & 0.287835 & 0.287835 & 0.287835 \\\hline
$\geq  2$ &  41659 & 201678 &    4 & 12.000 &   12 &   12 & 0.247697 & 0.247697 & 0.247697 \\\hline
\end{tabular}
\end{center}
\end{table}

\subsubsection{Eigenvectors: Communities in the Inner-Most Core}
We have the following communities.
\paragraph{Community 1 (size is 209).} 
\dbtext{something}, \dbtext{town}, \dbtext{word}, \dbtext{library}, \dbtext{school}, \dbtext{human}, \dbtext{plane}, \dbtext{it}, \dbtext{paper}, 
\dbtext{bed}, \dbtext{dirty}, \dbtext{gym}, \dbtext{office build}, \dbtext{box}, \dbtext{object}, \dbtext{mother}, \dbtext{coffee}, \dbtext{candle}, 
\dbtext{clean}, \dbtext{street}, \dbtext{bus}, \dbtext{computer}, \dbtext{line}, \dbtext{drawer}, \dbtext{storage}, \dbtext{car}, \dbtext{vehicle}, 
\dbtext{bottle}, \dbtext{material}, \dbtext{chair}, \dbtext{hat}, \dbtext{market}, \dbtext{house}, \dbtext{hotel}, \dbtext{hospital}, \dbtext{bank}, 
\dbtext{girl}, \dbtext{church}, \dbtext{family}, \dbtext{shop}, \dbtext{letter}, \dbtext{bathroom}, \dbtext{city}, \dbtext{desk}, \dbtext{office}, 
\dbtext{home}, \dbtext{couch}, \dbtext{kitchen}, \dbtext{build}, \dbtext{restaurant}, \dbtext{key}, \dbtext{e\-lec\-tric\-i\-ty}, \dbtext{stand}, \dbtext{pen}, 
\dbtext{bill}, \dbtext{magazine}, \dbtext{paint}, \dbtext{bedroom}, \dbtext{store}, \dbtext{airport}, \dbtext{grocery store}, \dbtext{hold}, 
\dbtext{refrigerator}, \dbtext{newspaper}, \dbtext{surface}, \dbtext{window}, \dbtext{cover}, \dbtext{garage}, \dbtext{napkin}, \dbtext{light}, 
\dbtext{glass}, \dbtext{telephone}, \dbtext{motel}, \dbtext{bookstore}, \dbtext{use}, \dbtext{cloth}, \dbtext{factory}, \dbtext{bottle wine}, 
\dbtext{doll}, \dbtext{pencil}, \dbtext{wheel}, \dbtext{name}, \dbtext{book}, \dbtext{museum}, \dbtext{pool}, \dbtext{instrument}, \dbtext{trash}, 
\dbtext{picture}, \dbtext{road}, \dbtext{seat}, \dbtext{boat}, \dbtext{clothe}, \dbtext{train station}, \dbtext{war}, \dbtext{mall}, \dbtext{wallet}, 
\dbtext{room}, \dbtext{cell}, \dbtext{sink}, \dbtext{pocket}, \dbtext{large}, \dbtext{shoe}, \dbtext{scale}, \dbtext{beer}, \dbtext{knife}, 
\dbtext{carpet}, \dbtext{soap}, \dbtext{expensive}, \dbtext{coin}, \dbtext{number}, \dbtext{heavy}, \dbtext{map}, \dbtext{fork}, \dbtext{steel}, 
\dbtext{piano}, \dbtext{wall}, \dbtext{cup}, \dbtext{square}, \dbtext{shelf}, \dbtext{friend house}, \dbtext{airplane}, \dbtext{phone}, \dbtext{this}, 
\dbtext{place}, \dbtext{radio}, \dbtext{tool}, \dbtext{bag}, \dbtext{doctor}, \dbtext{theater}, \dbtext{make}, \dbtext{measure}, \dbtext{flat}, 
\dbtext{plastic}, \dbtext{container}, \dbtext{bar}, \dbtext{live room}, \dbtext{toilet}, \dbtext{cabinet}, \dbtext{table}, \dbtext{furniture}, 
\dbtext{lamp}, \dbtext{dust}, \dbtext{hall}, \dbtext{dictionary}, \dbtext{closet}, \dbtext{boy}, \dbtext{door}, \dbtext{floor}, \dbtext{basement}, 
\dbtext{sofa}, \dbtext{page}, \dbtext{company}, \dbtext{dark}, \dbtext{college}, \dbtext{metal}, \dbtext{open}, \dbtext{alcohol}, \dbtext{university}, 
\dbtext{clock}, \dbtext{noun}, \dbtext{hard}, \dbtext{toy}, \dbtext{ring}, \dbtext{crowd}, \dbtext{draw}, \dbtext{thing}, \dbtext{rug}, \dbtext{change}, 
\dbtext{carry}, \dbtext{organization}, \dbtext{building}, \dbtext{business}, \dbtext{wash}, \dbtext{sock}, \dbtext{bell}, \dbtext{sign}, \dbtext{degree}, 
\dbtext{note}, \dbtext{card}, \dbtext{supermarket}, \dbtext{machine}, \dbtext{roof}, \dbtext{circle}, \dbtext{solid}, \dbtext{point}, \dbtext{useful}, 
\dbtext{handle}, \dbtext{department}, \dbtext{side}, \dbtext{stapler}, \dbtext{classroom}, \dbtext{transport}, \dbtext{step}, \dbtext{apartment}, 
\dbtext{part}, \dbtext{stage}, \dbtext{any large city}, \dbtext{comfort}, \dbtext{case}, \dbtext{board}, \dbtext{corner}, \dbtext{singular}, 
\dbtext{find house}, \dbtext{winery}, \dbtext{polish}, \dbtext{wax}, \dbtext{neighbor house}, \dbtext{usually}, \dbtext{unit}

\paragraph{Community 2 (size is 160).} 
\dbtext{man}, \dbtext{train}, \dbtext{work}, \dbtext{exercise}, \dbtext{take walk}, \dbtext{walk}, \dbtext{run marathon}, \dbtext{go walk}, 
\dbtext{play basketball}, \dbtext{fun}, \dbtext{wait table}, \dbtext{go work}, \dbtext{shower}, \dbtext{child}, \dbtext{smoke}, \dbtext{go fish}, 
\dbtext{play football}, \dbtext{live life}, \dbtext{play soccer}, \dbtext{go jog}, \dbtext{take shower}, \dbtext{play ball}, \dbtext{stretch}, 
\dbtext{play frisbee}, \dbtext{rest}, \dbtext{housework}, \dbtext{play tennis}, \dbtext{play}, \dbtext{take bus}, \dbtext{plan}, \dbtext{go run}, 
\dbtext{sleep}, \dbtext{hang out bar}, \dbtext{go swim}, \dbtext{ride bike}, \dbtext{dress}, \dbtext{one}, \dbtext{turn}, \dbtext{hurt}, \dbtext{game}, 
\dbtext{muscle}, \dbtext{woman}, \dbtext{death}, \dbtext{play sport}, \dbtext{effort}, \dbtext{drive car}, \dbtext{traveling}, \dbtext{run}, 
\dbtext{travel}, \dbtext{go store}, \dbtext{tire}, \dbtext{die}, \dbtext{run errand}, \dbtext{earn money}, \dbtext{drink water}, \dbtext{fatigue}, 
\dbtext{take break}, \dbtext{hike}, \dbtext{lie}, \dbtext{anger}, \dbtext{kill}, \dbtext{swim}, \dbtext{break}, \dbtext{verb}, \dbtext{drive}, 
\dbtext{wrestle}, \dbtext{dance}, \dbtext{fight}, \dbtext{play baseball}, \dbtext{attend rock concert}, \dbtext{pain}, \dbtext{stay healthy}, 
\dbtext{sweat}, \dbtext{sport}, \dbtext{stop}, \dbtext{transportation}, \dbtext{fall down}, \dbtext{practice}, \dbtext{lose}, \dbtext{bicycle}, 
\dbtext{go somewhere}, \dbtext{life}, \dbtext{go}, \dbtext{sex}, \dbtext{ski}, \dbtext{surf}, \dbtext{exhaustion}, \dbtext{procreate}, \dbtext{win}, 
\dbtext{go mall}, \dbtext{shape}, \dbtext{climb}, \dbtext{go home}, \dbtext{play game}, \dbtext{activity}, \dbtext{jog}, \dbtext{skill}, 
\dbtext{tiredness}, \dbtext{good health}, \dbtext{play hockey}, \dbtext{cool}, \dbtext{dive}, \dbtext{cash}, \dbtext{energy}, \dbtext{clean house}, 
\dbtext{fit}, \dbtext{move}, \dbtext{fly airplane}, \dbtext{ride horse}, \dbtext{fear}, \dbtext{go vacation}, \dbtext{breathe}, \dbtext{recreation}, 
\dbtext{jump}, \dbtext{ride bicycle}, \dbtext{health}, \dbtext{become tire}, \dbtext{action}, \dbtext{pass}, \dbtext{fall}, \dbtext{lose weight}, 
\dbtext{jump up down}, \dbtext{count}, \dbtext{healthy}, \dbtext{stand up}, \dbtext{lay}, \dbtext{jump rope}, \dbtext{bike}, \dbtext{motion}, 
\dbtext{feel better}, \dbtext{compete}, \dbtext{feel good}, \dbtext{accident}, \dbtext{stay fit}, \dbtext{injury}, \dbtext{ride}, 
\dbtext{release energy}, \dbtext{trip}, \dbtext{competition}, \dbtext{climb mountain}, \dbtext{fly kite}, \dbtext{race}, \dbtext{shock}, 
\dbtext{return work}, \dbtext{punch}, \dbtext{cool off}, \dbtext{skate}, \dbtext{movement}, \dbtext{go pub}, \dbtext{play lacrosse}, \dbtext{away}, 
\dbtext{physical activity}, \dbtext{get}, \dbtext{get drunk}, \dbtext{get exercise}, \dbtext{get tire}, \dbtext{get physical activity}, 
\dbtext{get shape}, \dbtext{do it}, \dbtext{get fit}

\paragraph{Community 3 (size is 245).} 
\dbtext{rock}, \dbtext{beach}, \dbtext{tree}, \dbtext{monkey}, \dbtext{soup}, \dbtext{weasel}, \dbtext{pant}, \dbtext{kitten}, \dbtext{arm}, 
\dbtext{beaver}, \dbtext{drink}, \dbtext{chicken}, \dbtext{state}, \dbtext{wiener dog}, \dbtext{ball}, \dbtext{eat food}, \dbtext{fungus}, 
\dbtext{park}, \dbtext{snake}, \dbtext{wood}, \dbtext{eat}, \dbtext{bridge}, \dbtext{cloud}, \dbtext{milk}, \dbtext{dog}, \dbtext{zoo}, 
\dbtext{live}, \dbtext{cat}, \dbtext{country}, \dbtext{fish}, \dbtext{lake}, \dbtext{plant}, \dbtext{hide}, \dbtext{animal}, \dbtext{cold}, 
\dbtext{moon}, \dbtext{pet}, \dbtext{cook}, \dbtext{bird}, \dbtext{shark}, \dbtext{water}, \dbtext{rosebush}, \dbtext{yard}, \dbtext{sloth}, 
\dbtext{bat}, \dbtext{lizard}, \dbtext{spoon}, \dbtext{butter}, \dbtext{beautiful}, \dbtext{eye}, \dbtext{nose}, \dbtext{smell}, \dbtext{well}, 
\dbtext{snow}, \dbtext{weather}, \dbtext{leg}, \dbtext{everything}, \dbtext{mouse}, \dbtext{hole}, \dbtext{nature}, \dbtext{bald eagle}, 
\dbtext{nest}, \dbtext{crab}, \dbtext{ficus}, \dbtext{sea}, \dbtext{anemone}, \dbtext{ocean}, \dbtext{sun}, \dbtext{sky}, \dbtext{food}, 
\dbtext{grape}, \dbtext{horse}, \dbtext{hot}, \dbtext{sugar}, \dbtext{basket}, \dbtext{foot}, \dbtext{rice}, \dbtext{liquid}, \dbtext{meadow}, 
\dbtext{camp}, \dbtext{oil}, \dbtext{plate}, \dbtext{dinner}, \dbtext{den}, \dbtext{cow}, \dbtext{earth}, \dbtext{garden}, \dbtext{poop}, 
\dbtext{potato}, \dbtext{outside}, \dbtext{frog}, \dbtext{salad}, \dbtext{fox}, \dbtext{forest}, \dbtext{cupboard}, \dbtext{marmot}, 
\dbtext{mountain}, \dbtext{salt}, \dbtext{drop}, \dbtext{bone}, \dbtext{meat}, \dbtext{rain}, \dbtext{body}, \dbtext{ferret}, \dbtext{small dog}, 
\dbtext{lemur}, \dbtext{black}, \dbtext{canada}, \dbtext{can}, \dbtext{a}, \dbtext{wind}, \dbtext{hand}, \dbtext{pee}, \dbtext{wild}, \dbtext{dish}, 
\dbtext{wet}, \dbtext{flower}, \dbtext{small}, \dbtext{new york}, \dbtext{farm}, \dbtext{color}, \dbtext{white}, \dbtext{red}, \dbtext{stone}, 
\dbtext{vegetable}, \dbtext{green}, \dbtext{burn}, \dbtext{soft}, \dbtext{steak}, \dbtext{fire}, \dbtext{finger}, \dbtext{dangerous}, 
\dbtext{marmoset}, \dbtext{bowl}, \dbtext{australia}, \dbtext{corn}, \dbtext{fridge}, \dbtext{leave}, \dbtext{fruit}, \dbtext{cuba}, 
\dbtext{france}, \dbtext{italy}, \dbtext{unite state}, \dbtext{hill}, \dbtext{apple tree}, \dbtext{god}, \dbtext{space}, \dbtext{apple}, 
\dbtext{mouth}, \dbtext{river}, \dbtext{blue}, \dbtext{grass}, \dbtext{cheese}, \dbtext{mammal}, \dbtext{bean}, \dbtext{lot}, \dbtext{hair}, 
\dbtext{utah}, \dbtext{wash hand}, \dbtext{bug}, \dbtext{tooth}, \dbtext{pizza}, \dbtext{sand}, \dbtext{rise}, \dbtext{not}, \dbtext{cut}, 
\dbtext{bite}, \dbtext{world}, \dbtext{air}, \dbtext{sheep}, \dbtext{statue}, \dbtext{warm}, \dbtext{big}, \dbtext{high}, \dbtext{squirrel}, 
\dbtext{roll}, \dbtext{general}, \dbtext{round}, \dbtext{heat}, \dbtext{skin}, \dbtext{top}, \dbtext{wine}, \dbtext{jar}, \dbtext{put}, 
\dbtext{duck}, \dbtext{edible}, \dbtext{wyom}, \dbtext{land}, \dbtext{pot}, \dbtext{little}, \dbtext{ear}, \dbtext{alive}, \dbtext{bread}, 
\dbtext{field}, \dbtext{wave}, \dbtext{face}, \dbtext{lawn}, \dbtext{tin}, \dbtext{oven}, \dbtext{long}, \dbtext{shade}, \dbtext{fly}, 
\dbtext{egg}, \dbtext{bee}, \dbtext{resturant}, \dbtext{bear}, \dbtext{head}, \dbtext{group}, \dbtext{pantry}, \dbtext{bullet}, \dbtext{gas}, 
\dbtext{brown}, \dbtext{cake}, \dbtext{dirt}, \dbtext{adjective}, \dbtext{alaska}, \dbtext{michigan}, \dbtext{maryland}, \dbtext{maine}, 
\dbtext{delaware}, \dbtext{kansa}, \dbtext{be}, \dbtext{steam}, \dbtext{pretty}, \dbtext{decoration}, \dbtext{out}, \dbtext{bush}, 
\dbtext{course}, \dbtext{countryside}, \dbtext{power}, \dbtext{same}, \dbtext{edge}, \dbtext{grow}, \dbtext{sunshine}, \dbtext{flea}, 
\dbtext{short}, \dbtext{outdoors}, \dbtext{stick}, \dbtext{find outside}, \dbtext{branch}, \dbtext{general term}, \dbtext{generic}, 
\dbtext{ground}, \dbtext{eaten}, \dbtext{generic term}

\paragraph{Community 4 (size is 255).} 
\dbtext{person}, \dbtext{type}, \dbtext{write program}, \dbtext{go concert}, \dbtext{hear music}, \dbtext{love}, \dbtext{bath}, \dbtext{listen}, 
\dbtext{go performance}, \dbtext{class}, \dbtext{entertain}, \dbtext{wait line}, \dbtext{attend lecture}, \dbtext{study}, \dbtext{bore}, 
\dbtext{go see film}, \dbtext{watch tv show}, \dbtext{wake up morning}, \dbtext{dream}, \dbtext{tell story}, \dbtext{surf web}, \dbtext{movie}, 
\dbtext{go restaurant}, \dbtext{visit museum}, \dbtext{study subject}, \dbtext{go sport event}, \dbtext{go play}, \dbtext{sit}, \dbtext{watch movie}, 
\dbtext{watch film}, \dbtext{go school}, \dbtext{surprise}, \dbtext{paint picture}, \dbtext{go film}, \dbtext{party}, \dbtext{listen radio}, 
\dbtext{kiss}, \dbtext{remember}, \dbtext{lunch}, \dbtext{watch tv}, \dbtext{attend school}, \dbtext{trouble}, \dbtext{comfortable}, 
\dbtext{conversation}, \dbtext{talk}, \dbtext{take course}, \dbtext{learn}, \dbtext{think}, \dbtext{plan vacation}, \dbtext{go see play}, 
\dbtext{attend class}, \dbtext{nothing}, \dbtext{buy}, \dbtext{eat restaurant}, \dbtext{tv}, \dbtext{stress}, \dbtext{boredom}, \dbtext{ticket}, 
\dbtext{music}, \dbtext{use television}, \dbtext{entertainment}, \dbtext{listen music}, \dbtext{enjoyment}, \dbtext{baby}, \dbtext{student}, 
\dbtext{go movie}, \dbtext{enlightenment}, \dbtext{stand line}, \dbtext{attend classical concert}, \dbtext{eat dinner}, \dbtext{concert}, 
\dbtext{knowledge}, \dbtext{teach}, \dbtext{call}, \dbtext{laugh joke}, \dbtext{read book}, \dbtext{education}, \dbtext{take note}, \dbtext{see}, 
\dbtext{story}, \dbtext{go sleep}, \dbtext{attention}, \dbtext{fall asleep}, \dbtext{money}, \dbtext{patience}, \dbtext{spend money}, \dbtext{cry}, 
\dbtext{pay bill}, \dbtext{television}, \dbtext{speak}, \dbtext{take bath}, \dbtext{band}, \dbtext{drink alcohol}, \dbtext{play chess}, 
\dbtext{friend}, \dbtext{read}, \dbtext{curiosity}, \dbtext{pay}, \dbtext{use computer}, \dbtext{take film}, \dbtext{smile}, \dbtext{fiddle}, 
\dbtext{we}, \dbtext{see new}, \dbtext{job}, \dbtext{smart}, \dbtext{excite}, \dbtext{hear news}, \dbtext{contemplate}, \dbtext{audience}, 
\dbtext{understand}, \dbtext{write}, \dbtext{research}, \dbtext{learn new}, \dbtext{nice}, \dbtext{headache}, \dbtext{fart}, 
\dbtext{read newspaper}, \dbtext{understand better}, \dbtext{bad}, \dbtext{show}, \dbtext{write story}, \dbtext{help}, \dbtext{close eye}, 
\dbtext{satisfaction}, \dbtext{time}, \dbtext{answer question}, \dbtext{perform}, \dbtext{need}, \dbtext{everyone}, \dbtext{sound}, \dbtext{good}, 
\dbtext{play card}, \dbtext{wait}, \dbtext{buy ticket}, \dbtext{gain knowledge}, \dbtext{news}, \dbtext{interest}, \dbtext{feel}, \dbtext{sit down}, 
\dbtext{teacher}, \dbtext{happiness}, \dbtext{sit chair}, \dbtext{laugh}, \dbtext{club}, \dbtext{theatre}, \dbtext{relax}, \dbtext{waste time}, 
\dbtext{pleas\-ure}, \dbtext{relaxation}, \dbtext{care}, \dbtext{watch}, \dbtext{funny}, \dbtext{flirt}, \dbtext{pass time}, \dbtext{noise}, 
\dbtext{view}, \dbtext{love else}, \dbtext{drunk}, \dbtext{peace}, \dbtext{sing}, \dbtext{buy beer}, \dbtext{internet}, \dbtext{kid}, \dbtext{like}, 
\dbtext{date}, \dbtext{record}, \dbtext{find}, \dbtext{song}, \dbtext{meet}, \dbtext{science}, \dbtext{quiet}, \dbtext{hobby}, \dbtext{birthday}, 
\dbtext{mean}, \dbtext{com\-mu\-ni\-ca\-tion}, \dbtext{drink coffee}, \dbtext{read magazine}, \dbtext{good time}, \dbtext{act}, \dbtext{eat ice cream}, 
\dbtext{learn language}, \dbtext{go zoo}, \dbtext{go internet}, \dbtext{art}, \dbtext{important}, \dbtext{read child}, \dbtext{enjoy yourself}, 
\dbtext{see movie}, \dbtext{kill person}, \dbtext{emotion}, \dbtext{view video}, \dbtext{play poker}, \dbtext{ex\-cite\-ment}, \dbtext{stay bed}, 
\dbtext{look}, \dbtext{voice}, \dbtext{event}, \dbtext{happy}, \dbtext{find information}, \dbtext{test}, \dbtext{enjoy}, \dbtext{hear}, 
\dbtext{com\-mu\-ni\-cate}, \dbtext{make money}, \dbtext{watch television}, \dbtext{end}, \dbtext{know}, \dbtext{learn subject}, \dbtext{joy}, 
\dbtext{information}, \dbtext{read letter}, \dbtext{cel\-e\-brate}, \dbtext{sadness}, \dbtext{watch musician perform}, \dbtext{play piano}, 
\dbtext{learn world}, \dbtext{see exhibit}, \dbtext{see art}, \dbtext{see excite story}, \dbtext{orgasm}, \dbtext{laughter}, 
\dbtext{express yourself}, \dbtext{discover truth}, \dbtext{see favorite show}, \dbtext{go party}, \dbtext{express in\-for\-ma\-tion}, 
\dbtext{attend meet}, \dbtext{examine}, \dbtext{meet friend}, \dbtext{read news}, \dbtext{see band}, \dbtext{visit art gallery}, 
\dbtext{earn live}, \dbtext{watch television show}, \dbtext{socialize}, \dbtext{create art}, \dbtext{crossword puzzle}, \dbtext{enjoy film}, 
\dbtext{feel happy}, \dbtext{socialis}, \dbtext{many person}, \dbtext{make person laugh}, \dbtext{make friend}, \dbtext{chat friend}, 
\dbtext{meet person}, \dbtext{meet interest person}, \dbtext{friend over}, \dbtext{enjoy company friend}, \dbtext{play game friend}, 
\dbtext{go opus}, \dbtext{sit quietly}, \dbtext{teach other person}, \dbtext{entertain person}, \dbtext{see person play game}

\subsection{Walktrap}
The algorithm used is the one implemented in \igraph 
which is based on \citep{communities:walktrap}.
Table \ref{tbl:communities:positive:walktrap:cores} presents the results when we apply the algorithm for 
subgraphs in which we include vertices with successively lower coreness and allow positive polarity on the edges only.
We use random walks of length $5$ throughout all the runs.

\begin{table}[ht]
\caption{Applying the Walktrap algorithm for community finding implemented in \igraph by successively including
vertices with lower coreness on the undirected graph induced by the assertions with positive polarity (self-loops are removed).
We use $5$ steps for every random walk generated throughout all our runs.
In every row we have the number of vertices and the number of edges of each such subgraph together with
the number of components ($\abs{C}$) that we find in that subgraph. 
The next three columns present the number of communities found by the algorithm;
the average among all runs, the minimum, and the maximum.
The next three columns present the modularity achieved by the algorithm due to the cut induced by the communities;
the average among all runs, the minimum, and the maximum.
The entire computation lasted $1,144.71$ seconds for 2 runs.}\label{tbl:communities:positive:walktrap:cores}
\begin{center}
\begin{tabular}{|c|r|r|c||r|r|r||r|r|r||}\hline
\multirow{2}{*}{coreness} & \multicolumn{1}{c|}{\multirow{2}{*}{$\abs{V}$}} & \multicolumn{1}{c|}{\multirow{2}{*}{$\abs{E}$}} & \multirow{2}{*}{$\abs{C}$} & \multicolumn{3}{c||}{communities found} & \multicolumn{3}{c||}{modularity} \\\cline{5-10}
          &        &        &   & \multicolumn{1}{c|}{avg} & \multicolumn{1}{c|}{min} & \multicolumn{1}{c||}{max} & \multicolumn{1}{c|}{avg} & \multicolumn{1}{c|}{min} & \multicolumn{1}{c||}{max} \\\hline\hline
$\geq 26$ &    869 &  20526 &    1 &   3.000 &    3 &    3 & 0.274815 & 0.274815 & 0.274815 \\\hline
$\geq 25$ &   1167 &  27810 &    1 &   3.000 &    3 &    3 & 0.282369 & 0.282369 & 0.282369 \\\hline
$\geq 24$ &   1358 &  32314 &    1 &   3.000 &    3 &    3 & 0.281649 & 0.281649 & 0.281649 \\\hline
$\geq 23$ &   1514 &  35870 &    1 &   4.000 &    4 &    4 & 0.275340 & 0.275340 & 0.275340 \\\hline
$\geq 22$ &   1709 &  40099 &    1 &   3.000 &    3 &    3 & 0.286858 & 0.286858 & 0.286858 \\\hline
$\geq 21$ &   1865 &  43330 &    1 &   3.000 &    3 &    3 & 0.283172 & 0.283172 & 0.283172 \\\hline
$\geq 20$ &   2007 &  46145 &    1 &   4.000 &    4 &    4 & 0.284822 & 0.284822 & 0.284822 \\\hline
$\geq 19$ &   2173 &  49265 &    1 &   3.000 &    3 &    3 & 0.273657 & 0.273657 & 0.273657 \\\hline
$\geq 18$ &   2384 &  53011 &    1 &   4.000 &    4 &    4 & 0.291989 & 0.291989 & 0.291989 \\\hline
$\geq 17$ &   2617 &  56939 &    1 &   5.000 &    5 &    5 & 0.279475 & 0.279475 & 0.279475 \\\hline
$\geq 16$ &   2847 &  60583 &    1 &   4.000 &    4 &    4 & 0.285183 & 0.285183 & 0.285183 \\\hline
$\geq 15$ &   3105 &  64412 &    1 &   4.000 &    4 &    4 & 0.284449 & 0.284449 & 0.284449 \\\hline
$\geq 14$ &   3407 &  68613 &    1 &   6.000 &    6 &    6 & 0.306199 & 0.306199 & 0.306199 \\\hline
$\geq 13$ &   3746 &  72978 &    1 &   5.000 &    5 &    5 & 0.291219 & 0.291219 & 0.291219 \\\hline
$\geq 12$ &   4160 &  77882 &    1 &   8.000 &    8 &    8 & 0.290368 & 0.290368 & 0.290368 \\\hline
$\geq 11$ &   4634 &  83039 &    1 &   8.000 &    8 &    8 & 0.320025 & 0.320025 & 0.320025 \\\hline
$\geq 10$ &   5182 &  88462 &    1 &  10.000 &   10 &   10 & 0.311006 & 0.311006 & 0.311006 \\\hline
$\geq  9$ &   5883 &  94709 &    1 &  13.000 &   13 &   13 & 0.318720 & 0.318720 & 0.318720 \\\hline
$\geq  8$ &   6750 & 101564 &    1 &  17.000 &   17 &   17 & 0.322721 & 0.322721 & 0.322721 \\\hline
$\geq  7$ &   7904 & 109561 &    1 &  15.000 &   15 &   15 & 0.327759 & 0.327759 & 0.327759 \\\hline
$\geq  6$ &   9392 & 118389 &    1 &  17.000 &   17 &   17 & 0.340760 & 0.340760 & 0.340760 \\\hline
$\geq  5$ &  11483 & 128731 &    1 &  31.000 &   31 &   31 & 0.338872 & 0.338872 & 0.338872 \\\hline
$\geq  4$ &  14864 & 142112 &    1 &  54.000 &   54 &   54 & 0.333880 & 0.333880 & 0.333880 \\\hline
$\geq  3$ &  21812 & 162691 &    1 & 223.000 &  223 &  223 & 0.342930 & 0.342930 & 0.342930 \\\hline
$\geq  2$ &  42576 & 208346 &    4 & --      &  -- &  -- & -- & -- & -- \\\hline
\end{tabular}
\end{center}
\end{table}

\begin{remark}[Walktrap Memory Requirements]
The drawback of the implementation is that it requires too much memory in order to run.
Applying the algorithm on the subgraph induced by vertices with coreness at least $3$ required
about $2.1$ GBytes of RAM.
\end{remark}

\subsubsection{Walktrap: Communities in the Inner-Most Core}
We have the following communities.
\paragraph{Community 1 (size is 292).}
\dbtext{town}, \dbtext{rock}, \dbtext{beach}, \dbtext{tree}, \dbtext{monkey}, \dbtext{soup}, \dbtext{weasel}, \dbtext{kitten}, \dbtext{plane}, 
\dbtext{beaver}, \dbtext{paper}, \dbtext{bed}, \dbtext{dirty}, \dbtext{chicken}, \dbtext{state}, \dbtext{office build}, \dbtext{wiener dog}, 
\dbtext{box}, \dbtext{object}, \dbtext{coffee}, \dbtext{candle}, \dbtext{fungus}, \dbtext{park}, \dbtext{snake}, \dbtext{wood}, \dbtext{cloud}, 
\dbtext{milk}, \dbtext{drawer}, \dbtext{storage}, \dbtext{zoo}, \dbtext{bottle}, \dbtext{material}, \dbtext{chair}, \dbtext{cat}, \dbtext{hat}, 
\dbtext{country}, \dbtext{market}, \dbtext{house}, \dbtext{lake}, \dbtext{hotel}, \dbtext{plant}, \dbtext{bank}, \dbtext{animal}, \dbtext{church}, 
\dbtext{moon}, \dbtext{pet}, \dbtext{cook}, \dbtext{bird}, \dbtext{bathroom}, \dbtext{city}, \dbtext{shark}, \dbtext{water}, \dbtext{rosebush}, 
\dbtext{yard}, \dbtext{desk}, \dbtext{office}, \dbtext{home}, \dbtext{sloth}, \dbtext{bat}, \dbtext{couch}, \dbtext{kitchen}, \dbtext{lizard}, 
\dbtext{build}, \dbtext{restaurant}, \dbtext{spoon}, \dbtext{butter}, \dbtext{beautiful}, \dbtext{key}, \dbtext{well}, \dbtext{pen}, \dbtext{snow}, 
\dbtext{weather}, \dbtext{mouse}, \dbtext{magazine}, \dbtext{hole}, \dbtext{nature}, \dbtext{bald eagle}, \dbtext{nest}, \dbtext{crab}, 
\dbtext{ficus}, \dbtext{sea}, \dbtext{anemone}, \dbtext{ocean}, \dbtext{sun}, \dbtext{sky}, \dbtext{food}, \dbtext{grape}, \dbtext{bedroom}, 
\dbtext{horse}, \dbtext{store}, \dbtext{airport}, \dbtext{sugar}, \dbtext{grocery store}, \dbtext{basket}, \dbtext{hold}, \dbtext{refrigerator}, 
\dbtext{newspaper}, \dbtext{rice}, \dbtext{surface}, \dbtext{liquid}, \dbtext{meadow}, \dbtext{window}, \dbtext{oil}, \dbtext{cover}, 
\dbtext{plate}, \dbtext{dinner}, \dbtext{cow}, \dbtext{earth}, \dbtext{garage}, \dbtext{garden}, \dbtext{poop}, \dbtext{potato}, \dbtext{outside}, 
\dbtext{frog}, \dbtext{napkin}, \dbtext{light}, \dbtext{salad}, \dbtext{fox}, \dbtext{forest}, \dbtext{glass}, \dbtext{cupboard}, \dbtext{marmot}, 
\dbtext{mountain}, \dbtext{salt}, \dbtext{motel}, \dbtext{bone}, \dbtext{meat}, \dbtext{bookstore}, \dbtext{ferret}, \dbtext{small dog}, 
\dbtext{cloth}, \dbtext{factory}, \dbtext{bottle wine}, \dbtext{pencil}, \dbtext{lemur}, \dbtext{black}, \dbtext{canada}, \dbtext{trash}, 
\dbtext{can}, \dbtext{a}, \dbtext{picture}, \dbtext{wild}, \dbtext{clothe}, \dbtext{dish}, \dbtext{mall}, \dbtext{flower}, \dbtext{wallet}, 
\dbtext{room}, \dbtext{small}, \dbtext{new york}, \dbtext{farm}, \dbtext{sink}, \dbtext{pocket}, \dbtext{color}, \dbtext{white}, \dbtext{red}, 
\dbtext{stone}, \dbtext{vegetable}, \dbtext{green}, \dbtext{large}, \dbtext{shoe}, \dbtext{scale}, \dbtext{soft}, \dbtext{steak}, \dbtext{beer}, 
\dbtext{knife}, \dbtext{marmoset}, \dbtext{carpet}, \dbtext{bowl}, \dbtext{australia}, \dbtext{corn}, \dbtext{fridge}, \dbtext{soap}, \dbtext{coin}, 
\dbtext{fruit}, \dbtext{fork}, \dbtext{cuba}, \dbtext{france}, \dbtext{italy}, \dbtext{steel}, \dbtext{wall}, \dbtext{unite state}, \dbtext{cup}, 
\dbtext{hill}, \dbtext{square}, \dbtext{apple tree}, \dbtext{shelf}, \dbtext{friend house}, \dbtext{space}, \dbtext{this}, \dbtext{place}, 
\dbtext{apple}, \dbtext{bag}, \dbtext{river}, \dbtext{blue}, \dbtext{grass}, \dbtext{cheese}, \dbtext{mammal}, \dbtext{bean}, \dbtext{flat}, 
\dbtext{utah}, \dbtext{plastic}, \dbtext{container}, \dbtext{bug}, \dbtext{live room}, \dbtext{toilet}, \dbtext{cabinet}, \dbtext{table}, 
\dbtext{furniture}, \dbtext{lamp}, \dbtext{pizza}, \dbtext{dust}, \dbtext{sand}, \dbtext{hall}, \dbtext{dictionary}, \dbtext{rise}, \dbtext{closet}, 
\dbtext{door}, \dbtext{floor}, \dbtext{basement}, \dbtext{dark}, \dbtext{world}, \dbtext{air}, \dbtext{sheep}, \dbtext{statue}, \dbtext{metal}, 
\dbtext{big}, \dbtext{squirrel}, \dbtext{alcohol}, \dbtext{clock}, \dbtext{round}, \dbtext{top}, \dbtext{wine}, \dbtext{jar}, \dbtext{put}, 
\dbtext{duck}, \dbtext{edible}, \dbtext{wyom}, \dbtext{land}, \dbtext{rug}, \dbtext{pot}, \dbtext{bread}, \dbtext{field}, \dbtext{lawn}, 
\dbtext{tin}, \dbtext{oven}, \dbtext{shade}, \dbtext{egg}, \dbtext{building}, \dbtext{bee}, \dbtext{resturant}, \dbtext{wash}, \dbtext{sock}, 
\dbtext{bear}, \dbtext{pantry}, \dbtext{supermarket}, \dbtext{roof}, \dbtext{brown}, \dbtext{cake}, \dbtext{solid}, \dbtext{dirt}, \dbtext{handle}, 
\dbtext{alaska}, \dbtext{michigan}, \dbtext{maryland}, \dbtext{maine}, \dbtext{delaware}, \dbtext{kansa}, \dbtext{department}, \dbtext{pretty}, 
\dbtext{decoration}, \dbtext{stapler}, \dbtext{apartment}, \dbtext{bush}, \dbtext{countryside}, \dbtext{any large city}, \dbtext{case}, 
\dbtext{sunshine}, \dbtext{corner}, \dbtext{outdoors}, \dbtext{stick}, \dbtext{find house}, \dbtext{find outside}, \dbtext{winery}, \dbtext{branch}, 
\dbtext{polish}, \dbtext{wax}, \dbtext{generic}, \dbtext{ground}, \dbtext{eaten}, \dbtext{neighbor house}, \dbtext{usually}

\paragraph{Community 2 (size is 302).} 
\dbtext{write program}, \dbtext{go concert}, \dbtext{hear music}, \dbtext{exercise}, \dbtext{listen}, \dbtext{go performance}, \dbtext{take walk}, 
\dbtext{entertain}, \dbtext{run marathon}, \dbtext{wait line}, \dbtext{attend lecture}, \dbtext{study}, \dbtext{go walk}, \dbtext{play basketball}, 
\dbtext{fun}, \dbtext{bore}, \dbtext{wait table}, \dbtext{go see film}, \dbtext{go work}, \dbtext{watch tv show}, \dbtext{wake up morning}, 
\dbtext{go fish}, \dbtext{tell story}, \dbtext{surf web}, \dbtext{play football}, \dbtext{go restaurant}, \dbtext{visit museum}, \dbtext{study subject}, 
\dbtext{live life}, \dbtext{go sport event}, \dbtext{go play}, \dbtext{play soccer}, \dbtext{go jog}, \dbtext{take shower}, \dbtext{play ball}, 
\dbtext{watch movie}, \dbtext{watch film}, \dbtext{stretch}, \dbtext{play frisbee}, \dbtext{go school}, \dbtext{surprise}, \dbtext{paint picture}, 
\dbtext{go film}, \dbtext{rest}, \dbtext{listen radio}, \dbtext{kiss}, \dbtext{remember}, \dbtext{housework}, \dbtext{watch tv}, \dbtext{attend school}, 
\dbtext{play tennis}, \dbtext{comfortable}, \dbtext{take bus}, \dbtext{conversation}, \dbtext{talk}, \dbtext{take course}, \dbtext{learn}, 
\dbtext{think}, \dbtext{go run}, \dbtext{sleep}, \dbtext{hang out bar}, \dbtext{plan vacation}, \dbtext{go see play}, \dbtext{attend class}, 
\dbtext{go swim}, \dbtext{ride bike}, \dbtext{eat restaurant}, \dbtext{stress}, \dbtext{boredom}, \dbtext{ticket}, \dbtext{use television}, 
\dbtext{entertainment}, \dbtext{listen music}, \dbtext{enjoyment}, \dbtext{student}, \dbtext{go movie}, \dbtext{enlightenment}, \dbtext{stand line}, 
\dbtext{attend classical concert}, \dbtext{death}, \dbtext{play sport}, \dbtext{effort}, \dbtext{drive car}, \dbtext{traveling}, \dbtext{knowledge}, 
\dbtext{teach}, \dbtext{laugh joke}, \dbtext{run}, \dbtext{read book}, \dbtext{education}, \dbtext{take note}, \dbtext{travel}, \dbtext{go store}, 
\dbtext{go sleep}, \dbtext{tire}, \dbtext{attention}, \dbtext{fall asleep}, \dbtext{run errand}, \dbtext{patience}, \dbtext{spend money}, \dbtext{cry}, 
\dbtext{pay bill}, \dbtext{earn money}, \dbtext{speak}, \dbtext{fatigue}, \dbtext{take break}, \dbtext{drink alcohol}, \dbtext{play chess}, 
\dbtext{anger}, \dbtext{read}, \dbtext{curiosity}, \dbtext{pay}, \dbtext{drive}, \dbtext{use computer}, \dbtext{take film}, \dbtext{smile}, 
\dbtext{fiddle}, \dbtext{wrestle}, \dbtext{see new}, \dbtext{dance}, \dbtext{job}, \dbtext{smart}, \dbtext{play baseball}, \dbtext{excite}, 
\dbtext{attend rock concert}, \dbtext{hear news}, \dbtext{contemplate}, \dbtext{pain}, \dbtext{understand}, \dbtext{stay healthy}, \dbtext{research}, 
\dbtext{learn new}, \dbtext{sweat}, \dbtext{headache}, \dbtext{fart}, \dbtext{read newspaper}, \dbtext{sport}, \dbtext{understand better}, 
\dbtext{write story}, \dbtext{transportation}, \dbtext{fall down}, \dbtext{practice}, \dbtext{help}, \dbtext{lose}, \dbtext{close eye}, 
\dbtext{satisfaction}, \dbtext{time}, \dbtext{answer question}, \dbtext{perform}, \dbtext{everyone}, \dbtext{go somewhere}, \dbtext{play card}, 
\dbtext{go}, \dbtext{sex}, \dbtext{wait}, \dbtext{buy ticket}, \dbtext{gain knowledge}, \dbtext{interest}, \dbtext{sit down}, \dbtext{surf}, 
\dbtext{teacher}, \dbtext{happiness}, \dbtext{exhaustion}, \dbtext{sit chair}, \dbtext{laugh}, \dbtext{relax}, \dbtext{waste time}, 
\dbtext{pleasure}, \dbtext{relaxation}, \dbtext{procreate}, \dbtext{funny}, \dbtext{win}, \dbtext{go mall}, \dbtext{flirt}, \dbtext{pass time}, 
\dbtext{go home}, \dbtext{love else}, \dbtext{drunk}, \dbtext{sing}, \dbtext{buy beer}, \dbtext{internet}, \dbtext{like}, \dbtext{date}, 
\dbtext{play game}, \dbtext{activity}, \dbtext{jog}, \dbtext{quiet}, \dbtext{skill}, \dbtext{hobby}, \dbtext{tiredness}, \dbtext{communication}, 
\dbtext{drink coffee}, \dbtext{read magazine}, \dbtext{good time}, \dbtext{good health}, \dbtext{play hockey}, \dbtext{eat ice cream}, 
\dbtext{learn language}, \dbtext{go zoo}, \dbtext{go internet}, \dbtext{cash}, \dbtext{read child}, \dbtext{enjoy yourself}, \dbtext{see movie}, 
\dbtext{e\-mo\-tion}, \dbtext{clean house}, \dbtext{fit}, \dbtext{view video}, \dbtext{play poker}, \dbtext{excitement}, \dbtext{fly airplane}, 
\dbtext{ride horse}, \dbtext{stay bed}, \dbtext{happy}, \dbtext{find information}, \dbtext{fear}, \dbtext{go vacation}, \dbtext{breathe}, 
\dbtext{recreation}, \dbtext{enjoy}, \dbtext{ride bicycle}, \dbtext{health}, \dbtext{communicate}, \dbtext{make money}, \dbtext{become tire}, 
\dbtext{lose weight}, \dbtext{jump up down}, \dbtext{watch television}, \dbtext{healthy}, \dbtext{know}, \dbtext{learn subject}, \dbtext{joy}, 
\dbtext{stand up}, \dbtext{information}, \dbtext{read letter}, \dbtext{jump rope}, \dbtext{celebrate}, \dbtext{sadness}, 
\dbtext{watch musician perform}, \dbtext{feel better}, \dbtext{compete}, \dbtext{feel good}, \dbtext{accident}, \dbtext{stay fit}, 
\dbtext{injury}, \dbtext{ride}, \dbtext{play piano}, \dbtext{learn world}, \dbtext{see exhibit}, \dbtext{release energy}, \dbtext{see art}, 
\dbtext{see excite story}, \dbtext{orgasm}, \dbtext{laughter}, \dbtext{express yourself}, \dbtext{discover truth}, \dbtext{see favorite show}, 
\dbtext{go party}, \dbtext{competition}, \dbtext{express information}, \dbtext{climb mountain}, \dbtext{attend meet}, \dbtext{fly kite}, 
\dbtext{examine}, \dbtext{meet friend}, \dbtext{read news}, \dbtext{shock}, \dbtext{return work}, \dbtext{see band}, \dbtext{visit art gallery}, 
\dbtext{earn live}, \dbtext{watch television show}, \dbtext{socialize}, \dbtext{create art}, \dbtext{crossword puzzle}, \dbtext{enjoy film}, 
\dbtext{go pub}, \dbtext{feel happy}, \dbtext{play la\-crosse}, \dbtext{socialis}, \dbtext{physical activity}, \dbtext{get}, 
\dbtext{make person laugh}, \dbtext{make friend}, \dbtext{chat friend}, \dbtext{meet person}, \dbtext{meet interest person}, 
\dbtext{get drunk}, \dbtext{friend over}, \dbtext{get exercise}, \dbtext{get tire}, \dbtext{enjoy company friend}, \dbtext{play game friend}, 
\dbtext{get physical activity}, \dbtext{go opus}, \dbtext{get shape}, \dbtext{sit quietly}, \dbtext{do it}, \dbtext{get fit}, 
\dbtext{teach other person}, \dbtext{entertain person}, \dbtext{see person play game}

\paragraph{Community 3 (size is 275).} 
\dbtext{something}, \dbtext{man}, \dbtext{person}, \dbtext{type}, \dbtext{train}, \dbtext{work}, \dbtext{word}, \dbtext{pant}, \dbtext{love}, 
\dbtext{library}, \dbtext{bath}, \dbtext{school}, \dbtext{arm}, \dbtext{human}, \dbtext{class}, \dbtext{walk}, \dbtext{drink}, \dbtext{it}, 
\dbtext{dream}, \dbtext{shower}, \dbtext{child}, \dbtext{smoke}, \dbtext{gym}, \dbtext{movie}, \dbtext{sit}, \dbtext{ball}, \dbtext{eat food}, 
\dbtext{mother}, \dbtext{party}, \dbtext{clean}, \dbtext{lunch}, \dbtext{street}, \dbtext{trouble}, \dbtext{play}, \dbtext{bus}, \dbtext{plan}, 
\dbtext{eat}, \dbtext{bridge}, \dbtext{nothing}, \dbtext{computer}, \dbtext{line}, \dbtext{buy}, \dbtext{tv}, \dbtext{car}, \dbtext{vehicle}, 
\dbtext{dog}, \dbtext{music}, \dbtext{dress}, \dbtext{live}, \dbtext{one}, \dbtext{turn}, \dbtext{fish}, \dbtext{baby}, \dbtext{hurt}, 
\dbtext{game}, \dbtext{hospital}, \dbtext{hide}, \dbtext{girl}, \dbtext{muscle}, \dbtext{woman}, \dbtext{cold}, \dbtext{family}, \dbtext{shop}, 
\dbtext{letter}, \dbtext{eat dinner}, \dbtext{concert}, \dbtext{call}, \dbtext{electricity}, \dbtext{eye}, \dbtext{see}, \dbtext{story}, 
\dbtext{nose}, \dbtext{smell}, \dbtext{stand}, \dbtext{die}, \dbtext{money}, \dbtext{bill}, \dbtext{leg}, \dbtext{everything}, \dbtext{television}, 
\dbtext{take bath}, \dbtext{band}, \dbtext{drink water}, \dbtext{paint}, \dbtext{hike}, \dbtext{lie}, \dbtext{friend}, \dbtext{hot}, \dbtext{kill}, 
\dbtext{swim}, \dbtext{break}, \dbtext{foot}, \dbtext{verb}, \dbtext{camp}, \dbtext{den}, \dbtext{we}, \dbtext{fight}, \dbtext{telephone}, 
\dbtext{audience}, \dbtext{drop}, \dbtext{rain}, \dbtext{body}, \dbtext{use}, \dbtext{write}, \dbtext{doll}, \dbtext{wheel}, \dbtext{name}, 
\dbtext{nice}, \dbtext{book}, \dbtext{museum}, \dbtext{pool}, \dbtext{instrument}, \dbtext{bad}, \dbtext{show}, \dbtext{wind}, \dbtext{hand}, 
\dbtext{pee}, \dbtext{stop}, \dbtext{road}, \dbtext{seat}, \dbtext{boat}, \dbtext{train station}, \dbtext{war}, \dbtext{wet}, \dbtext{cell}, 
\dbtext{bicycle}, \dbtext{need}, \dbtext{life}, \dbtext{burn}, \dbtext{sound}, \dbtext{good}, \dbtext{fire}, \dbtext{news}, \dbtext{finger}, 
\dbtext{feel}, \dbtext{dangerous}, \dbtext{ski}, \dbtext{expensive}, \dbtext{leave}, \dbtext{number}, \dbtext{heavy}, \dbtext{map}, \dbtext{piano}, 
\dbtext{club}, \dbtext{theatre}, \dbtext{god}, \dbtext{care}, \dbtext{airplane}, \dbtext{watch}, \dbtext{phone}, \dbtext{radio}, \dbtext{tool}, 
\dbtext{mouth}, \dbtext{doctor}, \dbtext{theater}, \dbtext{lot}, \dbtext{hair}, \dbtext{make}, \dbtext{noise}, \dbtext{measure}, \dbtext{shape}, 
\dbtext{climb}, \dbtext{wash hand}, \dbtext{bar}, \dbtext{view}, \dbtext{tooth}, \dbtext{peace}, \dbtext{kid}, \dbtext{boy}, \dbtext{record}, 
\dbtext{find}, \dbtext{song}, \dbtext{meet}, \dbtext{not}, \dbtext{sofa}, \dbtext{cut}, \dbtext{page}, \dbtext{company}, \dbtext{bite}, 
\dbtext{science}, \dbtext{college}, \dbtext{open}, \dbtext{warm}, \dbtext{high}, \dbtext{birthday}, \dbtext{university}, \dbtext{roll}, 
\dbtext{mean}, \dbtext{general}, \dbtext{act}, \dbtext{heat}, \dbtext{cool}, \dbtext{dive}, \dbtext{skin}, \dbtext{art}, \dbtext{noun}, 
\dbtext{hard}, \dbtext{important}, \dbtext{toy}, \dbtext{ring}, \dbtext{crowd}, \dbtext{draw}, \dbtext{thing}, \dbtext{energy}, 
\dbtext{kill person}, \dbtext{little}, \dbtext{change}, \dbtext{ear}, \dbtext{alive}, \dbtext{move}, \dbtext{wave}, \dbtext{look}, 
\dbtext{voice}, \dbtext{face}, \dbtext{event}, \dbtext{long}, \dbtext{carry}, \dbtext{fly}, \dbtext{test}, \dbtext{hear}, \dbtext{organization}, 
\dbtext{jump}, \dbtext{business}, \dbtext{action}, \dbtext{pass}, \dbtext{fall}, \dbtext{bell}, \dbtext{head}, \dbtext{sign}, 
\dbtext{count}, \dbtext{end}, \dbtext{group}, \dbtext{bullet}, \dbtext{degree}, \dbtext{note}, \dbtext{card}, \dbtext{machine}, 
\dbtext{lay}, \dbtext{gas}, \dbtext{circle}, \dbtext{point}, \dbtext{useful}, \dbtext{adjective}, \dbtext{be}, \dbtext{steam}, 
\dbtext{bike}, \dbtext{side}, \dbtext{motion}, \dbtext{classroom}, \dbtext{out}, \dbtext{transport}, \dbtext{step}, \dbtext{part}, 
\dbtext{course}, \dbtext{power}, \dbtext{same}, \dbtext{stage}, \dbtext{comfort}, \dbtext{trip}, \dbtext{edge}, \dbtext{grow}, 
\dbtext{board}, \dbtext{race}, \dbtext{flea}, \dbtext{punch}, \dbtext{cool off}, \dbtext{skate}, \dbtext{movement}, \dbtext{away}, 
\dbtext{short}, \dbtext{many person}, \dbtext{singular}, \dbtext{general term}, \dbtext{unit}, \dbtext{generic term}

\subsection{Betweenness}
The algorithm used is the one implemented in \igraph 
which is based on \citep{communities:betweenness}.
The idea of the algorithm is described below; it is taken from \igraph documentation online.
\begin{quote}
The idea is that the betweenness of the edges connecting two communities is typically high, 
as many of the shortest paths between nodes in separate communities go through them. 
So we gradually remove the edge with highest betweenness from the network, 
and recalculate edge betweenness after every removal. 
This way sooner or later the network falls off to two components, 
then after a while one of these components falls off to two smaller components, 
etc.~until all edges are removed. 
This is a divisive hierarchical approach, the result is a dendrogram. 
\end{quote}
The algorithm has complexity 
\OO{\abs{V}\abs{E}^2}, as the betweenness calculation 
requires \OO{\abs{V}\abs{E}} time and we do it $\abs{E}-1$ times.
Hence, we applied the algorithm only on the subgraph induced by the vertices with maximum
coreness \footnote{ Recall that the subgraph has no self-loops.} .

\medskip

One execution of the algorithm in the subgraph took about $8,671.19$ seconds of computation time. 
The algorithm found $42$ communities and the modularity achieved was $0.268508$.

\subsubsection{Edge Betweenness: Communities in the Inner-Most Core}
We have the following communities.
\paragraph{Community 1 (size is 423).} 
\dbtext{something}, \dbtext{person}, \dbtext{train}, \dbtext{town}, \dbtext{rock}, \dbtext{beach}, \dbtext{tree}, \dbtext{soup}, \dbtext{weasel}, 
\dbtext{word}, \dbtext{library}, \dbtext{school}, \dbtext{kitten}, \dbtext{arm}, \dbtext{human}, \dbtext{plane}, \dbtext{class}, \dbtext{it}, 
\dbtext{paper}, \dbtext{bed}, \dbtext{dirty}, \dbtext{chicken}, \dbtext{gym}, \dbtext{office build}, \dbtext{ball}, \dbtext{box}, \dbtext{object}, 
\dbtext{mother}, \dbtext{coffee}, \dbtext{candle}, \dbtext{street}, \dbtext{fungus}, \dbtext{park}, \dbtext{snake}, \dbtext{wood}, \dbtext{bus}, 
\dbtext{bridge}, \dbtext{cloud}, \dbtext{computer}, \dbtext{line}, \dbtext{milk}, \dbtext{drawer}, \dbtext{storage}, \dbtext{car}, \dbtext{vehicle}, 
\dbtext{dog}, \dbtext{zoo}, \dbtext{bottle}, \dbtext{live}, \dbtext{one}, \dbtext{material}, \dbtext{chair}, \dbtext{cat}, \dbtext{hat}, 
\dbtext{country}, \dbtext{market}, \dbtext{house}, \dbtext{lake}, \dbtext{hotel}, \dbtext{plant}, \dbtext{hospital}, \dbtext{bank}, \dbtext{girl}, 
\dbtext{woman}, \dbtext{animal}, \dbtext{church}, \dbtext{cold}, \dbtext{family}, \dbtext{moon}, \dbtext{pet}, \dbtext{cook}, \dbtext{shop}, 
\dbtext{letter}, \dbtext{bird}, \dbtext{concert}, \dbtext{bathroom}, \dbtext{city}, \dbtext{water}, \dbtext{yard}, \dbtext{desk}, \dbtext{office}, 
\dbtext{home}, \dbtext{bat}, \dbtext{couch}, \dbtext{kitchen}, \dbtext{lizard}, \dbtext{build}, \dbtext{restaurant}, \dbtext{spoon}, \dbtext{butter}, 
\dbtext{beautiful}, \dbtext{key}, \dbtext{eye}, \dbtext{nose}, \dbtext{stand}, \dbtext{well}, \dbtext{pen}, \dbtext{bill}, \dbtext{snow}, 
\dbtext{weather}, \dbtext{everything}, \dbtext{mouse}, \dbtext{magazine}, \dbtext{hole}, \dbtext{nature}, \dbtext{band}, \dbtext{bald eagle}, 
\dbtext{nest}, \dbtext{crab}, \dbtext{paint}, \dbtext{ficus}, \dbtext{sea}, \dbtext{ocean}, \dbtext{sun}, \dbtext{sky}, \dbtext{food}, 
\dbtext{grape}, \dbtext{bedroom}, \dbtext{horse}, \dbtext{store}, \dbtext{hot}, \dbtext{airport}, \dbtext{sugar}, \dbtext{grocery store}, 
\dbtext{basket}, \dbtext{hold}, \dbtext{kill}, \dbtext{foot}, \dbtext{refrigerator}, \dbtext{newspaper}, \dbtext{rice}, \dbtext{surface}, 
\dbtext{liquid}, \dbtext{meadow}, \dbtext{camp}, \dbtext{window}, \dbtext{oil}, \dbtext{cover}, \dbtext{plate}, \dbtext{dinner}, \dbtext{den}, 
\dbtext{cow}, \dbtext{earth}, \dbtext{garage}, \dbtext{garden}, \dbtext{poop}, \dbtext{potato}, \dbtext{outside}, \dbtext{frog}, \dbtext{napkin}, 
\dbtext{light}, \dbtext{salad}, \dbtext{fox}, \dbtext{forest}, \dbtext{glass}, \dbtext{cupboard}, \dbtext{telephone}, \dbtext{marmot}, 
\dbtext{mountain}, \dbtext{salt}, \dbtext{motel}, \dbtext{drop}, \dbtext{bone}, \dbtext{meat}, \dbtext{bookstore}, \dbtext{rain}, \dbtext{body}, 
\dbtext{use}, \dbtext{ferret}, \dbtext{small dog}, \dbtext{cloth}, \dbtext{factory}, \dbtext{bottle wine}, \dbtext{doll}, \dbtext{pencil}, 
\dbtext{wheel}, \dbtext{name}, \dbtext{nice}, \dbtext{book}, \dbtext{black}, \dbtext{instrument}, \dbtext{show}, \dbtext{trash}, \dbtext{can}, 
\dbtext{a}, \dbtext{wind}, \dbtext{hand}, \dbtext{pee}, \dbtext{picture}, \dbtext{road}, \dbtext{seat}, \dbtext{boat}, \dbtext{clothe}, 
\dbtext{dish}, \dbtext{train station}, \dbtext{war}, \dbtext{mall}, \dbtext{wet}, \dbtext{flower}, \dbtext{wallet}, \dbtext{room}, \dbtext{cell}, 
\dbtext{small}, \dbtext{bicycle}, \dbtext{new york}, \dbtext{farm}, \dbtext{sink}, \dbtext{pocket}, \dbtext{color}, \dbtext{white}, \dbtext{red}, 
\dbtext{stone}, \dbtext{vegetable}, \dbtext{green}, \dbtext{burn}, \dbtext{large}, \dbtext{shoe}, \dbtext{scale}, \dbtext{soft}, \dbtext{steak}, 
\dbtext{fire}, \dbtext{beer}, \dbtext{finger}, \dbtext{knife}, \dbtext{dangerous}, \dbtext{carpet}, \dbtext{bowl}, \dbtext{corn}, \dbtext{fridge}, 
\dbtext{soap}, \dbtext{expensive}, \dbtext{coin}, \dbtext{number}, \dbtext{fruit}, \dbtext{heavy}, \dbtext{map}, \dbtext{fork}, \dbtext{steel}, 
\dbtext{piano}, \dbtext{wall}, \dbtext{theatre}, \dbtext{cup}, \dbtext{hill}, \dbtext{square}, \dbtext{shelf}, \dbtext{god}, \dbtext{friend house}, 
\dbtext{airplane}, \dbtext{space}, \dbtext{phone}, \dbtext{this}, \dbtext{place}, \dbtext{radio}, \dbtext{tool}, \dbtext{apple}, \dbtext{mouth}, 
\dbtext{bag}, \dbtext{doctor}, \dbtext{theater}, \dbtext{river}, \dbtext{blue}, \dbtext{grass}, \dbtext{cheese}, \dbtext{mammal}, \dbtext{bean}, 
\dbtext{lot}, \dbtext{hair}, \dbtext{make}, \dbtext{measure}, \dbtext{flat}, \dbtext{utah}, \dbtext{plastic}, \dbtext{container}, \dbtext{bar}, 
\dbtext{bug}, \dbtext{view}, \dbtext{live room}, \dbtext{toilet}, \dbtext{tooth}, \dbtext{cabinet}, \dbtext{table}, \dbtext{furniture}, \dbtext{lamp}, 
\dbtext{pizza}, \dbtext{dust}, \dbtext{sand}, \dbtext{hall}, \dbtext{rise}, \dbtext{closet}, \dbtext{boy}, \dbtext{door}, \dbtext{floor}, \dbtext{not}, 
\dbtext{basement}, \dbtext{sofa}, \dbtext{cut}, \dbtext{page}, \dbtext{company}, \dbtext{bite}, \dbtext{dark}, \dbtext{college}, \dbtext{world}, 
\dbtext{air}, \dbtext{sheep}, \dbtext{statue}, \dbtext{metal}, \dbtext{open}, \dbtext{warm}, \dbtext{big}, \dbtext{high}, \dbtext{squirrel}, 
\dbtext{alcohol}, \dbtext{university}, \dbtext{roll}, \dbtext{general}, \dbtext{clock}, \dbtext{round}, \dbtext{heat}, \dbtext{cool}, \dbtext{skin}, 
\dbtext{art}, \dbtext{noun}, \dbtext{top}, \dbtext{wine}, \dbtext{jar}, \dbtext{hard}, \dbtext{put}, \dbtext{duck}, \dbtext{toy}, \dbtext{ring}, 
\dbtext{crowd}, \dbtext{draw}, \dbtext{edible}, \dbtext{thing}, \dbtext{land}, \dbtext{rug}, \dbtext{pot}, \dbtext{little}, \dbtext{change}, 
\dbtext{ear}, \dbtext{alive}, \dbtext{bread}, \dbtext{field}, \dbtext{wave}, \dbtext{face}, \dbtext{tin}, \dbtext{oven}, \dbtext{long}, \dbtext{shade}, 
\dbtext{carry}, \dbtext{fly}, \dbtext{egg}, \dbtext{building}, \dbtext{bee}, \dbtext{business}, \dbtext{pass}, \dbtext{resturant}, \dbtext{wash}, 
\dbtext{sock}, \dbtext{bear}, \dbtext{bell}, \dbtext{head}, \dbtext{sign}, \dbtext{pantry}, \dbtext{bullet}, \dbtext{degree}, \dbtext{note}, 
\dbtext{card}, \dbtext{supermarket}, \dbtext{machine}, \dbtext{gas}, \dbtext{roof}, \dbtext{brown}, \dbtext{circle}, \dbtext{cake}, \dbtext{solid}, 
\dbtext{dirt}, \dbtext{point}, \dbtext{useful}, \dbtext{handle}, \dbtext{adjective}, \dbtext{maine}, \dbtext{department}, \dbtext{be}, \dbtext{steam}, 
\dbtext{pretty}, \dbtext{side}, \dbtext{decoration}, \dbtext{stapler}, \dbtext{classroom}, \dbtext{transport}, \dbtext{step}, \dbtext{apartment}, 
\dbtext{part}, \dbtext{course}, \dbtext{countryside}, \dbtext{power}, \dbtext{same}, \dbtext{stage}, \dbtext{any large city}, \dbtext{edge}, 
\dbtext{case}, \dbtext{grow}, \dbtext{board}, \dbtext{flea}, \dbtext{corner}, \dbtext{short}, \dbtext{stick}, \dbtext{singular}, \dbtext{find house}, 
\dbtext{find outside}, \dbtext{winery}, \dbtext{branch}, \dbtext{polish}, \dbtext{wax}, \dbtext{general term}, \dbtext{generic}, \dbtext{ground}, 
\dbtext{eaten}, \dbtext{neighbor house}, \dbtext{usually}, \dbtext{unit}, \dbtext{generic term}

\paragraph{Community 2 (size is 406).} 
\dbtext{man}, \dbtext{type}, \dbtext{work}, \dbtext{write program}, \dbtext{go concert}, \dbtext{hear music}, \dbtext{exercise}, \dbtext{love}, 
\dbtext{bath}, \dbtext{listen}, \dbtext{go performance}, \dbtext{take walk}, \dbtext{walk}, \dbtext{entertain}, \dbtext{run marathon}, 
\dbtext{wait line}, \dbtext{attend lecture}, \dbtext{drink}, \dbtext{study}, \dbtext{go walk}, \dbtext{play basketball}, \dbtext{fun}, 
\dbtext{bore}, \dbtext{wait table}, \dbtext{go see film}, \dbtext{go work}, \dbtext{watch tv show}, \dbtext{wake up morning}, \dbtext{dream}, 
\dbtext{shower}, \dbtext{child}, \dbtext{smoke}, \dbtext{go fish}, \dbtext{tell story}, \dbtext{surf web}, \dbtext{play football}, \dbtext{movie},
 \dbtext{go restaurant}, \dbtext{visit museum}, \dbtext{study subject}, \dbtext{live life}, \dbtext{go sport event}, \dbtext{go play}, \dbtext{sit}, 
\dbtext{play soccer}, \dbtext{go jog}, \dbtext{take shower}, \dbtext{play ball}, \dbtext{watch movie}, \dbtext{watch film}, \dbtext{stretch}, 
\dbtext{play frisbee}, \dbtext{go school}, \dbtext{surprise}, \dbtext{paint picture}, \dbtext{go film}, \dbtext{party}, \dbtext{rest}, 
\dbtext{listen radio}, \dbtext{kiss}, \dbtext{remember}, \dbtext{housework}, \dbtext{clean}, \dbtext{lunch}, \dbtext{watch tv}, 
\dbtext{attend school}, \dbtext{play tennis}, \dbtext{trouble}, \dbtext{comfortable}, \dbtext{play}, \dbtext{take bus}, \dbtext{conversation}, 
\dbtext{talk}, \dbtext{take course}, \dbtext{learn}, \dbtext{plan}, \dbtext{think}, \dbtext{go run}, \dbtext{sleep}, \dbtext{hang out bar}, 
\dbtext{plan vacation}, \dbtext{go see play}, \dbtext{eat}, \dbtext{attend class}, \dbtext{go swim}, \dbtext{ride bike}, \dbtext{nothing}, 
\dbtext{buy}, \dbtext{eat restaurant}, \dbtext{tv}, \dbtext{stress}, \dbtext{boredom}, \dbtext{ticket}, \dbtext{music}, \dbtext{use television}, 
\dbtext{dress}, \dbtext{turn}, \dbtext{entertainment}, \dbtext{listen music}, \dbtext{enjoyment}, \dbtext{fish}, \dbtext{baby}, \dbtext{hurt}, 
\dbtext{game}, \dbtext{student}, \dbtext{muscle}, \dbtext{go movie}, \dbtext{enlightenment}, \dbtext{stand line}, \dbtext{attend classical concert}, 
\dbtext{death}, \dbtext{play sport}, \dbtext{effort}, \dbtext{drive car}, \dbtext{traveling}, \dbtext{knowledge}, \dbtext{teach}, \dbtext{call}, 
\dbtext{laugh joke}, \dbtext{run}, \dbtext{read book}, \dbtext{education}, \dbtext{take note}, \dbtext{travel}, \dbtext{electricity}, 
\dbtext{go store}, \dbtext{see}, \dbtext{story}, \dbtext{smell}, \dbtext{go sleep}, \dbtext{tire}, \dbtext{attention}, \dbtext{die}, 
\dbtext{fall asleep}, \dbtext{money}, \dbtext{leg}, \dbtext{run errand}, \dbtext{patience}, \dbtext{spend money}, \dbtext{cry}, 
\dbtext{pay bill}, \dbtext{earn money}, \dbtext{television}, \dbtext{speak}, \dbtext{drink water}, \dbtext{fatigue}, \dbtext{take break}, 
\dbtext{hike}, \dbtext{drink alcohol}, \dbtext{lie}, \dbtext{play chess}, \dbtext{friend}, \dbtext{anger}, \dbtext{read}, \dbtext{curiosity}, 
\dbtext{pay}, \dbtext{swim}, \dbtext{break}, \dbtext{verb}, \dbtext{drive}, \dbtext{use computer}, \dbtext{take film}, \dbtext{smile}, 
\dbtext{fiddle}, \dbtext{we}, \dbtext{wrestle}, \dbtext{see new}, \dbtext{dance}, \dbtext{fight}, \dbtext{job}, \dbtext{smart}, 
\dbtext{play baseball}, \dbtext{excite}, \dbtext{attend rock concert}, \dbtext{hear news}, \dbtext{contemplate}, \dbtext{pain}, 
\dbtext{audience}, \dbtext{understand}, \dbtext{write}, \dbtext{stay healthy}, \dbtext{research}, \dbtext{learn new}, \dbtext{sweat}, 
\dbtext{headache}, \dbtext{fart}, \dbtext{read newspaper}, \dbtext{sport}, \dbtext{understand better}, \dbtext{bad}, \dbtext{write story}, 
\dbtext{stop}, \dbtext{transportation}, \dbtext{fall down}, \dbtext{practice}, \dbtext{help}, \dbtext{lose}, \dbtext{close eye}, 
\dbtext{satisfaction}, \dbtext{time}, \dbtext{answer question}, \dbtext{perform}, \dbtext{need}, \dbtext{everyone}, \dbtext{go somewhere}, 
\dbtext{life}, \dbtext{sound}, \dbtext{good}, \dbtext{play card}, \dbtext{go}, \dbtext{sex}, \dbtext{wait}, \dbtext{buy ticket}, 
\dbtext{gain knowledge}, \dbtext{news}, \dbtext{interest}, \dbtext{feel}, \dbtext{sit down}, \dbtext{ski}, \dbtext{surf}, \dbtext{teacher}, 
\dbtext{leave}, \dbtext{happiness}, \dbtext{exhaustion}, \dbtext{sit chair}, \dbtext{laugh}, \dbtext{relax}, \dbtext{waste time}, 
\dbtext{pleasure}, \dbtext{relaxation}, \dbtext{care}, \dbtext{procreate}, \dbtext{watch}, \dbtext{funny}, \dbtext{win}, \dbtext{go mall}, 
\dbtext{flirt}, \dbtext{pass time}, \dbtext{noise}, \dbtext{shape}, \dbtext{climb}, \dbtext{go home}, \dbtext{love else}, \dbtext{drunk}, 
\dbtext{peace}, \dbtext{sing}, \dbtext{buy beer}, \dbtext{internet}, \dbtext{kid}, \dbtext{like}, \dbtext{date}, \dbtext{record}, \dbtext{find}, 
\dbtext{song}, \dbtext{play game}, \dbtext{meet}, \dbtext{activity}, \dbtext{science}, \dbtext{jog}, \dbtext{quiet}, \dbtext{skill}, 
\dbtext{hobby}, \dbtext{birthday}, \dbtext{tiredness}, \dbtext{mean}, \dbtext{communication}, \dbtext{drink coffee}, \dbtext{read magazine}, 
\dbtext{good time}, \dbtext{good health}, \dbtext{act}, \dbtext{play hockey}, \dbtext{eat ice cream}, \dbtext{learn language}, \dbtext{dive}, 
\dbtext{go zoo}, \dbtext{go internet}, \dbtext{cash}, \dbtext{important}, \dbtext{read child}, \dbtext{enjoy yourself}, \dbtext{see movie}, 
\dbtext{energy}, \dbtext{kill person}, \dbtext{emotion}, \dbtext{clean house}, \dbtext{fit}, \dbtext{view video}, \dbtext{play poker}, 
\dbtext{excitement}, \dbtext{move}, \dbtext{fly airplane}, \dbtext{ride horse}, \dbtext{stay bed}, \dbtext{look}, \dbtext{voice}, 
\dbtext{event}, \dbtext{happy}, \dbtext{find information}, \dbtext{fear}, \dbtext{go vacation}, \dbtext{breathe}, \dbtext{recreation}, 
\dbtext{test}, \dbtext{enjoy}, \dbtext{hear}, \dbtext{jump}, \dbtext{ride bicycle}, \dbtext{health}, \dbtext{communicate}, \dbtext{make money}, 
\dbtext{become tire}, \dbtext{action}, \dbtext{fall}, \dbtext{lose weight}, \dbtext{jump up down}, \dbtext{watch television}, \dbtext{count}, 
\dbtext{healthy}, \dbtext{end}, \dbtext{know}, \dbtext{learn subject}, \dbtext{joy}, \dbtext{stand up}, \dbtext{information}, \dbtext{read letter}, 
\dbtext{lay}, \dbtext{jump rope}, \dbtext{celebrate}, \dbtext{sadness}, \dbtext{bike}, \dbtext{watch musician perform}, \dbtext{motion}, 
\dbtext{feel better}, \dbtext{compete}, \dbtext{out}, \dbtext{feel good}, \dbtext{accident}, \dbtext{stay fit}, \dbtext{injury}, \dbtext{ride}, 
\dbtext{play piano}, \dbtext{learn world}, \dbtext{see exhibit}, \dbtext{release energy}, \dbtext{see art}, \dbtext{see excite story}, 
\dbtext{comfort}, \dbtext{orgasm}, \dbtext{trip}, \dbtext{laughter}, \dbtext{express yourself}, \dbtext{discover truth}, \dbtext{see favorite show}, 
\dbtext{go party}, \dbtext{competition}, \dbtext{express information}, \dbtext{climb mountain}, \dbtext{attend meet}, \dbtext{fly kite}, 
\dbtext{examine}, \dbtext{race}, \dbtext{meet friend}, \dbtext{read news}, \dbtext{shock}, \dbtext{return work}, \dbtext{see band}, 
\dbtext{visit art gallery}, \dbtext{earn live}, \dbtext{punch}, \dbtext{cool off}, \dbtext{watch television show}, \dbtext{socialize}, 
\dbtext{skate}, \dbtext{movement}, \dbtext{create art}, \dbtext{crossword puzzle}, \dbtext{enjoy film}, \dbtext{go pub}, \dbtext{feel happy}, 
\dbtext{play lacrosse}, \dbtext{socialis}, \dbtext{away}, \dbtext{physical activity}, \dbtext{get}, \dbtext{many person}, \dbtext{make person laugh}, 
\dbtext{make friend}, \dbtext{chat friend}, \dbtext{meet person}, \dbtext{meet interest person}, \dbtext{get drunk}, \dbtext{friend over}, 
\dbtext{get exercise}, \dbtext{get tire}, \dbtext{enjoy company friend}, \dbtext{play game friend}, \dbtext{get physical activity}, 
\dbtext{go opus}, \dbtext{get shape}, \dbtext{sit quietly}, \dbtext{do it}, \dbtext{get fit}, \dbtext{teach other person}, 
\dbtext{entertain person}, \dbtext{see person play game}

\setlength{\columnseprule}{.2pt}
\begin{multicols}{2}
{ 
\paragraph{Community 3 (size is 1).} \dbtext{monkey}
\paragraph{Community 4 (size is 1).} \dbtext{pant}
\paragraph{Community 5 (size is 1).} \dbtext{beaver}
\paragraph{Community 6 (size is 1).} \dbtext{state}
\paragraph{Community 7 (size is 1).} \dbtext{wiener dog}
\paragraph{Community 8 (size is 1).} \dbtext{eat food}
\paragraph{Community 9 (size is 1).} \dbtext{hide}
\paragraph{Community 10 (size is 1).} \dbtext{eat dinner}
\paragraph{Community 11 (size is 1).} \dbtext{shark}
\paragraph{Community 12 (size is 1).} \dbtext{rosebush}
\paragraph{Community 13 (size is 1).} \dbtext{sloth}
\paragraph{Community 14 (size is 1).} \dbtext{take bath}
\paragraph{Community 15 (size is 1).} \dbtext{anemone}
\paragraph{Community 16 (size is 1).} \dbtext{lemur}
\paragraph{Community 17 (size is 1).} \dbtext{museum}
\paragraph{Community 18 (size is 1).} \dbtext{pool}
\paragraph{Community 19 (size is 1).} \dbtext{canada}
\paragraph{Community 20 (size is 1).} \dbtext{wild}
\paragraph{Community 21 (size is 1).} \dbtext{marmoset}
\paragraph{Community 22 (size is 1).} \dbtext{australia}
\paragraph{Community 23 (size is 1).} \dbtext{cuba}
\paragraph{Community 24 (size is 1).} \dbtext{france}
\paragraph{Community 25 (size is 1).} \dbtext{italy}
\paragraph{Community 26 (size is 1).} \dbtext{club}
\paragraph{Community 27 (size is 1).} \dbtext{unite state}
\paragraph{Community 28 (size is 1).} \dbtext{apple tree}
\paragraph{Community 29 (size is 1).} \dbtext{wash hand}
\paragraph{Community 30 (size is 1).} \dbtext{dictionary}
\paragraph{Community 31 (size is 1).} \dbtext{wyom}
\paragraph{Community 32 (size is 1).} \dbtext{lawn}
\paragraph{Community 33 (size is 1).} \dbtext{organization}
\paragraph{Community 34 (size is 1).} \dbtext{group}
\paragraph{Community 35 (size is 1).} \dbtext{alaska}
\paragraph{Community 36 (size is 1).} \dbtext{michigan}
\paragraph{Community 37 (size is 1).} \dbtext{maryland}
\paragraph{Community 38 (size is 1).} \dbtext{delaware}
\paragraph{Community 39 (size is 1).} \dbtext{kansa}
\paragraph{Community 40 (size is 1).} \dbtext{bush}
\paragraph{Community 41 (size is 1).} \dbtext{sunshine}
\paragraph{Community 42 (size is 1).} \dbtext{outdoors}
}
\end{multicols}

\subsection{Fast Greedy}
The algorithm used is the one implemented in \igraph 
which is based on \citep{communities:fast-greedy}.
According to igraph version $0.6.1$ which was used at the time of the writing,
some improvements mentioned in \citep{communities:fast-greedy:improvements}
have also been implemented. 
Table \ref{tbl:communities:positive:fast-greedy:cores} presents the results when we apply the algorithm for 
subgraphs in which we include vertices with successively lower coreness
and allow positive polarity on the edges only.

\begin{table}[ht]
\caption{Applying the Fast Greedy algorithm for community finding implemented in \igraph by successively including
vertices with lower coreness on the undirected graph induced by the assertions with positive polarity (self-loops are removed).
In every row we have the number of vertices and the number of edges of each such subgraph together with
the number of components ($\abs{C}$) that we find in that subgraph. 
The next three columns present the number of communities found by the algorithm;
the average among all runs, the minimum, and the maximum.
The next three columns present the modularity achieved by the algorithm due to the cut induced by the communities;
the average among all runs, the minimum, and the maximum.
The entire computation lasted about $1,251.1$ seconds for $2$ runs; 
that is about $625.55$ seconds per run.}\label{tbl:communities:positive:fast-greedy:cores}
\begin{center}
\begin{tabular}{|c|r|r|c||r|r|r||r|r|r||}\hline
\multirow{2}{*}{coreness} & \multicolumn{1}{c|}{\multirow{2}{*}{$\abs{V}$}} & \multicolumn{1}{c|}{\multirow{2}{*}{$\abs{E}$}} & \multirow{2}{*}{$\abs{C}$} & \multicolumn{3}{c||}{communities found} & \multicolumn{3}{c||}{modularity} \\\cline{5-10}
          &        &        &   & \multicolumn{1}{c|}{avg} & \multicolumn{1}{c|}{min} & \multicolumn{1}{c||}{max} & \multicolumn{1}{c|}{avg} & \multicolumn{1}{c|}{min} & \multicolumn{1}{c||}{max} \\\hline\hline
$\geq 26$ &    869 &  20526 &    1 &   4.000 &    4 &    4 & 0.286729 & 0.286729 & 0.286729 \\\hline
$\geq 25$ &   1167 &  27810 &    1 &   3.000 &    3 &    3 & 0.294925 & 0.294925 & 0.294925 \\\hline
$\geq 24$ &   1358 &  32314 &    1 &   4.000 &    4 &    4 & 0.285080 & 0.285080 & 0.285080 \\\hline
$\geq 23$ &   1514 &  35870 &    1 &   4.000 &    4 &    4 & 0.283817 & 0.283817 & 0.283817 \\\hline
$\geq 22$ &   1709 &  40099 &    1 &   5.000 &    5 &    5 & 0.283268 & 0.283268 & 0.283268 \\\hline
$\geq 21$ &   1865 &  43330 &    1 &   4.000 &    4 &    4 & 0.292439 & 0.292439 & 0.292439 \\\hline
$\geq 20$ &   2007 &  46145 &    1 &   3.000 &    3 &    3 & 0.291441 & 0.291441 & 0.291441 \\\hline
$\geq 19$ &   2173 &  49265 &    1 &   4.000 &    4 &    4 & 0.294087 & 0.294087 & 0.294087 \\\hline
$\geq 18$ &   2384 &  53011 &    1 &   6.000 &    6 &    6 & 0.286722 & 0.286722 & 0.286722 \\\hline
$\geq 17$ &   2617 &  56939 &    1 &   6.000 &    6 &    6 & 0.285408 & 0.285408 & 0.285408 \\\hline
$\geq 16$ &   2847 &  60583 &    1 &   6.000 &    6 &    6 & 0.294509 & 0.294509 & 0.294509 \\\hline
$\geq 15$ &   3105 &  64412 &    1 &   6.000 &    6 &    6 & 0.291864 & 0.291864 & 0.291864 \\\hline
$\geq 14$ &   3407 &  68613 &    1 &   5.000 &    5 &    5 & 0.303075 & 0.303075 & 0.303075 \\\hline
$\geq 13$ &   3746 &  72978 &    1 &   6.000 &    6 &    6 & 0.299836 & 0.299836 & 0.299836 \\\hline
$\geq 12$ &   4160 &  77882 &    1 &  10.000 &   10 &   10 & 0.292280 & 0.292280 & 0.292280 \\\hline
$\geq 11$ &   4634 &  83039 &    1 &  12.000 &   12 &   12 & 0.298413 & 0.298413 & 0.298413 \\\hline
$\geq 10$ &   5182 &  88462 &    1 &  10.000 &   10 &   10 & 0.306681 & 0.306681 & 0.306681 \\\hline
$\geq  9$ &   5883 &  94709 &    1 &  14.000 &   14 &   14 & 0.318836 & 0.318836 & 0.318836 \\\hline
$\geq  8$ &   6750 & 101564 &    1 &  13.000 &   13 &   13 & 0.318698 & 0.318698 & 0.318698 \\\hline
$\geq  7$ &   7904 & 109561 &    1 &  18.000 &   18 &   18 & 0.326037 & 0.326037 & 0.326037 \\\hline
$\geq  6$ &   9392 & 118389 &    1 &  17.000 &   17 &   17 & 0.333980 & 0.333980 & 0.333980 \\\hline
$\geq  5$ &  11483 & 128731 &    1 &  24.000 &   24 &   24 & 0.332774 & 0.332774 & 0.332774 \\\hline
$\geq  4$ &  14864 & 142112 &    1 &  37.000 &   37 &   37 & 0.339470 & 0.339470 & 0.339470 \\\hline
$\geq  3$ &  21812 & 162691 &    1 &  61.000 &   61 &   61 & 0.372741 & 0.372741 & 0.372741 \\\hline
$\geq  2$ &  41659 & 201678 &    4 & 252.000 &  252 &  252 & 0.409148 & 0.409148 & 0.409148 \\\hline
\end{tabular}
\end{center}
\end{table}

\subsubsection{Fast Greedy: Communities in the Inner-Most Core}
We have the following communities.
\paragraph{Community 1 (size is 20).} 
\dbtext{class}, \dbtext{gym}, \dbtext{turn}, \dbtext{woman}, \dbtext{smell}, \dbtext{take bath}, \dbtext{pool}, \dbtext{pee}, \dbtext{wet}, 
\dbtext{radio}, \dbtext{shape}, \dbtext{view}, \dbtext{science}, \dbtext{art}, \dbtext{change}, \dbtext{wave}, \dbtext{test}, \dbtext{pass}, 
\dbtext{wash}, \dbtext{course}

\paragraph{Community 2 (size is 440).} 
\dbtext{something}, \dbtext{town}, \dbtext{rock}, \dbtext{beach}, \dbtext{tree}, \dbtext{monkey}, \dbtext{soup}, \dbtext{weasel}, \dbtext{pant}, 
\dbtext{library}, \dbtext{school}, \dbtext{kitten}, \dbtext{arm}, \dbtext{human}, \dbtext{plane}, \dbtext{beaver}, \dbtext{it}, \dbtext{paper}, 
\dbtext{bed}, \dbtext{dirty}, \dbtext{chicken}, \dbtext{state}, \dbtext{office build}, \dbtext{wiener dog}, \dbtext{ball}, \dbtext{box}, 
\dbtext{object}, \dbtext{mother}, \dbtext{coffee}, \dbtext{candle}, \dbtext{street}, \dbtext{fungus}, \dbtext{park}, \dbtext{snake}, \dbtext{wood}, 
\dbtext{bus}, \dbtext{eat}, \dbtext{bridge}, \dbtext{cloud}, \dbtext{computer}, \dbtext{line}, \dbtext{milk}, \dbtext{drawer}, \dbtext{storage}, 
\dbtext{car}, \dbtext{vehicle}, \dbtext{dog}, \dbtext{zoo}, \dbtext{bottle}, \dbtext{live}, \dbtext{one}, \dbtext{material}, \dbtext{chair}, 
\dbtext{cat}, \dbtext{hat}, \dbtext{country}, \dbtext{market}, \dbtext{house}, \dbtext{fish}, \dbtext{lake}, \dbtext{baby}, \dbtext{hotel}, 
\dbtext{plant}, \dbtext{game}, \dbtext{hospital}, \dbtext{bank}, \dbtext{hide}, \dbtext{girl}, \dbtext{animal}, \dbtext{church}, \dbtext{cold}, 
\dbtext{family}, \dbtext{moon}, \dbtext{pet}, \dbtext{cook}, \dbtext{shop}, \dbtext{letter}, \dbtext{bird}, \dbtext{bathroom}, \dbtext{city}, 
\dbtext{shark}, \dbtext{water}, \dbtext{rosebush}, \dbtext{yard}, \dbtext{desk}, \dbtext{office}, \dbtext{home}, \dbtext{sloth}, \dbtext{bat}, 
\dbtext{couch}, \dbtext{kitchen}, \dbtext{lizard}, \dbtext{build}, \dbtext{restaurant}, \dbtext{spoon}, \dbtext{butter}, \dbtext{beautiful}, 
\dbtext{key}, \dbtext{electricity}, \dbtext{eye}, \dbtext{nose}, \dbtext{stand}, \dbtext{well}, \dbtext{pen}, \dbtext{bill}, \dbtext{snow}, 
\dbtext{weather}, \dbtext{leg}, \dbtext{everything}, \dbtext{mouse}, \dbtext{magazine}, \dbtext{hole}, \dbtext{nature}, \dbtext{bald eagle}, 
\dbtext{nest}, \dbtext{crab}, \dbtext{paint}, \dbtext{ficus}, \dbtext{sea}, \dbtext{anemone}, \dbtext{ocean}, \dbtext{sun}, \dbtext{sky}, 
\dbtext{food}, \dbtext{grape}, \dbtext{bedroom}, \dbtext{horse}, \dbtext{store}, \dbtext{hot}, \dbtext{airport}, \dbtext{sugar}, 
\dbtext{grocery store}, \dbtext{basket}, \dbtext{hold}, \dbtext{foot}, \dbtext{refrigerator}, \dbtext{newspaper}, \dbtext{rice}, \dbtext{surface}, 
\dbtext{liquid}, \dbtext{meadow}, \dbtext{camp}, \dbtext{window}, \dbtext{oil}, \dbtext{cover}, \dbtext{plate}, \dbtext{dinner}, \dbtext{den}, 
\dbtext{cow}, \dbtext{earth}, \dbtext{garage}, \dbtext{garden}, \dbtext{poop}, \dbtext{potato}, \dbtext{outside}, \dbtext{frog}, \dbtext{napkin}, 
\dbtext{light}, \dbtext{salad}, \dbtext{fox}, \dbtext{forest}, \dbtext{glass}, \dbtext{cupboard}, \dbtext{telephone}, \dbtext{marmot}, 
\dbtext{mountain}, \dbtext{salt}, \dbtext{motel}, \dbtext{drop}, \dbtext{bone}, \dbtext{meat}, \dbtext{bookstore}, \dbtext{rain}, 
\dbtext{body}, \dbtext{use}, \dbtext{ferret}, \dbtext{small dog}, \dbtext{cloth}, \dbtext{factory}, \dbtext{bottle wine}, \dbtext{doll}, 
\dbtext{pencil}, \dbtext{wheel}, \dbtext{lemur}, \dbtext{name}, \dbtext{nice}, \dbtext{book}, \dbtext{museum}, \dbtext{black}, \dbtext{canada}, 
\dbtext{instrument}, \dbtext{trash}, \dbtext{can}, \dbtext{a}, \dbtext{wind}, \dbtext{hand}, \dbtext{picture}, \dbtext{road}, \dbtext{seat}, 
\dbtext{boat}, \dbtext{wild}, \dbtext{clothe}, \dbtext{dish}, \dbtext{train station}, \dbtext{war}, \dbtext{mall}, \dbtext{flower}, \dbtext{wallet}, 
\dbtext{room}, \dbtext{cell}, \dbtext{small}, \dbtext{bicycle}, \dbtext{new york}, \dbtext{farm}, \dbtext{sink}, \dbtext{pocket}, \dbtext{color}, 
\dbtext{white}, \dbtext{red}, \dbtext{stone}, \dbtext{vegetable}, \dbtext{green}, \dbtext{life}, \dbtext{burn}, \dbtext{large}, \dbtext{shoe}, 
\dbtext{scale}, \dbtext{soft}, \dbtext{steak}, \dbtext{fire}, \dbtext{beer}, \dbtext{finger}, \dbtext{knife}, \dbtext{dangerous}, \dbtext{marmoset}, 
\dbtext{carpet}, \dbtext{bowl}, \dbtext{australia}, \dbtext{corn}, \dbtext{fridge}, \dbtext{soap}, \dbtext{expensive}, \dbtext{coin}, \dbtext{number}, 
\dbtext{fruit}, \dbtext{heavy}, \dbtext{map}, \dbtext{fork}, \dbtext{cuba}, \dbtext{france}, \dbtext{italy}, \dbtext{steel}, \dbtext{piano}, 
\dbtext{wall}, \dbtext{theatre}, \dbtext{unite state}, \dbtext{cup}, \dbtext{hill}, \dbtext{square}, \dbtext{apple tree}, \dbtext{shelf}, 
\dbtext{god}, \dbtext{friend house}, \dbtext{airplane}, \dbtext{space}, \dbtext{phone}, \dbtext{this}, \dbtext{place}, \dbtext{tool}, 
\dbtext{apple}, \dbtext{mouth}, \dbtext{bag}, \dbtext{doctor}, \dbtext{theater}, \dbtext{river}, \dbtext{blue}, \dbtext{grass}, \dbtext{cheese}, 
\dbtext{mammal}, \dbtext{bean}, \dbtext{lot}, \dbtext{hair}, \dbtext{make}, \dbtext{measure}, \dbtext{flat}, \dbtext{utah}, \dbtext{plastic}, 
\dbtext{container}, \dbtext{bar}, \dbtext{bug}, \dbtext{live room}, \dbtext{toilet}, \dbtext{tooth}, \dbtext{cabinet}, \dbtext{table}, 
\dbtext{furniture}, \dbtext{lamp}, \dbtext{pizza}, \dbtext{dust}, \dbtext{sand}, \dbtext{hall}, \dbtext{dictionary}, \dbtext{rise}, 
\dbtext{closet}, \dbtext{boy}, \dbtext{door}, \dbtext{floor}, \dbtext{basement}, \dbtext{sofa}, \dbtext{cut}, \dbtext{page}, \dbtext{company}, 
\dbtext{bite}, \dbtext{dark}, \dbtext{college}, \dbtext{world}, \dbtext{air}, \dbtext{sheep}, \dbtext{statue}, \dbtext{metal}, \dbtext{open}, 
\dbtext{warm}, \dbtext{big}, \dbtext{high}, \dbtext{squirrel}, \dbtext{alcohol}, \dbtext{university}, \dbtext{roll}, \dbtext{general}, 
\dbtext{clock}, \dbtext{round}, \dbtext{heat}, \dbtext{skin}, \dbtext{noun}, \dbtext{top}, \dbtext{wine}, \dbtext{jar}, \dbtext{hard}, 
\dbtext{put}, \dbtext{duck}, \dbtext{toy}, \dbtext{ring}, \dbtext{draw}, \dbtext{edible}, \dbtext{wyom}, \dbtext{thing}, \dbtext{land}, 
\dbtext{rug}, \dbtext{pot}, \dbtext{little}, \dbtext{ear}, \dbtext{alive}, \dbtext{bread}, \dbtext{field}, \dbtext{face}, \dbtext{lawn}, 
\dbtext{tin}, \dbtext{oven}, \dbtext{long}, \dbtext{shade}, \dbtext{carry}, \dbtext{fly}, \dbtext{egg}, \dbtext{building}, \dbtext{bee}, 
\dbtext{business}, \dbtext{resturant}, \dbtext{sock}, \dbtext{bear}, \dbtext{bell}, \dbtext{head}, \dbtext{sign}, \dbtext{pantry}, 
\dbtext{bullet}, \dbtext{degree}, \dbtext{note}, \dbtext{card}, \dbtext{supermarket}, \dbtext{machine}, \dbtext{gas}, \dbtext{roof}, 
\dbtext{brown}, \dbtext{circle}, \dbtext{cake}, \dbtext{solid}, \dbtext{dirt}, \dbtext{point}, \dbtext{useful}, \dbtext{handle}, 
\dbtext{adjective}, \dbtext{alaska}, \dbtext{michigan}, \dbtext{maryland}, \dbtext{maine}, \dbtext{delaware}, \dbtext{kansa}, 
\dbtext{department}, \dbtext{be}, \dbtext{steam}, \dbtext{pretty}, \dbtext{side}, \dbtext{decoration}, \dbtext{stapler}, \dbtext{classroom}, 
\dbtext{transport}, \dbtext{step}, \dbtext{apartment}, \dbtext{part}, \dbtext{bush}, \dbtext{countryside}, \dbtext{power}, \dbtext{same}, 
\dbtext{stage}, \dbtext{any large city}, \dbtext{comfort}, \dbtext{edge}, \dbtext{case}, \dbtext{grow}, \dbtext{board}, \dbtext{sunshine}, 
\dbtext{flea}, \dbtext{corner}, \dbtext{short}, \dbtext{outdoors}, \dbtext{stick}, \dbtext{singular}, \dbtext{find house}, \dbtext{find outside}, 
\dbtext{winery}, \dbtext{branch}, \dbtext{polish}, \dbtext{wax}, \dbtext{general term}, \dbtext{generic}, \dbtext{ground}, \dbtext{eaten}, 
\dbtext{neighbor house}, \dbtext{usually}, \dbtext{unit}, \dbtext{generic term}

\paragraph{Community 3 (size is 12).} 
\dbtext{word}, \dbtext{concert}, \dbtext{band}, \dbtext{kill}, \dbtext{club}, \dbtext{not}, \dbtext{mean}, \dbtext{cool}, \dbtext{crowd}, 
\dbtext{organization}, \dbtext{end}, \dbtext{group}

\paragraph{Community 4 (size is 397).} 
\dbtext{man}, \dbtext{person}, \dbtext{type}, \dbtext{train}, \dbtext{work}, \dbtext{write program}, \dbtext{go concert}, \dbtext{hear music}, 
\dbtext{exercise}, \dbtext{love}, \dbtext{bath}, \dbtext{listen}, \dbtext{go performance}, \dbtext{take walk}, \dbtext{walk}, \dbtext{entertain}, 
\dbtext{run marathon}, \dbtext{wait line}, \dbtext{attend lecture}, \dbtext{drink}, \dbtext{study}, \dbtext{go walk}, \dbtext{play basketball}, 
\dbtext{fun}, \dbtext{bore}, \dbtext{wait table}, \dbtext{go see film}, \dbtext{go work}, \dbtext{watch tv show}, \dbtext{wake up morning}, 
\dbtext{dream}, \dbtext{shower}, \dbtext{child}, \dbtext{smoke}, \dbtext{go fish}, \dbtext{tell story}, \dbtext{surf web}, \dbtext{play foot\-ball}, 
\dbtext{movie}, \dbtext{go restaurant}, \dbtext{visit museum}, \dbtext{study subject}, \dbtext{live life}, \dbtext{go sport event}, 
\dbtext{go play}, \dbtext{sit}, \dbtext{play soccer}, \dbtext{go jog}, \dbtext{take shower}, \dbtext{play ball}, \dbtext{eat food}, 
\dbtext{watch movie}, \dbtext{watch film}, \dbtext{stretch}, \dbtext{play frisbee}, \dbtext{go school}, \dbtext{surprise}, \dbtext{paint picture}, 
\dbtext{go film}, \dbtext{party}, \dbtext{rest}, \dbtext{listen radio}, \dbtext{kiss}, \dbtext{remember}, \dbtext{housework}, \dbtext{clean}, 
\dbtext{lunch}, \dbtext{watch tv}, \dbtext{attend school}, \dbtext{play tennis}, \dbtext{trouble}, \dbtext{comfortable}, \dbtext{play}, 
\dbtext{take bus}, \dbtext{conversation}, \dbtext{talk}, \dbtext{take course}, \dbtext{learn}, \dbtext{plan}, \dbtext{think}, \dbtext{go run}, 
\dbtext{sleep}, \dbtext{hang out bar}, \dbtext{plan vacation}, \dbtext{go see play}, \dbtext{attend class}, \dbtext{go swim}, \dbtext{ride bike}, 
\dbtext{nothing}, \dbtext{buy}, \dbtext{eat restaurant}, \dbtext{tv}, \dbtext{stress}, \dbtext{boredom}, \dbtext{ticket}, \dbtext{music}, 
\dbtext{use television}, \dbtext{dress}, \dbtext{entertainment}, \dbtext{listen music}, \dbtext{enjoyment}, \dbtext{hurt}, \dbtext{student}, 
\dbtext{muscle}, \dbtext{go movie}, \dbtext{enlightenment}, \dbtext{stand line}, \dbtext{attend classical concert}, \dbtext{death}, 
\dbtext{play sport}, \dbtext{eat dinner}, \dbtext{effort}, \dbtext{drive car}, \dbtext{traveling}, \dbtext{knowledge}, \dbtext{teach}, 
\dbtext{call}, \dbtext{laugh joke}, \dbtext{run}, \dbtext{read book}, \dbtext{education}, \dbtext{take note}, \dbtext{travel}, \dbtext{go store}, 
\dbtext{see}, \dbtext{story}, \dbtext{go sleep}, \dbtext{tire}, \dbtext{attention}, \dbtext{die}, \dbtext{fall asleep}, \dbtext{money}, 
\dbtext{run errand}, \dbtext{patience}, \dbtext{spend money}, \dbtext{cry}, \dbtext{pay bill}, \dbtext{earn money}, \dbtext{television}, 
\dbtext{speak}, \dbtext{drink water}, \dbtext{fatigue}, \dbtext{take break}, \dbtext{hike}, \dbtext{drink alcohol}, \dbtext{lie}, 
\dbtext{play chess}, \dbtext{friend}, \dbtext{anger}, \dbtext{read}, \dbtext{curiosity}, \dbtext{pay}, \dbtext{swim}, \dbtext{break}, 
\dbtext{verb}, \dbtext{drive}, \dbtext{use computer}, \dbtext{take film}, \dbtext{smile}, \dbtext{fiddle}, \dbtext{we}, \dbtext{wrestle}, 
\dbtext{see new}, \dbtext{dance}, \dbtext{fight}, \dbtext{job}, \dbtext{smart}, \dbtext{play baseball}, \dbtext{excite}, 
\dbtext{attend rock concert}, \dbtext{hear news}, \dbtext{contemplate}, \dbtext{pain}, \dbtext{audience}, \dbtext{understand}, \dbtext{write}, 
\dbtext{stay healthy}, \dbtext{research}, \dbtext{learn new}, \dbtext{sweat}, \dbtext{headache}, \dbtext{fart}, \dbtext{read newspaper}, 
\dbtext{sport}, \dbtext{understand better}, \dbtext{bad}, \dbtext{show}, \dbtext{write story}, \dbtext{stop}, \dbtext{transportation}, 
\dbtext{fall down}, \dbtext{practice}, \dbtext{help}, \dbtext{lose}, \dbtext{close eye}, \dbtext{satisfaction}, \dbtext{time}, 
\dbtext{answer question}, \dbtext{perform}, \dbtext{need}, \dbtext{everyone}, \dbtext{go somewhere}, \dbtext{sound}, \dbtext{good}, 
\dbtext{play card}, \dbtext{go}, \dbtext{sex}, \dbtext{wait}, \dbtext{buy ticket}, \dbtext{gain knowledge}, \dbtext{news}, \dbtext{interest}, 
\dbtext{feel}, \dbtext{sit down}, \dbtext{ski}, \dbtext{surf}, \dbtext{teacher}, \dbtext{leave}, \dbtext{happiness}, \dbtext{exhaustion}, 
\dbtext{sit chair}, \dbtext{laugh}, \dbtext{relax}, \dbtext{waste time}, \dbtext{pleasure}, \dbtext{relaxation}, \dbtext{care}, 
\dbtext{procreate}, \dbtext{watch}, \dbtext{funny}, \dbtext{win}, \dbtext{go mall}, \dbtext{flirt}, \dbtext{pass time}, \dbtext{noise}, 
\dbtext{climb}, \dbtext{wash hand}, \dbtext{go home}, \dbtext{love else}, \dbtext{drunk}, \dbtext{peace}, \dbtext{sing}, \dbtext{buy beer}, 
\dbtext{internet}, \dbtext{kid}, \dbtext{like}, \dbtext{date}, \dbtext{record}, \dbtext{find}, \dbtext{song}, \dbtext{play game}, \dbtext{meet}, 
\dbtext{activity}, \dbtext{jog}, \dbtext{quiet}, \dbtext{skill}, \dbtext{hobby}, \dbtext{birthday}, \dbtext{tiredness}, \dbtext{communication}, 
\dbtext{drink coffee}, \dbtext{read magazine}, \dbtext{good time}, \dbtext{good health}, \dbtext{act}, \dbtext{play hockey}, \dbtext{eat ice cream}, 
\dbtext{learn language}, \dbtext{dive}, \dbtext{go zoo}, \dbtext{go internet}, \dbtext{cash}, \dbtext{important}, \dbtext{read child}, 
\dbtext{enjoy yourself}, \dbtext{see movie}, \dbtext{energy}, \dbtext{kill person}, \dbtext{emotion}, \dbtext{clean house}, \dbtext{fit}, 
\dbtext{view video}, \dbtext{play poker}, \dbtext{excitement}, \dbtext{move}, \dbtext{fly airplane}, \dbtext{ride horse}, \dbtext{stay bed}, 
\dbtext{look}, \dbtext{voice}, \dbtext{event}, \dbtext{happy}, \dbtext{find information}, \dbtext{fear}, \dbtext{go vacation}, \dbtext{breathe}, 
\dbtext{recreation}, \dbtext{enjoy}, \dbtext{hear}, \dbtext{jump}, \dbtext{ride bicycle}, \dbtext{health}, \dbtext{communicate}, \dbtext{make money}, 
\dbtext{become tire}, \dbtext{action}, \dbtext{fall}, \dbtext{lose weight}, \dbtext{jump up down}, \dbtext{watch television}, \dbtext{count}, 
\dbtext{healthy}, \dbtext{know}, \dbtext{learn subject}, \dbtext{joy}, \dbtext{stand up}, \dbtext{information}, \dbtext{read letter}, \dbtext{lay}, 
\dbtext{jump rope}, \dbtext{celebrate}, \dbtext{sad\-ness}, \dbtext{bike}, \dbtext{watch musician perform}, \dbtext{motion}, \dbtext{feel better}, 
\dbtext{compete}, \dbtext{out}, \dbtext{feel good}, \dbtext{accident}, \dbtext{stay fit}, \dbtext{injury}, \dbtext{ride}, \dbtext{play piano}, 
\dbtext{learn world}, \dbtext{see exhibit}, \dbtext{release energy}, \dbtext{see art}, \dbtext{see excite story}, \dbtext{orgasm}, \dbtext{trip}, 
\dbtext{laughter}, \dbtext{express yourself}, \dbtext{discover truth}, \dbtext{see favorite show}, \dbtext{go party}, \dbtext{competition}, 
\dbtext{express information}, \dbtext{climb mountain}, \dbtext{attend meet}, \dbtext{fly kite}, \dbtext{examine}, \dbtext{race}, \dbtext{meet friend}, 
\dbtext{read news}, \dbtext{shock}, \dbtext{return work}, \dbtext{see band}, \dbtext{visit art gallery}, \dbtext{earn live}, \dbtext{punch}, 
\dbtext{cool off}, \dbtext{watch television show}, \dbtext{so\-cial\-ize}, \dbtext{skate}, \dbtext{movement}, \dbtext{create art}, 
\dbtext{crossword puzzle}, \dbtext{enjoy film}, \dbtext{go pub}, \dbtext{feel happy}, \dbtext{play lacrosse}, \dbtext{socialis}, \dbtext{away}, 
\dbtext{physical activity}, \dbtext{get}, \dbtext{many person}, \dbtext{make person laugh}, \dbtext{make friend}, \dbtext{chat friend}, 
\dbtext{meet person}, \dbtext{meet interest person}, \dbtext{get drunk}, \dbtext{friend over}, \dbtext{get exercise}, \dbtext{get tire}, 
\dbtext{enjoy company friend}, \dbtext{play game friend}, \dbtext{get physical activity}, \dbtext{go opus}, \dbtext{get shape}, \dbtext{sit quietly}, 
\dbtext{do it}, \dbtext{get fit}, \dbtext{teach other person}, \dbtext{entertain person}, \dbtext{see person play game}

\subsection{Multilevel}
The algorithm used is the one implemented in \igraph 
which is based on \citep{communities:multilevel}.
Table \ref{tbl:communities:positive:multilevel:cores} presents the results when we apply the algorithm for 
subgraphs in which we include vertices with successively lower coreness
and allow positive polarity on the edges only.

\begin{table}[ht]
\caption{Applying the Multilevel algorithm for community finding implemented in \igraph by successively including
vertices with lower coreness on the undirected graph induced by the assertions with positive polarity (self-loops are removed).
In every row we have the number of vertices and the number of edges of each such subgraph together with
the number of components ($\abs{C}$) that we find in that subgraph. 
The next three columns present the number of communities found by the algorithm;
the average among all runs, the minimum, and the maximum.
The next three columns present the modularity achieved by the algorithm due to the cut induced by the communities;
the average among all runs, the minimum, and the maximum.
The entire computation lasted $687.9$ seconds for $100$ runs; 
that is about $6.879$ seconds per run.}\label{tbl:communities:positive:multilevel:cores}
\begin{center}
\begin{tabular}{|c|r|r|c||r|r|r||r|r|r||}\hline
\multirow{2}{*}{coreness} & \multicolumn{1}{c|}{\multirow{2}{*}{$\abs{V}$}} & \multicolumn{1}{c|}{\multirow{2}{*}{$\abs{E}$}} & \multirow{2}{*}{$\abs{C}$} & \multicolumn{3}{c||}{communities found} & \multicolumn{3}{c||}{modularity} \\\cline{5-10}
          &        &        &   & \multicolumn{1}{c|}{avg} & \multicolumn{1}{c|}{min} & \multicolumn{1}{c||}{max} & \multicolumn{1}{c|}{avg} & \multicolumn{1}{c|}{min} & \multicolumn{1}{c||}{max} \\\hline\hline
$\geq 26$ &    869 &  20526 &    1 &  5.000 &    5 &    5 & 0.322157 & 0.322157 & 0.322157 \\\hline
$\geq 25$ &   1167 &  27810 &    1 &  6.000 &    6 &    6 & 0.320322 & 0.320322 & 0.320322 \\\hline 
$\geq 24$ &   1358 &  32314 &    1 &  7.000 &    7 &    7 & 0.320777 & 0.320777 & 0.320777 \\\hline
$\geq 23$ &   1514 &  35870 &    1 &  6.000 &    6 &    6 & 0.322224 & 0.322224 & 0.322224 \\\hline
$\geq 22$ &   1709 &  40099 &    1 &  7.000 &    7 &    7 & 0.322961 & 0.322961 & 0.322961 \\\hline
$\geq 21$ &   1865 &  43330 &    1 &  8.000 &    8 &    8 & 0.314372 & 0.314372 & 0.314372 \\\hline
$\geq 20$ &   2007 &  46145 &    1 &  7.000 &    7 &    7 & 0.320638 & 0.320638 & 0.320638 \\\hline
$\geq 19$ &   2173 &  49265 &    1 &  7.000 &    7 &    7 & 0.321349 & 0.321349 & 0.321349 \\\hline
$\geq 18$ &   2384 &  53011 &    1 &  8.000 &    8 &    8 & 0.329111 & 0.329111 & 0.329111 \\\hline
$\geq 17$ &   2617 &  56939 &    1 &  9.000 &    9 &    9 & 0.330618 & 0.330618 & 0.330618 \\\hline
$\geq 16$ &   2847 &  60583 &    1 &  6.000 &    6 &    6 & 0.331601 & 0.331601 & 0.331601 \\\hline
$\geq 15$ &   3105 &  64412 &    1 &  7.000 &    7 &    7 & 0.336495 & 0.336495 & 0.336495 \\\hline
$\geq 14$ &   3407 &  68613 &    1 &  9.000 &    9 &    9 & 0.347883 & 0.347883 & 0.347883 \\\hline
$\geq 13$ &   3746 &  72978 &    1 &  9.000 &    9 &    9 & 0.344480 & 0.344480 & 0.344480 \\\hline
$\geq 12$ &   4160 &  77882 &    1 & 10.000 &   10 &   10 & 0.349346 & 0.349346 & 0.349346 \\\hline
$\geq 11$ &   4634 &  83039 &    1 & 10.000 &   10 &   10 & 0.348056 & 0.348056 & 0.348056 \\\hline
$\geq 10$ &   5182 &  88462 &    1 & 10.000 &   10 &   10 & 0.361789 & 0.361789 & 0.361789 \\\hline
$\geq  9$ &   5883 &  94709 &    1 &  9.000 &    9 &    9 & 0.363861 & 0.363861 & 0.363861 \\\hline
$\geq  8$ &   6750 & 101564 &    1 & 10.000 &   10 &   10 & 0.368195 & 0.368195 & 0.368195 \\\hline
$\geq  7$ &   7904 & 109561 &    1 & 10.000 &   10 &   10 & 0.374810 & 0.374810 & 0.374810 \\\hline
$\geq  6$ &   9392 & 118389 &    1 & 10.000 &   10 &   10 & 0.386815 & 0.386815 & 0.386815 \\\hline
$\geq  5$ &  11483 & 128731 &    1 & 12.000 &   12 &   12 & 0.391540 & 0.391540 & 0.391540 \\\hline
$\geq  4$ &  14864 & 142112 &    1 & 11.000 &   11 &   11 & 0.401597 & 0.401597 & 0.401597 \\\hline
$\geq  3$ &  21812 & 162691 &    1 & 13.000 &   13 &   13 & 0.419143 & 0.419143 & 0.419143 \\\hline
$\geq  2$ &  41659 & 201678 &    4 & 25.000 &   25 &   25 & 0.449455 & 0.449455 & 0.449455 \\\hline
\end{tabular}
\end{center}
\end{table}

\subsubsection{Multilevel: Communities in the Inner-Most Core}
We have the following communities.
\paragraph{Community 1 (size is 92).}
\dbtext{soup}, \dbtext{drink}, \dbtext{chicken}, \dbtext{eat food}, \dbtext{coffee}, \dbtext{lunch}, \dbtext{eat}, 
\dbtext{milk}, \dbtext{bottle}, \dbtext{market}, \dbtext{fish}, \dbtext{cook}, \dbtext{shop}, \dbtext{eat dinner}, 
\dbtext{kitchen}, \dbtext{restaurant}, \dbtext{spoon}, \dbtext{butter}, \dbtext{eye}, \dbtext{food}, \dbtext{grape}, 
\dbtext{sugar}, \dbtext{grocery store}, \dbtext{basket}, \dbtext{hold}, \dbtext{refrigerator}, \dbtext{rice}, \dbtext{liquid}, 
\dbtext{oil}, \dbtext{plate}, \dbtext{dinner}, \dbtext{potato}, \dbtext{napkin}, \dbtext{salad}, \dbtext{glass}, \dbtext{cupboard}, 
\dbtext{salt}, \dbtext{bone}, \dbtext{meat}, \dbtext{bottle wine}, \dbtext{can}, \dbtext{a}, \dbtext{hand}, \dbtext{dish}, 
\dbtext{sink}, \dbtext{white}, \dbtext{vegetable}, \dbtext{steak}, \dbtext{beer}, \dbtext{knife}, \dbtext{bowl}, \dbtext{corn}, 
\dbtext{fridge}, \dbtext{soap}, \dbtext{fruit}, \dbtext{fork}, \dbtext{cup}, \dbtext{apple}, \dbtext{mouth}, \dbtext{cheese}, 
\dbtext{bean}, \dbtext{container}, \dbtext{wash hand}, \dbtext{tooth}, \dbtext{cabinet}, \dbtext{pizza}, \dbtext{cut}, \dbtext{open}, 
\dbtext{alcohol}, \dbtext{round}, \dbtext{skin}, \dbtext{wine}, \dbtext{jar}, \dbtext{edible}, \dbtext{pot}, \dbtext{ear}, \dbtext{bread}, 
\dbtext{face}, \dbtext{tin}, \dbtext{oven}, \dbtext{egg}, \dbtext{resturant}, \dbtext{wash}, \dbtext{head}, \dbtext{pantry}, 
\dbtext{supermarket}, \dbtext{cake}, \dbtext{steam}, \dbtext{power}, \dbtext{same}, \dbtext{eaten}, \dbtext{usually}

\paragraph{Community 2 (size is 180).} 
\dbtext{rock}, \dbtext{beach}, \dbtext{tree}, \dbtext{monkey}, \dbtext{weasel}, \dbtext{pant}, \dbtext{kitten}, \dbtext{arm}, 
\dbtext{beaver}, \dbtext{smoke}, \dbtext{state}, \dbtext{wiener dog}, \dbtext{fungus}, \dbtext{park}, \dbtext{snake}, \dbtext{wood}, 
\dbtext{bridge}, \dbtext{cloud}, \dbtext{nothing}, \dbtext{dog}, \dbtext{zoo}, \dbtext{live}, \dbtext{cat}, \dbtext{hat}, \dbtext{country}, 
\dbtext{lake}, \dbtext{baby}, \dbtext{plant}, \dbtext{hide}, \dbtext{animal}, \dbtext{cold}, \dbtext{moon}, \dbtext{pet}, \dbtext{bird}, 
\dbtext{shark}, \dbtext{water}, \dbtext{rosebush}, \dbtext{yard}, \dbtext{sloth}, \dbtext{bat}, \dbtext{lizard}, \dbtext{beautiful}, 
\dbtext{nose}, \dbtext{well}, \dbtext{snow}, \dbtext{weather}, \dbtext{everything}, \dbtext{mouse}, \dbtext{hole}, \dbtext{nature}, 
\dbtext{bald eagle}, \dbtext{nest}, \dbtext{crab}, \dbtext{ficus}, \dbtext{sea}, \dbtext{anemone}, \dbtext{ocean}, \dbtext{sun}, 
\dbtext{sky}, \dbtext{horse}, \dbtext{hot}, \dbtext{meadow}, \dbtext{camp}, \dbtext{den}, \dbtext{cow}, \dbtext{earth}, \dbtext{garden}, 
\dbtext{poop}, \dbtext{outside}, \dbtext{frog}, \dbtext{light}, \dbtext{fox}, \dbtext{forest}, \dbtext{marmot}, \dbtext{mountain}, 
\dbtext{rain}, \dbtext{body}, \dbtext{ferret}, \dbtext{small dog}, \dbtext{lemur}, \dbtext{nice}, \dbtext{museum}, \dbtext{black}, 
\dbtext{canada}, \dbtext{wind}, \dbtext{wild}, \dbtext{flower}, \dbtext{small}, \dbtext{new york}, \dbtext{farm}, \dbtext{color}, 
\dbtext{red}, \dbtext{stone}, \dbtext{green}, \dbtext{life}, \dbtext{burn}, \dbtext{large}, \dbtext{soft}, \dbtext{fire}, \dbtext{dangerous}, 
\dbtext{marmoset}, \dbtext{australia}, \dbtext{leave}, \dbtext{heavy}, \dbtext{cuba}, \dbtext{france}, \dbtext{italy}, \dbtext{unite state}, 
\dbtext{hill}, \dbtext{apple tree}, \dbtext{god}, \dbtext{space}, \dbtext{river}, \dbtext{blue}, \dbtext{grass}, \dbtext{mammal}, \dbtext{lot}, 
\dbtext{hair}, \dbtext{utah}, \dbtext{bug}, \dbtext{view}, \dbtext{sand}, \dbtext{dictionary}, \dbtext{rise}, \dbtext{bite}, \dbtext{dark}, 
\dbtext{science}, \dbtext{world}, \dbtext{air}, \dbtext{sheep}, \dbtext{statue}, \dbtext{warm}, \dbtext{big}, \dbtext{high}, \dbtext{squirrel}, 
\dbtext{general}, \dbtext{heat}, \dbtext{cool}, \dbtext{art}, \dbtext{hard}, \dbtext{duck}, \dbtext{wyom}, \dbtext{land}, \dbtext{little}, 
\dbtext{alive}, \dbtext{field}, \dbtext{lawn}, \dbtext{long}, \dbtext{shade}, \dbtext{fly}, \dbtext{bee}, \dbtext{bear}, \dbtext{gas}, 
\dbtext{brown}, \dbtext{solid}, \dbtext{dirt}, \dbtext{adjective}, \dbtext{alaska}, \dbtext{michigan}, \dbtext{maryland}, \dbtext{maine}, 
\dbtext{delaware}, \dbtext{kansa}, \dbtext{be}, \dbtext{pretty}, \dbtext{bush}, \dbtext{course}, \dbtext{countryside}, \dbtext{grow}, 
\dbtext{sunshine}, \dbtext{flea}, \dbtext{short}, \dbtext{outdoors}, \dbtext{stick}, \dbtext{find outside}, \dbtext{branch}, \dbtext{wax}, 
\dbtext{generic}, \dbtext{ground}, \dbtext{generic term}

\paragraph{Community 3 (size is 184).}
\dbtext{something}, \dbtext{type}, \dbtext{town}, \dbtext{word}, \dbtext{library}, \dbtext{school}, \dbtext{human}, \dbtext{plane}, 
\dbtext{class}, \dbtext{it}, \dbtext{paper}, \dbtext{bed}, \dbtext{dirty}, \dbtext{office build}, \dbtext{sit}, \dbtext{box}, \dbtext{object}, 
\dbtext{mother}, \dbtext{candle}, \dbtext{street}, \dbtext{bus}, \dbtext{computer}, \dbtext{line}, \dbtext{drawer}, \dbtext{storage}, \dbtext{car}, 
\dbtext{vehicle}, \dbtext{turn}, \dbtext{material}, \dbtext{chair}, \dbtext{house}, \dbtext{hotel}, \dbtext{game}, \dbtext{hospital}, \dbtext{bank}, 
\dbtext{girl}, \dbtext{church}, \dbtext{family}, \dbtext{letter}, \dbtext{bathroom}, \dbtext{city}, \dbtext{desk}, \dbtext{office}, \dbtext{home}, 
\dbtext{couch}, \dbtext{build}, \dbtext{key}, \dbtext{electricity}, \dbtext{stand}, \dbtext{pen}, \dbtext{bill}, \dbtext{magazine}, \dbtext{band}, 
\dbtext{paint}, \dbtext{bedroom}, \dbtext{store}, \dbtext{airport}, \dbtext{newspaper}, \dbtext{surface}, \dbtext{window}, \dbtext{cover}, 
\dbtext{garage}, \dbtext{telephone}, \dbtext{motel}, \dbtext{bookstore}, \dbtext{use}, \dbtext{write}, \dbtext{cloth}, \dbtext{factory}, 
\dbtext{doll}, \dbtext{pencil}, \dbtext{name}, \dbtext{book}, \dbtext{instrument}, \dbtext{trash}, \dbtext{picture}, \dbtext{road}, \dbtext{seat}, 
\dbtext{boat}, \dbtext{clothe}, \dbtext{train station}, \dbtext{mall}, \dbtext{wallet}, \dbtext{room}, \dbtext{cell}, \dbtext{pocket}, \dbtext{shoe}, 
\dbtext{scale}, \dbtext{finger}, \dbtext{carpet}, \dbtext{expensive}, \dbtext{coin}, \dbtext{number}, \dbtext{map}, \dbtext{steel}, \dbtext{piano}, 
\dbtext{wall}, \dbtext{club}, \dbtext{square}, \dbtext{shelf}, \dbtext{friend house}, \dbtext{airplane}, \dbtext{phone}, \dbtext{this}, \dbtext{place}, 
\dbtext{radio}, \dbtext{tool}, \dbtext{bag}, \dbtext{doctor}, \dbtext{theater}, \dbtext{make}, \dbtext{measure}, \dbtext{flat}, \dbtext{plastic}, 
\dbtext{bar}, \dbtext{live room}, \dbtext{toilet}, \dbtext{table}, \dbtext{furniture}, \dbtext{lamp}, \dbtext{dust}, \dbtext{hall}, \dbtext{closet}, 
\dbtext{boy}, \dbtext{door}, \dbtext{floor}, \dbtext{basement}, \dbtext{sofa}, \dbtext{page}, \dbtext{company}, \dbtext{college}, \dbtext{metal}, 
\dbtext{university}, \dbtext{clock}, \dbtext{noun}, \dbtext{top}, \dbtext{put}, \dbtext{toy}, \dbtext{ring}, \dbtext{crowd}, \dbtext{draw}, 
\dbtext{thing}, \dbtext{rug}, \dbtext{change}, \dbtext{carry}, \dbtext{test}, \dbtext{organization}, \dbtext{building}, \dbtext{business}, 
\dbtext{sock}, \dbtext{bell}, \dbtext{sign}, \dbtext{group}, \dbtext{bullet}, \dbtext{degree}, \dbtext{note}, \dbtext{card}, \dbtext{machine}, 
\dbtext{roof}, \dbtext{circle}, \dbtext{point}, \dbtext{useful}, \dbtext{handle}, \dbtext{department}, \dbtext{side}, \dbtext{decoration}, 
\dbtext{stapler}, \dbtext{classroom}, \dbtext{apartment}, \dbtext{part}, \dbtext{stage}, \dbtext{any large city}, \dbtext{comfort}, \dbtext{edge}, 
\dbtext{case}, \dbtext{board}, \dbtext{corner}, \dbtext{singular}, \dbtext{find house}, \dbtext{winery}, \dbtext{polish}, \dbtext{general term}, 
\dbtext{neighbor house}, \dbtext{unit}

\paragraph{Community 4 (size is 149).} 
\dbtext{man}, \dbtext{train}, \dbtext{work}, \dbtext{exercise}, \dbtext{take walk}, \dbtext{walk}, \dbtext{run marathon}, \dbtext{go walk}, 
\dbtext{play basketball}, \dbtext{wait table}, \dbtext{go work}, \dbtext{wake up morning}, \dbtext{shower}, \dbtext{child}, \dbtext{gym}, 
\dbtext{play football}, \dbtext{play soccer}, \dbtext{go jog}, \dbtext{take shower}, \dbtext{play ball}, \dbtext{ball}, \dbtext{stretch}, 
\dbtext{play frisbee}, \dbtext{housework}, \dbtext{clean}, \dbtext{play tennis}, \dbtext{plan}, \dbtext{go run}, \dbtext{go swim}, 
\dbtext{ride bike}, \dbtext{stress}, \dbtext{dress}, \dbtext{one}, \dbtext{hurt}, \dbtext{muscle}, \dbtext{woman}, \dbtext{death}, 
\dbtext{play sport}, \dbtext{effort}, \dbtext{drive car}, \dbtext{traveling}, \dbtext{run}, \dbtext{travel}, \dbtext{go store}, \dbtext{smell}, 
\dbtext{tire}, \dbtext{die}, \dbtext{leg}, \dbtext{run errand}, \dbtext{earn money}, \dbtext{drink water}, \dbtext{fatigue}, \dbtext{hike}, 
\dbtext{kill}, \dbtext{swim}, \dbtext{break}, \dbtext{foot}, \dbtext{verb}, \dbtext{drive}, \dbtext{wrestle}, \dbtext{fight}, \dbtext{play baseball}, 
\dbtext{pain}, \dbtext{drop}, \dbtext{stay healthy}, \dbtext{wheel}, \dbtext{sweat}, \dbtext{pool}, \dbtext{sport}, \dbtext{bad}, \dbtext{pee}, 
\dbtext{stop}, \dbtext{transportation}, \dbtext{fall down}, \dbtext{practice}, \dbtext{lose}, \dbtext{war}, \dbtext{wet}, \dbtext{bicycle}, 
\dbtext{go somewhere}, \dbtext{go}, \dbtext{ski}, \dbtext{exhaustion}, \dbtext{win}, \dbtext{shape}, \dbtext{climb}, \dbtext{not}, \dbtext{activity}, 
\dbtext{jog}, \dbtext{roll}, \dbtext{tiredness}, \dbtext{mean}, \dbtext{drink coffee}, \dbtext{good health}, \dbtext{play hockey}, \dbtext{dive}, 
\dbtext{energy}, \dbtext{kill person}, \dbtext{clean house}, \dbtext{fit}, \dbtext{move}, \dbtext{ride horse}, \dbtext{wave}, \dbtext{fear}, 
\dbtext{jump}, \dbtext{ride bicycle}, \dbtext{health}, \dbtext{become tire}, \dbtext{action}, \dbtext{pass}, \dbtext{fall}, \dbtext{lose weight}, 
\dbtext{jump up down}, \dbtext{count}, \dbtext{healthy}, \dbtext{end}, \dbtext{stand up}, \dbtext{jump rope}, \dbtext{bike}, \dbtext{motion}, 
\dbtext{feel better}, \dbtext{compete}, \dbtext{out}, \dbtext{accident}, \dbtext{transport}, \dbtext{stay fit}, \dbtext{injury}, \dbtext{ride}, 
\dbtext{step}, \dbtext{release energy}, \dbtext{trip}, \dbtext{competition}, \dbtext{climb mountain}, \dbtext{race}, \dbtext{shock}, 
\dbtext{return work}, \dbtext{earn live}, \dbtext{punch}, \dbtext{cool off}, \dbtext{skate}, \dbtext{movement}, \dbtext{play lacrosse}, \dbtext{away}, 
\dbtext{physical activity}, \dbtext{get exercise}, \dbtext{get tire}, \dbtext{get physical activity}, \dbtext{get shape}, \dbtext{get fit}

\paragraph{Community 5 (size is 264).} 
\dbtext{person}, \dbtext{write program}, \dbtext{go concert}, \dbtext{hear music}, \dbtext{love}, \dbtext{bath}, \dbtext{listen}, 
\dbtext{go performance}, \dbtext{entertain}, \dbtext{wait line}, \dbtext{attend lecture}, \dbtext{study}, \dbtext{fun}, \dbtext{bore}, 
\dbtext{go see film}, \dbtext{watch tv show}, \dbtext{dream}, \dbtext{go fish}, \dbtext{tell story}, \dbtext{surf web}, \dbtext{movie}, 
\dbtext{go restaurant}, \dbtext{visit museum}, \dbtext{study subject}, \dbtext{live life}, \dbtext{go sport event}, \dbtext{go play}, 
\dbtext{watch movie}, \dbtext{watch film}, \dbtext{go school}, \dbtext{surprise}, \dbtext{paint picture}, \dbtext{go film}, \dbtext{party}, 
\dbtext{rest}, \dbtext{listen radio}, \dbtext{kiss}, \dbtext{remember}, \dbtext{watch tv}, \dbtext{attend school}, \dbtext{trouble}, 
\dbtext{comfortable}, \dbtext{play}, \dbtext{take bus}, \dbtext{con\-ver\-sa\-tion}, \dbtext{talk}, \dbtext{take course}, \dbtext{learn}, 
\dbtext{think}, \dbtext{sleep}, \dbtext{hang out bar}, \dbtext{plan vacation}, \dbtext{go see play}, \dbtext{attend class}, \dbtext{buy}, 
\dbtext{eat restaurant}, \dbtext{tv}, \dbtext{boredom}, \dbtext{ticket}, \dbtext{music}, \dbtext{use television}, \dbtext{entertainment}, 
\dbtext{listen music}, \dbtext{enjoyment}, \dbtext{student}, \dbtext{go movie}, \dbtext{enlightenment}, \dbtext{stand line}, 
\dbtext{attend classical concert}, \dbtext{concert}, \dbtext{knowl\-edge}, \dbtext{teach}, \dbtext{call}, \dbtext{laugh joke}, 
\dbtext{read book}, \dbtext{education}, \dbtext{take note}, \dbtext{see}, \dbtext{story}, \dbtext{go sleep}, \dbtext{attention}, 
\dbtext{fall asleep}, \dbtext{money}, \dbtext{patience}, \dbtext{spend money}, \dbtext{cry}, \dbtext{pay bill}, \dbtext{television}, 
\dbtext{speak}, \dbtext{take bath}, \dbtext{take break}, \dbtext{drink alcohol}, \dbtext{lie}, \dbtext{play chess}, \dbtext{friend}, 
\dbtext{anger}, \dbtext{read}, \dbtext{curiosity}, \dbtext{pay}, \dbtext{use computer}, \dbtext{take film}, \dbtext{smile}, 
\dbtext{fiddle}, \dbtext{we}, \dbtext{see new}, \dbtext{dance}, \dbtext{job}, \dbtext{smart}, \dbtext{excite}, \dbtext{attend rock concert}, 
\dbtext{hear news}, \dbtext{contemplate}, \dbtext{audience}, \dbtext{understand}, \dbtext{research}, \dbtext{learn new}, \dbtext{headache}, 
\dbtext{fart}, \dbtext{read newspaper}, \dbtext{understand better}, \dbtext{show}, \dbtext{write story}, \dbtext{help}, \dbtext{close eye}, 
\dbtext{satisfaction}, \dbtext{time}, \dbtext{answer question}, \dbtext{perform}, \dbtext{need}, \dbtext{everyone}, \dbtext{sound}, 
\dbtext{good}, \dbtext{play card}, \dbtext{sex}, \dbtext{wait}, \dbtext{buy ticket}, \dbtext{gain knowledge}, \dbtext{news}, \dbtext{interest}, 
\dbtext{feel}, \dbtext{sit down}, \dbtext{surf}, \dbtext{teacher}, \dbtext{happiness}, \dbtext{sit chair}, \dbtext{laugh}, \dbtext{theatre}, 
\dbtext{relax}, \dbtext{waste time}, \dbtext{pleasure}, \dbtext{relaxation}, \dbtext{care}, \dbtext{procreate}, \dbtext{watch}, \dbtext{funny}, 
\dbtext{go mall}, \dbtext{flirt}, \dbtext{pass time}, \dbtext{noise}, \dbtext{go home}, \dbtext{love else}, \dbtext{drunk}, \dbtext{peace}, 
\dbtext{sing}, \dbtext{buy beer}, \dbtext{internet}, \dbtext{kid}, \dbtext{like}, \dbtext{date}, \dbtext{record}, \dbtext{find}, \dbtext{song}, 
\dbtext{play game}, \dbtext{meet}, \dbtext{quiet}, \dbtext{skill}, \dbtext{hobby}, \dbtext{birthday}, \dbtext{communication}, \dbtext{read magazine}, 
\dbtext{good time}, \dbtext{act}, \dbtext{eat ice cream}, \dbtext{learn language}, \dbtext{go zoo}, \dbtext{go internet}, \dbtext{cash}, 
\dbtext{important}, \dbtext{read child}, \dbtext{enjoy yourself}, \dbtext{see movie}, \dbtext{emotion}, \dbtext{view video}, \dbtext{play poker}, 
\dbtext{excitement}, \dbtext{fly airplane}, \dbtext{stay bed}, \dbtext{look}, \dbtext{voice}, \dbtext{event}, \dbtext{happy}, \dbtext{find information}, 
\dbtext{go vacation}, \dbtext{breathe}, \dbtext{recreation}, \dbtext{enjoy}, \dbtext{hear}, \dbtext{communicate}, \dbtext{make money}, 
\dbtext{watch tel\-e\-vi\-sion}, \dbtext{know}, \dbtext{learn subject}, \dbtext{joy}, \dbtext{information}, \dbtext{read letter}, \dbtext{lay}, 
\dbtext{celebrate}, \dbtext{sadness}, \dbtext{watch musician perform}, \dbtext{feel good}, \dbtext{play piano}, \dbtext{learn world}, 
\dbtext{see exhibit}, \dbtext{see art}, \dbtext{see excite story}, \dbtext{orgasm}, \dbtext{laughter}, \dbtext{express yourself}, 
\dbtext{discover truth}, \dbtext{see favorite show}, \dbtext{go party}, \dbtext{express information}, \dbtext{attend meet}, \dbtext{fly kite}, 
\dbtext{examine}, \dbtext{meet friend}, \dbtext{read news}, \dbtext{see band}, \dbtext{visit art gallery}, \dbtext{watch television show}, 
\dbtext{so\-cial\-ize}, \dbtext{create art}, \dbtext{crossword puzzle}, \dbtext{enjoy film}, \dbtext{go pub}, \dbtext{feel happy}, \dbtext{socialis}, 
\dbtext{get}, \dbtext{many person}, \dbtext{make person laugh}, \dbtext{make friend}, \dbtext{chat friend}, \dbtext{meet person}, 
\dbtext{meet interest person}, \dbtext{get drunk}, \dbtext{friend over}, \dbtext{enjoy company friend}, \dbtext{play game friend}, 
\dbtext{go opus}, \dbtext{sit quietly}, \dbtext{do it}, \dbtext{teach other person}, \dbtext{en\-ter\-tain person}, 
\dbtext{see person play game}

\subsection{Label Propagation}
The algorithm used is the one implemented in \igraph 
which is based on \citep{communities:lp}.
Table \ref{tbl:communities:positive:lp:cores} presents the results when we apply the algorithm for 
subgraphs in which we include vertices with successively lower coreness
and allow positive polarity on the edges only.

\begin{table}[ht]
\caption{Applying the Label Propagation algorithm for community finding implemented in \igraph by successively including
vertices with lower coreness on the induced undirected graph (self-loops are removed).
In every row we have the number of vertices and the number of edges of each such subgraph together with
the number of components ($\abs{C}$) that we find in that subgraph. 
The next three columns present the number of communities found by the algorithm;
the average among all runs, the minimum, and the maximum.
The next three columns present the modularity achieved by the algorithm due to the cut induced by the communities;
the average among all runs, the minimum, and the maximum.
Finally the last two columns
present in how many runs the algorithm computed as many communities as we had components in that subgraph.
The entire computation lasted $2065.4$ seconds for $100$ runs; that is about $20.65$ seconds per run.}\label{tbl:communities:positive:lp:cores}
\begin{center}
\begin{tabular}{|c|r|r|c||r|r|r||r|r|r||r|r|}\hline
\multirow{2}{*}{coreness} & \multicolumn{1}{c|}{\multirow{2}{*}{$\abs{V}$}} & \multicolumn{1}{c|}{\multirow{2}{*}{$\abs{E}$}} & \multirow{2}{*}{$\abs{C}$} & \multicolumn{3}{c||}{communities found} & \multicolumn{3}{c||}{modularity} & \multicolumn{2}{c|}{agreement} \\\cline{5-12}
          &        &        &   & \multicolumn{1}{c|}{avg} & \multicolumn{1}{c|}{min} & \multicolumn{1}{c||}{max} & \multicolumn{1}{c|}{avg} & \multicolumn{1}{c|}{min} & \multicolumn{1}{c||}{max} & \multicolumn{1}{c|}{Y} & \multicolumn{1}{c|}{N} \\\hline\hline
$\geq 26$ &    869 &  20526 &    1 &  1.000 &    1 &    1 & 0.000000 & 0.000000 & 0.000000 & 100 &   0 \\\hline
$\geq 25$ &   1167 &  27810 &    1 &  1.010 &    1 &    2 & 0.002823 & 0.000000 & 0.282317 &  99 &   1 \\\hline
$\geq 24$ &   1358 &  32314 &    1 &  1.060 &    1 &    2 & 0.016925 & 0.000000 & 0.284380 &  94 &   6 \\\hline
$\geq 23$ &   1514 &  35870 &    1 &  1.000 &    1 &    1 & 0.000000 & 0.000000 & 0.000000 & 100 &   0 \\\hline
$\geq 22$ &   1709 &  40099 &    1 &  1.010 &    1 &    2 & 0.002825 & 0.000000 & 0.282457 &  99 &   1 \\\hline
$\geq 21$ &   1865 &  43330 &    1 &  1.000 &    1 &    1 & 0.000000 & 0.000000 & 0.000000 & 100 &   0 \\\hline
$\geq 20$ &   2007 &  46145 &    1 &  1.030 &    1 &    2 & 0.008194 & 0.000000 & 0.275439 &  97 &   3 \\\hline
$\geq 19$ &   2173 &  49265 &    1 &  1.010 &    1 &    2 & 0.002720 & 0.000000 & 0.272007 &  99 &   1 \\\hline
$\geq 18$ &   2384 &  53011 &    1 &  1.010 &    1 &    2 & 0.002776 & 0.000000 & 0.277615 &  99 &   1 \\\hline
$\geq 17$ &   2617 &  56939 &    1 &  1.000 &    1 &    1 & 0.000000 & 0.000000 & 0.000000 & 100 &   0 \\\hline
$\geq 16$ &   2847 &  60583 &    1 &  1.020 &    1 &    2 & 0.005575 & 0.000000 & 0.278782 &  98 &   2 \\\hline
$\geq 15$ &   3105 &  64412 &    1 &  1.000 &    1 &    1 & 0.000000 & 0.000000 & 0.000000 & 100 &   0 \\\hline
$\geq 14$ &   3407 &  68613 &    1 &  1.010 &    1 &    2 & 0.002807 & 0.000000 & 0.280742 &  99 &   1 \\\hline
$\geq 13$ &   3746 &  72978 &    1 &  1.000 &    1 &    1 & 0.000000 & 0.000000 & 0.000000 & 100 &   0 \\\hline
$\geq 12$ &   4160 &  77882 &    1 &  1.000 &    1 &    1 & 0.000000 & 0.000000 & 0.000000 & 100 &   0 \\\hline
$\geq 11$ &   4634 &  83039 &    1 &  1.010 &    1 &    2 & 0.002830 & 0.000000 & 0.282966 &  99 &   1 \\\hline
$\geq 10$ &   5182 &  88462 &    1 &  1.000 &    1 &    1 & 0.000000 & 0.000000 & 0.000000 & 100 &   0 \\\hline
$\geq  9$ &   5883 &  94709 &    1 &  1.020 &    1 &    3 & 0.002928 & 0.000000 & 0.292798 &  99 &   1 \\\hline
$\geq  8$ &   6750 & 101564 &    1 &  1.000 &    1 &    1 & 0.000000 & 0.000000 & 0.000000 & 100 &   0 \\\hline
$\geq  7$ &   7904 & 109561 &    1 &  1.020 &    1 &    3 & 0.002960 & 0.000000 & 0.295950 &  99 &   1 \\\hline
$\geq  6$ &   9392 & 118389 &    1 &  1.000 &    1 &    1 & 0.000000 & 0.000000 & 0.000000 & 100 &   0 \\\hline
$\geq  5$ &  11483 & 128731 &    1 &  1.000 &    1 &    1 & 0.000000 & 0.000000 & 0.000000 & 100 &   0 \\\hline
$\geq  4$ &  14864 & 142112 &    1 &  1.120 &    1 &    2 & 0.000018 & 0.000000 & 0.000211 &  88 &  12 \\\hline
$\geq  3$ &  21812 & 162691 &    1 &  2.630 &    1 &    4 & 0.006172 & 0.000000 & 0.315545 &   5 &  95 \\\hline
$\geq  2$ &  41659 & 201678 &    4 & 12.710 &    9 &   19 & 0.009882 & 0.000654 & 0.328871 &   0 & 100 \\\hline
\end{tabular}
\end{center}
\end{table}

\subsubsection{Label Propagation: Communities in the Inner-Most Core}
The entire core is one big community. This is the result in all of our runs.

\subsection{InfoMAP}
The algorithm used is the one implemented in \igraph 
which is based on \citep{communities:infoMAP};
see also \citep{communities:infoMAP:additional}.
Table \ref{tbl:communities:positive:infoMAP:cores} presents the results when we apply the algorithm for 
subgraphs in which we include vertices with successively lower coreness
and allow edges with positive polarity only.
The algorithm can exhibit in some cases wild variations both in terms of the computed communities
as well as of the induced modularity.

\begin{table}[ht]
\caption{Applying the InfoMAP algorithm for community finding implemented in \igraph by successively including
vertices with lower coreness on the undirected graph induced by the assertions with positive polarity (self-loops are removed).
In every row we have the number of vertices and the number of edges of each such subgraph together with
the number of components ($\abs{C}$) that we find in that subgraph. 
The next three columns present the number of communities found by the algorithm;
the average among all runs, the minimum, and the maximum.
The next three columns present the codelength of the partitioning found by the algorithm;
the average among all runs, the minimum, and the maximum.
The next three columns present the modularity achieved by the algorithm due to the cut induced by the communities;
the average among all runs, the minimum, and the maximum.
Finally the last two columns
present in how many runs the algorithm computed as many communities as we had components in that subgraph.
The entire computation lasted $13,376.27$ seconds for $10$ runs; 
that is about $1,337.63$ seconds per run.}\label{tbl:communities:positive:infoMAP:cores}
\begin{center}
\resizebox{\textwidth}{!}{
\begin{tabular}{|r|r|r|c||r|r|r||r|r|r||r|r|r||r|r|}\cline{14-15}
\multicolumn{13}{c|}{} & \multicolumn{2}{c|}{agree-} \\\cline{1-13}
$\uppi$ & \multicolumn{1}{c|}{\multirow{2}{*}{$\abs{V}$}} & \multicolumn{1}{c|}{\multirow{2}{*}{$\abs{E}$}} & \multirow{2}{*}{$\abs{C}$} & \multicolumn{3}{c||}{communities found} & \multicolumn{3}{c||}{codelength} & \multicolumn{3}{c||}{modularity} & \multicolumn{2}{c|}{ment} \\\cline{5-15}
$\geq$ &      &        &   & \multicolumn{1}{c|}{avg} & \multicolumn{1}{c|}{min} & \multicolumn{1}{c||}{max} & \multicolumn{1}{c|}{avg} & \multicolumn{1}{c|}{min} & \multicolumn{1}{c||}{max} & \multicolumn{1}{c|}{avg} & \multicolumn{1}{c|}{min} & \multicolumn{1}{c||}{max} & \multicolumn{1}{c|}{Y} & \multicolumn{1}{c|}{N} \\\hline\hline
$26$ &    869 &  20526 &    1 &    1.000 &    1 &    1 &  9.595067 &  9.595067 &  9.595067 & 0.000000 & 0.000000 & 0.000000 &  10 &   0 \\\hline
$25$ &   1167 &  27810 &    1 &    1.000 &    1 &    1 &  9.997333 &  9.997333 &  9.997333 & 0.000000 & 0.000000 & 0.000000 &  10 &   0 \\\hline
$24$ &   1358 &  32314 &    1 &    1.000 &    1 &    1 & 10.202537 & 10.202537 & 10.202537 & 0.000000 & 0.000000 & 0.000000 &  10 &   0 \\\hline
$23$ &   1514 &  35870 &    1 &    1.000 &    1 &    1 & 10.348016 & 10.348016 & 10.348016 & 0.000000 & 0.000000 & 0.000000 &  10 &   0 \\\hline
$22$ &   1709 &  40099 &    1 &    1.000 &    1 &    1 & 10.510817 & 10.510817 & 10.510817 & 0.000000 & 0.000000 & 0.000000 &  10 &   0 \\\hline
$21$ &   1865 &  43330 &    1 &    1.000 &    1 &    1 & 10.625889 & 10.625889 & 10.625889 & 0.000000 & 0.000000 & 0.000000 &  10 &   0 \\\hline
$20$ &   2007 &  46145 &    1 &    1.000 &    1 &    1 & 10.721326 & 10.721326 & 10.721326 & 0.000000 & 0.000000 & 0.000000 &  10 &   0 \\\hline
$19$ &   2173 &  49265 &    1 &    1.000 &    1 &    1 & 10.824262 & 10.824262 & 10.824262 & 0.000000 & 0.000000 & 0.000000 &  10 &   0 \\\hline
$18$ &   2384 &  53011 &    1 &    1.000 &    1 &    1 & 10.943376 & 10.943376 & 10.943376 & 0.000000 & 0.000000 & 0.000000 &  10 &   0 \\\hline
$17$ &   2617 &  56939 &    1 &    2.000 &    2 &    2 & 11.057997 & 11.057997 & 11.057997 & 0.009477 & 0.009477 & 0.009477 &   0 &  10 \\\hline
$16$ &   2847 &  60583 &    1 &    2.000 &    2 &    2 & 11.158651 & 11.158651 & 11.158651 & 0.011783 & 0.011783 & 0.011783 &   0 &  10 \\\hline
$15$ &   3105 &  64412 &    1 &    2.000 &    2 &    2 & 11.264781 & 11.264781 & 11.264781 & 0.011999 & 0.011999 & 0.011999 &   0 &  10 \\\hline
$14$ &   3407 &  68613 &    1 &    3.200 &    2 &   14 & 11.384844 & 11.374820 & 11.475059 & 0.038630 & 0.012607 & 0.272836 &   0 &  10 \\\hline
$13$ &   3746 &  72978 &    1 &    2.000 &    2 &    2 & 11.484748 & 11.484741 & 11.484808 & 0.013978 & 0.013944 & 0.014283 &   0 &  10 \\\hline
$12$ &   4160 &  77882 &    1 &    3.300 &    2 &    9 & 11.615842 & 11.607225 & 11.650572 & 0.041794 & 0.014052 & 0.153033 &   0 &  10 \\\hline
$11$ &   4634 &  83039 &    1 &   18.000 &    2 &   34 & 11.772474 & 11.733024 & 11.800385 & 0.187870 & 0.013987 & 0.338438 &   0 &  10 \\\hline
$10$ &   5182 &  88462 &    1 &   34.500 &    3 &   46 & 11.885121 & 11.859279 & 11.894615 & 0.269445 & 0.014270 & 0.345652 &   0 &  10 \\\hline
$ 9$ &   5883 &  94709 &    1 &   55.900 &   51 &   61 & 11.997908 & 11.994686 & 12.002033 & 0.347252 & 0.341475 & 0.354718 &   0 &  10 \\\hline
$ 8$ &   6750 & 101564 &    1 &   81.100 &   73 &   92 & 12.111098 & 12.107090 & 12.114428 & 0.345340 & 0.335496 & 0.356027 &   0 &  10 \\\hline
$ 7$ &   7904 & 109561 &    1 &  117.100 &  107 &  123 & 12.231296 & 12.228096 & 12.236937 & 0.348036 & 0.342610 & 0.353833 &   0 &  10 \\\hline
$ 6$ &   9392 & 118389 &    1 &  166.800 &  161 &  177 & 12.344162 & 12.338383 & 12.349301 & 0.349823 & 0.346514 & 0.353298 &   0 &  10 \\\hline
$ 5$ &  11483 & 128731 &    1 &  251.600 &  241 &  260 & 12.461265 & 12.456851 & 12.465886 & 0.352773 & 0.347126 & 0.357310 &   0 &  10 \\\hline
$ 4$ &  14864 & 142112 &    1 &  407.900 &  392 &  424 & 12.580570 & 12.575446 & 12.583878 & 0.350158 & 0.343223 & 0.355421 &   0 &  10 \\\hline
$ 3$ &  21812 & 162691 &    1 &  747.500 &  738 &  757 & 12.679398 & 12.673615 & 12.682357 & 0.348008 & 0.341139 & 0.353773 &   0 &  10 \\\hline
$ 2$ &  41659 & 201678 &    4 & 1753.500 & 1734 & 1769 & 12.602572 & 12.598905 & 12.604815 & 0.351349 & 0.347792 & 0.355575 &   0 &  10 \\\hline
\end{tabular}
}
\end{center}
\end{table}

\subsubsection{InfoMAP: Communities in the Inner-Most Core}
The entire core is one big community. This is the result in all of our runs.

\chapter{Overlapping Communities}\label{ch:communities:overlapping}
Here we examine communities that were obtained with \cfinder \citep{cfinder}.

\section{Negative Polarity}
In this section we examine some communities found in the graph with negative polarity
by percolating cliques.

\paragraph{Percolating Cliques.}
Table \ref{tbl:percolate:negative:distribution:communities} presents how many communities were found
by percolating cliques of different sizes, while Table \ref{tbl:percolate:negative:distribution:community-sizes}
presents the distribution of the community sizes by percolating cliques of different sizes.
Table \ref{tbl:percolate:negative:distribution:membership} presents the distribution of the concepts
participating in different communities by percolating cliques of different sizes.
Figure \ref{fig:communities:percolate:negative} gives some examples of communities obtained
by percolating cliques of size $3$ and $4$.

\begin{table}[ht]
\caption{Number of communities found in the undirected graph with negative polarity 
with \cfinder by percolating cliques of certain size.}\label{tbl:percolate:negative:distribution:communities}
\centering
\begin{tabular}{|r|r|r|}\hline
clique size &   3 &  4 \\\hline
communities & 126 & 24 \\\hline
\end{tabular}
\end{table}

\begin{table}[ht]
\caption{Distribution of community sizes found in the undirected graph induced by the assertions
of the English language with positive score and negative polarity by percolating cliques of different
sizes.}\label{tbl:percolate:negative:distribution:community-sizes}
\centering
\begin{tabular}{|c||r|r|r|r|r|r|r|r|r|r|}\hline
percolating     & \multicolumn{10}{c|}{community size} \\\cline{2-11}
cliques of size & \multicolumn{1}{c|}{3} & \multicolumn{1}{c|}{4} & \multicolumn{1}{c|}{5} & \multicolumn{1}{c|}{6} & \multicolumn{1}{c|}{7} & \multicolumn{1}{c|}{9} & \multicolumn{1}{c|}{10} & \multicolumn{1}{c|}{11} & \multicolumn{1}{c|}{18} & \multicolumn{1}{c|}{457} \\\hline
3               & 75 & 26 & 14 & 6 & 2 &  1 & -- &  1 & -- &   1 \\\hline
4               & -- & 16 &  3 & 2 & 1 & -- &  1 & -- &  1 &  -- \\\hline
\end{tabular}
\end{table}

\begin{table}[ht]
\caption{Distribution of concepts participating in different communities in the undirected graph
induced by the assertions of the English language with positive score and negative polarity
by percolating cliques of different sizes.}\label{tbl:percolate:negative:distribution:membership}
\centering
\begin{tabular}{|c||r|r|r|r|r|r|r|}\hline
percolating     & \multicolumn{7}{c|}{number of communities} \\\cline{2-8}
cliques of size & \multicolumn{1}{c|}{0} & \multicolumn{1}{c|}{1} & \multicolumn{1}{c|}{2} & \multicolumn{1}{c|}{3} & \multicolumn{1}{c|}{4} & \multicolumn{1}{c|}{10} & \multicolumn{1}{c|}{13} \\\hline
3               & 10,898 & 716 & 79 & 10 &  3 &  1 & -- \\\hline
4               & 11,603 &  94 &  8 &  1 & -- & -- &  1 \\\hline
\end{tabular}
\end{table}

\begin{figure}[p]
\centering
\begin{subfigure}[b]{0.4\textwidth}
\includegraphics[width=\columnwidth]{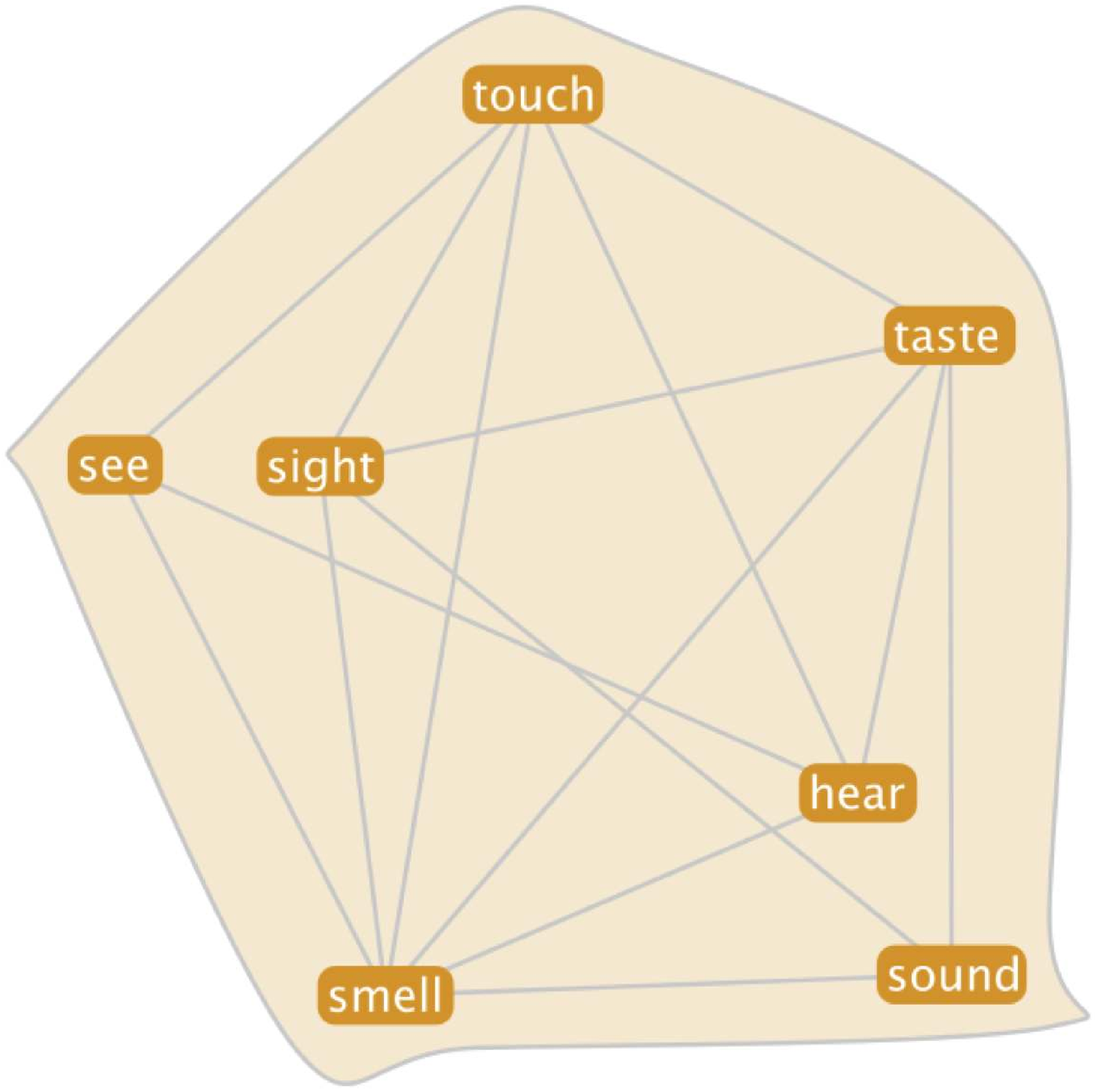}
\caption{Percolating senses. Seven nodes by percolating cliques of size $4$.
A link is missing between \dbtext{sight} and \dbtext{hear} generating a clique of size $5$.}\label{fig:communities:percolate:negative:senses}
\end{subfigure}
\hspace{\fill}
\begin{subfigure}[b]{0.4\textwidth}
\includegraphics[width=\columnwidth]{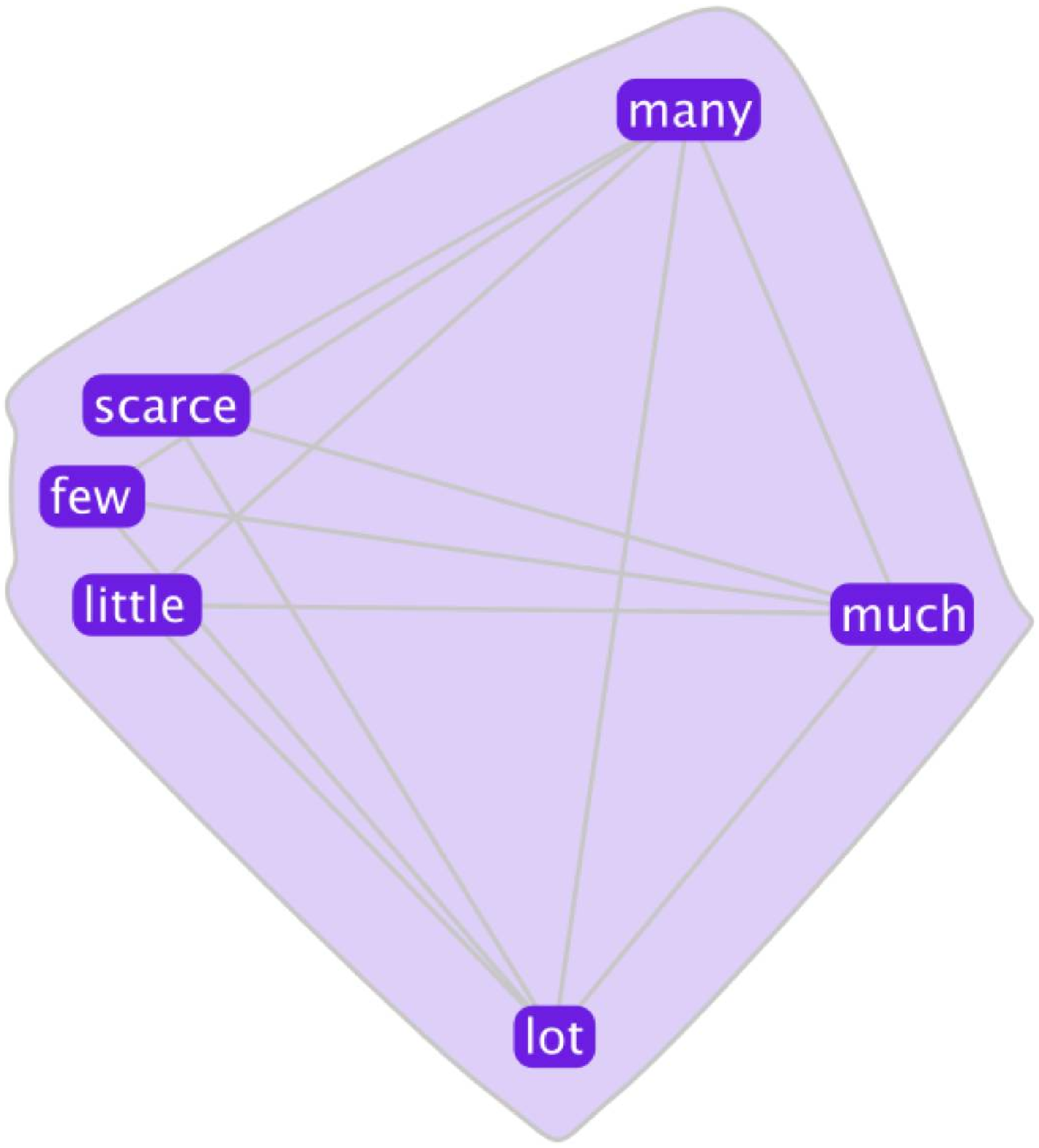}
\caption{Percolating frequent. Six nodes by percolating cliques of size $4$.}\label{fig:communities:percolate:negative:frequent}
\end{subfigure}

\begin{subfigure}[b]{0.4\textwidth}
\includegraphics[width=\columnwidth]{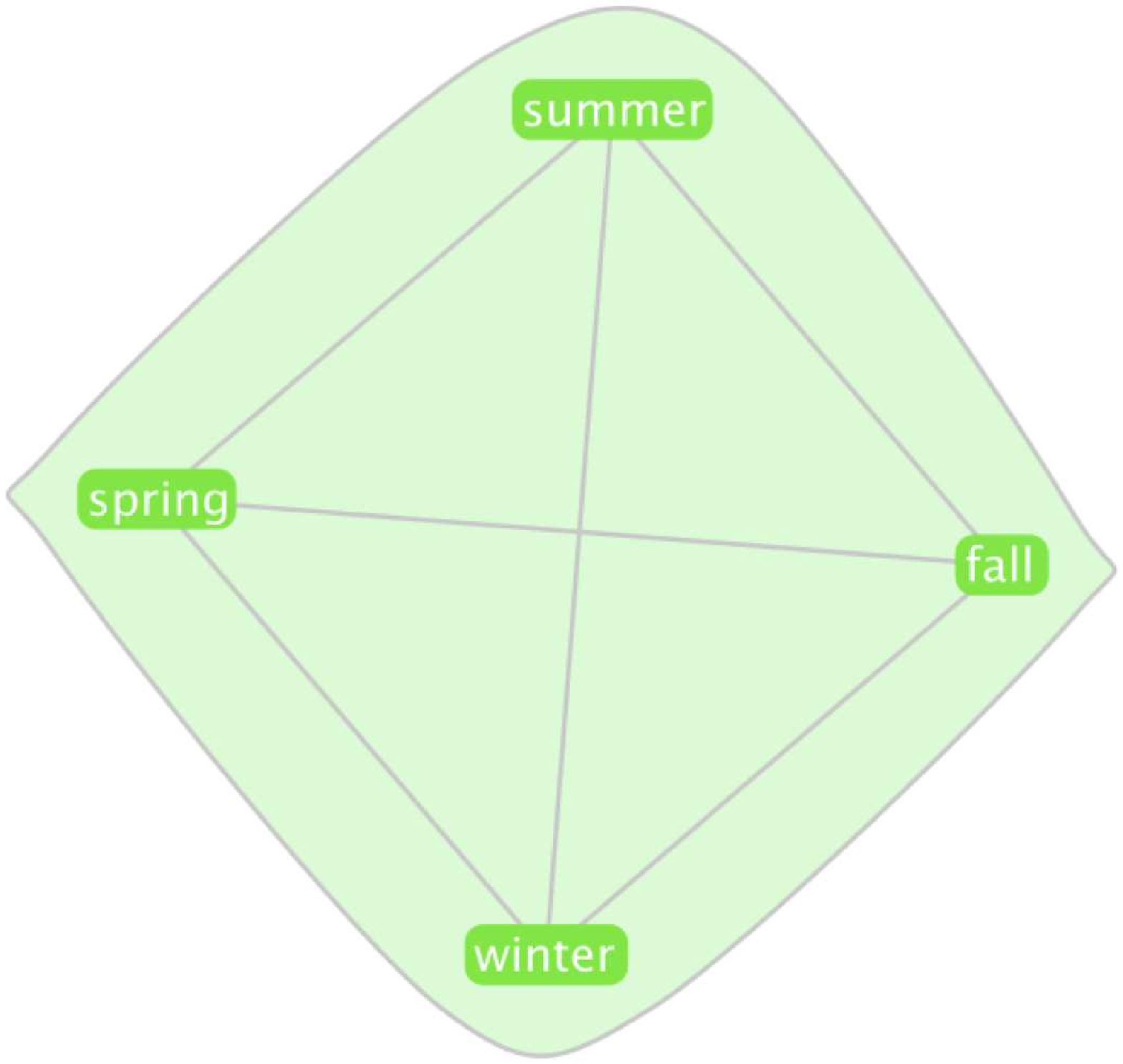}
\caption{Percolating year. Four nodes by percolating cliques of size $4$.}\label{fig:communities:percolate:negative:year}
\end{subfigure}
\hspace{\fill}
\begin{subfigure}[b]{0.4\textwidth}
\includegraphics[width=\columnwidth]{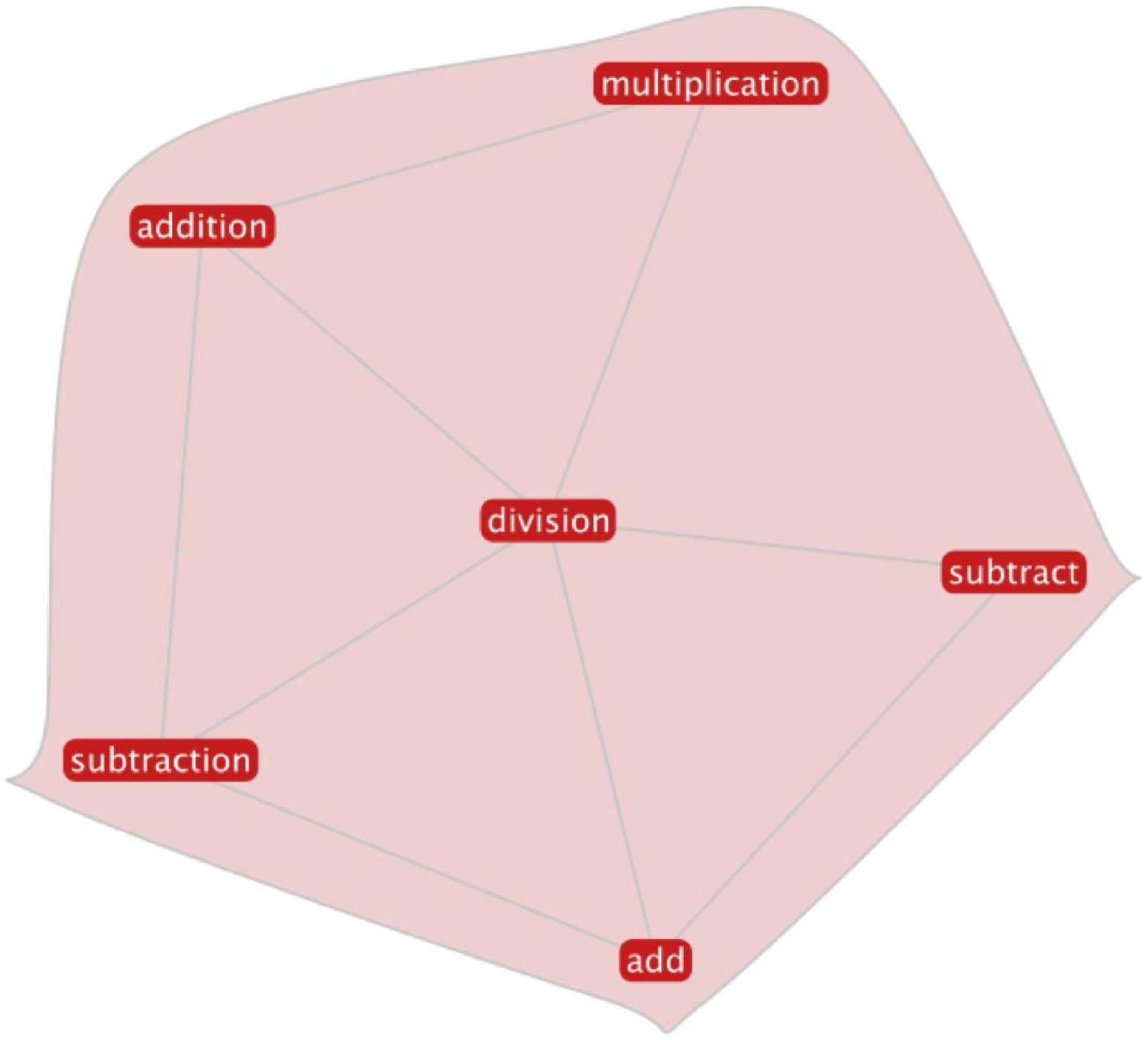}
\caption{Percolating arithmetic. Six nodes by percolating cliques of size $3$.}\label{fig:communities:percolate:negative:arithmetic}
\end{subfigure}

\begin{subfigure}[b]{0.45\textwidth}
\includegraphics[width=\columnwidth]{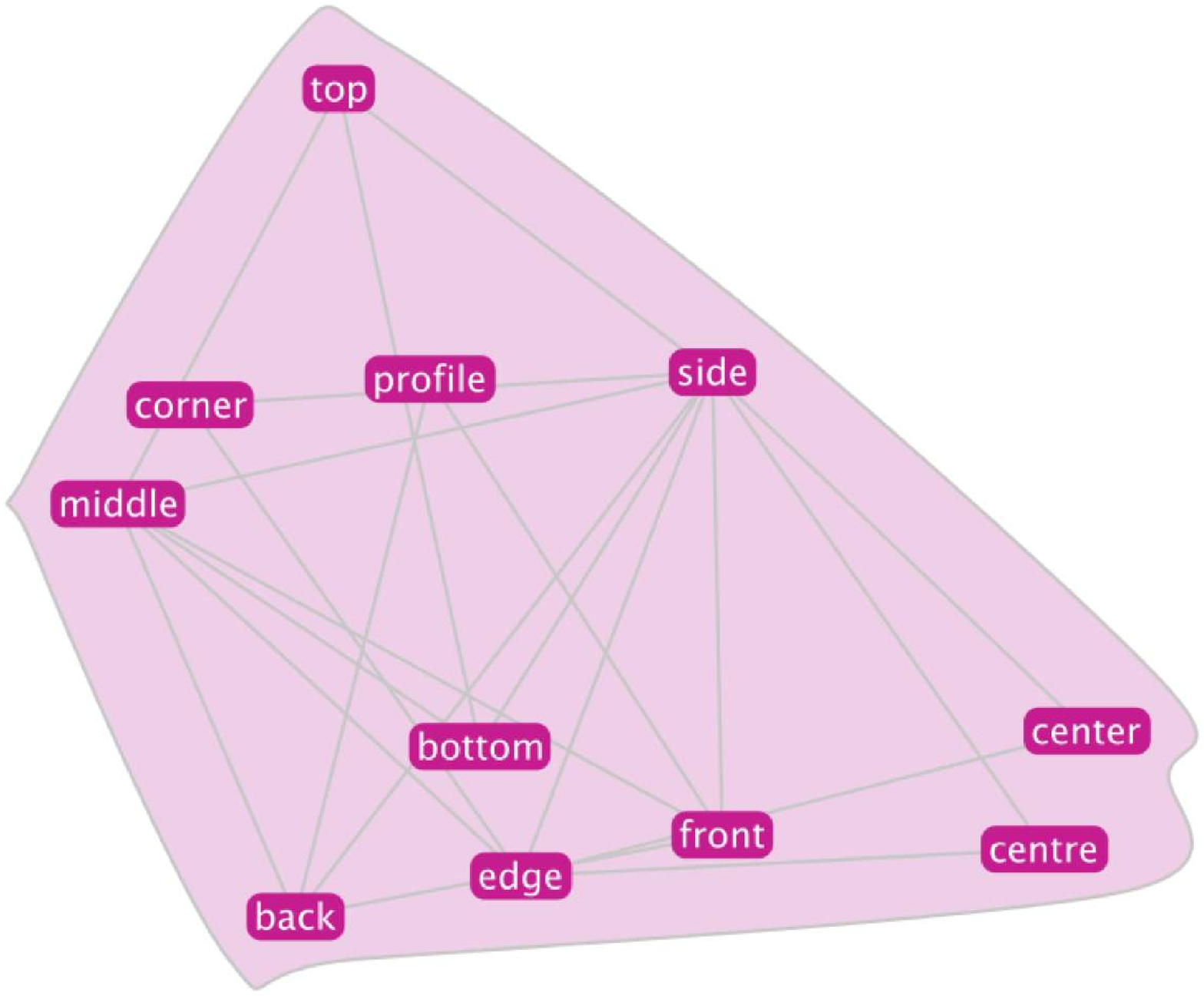}
\caption{Percolating orientation. Eleven nodes by percolating cliques of size $3$.
The concept \dbtext{profile} appears here but not in 
Figure \ref{fig:strongly-connected:negative:12}.}\label{fig:communities:percolate:negative:orientation}
\end{subfigure}
\hspace{\fill}
\begin{subfigure}[b]{0.45\textwidth}
\includegraphics[width=\columnwidth]{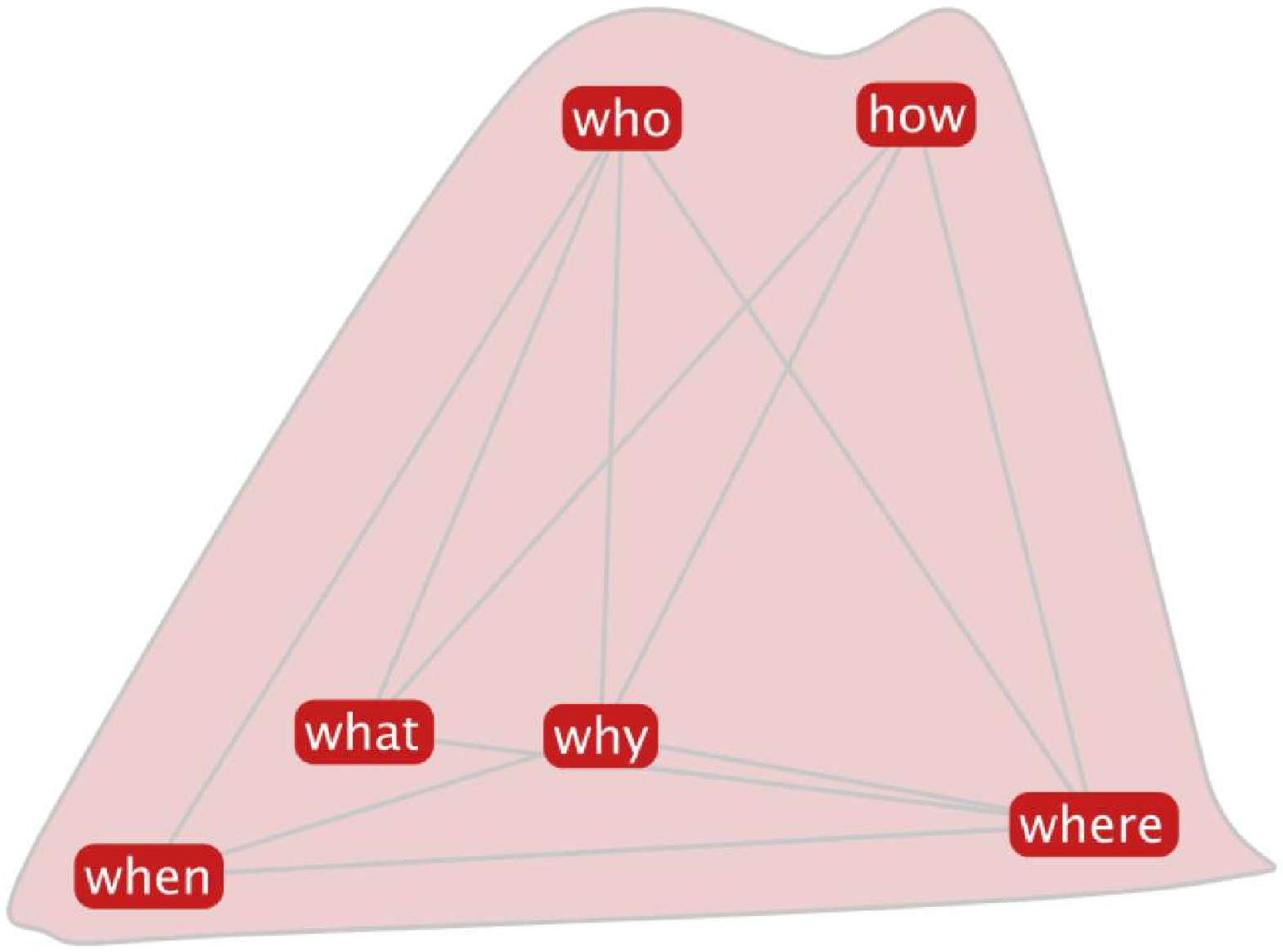}
\caption{Percolating questions. Six nodes by percolating cliques of size $3$.}\label{fig:communities:percolate:negative:questions}
\end{subfigure}
\caption{Instances of communities that are generated by percolating cliques of size $3$ and $4$.}\label{fig:communities:percolate:negative}
\end{figure}

\suppressfloats[t]

\paragraph{Overlapping Cliques.}
Overlapping cliques might prove useful in the future.
They can be used for example for further clarification when posing or processing questions.
We might be able to use them in order to isolate lower degree concepts related to specific
questions which in turn might help by contributing in a spreading activation process.
Figures \ref{fig:communities:overlapping:negative:middle} and \ref{fig:communities:overlapping:negative:second}
give some examples of overlapping communities in the graph induced by the assertions with negative polarity
(and positive score).

\begin{figure}[ht]
\centering
\begin{subfigure}[b]{0.65\textwidth}
\includegraphics[width=\columnwidth]{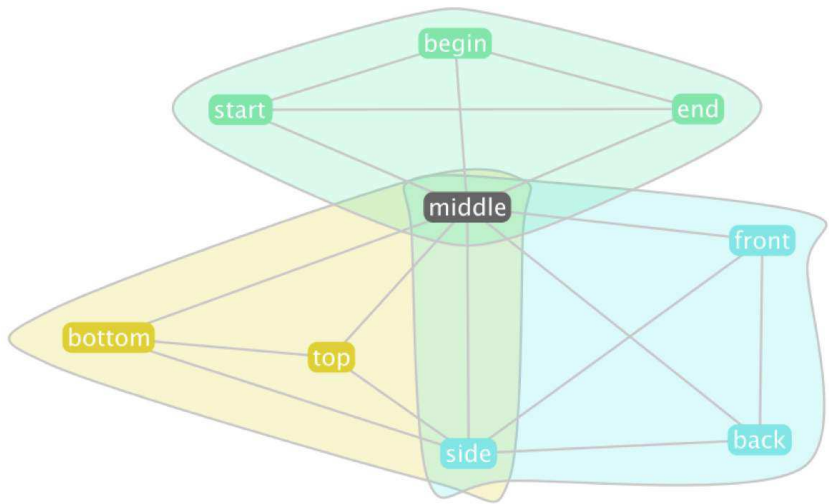}
\caption{Overlapping middle.}\label{fig:communities:overlapping:negative:middle}
\end{subfigure}
\hspace{\fill}
\begin{subfigure}[b]{0.65\textwidth}
\includegraphics[width=\columnwidth]{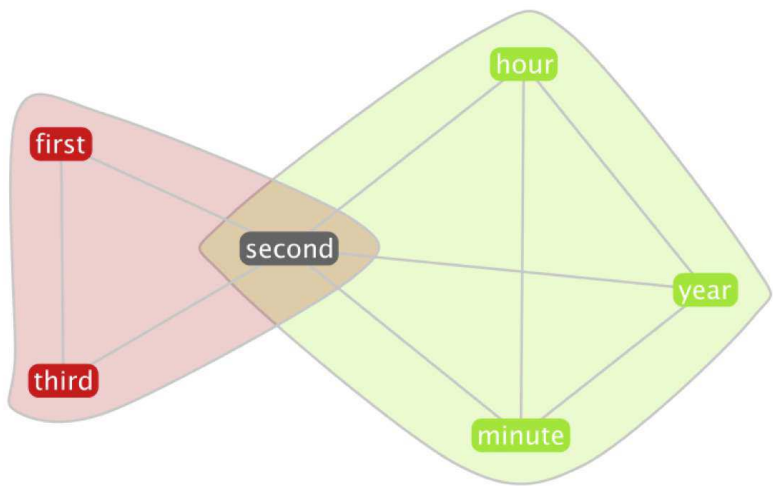}
\caption{Overlapping second.}\label{fig:communities:overlapping:negative:second}
\end{subfigure}
\caption{Ovelapping communities; negative polarity.}\label{fig:communities:overlapping:negative}
\end{figure}

\section{Positive Polarity}
In this section we examine some communities found in the graph with positive polarity
by percolating cliques.

\paragraph{Percolating Cliques.}
Table \ref{tbl:percolate:positive:distribution:communities} presents how many communities were found
by percolating cliques of different sizes, while Table \ref{tbl:percolate:positive:distribution:community-sizes}
presents the distribution of the community sizes by percolating cliques of different sizes.

\begin{table}[ht]
\caption{Number of communities found in the undirected graph with positive polarity 
with \cfinder by percolating cliques of certain size.}\label{tbl:percolate:positive:distribution:communities}
\centering
\begin{tabular}{|r|r|r|r|r|r|r|r|r|r|r|}\hline
size  &   3 &   4 &   5 &   6 &   7 &  8 &  9 & 10 & 11 & 12 \\\hline
comm. & 362 & 290 & 287 & 209 & 120 & 84 & 16 & 12 &  6 &  1 \\\hline
\end{tabular}
\end{table}

\begin{table}[ht]
\caption{Distribution of community sizes found in the undirected graph induced by the assertions
of the English language with positive score and positive polarity by percolating cliques of different
sizes.}\label{tbl:percolate:positive:distribution:community-sizes}
\centering
\begin{tabular}{|r||r|r|r|r|r|r|r|r|r|r|}\hline
\multicolumn{1}{|c||}{community} & \multicolumn{10}{c|}{percolating cliques of size} \\\cline{2-11}
\multicolumn{1}{|c||}{size}      & \multicolumn{1}{c|}{3} & \multicolumn{1}{c|}{4} & \multicolumn{1}{c|}{5} & \multicolumn{1}{c|}{6} & \multicolumn{1}{c|}{7} & \multicolumn{1}{c|}{8} & \multicolumn{1}{c|}{9} & \multicolumn{1}{c|}{10} & \multicolumn{1}{c|}{11} & \multicolumn{1}{c|}{12} \\\hline\hline
        3 & 320 &  -- &  -- &  -- &  -- &  -- &  -- &  -- &  -- &  -- \\\hline
        4 &  31 & 236 &  -- &  -- &  -- &  -- &  -- &  -- &  -- &  -- \\\hline
        5 &   8 &  35 & 204 &  -- &  -- &  -- &  -- &  -- &  -- &  -- \\\hline
        6 &   1 &  10 &  43 & 122 &  -- &  -- &  -- &  -- &  -- &  -- \\\hline
        7 &  -- &   3 &  18 &  41 &  61 &  -- &  -- &  -- &  -- &  -- \\\hline
        8 &   1 &   2 &   9 &  21 &  22 &  38 &  -- &  -- &  -- &  -- \\\hline
        9 &  -- &   1 &   7 &   7 &   9 &  19 &   6 &  -- &  -- &  -- \\\hline
       10 &  -- &  -- &   1 &   5 &   5 &   8 &   2 &   3 &  -- &  -- \\\hline
       11 &  -- &  -- &  -- &   4 &   5 &   5 &   1 &   5 &   3 &  -- \\\hline
       12 &  -- &   1 &   1 &   3 &   3 &   3 &   1 &   1 &   1 &   1 \\\hline
       13 &  -- &  -- &   1 &  -- &   2 &   1 &   2 &  -- &  -- &  -- \\\hline
       14 &  -- &  -- &  -- &   3 &   5 &   4 &   1 &   1 &   1 &  -- \\\hline
       15 &  -- &  -- &  -- &   1 &  -- &  -- &  -- &   1 &  -- &  -- \\\hline
       16 &  -- &  -- &  -- &  -- &   2 &   1 &  -- &  -- &  -- &  -- \\\hline
       17 &  -- &  -- &   1 &  -- &  -- &   1 &   1 &  -- &  -- &  -- \\\hline
       18 &  -- &  -- &   1 &  -- &  -- &   2 &   1 &  -- &  -- &  -- \\\hline
       21 &  -- &   1 &  -- &  -- &  -- &  -- &  -- &  -- &  -- &  -- \\\hline
       22 &  -- &  -- &  -- &  -- &   1 &  -- &  -- &  -- &  -- &  -- \\\hline
       23 &  -- &  -- &  -- &  -- &   1 &  -- &  -- &  -- &  -- &  -- \\\hline
       24 &  -- &  -- &  -- &   1 &  -- &  -- &  -- &  -- &   1 &  -- \\\hline
       25 &  -- &  -- &  -- &  -- &   1 &  -- &  -- &  -- &  -- &  -- \\\hline
       37 &  -- &  -- &  -- &  -- &  -- &   1 &  -- &  -- &  -- &  -- \\\hline
       47 &  -- &  -- &  -- &  -- &   1 &  -- &  -- &  -- &  -- &  -- \\\hline
       49 &  -- &  -- &  -- &  -- &   1 &  -- &  -- &   1 &  -- &  -- \\\hline
      128 &  -- &  -- &  -- &  -- &  -- &  -- &   1 &  -- &  -- &  -- \\\hline
      278 &  -- &  -- &  -- &  -- &  -- &   1 &  -- &  -- &  -- &  -- \\\hline
      796 &  -- &  -- &  -- &  -- &   1 &  -- &  -- &  -- &  -- &  -- \\\hline
     1944 &  -- &  -- &  -- &   1 &  -- &  -- &  -- &  -- &  -- &  -- \\\hline
     3868 &  -- &  -- &   1 &  -- &  -- &  -- &  -- &  -- &  -- &  -- \\\hline
     8208 &  -- &   1 &  -- &  -- &  -- &  -- &  -- &  -- &  -- &  -- \\\hline
    22533 &   1 &  -- &  -- &  -- &  -- &  -- &  -- &  -- &  -- &  -- \\\hline
\end{tabular}
\end{table}

\begin{table}[ht]
\caption{Distribution of concepts participating in different communities in the undirected graph
induced by the assertions of the English language with positive score and positive polarity
by percolating cliques of different sizes.}\label{tbl:percolate:positive:distribution:membership}
\resizebox{\columnwidth}{!}{
\begin{tabular}{|r||r|r|r|r|r|r|r|r|r|r|}\hline
\multicolumn{1}{|c||}{number of}   & \multicolumn{10}{c|}{percolating cliques of size} \\\cline{2-11}
\multicolumn{1}{|c||}{communities} & \multicolumn{1}{c|}{3} & \multicolumn{1}{c|}{4} & \multicolumn{1}{c|}{5} & \multicolumn{1}{c|}{6} & \multicolumn{1}{c|}{7} & \multicolumn{1}{c|}{8} & \multicolumn{1}{c|}{9} & \multicolumn{1}{c|}{10} & \multicolumn{1}{c|}{11} & \multicolumn{1}{c|}{12} \\\hline\hline
  0 & 233,812 & 248,357 & 252,646 & 254,592 & 255,701 & 256,285 & 256,629 & 256,760 & 256,810 & 256,834 \\\hline
  1 &  22,451 &   7,727 &   3,433 &   1,691 &     805 &     361 &     162 &      50 &      16 &      12 \\\hline
  2 &     547 &     629 &     526 &     326 &     197 &      90 &      43 &      13 &       7 &      -- \\\hline
  3 &      29 &      94 &     159 &     140 &      70 &      59 &       5 &      12 &       4 &      -- \\\hline
  4 &       4 &      30 &      46 &      51 &      25 &      13 &       5 &       3 &       5 &      -- \\\hline
  5 &      -- &       5 &      13 &      22 &      21 &      12 &      -- &       3 &       3 &      -- \\\hline
  6 &       1 &       2 &      10 &       7 &       6 &      10 &      -- &       2 &       1 &      -- \\\hline
  7 &       1 &      -- &       4 &       5 &       5 &       3 &      -- &       2 &      -- &      -- \\\hline
  8 &      -- &       1 &       1 &       3 &       5 &       6 &       1 &      -- &      -- &      -- \\\hline
  9 &      -- &      -- &       3 &      -- &       4 &       1 &       1 &      -- &      -- &      -- \\\hline
 10 &       1 &      -- &       1 &      -- &       1 &      -- &      -- &       1 &      -- &      -- \\\hline
 11 &      -- &      -- &      -- &       2 &       2 &      -- &      -- &      -- &      -- &      -- \\\hline
 12 &      -- &      -- &       1 &       1 &       2 &       1 &      -- &      -- &      -- &      -- \\\hline
 13 &      -- &      -- &       1 &       2 &      -- &      -- &      -- &      -- &      -- &      -- \\\hline
 14 &      -- &      -- &      -- &      -- &      -- &       1 &      -- &      -- &      -- &      -- \\\hline
 16 &      -- &      -- &      -- &      -- &      -- &       1 &      -- &      -- &      -- &      -- \\\hline
 18 &      -- &      -- &      -- &       1 &      -- &       1 &      -- &      -- &      -- &      -- \\\hline
 19 &      -- &      -- &      -- &       1 &      -- &       1 &      -- &      -- &      -- &      -- \\\hline
 21 &      -- &      -- &       1 &      -- &      -- &      -- &      -- &      -- &      -- &      -- \\\hline
 24 &      -- &      -- &      -- &      -- &       1 &      -- &      -- &      -- &      -- &      -- \\\hline
 25 &      -- &      -- &      -- &       1 &      -- &      -- &      -- &      -- &      -- &      -- \\\hline
 34 &      -- &       1 &      -- &      -- &      -- &      -- &      -- &      -- &      -- &      -- \\\hline
 52 &      -- &      -- &      -- &      -- &      -- &       1 &      -- &      -- &      -- &      -- \\\hline
 74 &      -- &      -- &      -- &      -- &       1 &      -- &      -- &      -- &      -- &      -- \\\hline
 87 &      -- &      -- &       1 &      -- &      -- &      -- &      -- &      -- &      -- &      -- \\\hline
105 &      -- &      -- &      -- &       1 &      -- &      -- &      -- &      -- &      -- &      -- \\\hline
\end{tabular}
}
\end{table}

Figures \ref{fig:communities:percolate:a} and \ref{fig:communities:percolate:b}
present communities that occur by percolating cliques of various sizes. 
Note that in the case of Figure \ref{fig:communities:percolate:reproduction}
the concept \dbtext{boy} does make it and is part of the community as it would be expected
contrasting the fact that it does not appear in the relevant clique of size $11$
shown in Table \ref{tbl:cliques:positive:moderate-frequency}.
As another example, one would also expect the concept \dbtext{dishonest} or \dbtext{dishonesty}
to appear in the community shown in Figure \ref{fig:communities:percolate:dishonest}.
Moreover, through percolation we can get hints about missing edges.

\begin{figure}[p]
\centering
\begin{subfigure}[b]{0.48\textwidth}
\includegraphics[width=\columnwidth]{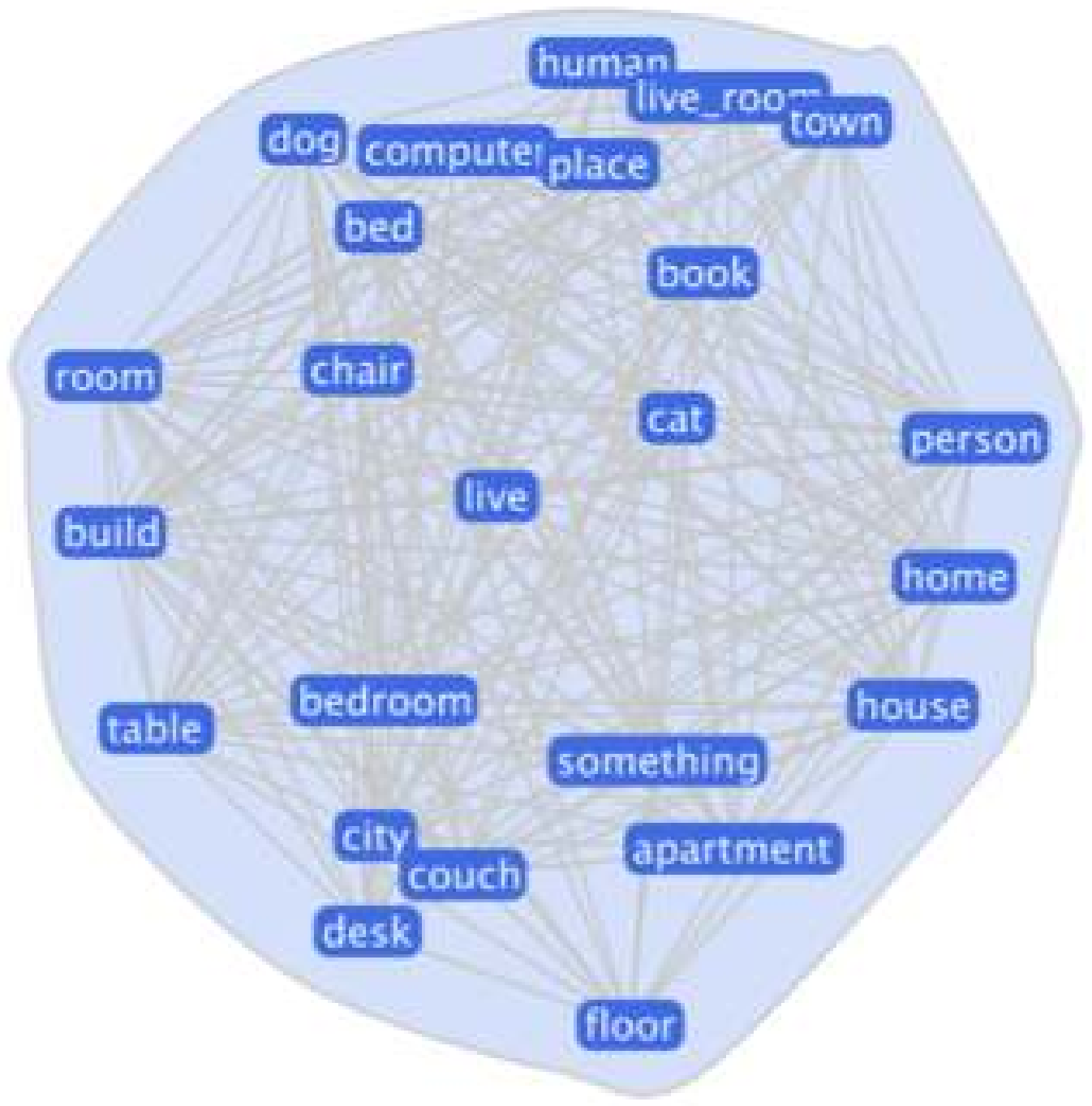}
\caption{Percolating house. Twenty four nodes by percolating cliques of size $11$.}\label{fig:communities:percolate:apartment}
\end{subfigure}
\hspace{\fill}
\begin{subfigure}[b]{0.48\textwidth}
\includegraphics[width=\columnwidth]{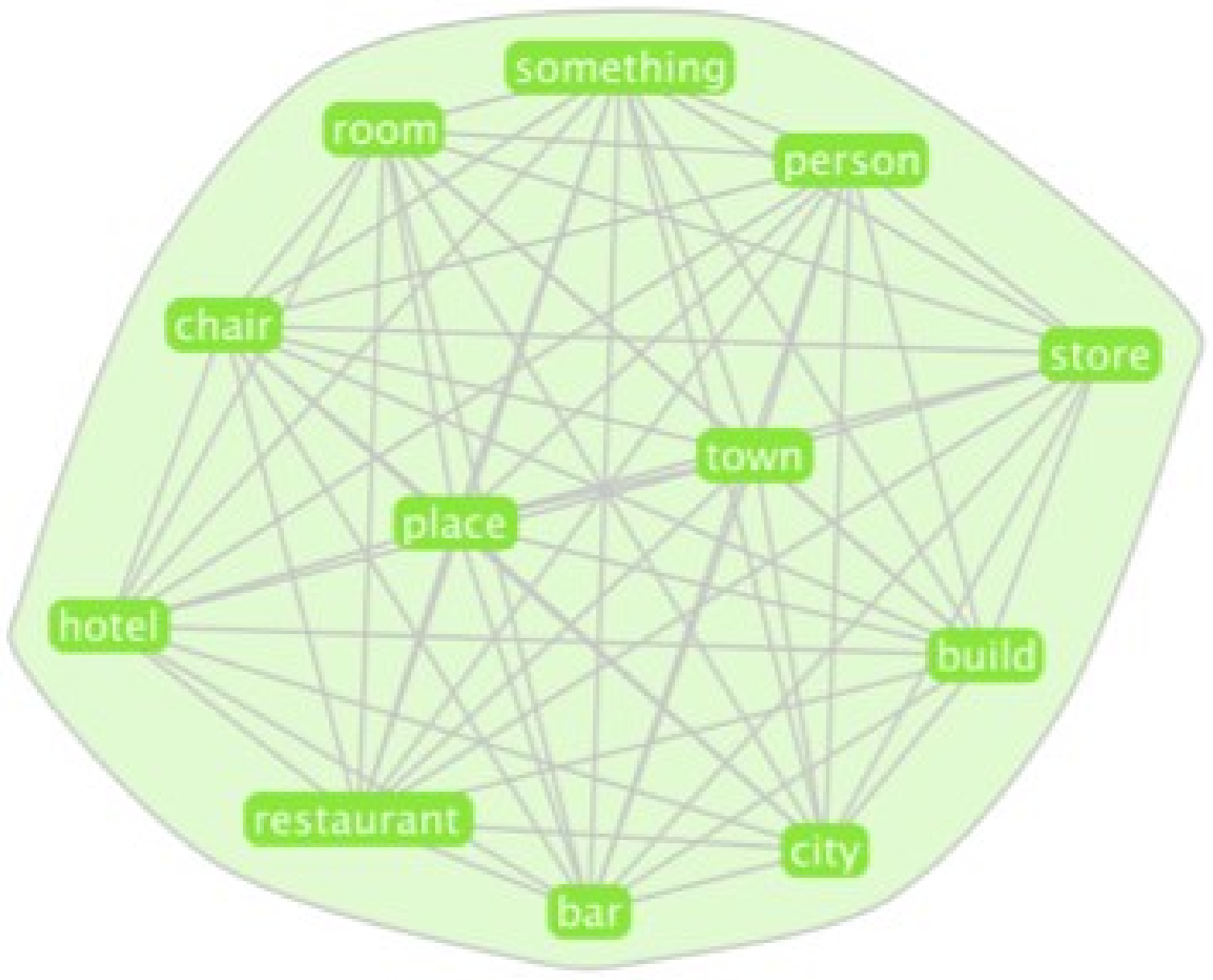}
\caption{Percolating neighborhood. Twelve nodes by percolating cliques of size $11$.}\label{fig:communities:percolate:neighborhood}
\end{subfigure}

\begin{subfigure}[b]{0.48\textwidth}
\includegraphics[width=\columnwidth]{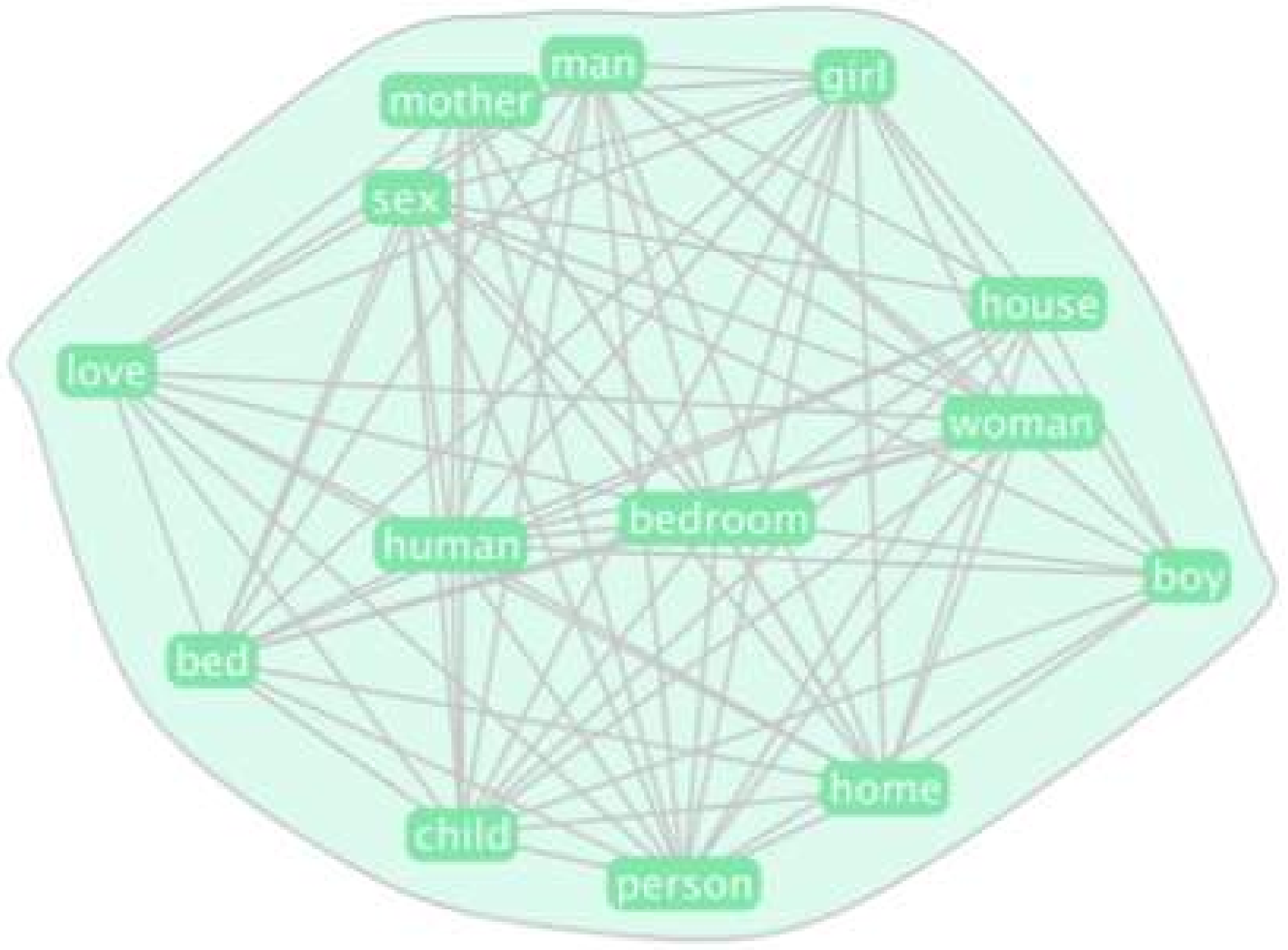}
\caption{Percolating human reproduction. Fourteen nodes by percolating cliques of size $10$.}\label{fig:communities:percolate:reproduction}
\end{subfigure}
\hspace{\fill}
\begin{subfigure}[b]{0.48\textwidth}
\includegraphics[width=\columnwidth]{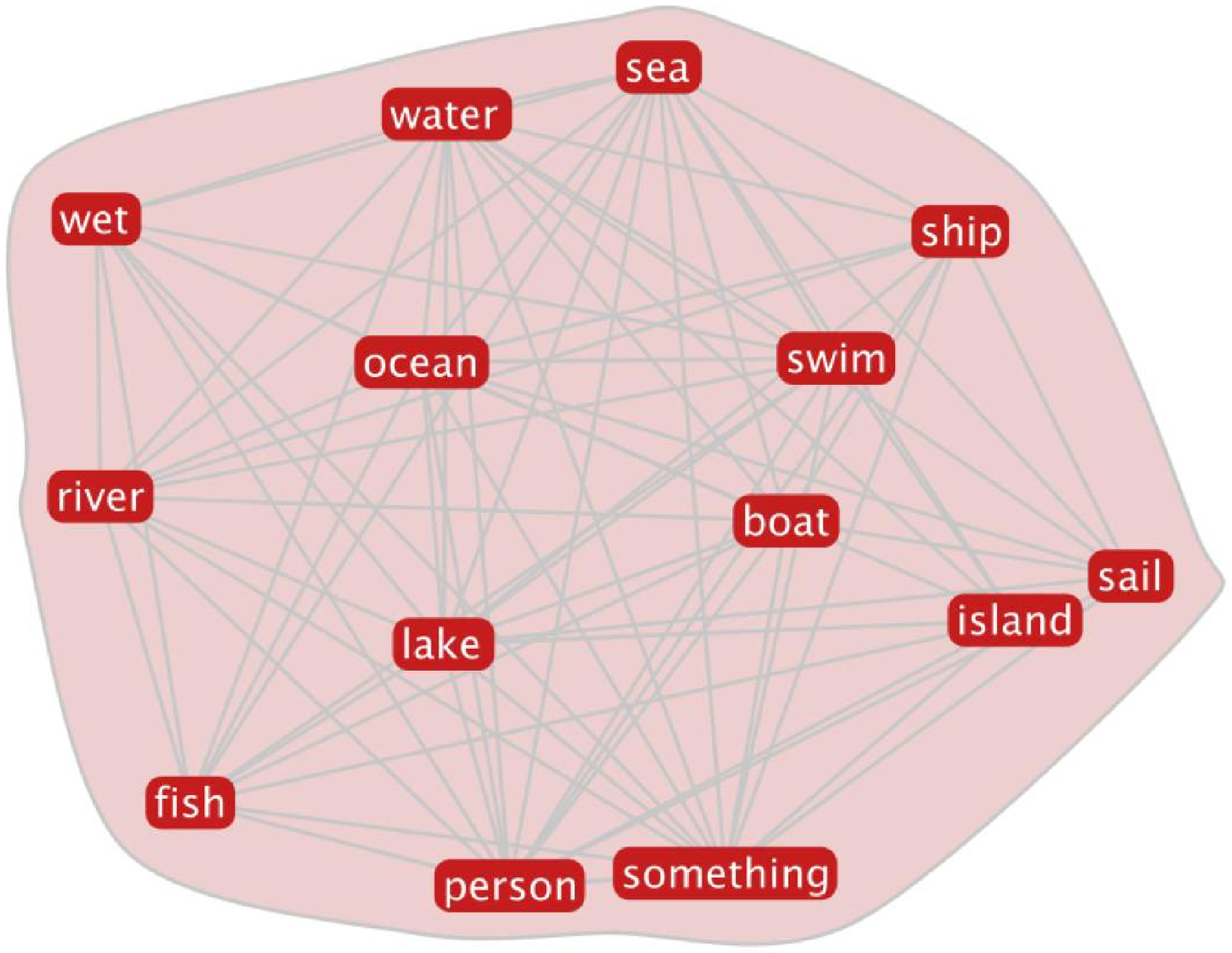}
\caption{Percolating sea. Fourteen nodes by percolating cliques of size $9$.}\label{fig:communities:percolate:sea}
\end{subfigure}

\begin{subfigure}[b]{0.48\textwidth}
\includegraphics[width=\columnwidth]{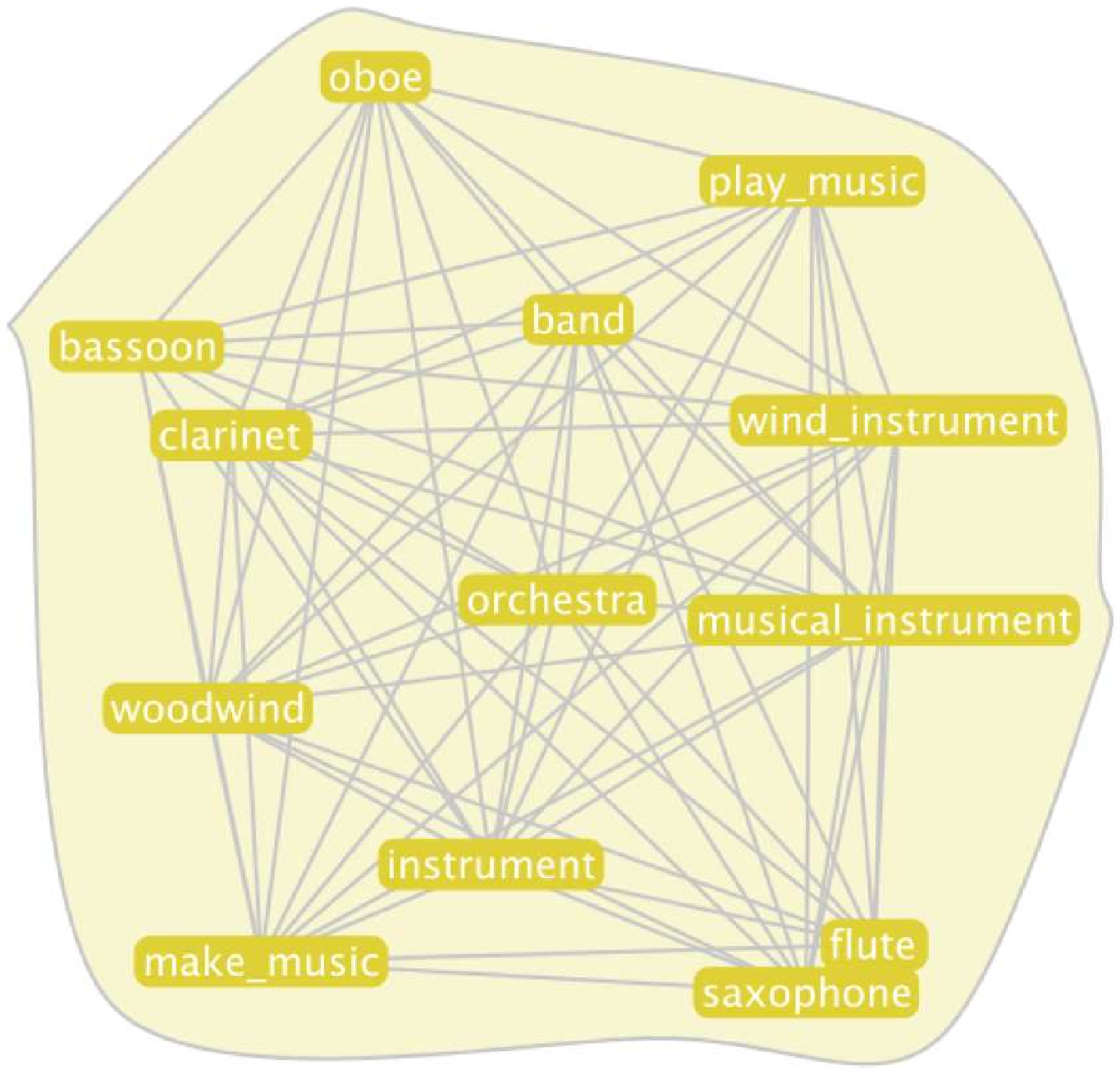}
\caption{Percolating music. Fourteen nodes by percolating cliques of size $9$.}\label{fig:communities:percolate:music:9}
\end{subfigure}
\hspace{\fill}
\begin{subfigure}[b]{0.48\textwidth}
\includegraphics[width=\columnwidth]{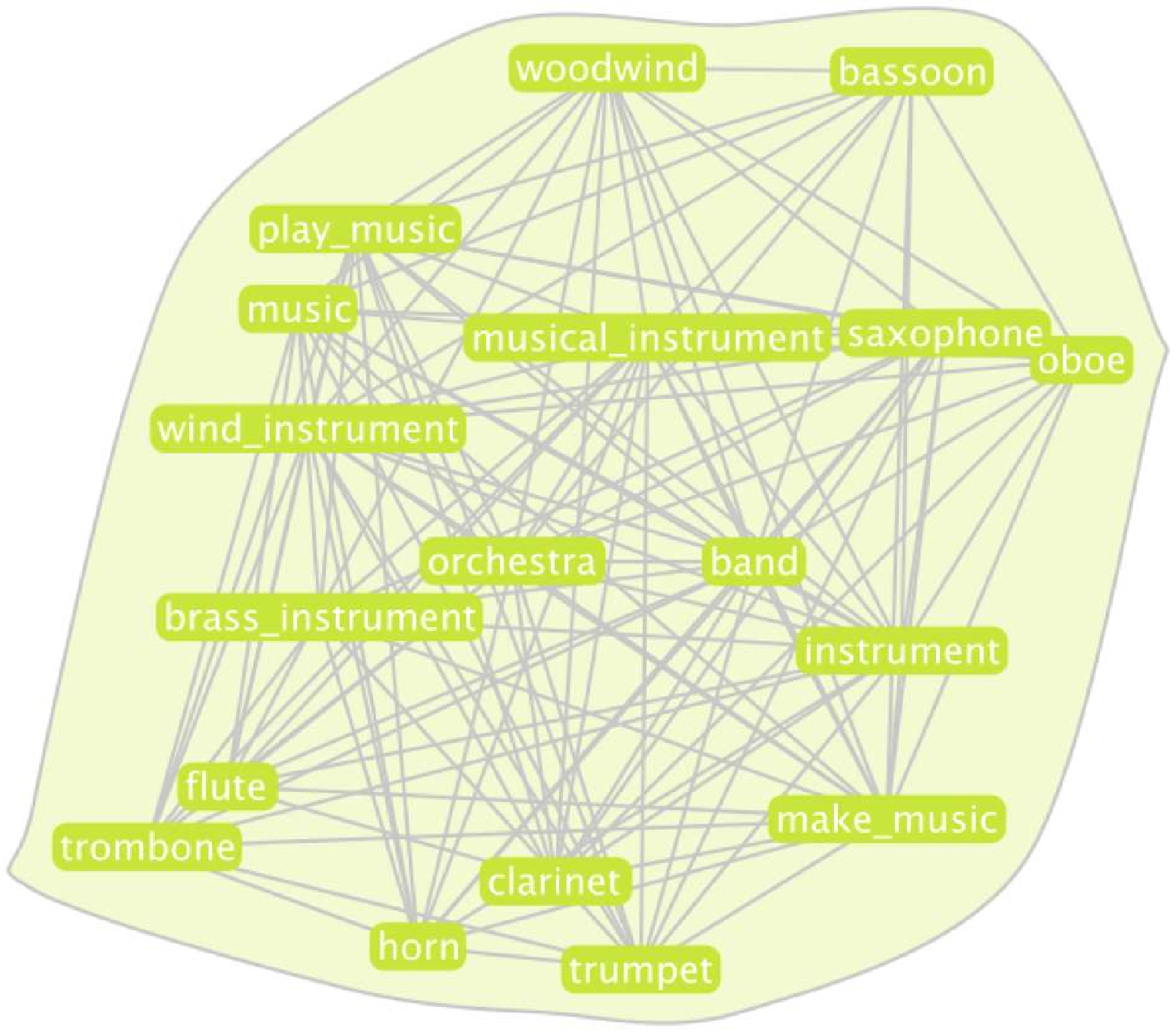}
\caption{Percolating music. Fourteen nodes by percolating cliques of size $8$.}\label{fig:communities:percolate:music:8}
\end{subfigure}
\caption{Percolating cliques of sizes $8, 9, 10$, and $11$ and some interesting communities.}\label{fig:communities:percolate:a}
\end{figure}

\begin{figure}[ht]
\centering
\begin{subfigure}[b]{0.48\textwidth}
\includegraphics[width=\columnwidth]{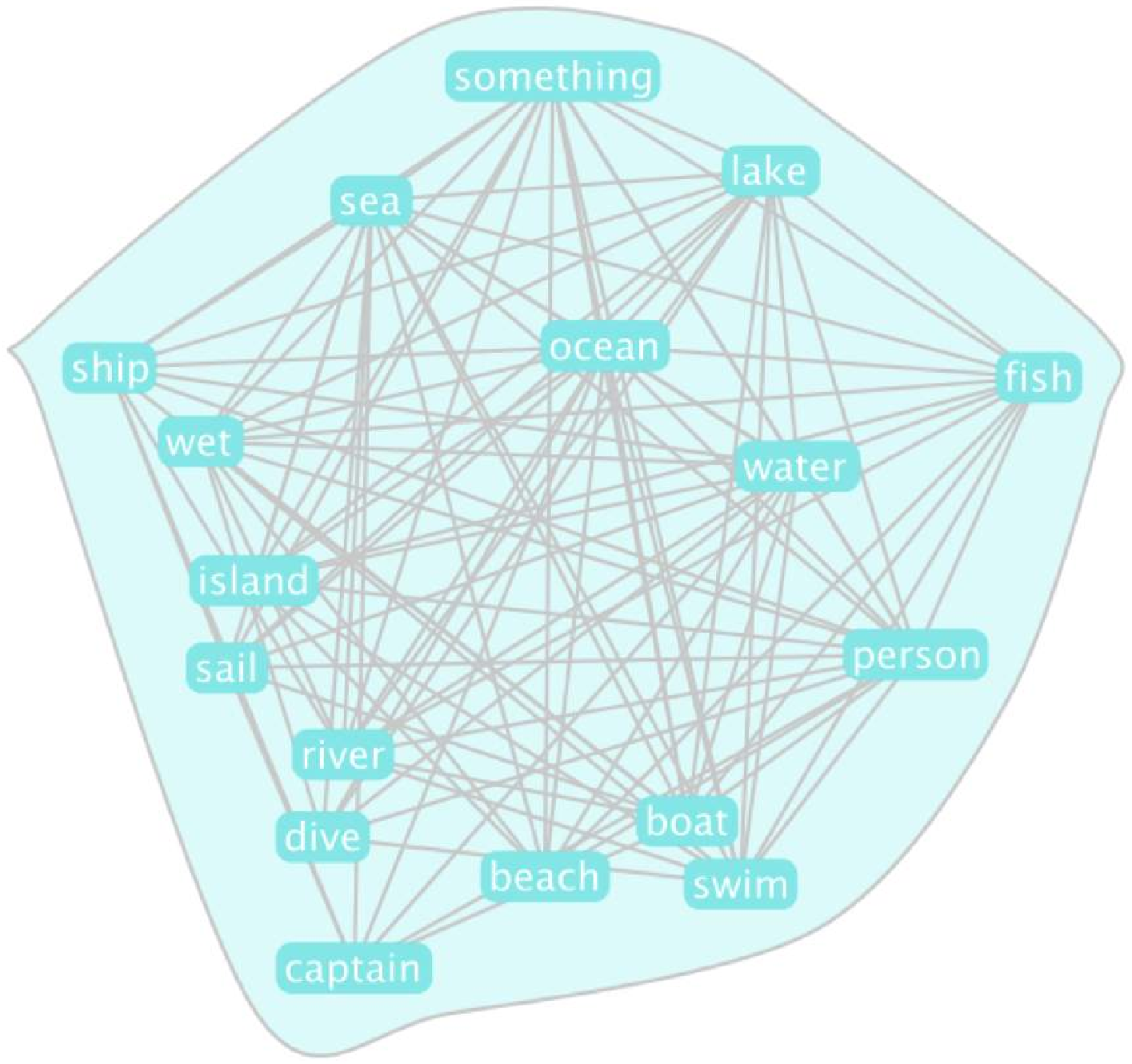}
\caption{Percolating water. Fourteen nodes by percolating cliques of size $8$.}\label{fig:communities:percolate:water}
\end{subfigure}
\hspace{\fill}
\begin{subfigure}[b]{0.48\textwidth}
\includegraphics[width=\columnwidth]{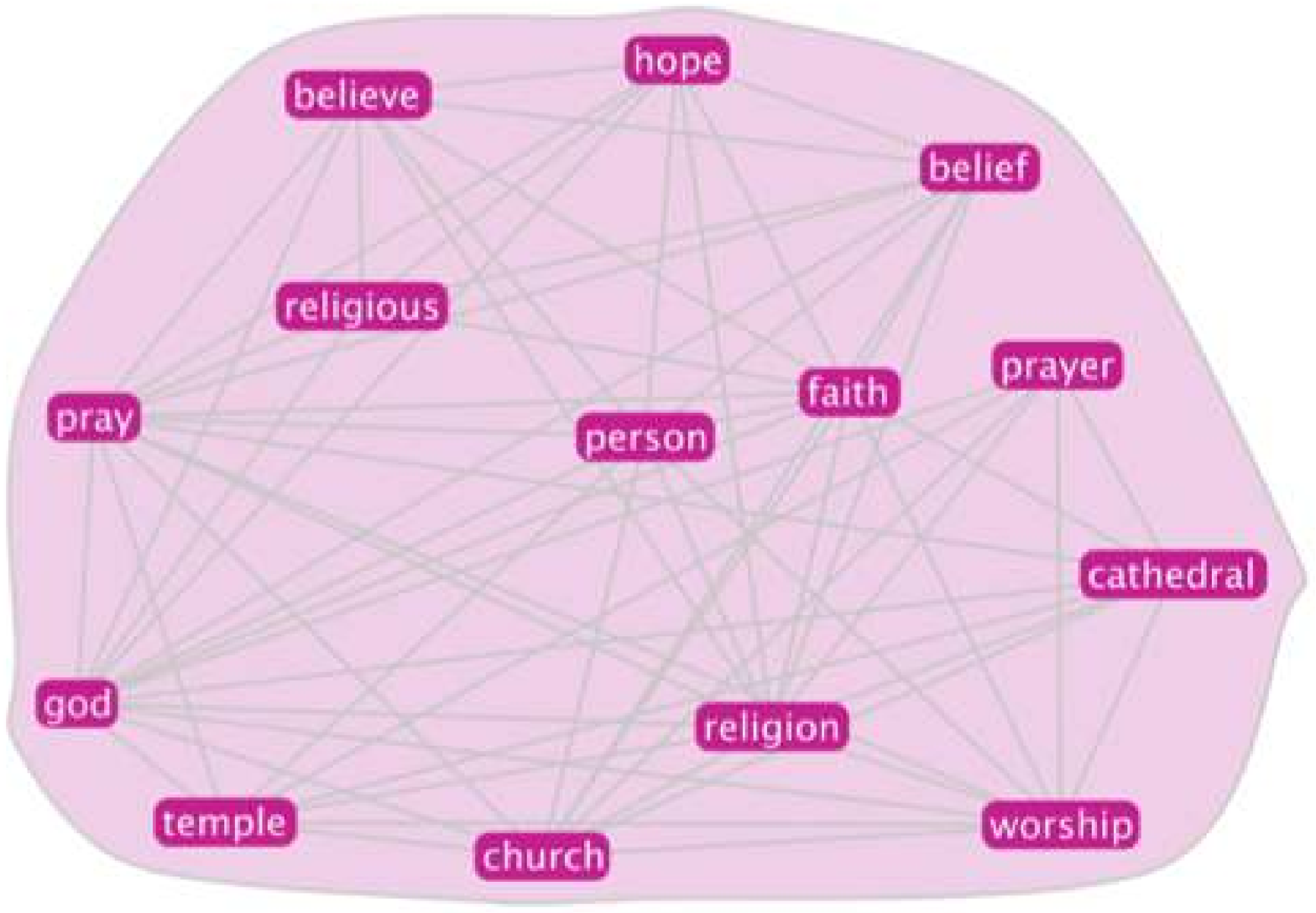}
\caption{Percolating religion. Fourteen nodes by percolating cliques of size $7$.}\label{fig:communities:percolate:religion}
\end{subfigure}

\begin{subfigure}[b]{0.48\textwidth}
\includegraphics[width=\columnwidth]{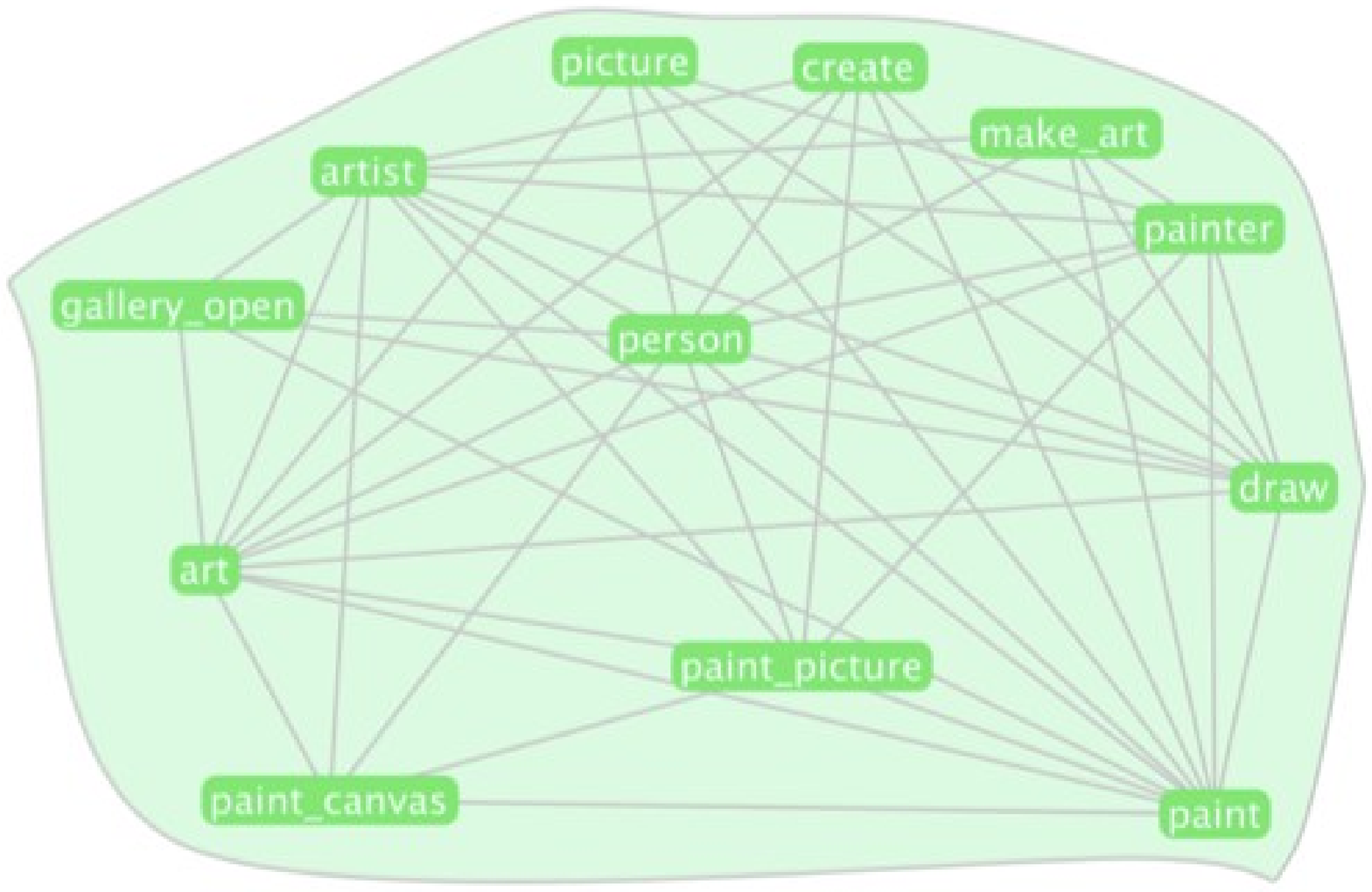}
\caption{Percolating painter. Twelve nodes by percolating cliques of size $6$.}\label{fig:communities:percolate:painter}
\end{subfigure}
\hspace{\fill}
\begin{subfigure}[b]{0.48\textwidth}
\includegraphics[width=\columnwidth]{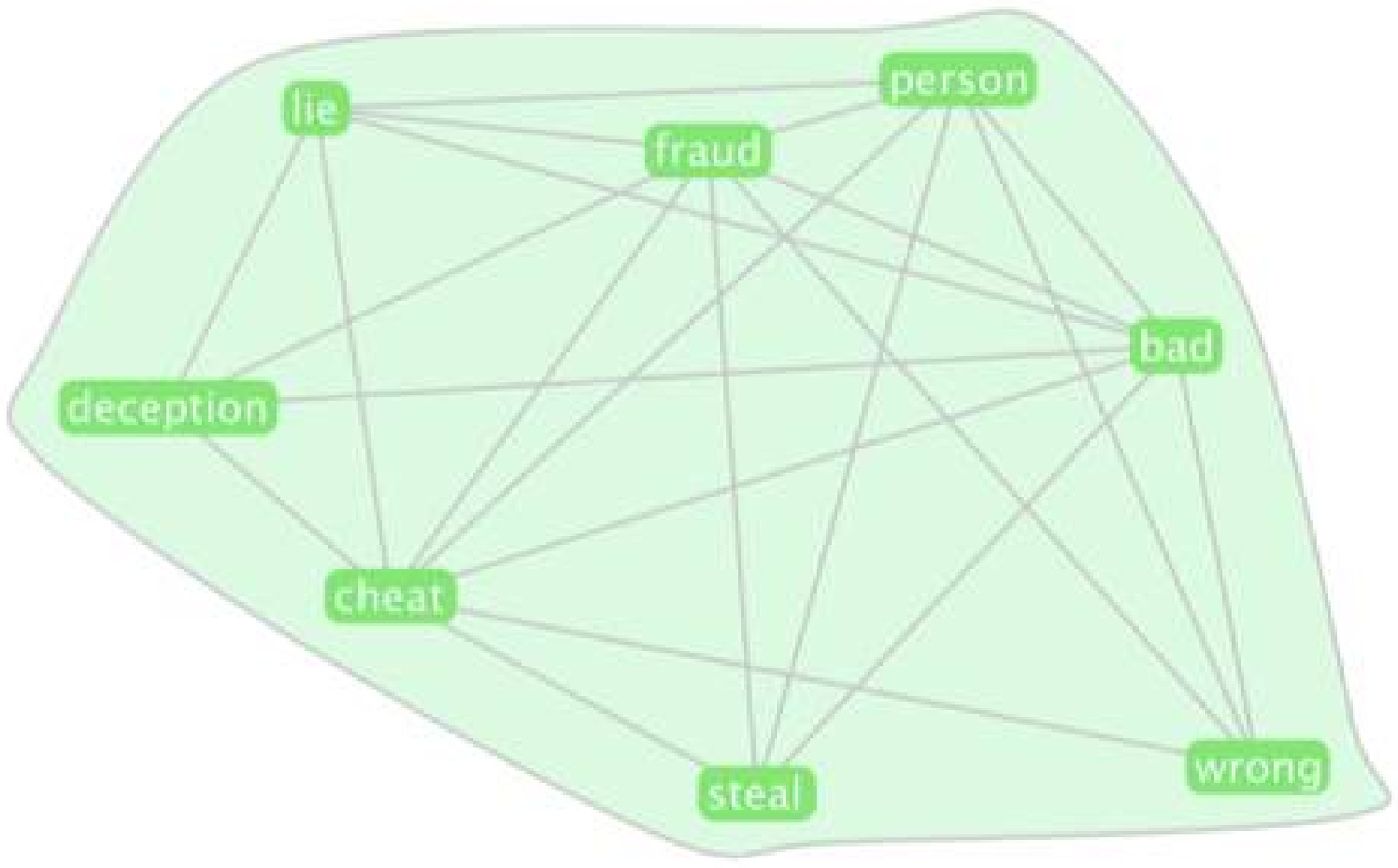}
\caption{Percolating dishonest/dishonesty. Eight nodes by percolating cliques of size $5$.
Note that \dbtext{dishonest}/\dbtext{dishonesty} is missing from the community.}\label{fig:communities:percolate:dishonest}
\end{subfigure}

\begin{subfigure}[b]{0.48\textwidth}
\includegraphics[width=\columnwidth]{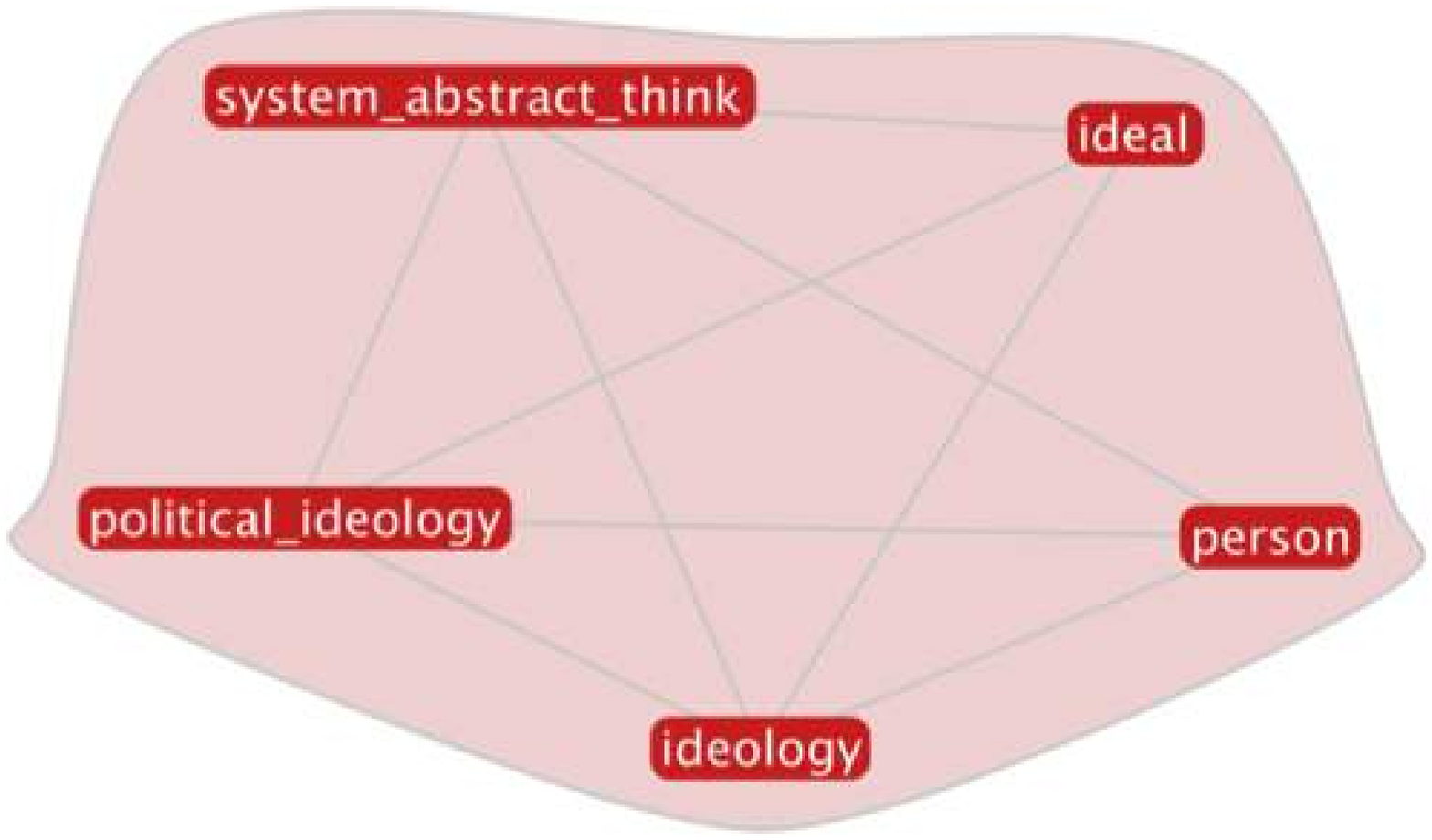}
\caption{Percolating ideals. Five nodes by percolating cliques of size $4$.}\label{fig:communities:percolate:ideals}
\end{subfigure}
\hspace{\fill}
\begin{subfigure}[b]{0.48\textwidth}
\includegraphics[width=\columnwidth]{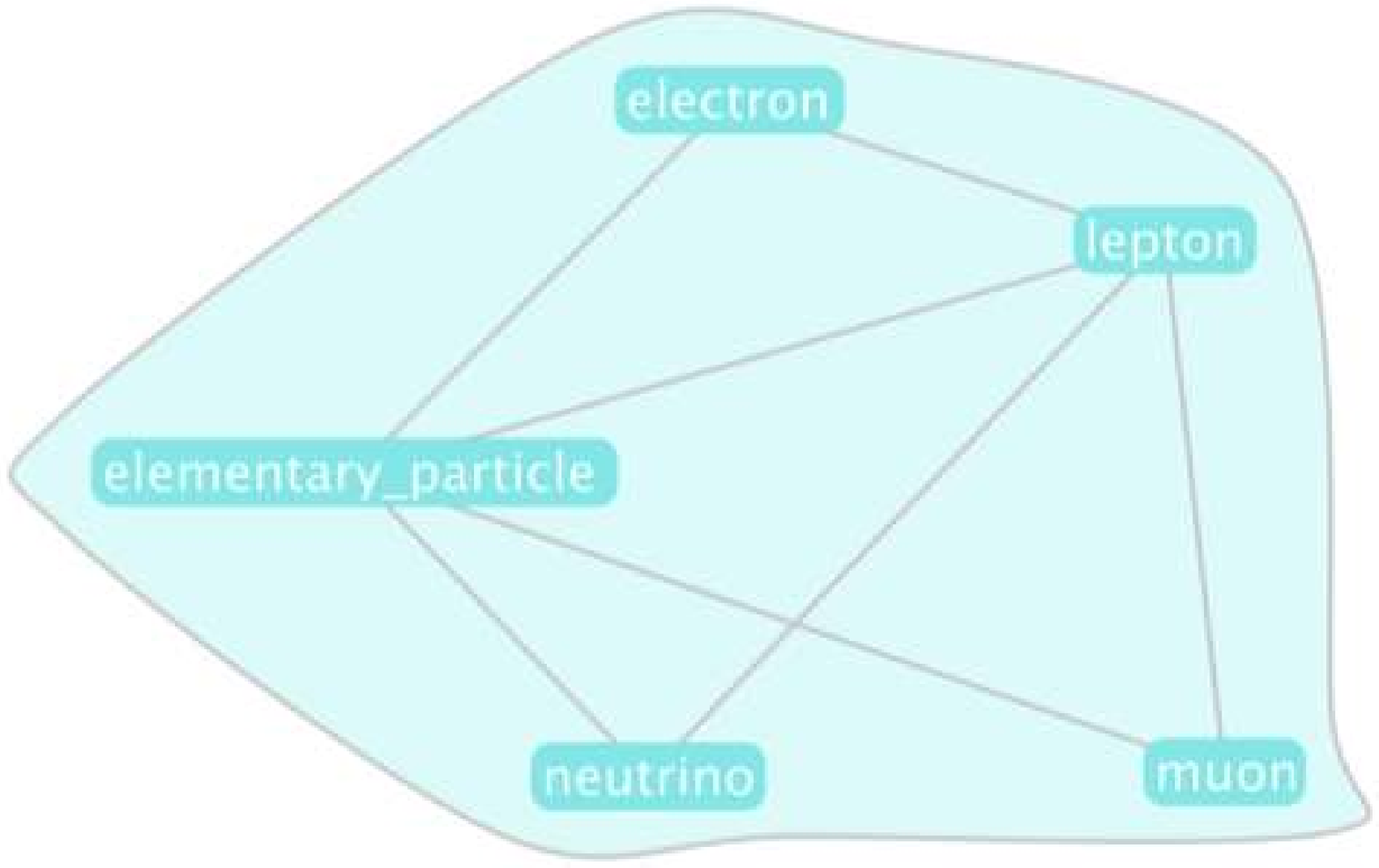}
\caption{Percolating particles. Five nodes by percolating cliques of size $3$.}\label{fig:communities:percolate:particles}
\end{subfigure}
\caption{Percolating cliques of sizes $3, 4, 5, 6, 7$ and $8$ and one interesting community in each case.}\label{fig:communities:percolate:b}
\end{figure}

\paragraph{Nested Clique.}
We can observe \emph{nested} cliques in \conceptnet.
One such instance appears by percolating cliques of size $9$
and is shown in Figure \ref{fig:communities:nested}.
The community shown in Figure \ref{fig:communities:percolate:bigCommunityPercolate9} 
is composed of $128$ concepts, while Figure \ref{fig:communities:percolate:basementPercolate9}
presents a community composed by a clique of size $9$ which does not percolate
to include more concepts. In the big clique we can see the concepts appearing in the smaller clique
either on the lower right hand side, or in the middle.

\begin{figure}[ht]
\centering
\begin{subfigure}[b]{\textwidth}
\includegraphics[width=\columnwidth]{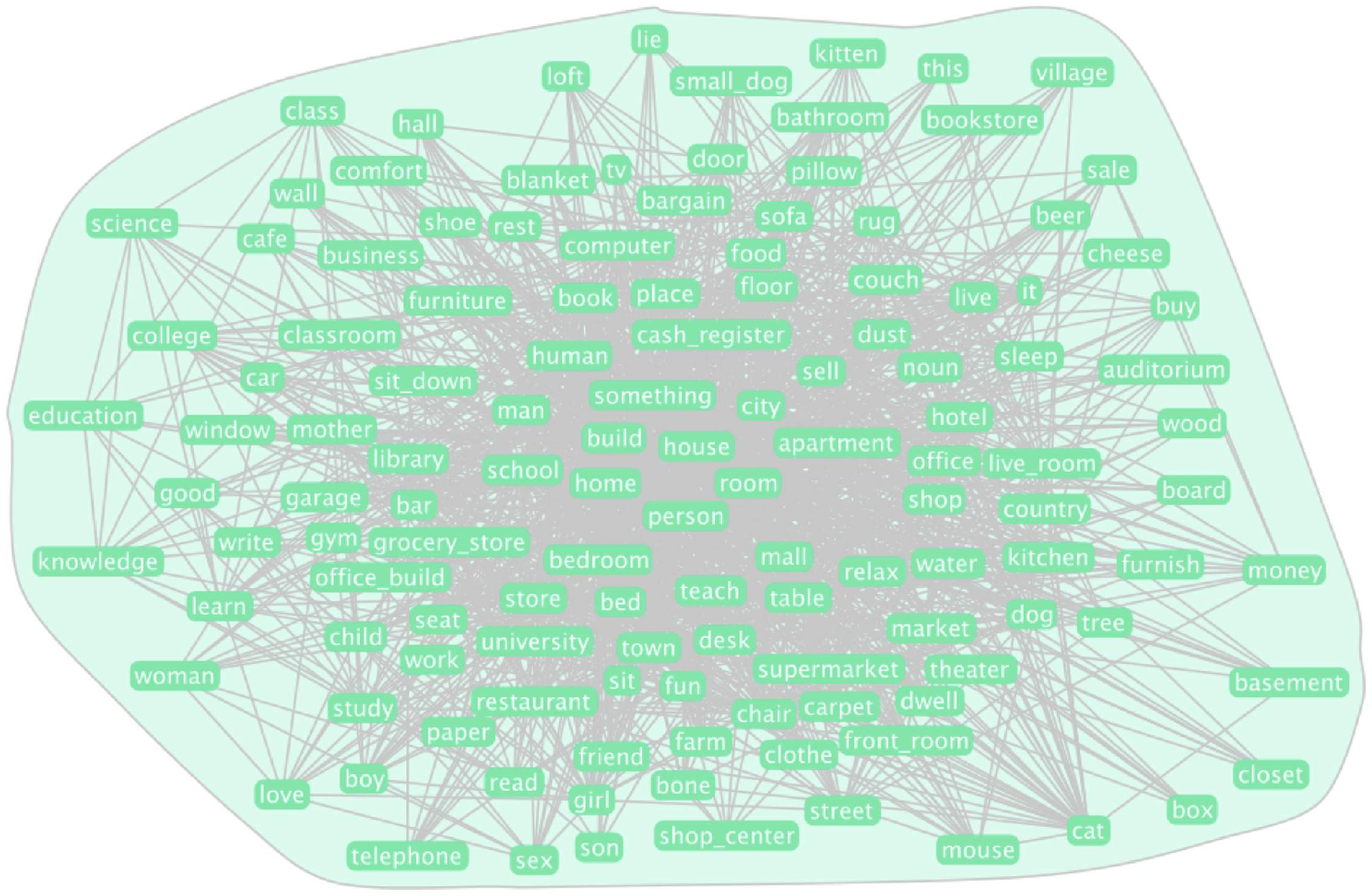}
\caption{A big community composed of $128$ concepts by percolating cliques of size $9$.}\label{fig:communities:percolate:bigCommunityPercolate9}
\end{subfigure}

\begin{subfigure}[b]{0.48\textwidth}
\includegraphics[width=\columnwidth]{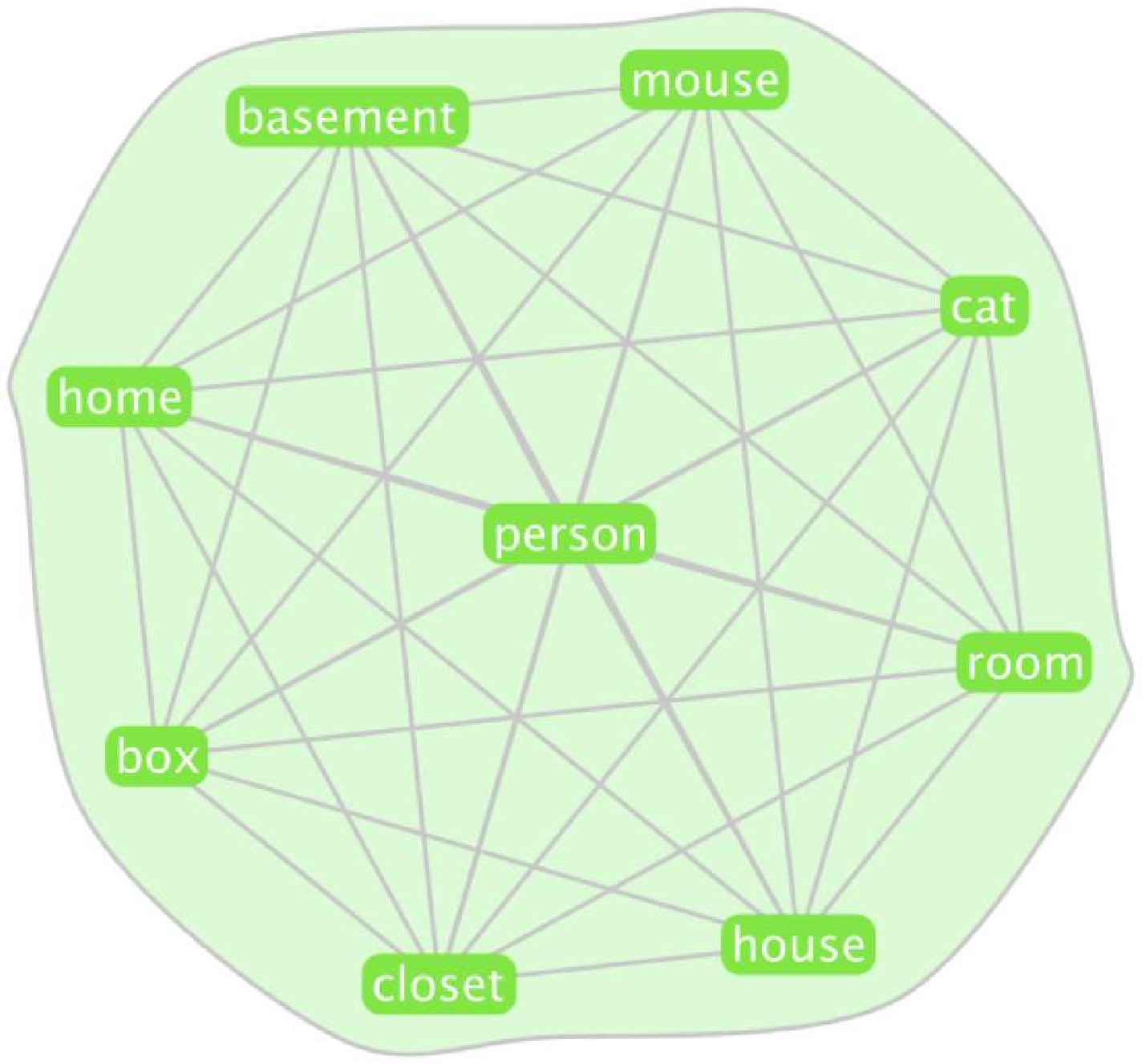}
\caption{A community composed of $9$ concepts that form a clique which does not percolate to include more concepts.}\label{fig:communities:percolate:basementPercolate9}
\end{subfigure}
\caption{Nested cliques. A clique of size $9$ has concepts which appear in a bigger clique.}\label{fig:communities:nested}
\end{figure}

\suppressfloats[t]

\paragraph{Overlapping Cliques.}
Figures \ref{fig:communities:overlapping:health}, \ref{fig:communities:overlapping:cut}, and \ref{fig:communities:overlapping:talk}
give some examples of overlapping communities in the graph induced by the assertions with positive polarity
(and positive score).

\begin{figure}[ht]
\centering
\begin{subfigure}[b]{0.8\textwidth}
\includegraphics[width=\columnwidth]{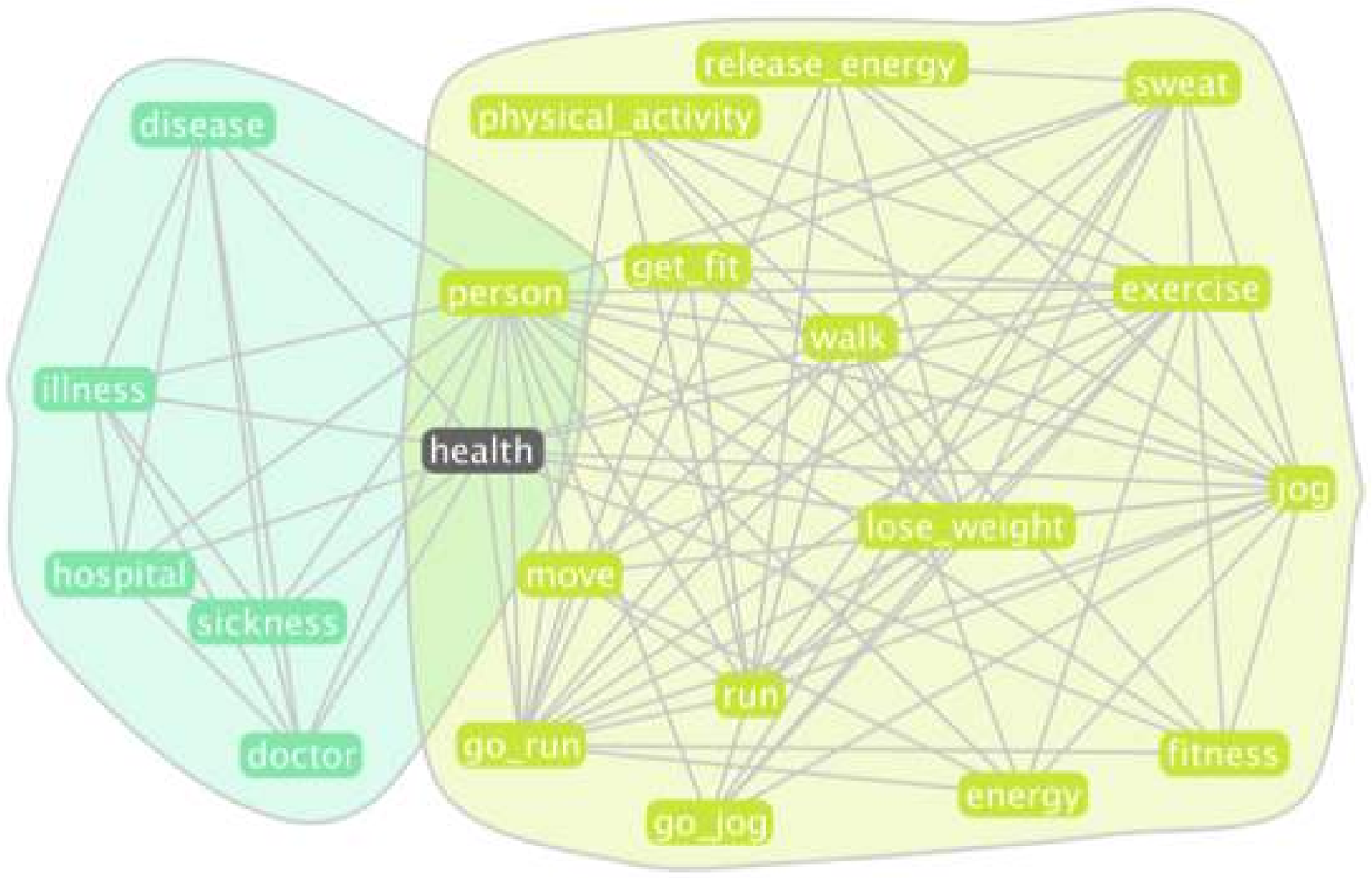}
\caption{Overlapping health.}\label{fig:communities:overlapping:health}
\end{subfigure}

\begin{subfigure}[b]{0.8\textwidth}
\includegraphics[width=\columnwidth]{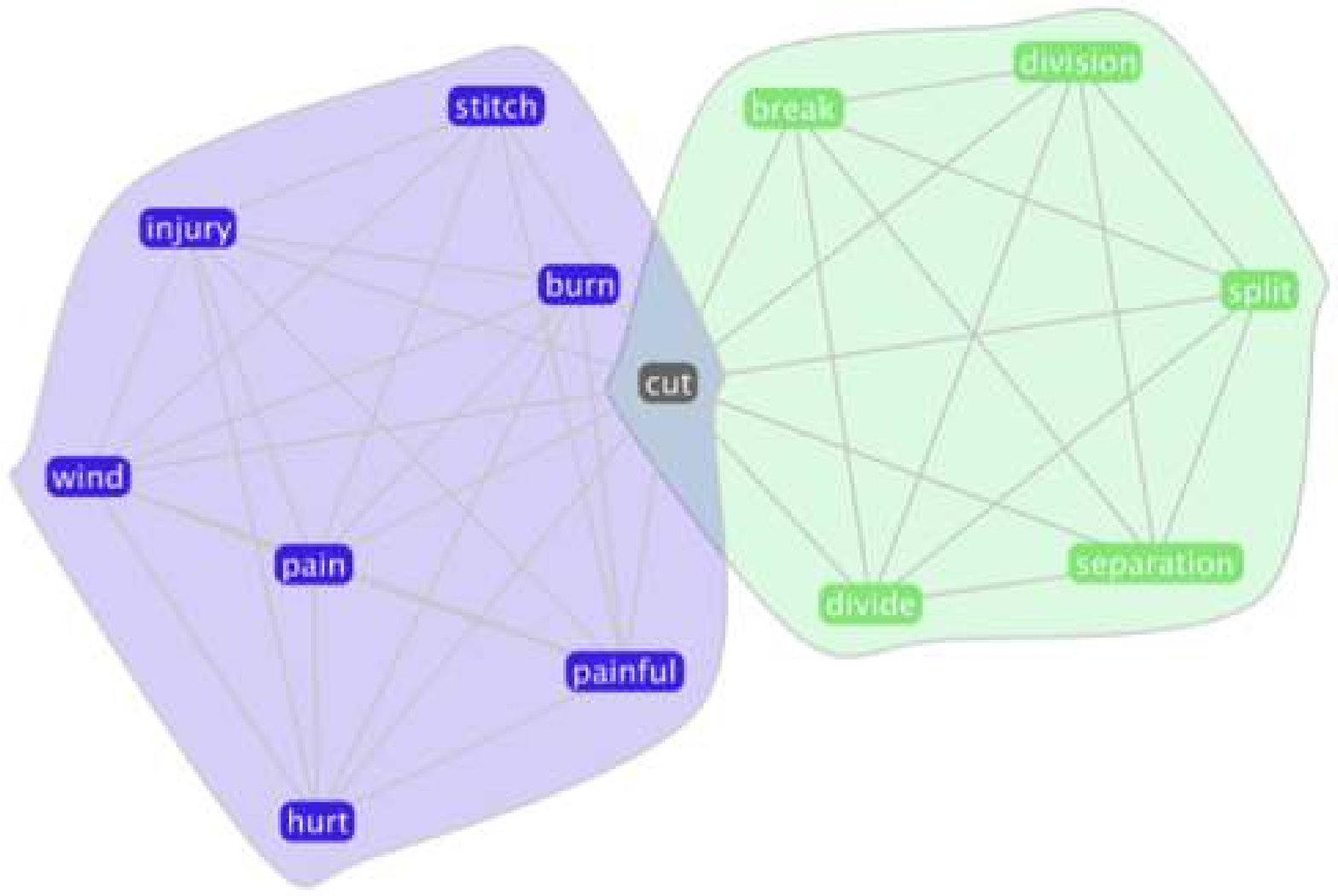}
\caption{Overlapping cut.}\label{fig:communities:overlapping:cut}
\end{subfigure}
\caption{Concepts participating in more than one communities.}\label{fig:communities:overlapping:a}
\end{figure}

\begin{figure}[ht]
\centering
\begin{subfigure}[b]{0.7\textwidth}
\includegraphics[width=\textwidth]{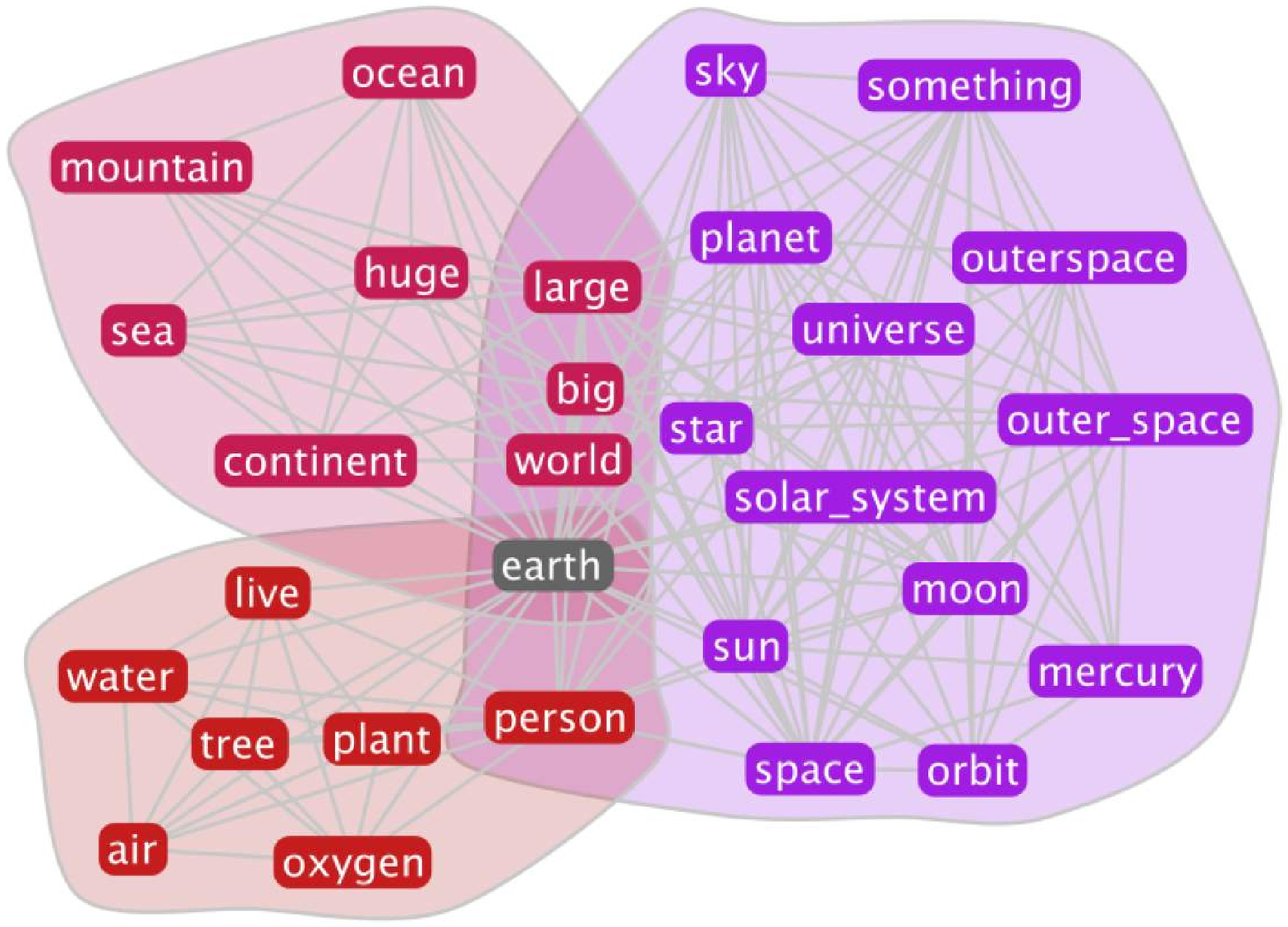}
\caption{Overlapping earth.}\label{fig:communities:overlapping:earth}
\end{subfigure}

\begin{subfigure}[b]{\textwidth}
\includegraphics[width=\textwidth]{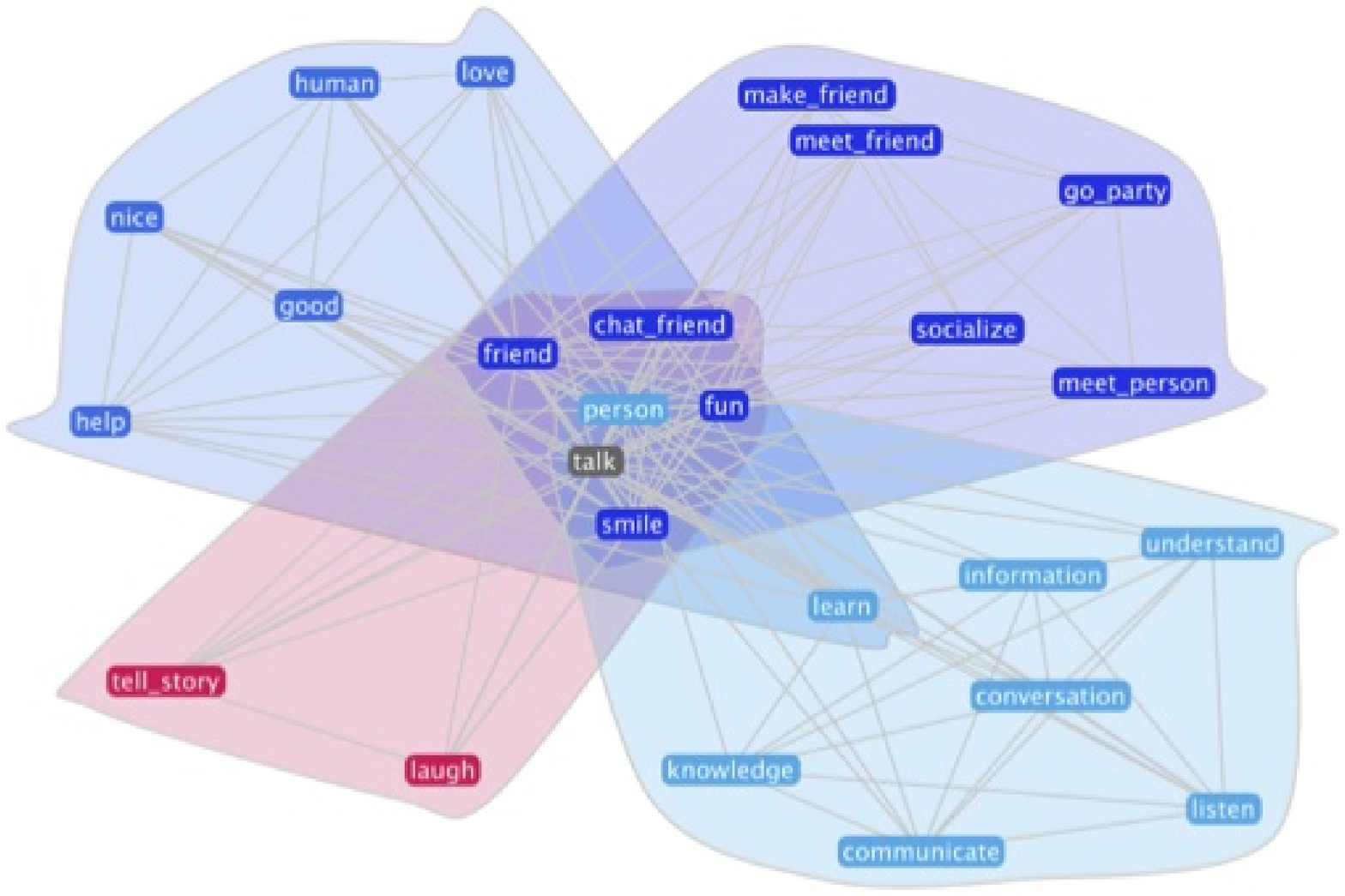}
\caption{Overlapping talk.}\label{fig:communities:overlapping:talk}
\end{subfigure}
\caption{Concepts participating in more than one communities.}\label{fig:communities:overlapping:b}
\end{figure}

\part{Mining}

\chapter{Mining Rules}
In this chapter we discuss the application of data mining towards the automated construction of a background theory
for the relations used in the knowledge base. We consider rules of the simplest form, mainly for computational considerations.

A \emph{rule} is given by an ordered triple of relations (\dbtext{X}, \dbtext{Y}, \dbtext{Z}),
where \dbtext{X}, \dbtext{Y} are the \emph{premisses} and \dbtext{Z} is
the \emph{conclusion}. For such a triple we consider triples of
concepts (\dbtext{a}, \dbtext{b}, \dbtext{c}) such that the assertions
\begin{center}
(\dbtext{a}, \dbtext{X}, \dbtext{b}) and (\dbtext{b}, \dbtext{Y}, \dbtext{c})
\end{center}
are in the knowledge base.
Such triples form the \emph{support} of the rule.
If (\dbtext{a}, \dbtext{Z}, \dbtext{c}) is also in the knowledge base then
(\dbtext{a}, \dbtext{b}, \dbtext{c}) is a \emph{success}
for the rule (\dbtext{X}, \dbtext{Y}, \dbtext{Z}), otherwise it is a \emph{failure}.
The \emph{success rate} of a rule is the percentage of successes in the support.
Consider, for example, the rule
(\dbtext{Desires}, \dbtext{Located\-Near}, \dbtext{At\-Lo\-ca\-tion}) and the triple of concepts
(\dbtext{human}, \dbtext{drink}, \dbtext{bar}).
The assertions
(\dbtext{human}, \dbtext{Desires}, \dbtext{drink}) and (\dbtext{drink}, \dbtext{Lo\-cat\-ed\-Near}, \dbtext{bar}) are
both in the knowledge base.
Therefore, we check whether the assertion
(\dbtext{human}, \dbtext{AtLocation}, \dbtext{bar})
is in the knowledge base.
It is, so (\dbtext{human}, \dbtext{drink}, \dbtext{bar}) is a success for the rule
(\dbtext{Desires}, \dbtext{LocatedNear}, \dbtext{At\-Lo\-ca\-tion}).

A triple of concepts (\dbtext{a}, \dbtext{b}, \dbtext{c}) is \emph{valid} for a rule
(\dbtext{X}, \dbtext{Y}, \dbtext{Z}) if the claim
\begin{center}
(\dbtext{a}, \dbtext{X}, \dbtext{b})
and
(\dbtext{b}, \dbtext{Y}, \dbtext{c})
therefore
(\dbtext{a}, \dbtext{Z}, \dbtext{c})
\end{center}
makes sense as a reasoning step. Otherwise (\dbtext{a}, \dbtext{b}, \dbtext{c}) is \emph{invalid}.
\emph{Making sense is a subjective judgement} and its intended meaning is up for discussion.
In what follows we use the sense \textquotedblleft given that the premisses hold it is reasonable
to assume that the conclusion holds\textquotedblright.
For example, (\dbtext{human}, \dbtext{drink}, \dbtext{bar}) is valid for the rule
(\dbtext{Desires}, \dbtext{LocatedNear}, \dbtext{AtLocation}).
Note that by the nature of its definition, deciding about validity requires an (often ambiguous)
decision by a human and so computing precise statistics about it is difficult.

We performed an exhaustive test for all possible rules
involving relations that have at least $300$ assertions with positive score regardless of their polarity.
We searched for \emph{frequent} rules, with support at least $300$ and success rate at least $5\%$
\footnote{For rules involving more than three concepts
such an exhaustive search is not feasible, and it will be necessary to use more advanced data mining techniques.}.
Success rates are expected to be low even for correct rules due to the sparsity of the network.
Tables \ref{tbl:mining:implications:list-of-relations:a} and \ref{tbl:mining:implications:list-of-relations:b}
present the $76$ triples of relations that satisfy these conditions; that is, at least $300$
assertions are in the support and at least $5\%$ success rate.
Below we give examples of some such relations,
plus an interesting one with low success rate, and comment on issues raised by these examples.

\begin{table}[ht]
\caption{The first $37$ out of the $76$ triples that appear to have support of at least $300$
and success rate at least $5\%$. Note that we have neglected from the computation all those relations
that do not have at least $300$ assertions with positive score regardless of their polarity.}\label{tbl:mining:implications:list-of-relations:a}
\begin{center}
\resizebox{\textwidth}{!}{
\begin{tabular}{|l|l|l|c|r|r|}\hline
relation $X$               & relation $Y$               & relation $Z$               & ratio    & successes & support \\\hline
HasFirstSubevent      ( 1) & MadeOf                ( 4) & Causes                (18) & 0.058282 &     19 &    326 \\\hline
AtLocation            ( 6) & AtLocation            ( 6) & AtLocation            ( 6) & 0.053967 &  29053 & 538349 \\\hline
AtLocation            ( 6) & PartOf                (21) & AtLocation            ( 6) & 0.085754 &   2394 &  27917 \\\hline
AtLocation            ( 6) & LocatedNear           (30) & AtLocation            ( 6) & 0.073569 &   4942 &  67175 \\\hline
AtLocation            ( 6) & SimilarSize           (31) & AtLocation            ( 6) & 0.082580 &   1919 &  23238 \\\hline
CapableOf             ( 8) & AtLocation            ( 6) & AtLocation            ( 6) & 0.115059 &   3067 &  26656 \\\hline
CapableOf             ( 8) & CausesDesire          (17) & CapableOf             ( 8) & 0.105263 &    424 &   4028 \\\hline
CapableOf             ( 8) & LocatedNear           (30) & AtLocation            ( 6) & 0.084516 &    530 &   6271 \\\hline
Desires               (10) & AtLocation            ( 6) & AtLocation            ( 6) & 0.109640 &   2707 &  24690 \\\hline
Desires               (10) & CausesDesire          (17) & CapableOf             ( 8) & 0.117987 &    286 &   2424 \\\hline
Desires               (10) & PartOf                (21) & AtLocation            ( 6) & 0.055444 &     55 &    992 \\\hline
Desires               (10) & CreatedBy             (25) & IsA                   ( 5) & 0.055233 &     19 &    344 \\\hline
Desires               (10) & CreatedBy             (25) & AtLocation            ( 6) & 0.055233 &     19 &    344 \\\hline
Desires               (10) & CreatedBy             (25) & CapableOf             ( 8) & 0.055233 &     19 &    344 \\\hline
Desires               (10) & LocatedNear           (30) & AtLocation            ( 6) & 0.122439 &    251 &   2050 \\\hline
Desires               (10) & LocatedNear           (30) & Desires               (10) & 0.053659 &    110 &   2050 \\\hline
Desires               (10) & SimilarSize           (31) & AtLocation            ( 6) & 0.058719 &     33 &    562 \\\hline
ConceptuallyRelatedTo (12) & ConceptuallyRelatedTo (12) & IsA                   ( 5) & 0.052977 &   7052 & 133114 \\\hline
ConceptuallyRelatedTo (12) & ConceptuallyRelatedTo (12) & HasProperty           (20) & 0.054194 &   7214 & 133114 \\\hline
ConceptuallyRelatedTo (12) & PartOf                (21) & AtLocation            ( 6) & 0.059574 &    579 &   9719 \\\hline
ConceptuallyRelatedTo (12) & PartOf                (21) & ConceptuallyRelatedTo (12) & 0.051960 &    505 &   9719 \\\hline
ConceptuallyRelatedTo (12) & PartOf                (21) & HasProperty           (20) & 0.051549 &    501 &   9719 \\\hline
ConceptuallyRelatedTo (12) & LocatedNear           (30) & IsA                   ( 5) & 0.060326 &   1611 &  26705 \\\hline
ConceptuallyRelatedTo (12) & LocatedNear           (30) & AtLocation            ( 6) & 0.062872 &   1679 &  26705 \\\hline
ConceptuallyRelatedTo (12) & LocatedNear           (30) & ConceptuallyRelatedTo (12) & 0.065007 &   1736 &  26705 \\\hline
ConceptuallyRelatedTo (12) & LocatedNear           (30) & HasProperty           (20) & 0.063471 &   1695 &  26705 \\\hline
ConceptuallyRelatedTo (12) & SimilarSize           (31) & IsA                   ( 5) & 0.067656 &    629 &   9297 \\\hline
ConceptuallyRelatedTo (12) & SimilarSize           (31) & ConceptuallyRelatedTo (12) & 0.068409 &    636 &   9297 \\\hline
ConceptuallyRelatedTo (12) & SimilarSize           (31) & HasProperty           (20) & 0.065182 &    606 &   9297 \\\hline
HasA                  (16) & PartOf                (21) & IsA                   ( 5) & 0.072956 &    706 &   9677 \\\hline
HasA                  (16) & PartOf                (21) & AtLocation            ( 6) & 0.050842 &    492 &   9677 \\\hline
HasA                  (16) & PartOf                (21) & ConceptuallyRelatedTo (12) & 0.059729 &    578 &   9677 \\\hline
HasA                  (16) & PartOf                (21) & HasProperty           (20) & 0.059419 &    575 &   9677 \\\hline
HasA                  (16) & LocatedNear           (30) & IsA                   ( 5) & 0.060117 &   1233 &  20510 \\\hline
HasA                  (16) & LocatedNear           (30) & AtLocation            ( 6) & 0.052365 &   1074 &  20510 \\\hline
HasA                  (16) & LocatedNear           (30) & ConceptuallyRelatedTo (12) & 0.060751 &   1246 &  20510 \\\hline
HasA                  (16) & LocatedNear           (30) & HasProperty           (20) & 0.059191 &   1214 &  20510 \\\hline
\end{tabular}
}
\end{center}
\end{table}

\begin{table}[ht]
\caption{The last $39$ out of the $76$ triples that appear to have support of at least $300$
and success rate at least $5\%$. Note that we have neglected from the computation all those relations
that do not have at least $300$ assertions with positive score regardless of their polarity.}\label{tbl:mining:implications:list-of-relations:b}
\begin{center}
\resizebox{\textwidth}{!}{
\begin{tabular}{|l|l|l|c|r|r|}\hline
relation $X$               & relation $Y$               & relation $Z$               & ratio    & successes & support \\\hline
HasProperty           (20) & LocatedNear           (30) & AtLocation            ( 6) & 0.051700 &   1743 &  33714 \\\hline
HasProperty           (20) & LocatedNear           (30) & HasProperty           (20) & 0.053183 &   1793 &  33714 \\\hline
HasProperty           (20) & SimilarSize           (31) & IsA                   ( 5) & 0.057224 &    642 &  11219 \\\hline
HasProperty           (20) & SimilarSize           (31) & ConceptuallyRelatedTo (12) & 0.055263 &    620 &  11219 \\\hline
HasProperty           (20) & SimilarSize           (31) & HasProperty           (20) & 0.063909 &    717 &  11219 \\\hline
CreatedBy             (25) & LocatedNear           (30) & AtLocation            ( 6) & 0.058252 &     42 &    721 \\\hline
LocatedNear           (30) & ConceptuallyRelatedTo (12) & HasProperty           (20) & 0.051206 &   2416 &  47182 \\\hline
LocatedNear           (30) & PartOf                (21) & IsA                   ( 5) & 0.059501 &    253 &   4252 \\\hline
LocatedNear           (30) & PartOf                (21) & AtLocation            ( 6) & 0.097601 &    415 &   4252 \\\hline
LocatedNear           (30) & PartOf                (21) & ConceptuallyRelatedTo (12) & 0.063735 &    271 &   4252 \\\hline
LocatedNear           (30) & PartOf                (21) & HasProperty           (20) & 0.070790 &    301 &   4252 \\\hline
LocatedNear           (30) & PartOf                (21) & PartOf                (21) & 0.062088 &    264 &   4252 \\\hline
LocatedNear           (30) & LocatedNear           (30) & IsA                   ( 5) & 0.062558 &    738 &  11797 \\\hline
LocatedNear           (30) & LocatedNear           (30) & AtLocation            ( 6) & 0.080698 &    952 &  11797 \\\hline
LocatedNear           (30) & LocatedNear           (30) & ConceptuallyRelatedTo (12) & 0.068407 &    807 &  11797 \\\hline
LocatedNear           (30) & LocatedNear           (30) & HasProperty           (20) & 0.070018 &    826 &  11797 \\\hline
LocatedNear           (30) & LocatedNear           (30) & LocatedNear           (30) & 0.058913 &    695 &  11797 \\\hline
LocatedNear           (30) & SimilarSize           (31) & IsA                   ( 5) & 0.051922 &    204 &   3929 \\\hline
LocatedNear           (30) & SimilarSize           (31) & AtLocation            ( 6) & 0.062611 &    246 &   3929 \\\hline
LocatedNear           (30) & SimilarSize           (31) & ConceptuallyRelatedTo (12) & 0.059812 &    235 &   3929 \\\hline
LocatedNear           (30) & SimilarSize           (31) & HasProperty           (20) & 0.052940 &    208 &   3929 \\\hline
SimilarSize           (31) & MadeOf                ( 4) & HasA                  (16) & 0.063559 &     60 &    944 \\\hline
SimilarSize           (31) & MadeOf                ( 4) & HasProperty           (20) & 0.050847 &     48 &    944 \\\hline
SimilarSize           (31) & ConceptuallyRelatedTo (12) & IsA                   ( 5) & 0.067418 &    933 &  13839 \\\hline
SimilarSize           (31) & ConceptuallyRelatedTo (12) & ConceptuallyRelatedTo (12) & 0.070020 &    969 &  13839 \\\hline
SimilarSize           (31) & ConceptuallyRelatedTo (12) & HasProperty           (20) & 0.070742 &    979 &  13839 \\\hline
SimilarSize           (31) & PartOf                (21) & IsA                   ( 5) & 0.060452 &     83 &   1373 \\\hline
SimilarSize           (31) & PartOf                (21) & AtLocation            ( 6) & 0.072833 &    100 &   1373 \\\hline
SimilarSize           (31) & PartOf                (21) & ConceptuallyRelatedTo (12) & 0.065550 &     90 &   1373 \\\hline
SimilarSize           (31) & PartOf                (21) & HasProperty           (20) & 0.067007 &     92 &   1373 \\\hline
SimilarSize           (31) & CreatedBy             (25) & ConceptuallyRelatedTo (12) & 0.055838 &     22 &    394 \\\hline
SimilarSize           (31) & LocatedNear           (30) & IsA                   ( 5) & 0.056172 &    162 &   2884 \\\hline
SimilarSize           (31) & LocatedNear           (30) & AtLocation            ( 6) & 0.074202 &    214 &   2884 \\\hline
SimilarSize           (31) & LocatedNear           (30) & ConceptuallyRelatedTo (12) & 0.065534 &    189 &   2884 \\\hline
SimilarSize           (31) & LocatedNear           (30) & HasProperty           (20) & 0.057559 &    166 &   2884 \\\hline
SimilarSize           (31) & SimilarSize           (31) & IsA                   ( 5) & 0.080597 &    108 &   1340 \\\hline
SimilarSize           (31) & SimilarSize           (31) & ConceptuallyRelatedTo (12) & 0.092537 &    124 &   1340 \\\hline
SimilarSize           (31) & SimilarSize           (31) & HasProperty           (20) & 0.088060 &    118 &   1340 \\\hline
SimilarSize           (31) & SimilarSize           (31) & SimilarSize           (31) & 0.053731 &     72 &   1340 \\\hline
\end{tabular}
}
\end{center}
\end{table}

\smallskip

Our first example is the rule (\dbtext{Desires}, \dbtext{LocatedNear}, \dbtext{AtLocation}).
This is the highest scoring rule with 251 successes and support 2050 (12\% success rate).
The triples (\dbtext{human}, \dbtext{drink}, \dbtext{bar}) and (\dbtext{bird}, \dbtext{seed}, \dbtext{garden})
are successful and valid. The triple (\dbtext{human}, \dbtext{love}, \dbtext{heart}) is successful but invalid.
The triple (\dbtext{bird}, \dbtext{seed}, \dbtext{plant garden}) is a failure but it is valid.
The reason for the failure is that the assertion
(\dbtext{bird},  \dbtext{AtLocation}, \dbtext{plant garden}) is missing from the knowledge base.
This is an example of using the mined
rules to identify missing entries.

\medskip

The rule (\dbtext{AtLocation}, \dbtext{PartOf}, \dbtext{AtLocation})
has 2,394 successes and support 27,917 (8.5\% success rate). The triple
(\dbtext{text book}, \dbtext{classroom}, \dbtext{school}) is successful and valid.
On the other hand, (\dbtext{text book}, \dbtext{classroom}, \dbtext{school system})
is a failure. In contrast to the failure discussed for the first rule above,
this is not due to a missing assertion, because the triple is \emph{invalid}.
This points to a general problem with this rule: it is only expected to hold if the
third concept is a physical object, like \dbtext{school} and unlike \dbtext{school system}.
Thus examining this example suggests a weakening of the rule.

The rule (\dbtext{PartOf}, \dbtext{AtLocation}, \dbtext{AtLocation})
is similar to the previous one. However, its success rate
is much smaller, only 1.4\% (with support 78,804, but only 1,112 successes).
 A possible explanation of the discrepancy can be illustrated by the triple
(\dbtext{engine oil}, \dbtext{car}, \dbtext{town}).
It is a failure as the assertion (\dbtext{engine oil}, \dbtext{AtLocation}, \dbtext{town})
is not in the knowledge base. Its validity depends on the
status of (\dbtext{engine oil}, \dbtext{AtLocation}, \dbtext{town}).
This assertion is not to be expected as input from a user (or from a text). On the other
hand, it is reasonable as a factual statement about the world.

Let us elaborate on the difference between the two rules.
For (\dbtext{AtLocation}, \dbtext{PartOf}, \dbtext{AtLocation}), the combined facts that
\dbtext{a} is an appropriate\footnote{By \emph{appropriate} we mean
``makes common sense for users asked to give natural language statements''.}
 left argument for \dbtext{AtLocation},
\dbtext{b} is an appropriate right argument for \dbtext{AtLocation},
and (\dbtext{b}, \dbtext{PartOf}, \dbtext{c}) mean that
if \dbtext{c} is an appropriate right argument for \dbtext{AtLocation}
(like \dbtext{school} but unlike \dbtext{school system}) then the
assertion (\dbtext{a}, \dbtext{AtLocation}, \dbtext{c})
\emph{makes sense both as a factual statement about the world and in terms of natural
language usage}.
By way of contrast, for (\dbtext{PartOf}, \dbtext{AtLocation}, \dbtext{AtLocation}),
things that are appropriate as left arguments for \dbtext{PartOf} are normally not
thought of as appropriate left arguments for \dbtext{AtLocation}; if they do occur
as such a left argument then they occur as being \dbtext{AtLocation} of the thing they are
part of. Thus, in this case (\dbtext{a}, \dbtext{AtLocation}, \dbtext{c}) \emph{may make sense
as a factual statement about the world but not in terms of natural language usage}.
Thus, the observed difference between the success rates of two similar rules points to a \emph{possible mismatch between
natural language usage and intended question answering applications}. This may be an issue to consider for
further knowledge base development.

The rule (\dbtext{LocatedNear}, \dbtext{PartOf}, \dbtext{IsA})
does not make much sense 
even if
it has $253$ successes and support $4,252$ (6\% success rate).
Most successes we examined are false or nonsensical.
This is an example of a rule with high success rate but with many successful, invalid triples.
An example is the triple (\dbtext{desk}, \dbtext{classroom}, \dbtext{school}).
The wrong assertion (\dbtext{desk}, \dbtext{IsA}, \dbtext{school}) comes from the sentence \dbtext{Schools have desks} through
the intermediate form \dbtext{Desk is a type of school}. Thus the problem presumably comes
from a programming error and fixing it might eliminate many wrong assertions.
Hence
this in an example where rule mining can be used to correct mistakes.



\newpage
\bibliography{ConceptNet}


\newpage
\appendix

\chapter{Tables and Files in CSV Format}\label{chapter:appendix:conceptnet-db}
Here we have a brief presentation of the CSV files.

\section*{conceptnet\_assertion}
This is the main table of the database. 
Table \ref{tbl:csv:assertion} presents the first two lines.

\begin{table}[p]
\caption{The beginning of conceptnet\_assertion.}\label{tbl:csv:assertion}
\begin{center}
\resizebox{\textwidth}{!}{
\begin{tabular}{|c|c|c|c|c|c|c|c|c|c|c|}\hline
id & language\_id & relation\_id & concept1\_id & concept2\_id & score & frequency\_id & best\_surface1\_id & best surface2\_id & best\_raw\_id & best\_frame\_id \\\hline\hline
2 & en & 6 & 5 & 6 & 1 & 1 & 5 & 6 & 3 & 3 \\\hline
3 & en & 7 & 7 & 8 & 1 & 1 & 7 & 8 & 4 & 4 \\\hline
\vdots & \vdots & \vdots & \vdots & \vdots & \vdots & \vdots & \vdots & \vdots & \vdots & \vdots 
\end{tabular}
}
\end{center}
\end{table}

\section*{conceptnet\_concept}
This table has information related to concepts.
Table \ref{tbl:csv:concept} presents the first two lines.

\begin{table}[p]
\caption{The beginning of conceptnet\_concept.}\label{tbl:csv:concept}
\begin{center}
\begin{tabular}{|c|c|c|c|c|c|}\hline
id & language\_id & text & num\_assertions & words & visible \\\hline\hline
5 & en & something & 2887 & 1 & 1 \\\hline
6 & en & to & 71 & 1 & 1 \\\hline
\vdots & \vdots & \vdots & \vdots & \vdots & \vdots 
\end{tabular}
\end{center}
\end{table}

\section*{conceptnet\_relation}
This table describes the relations that are used to form assertions.
Table \ref{tbl:csv:relation} presents the first two lines.

\begin{table}[p]
\caption{The beginning of conceptnet\_relation.}\label{tbl:csv:relation}
\begin{center}
\begin{tabular}{|c|c|c|}\hline
id & name & description \\\hline\hline
1 & HasFirstSubevent & What do you do first to accomplish it? \\\hline
2 & HasLastSubevent & What do you do last to accomplish it? \\\hline
\vdots & \vdots & \vdots 
\end{tabular}
\end{center}
\end{table}

\section*{nl\_frequency}
Table \ref{tbl:csv:frequency} describes the frequencies that are used in order to classify
the extent to which a relation holds between two concepts in the assertions.
It ranges from \emph{never} (\emph{polarity} is $-1$) to \emph{always} (\emph{polarity} is $+1$).

\begin{table}
\caption{The beginning of nl\_frequency.}\label{tbl:csv:frequency}
\begin{center}
\begin{tabular}{|c|c|c|c|}\hline
id & language\_id & text & value \\\hline\hline
 1 & en &       & 5 \\\hline
 2 & en & often & 6 \\\hline
\vdots & \vdots & \vdots & \vdots
\end{tabular}
\end{center}
\end{table}

\section*{conceptnet\_frame}
This table has information related to frames.
Table \ref{tbl:csv:frame} presents the first two lines.

\begin{table}[p]
\caption{The beginning of conceptnet\_frame.}\label{tbl:csv:frame}
\begin{center}
\resizebox{\textwidth}{!}{
\begin{tabular}{|c|c|c|c|c|c|c|c|c|}\hline
 id & language\_id & text & relation\_id & goodness & frequency\_id & question\_yn & question1 & question2 \\\hline\hline
3 & en & Somewhere \{1\} can be is next \{2\} & 6 & 1 & 1 & & & \\\hline
4 & en & You can use \{1\} to \{2\} & 7 & 2 & 1 & & & \\\hline
\vdots & \vdots & \vdots & \vdots & \vdots & \vdots &  &  & 
\end{tabular}
}
\end{center}
\end{table}

\section*{conceptnet\_surfaceform}
This table has the information related to surface forms.
Table \ref{tbl:csv:surfaceform} presents the first two lines.

\begin{table}[p]
\caption{The beginning of conceptnet\_surfaceform.}\label{tbl:csv:surfaceform}
\begin{center}
\begin{tabular}{|c|c|c|c|c|c|}\hline
id & language\_id & concept\_id & text & residue & use\_count \\\hline\hline
5 & en & 5 & something & 1ly & 3979 \\\hline
6 & en & 6 & to & 1 & 59 \\\hline
\vdots & \vdots & \vdots & \vdots & \vdots & \vdots 
\end{tabular}
\end{center}
\end{table}

\section*{conceptnet\_rawassertion}
This table has information related to raw assertions.
Table \ref{tbl:csv:rawassertion} presents the first two lines.

\begin{table}[ht]
\caption{The beginning of conceptnet\_rawassertion.}\label{tbl:csv:rawassertion}
\begin{center}
\resizebox{\textwidth}{!}{
\begin{tabular}{|c|c|c|c|c|c|c|c|c|c|c|c|}\hline
id & created & updated & sentence\_id & assertion\_id & creator\_id & surface1\_id & surface2\_id & frame\_id & batch\_id & language\_id & score \\\hline\hline
3 & 2009\ldots & 2009\ldots & 715991 & 2 & 997 & 5 & 6 & 3 & & en & 1 \\\hline
4 & 2009\ldots & 2009\ldots & 715993 & 3 & 992 & 7 & 8 & 4 & & en & 1 \\\hline
\vdots & \vdots & \vdots & \vdots & \vdots & \vdots & \vdots & \vdots & \vdots &  & \vdots & \vdots 
\end{tabular}
}
\end{center}
\end{table}

\section*{corpus\_sentence}
This table essentially has the actual sentence to which a raw assertion points to.
Table \ref{tbl:csv:sentence} presents the first two lines.

\begin{table}[ht]
\caption{The beginning of corpus\_sentence.}\label{tbl:csv:sentence}
\begin{center}
\resizebox{\textwidth}{!}{
\begin{tabular}{|c|c|c|c|c|c|c|}\hline
id & text & creator\_id & created\_on & language\_id & activity\_id & score \\\hline\hline
715991 & Somewhere something can be is next to & 997 & 2006\ldots & en & 27 & 1 \\\hline
715992 & picture description: an old house made of brick & 1002 & 2006\ldots & en & 27 & 1 \\\hline
\vdots & \vdots & \vdots & \vdots & \vdots & \vdots & \vdots 
\end{tabular}
}
\end{center}
\end{table}

\suppressfloats[t]

\section{Database Entries: Tables with Relations and Frequencies}\label{sec:appendix:db-related}
\paragraph{Relations.}
\conceptnet has $30$ relations; $27$ appear among the assertions in the English language.
Table \ref{tbl:relations} gives an overview of all the relations found in \conceptnet.

\begin{table}[ht]
\caption{The relations that we can find in \conceptnet. Note that three of them do not appear among assertions in the English language
and in these cases the index assigned to them is \xmark.}\label{tbl:relations}
\begin{center}
\begin{tabular}{|c||r|l|l|}\hline
index & id & \multicolumn{1}{c|}{name} & \multicolumn{1}{c|}{description} \\\hline
  0 &  1 & HasFirstSubevent      & What do you do first to accomplish it? \\\hline
  1 &  2 & HasLastSubevent       & What do you do last to accomplish it?  \\\hline
  2 &  3 & HasPrerequisite       & What do you need to do first?          \\\hline
  3 &  4 & MadeOf                & What is it made of? \\\hline
  4 &  5 & IsA                   & What kind of thing is it? \\\hline
  5 &  6 & AtLocation            & Where would you find it? \\\hline
  6 &  7 & UsedFor               & What do you use it for? \\\hline
  7 &  8 & CapableOf             & What can it do? \\\hline
  8 &  9 & MotivatedByGoal       & Why would you do it? \\\hline
  9 & 10 & Desires               & What does it want? \\\hline
\xmark & 11 & [deprecated 1]        &  \\\hline
 10 & 12 & ConceptuallyRelatedTo &  \\\hline
 11 & 13 & DefinedAs             & How do you define it? \\\hline
 12 & 14 & InstanceOf            & *What type of thing is it a specific example of? \\\hline
 13 & 15 & SymbolOf              &  \\\hline
 14 & 16 & HasA                  &  \\\hline
 15 & 17 & CausesDesire          & What does it make you want to do? \\\hline
 16 & 18 & Causes                & What does it make happen? \\\hline
 17 & 19 & HasSubevent           & What do you do to accomplish it? \\\hline
 18 & 20 & HasProperty           & What properties does it have? \\\hline
 19 & 21 & PartOf                & What is it part of? \\\hline
 20 & 22 & ReceivesAction        & What can you do to it? \\\hline
\xmark & 23 & ObstructedBy          &  \\\hline
 21 & 24 & InheritsFrom          &  \\\hline
 22 & 25 & CreatedBy             & How do you bring it into existence? \\\hline
\xmark & 26 & Translation           &  \\\hline
 23 & 28 & HasPainCharacter      & *What is the character of pain associated with it? \\\hline
 24 & 29 & HasPainIntensity      & *What is the intensity of pain associated with it? \\\hline
 25 & 30 & LocatedNear           &  \\\hline
 26 & 31 & SimilarSize           &  \\\hline
\end{tabular}
\end{center}
\end{table}

\paragraph{Frequencies.}
Table \ref{tbl:different-frequencies} presents the different frequencies that we
can encounter in \conceptnet in the assertions of the English language.

\begin{table}[ht]
\caption{The different frequencies that we can encounter in the table \texttt{nl\_frequency}
in the English language.}\label{tbl:different-frequencies}
\begin{center}
\begin{tabular}{|r||r|l|r|}\hline
\multicolumn{1}{|c||}{index} & \multicolumn{1}{|c|}{id} & \multicolumn{1}{c|}{text} & \multicolumn{1}{c|}{value} \\\hline\hline
 0 &    1 &               &   5 \\\hline
 1 &    2 &         often &   7 \\\hline
 2 &    3 & (UNSPECIFIED) &   5 \\\hline
 3 &   11 &        always &  10 \\\hline
 4 &   21 &         never & -10 \\\hline
 5 &   22 &        rarely &  -2 \\\hline
 6 &   25 &           not &  -5 \\\hline
 7 & 1209 &     sometimes &   4 \\\hline
 8 & 1215 &       usually &   8 \\\hline
 9 & 1368 &  occasionally &   2 \\\hline
10 & 1403 & almost always &   9 \\\hline
\end{tabular}
\end{center}
\end{table}

\chapter{Derived Input Files}
In this section we present the format and properties of the input
files that we are going to use.

\section{Special Indices}
Throughout the files we have reserved two special \emph{indices} that may appear
in the input files. These are the following:
\begin{description}
 \item [Null Index = -1.] This index may appear in fields where the relevant field in the relevant table of \conceptnet 
had a null entry (i.e.~the empty string was the actual input).

 \item [Undefined Index = -2.] This index is useful when an entry in a field refers to an object
that does \emph{not} actually appear in the appropriate table for that object. 
In other words, an index equal to $-2$ in a specific field implies that the 
index found originally for that field in the \conceptnet database
was pointing to an object that did not actually exist in the 
database \footnote{Actually, in the original tables of \conceptnet we find IDs instead of indices, but this
is a more convenient description.}.

%
\end{description}

\section{Files with the Tables of the Database}
In this part we describe the tables that we derived from the original tables of the 
\conceptnet database.

\subsection*{Assertions}
Table \ref{tbl:input:table:assertions} presents the format of each line in the 
file describing the assertions.
Each line is composed of $14$ integers separated by a space.
Each line ends with a new line character \texttt{'\textbackslash n'} right after the last integer.
There are four indicators per assertion; in Table \ref{tbl:input:table:assertions} we have
compressed them in one entry (the last entry) for clarity of presentation.
These four indicators are, in that order, the \emph{frame indicator}, the \emph{surface form indicator},
the \emph{raw assertion indicator}, and the \emph{score indicator}.

\begin{table}[ht]
\caption{The format of each line in the file describing the assertions.
All the entries are integers separated by a space. 
Each line ends with a new line character \texttt{'\textbackslash n'} right after the last integer.
There are four indicators per assertion; the frame indicator, the surface form indicator, the raw assertion indicator, 
and the score indicator.}\label{tbl:input:table:assertions}
\begin{center}
\resizebox{\textwidth}{!}{
\begin{tabular}{|c|c|c|c|c|c|c|c|c|c|c|}\hline
\multirow{2}{*}{id} & concept 1 & concept 2 & relation & frequency & best frame & best surface 1 & best surface 2 & best raw assertion & \multirow{2}{*}{score} & \multirow{2}{*}{indicators}\\
                    &  index    &  index    &  index   &  index    & index      &   index        &  index         & index              &       & \\\hline
\end{tabular}
}
\end{center}
\end{table}

The first 2 lines of the file are shown below:
\begin{code}
\$ head -n 2 inputFiles/assertions/ConceptNet4Assertions.txt \\
2 0 1 5 0 0 0 1 0 1 0 0 0 0\\
3 2 3 6 0 1 2 3 1 1 0 0 0 0\\
\$
\end{code}

\paragraph{Number of Lines.} The file has $566094$ lines.

\paragraph{Permissible Values for each Field.}
The permissible values are described below.
\begin{description} 
 \item [id:] The ID of each assertion in the original \conceptnet database. There are $566094$
different values from the set $\{2, 3, \ldots, 898685\}$. This is the only set where not all integers are covered.
 \item [concept 1 index, concept 2 index:] There are $279497$ different concept IDs that appear in the assertions.
Hence, the values come from the set $\{0, 1, \ldots, 279496\}$.
 \item [relation index:] Integers from the set $\{0, 1, \ldots, 26\}$. 
 \item [frequency index:] Integers from the set $\{0, 1, \ldots, 10\}$.
 \item [best frame\ index:] Integers from the set $\{-1\}\cup\{0, 1, \ldots, 2752\}$.
Note that the frame can be null.
 \item [best surface 1 index, best surface 2 index:] Integers from the set $\{-1\}\cup\{0, 1, \ldots, 375589\}$.
Note that a surface form can be null.
 \item [best raw assertion index:] Integers from the set $\{-2, -1\}\cup\{0, 1, \ldots, 525179\}$.
Note that the raw assertion can be null \emph{or undefined}.
\item [score:] Integers from the set $\{-10, -9, \ldots, 147\}$.
\item [frame indicator:] See Table \ref{tbl:frames:indicator}.
\item [surface form indicator:] See Table \ref{tbl:surfaceforms:indicator}.
\item [raw assertion indicator:] See Table \ref{tbl:rawassertions:indicator-distribution}.
\item [score indicator:] Integers from the set $\{0, 1, \ldots, 9\}$.
See Table \ref{tbl:scores:discrepancy-classes}.
\end{description}

\subsection*{Concepts}
Table \ref{tbl:input:table:concepts} presents the format of each line in the 
file describing the concepts.
Each line is composed of one integer and a text description of the concept.
These two are separated by a space. Hence, once we read an integer and a space, whatever remains
until a new line character \texttt{'\textbackslash n'} is encountered is the text description of 
the particular concept.

\begin{table}[ht]
\caption{The format of each line in the file describing the concepts.
Each line has two entries; a number and a string describing the concept. 
Each line ends with a new line character \texttt{'\textbackslash n'} right at the end of the string
describing the concept.}\label{tbl:input:table:concepts}
\begin{center}
\begin{tabular}{|c|c|}\hline
id & text\\\hline
\end{tabular}
\end{center}
\end{table}

The first 2 lines of the file are shown below:
\begin{code}
\$ head -n 2 inputFiles/concepts/ConceptNet4Concepts.txt\\
5 something\\
6 to\\
\$
\end{code}

\paragraph{Number of Lines.} The file has $279885$ lines.

\paragraph{Permissible Values for each Field.}
\begin{description}
 \item [id:] Integers from the set $\{5, 6, \ldots, 482783\}$.
 \item [text:] Longest string length is $204$ characters (ID $211344$).
\end{description}

\subsection*{Relations}
Table \ref{tbl:input:table:relations} presents the format of each line in the 
file describing the relations.
Each line is composed of one integer and two strings describing the relation.
Each field is separated by a space. 
Note that this does not leave any ambiguity among the strings, since the field \texttt{name}
is a single word. Each line ends with a new line character \texttt{'\textbackslash n'}.

\begin{table}[ht]
\caption{The format of each line in the file describing the relations.
Each line has three entries; a number and two strings.
The first of the two strings (\texttt{name}) is a single word and the second
string (\texttt{description}) is a more detailed description of the relation. 
Each line ends with a new line character \texttt{'\textbackslash n'} 
right at the end of the second string.}\label{tbl:input:table:relations}
\begin{center}
\begin{tabular}{|c|c|c|}\hline
id & name & description\\\hline
\end{tabular}
\end{center}
\end{table}

The first 2 lines of the file are shown below:
\begin{code}
\$ head -n 2 inputFiles/relations/ConceptNet4Relations.txt\\
1 HasFirstSubevent What do you do first to accomplish it?\\
2 HasLastSubevent What do you do last to accomplish it?\\
\$
\end{code}

\paragraph{Number of Lines.} The file has $27$ lines.

\paragraph{Permissible Values for each Field.}
\begin{description}
 \item [id:] Integers from the set $\{1, 2, \ldots, 31\}$.
 \item [name:] Longest string length is $21$ characters (ID $12$).
 \item [description:] Longest string length is $50$ characters (ID $28$).
\end{description}

\subsection*{Frequencies}
Table \ref{tbl:input:table:frequencies} presents the format of each line in the 
file describing the frequencies.
Each line is composed of two integers and one string describing the frequency.
Each field is separated by a space. 
Note that this does not leave any ambiguity among the strings, since the field \texttt{name}
is a single word. Each line ends with a new line character \texttt{'\textbackslash n'}.

\begin{table}[ht]
\caption{The format of each line in the file describing the frequencies.
Each line has three entries; two numbers and one string. 
Each line ends with a new line character \texttt{'\textbackslash n'}.}\label{tbl:input:table:frequencies}
\begin{center}
\begin{tabular}{|c|c|c|}\hline
id & value & text\\\hline
\end{tabular}
\end{center}
\end{table}

The first 2 lines of the file are shown below:
\begin{code}
\$ head -n 2 inputFiles/frequencies/ConceptNet4Frequencies.txt \\
1 5 \\
2 7 often\\
\$
\end{code}

\paragraph{Number of Lines.} The file has $11$ lines.

\paragraph{Permissible Values for each Field.}
\begin{description}
 \item [id:] Integers from the set $\{1, 2, \ldots, 1403\}$.
 \item [value:] Integers from the set $\{-10, -9, \ldots, 10\}$.
 \item [text:] Longest string length is $13$ characters (ID $3$).
\end{description}

\subsection*{Frames}
Table \ref{tbl:input:table:frames} presents the format of each line in the 
file describing the frames.
Each line is composed of three integers and one string describing the frame.
Each field is separated by a space. 
Each line ends with a new line character \texttt{'\textbackslash n'}.

\begin{table}[ht]
\caption{The format of each line in the file describing the frames.
Each line has four entries; three numbers and one string. 
Each line ends with a new line character \texttt{'\textbackslash n'}.}\label{tbl:input:table:frames}
\begin{center}
\begin{tabular}{|c|c|c|c|}\hline
id & relation index & frequency index & text\\\hline
\end{tabular}
\end{center}
\end{table}

The first 2 lines of the file are shown below:
\begin{code}
\$ head -n 2 inputFiles/frames/ConceptNet4Frames.txt \\
3 5 0 Somewhere \{1\} can be is next \{2\}\\
4 6 0 You can use \{1\} to \{2\}\\
\$
\end{code}

\paragraph{Number of Lines.} The file has $2753$ lines.

\paragraph{Permissible Values for each Field.}
\begin{description}
 \item [id:] Integers from the set $\{3, 4, \ldots, 3831\}$.
 \item [relation index:] Integers from the set $\{0, 1, \ldots, 26\}$.
 \item [frequency index:] Integers from the set $\{0, 1, \ldots, 10\}$.
 \item [text:] Longest string length is $131$ characters (ID $2788$).
\end{description}

\subsection*{Surface Forms}
Table \ref{tbl:input:table:surfaceforms} presents the format of each line in the 
file describing the surface forms.
Each line is composed of two integers and one string describing the surface form.
Each field is separated by a space. 
Each line ends with a new line character \texttt{'\textbackslash n'}.

\begin{table}[ht]
\caption{The format of each line in the file describing the surface forms.
Each line has three entries; two numbers and one string. 
Each line ends with a new line character \texttt{'\textbackslash n'}.}\label{tbl:input:table:surfaceforms}
\begin{center}
\begin{tabular}{|c|c|c|}\hline
id & concept index & text\\\hline
\end{tabular}
\end{center}
\end{table}

The first 2 lines of the file are shown below:
\begin{code}
\$ head -n 2 inputFiles/surfaceForms/ConceptNet4SurfaceForms.txt \\
5 0 something\\
6 1 to\\
\$
\end{code}

\paragraph{Number of Lines.} The file has $375590$ lines.

\paragraph{Permissible Values for each Field.}
\begin{description}
 \item [id:] Integers from the set $\{5, 6, \ldots, 580314\}$.
 \item [concept index:] Integers from the set $\{0, 1, \ldots, 279884\}$.
 \item [text:] Longest string length is $255$ characters (IDs $286820$).
\end{description}

\subsection*{Raw Assertions}
Table \ref{tbl:input:table:rawassertions} presents the format of each line in the 
file describing the raw assertions.
Each line is composed of seven integers separated by a space.
Each line ends with a new line character \texttt{'\textbackslash n'}.

\begin{table}[ht]
\caption{The format of each line in the file describing the raw assertions.
Each line has seven integers separated by a space.
Each line ends with a new line character \texttt{'\textbackslash n'}.}\label{tbl:input:table:rawassertions}
\begin{center}
\begin{tabular}{|c|c|c|c|c|c|c|}\hline
id & sentence index & assertion index & surface 1 index & surface 2 index & frame index & score \\\hline
\end{tabular}
\end{center}
\end{table}

The first two lines of the file are shown below:
\begin{code}
\$ head -n 2 inputFiles/rawAssertions/ConceptNet4RawAssertions.txt \\
3 0 0 0 1 0 1\\
4 1 1 2 3 1 1\\
\$
\end{code}

\paragraph{Number of Lines.} The file has $525180$ lines.

\paragraph{Permissible Values for each Field.}
\begin{description}
 \item [id:] Integers from the set $\{3, 4, \ldots, 1277256\}$.
 \item [sentence index:] Integers from the set $\{0, 2, \ldots, 525170\}$.
 \item [assertion index:] Integers from the set $\{0, 1, \ldots, 566093\}$.
 \item [surface 1 index, surface 2 index:] Integers from the set $\{0, 1, \ldots, 375589\}$.
 \item [frame index:] Integers from the set $\{0, 1, \ldots, 2752\}$.
 \item [score:] Integers from the set $\{-10, -9, \ldots 124\}$.
\end{description}

\subsection*{Sentences}
Table \ref{tbl:input:table:sentences} presents the format of each line in the 
file describing the sentences.
Each line is composed of three integers and one string describing the frame.
Each field is separated by a space. 
Each line ends with a new line character \texttt{'\textbackslash n'}.

\begin{table}[ht]
\caption{The format of each line in the file describing the sentences.
Each line has three entries; two numbers and one string which is the actual sentence. 
Each line ends with a new line character \texttt{'\textbackslash n'}.}\label{tbl:input:table:sentences}
\begin{center}
\begin{tabular}{|c|c|c|}\hline
id & score & text\\\hline
\end{tabular}
\end{center}
\end{table}

The last two lines of the file are shown below:
\begin{code}
\$ tail -n 2 inputFiles/sentences/ConceptNet4Sentences.txt \\
2608286 1 A cloned animal is made of D.N.A.\\
2608290 1 An U.F.O, is made of alien material.\\
\$
\end{code}

\paragraph{Number of Lines.} The file has $525171$ lines.

\paragraph{Permissible Values for each Field.}
\begin{description}
 \item [id:] Integers from the set $\{715991, 715992, \ldots, 2608290\}$.
 \item [score:] Integers from the set $\{-10, -9, \ldots, 48\}$.
 \item [text:] The length of the largest string is $1216$ characters (ID $1023955$).
\end{description}

\section{Mapping From ConceptNet 4}\label{sec:derived:map-from}
Here we describe the structure of the files that map IDs for various objects from \conceptnet
to the IDs that we use for input. All the files have one integer per line.
These integers refer to the \emph{indices} in the appropriate table where the objects can be found.
Hence, valid indices are non-negative integers. However, we may encounter either a null index ($-1$),
or an undefined index ($-2$). Null indices indicate that there was no object with such an ID in the 
original \conceptnet database. On the other hand, undefined indices indicate that there was an object
with such an ID in the original \conceptnet database, but it turns out that this object does not appear
in the closure of the input defined by the assertions of the English language.

\subsection*{Map Assertion IDs From ConceptNet 4}
We can see the first 10 lines below:
\begin{code}
\$ head -n 10 inputFiles/assertions/MapAssertionIDsFromConceptNet4.txt \\
-1\\
-1\\
0\\
1\\
-1\\
-1\\
2\\
3\\
4\\
5\\
\$
\end{code}
Hence, there are no assertions with IDs $0, 1, 4$ and $5$ in the \conceptnet database
when we restrict the search in the set $\{0, 1, \ldots, 9\}$.
On the other hand assertion with ID $2$ appears in index $0$ of the table of assertions,
assertion with ID $3$ appears in index $1$ in the table of assertions, and so on.

\paragraph{Number of Lines.} The file has $898686$ lines.

\subsection*{Map Concept IDs From ConceptNet 4}
We can see the first 10 lines below:
\begin{code}
\$ head -n 10 inputFiles/concepts/MapConceptIDsFromConceptNet4.txt \\
-1\\
-1\\
-1\\
-1\\
-1\\
0\\
1\\
2\\
3\\
4\\
\$
\end{code}
Hence, there are no concepts with IDs $0, 1, \ldots, 4$ in the \conceptnet database
when we restrict the search in the set $\{0, 1, \ldots, 9\}$.
On the other hand concept with ID $5$ appears in index $0$ of the table of concepts,
concept with ID $6$ appears in index $1$ in the table of concepts, and so on.

\paragraph{Number of Lines.} The file has $482784$ lines.

\subsection*{Map Relation IDs From ConceptNet 4}
We can see the first 10 lines below:
\begin{code}
\$ head -n 10 inputFiles/relations/MapRelationIDsFromConceptNet4.txt \\
-1\\
0\\
1\\
2\\
3\\
4\\
5\\
6\\
7\\
8\\
\$
\end{code}
Hence, there is no relation with ID $0$ in the \conceptnet database
when we restrict the search in the set $\{0, 1, \ldots, 9\}$.
On the other hand relation with ID $1$ appears in index $0$ of the table of relations,
relation with ID $2$ appears in index $1$ in the table of relations, and so on.

\paragraph{Number of Lines.} The file has $32$ lines.

\subsection*{Map Frequency IDs From ConceptNet 4}
We can see the first 10 lines below:
\begin{code}
\$ head -n 10 inputFiles/frequencies/MapFrequencyIDsFromConceptNet4.txt \\
-1\\
0\\
1\\
2\\
-1\\
-1\\
-1\\
-1\\
-1\\
-1\\
\$
\end{code}
Hence, 
when we restrict the search in the set $\{0, 1, \ldots, 9\}$,
only the IDs $1$, $2$, $3$ actually appear in the original \conceptnet database
and these are mapped respectively to indices $0, 1$, and $2$.
All the other frequency IDs are invalid ($-1$) in that region.

\paragraph{Number of Lines.} The file has $1404$ lines.

\subsection*{Map Frame IDs From ConceptNet 4}
We can see the first 10 lines below:
\begin{code}
\$ head -n 10 inputFiles/frames/MapFrameIDsFromConceptNet4.txt \\
-1\\
-1\\
-1\\
0\\
1\\
2\\
3\\
4\\
5\\
6\\
\$
\end{code}
Hence, there are no frames with IDs $0, 1$, and $2$ in the \conceptnet database
when we restrict the search in the set $\{0, 1, \ldots, 9\}$.
On the other hand frame with ID $3$ appears in index $0$ of the table of frames,
frame with ID $4$ appears in index $1$ in the table of frames, and so on.

\paragraph{Number of Lines.} The file has $3832$ lines.

\subsection*{Map Surface Form IDs From ConceptNet 4}
We can see the first 10 lines below:
\begin{code}
\$ head -n 10 inputFiles/surfaceForms/MapSurfaceFormIDsFromConceptNet4.txt \\
-1\\
-1\\
-1\\
-1\\
-1\\
0\\
1\\
2\\
3\\
4\\
\$ 
\end{code}
Hence, there are no surface forms with IDs $0, 1, \ldots, 3$, and $4$ in the \conceptnet database
when we restrict the search in the set $\{0, 1, \ldots, 9\}$.
On the other hand surface form with ID $5$ appears in index $0$ of the table of surface forms,
surface form with ID $6$ appears in index $1$ in the table of surface forms, and so on.

\paragraph{Number of Lines.} The file has $580315$ lines.

\subsection*{Map Raw Assertion IDs From ConceptNet 4}
We can see the first 10 lines below:
\begin{code}
\$ head -n 10 inputFiles/rawAssertions/MapRawAssertionIDsFromConceptNet4.txt \\
-1\\
-1\\
-1\\
0\\
1\\
-1\\
-1\\
2\\
3\\
4\\
\$
\end{code}
Hence, there are no raw assertions with IDs $0, 1, 2, 5$, and $6$ in the \conceptnet database
when we restrict the search in the set $\{0, 1, \ldots, 9\}$.
On the other hand raw assertion with ID $3$ appears in index $0$ of the table of raw assertions,
raw assertion with ID $4$ appears in index $1$ in the table of raw assertions, and so on.

\paragraph{Number of Lines.} The file has $1277257$ lines.

\subsection*{Map Sentence IDs From ConceptNet 4}
We can see the last 10 lines below:
\begin{code}
\$ tail -n 10 inputFiles/sentences/MapSentenceIDsFromConceptNet4.txt \\
525164\\
525165\\
525166\\
525167\\
525168\\
525169\\
-2\\
-2\\
-2\\
525170\\
\$ 
\end{code}
Hence, \emph{there are} sentences with IDs $2608287, 2608288, 2608289$ in the \conceptnet database
when we restrict the search in the set $\{2608281, 2608282, \ldots, 2608290\}$.
However, it turns out that these sentences are \emph{not} referenced by any raw assertion which appears 
in the database. Recall that we include only those raw assertions that appear as \emph{best}
raw assertions for at least one assertion in the English language.

On the other hand sentences with IDs $2608281, 2608282, \ldots, 2608286$ appear in 
indices $525164, 525165, \ldots, 525169$ of the table of sentences.
Moreover, sentence with ID $2608290$ is mapped to index $525170$ in the table of sentences.

\paragraph{Number of Lines.} The file has $2608291$ lines.

\section{Mapping To ConceptNet 4}\label{sec:derived:map-to}
No file is needed for this direction. Every data structure has an entry \texttt{id} that stores
the integer of the actual ID for that particular object found in \conceptnet.

\section{Lists of Edges: Directed and Undirected Multigraph}\label{sec:derived:edgelists:multigraph}
The same file is used both for the directed multigraph as well as the undirected multigraph.
Table \ref{tbl:input:edges:multigraph} shows the structure of the file.

\begin{table}[ht]
\caption{The structure of the file with the edges in the case of the directed and undirected multigraph.}\label{tbl:input:edges:multigraph}
\begin{center}
\begin{tabular}{|c|c||c|}\hline
concept 1 & concept 2 & assertion \\
index     & index     & index     \\\hline
\end{tabular}
\end{center}
\end{table}

The first ten lines of the file are shown below.
\begin{code}
\$ head -n 10 inputFiles/edges/ConceptNet4EdgesDM.txt \\
0 1 0\\
2 3 1\\
7 3 2\\
8 9 3\\
10 3 4\\
11 203359 5\\
12 13 6\\
100569 15 7\\
46006 20 8\\
22 203360 9\\
\$
\end{code}
Hence, the first edge is between concepts with indices (\emph{not IDs}) $0$ and $1$ and the assertion
justifying that edge has index $0$ (the very first one). The second edge is between concepts
with indices (\emph{not IDs}) $2$ and $3$ and the assertion justifying that edge has index $1$,
and so on for the rest.

\paragraph{Number of Lines.} The file has $566094$ lines.

\section{Lists of Edges: Directed Graph}\label{sec:derived:edgelists:directed-graph}
Table \ref{tbl:input:edges:directed-graph} shows the structure 
of the file with the edges of the directed graph.

\begin{table}[ht]
\caption{The structure of the file with the edges in the case of the directed graph.}\label{tbl:input:edges:directed-graph}
\begin{center}
\begin{tabular}{|c|c||c||c|c|c|c|c|}\hline
concept 1 & concept 2 & number of  & assertion 1 & assertion 2 & \multirow{2}{*}{$\cdots$} & assertion N \\
index     & index     & assertions & index       & index       &                           &  index     \\\hline
\end{tabular}
\end{center}
\end{table}

The first ten lines of the file are shown below.
\begin{code}
\$ head -n 10 inputFiles/edges/ConceptNet4EdgesDG.txt \\
0 0 2 102882 691\\
0 1 1 0\\
0 3 1 176972\\
0 4 1 14259\\
0 6 1 42755\\
0 7 1 31529\\
0 11 1 29344\\
0 13 1 161947\\
0 14 1 144144\\
0 15 1 35915\\
\$
\end{code}
Hence, the first edge is a self loop for the concept with index (\emph{not ID!}) $0$ and there are two 
assertions justifying that loop; those with indices $102882$ and $691$.
The second edge is an edge between the concepts with indices (\emph{again, not IDs!}) $0$ and $1$,
and there is one assertion justifying that edge which has index $0$. Similarly for the rest.

\paragraph{Number of Lines.} The file has $478929$ lines.

\section{Lists of Edges: Undirected Graph}\label{sec:derived:edgelists:undirected-graph}
Table \ref{tbl:input:edges:undirected-graph} shows the structure 
of the file with the edges of the directed graph.

\begin{table}[ht]
\caption{The structure of the file with the edges in the case of the undirected graph.}\label{tbl:input:edges:undirected-graph}
\begin{center}
\begin{tabular}{|c|c||c||c|c|c|c|c|}\hline
concept 1 & concept 2 & number of  & assertion 1 & assertion 2 & \multirow{2}{*}{$\cdots$} & assertion N \\
index     & index     & assertions & index       & index       &                           &  index     \\\hline
\end{tabular}
\end{center}
\end{table}

The first ten lines of the file are shown below.
\begin{code}
\$ head -n 10 inputFiles/edges/ConceptNet4EdgesUG.txt \\
0 0 2 102882 691\\
0 1 1 0\\
0 3 1 176972\\
0 4 5 14259 338737 174978 192462 156888\\
0 6 1 42755\\
0 7 1 31529\\
0 11 1 29344\\
0 13 1 161947\\
0 14 1 144144\\
0 15 1 35915\\
\$
\end{code}
Hence, the first edge is a self loop for the concept with index (\emph{not ID!}) $0$ and there are two 
assertions justifying that loop; those with indices $102882$ and $691$.
The second edge is an edge between the concepts with indices (\emph{again, not IDs!}) $0$ and $1$,
and there is one assertion justifying that edge which has index $0$. Similarly for the rest.
Note that this time the fourth edge (i.e.~has index $3$) 
between concepts with indices $0$ and $4$ 
is justified by $5$ assertions as opposed to the $1$ assertion shown in Table \ref{tbl:input:edges:directed-graph}.
The reason is that for the undirected edges we can not distinguish between the source and the destination of the edge
as there is no orientation on the edge. Hence, when we write down the undirected edges, the convention that we follow
is that we write first the node (concept) with smallest index, and second the node (concept) with largest index.

\paragraph{Number of Lines.} The file has $465072$ lines.

\chapter{Directory Structure, Timestamps and File Sizes}
Below we see the directory structure, creation time and date for each file, and finally the size of each file.
\begin{verbatim}
$ ls -lR inputFiles/
total 0
drwxr-xr-x  2 user  staff  136 Oct 24 12:13 assertions/
drwxr-xr-x  2 user  staff  136 Oct 24 12:13 concepts/
drwxr-xr-x  2 user  staff  170 Oct 24 12:13 edges/
drwxr-xr-x  2 user  staff  136 Oct 24 12:13 frames/
drwxr-xr-x  2 user  staff  136 Oct 24 12:13 frequencies/
drwxr-xr-x  2 user  staff  136 Oct 24 12:13 rawAssertions/
drwxr-xr-x  2 user  staff  136 Oct 24 12:13 relations/
drwxr-xr-x  2 user  staff  136 Oct 24 12:13 sentences/
drwxr-xr-x  2 user  staff  136 Oct 24 12:13 surfaceForms/

inputFiles//assertions:
total 69744
-rw-r--r--  1 user  staff  30858317 Oct 24 12:13 ConceptNet4Assertions.txt
-rw-r--r--  1 user  staff   4849324 Oct 24 12:13 MapAssertionIDsFromConceptNet4.txt

inputFiles//concepts:
total 17048
-rw-r--r--  1 user  staff  6270698 Oct 24 12:13 ConceptNet4Concepts.txt
-rw-r--r--  1 user  staff  2456782 Oct 24 12:13 MapConceptIDsFromConceptNet4.txt

inputFiles//edges:
total 59432
-rw-r--r--  1 user  staff  10200385 Oct 24 12:13 ConceptNet4EdgesDG.txt
-rw-r--r--  1 user  staff  10188521 Oct 24 12:13 ConceptNet4EdgesDM.txt
-rw-r--r--  1 user  staff  10031293 Oct 24 12:13 ConceptNet4EdgesUG.txt

inputFiles//frames:
total 200
-rw-r--r--  1 user  staff  85312 Oct 24 12:13 ConceptNet4Frames.txt
-rw-r--r--  1 user  staff  15892 Oct 24 12:13 MapFrameIDsFromConceptNet4.txt

inputFiles//frequencies:
total 24
-rw-r--r--  1 user  staff   155 Oct 24 12:13 ConceptNet4Frequencies.txt
-rw-r--r--  1 user  staff  4202 Oct 24 12:13 MapFrequencyIDsFromConceptNet4.txt

inputFiles//rawAssertions:
total 50208
-rw-r--r--  1 user  staff  19879266 Oct 24 12:13 ConceptNet4RawAssertions.txt
-rw-r--r--  1 user  staff   5821381 Oct 24 12:13 MapRawAssertionIDsFromConceptNet4.txt

inputFiles//relations:
total 16
-rw-r--r--  1 user  staff  1028 Oct 24 12:13 ConceptNet4Relations.txt
-rw-r--r--  1 user  staff    86 Oct 24 12:13 MapRelationIDsFromConceptNet4.txt

inputFiles//sentences:
total 69992
-rw-r--r--  1 user  staff  26015887 Oct 24 12:13 ConceptNet4Sentences.txt
-rw-r--r--  1 user  staff   9814447 Oct 24 12:13 MapSentenceIDsFromConceptNet4.txt

inputFiles//surfaceForms:
total 29112
-rw-r--r--  1 user  staff  11768624 Oct 24 12:13 ConceptNet4SurfaceForms.txt
-rw-r--r--  1 user  staff   3132195 Oct 24 12:13 MapSurfaceFormIDsFromConceptNet4.txt
$
\end{verbatim}

\chapter{Further Issues with the Database}
In this appendix we describe further issues that we have observed on \conceptnet
but did not appear during the derivation process of the input files.

\section{num\_assertions on conceptnet\_concept}
In theory the entries found in that column of the table conceptnet\_concept could
be used in order to calculate the degree of the node (concept) in the induced
directed multigraph. However, this is not the case.

Let us take the very first concept that has ID equal to $5$ and first of all ignore the scores.
The first line below indicates that concept with ID $5$ does not appear among assertions
that are not in the English language. This way we do not have to restrict the language 
being English in further SQL queries.
\begin{code}
sqlite> select count(*) from conceptnet\_assertion where\\
\phantom{sql}...> ((concept1\_id = 5) or (concept2\_id = 5)) and (language\_id is not 'en');\\
0\\
sqlite> select count(id) from conceptnet\_assertion where (concept1\_id = 5);\\
2816\\
sqlite> select count(id) from conceptnet\_assertion where (concept2\_id = 5);\\
147\\
sqlite> select count(id) from conceptnet\_assertion where\\
\phantom{sql}...> (concept1\_id = 5) and (concept2\_id = 5);\\
2\\
sqlite> select count(id) from conceptnet\_assertion where\\
\phantom{sql}...> (concept1\_id = 5) or (concept2\_id = 5);\\
2961
\end{code}

Note that $2961 = 2816 + 147 - 2$. However, all these numbers are different from $2887$ which is the entry in the
num\_assertions column of the table conceptnet\_concept. When we restrict to scores being
positive, we still can not justify the numbers.
\begin{code}
sqlite> select count(id) from conceptnet\_assertion where\\
\phantom{sql}...> (concept1\_id = 5) and (score > 0);\\
2754\\
sqlite> select count(id) from conceptnet\_assertion where\\
\phantom{sql}...> (concept2\_id = 5) and (score > 0);\\
139\\
sqlite> select count(id) from conceptnet\_assertion where\\
\phantom{sql}...> (concept1\_id = 5) and (concept2\_id = 5) and (score > 0);\\
2\\
sqlite> select count(id) from conceptnet\_assertion where\\
\phantom{sql}...> ((concept1\_id = 5) or (concept2\_id = 5)) and (score > 0);\\
2891
\end{code}

\end{document}